\documentclass[11pt]{article}
\pdfoutput=1

\usepackage[margin=1in]{geometry}

\usepackage[]{amsmath,amssymb,epsfig}
\usepackage{amsthm}
\usepackage{amsmath}
\usepackage{amssymb}
\usepackage{graphicx}
\usepackage{epstopdf}
\usepackage{comment}
\usepackage{array}
\usepackage{algorithm}
\usepackage{url}

\usepackage{lscape}
\usepackage{algpseudocode}
\usepackage{setspace}
\usepackage{multicol}
\usepackage{multirow}
\usepackage{color}
\usepackage{colortbl}
\usepackage{xcolor}

\usepackage{placeins}

\usepackage[%
    font={small,sf},
    labelfont=bf,
    format=hang,    
    format=plain,
    margin=0pt,
    width=0.8\textwidth,
]{caption}
\usepackage[list=true]{subcaption}

\newtheorem{theorem}{Theorem}

\newtheorem{corollary}{Corollary}

\newtheorem{lemma}{Lemma}

\newtheorem{proposition}{Proposition}
\newtheorem{remark}{Remark}

\numberwithin{equation}{section}

\newenvironment{rcases}
  {\left.\begin{aligned}}
  {\end{aligned}\right\rbrace}
	
	\renewcommand{\thefootnote}{\arabic{footnote}}

\newcommand{\calC}{\ensuremath{\mathcal{C}}}

\newcommand{\calG}{\ensuremath{\mathcal{G}}}

\newcommand{\calS}{\ensuremath{\mathcal{S}}}

\newcommand{\calN}{\ensuremath{\mathcal{N}}}

\DeclareMathOperator{\tr}{tr}

\newcommand{\norm}[1]{\parallel{#1}\parallel}
\newcommand{\abs}[1]{|{#1}|}

\newcommand{\set}[1]{\left\{{#1}\right\}}
\newcommand{\dotprod}[2]{\langle#1,#2\rangle}
\newcommand{\est}[1]{\widehat{#1}}
\newcommand{\expec}{\ensuremath{\mathbb{E}}}
\newcommand{\matR}{\ensuremath{\mathbb{R}}}

\newcommand{\argmin}[1]{\underset{#1}{\operatorname{argmin}}}

\newcommand{\prob}{\ensuremath{\mathbb{P}}}

\newcommand{\vecx}{\mathbf x}

\newcommand{\vecq}{\mathbf q}
\newcommand{\veci}{\mathbf i}
\newcommand{\vecj}{\mathbf j}

\newcommand{\vecv}{\mathbf v}

\newcommand{\vecb}{\mathbf b}

\newcommand{\vecy}{\mathbf y}
\newcommand{\vecf}{\mathbf f}
\newcommand{\vecz}{\mathbf z}
\newcommand{\veca}{\mathbf a}
\newcommand{\vecg}{\mathbf g}

\newcommand{\matA}{\ensuremath{\mathbf{A}}}

\newcommand{\matI}{\ensuremath{I}} 

\newcommand{\matP}{\ensuremath{\mathbf{P}}}

\newcommand{\real}{\text{Re}}
\newcommand{\imag}{\text{Im}}
\newcommand{\vecgbar}{\bar{\mathbf g}}
\newcommand{\vecgbarest}{\est{\vecgbar}}
\newcommand{\Hbar}{H}
\newcommand{\veczbar}{\bar{\mathbf z}}

\newcommand{\vecgest}{\est{\vecg}}
\newcommand{\gest}{\est{g}}
\newcommand{\qtathresh}{\zeta}
\newcommand{\sign}{\text{sign}}

\newcommand{\htil}{h}

\newcommand{\vechtil}{\mathbf{h}}
\newcommand{\vechtilbar}{\bar{\mathbf{h}}}

\newcommand{\mean}{{\bar{\boldsymbol{\mu}}}}

\newcommand{\trseq}{\ensuremath{TSR_{=}}}
\newcommand{\trsineq}{\ensuremath{TSR_{\leq}}}
\newcommand{\mb}[1]{\mbox{\boldmath$#1$}} 
  
\usepackage{rotating}

\usepackage{caption}

 \usepackage{float}

\usepackage{ifthen}
\newboolean{ShowNoiselessMultivar}
\setboolean{ShowNoiselessMultivar}{false}

\newcommand\mydef{\mathrel{\overset{\makebox[0pt]{\mbox{\normalfont\tiny\sffamily def}}}{=}}}

\newboolean{show_Comments}
\setboolean{show_Comments}{false}

\begin{document}
 
%
\title{Provably robust estimation of modulo $1$ samples of a smooth function with applications to phase unwrapping}
%
\author{Mihai~Cucuringu\footnotemark[1] \footnotemark[3] , Hemant Tyagi \footnotemark[2]  }

\renewcommand{\thefootnote}{\fnsymbol{footnote}}
\footnotetext[1]{Department of Statistics and Mathematical Institute, University of Oxford, Oxford, UK. Email: mihai.cucuringu@stats.ox.ac.uk}
\footnotetext[2]{INRIA Lille - Nord Europe, Lille, France.  Email: hemant.tyagi@inria.fr}
\footnotetext[3]{Alan Turing Institute, London, UK.}
\renewcommand{\thefootnote}{\arabic{footnote}}


\maketitle

\begin{abstract}
Consider an unknown smooth function  $f: [0,1]^d \rightarrow \matR$, and assume we are given 
$n$ noisy mod 1 samples of $f$, i.e., $y_i = (f(x_i) + \eta_i)\mod 1$, for  $x_i \in [0,1]^d$,  
where $\eta_i$ denotes the noise. Given the samples $(x_i,y_i)_{i=1}^{n}$, our goal is to recover smooth, robust estimates of the clean  samples $f(x_i) \bmod 1$. We formulate a natural approach for solving this problem, which works with angular embeddings of
the noisy  mod 1 samples over the unit circle, inspired by the angular synchronization framework. 
This amounts to solving a smoothness regularized least-squares problem -- a quadratically constrained quadratic program (QCQP) -- where the 
variables are constrained to lie on the unit circle. Our proposed approach is based on solving its relaxation, which is a  \emph{trust-region sub-problem} and hence solvable efficiently. We provide theoretical guarantees demonstrating its robustness to noise for adversarial, as well as random Gaussian and Bernoulli noise models. To the best of our knowledge, these are the first such theoretical results for this problem.
We demonstrate the robustness and efficiency
 of our  proposed approach via extensive numerical simulations on synthetic data, along with a simple least-squares based solution for the unwrapping stage, that recovers  the original samples of $f$ (up to a global shift). It is shown to perform well at high levels of noise, when taking as input the denoised modulo $1$ samples.

Finally, we also consider two other approaches for denoising the modulo 1 samples that leverage tools from Riemannian optimization on manifolds, including a Burer-Monteiro approach for a semidefinite programming relaxation of our formulation. For the  two-dimensional version of the problem, which has applications in synthetic aperture radar interferometry (InSAR),  
we are able to solve instances of real-world data with a million sample points in under 10 seconds, on a personal laptop.
\end{abstract}


\tableofcontents

\section{Introduction} \label{sec:Intro}
In many domains of science and engineering, one is given access to noisy samples of a signal $f$, and the 
goal is to recover the original clean samples. The signal $f$ is typically smooth in some sense, and  one would like to have an algorithm that 
is not only robust to noise, 
but also outputs smooth estimates. 
More formally, we can think of the signal $f: \matR^d \rightarrow \matR$ 
for which we are typically given samples $f(\vecx_i); \ i=1,\dots,n$ for $\vecx_i \in \mathcal{U}$ for a compact $\mathcal U \subset \matR^d$. Perhaps one of the most important applications of this problem arises  in image denoising, with a rich body of literature (see for eg., \cite{elad2006image, lou2010imageBertozzi}).

Interestingly, there are applications where one does not observe the samples directly, but only the 
\emph{modulo} samples of $f$, i.e., $f(\vecx_i) \bmod \zeta$ for some $\zeta \in \matR^{+}$. For instance, when $\zeta = 1$, 
 $f(\vecx) \bmod 1$ simply corresponds to the fractional part of $f(\vecx)$. Such measurements are typically obtained due to 
constraints on the sampling hardware, or due to physical constraints imposed by the specific nature of the  problem. 
To see this, we discuss two important applications below.  

\paragraph{Self-reset analog-to-digital converters (ADCs).} 
Traditional ADCs have voltage limits in place that cut off the signal, i.e., saturate, whenever the signal 
value lies outside the limits. Recently, a new generation of ADCs  have adopted a different approach to this problem, wherein  
they simply \emph{reset} the signal value to the other threshold value (\cite{Kester,RHEE03,kavusi04,sasa16,yamaguchi16}). 
For example, if the voltage range is $[0,b]$, the reset operation simply corresponds to a modulo $b$ operation on the signal value. 
This especially helps in working with signals whose dynamic range is much larger than what can be handled by a standard ADC.
This signal acquisition mechanism was the main motivation behind the recent work of \cite{bhandari17} (also covered in the media\footnote{\url{http://news.mit.edu/2017/ultra-high-contrast-digital-sensing-cameras-0714}}), wherein the authors 
derived conditions for exact recovery of a band limited function from its samples, in the noiseless setting \cite[Theorem 1]{bhandari17}. 
Let us note that band limitedness is implicitly a smoothness assumption on $f$, such an assumption being clearly necessary in order to be able to provably recover the original (unwrapped) samples of $f$. 

\paragraph{Phase unwrapping.} Phase unwrapping refers to the problem of recovering the original phase values of a signal $\phi$ (at different 
spatial locations), from their modulo $2\pi$ (in radians) versions. The case where $\phi :\matR^2 \rightarrow \matR$ has received considerable attention,  in particular due to important applications arising for instance in synthetic aperture radar interferometry (InSAR) (\cite{graham_insar,zebker86}), 
magnetic resonance imaging (MRI) (\cite{hedley92,laut73}), optics (\cite{Venema08}), diffraction tomography (\cite{pratt88}) and 
non destructive testing of components (\cite{paoletti94,hung96}), to name a few. Generally speaking, 
remote sensing systems typically obtain information about the structure of an object by measuring the phase coherence between the transmitted and 
scattered waveforms. For instance, in radar interferometry, information about the terrain elevation is inherently present in the phase values. 
In MRI, information regarding the velocity of blood flow or the position of veins in tissues can be obtained from the phase values.
The higher-dimensional case has received comparatively less attention. The three-dimensional version of the problem has applications 
in $3$D MRI imaging (\cite{jenkinson03}) and radar interferometry (\cite{Hooper_2007_PhaseUnwrap3D, Osmanoglu_3D_2014}).
Lastly, some papers have also formulated methods for the general $d$ dimensional case (\cite{jenkinson03,fang06}).


Our work can also be placed in the context of denoising smooth functions taking values in a nonlinear space. \cite{DonohoManifoldData} considered multiscale representations for manifold-valued data observed on equispaced grids and taking values on manifolds such as the sphere $S^2$, the special orthogonal group  $SO(3)$ or the Special Euclidean Group $SE(3)$. They proposed a method that generalizes wavelet analysis from the traditional setting where functions defined on the equispaced values in a cartesian grid no longer take simple  values such as  numerical array, but rather arrays whose entries have highly structured values obeying nonlinear constraints, and show that such representations are successful in tasks such as denoising, compression or contrast enhancement.

%
As a word of caution  to the reader, we note that this problem is different from the celebrated  \textit{phase retrieval}, a classical problem in optics that has attracted a surge of interest in recent years (see \cite{phase_retrieval_survery_candes, phase_retrieval_survery}). There, one attempts to recover an unknown signal from the magnitude (intensity)  of its Fourier transform.  Just like phase retrieval, the recovery of a function from mod 1 measurements is, by its very nature, an ill-posed problem, and one needs to incorporate prior structure on the signal. In our case is smoothness of $f$ (in an analogous way  to how enforcing sparsity renders the phase retrieval problem well-posed).  While there have  been a variety of approaches to phase retrieval, recent progress  in the compressed sensing and convex optimization-based signal processing have inspired new potential research directions.  The approach we pursue in this paper is inspired by developments in  the \textit{trust region sub-problem}  (\cite{Naka17}) and \textit{group synchronization} (\cite{sync, syncRank}) literatures.

\paragraph{Overview of approach and contributions.} 
At a high level, one would like to recover denoised samples (i.e., smooth, robust estimates) of $f$ from 
its noisy mod  1 versions. A natural approach to tackle this problem is the following two-stage method. 
In the first stage, one recovers denoised  mod 1 samples of $f$, while in the (unwrapping) second stage, 
one uses these samples to recover the original real-valued samples of $f$. 
In this paper, we mainly focus on the first stage, which is a challenging problem in itself.  To the best of our knowledge, we provide the first algorithm for denoising $\bmod 1$ samples of a function, which comes with  robustness guarantees. 
In particular, we make the following contributions\footnote{A preliminary version of this paper containing the results of Sections \ref{sec:trust_reg_relax}, \ref{sec:bound_noise_analysis} and parts of Section \ref{sec:num_exps} will appear in AISTATS 2018 (\cite{mod1hemant}). This is a significantly expanded version containing 
additional theoretical results, numerical experiments and a detailed discussion.}.

\begin{enumerate} 
\item We formulate a general framework for denoising the $\bmod 1$  samples of $f$ which involves 
mapping the noisy  mod 1 values (lying in $[0,1)$) to the angular domain (i.e. in $[0,2\pi)$)) and solving 
a smoothness regularized least-squares problem. This amounts to a quadratically constrained quadratic program (QCQP) with 
non-convex constraints wherein the variables are required to lie on the unit circle. 
We consider solving a relaxation of this QCQP which is a \emph{trust-region sub-problem}, and hence solvable efficiently.

\item We provide a detailed theoretical analysis for the above approach, proving its robustness 
to noise for different noise models, provided the noise level is not large. Specifically, for $d = 1$, we show this for arbitrary bounded 
noise (see \eqref{eq:arb_bd_noise_model},\eqref{eq:BoundedNoiseModel}, Theorem \ref{thm:arb_noise_model}), 
Bernoulli-uniform noise (see \eqref{eq:bern_unif_noise_model}, Theorem \ref{thm:bern_unif_noise_model}) 
and Gaussian noise (see \eqref{eq:gauss_noise_model}, Theorem \ref{thm:gauss_noise_model}). For the multivariate case $d \geq 1$, we show 
this for arbitrary bounded noise (see Theorem \ref{thm:arb_noise_multiv}).

\item We test the above trust-region based method on synthetic data which demonstrates that it performs well for reasonably high levels of noise.
To complete the picture, we also implement the second stage with a simple least-squares based method for recovering the (real-valued) 
samples of $f$, and show that it performs surprisingly well via extensive simulations. 

\item Finally, we also consider two other approaches for denoising the modulo 1 samples that leverage tools 
from Riemannian optimization on manifolds (\cite{absil2007trust}).
The first one is based on a semidefinite programming (SDP) relaxation of the QCQP which we solve via the Burer-Monteiro approach. The second one involves solving 
the original QCQP by optimizing over the manifold associated with the constraints. We implement both approaches    using the Manopt toolbox (\cite{manopt}), and  highlight their scalability and robustness to noise
via extensive experiments on the two-dimensional version of the problem. 
In particular, 
we are able to solve  instances of real-world problems containing a million samples in under 10 seconds on a personal laptop.
\end{enumerate}
%
%

\paragraph{Outline of paper.} Section \ref{sec:ProblemSetup} formulates the problem formally, and introduces notation for the $d=1$ case.
Section \ref{sec:AngularLS} sets up the  mod 1 denoising problem as a smoothness regularized least-squares problem 
in the angular domain which is a QCQP. 
Section \ref{sec:trust_reg_relax} describes 
its relaxation to a \textit{trust-region sub-problem}, and some simple approaches for unwrapping, i.e., recovering the samples of $f$, along 
with our complete two-stage algorithm. 
Section \ref{sec:bound_noise_analysis} contains  approximation guarantees for the arbitrary bounded noise model for the trust region based 
relaxation for recovering the denoised  $\bmod 1$ samples of $f$ (when $d=1$).  Section \ref{sec:analysis_rand_noise} contains similar results for two 
random noise models (Gaussian and Bernoulli-uniform). Section \ref{sec:multiv_analysis_bdnoise} discusses the generalization of our approach to 
the multivariate setting when $d \geq 1$, along with approximation guarantees for the bounded noise model.
Section \ref{sec:num_exps} contains experiments on synthetic data with $d=1$ for different noise models, as well a comparison with the algorithm of
\cite{bhandari17}.
In Section \ref{sec:opt_manifolds}, we describe two other approaches based on optimization on manifolds for denoising the modulo 1 samples. 
It also contains experiments involving these approaches for the $d=2$ setting, on synthetic and real-world data.
Section  \ref{sec:related_work} surveys a number of related approaches and applications, with a focus on those arising in the phase unwrapping literature.
Section \ref{sec:conclusion}, summarizes our findings, and also contains a discussion of  possible future research directions.  
Finally, the Appendix contains supplementary material related to proofs and additional numerical experiments.

\section{Problem setup} \label{sec:ProblemSetup}
We begin with the problem setup for the univariate case $d = 1$, as much of the analysis in the ensuing sections is for this 
particular setting. The multivariate case where $d \geq 1$ is treated separately in Section \ref{sec:multiv_analysis_bdnoise}.

Consider a smooth, unknown function $f : [0,1] \rightarrow \mathbb{R}$, and a uniform grid on $[0,1]$, 
\begin{equation}
0 = x_1 < x_2 < \cdots < x_n = 1 \ \text{with} \ x_i = \frac{i-1}{n-1}.
\end{equation}
%
We assume that we are given \emph{mod 1} samples of $f$ on the above grid. Note that for each  sample we can decompose the function as 
%
\begin{equation}   \label{eq:fmod1}
 f(x_i) = q_i + r_i \in \mathbb{R}, 
\end{equation}
%
%
with $q_i \in \mathbb{Z}$ and $r_i \in [0,1)$, we have 
$r_i = f(x_i) \bmod 1$. The modulus is fixed to 1 without loss of generality since 
$\frac{f \bmod s}{s} = \frac{f}{s} \bmod 1$. This is easily seen by writing $ f = sq + r $, with $q \in \mathbb{Z}$, 
and observing that $ \frac{f}{s} \bmod 1 =  \frac{sq + r}{s} \bmod 1 = \frac{r}{s} = \frac{f \bmod  s}{s}$.
In particular, we assume that the mod 1 samples are \emph{noisy}, and consider the following three noise models. 
%
%
\begin{enumerate}
\item \textbf{Arbitrary bounded noise}
%
\begin{equation} \label{eq:arb_bd_noise_model}
y_i = (f(x_i) + \delta_i) \bmod 1; \ \abs{\delta_i} \in (0,1/2); \quad i=1,\dots,n.
\end{equation}

\item \textbf{Bernoulli-Uniform noise}
\begin{equation} \label{eq:bern_unif_noise_model}
y_i = \left\{
\begin{array}{rl}
f(x_i) \mod 1 \quad ; & \text{w.p} \ 1-p \\
\sim U[0,1] \quad ; & \text{w.p} \ p
\end{array} \right. \quad \text{i.i.d}; \quad i=1,\dots,n.
\end{equation}
Hence for some parameter $p \in (0,1)$ we either observe the clean sample with probability $1-p$, 
or some garbage value generated uniformly at random in $[0,1]$.

%
\item \textbf{Gaussian noise}
%
\begin{equation} \label{eq:gauss_noise_model}
y_i = (f(x_i) + \eta_i) \mod 1; \quad i=1,\dots,n, 
\end{equation}
%
where $\eta_i \sim \mathcal{N}(0,\sigma^2)$ i.i.d.

\end{enumerate}
We will denote $f(x_i)$ by $f_i$ for convenience. Our aim is to recover smooth, robust estimates (up to a global shift) 
of the original samples $(f_i)_{i=1}^{n}$ from the measurements $(x_i, y_i)_{i=1}^{n}$. 
We will assume $f$ to be H\"older continuous meaning that for constants $M > 0$, $\alpha \in (0,1]$, 
%
%
\begin{equation} \label{eq:f_smooth_hold}
\abs{f(x) - f(y)} \leq M \abs{x - y}^{\alpha}; \quad \forall \ x,y \in [0,1].
\end{equation}  

The above assumption is quite general and reduces to Lipschitz continuity when $\alpha = 1$. 
%

\paragraph{Notation.} Scalars and matrices are denoted by lower case and upper cases symbols respectively, 
while vectors are denoted by lower bold face symbols. Sets are denoted by calligraphic symbols (eg., $\calN$), 
with the exception of $[n] = \set{1,\dots,n}$ for $n \in \mathbb{N}$. The imaginary unit is denoted by $\iota = \sqrt{-1}$.  
The  notation introduced throughout  Sections  \ref{sec:AngularLS}  and   \ref{sec:trust_reg_relax}   is summarized in Table \ref{tab:methodAbbrev}.
We will denote the $\ell_p$ ($1 \leq p \leq \infty$) norm of a vector $\vecx \in \matR^n$ by 
$\norm{\vecx}_p$ (defined as $(\sum_i \abs{x_i}^p)^{1/p}$). In particular, $\norm{\vecx}_{\infty} := \max_i \abs{x_i}$. 
For a matrix $A \in \matR^{m \times n}$, we will denote its spectral norm (i.e., largest singular value) by $\norm{A}$ 
and its Frobenius norm by $\norm{A}_F$ (defined as $(\sum_{i,j} A_{i,j}^2)^{1/2}$). For a square matrix, we denote its 
trace by $\tr(\cdot)$.


\section{Smoothness regularized least-squares in the angular domain}  \label{sec:AngularLS}
Our algorithm essentially works in two stages.
%
\begin{enumerate}
\item \textbf{Denoising stage}. Our goal here is to denoise the mod 1 samples, 
which is also the main focus of this paper. In a nutshell, we map the given noisy mod 1 samples 
to  points on the unit complex circle, 
and solve a smoothness regularized, constrained least-squares problem. 
The solution to this problem, followed by a simple post-processing step, yields the final denoised mod 1 samples  of $f$.

\item \textbf{Unwrapping stage}. The second stage takes as input the above denoised  mod 1 samples, and produces  an estimate 
of  the original real-valued samples of $f$ (up to a global shift).
\end{enumerate}
%
We start the denoising stage by mapping the  mod 1 samples to the angular domain, with 
\begin{equation} \label{eq:unit_complex_circ_rep}
\hspace{-3mm} \htil_i := \exp(2 \pi \iota f_i) = \exp(2 \pi \iota r_i), \; z_i  := \exp(2 \pi \iota y_i), 
\end{equation}
denoting the respective representations of the clean mod 1 and noisy mod 1 samples on the unit  circle in $\mathbb{C}$, where the first equality is due to the fact that $f_i = q_i + r_i$, with $q_i \in \mathbb{Z}$.    
\begin{figure}
\centering
\subcaptionbox[Short Subcaption]{ Clean $ f $ mod 1 
}[ 0.3\textwidth ]
{\includegraphics[width=0.3\textwidth] {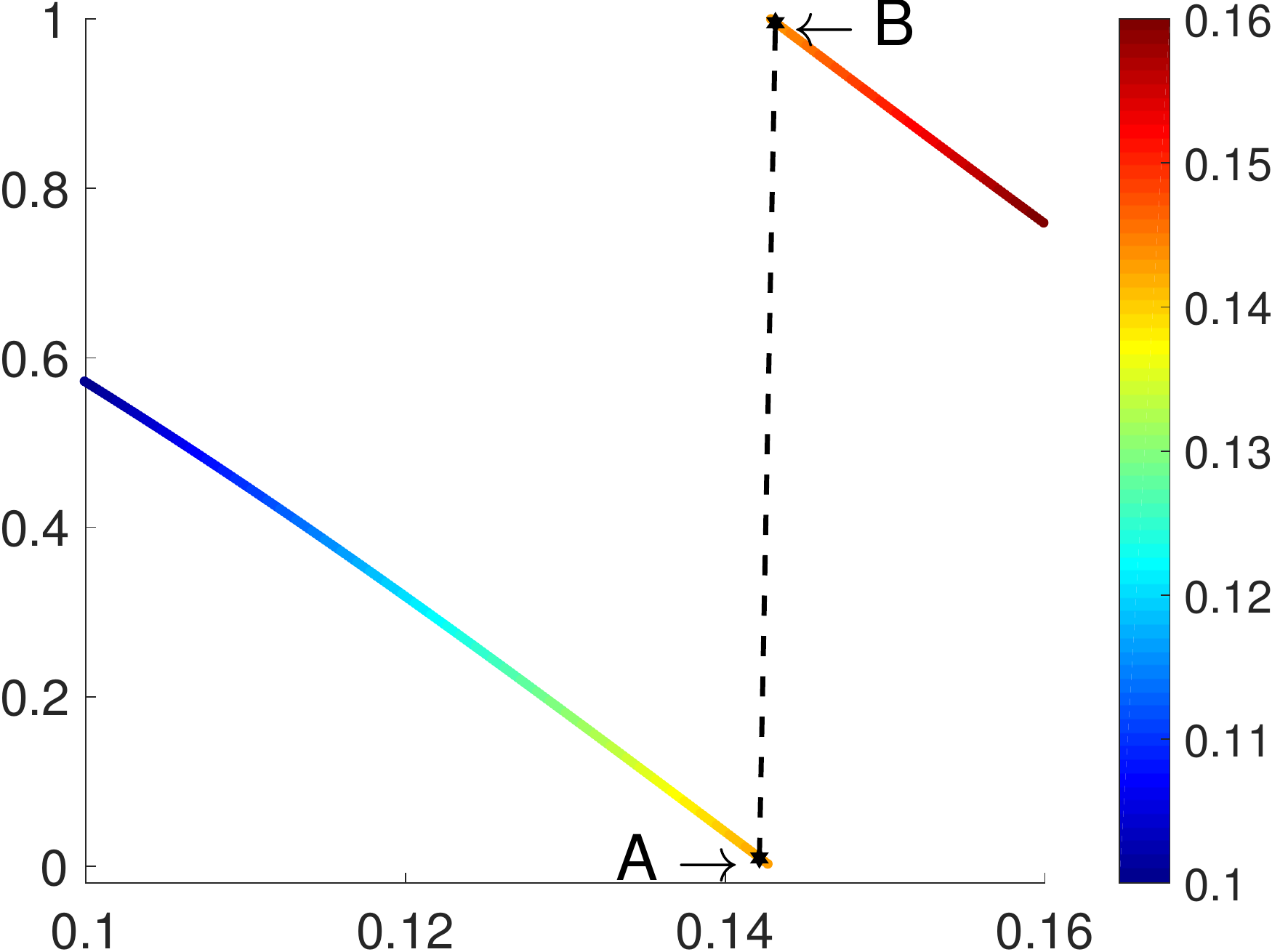} }
\subcaptionbox[Short Subcaption]{ Angular embedding 
}[ 0.3\textwidth ]
{\includegraphics[width=0.3\textwidth] {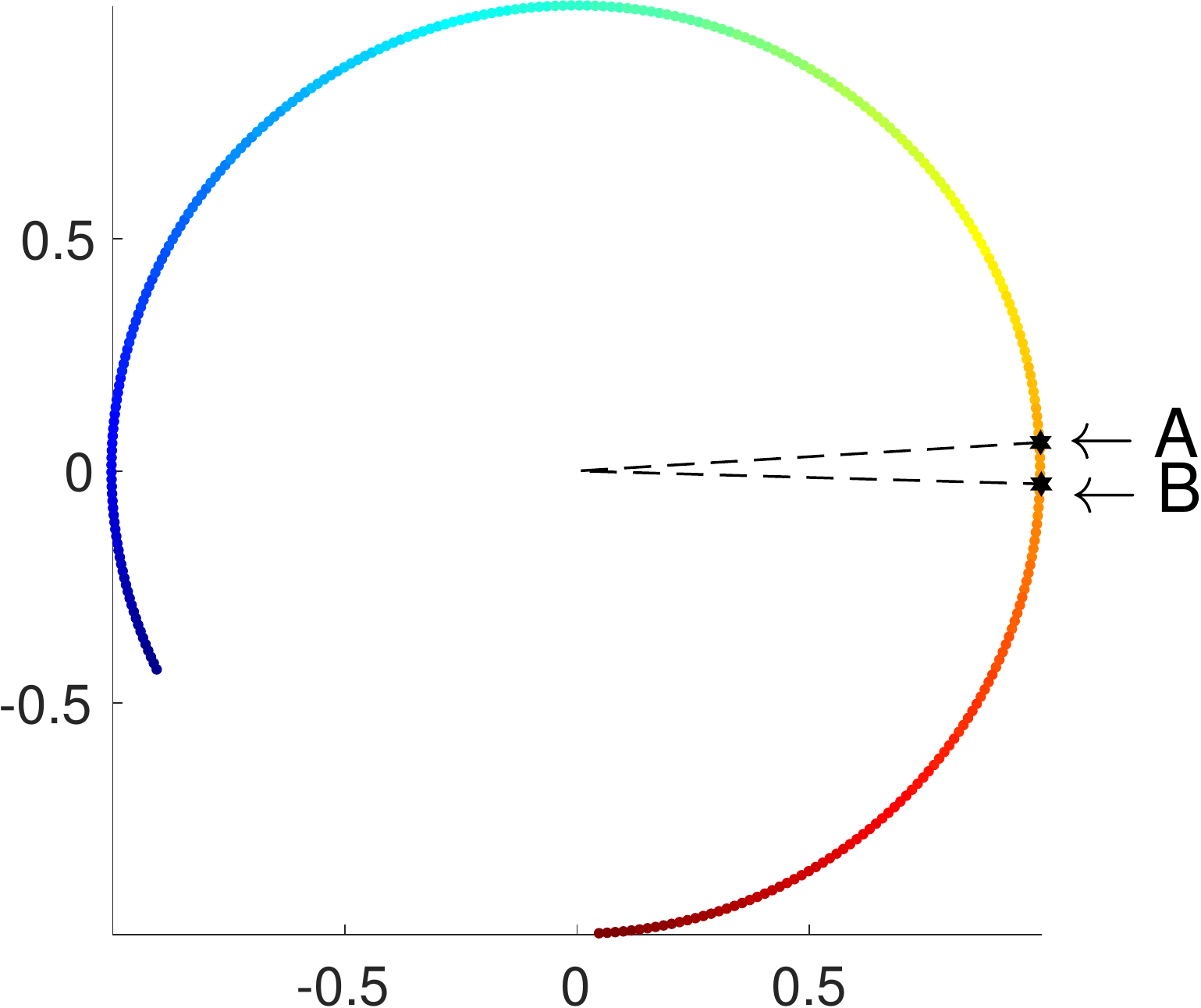} }
%
\caption[Short Caption]{Motivation for the angular embedding approach.}
\label{fig:toyExampleFarNear}
\end{figure}
%
%
%
%
The choice of representing the mod 1 samples in \eqref{eq:unit_complex_circ_rep} is very natural for the following reason. 
For points $x_i,x_j$ sufficiently close, the samples $f_i, f_j$ will also be close (by H\"older continuity of $f$). 
While the corresponding wrapped samples $f_i \mod 1, f_j \mod 1$ can still be far apart,  the 
complex numbers $\exp(\iota 2\pi f_i)$ and $\exp(\iota 2\pi f_j)$ will necessarily be close 
to each other\footnote{Indeed, $\abs{\exp(\iota 2\pi f_i) - \exp(\iota 2\pi f_j)}$ $= \abs{1-\exp(\iota 2\pi (f_j - f_i))}$ 
$= 2\abs{\sin (\pi (f_j - f_i))} \leq 2\pi\abs{f_j - f_i}$ (since $\abs{\sin x} \leq \abs{x}$ $\forall x \in \matR$).}.
This is illustrated in the toy example in Figure \ref{fig:toyExampleFarNear}.  

Figure  \ref{fig:toyExampleFarNear_Noisy} is the analogue of Figure  \ref{fig:toyExampleFarNear}, but for a noisy instance of the problem, making the point that the angular embedding  facilitates the denoising process. For points $x_i,x_j$ sufficiently close, the corresponding  samples  $f_i, f_j$ will also be close in the real domain, by H\"older continuity of $f$. When measurements get perturbed by noise, the distance in the real domain between the noisy mod 1 samples can greatly increase and become close to 1 (in this example, the  point B gets perturbed by noise, hits the ``floor'' and ``resets'' itself). However, in the angular embedding space, the two points still remain close to each other, as depicted in Figure \ref{subfig:toyExampleFarNear_Noisy_c}.

%

%
\begin{figure}
\subcaptionbox[Short Subcaption]{ Clean $ f \mod 1 $ 
}[ 0.32\textwidth ]
{\includegraphics[width=0.32\textwidth] {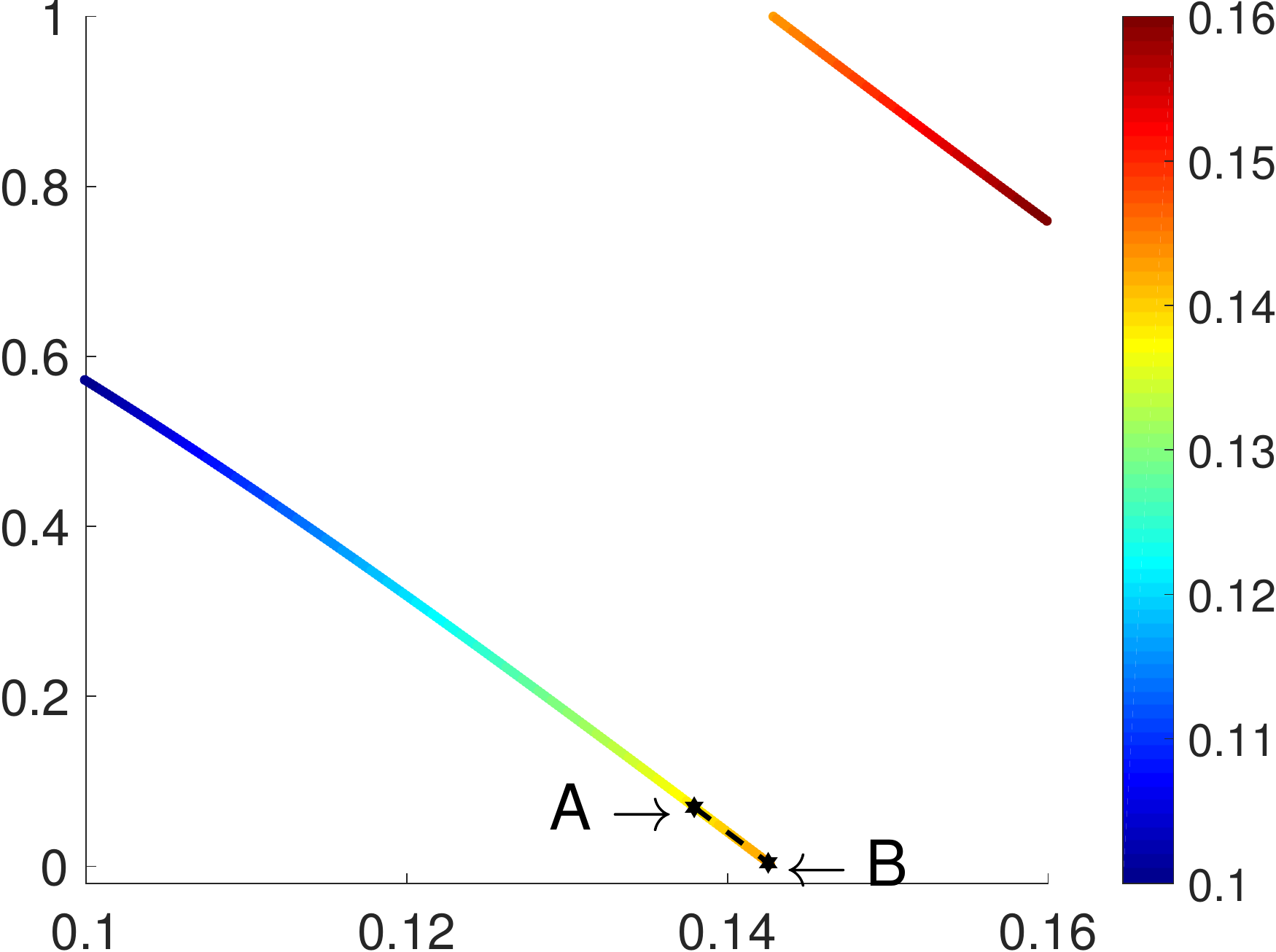} }
\hspace{0.02\textwidth} 
\subcaptionbox[Short Subcaption]{Noisy values  $ f \mod 1 $ 
}[ 0.32\textwidth ]
{\includegraphics[width=0.32\textwidth] {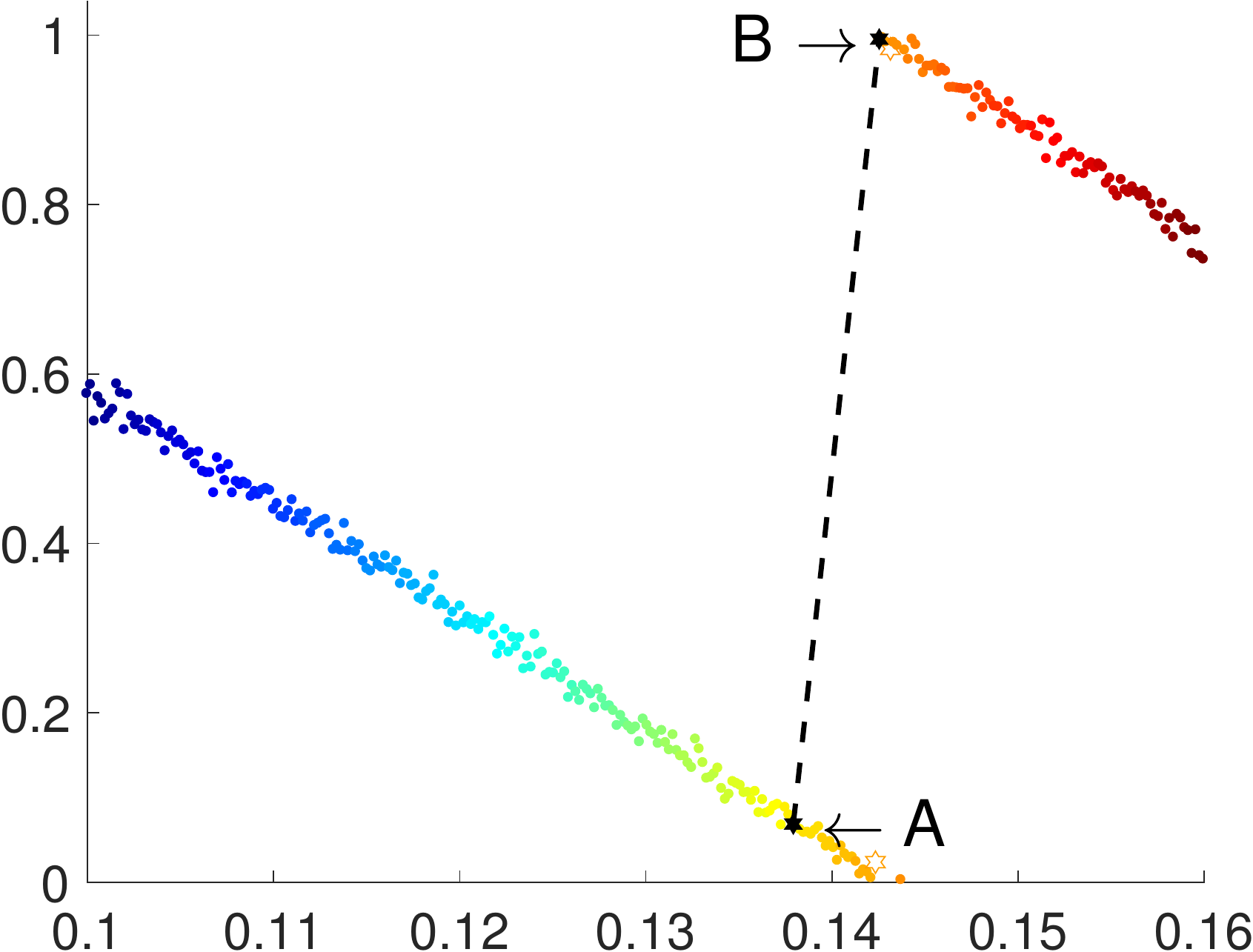} }
\subcaptionbox[Short Subcaption]{ Noisy angular embedding.
\label{subfig:toyExampleFarNear_Noisy_c}
}[ 0.28\textwidth ]
{\includegraphics[width=0.28\textwidth] {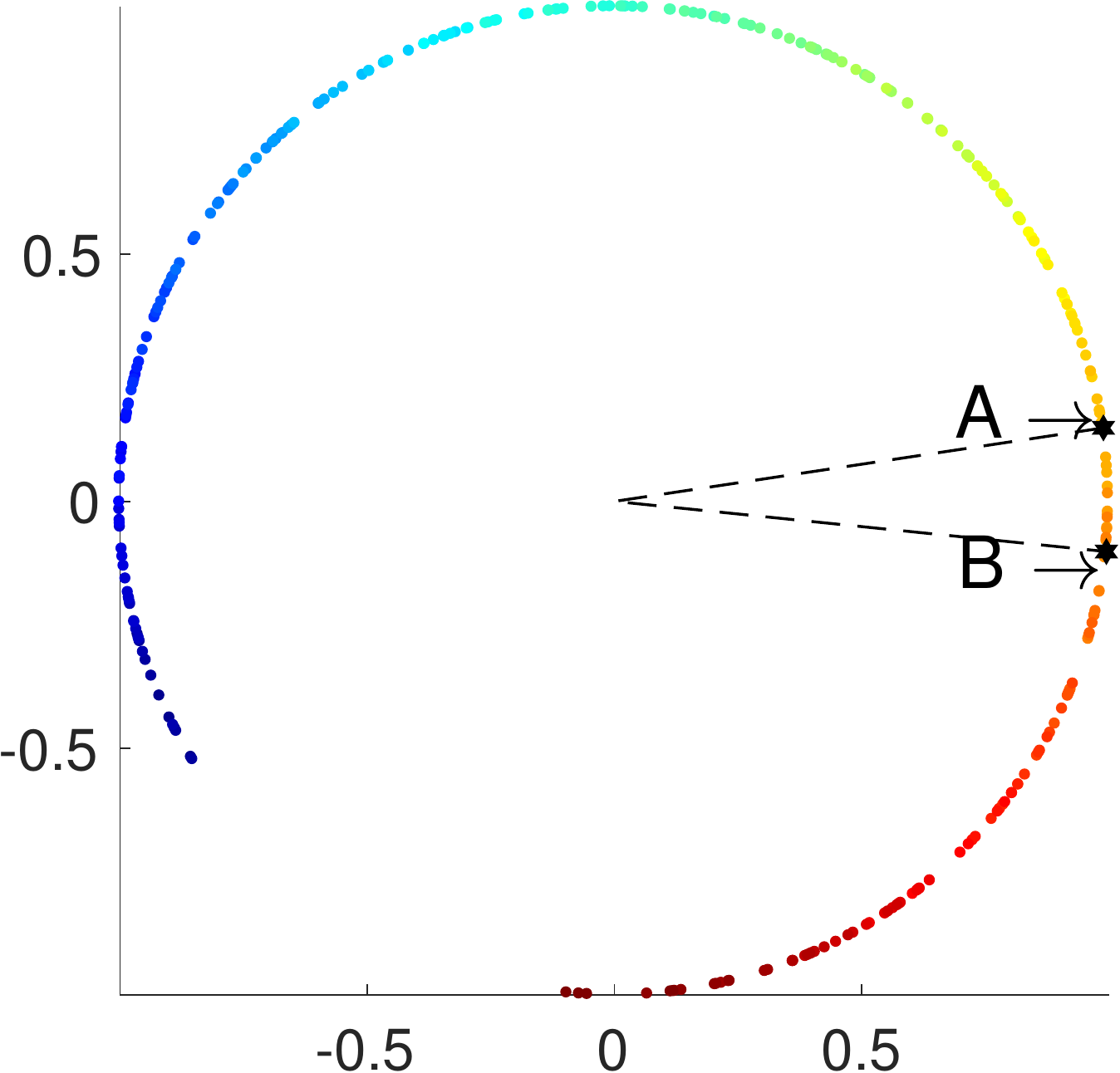} }
\hspace{0.02\textwidth} 
\vspace{-1mm} 
\captionsetup{width=0.95\linewidth}
\caption[Short Caption]{Motivation for the angular embedding approach. Noise perturbations may take nearby points (mod 1 samples) far away in the real domain, yet the points will remain close in the angular domain.}
\label{fig:toyExampleFarNear_Noisy}
\end{figure}

Consider an undirected graph $G = (V,E)$ with $V = \{1,2,\ldots,n\}$ where index $i$ corresponds to the point $x_i$ on our grid, 
and $E = \{\set{i,j}: i,j \in V, i \neq j, |i-j| \leq k\}$ denotes the set of edges for a suitable parameter $k \in \mathbb{N}$. 
A natural approach for recovering smooth estimates of $(h_i)_{i=1}^n$ would be to solve the following optimization problem
%
%
\begin{eqnarray}
\min_{ g_1,\ldots,g_n \in \mathbb{C}; |g_i|=1 } \sum_{i=1}^{n} |g_i - z_i|^2 + \lambda  \sum_{\set{i,j} \in E} |g_i - g_{j}|^2. 
\label{eq:orig_denoise_hard}
\end{eqnarray}
%
%
%
Here, $\lambda > 0$ is a regularization parameter, which along with $k$, controls the smoothness of the solution. 
We  denote by  $L \in \mathbb{R}^{n \times n}$ the Laplacian matrix associated with $G$, defined as
%
%
\begin{equation} \label{eq:laplacian_def}
L_{i,j} = \left\{
\begin{array}{rl}
\text{deg}(i), \quad   & i = j, \\
-1, \quad  & \set{i,j} \in E,   \\
0,  \quad  & \text{otherwise}.  
\end{array} \right.
\end{equation} 
%
%
%
Denoting $\vecg = [g_1 \ g_2 \ \dots \ g_n]^T \in \mathbb{C}^n$,  the second term in \eqref{eq:orig_denoise_hard} can be simplified to 
\begin{align}
\hspace{-5mm}  \lambda \left( \sum_{i \in V} \text{deg}(i) \abs{g_i}^2  -  \hspace{-2mm} \sum_{\set{i,j} \in E} (g_i g_j^{*} + g_i^{*} g_{j}) \right)  
= \lambda \vecg^{*} L \vecg. \label{eq:quadform_exp_compl}
\end{align}
%
%
%
Next, denoting $\vecz = [z_1 \ z_2 \ \dots \ z_n]^T \in \mathbb{C}^n$, we can further simplify the first 
term in \eqref{eq:orig_denoise_hard} as
%
%
\begin{align}
\sum_{i=1}^{n} |g_i - z_i|^2 
&= \sum_{i=1}^{n} (\abs{g_i}^2 + \abs{z_i}^2 - g_i z^{*}_i - g^{*}_i z_i) \\
&= 2n - 2Re(\vecg^{*}\vecz).
\end{align}
%
%
%
This gives us the following equivalent form of \eqref{eq:orig_denoise_hard}
%
%
%
\begin{eqnarray} \label{eq:orig_denoise_hard_1}
\min_{ \vecg \in \mathbb{C}^n: \abs{g_i} = 1} \lambda \vecg^{*}  L \vecg - 2Re(\vecg^{*}\vecz).
\end{eqnarray}
%
%
\begin{table}[tpb]
\begin{minipage}[b]{0.99\linewidth}
\begin{center}
\begin{tabular}{|c|c|c|}
\hline
Symbol & Description   \\
\hline
$f$ 		& unknown real-valued function    \\
$r$ 	    & clean $ f \mod 1$    \\
$q$     &  clean reminder $q = f-r $  \\
$y$ 		&  noisy $ f \mod 1$   \\
\hline
$ \vechtil $     &    clean signal in angular domain  \\
$ \vecz $ 	 	& noisy signal in angular domain    \\
$ \vecg $        &  free complex-valued variable  \\
\hline
$ \vechtilbar $     &  real-valued version of $\vechtil$   \\
$ \veczbar $    &  real-valued version of $\vecz$  \\
$\vecgbar $     &  real-valued version of  $\vecg $   \\
\hline
$L$ & $n \times n$ Laplacian matrix of graph $G$  \\
$\Hbar$ & $2n \times 2n $ block diagonal version of L \\
\hline
\end{tabular}
\end{center}
\end{minipage}
\vspace{-2mm}
\caption{Summary of frequently used symbols in the paper.}
\label{tab:methodAbbrev}
\end{table}
%
%

\section{A trust-region based relaxation for denoising modulo 1 samples} \label{sec:trust_reg_relax} 
The optimization problem in \eqref{eq:orig_denoise_hard_1} is over a non-convex set $\calC_n := \set{\vecg \in \mathbb{C}^n: \abs{g_i} = 1}$. 
In general, the problem $\min_{\vecg \in \calC_n} \vecg^{*} \matA \vecg$, where $\matA \in \mathbb{C}^{n\times n}$ is positive semidefinite, 
is known to be NP-hard \cite[Proposition 3.3]{zhang06}. Note that we can get rid of the linear term $2Re(\vecg^{*}\vecz)$ by 
rewriting \eqref{eq:orig_denoise_hard_1} as
\begin{equation} \label{eq:orig_denoise_hard_2}
\min_{\vecg \in \calC_n} \quad [\vecg^* \ 1]  \left[
  \begin{array}{cc}
 \lambda L & -\vecz \\ 
       -\vecz^* & 0
\end{array}
 \right]   \left[\begin{array}{cc} \vecg \\  1 \end{array} \right].
\end{equation}
The quadratic term in \eqref{eq:orig_denoise_hard_2} is Hermitian, and is of course structured, as 
$L$ is the Laplacian of a nearest neighbor graph. However the complexity of \eqref{eq:orig_denoise_hard_2} is still unclear. 
As pointed out by a reviewer of a preliminary version of this paper (\cite{mod1hemant}), one possible approach is to discretize the  angular domain, and solve \eqref{eq:orig_denoise_hard_1} approximately via dynamic programming. Since the graph $G$ has 
tree-width $k$, the computational cost of this approach may be\footnote{ Note that naive dynamic programming will have a running time that is exponential in $k$. However, it is unclear whether this exponential dependence on $k$ is unavoidable.} exponential in $k$. 

Our approach is to relax the constraints corresponding to $\calC_n$, to one where the points lie on a sphere of radius $n$, which amounts to the following optimization problem
%
%
\begin{eqnarray} \label{eq:qcqp_denoise_complex}
\min_{\vecg \in \mathbb{C}^n: \norm{\vecg}_2^2 = n} \lambda \vecg^{*} L \vecg  - 2 Re (\vecg^{*} \vecz).
\end{eqnarray}
%
%
It is straightforward to reformulate \eqref{eq:qcqp_denoise_complex} in terms of real variables. 
We do so by introducing the following notation 
for the real-valued versions of the variables $\vechtil$ (clean signal), $\vecz$ (noisy signal), and  $\vecg$ (free variable) 
%
%
\begin{eqnarray} \label{eq:real_notat}
\hspace{-2mm}
\vechtilbar = \begin{pmatrix}
    \real(\vechtil)    \\    \imag(\vechtil) 
   \end{pmatrix}, \    
\veczbar = \begin{pmatrix}
  \real(\vecz) \\ \imag(\vecz) 
 \end{pmatrix}, \ 
\vecgbar = \begin{pmatrix}
  \real(\vecg) \\ \imag(\vecg) 
 \end{pmatrix} \in \matR^{2n}, \;
\end{eqnarray}
and the corresponding block-diagonal Laplacian 
%
\begin{eqnarray} \label{eq:def_H}
\Hbar = \begin{pmatrix}
  \lambda L \quad & 0 \\ 0 \quad & \lambda L  
 \end{pmatrix}  =  \lambda \begin{pmatrix}
   1 & 0 \\ 
   0 & 1   
 \end{pmatrix} \otimes L
 \;\;    \in    \matR^{2n \times 2n}.
\end{eqnarray} 
%
%
In light of this, the optimization problem \eqref{eq:qcqp_denoise_complex} can be 
equivalently formulated as 
%
%
\begin{eqnarray} \label{eq:qcqp_denoise_real}
\min_{ \vecgbar \in \mathbb{R}^{2n}: \norm{\vecgbar}_2^2 = n} \vecgbar^{T} \Hbar \vecgbar  - 2 \vecgbar^{T} \veczbar,  
\end{eqnarray}
%
%
which is formally shown in the appendix for completeness. Let us note that 
the Laplacian matrix $L$ is positive semi-definite (p.s.d),  with its smallest eigenvalue $ \lambda_1(L) = 0$ with multiplicity $1$ (since $G$ is connected). 
Therefore,  $\Hbar$ is also p.s.d, with smallest eigenvalue $ \lambda_1(H) = 0$ with 
multiplicity $2$.

The optimization problem   \eqref{eq:qcqp_denoise_real} is actually an instance of the so-called trust-region sub-problem (TRS) with equality 
constraint (which we denote by $\trseq$ from now on),  where one minimizes a general quadratic function (not necessarily convex), subject to a sphere constraint. 
For completeness, we also mention the closely related trust-region sub-problem with inequality constraint (denoted by  $\trsineq$), 
where we have a $\ell_2$ ball constraint. There exist several algorithms that efficiently 
solve $\trsineq$ (\cite{Sorensen82,MoreSoren83,Rojas01,Rendl97,Gould99,Naka17}),  
and also some which explicitly solve $\trseq$ (\cite{Hager01,Naka17}). 
In particular, \cite{Naka17} recently showed that trust-region sub-problems can be solved to high accuracy via a single generalized eigenvalue problem. 
The computational complexity is $O(n^3)$ in the worst case, but improves when the matrices involved are sparse. 
In our case, the Laplacian is sparse when $k$ is not large. In the experiments, we employ their algorithm for solving \eqref{eq:qcqp_denoise_real}.
  
\begin{remark}
We remark that the $\trseq$ formulation is just one possible relaxation. One could also, for instance, consider a semidefinite programming based 
relaxation for the QCQP. We discuss this in Section \ref{sec:opt_manifolds}, and  solve it numerically  via the Burer-Monteiro approach.  
Moreover, we also consider solving the original QCQP \eqref{eq:orig_denoise_hard_1} using tools from optimization on manifolds (\cite{absil2007trust}), 
as $\calC_n$ is a manifold (\cite{AbsMahSep2008}). 
\end{remark}
%
\begin{remark}
Note that \eqref{eq:orig_denoise_hard_1} is similar to the angular synchronization problem (see for eg., \cite{sync}) as they 
both optimize a quadratic form subject to entries lying on the unit circle. The fundamental difference is that
the matrix in the quadratic term in synchronization is formed using the given noisy pairwise angle offsets (embedded on the
unit circle), and thus depends on the data. In our setup, the quadratic term is formed using the Laplacian of the smoothness
regularization graph, and thus is independent of the given data (noisy mod 1 samples). 
\end{remark}

Rather surprisingly, one can fully characterize\footnote{Discussed in detail in the appendix for completeness.} 
the solutions to both $\trseq$ and  $\trsineq$.  
The following Lemma \ref{lemma:qcqp_denoise_real} characterizes the solution for \eqref{eq:qcqp_denoise_real}; 
it follows directly from \cite[Lemma 2.4, 2.8]{Sorensen82} (see also \cite[Lemma 1]{Hager01}).
%
%
\begin{lemma} \label{lemma:qcqp_denoise_real}
$\vecgbarest$ is a solution to \eqref{eq:qcqp_denoise_real} iff $\norm{\vecgbarest}_2^2 = n$ and $\exists \mu^{*}$ such that
(a) $2\Hbar + \mu^{*}\matI \succeq 0$ and (b) $(2\Hbar + \mu^{*}\matI) \vecgbarest = 2\veczbar$. 
Moreover, if $2\Hbar + \mu^{*}\matI \succ 0$, then the solution is unique.
\end{lemma}
%
Let $\set{\lambda_j(\Hbar)}_{j=1}^{2n}$, with $\lambda_1(\Hbar) \leq \lambda_2(\Hbar) \leq \cdots  \lambda_{2n}(\Hbar) $, 
and $\set{\vecq_j}_{j=1}^{2n}$ denote the eigenvalues, respectively eigenvectors,  of  $\Hbar$. 
Note that $\lambda_1(\Hbar) = \lambda_2(\Hbar) = 0$, and $\lambda_3(\Hbar) > 0$, since $G$ is connected. 
Let us denote the null space of $\Hbar$ by $\calN(\Hbar)$, so $\calN(\Hbar) = \text{span} \set{\vecq_1,\vecq_2}$. 
Next, we analyze the solution to \eqref{eq:qcqp_denoise_real} with the help of Lemma \ref{lemma:qcqp_denoise_real}, by considering 
the following two cases. 
%

\textbf{Case 1. }  \underline{\textit{$\veczbar \not\perp \calN(\Hbar)$.}} The solution is given by 
%
\begin{equation} 
\hspace{-3mm} \vecgbarest(\mu^{*}) = 2(2\Hbar + \mu^{*}\matI)^{-1} \veczbar = 2\sum_{j=1}^{2n} \frac{\dotprod{\veczbar}{\vecq_j}}{2\lambda_j(\Hbar) + \mu^{*}}\vecq_j, 
\end{equation}
%
for a unique $\mu^{*} \in (0,\infty)$ satisfying $\norm{\vecgbarest(\mu^{*})}_2^2 = n$. Indeed, 
denoting $\phi(\mu) = \norm{\vecgbarest(\mu)}_2^2 = 4\sum_{j=1}^{2n} \frac{\dotprod{\veczbar}{\vecq_j}^2}{(2\lambda_j(\Hbar) + \mu)^2}$, 
we can see that $\phi(\mu)$ has a pole at $\mu = 0$ and decreases monotonically to $0$ as $\mu \rightarrow \infty$. Hence, there 
exists a unique $\mu^{*} \in (0,\infty)$ such that $\norm{\vecgbarest(\mu^{*})}_2^2 = n$. The solution $\vecgbarest(\mu^{*})$ will 
be unique by Lemma \ref{lemma:qcqp_denoise_real},  since $2\Hbar + \mu^{*}\matI \succ 0$ holds.

\textbf{Case 2. } \underline{\textit{$\veczbar \perp \calN(\Hbar)$.}} This second scenario requires additional attention. To begin with, note that   
%
\begin{equation}
\phi(0) = 4\sum_{j=1}^{2n} \frac{\dotprod{\veczbar}{\vecq_j}^2}{(2\lambda_j(\Hbar))^2} = \sum_{j=3}^{2n} \frac{\dotprod{\veczbar}{\vecq_j}^2}{\lambda_j(\Hbar)^2}
\end{equation}
%
is now well defined, i.e., $0$ is not a pole of $\phi(\mu)$ anymore. 
If $\phi(0) > n$, then as before, we can again find a unique $\mu^{*} \in (0,\infty)$ satisfying $\phi(\mu^{*}) = n$. 
The solution is given by $\vecgbarest(\mu^{*}) = 2(2\Hbar + \mu^{*}\matI)^{-1} \veczbar$ and is unique since  $2\Hbar + \mu^{*}\matI \succ 0$ (by Lemma \ref{lemma:qcqp_denoise_real}). 

In case $\phi(0) \leq n$, we set $\mu^{*} = 0$ and define our solution to be of the form 
\begin{equation}
\vecgbarest(\theta,\vecv) = (\Hbar)^{\dagger} \veczbar + \theta \vecv; \quad \vecv \in \calN(\Hbar), \norm{\vecv}_2 = 1,
\end{equation}
where $\dagger$ denotes pseudo-inverse and $\theta \in \matR$. In particular, for any given $\vecv \in \calN(\Hbar), \norm{\vecv}_2 = 1$, 
we obtain $\vecgbarest(\theta^{*},\vecv), \vecgbarest(-\theta^{*},\vecv)$ 
as the solutions to \eqref{eq:qcqp_denoise_real}, with $\pm \theta^{*}$ being the solutions to the equation

\begin{equation}
\norm{\vecgbarest(\theta,\vecv)}_2^2 = n \Leftrightarrow \underbrace{\norm{(\Hbar)^{\dagger} \veczbar}_2^2}_{= \phi(0) \leq n} + \theta^2 = n
\end{equation}

Hence the solution is not unique if $\phi(0) < n$. 
%
%
%

\subsection{Recovering the denoised mod 1 samples} \label{subsec:rec_denoised_mod1} 
%
The solution to \eqref{eq:qcqp_denoise_real} is 
a vector $\vecgbarest \in \matR^{2n}$. Let $\vecgest \in \mathbb{C}^n$ be the complex representation of $\vecgbarest$ 
as per \eqref{eq:real_notat}, so that $\vecgbarest = [\real(\vecgest)^T \ \imag(\vecgest)^T]^T$. Denoting $\gest_i \in \mathbb{C}$ to 
be the $i^{th}$ component of $\vecgest$, note that $\abs{\gest_i}$ is not necessarily equal to one. On the other hand, recall that 
$h_i  =\exp(\iota 2\pi f_i \bmod 1),$ $\forall i=i,\dots,n$ for the ground truth $\vechtil \in \mathbb{C}^n$. 
We obtain our final estimate $\widehat{f_i} \bmod 1$ to $f_i \bmod 1$ by projecting $\gest_i$ onto 
the unit complex disk 
%
%
\begin{equation} \label{eq:extr_mod1_vals}
\exp(\iota 2\pi (\widehat{f_i} \bmod 1)) = \frac{\gest_i}{\abs{\gest_i}}; \quad i=1,\dots,n.
\end{equation}
%
%
In order to measure the distance between $\widehat{f_i} \bmod 1$ and $f_i \bmod 1$, we will use 
the so-called  \emph{wrap-around} distance on $[0,1]$ denoted by $d_w : [0,1]^2 \rightarrow [0,1/2]$, where
%
\begin{align}
d_w(t_1,t_2) := \min\set{\abs{t_1-t_2}, 1-\abs{t_1-t_2}},
\end{align}
%
for $t_1,t_2 \in [0,1]$. We will now show that if $\gest_i$ is sufficiently close to $h_i$ for each $i=1,\dots,n$, then 
each $d_w(\widehat{f_i} \bmod 1,f_i\bmod 1)$ will be correspondingly small. This is stated precisely in the 
following lemma.
%
\begin{lemma} \label{lem:wrap_dist_fin_bd}
For $0 < \epsilon < 1/2$, let $\abs{\gest_i - h_i} \leq \epsilon$ hold for each $i=1,\dots,n$. Then,  for each $i=1,\dots,n$
%
\begin{equation} \label{eq:wrap_dist_fin_bd}
d_w(\widehat{f_i} \bmod 1,f_i\bmod 1) \leq \frac{1}{\pi} \sin^{-1}\left(\frac{\epsilon}{1-\epsilon}\right).
\end{equation}
\end{lemma}
%
%
\begin{proof}
To begin with, note that $\abs{\gest_i - h_i} \leq \epsilon$  implies $\abs{\gest_i} \in [1-\epsilon,1+\epsilon]$. This means that  
$\abs{\gest_i} > 0$ holds if $\epsilon < 1$. Consequently, we obtain
\begin{align}
\Bigl| \frac{\gest_i}{\abs{\gest_i}} - h_i\Bigr| 
&= \Bigl| \frac{\gest_i}{\abs{\gest_i}} - \frac{h_i}{\abs{\gest_i}} + \frac{h_i}{\abs{\gest_i}} - h_i\Bigr| \\
&\leq \frac{\abs{\gest_i - h_i}}{\abs{\gest_i}} + \abs{h_i}\left(\frac{\abs{\abs{\gest_i} - 1}}{\abs{\gest_i}}\right) \\
&\leq \frac{2\epsilon}{\abs{\gest_i}} \leq \frac{2\epsilon}{1-\epsilon}. \label{eq:denoise_mod1_rec_temp1}
\end{align}  
We will now show that provided $0 < \epsilon < 1/2$ holds, then \eqref{eq:denoise_mod1_rec_temp1} implies the bound \eqref{eq:wrap_dist_fin_bd}.
Indeed, from the definition of $h_i$, and of $\gest_i/\abs{\gest_i}$ (from \eqref{eq:extr_mod1_vals}), we have 
\begin{align}
\Bigl| \frac{\gest_i}{\abs{\gest_i}} - h_i\Bigr| 
&= \abs{\exp(\iota 2\pi (\widehat{f_i} \bmod 1)) - \exp(\iota 2\pi (f_i \bmod 1))} \\
&= \abs{1 - \exp(\iota 2\pi (f_i \bmod 1 - \widehat{f_i} \bmod 1))} \\
&= 2\abs{\sin [\pi \underbrace{(f_i \bmod 1 - \widehat{f_i} \bmod 1)}_{\in (-1,1)}]} \\
&= 2\sin [\pi \abs{(f_i \bmod 1 - \widehat{f_i} \bmod 1)}] \\ 
&= 2\sin [\pi (1- \abs{(f_i \bmod 1 - \widehat{f_i} \bmod 1)})]. \label{eq:denoise_mod1_rec_temp2}
\end{align}
Then, \eqref{eq:wrap_dist_fin_bd} follows from \eqref{eq:denoise_mod1_rec_temp2}, \eqref{eq:denoise_mod1_rec_temp1} 
and by noting that $0 < \epsilon/(1-\epsilon) < 1$ for $0 < \epsilon < 1/2$.
\end{proof}
%
%

\subsection{Unwrapping stage and main algorithm} \label{subsec:unwrap_stage_and_algo} 

Having recovered the denoised mod 1 samples $\widehat{f_i} \bmod 1$ for $i=1,\dots,n$, 
we now move onto the next stage of our method where the goal is to recover the samples $f$, for which we discuss two possible approaches. 
%
%
%

\textbf{1. Quotient tracker (QT) method.} 
The first approach for unwrapping the mod 1 samples is perhaps 
the most natural one. It is based on the idea that, provided the denoised mod 1 samples are very close estimates 
to the original clean mod 1 samples, then we can sequentially find the quotient terms, by checking whether 
$\abs{\widehat{f}_{i+1} \bmod 1 - \widehat{f_i} \bmod 1} \geq \qtathresh$, for a suitable threshold parameter $\qtathresh \in (0,1)$. 
More formally, after  initializing $\widehat{q}_1 = 0$, consider the rule
%
%
\begin{align} \label{eq:qta_recovery_rule}
\widehat{q}_{i+1} &= \widehat{q}_i + \sign_{\qtathresh}(\widehat{f}_{i+1} \bmod 1 - \widehat{f_i} \bmod 1); \nonumber \\  
\hspace{8mm}
\sign_{\qtathresh}(t) &= \left\{
\begin{array}{rl}
-1 ; &  t \geq \qtathresh \\
0 ; &  \abs{t} < \qtathresh \\ 
1; & t \leq -\qtathresh
\end{array}. \right.
\end{align}
%
%
%
Clearly, if $\widehat{f_{i}} \bmod 1 \approx f_i \bmod 1$ for each $i$, then for $n$ sufficiently large,
the procedure \eqref{eq:qta_recovery_rule} will result in correct recovery of the quotients. 
However, as illustrated in Figure \ref{fig:QT_examples_Gaussian}, it is also obviously sensitive to noise, and hence would not be a robust solution for high levels of noise.

\begin{figure}[!ht] 
\centering
\subcaptionbox[]{  QCQP + QT: $ \sigma=0.05$, \\ RMSE=0.141 \label{subfig:iF1_QCQP_QTA_Gaussian_5}
}[ 0.32\textwidth ]
{\includegraphics[width=0.32\textwidth] {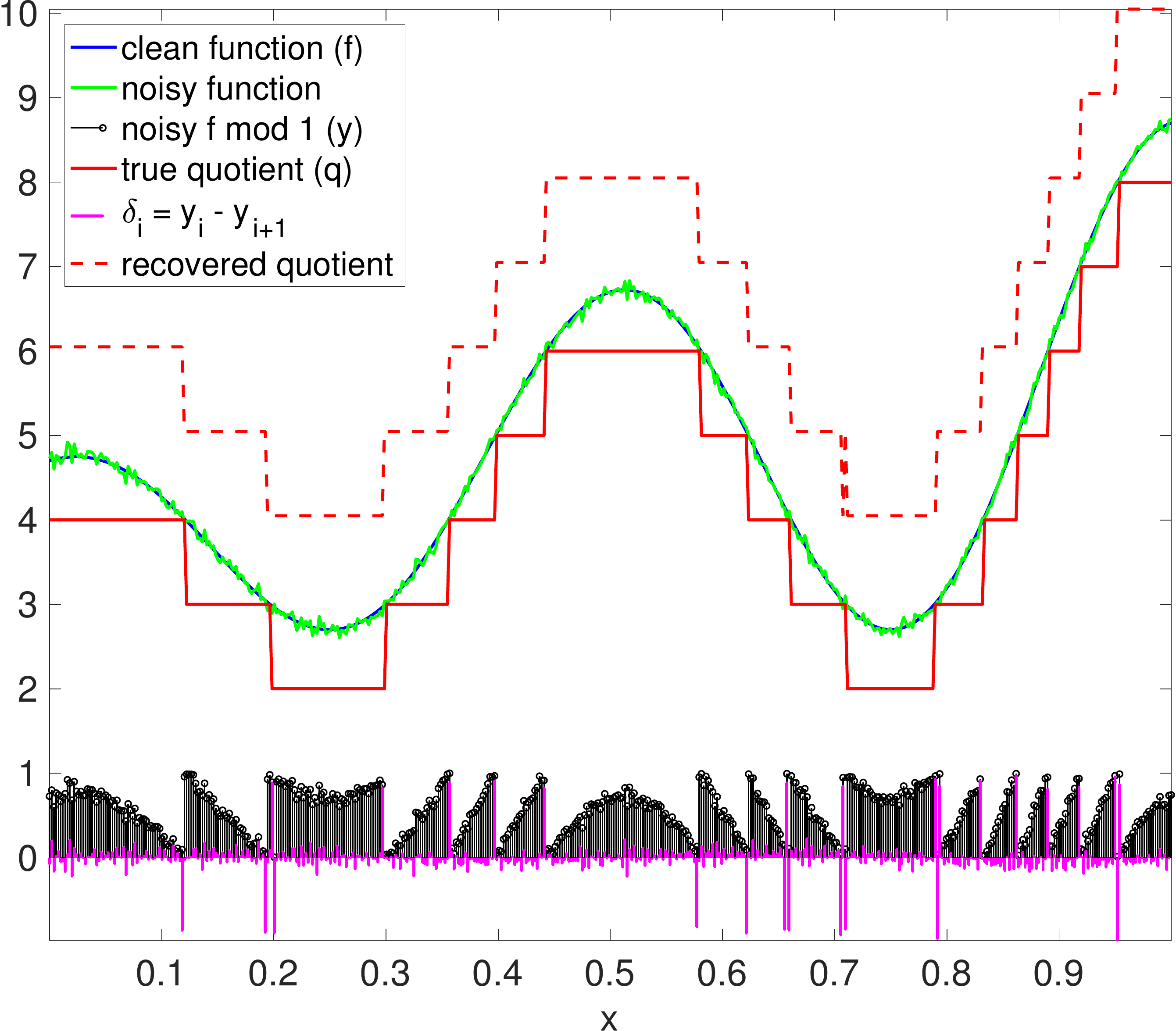} }
\subcaptionbox[]{  QCQP + QT: $ \sigma=0.10$,    \\ RMSE=0.206  \label{subfig:iF1_QCQP_QTA_Gaussian_10}
}[ 0.32\textwidth ]
{\includegraphics[width=0.32\textwidth] {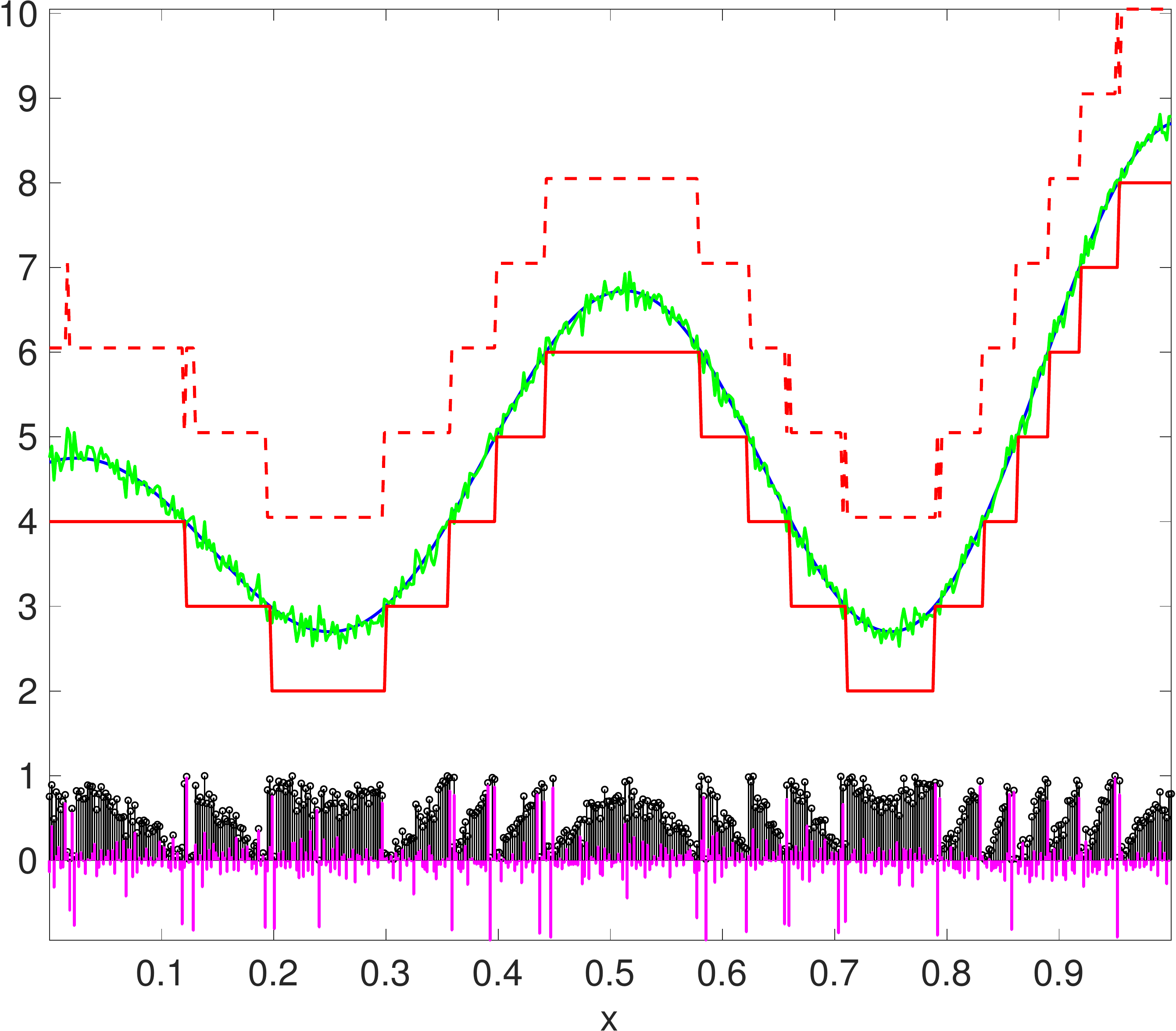} }
\subcaptionbox[]{  QCQP + QT: $ \sigma=0.15$,    \\ RMSE=1.008 \label{subfig:iF1_QCQP_QTA_Gaussian_15}
}[ 0.32\textwidth ]
{\includegraphics[width=0.32\textwidth] {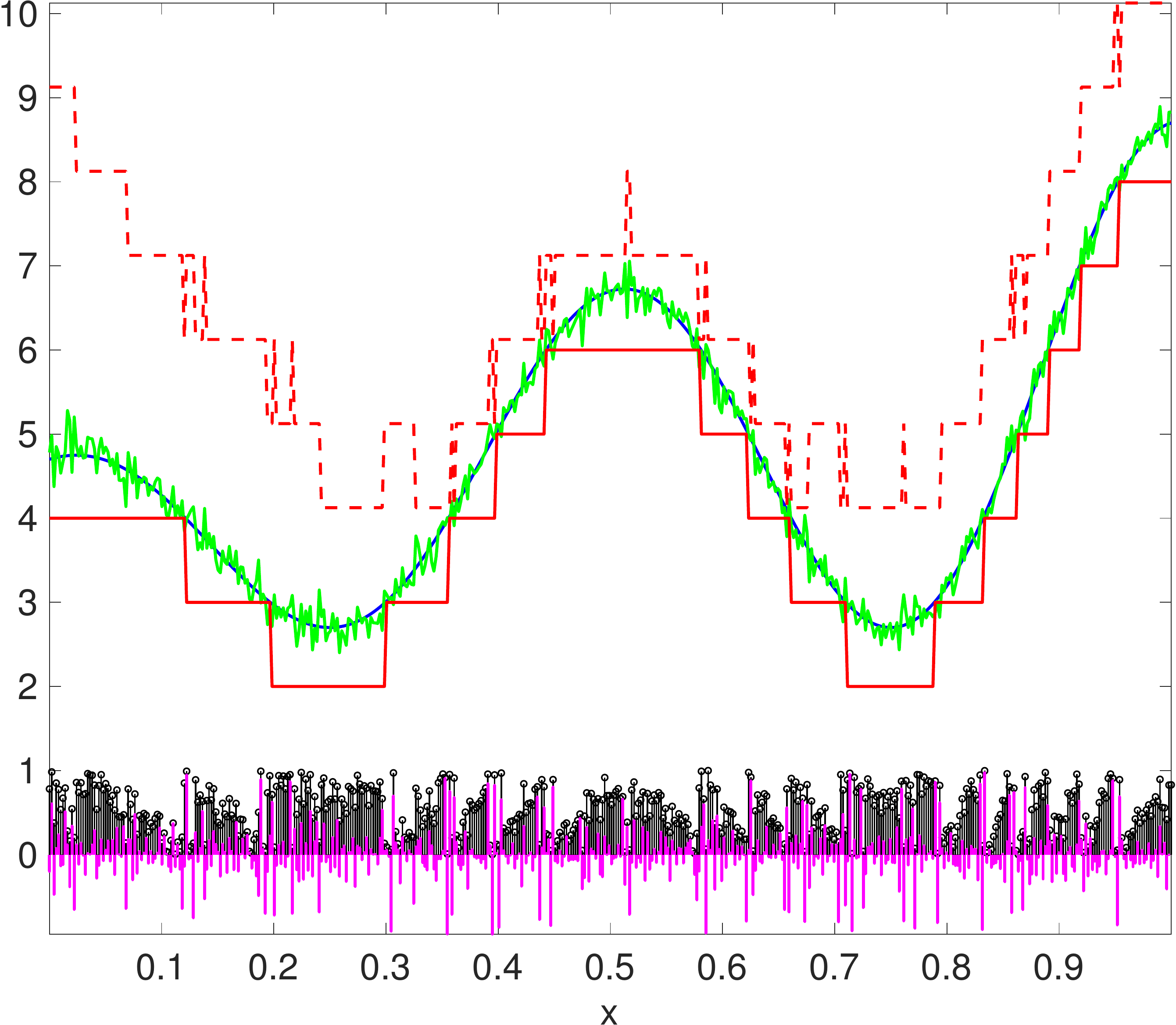} }
\subcaptionbox[]{  QCQP + OLS: $ \sigma=0.05$,    \\ RMSE=0.141 \label{subfig:iF1_QCQP_OLS_Gaussian_5}
}[ 0.32\textwidth ]
{\includegraphics[width=0.32\textwidth] {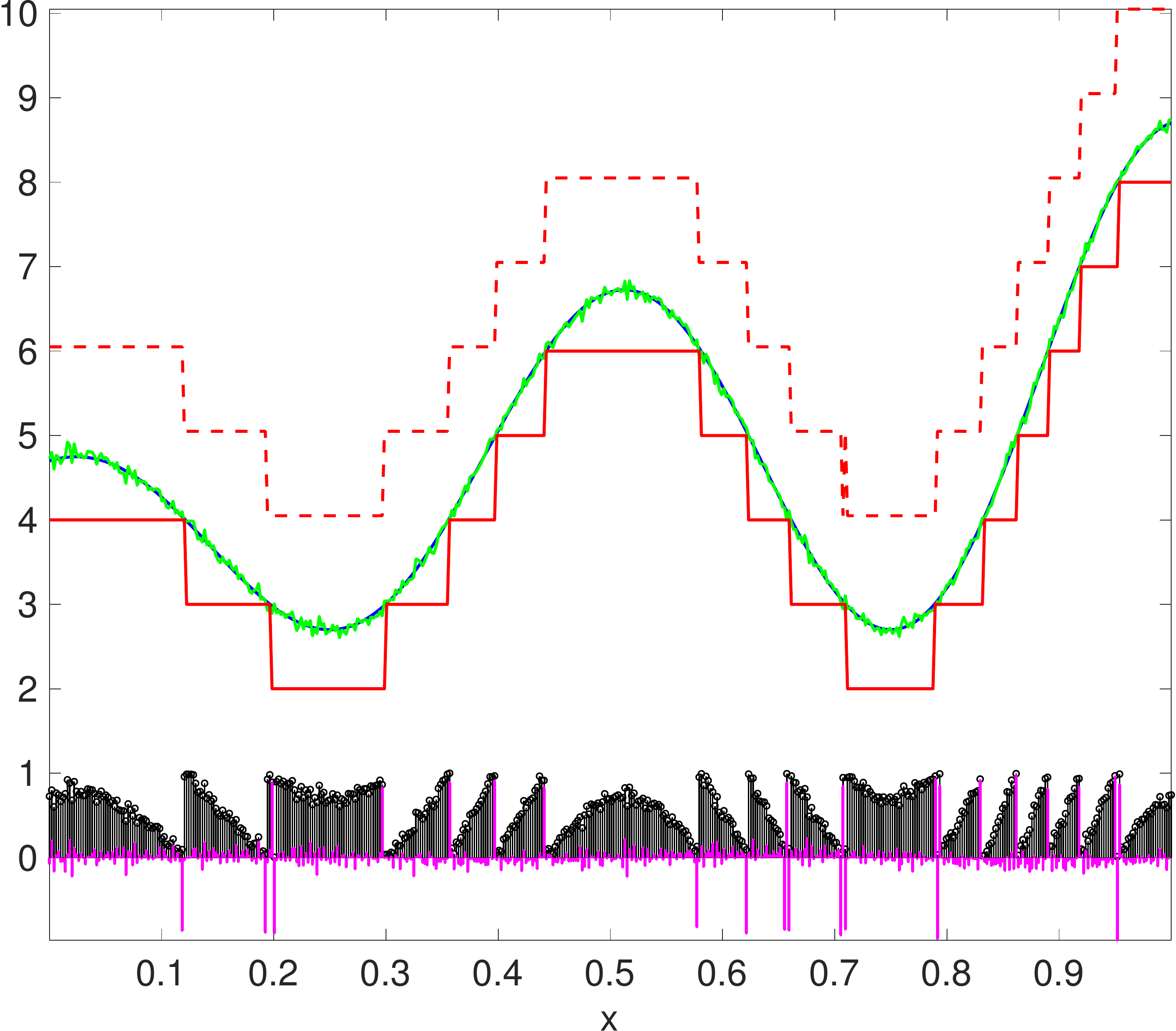} }
\subcaptionbox[]{  QCQP + OLS: $ \sigma=0.10$,   \\   RMSE=0.206 \label{subfig:iF1_QCQP_OLS_Gaussian_10}
}[0.32\textwidth]
{\includegraphics[width=0.32\textwidth] {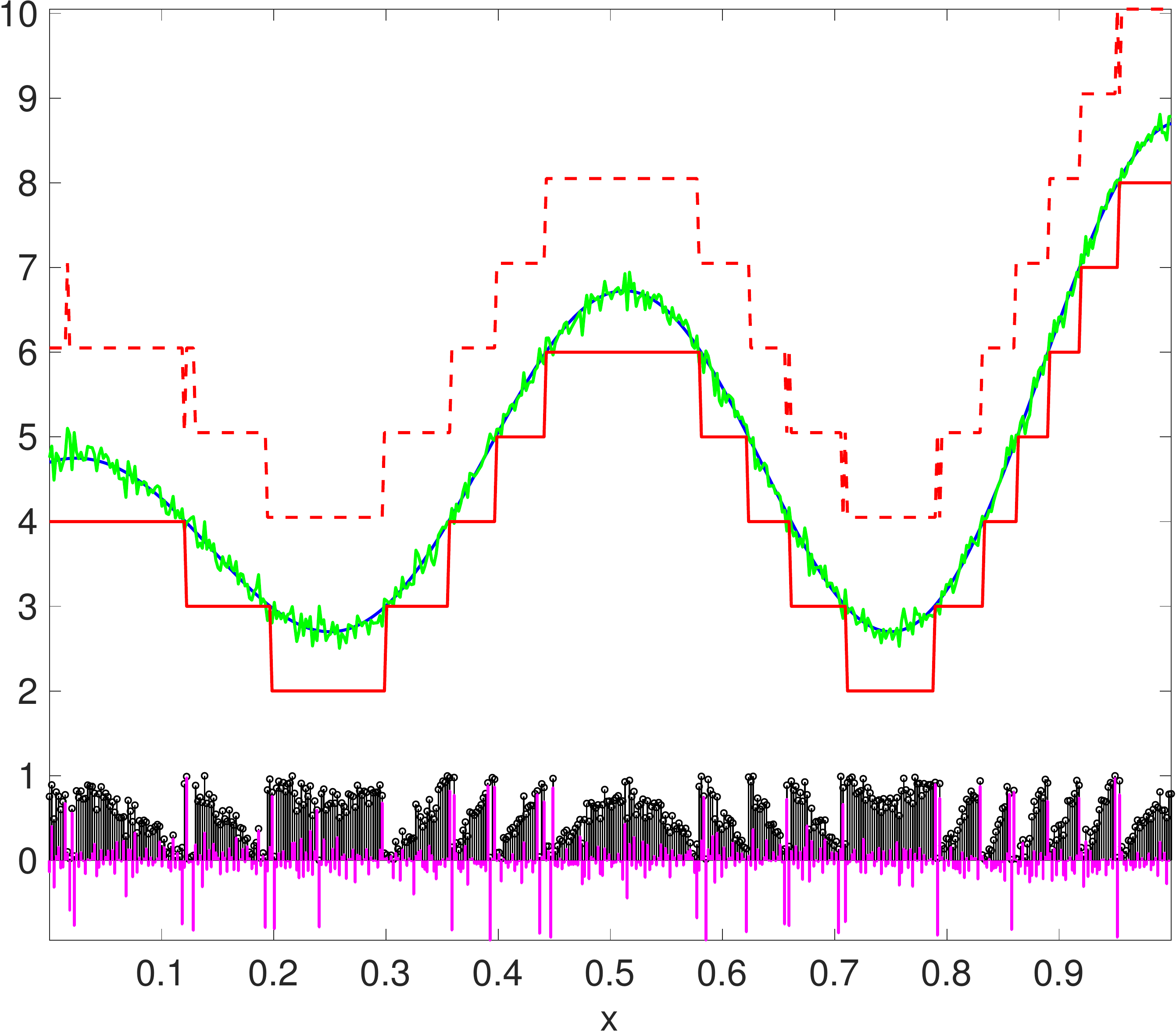} }
\subcaptionbox[]{  QCQP + OLS: $ \sigma=0.15$,   \\  RMSE=0.646   \label{subfig:iF1_QCQP_OLS_Gaussian_15}
}[0.32\textwidth]
{\includegraphics[width=0.32\textwidth] {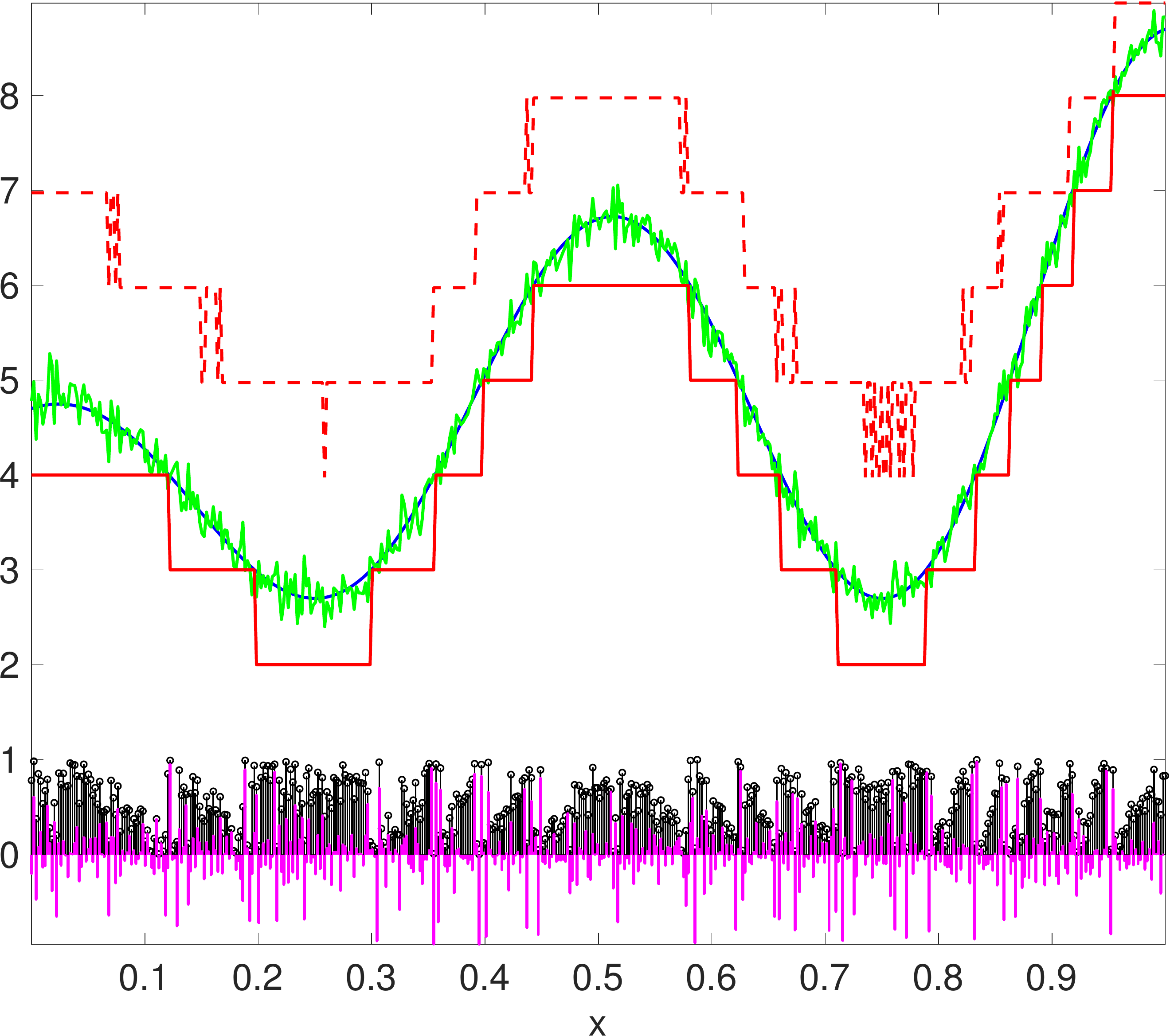} } 
%
\caption{Recovery of the estimated quotient at each sample point, via the  Quotient Tracker (\textbf{QT}) algorithm \eqref{eq:qta_recovery_rule} and the Ordinary Least Squares (\textbf{OLS}) \eqref{eq:ols_unwrap_lin_system}, when the input is given by the denoised modulo samples produced by our \textbf{QCQP} method in  Algorithm \ref{algo:two_stage_denoise}, for which  the unwrapping stage is  performed via \textbf{OLS} \eqref{eq:ols_unwrap_lin_system}.   Different columns pertain to varying levels of noise under the Gaussian noise model. We keep fixed $\lambda=0.1$, and $k=3$. We also plot the difference between consecutive noisy mod 1 measurements, $\delta_i = y_i - y_{i+1}$, highlighting the fact that QT lacks robustness at high levels of noise. 
}
\label{fig:QT_examples_Gaussian}
\vspace{-2mm}
\end{figure}

\textbf{2. Ordinary least-squares (OLS) based method.} 
A robust alternative to the aforementioned approach 
is based on directly recovering the function via a simple least-squares problem. Recall that in the noise-free case, $f_i = q_i + r_i$, $ q_i \in \mathbb{Z}, r_i \in [0,1) $, and consider, for a pair of nearby points $i,j$, the difference 
\begin{equation}   \label{eg:pairwise_f_diffs}
 f_i - f_j = q_i - q_j  + r_i - r_j.
\end{equation}
The OLS formulation we solve stems from the observation that, if $| r_i - r_j | < \qtathresh$ for a small $\qtathresh$, then $q_i = q_j$. This intuition can be easily gained from the top plots of Figure \ref{fig:instances_f1_Bounded_delta_cors}, especially   \ref{subfig:instances_f1_Bounded_delta_cors__Cor_gamma0p15}, which pertains to the noisy case (but in the low noise regime $\gamma=0.15$), that plots   $l_{i+1} - l_i$  versus  $y_i - y_{i+1}$, where  
$l_i$ denotes the noisy quotient of sample $i$, and $y_i$ the noisy remainder. For small enough $|y_i - y_{i+1}|$, we observe that  $| l_{i+1} - l_i | = 0$. Whenever  $y_i - y_{i+1} > \qtathresh$, we see that $ l_{i+1} - l_i = 1$, while $y_i - y_{i+1} < - \qtathresh$, indicates that $ l_{i+1} - l_i = -1$. Throughout  all our experiments we set $\qtathresh=0.5$. 
In Figure \ref{fig:instances_f1_Bounded} we also plot the true quotient $q$, which can be observed to be piecewise constant, in agreement with our above intuition.


For a graph\footnote{In the sequel, we will consider the same graph $G$ as for the denoising stage. But this is of course not necessary, one can consider a different graph here as well.} $G = (V,E)$ with $k \in \mathbb{N}$, and for a suitable threshold parameter $\qtathresh \in (0,1)$, this intuition leads us to estimate the function values $f_i$ as the least-squares solution to an overdetermined system of linear equations without involving  the quotients $q_1, \ldots, q_n$. To this end, we consider a linear system of equations involving the differences $\est{f}_i - \est{f}_j, \ \forall \set{i,j} \in E$ 
\begin{equation} \label{eq:ols_unwrap_lin_system}
	\est{f}_i - \est{f}_j = l_i - l_j  + y_i - y_j =  \sign_{\qtathresh}( y_{i}  - y_{j})   +   y_i - y_j,
\end{equation}
%
%
%
and solve it in the least-squares sense. \eqref{eq:ols_unwrap_lin_system} is analogous to \eqref{eq:qta_recovery_rule}, except that we now recover 
$(\widehat{f}_{i})_{i=1}^n$ collectively as the least-squares solution to \eqref{eq:ols_unwrap_lin_system}. Denoting by $T$ the least-squares matrix associated with the overdetermined linear system \eqref{eq:ols_unwrap_lin_system}, $b_{i,j} = \sign_{\qtathresh}( y_{i}  - y_{j})   +   y_i - y_j $, and $\est{f}_{i,j} = \est{f}_i - \est{f}_j$, the system of equations can be written as $ T \est{\vecf} = \vecb$ where $\est{\vecf}, \vecb \in \matR^{\abs{E}}$. Note that the matrix $T$ is sparse with only two non-zero entries per row, and that the all-ones vector $\mb{1} = ( 1, 1, \ldots, 1)^T$ lies in the null space of $T$, i.e., $T \mb{1} = 0 $. Therefore, we will find the minimum norm least-squares solution to \eqref{eq:ols_unwrap_lin_system}, and  recover $f_i$'s only up to a global shift. We remark that the above line of thought,  initiated by  \eqref{eg:pairwise_f_diffs} is very similar to Step 3 of the ASAP algorithm introduced by \cite{asap2d} (Section 3.3) in the context of the graph realization problem, that performed  synchronization over $\mathbb{R}^d$ to estimate the unknown translations. 
Algorithm \ref{algo:two_stage_denoise} summarizes our two-stage method for recovering the samples of $f$ (up to a global shift). 
Figure \ref{fig:instances_f1_Bounded_delta_cors} shows additional  noisy instances of the Uniform noise model. 
The scatter plots on the first row show that, as the noise  level  
increases, the function \eqref{eq:qta_recovery_rule} will produce more and more errors in \eqref{eq:ols_unwrap_lin_system}. The remaining plots show the corresponding f mod 1 signal (clean, noisy, and denoised via Algorithm \ref{algo:two_stage_denoise}) for three levels of noise.

%
%
\begin{algorithm}[!ht]
\caption{Algorithm for recovering the samples $f_i$} \label{algo:two_stage_denoise} 
\begin{algorithmic}[1] 
\State \textbf{Input:} $(y_i)_{i=1}^n$ (noisy mod 1 samples), $k$, $\lambda$, $n$, $G=(V,E)$. 
\State \textbf{Output:} Denoised mod 1 samples $\widehat{f_i}\bmod 1$; $i=1,\dots,n$. 

\textsc{// Stage 1: Recovering denoised} mod 1 \textsc{samples of} $f$. 

\State Form $\Hbar \in R^{2n \times 2n}$ using $\lambda,L$ as in \eqref{eq:def_H}.

\State Form $\veczbar = [\real(\vecz)^T \imag(\vecz)^T]^T \in \matR^{2n}$ as in \eqref{eq:real_notat}.

\State Obtain $\vecgbarest \in \matR^{2n}$ as the solution to \eqref{eq:qcqp_denoise_real}, i.e., 
%
$$\vecgbarest = \argmin{\vecgbar \in \mathbb{R}^{2n}: \norm{\vecgbar}_2^2 = n} \vecgbar^{T} \Hbar \vecgbar  - 2 \vecgbar^{T} \veczbar.$$

\State Obtain $\vecgest \in \mathbb{C}^n$ from $\vecgbarest$ where $\vecgbarest = [\real(\vecgest)^T \imag(\vecgest)^T]^T$.

\State Recover $\est{f_i}\bmod 1 \in [0,1)$ from $\frac{\gest_i}{\abs{\gest_i}}$ for each $i=1,\dots,n$, as in \eqref{eq:extr_mod1_vals}.

\textsc{// Stage 2: Recovering denoised real valued samples of} $f$.

\State \textbf{Input:} $(\est{f_i}\bmod 1)_{i=1}^n$ (denoised mod 1 samples), $G=(V,E)$, $\qtathresh \in (0,1)$. 

\State \textbf{Output:} Denoised samples $\widehat{f_i}$; $i=1,\dots,n$. 

\State Obtain $(\est{f}_i)_{i=1}^n$ via the Quotient Tracker (QT) or 
       OLS based method for suitable threshold $\qtathresh$.

\end{algorithmic}
\end{algorithm}

\vspace{-2mm}

\vspace{-2mm}

\begin{figure}
\centering
\subcaptionbox[]{  $\gamma=0.15$
 \label{subfig:instances_f1_Bounded_delta_cors__Cor_gamma0p15}
}[ 0.32\textwidth ]
{\includegraphics[width=0.31\textwidth] {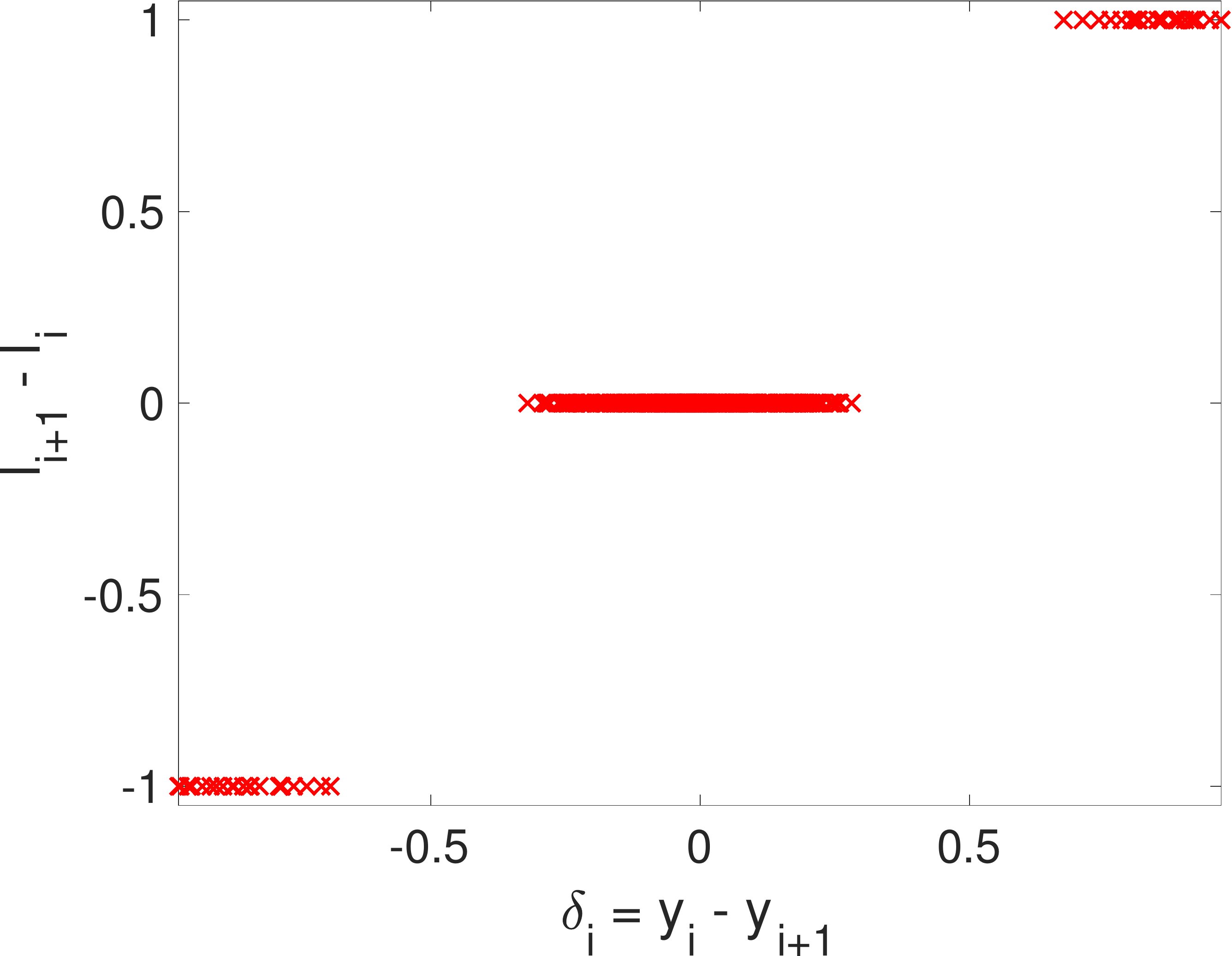} }
%
\subcaptionbox[]{  $\gamma=0.25$
}[ 0.32\textwidth ]
{\includegraphics[width=0.31\textwidth] {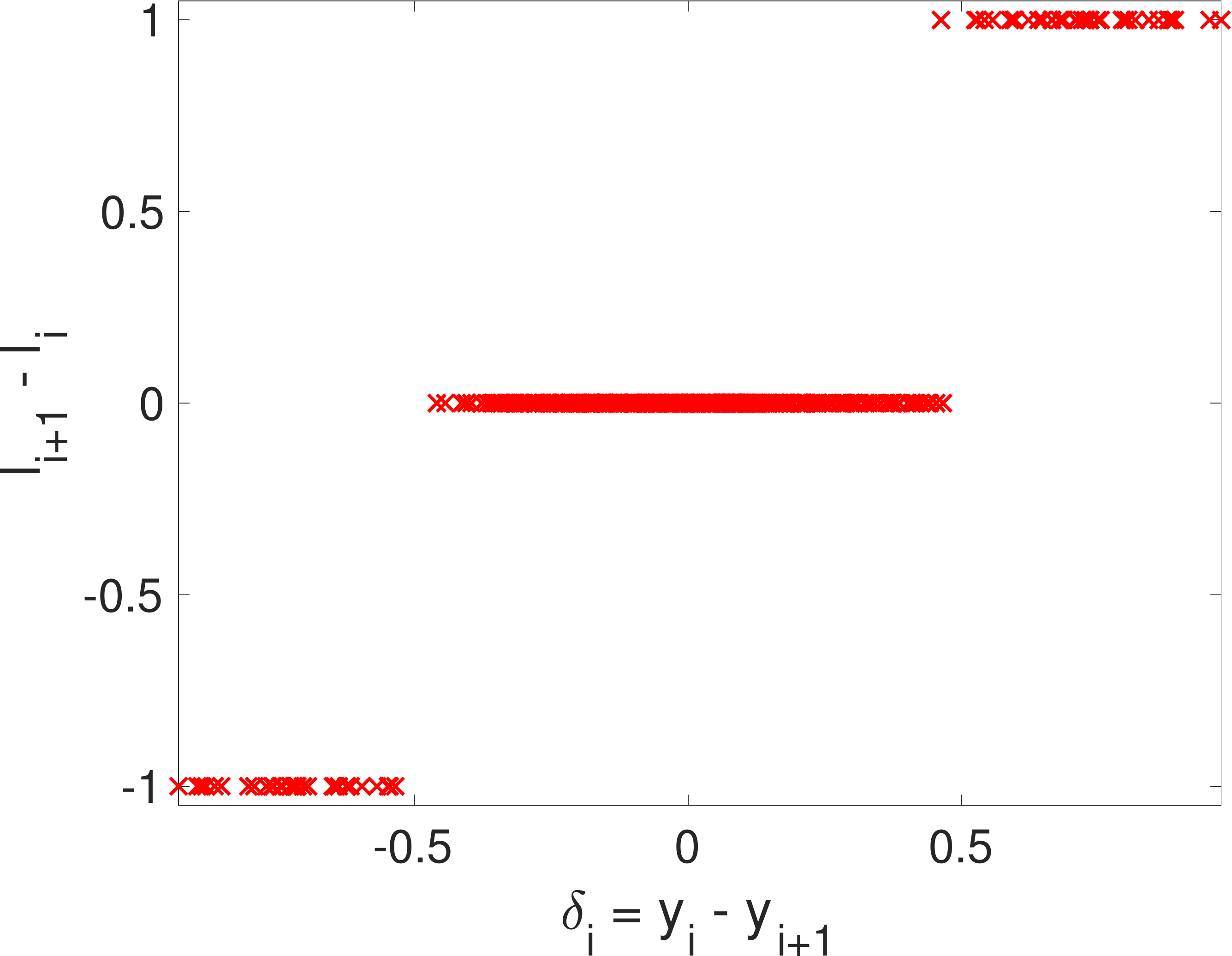} }
%
\subcaptionbox[]{  $\gamma=0.30$
}[ 0.32\textwidth ]
{\includegraphics[width=0.31\textwidth] {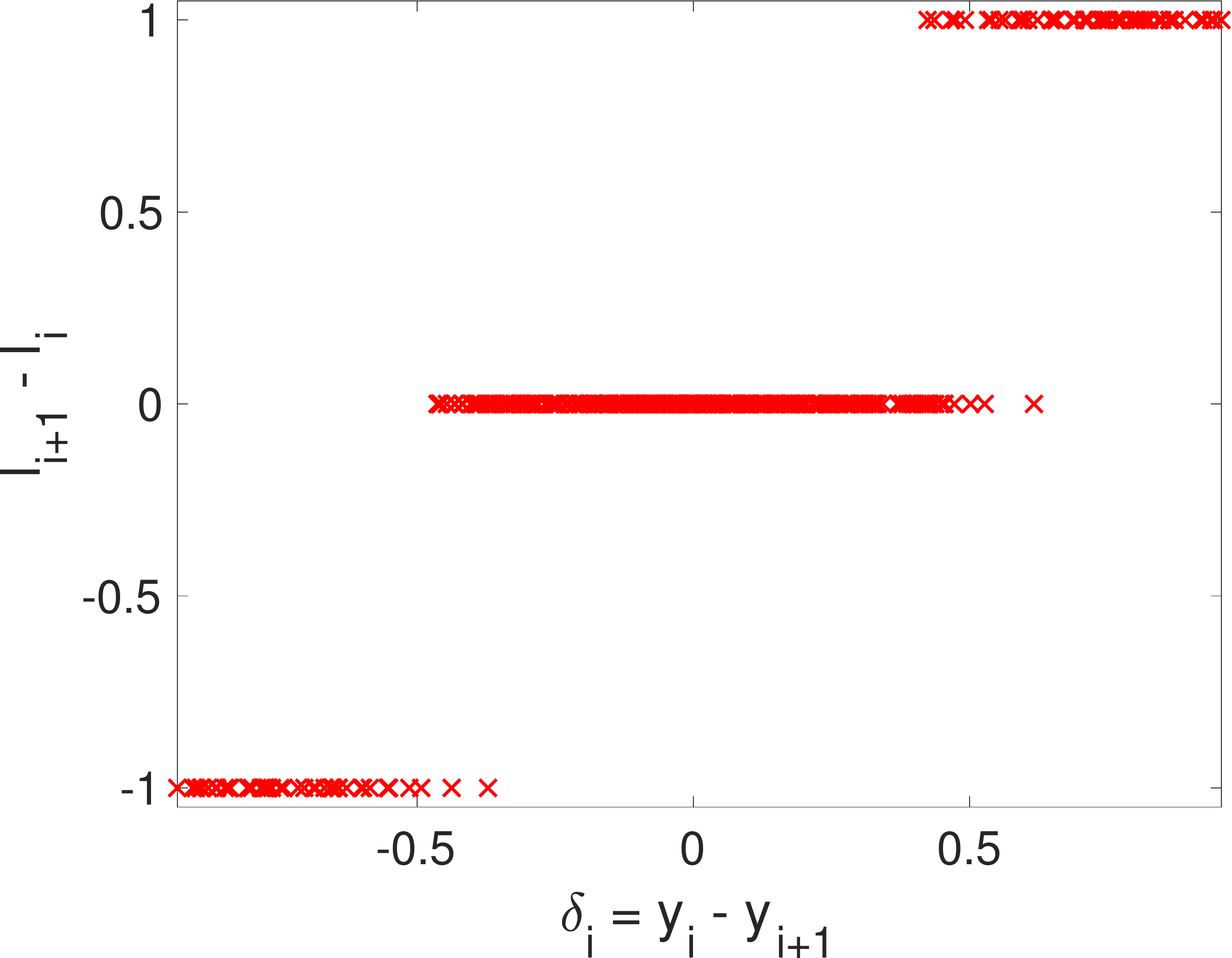} }
%
\vspace{2mm}

\subcaptionbox[]{  $\gamma=0.15$
}[ 0.96\textwidth ]
{\includegraphics[width=0.95\textwidth] {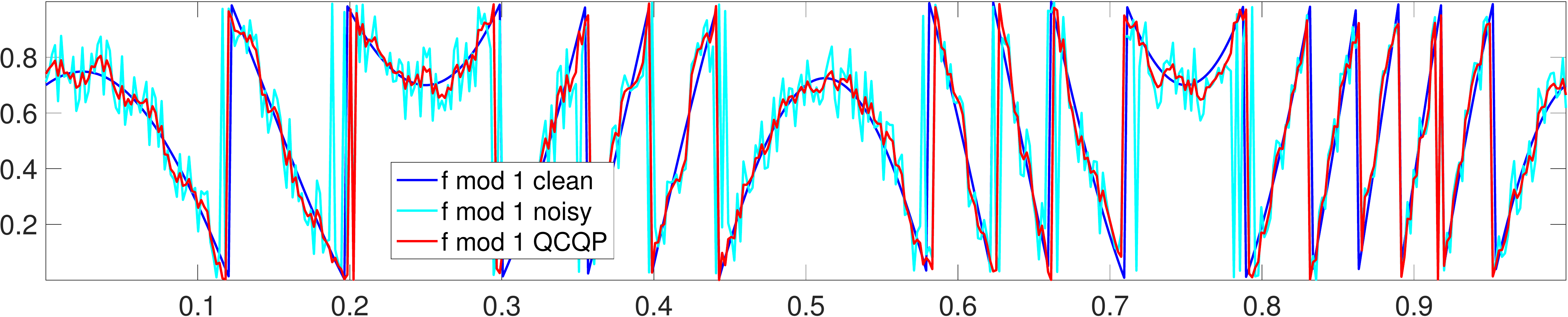} }
\vspace{2mm} 

\subcaptionbox[]{  $\gamma=0.25$
}[ 0.96\textwidth ]
{\includegraphics[width=0.95\textwidth] {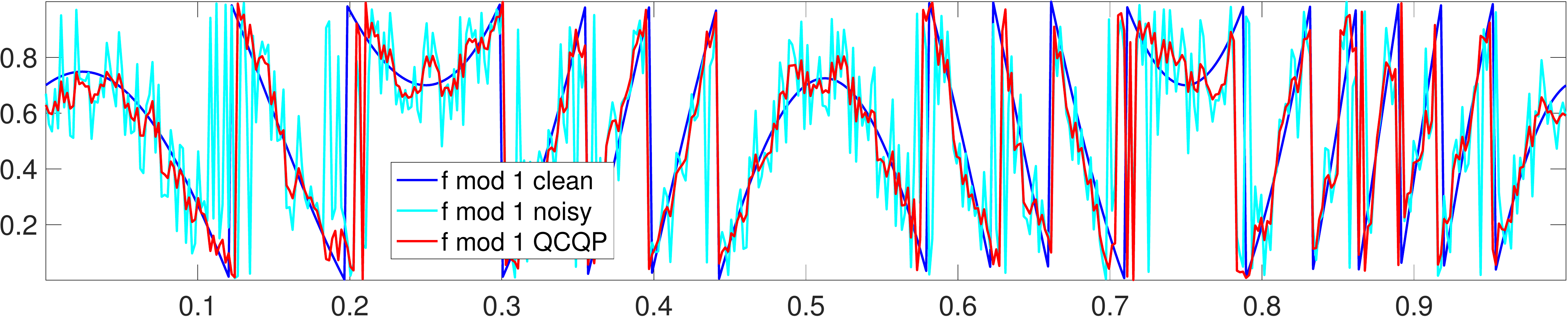} }
\vspace{2mm}

\subcaptionbox[]{  $\gamma=0.30$
}[ 0.96\textwidth ]
{\includegraphics[width=0.95\textwidth] {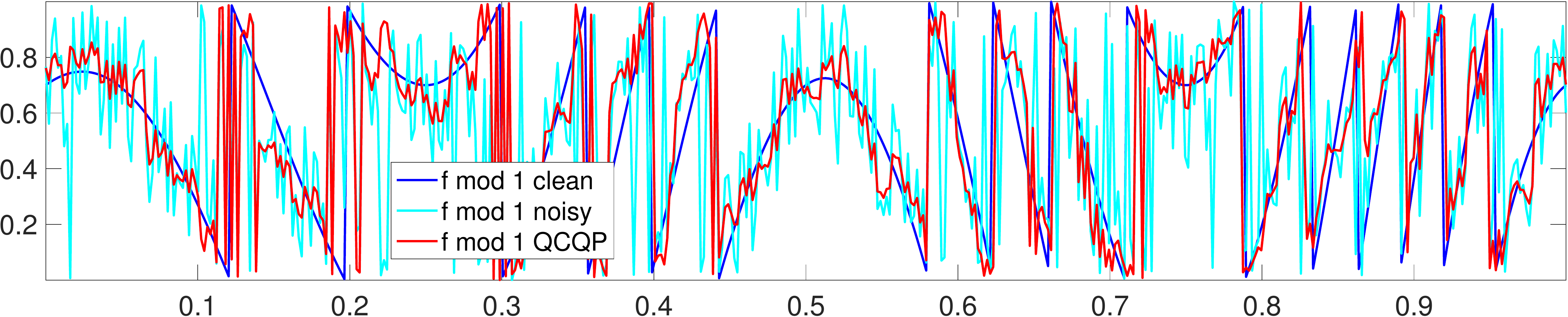} }
%
\vspace{-1mm}

\captionsetup{width=0.95\linewidth}
\caption[Short Caption]{Noisy instances of the Uniform noise model ($n=500$). Top row (a)-(c) shows scatter plots of change in $y$ (the observed noisy $f$ mod 1 values) versus change in $l$ (the noisy quotient). Plots (d)-(f) show  the clean $f$ mod 1 values (blue), the noisy $f$ mod 1 values (cyan) and the denoised (via   \textbf{QCQP}) $f$ mod 1 values  (red), for increasing levels of noise.   \textbf{QCQP} denotes Algorithm \ref{algo:two_stage_denoise} without  the unwrapping stage performed by \textbf{OLS} \eqref{eq:ols_unwrap_lin_system}.       
}
\label{fig:instances_f1_Bounded_delta_cors}
\end{figure}

\section{Analysis for the arbitrary bounded noise model} \label{sec:bound_noise_analysis}
%
This section  provides approximation guarantees for the solution $\vecgbarest \in \matR^{2n}$ 
to \eqref{eq:qcqp_denoise_real} for the arbitrary bounded noise model \eqref{eq:arb_bd_noise_model}. 
In particular, we consider a slightly modified version of this model, assuming 
%
\begin{equation}  \label{eq:BoundedNoiseModel}
\norm{\veczbar - \vechtilbar}_2 \leq \delta \sqrt{n}
\end{equation}
%
holds true for some $\delta \in [0,1]$. This is reasonable, since    
$\norm{\veczbar - \vechtilbar}_2 \leq 2 \sqrt{n}$ holds in general by triangle inequality. 
Also, note for \eqref{eq:arb_bd_noise_model} that $\abs{z_i - h_i} = 2\abs{\sin (\pi (\delta_i \bmod 1))} \leq 2\pi\abs{\delta_i}$, and thus 
$\norm{\veczbar - \vechtilbar}_2 = \norm{\vecz - \vechtil}_2 \leq  2 \pi \max_i ( |\delta_i| ) \sqrt{n}$.
Hence, while a small enough uniform bound on $\max_i (|\delta_i|)$ would of course imply \eqref{eq:BoundedNoiseModel}, however, clearly  
\eqref{eq:BoundedNoiseModel} can also hold even if some of the $\delta_i$'s are large.  

\vspace{-1mm}
\begin{theorem} \label{thm:arb_noise_model} 
Under the above notation and assumptions,  
consider the arbitrary bounded
noise model in \eqref{eq:arb_bd_noise_model}, with $\veczbar$ satisfying 
$\norm{\veczbar - \vechtilbar}_2 \leq \delta \sqrt{n}$ for $\delta \in [0,1]$. 
Let $n \geq 2$, and let $\calN(\Hbar)$ denote the null space of $\Hbar$.
\begin{enumerate}
\item If $\veczbar \not\perp \calN(\Hbar)$ then $\vecgbarest$ is the unique solution 
to \eqref{eq:qcqp_denoise_real} satisfying 
\begin{align}
\frac{1}{n}\dotprod{\vechtilbar}{\vecgbarest} \geq 1 - \frac{3\delta}{2} - \frac{\lambda \pi^2 M^2 (2k)^{2\alpha + 1}}{n^{2\alpha}}  
+ \frac{1}{(4\lambda k + 1)^2} \left(\frac{1}{2n} \veczbar^T \Hbar \veczbar\right).
\end{align}
%
%
\item If $\veczbar \perp \calN(\Hbar)$ and $\lambda < \frac{1}{4k}$ then $\vecgbarest$ is the unique solution 
to \eqref{eq:qcqp_denoise_real} satisfying
\begin{align}
\frac{1}{n}\dotprod{\vechtilbar}{\vecgbarest} \geq 1 - \frac{3\delta}{2} - \frac{\lambda \pi^2 M^2 (2k)^{2\alpha + 1}}{n^{2\alpha}} 
+  \frac{1}{\left(1 + 4\lambda k - 4 \lambda k \sin^2\left(\frac{\pi}{2n}\right) \right)^2} \left(\frac{1}{2n} \veczbar^T \Hbar \veczbar\right).
\end{align}
\end{enumerate}
\end{theorem}
%
%
%

The following useful Corollary of Theorem \ref{thm:arb_noise_model} is a direct consequence of the fact that 
$1/(2n) \veczbar^T \Hbar \veczbar \geq 0$ 
for all $\veczbar \in \matR^{2n}$, since $\Hbar$ is positive semi-definite.
%
\begin{corollary} \label{cor:main_thm_bounded_noise}
Consider the arbitrary bounded noise model defined  in \eqref{eq:arb_bd_noise_model}, with $\veczbar$ satisfying 
$\norm{\veczbar - \vechtilbar}_2 \leq \delta \sqrt{n}$ for $\delta \in [0,1]$. 
Let $n \geq 2$. If $\lambda < \frac{1}{4k}$ then $\vecgbarest$ is the unique solution 
to \eqref{eq:qcqp_denoise_real} satisfying 
\begin{equation}
\frac{1}{n}\dotprod{\vechtilbar}{\vecgbarest} \geq 1 - \frac{3\delta}{2} - \frac{\lambda \pi^2 M^2 (2k)^{2\alpha + 1}}{n^{2\alpha}}.
\end{equation}
\end{corollary}
%

Let us note that the above bounds are meaningful only when $\delta$ is small enough,  specifically $\delta < 2/3$. 
Before presenting the proof of Theorem \ref{thm:arb_noise_model}, some remarks are in order.
\begin{enumerate}
\item Theorem \ref{thm:arb_noise_model} give us a lower bound on the correlation between 
$\vechtilbar, \vecgbarest \in \matR^{2n}$, where clearly, $\frac{1}{n}\dotprod{\vechtilbar}{\vecgbarest} \in [-1,1]$. 
Note that the correlation improves when the noise term $\delta$ decreases, as one would expect.  
The term $\frac{\lambda \pi^2 M^2 (2k)^{2\alpha + 1}}{n^{2\alpha}}$ effectively arises on account of the smoothness 
of $f$, and is an upper bound on the term $\frac{1}{2n} \vechtilbar^T \Hbar \vechtilbar$ (made clear in Lemma \ref{lem:upbd_clean_quad_form}). 
Hence,  as the number of samples increases, $\frac{1}{2n} \vechtilbar^T \Hbar \vechtilbar$ goes to zero at the rate $n^{-2\alpha}$ (for fixed $k,\lambda$). 
Also note that the lower bound on $\frac{1}{n}\dotprod{\vechtilbar}{\vecgbarest}$ readily implies the $\ell_2$ norm bound 
$\norm{\vecgbarest-\vechtilbar}^2_2 = O(\delta n + n^{1-2\alpha})$.

\item The term $\frac{1}{2n} \veczbar^T \Hbar \veczbar$ represents the smoothness of the observed noisy samples. 
While an increasing amount of noise would usually render 
$\veczbar$ to be more and more non-smooth, and thus typically increase $\frac{1}{2n} \veczbar^T \Hbar \veczbar$, note that this would be met by a corresponding 
increase in $\delta$, and hence the lower bound on the correlation would not necessarily improve.

\item It is easy to verify that \eqref{eq:BoundedNoiseModel} implies 
$\dotprod{\veczbar}{\vechtilbar}/n \geq 1 - \frac{\delta^2}{2} \geq 1- \frac{\delta}{2}$.
Thus for $\veczbar$, which is feasible for \eqref{eq:qcqp_denoise_real}, we have a bound on correlation which is better than the bound in 
Corollary \ref{cor:main_thm_bounded_noise} by a $\delta + O(n^{-2\alpha})$ term. However, the solution $\vecgbarest$ 
to \eqref{eq:qcqp_denoise_real} is a \emph{smooth} estimate of $\vechtilbar$ (and hence more interpretable), while $\veczbar$ is typically highly 
non-smooth. 
\end{enumerate}
%
%
\begin{proof}[Proof of Theorem \ref{thm:arb_noise_model}] 
The proof of Theorem \ref{thm:arb_noise_model} relies heavily on Lemma \ref{lem:lowbd_feas_based} 
outlined below.
\begin{lemma} \label{lem:lowbd_feas_based}
Consider the arbitrary bounded noise model in \eqref{eq:arb_bd_noise_model}, with $\veczbar$ satisfying 
$\norm{\veczbar - \vechtilbar}_2 \leq \delta \sqrt{n}$ for $\delta \in [0,1]$. 
Any solution $\vecgbarest$ to \eqref{eq:qcqp_denoise_real} satisfies
\begin{equation} \label{eq:lowbd_ang_feas_form}
\frac{1}{n}\dotprod{\vechtilbar}{\vecgbarest} \geq  1 - \frac{3\delta}{2} 
 - \; \;   \frac{1}{2n} \vechtilbar^T \Hbar \vechtilbar + \; \; \frac{1}{2n} \vecgbarest^T \Hbar \vecgbarest. 
\end{equation}
\end{lemma}
%
%
%
\begin{proof}
To begin with, note that
\begin{equation} \label{eq:lem1_temp1}
\norm{\veczbar - \vechtilbar}_2 \leq \delta\sqrt{n} \Leftrightarrow \dotprod{\veczbar}{\vechtilbar} \geq n - \frac{\delta^2 n}{2}.
\end{equation}
Since $\vechtilbar \in \matR^{2n}$ is feasible for \eqref{eq:qcqp_denoise_real}, we get 
\begin{align}
\vechtilbar^{T} \Hbar \vechtilbar  - 2 \dotprod{\vechtilbar}{\veczbar} 
&\geq \vecgbarest^{T} \Hbar \vecgbarest  - 2 \dotprod{\vecgbarest}{\veczbar} \\
\Leftrightarrow \dotprod{\vecgbarest}{\veczbar} &\geq \dotprod{\vechtilbar}{\veczbar} - \frac{1}{2} \vechtilbar^{T} \Hbar \vechtilbar
+ \frac{1}{2} \vecgbarest^{T} \Hbar \vecgbarest \\
&\geq n - \frac{\delta^2 n}{2} - \frac{1}{2} \vechtilbar^{T} \Hbar \vechtilbar 
+ \frac{1}{2} \vecgbarest^{T} \Hbar \vecgbarest \quad \text{(from \eqref{eq:lem1_temp1})}.  \label{eq:lem1_temp2}
\end{align}
Moreover, we can upper bound $\dotprod{\vecgbarest}{\veczbar}$ as follows.
\begin{align}
\dotprod{\vecgbarest}{\veczbar} 
&= \dotprod{\vecgbarest}{\veczbar - \vechtilbar} + \dotprod{\vecgbarest}{\vechtilbar} \\
&\leq \norm{\vecgbarest}_2 \norm{{\veczbar - \vechtilbar}}_2 + \dotprod{\vecgbarest}{\vechtilbar} \quad \text{(Cauchy-Schwarz)} \\
&\leq \sqrt{n} (\sqrt{n} \delta) + \dotprod{\vecgbarest}{\vechtilbar} \quad \text{(from \eqref{eq:lem1_temp1})}. \label{eq:lem1_temp3}
\end{align}
Plugging \eqref{eq:lem1_temp3} in \eqref{eq:lem1_temp2} and using $\delta^2 \leq \delta$ for $\delta \in [0,1]$ completes the proof.
\end{proof}
We now upper bound the term $\frac{1}{2n} \vechtilbar^T \Hbar \vechtilbar$ 
in \eqref{eq:lowbd_ang_feas_form} using the H\"older continuity of $f$. 
This is formally shown below in Lemma \ref{lem:upbd_clean_quad_form}.
\begin{lemma} \label{lem:upbd_clean_quad_form}
For $n \geq 2$, it holds true that 
\begin{equation}
\frac{1}{2n} \vechtilbar^T \Hbar \vechtilbar \leq \frac{\lambda \pi^2 M^2 (2k)^{2\alpha + 1}}{n^{2\alpha}}, 
\end{equation} 
where $\alpha \in (0,1]$ and $M > 0$ are related to the smoothness of $f$ and defined in \eqref{eq:f_smooth_hold}, and  
$\lambda \geq 0$ is the regularization parameter in \eqref{eq:orig_denoise_hard}.
\end{lemma}
%
%
\begin{proof}
Denoting $\vechtil = (\htil_1 \dots \htil_n)^T \in \mathbb{C}^n$ to be the complex valued 
representation of $\vechtilbar \in \matR^{2n}$ as per \eqref{eq:real_notat}, clearly
{\small
\begin{align} 
\frac{1}{2n} \vechtilbar^T \Hbar \vechtilbar = \frac{1}{2n} \vechtil^{*} (\lambda L) \vechtil
= \frac{\lambda}{2n} \sum_{\set{i,j} \in E} \abs{\htil_i - \htil_j}^2 
\leq \frac{\lambda}{2n} \abs{E} \max_{\set{i,j} \in E} \abs{\htil_i - \htil_j}^2,  \label{eq:lem2_temp1}
\end{align}}
where the first equality is shown in Appendix \ref{sec:qcqp_compl_to_real}. 
Since for each $i \in V$ we have $\text{deg}(i) \leq 2k$, hence $\abs{E} = (1/2) \sum_{i \in V} \text{deg}(i) \leq k n$. 
Next, for any $\set{i,j} \in E$ note that by H\"older continuity of $f$ we have
\begin{equation} \label{eq:lem2_temp2}
\abs{f_i - f_j} \leq M \abs{x_i - x_j}^{\alpha} \leq M \left(\frac{k}{n-1}\right)^{\alpha} \leq M\left(\frac{2k}{n}\right)^{\alpha}, 
\end{equation}
if $n \geq 2$ (since then $n-1 \geq n/2$). Finally, we can bound $\abs{\htil_i - \htil_j}$ as follows.
%
%
\begin{align}
\abs{\htil_i - \htil_j} &= \abs{1 - \exp(\iota 2\pi (f_j - f_i))} \\
&= 2\abs{\sin \pi (f_j - f_i)} \\
&\leq 2\pi \abs{f_j - f_i} \quad (\text{since} \ \abs{\sin x} \leq \abs{x}; \ \forall x \in \matR) \label{eq:lem2_newround_temp1} \\ 
&\leq \frac{2\pi M (2k)^{\alpha}}{n^{\alpha}} \quad (\text{using} \ \eqref{eq:lem2_temp2}).  \label{eq:lem2_temp3}
\end{align}
Plugging \eqref{eq:lem2_temp3} in \eqref{eq:lem2_temp1} with the bound $\abs{E} \leq kn$ yields the desired  bound.
\end{proof}
Lastly, we lower bound the term $\frac{1}{2n} \vecgbarest^T \Hbar \vecgbarest$ in \eqref{eq:lowbd_ang_feas_form}
using knowledge of the structure of the solution $\vecgbar$. This is outlined below as Lemma \ref{lem:lowbd_sol_quad_form}. 
\begin{lemma} \label{lem:lowbd_sol_quad_form}
Denoting $\calN(\Hbar)$ to be the null space of $\Hbar$, the following holds 
for the solution $\vecgbarest$ to \eqref{eq:qcqp_denoise_real}.
%
%
\begin{enumerate}
\item If $\veczbar \not\perp \calN(\Hbar)$ then $\vecgbarest$ is unique and 
\begin{equation}
\frac{1}{2n} \vecgbarest^T \Hbar \vecgbarest \geq \frac{1}{\left(1 + 4\lambda k \right)^2} \left(\frac{1}{2n} \veczbar^T \Hbar \veczbar\right).
\end{equation}

\item If $\veczbar \perp \calN(\Hbar)$ and $\lambda < \frac{1}{4k}$, then $\vecgbarest$ is unique and 
{\small
\begin{equation}
\frac{1}{2n} \vecgbarest^T \Hbar \vecgbarest \geq 
\frac{1}{\left(1 + 4\lambda k - 4 \lambda k \sin^2\left(\frac{\pi}{2n}\right)\right)^2} \left(\frac{1}{2n} \veczbar^T \Hbar \veczbar\right).
\end{equation}}
\end{enumerate}
\end{lemma}
\begin{proof}
Let $\set{\lambda_j(\Hbar)}_{j=1}^{2n}$ (with $\lambda_1(\Hbar) \leq \lambda_2(\Hbar) \leq \cdots$) and $\set{\vecq_j}_{j=1}^{2n}$ denote 
the eigenvalues and eigenvectors respectively for $\Hbar$.  
Also, let $0 = \beta_1(L) < \beta_2(L) \leq \beta_3(L) \leq \cdots \leq \beta_n(L)$ denote the eigenvalues of the Laplacian $L$. 
Note that $\beta_2(L) > 0$ since $G$ is connected. By Gershgorin's disk theorem, 
it is easy to see\footnote{Denote $L_{ij}$ to be the $(i,j)^{\text{th}}$ entry of $L$. Then by Gershgorin's disk theorem, we 
know that each eigenvalue lies in $\bigcup_{i=1}^{2n}\set{x:\abs{x-L_{ii}} \leq \sum_{j \neq i} \abs{L_{ij}}}$. Since $L_{ii} \leq 2k$ 
and $\sum_{j \neq i} \abs{L_{ij}} \leq 2k$ holds for each $i$, the claim follows.}
that $\beta_n(L) \leq 4k$ for the graph $G$. Hence,
\begin{equation} \label{eq:lem3_temp0}
0 = \lambda_1(\Hbar) = \lambda_2(\Hbar) < \lambda_3(\Hbar) \leq \cdots \leq \lambda_{2n}(\Hbar) \leq 4\lambda k
\end{equation}
and $\calN(\Hbar) = \text{span}\set{\vecq_1,\vecq_2}$. We now consider the two cases separately below. 
\begin{enumerate}
\item Consider the case where $\veczbar \not\perp \calN(\Hbar)$. 
We know that $\vecgbarest = 2(2\Hbar + \mu^{*}\matI)^{-1}\veczbar$ 
for a unique $\mu^{*} \in (0,\infty)$ (and so $\vecgbarest$ is the unique solution to \eqref{eq:qcqp_denoise_real} by 
Lemma \ref{lemma:qcqp_denoise_real} since $2\Hbar + \mu^{*}\matI \succ 0$) satisfying 
\begin{equation}\label{eq:lem3_temp1}
\norm{\vecgbarest}_2^2 = 4 \sum_{j=1}^{2n} \frac{\dotprod{\veczbar}{\vecq_j}^2}{(2\lambda_j(\Hbar) + \mu^{*})^2} = n.
\end{equation}
Since $\lambda_j(\Hbar) \geq 0$ for all $j$,  we obtain from \eqref{eq:lem3_temp1} that
\begin{equation}
n  \leq 4\sum_{j=1}^{2n} \frac{\dotprod{\veczbar}{\vecq_j}^2}{(\mu^{*})^2} = \frac{4n}{(\mu^{*})^2}  \quad \quad 
\Longrightarrow  \quad \quad    \mu^{*} \leq 2. \label{eq:lem3_temp2}
\end{equation}
Note that equality holds in \eqref{eq:lem3_temp2} if $\veczbar \in \calN(\Hbar)$. We can now lower bound 
$\frac{1}{2n} \vecgbarest^T \Hbar \vecgbarest$ as follows 
\begin{align}
\frac{1}{2n} \vecgbarest^T \Hbar \vecgbarest 
&= \frac{2}{n}\veczbar^T (2\Hbar + \mu^{*}\matI)^{-1} \Hbar (2\Hbar + \mu^{*}\matI)^{-1} \veczbar \\
& = \frac{2}{n} \sum_{j=1}^{2n} \frac{\dotprod{\veczbar}{\vecq_j}^2 \lambda_j(\Hbar)}{(2\lambda_j(\Hbar) + \mu^{*})^2} \\
&\geq \frac{2}{n (8 \lambda k + 2)^2} \sum_{j=1}^{2n}\dotprod{\veczbar}{\vecq_j}^2 \lambda_j(\Hbar) \quad (\text{from} \ \eqref{eq:lem3_temp0}, \eqref{eq:lem3_temp2}) \\ 
&= \frac{1}{(4\lambda k + 1)^2} \left(\frac{1}{2n} \veczbar^T \Hbar \veczbar\right).
\end{align}
%
\item Let us now consider the case where $\veczbar \perp \calN(\Hbar)$. 
Denote 
\begin{equation}
\phi(\mu) := \norm{\vecgbarest(\mu)}_2^2 = 4 \sum_{j=1}^{2n} \frac{\dotprod{\veczbar}{\vecq_j}^2}{(2\lambda_j(\Hbar) + \mu)^2} 
= 4 \sum_{j=3}^{2n} \frac{\dotprod{\veczbar}{\vecq_j}^2}{(2\lambda_j(\Hbar) + \mu)^2}. 
\end{equation}
Observe that $\phi$ does not have a pole at $0$ anymore, $\phi(0)$ is well defined. In order to have a unique 
$\vecgbarest$, it is sufficient if $\phi(0) > n$ holds. Indeed, we would then have a unique $\mu^{*} \in (0,\infty)$ 
such that $\norm{\vecgbarest(\mu^{*})}_2^2 = n$. Consequently, $\vecgbarest(\mu^{*})$ will be the unique solution to \eqref{eq:qcqp_denoise_real} by 
Lemma \ref{lemma:qcqp_denoise_real} since $2\Hbar + \mu^{*}\matI \succ 0$. Now let us note that 
\begin{equation} \label{eq:lem3_case2_temp1}
\phi(0) = \sum_{j=3}^{2n} \frac{\dotprod{\veczbar}{\vecq_j}^2}{\lambda_j(\Hbar)^2} \geq  \frac{n}{16\lambda^2 k^2}
\end{equation}
since $\lambda_j(\Hbar) \leq 4\lambda k$ for all $j$ (recall \eqref{eq:lem3_temp0}). Therefore clearly, the choice 
$\lambda < \frac{1}{4k}$ implies $\phi(0) > n$, and consequently that the solution $\vecgbarest$ is unique. 
Assuming $\lambda < \frac{1}{4k}$ holds, we can derive an upper bound on $\mu^{*}$ as follows
\begin{align}
n = 4 \sum_{j=3}^{2n} \frac{\dotprod{\veczbar}{\vecq_j}^2}{(2\lambda_j(\Hbar) + \mu^{*})^2} 
\leq 4 \sum_{j=3}^{2n} \frac{\dotprod{\veczbar}{\vecq_j}^2}{(2\lambda_3(\Hbar) + \mu^{*})^2}  
&= \frac{4 n}{(2\lambda_3(\Hbar) + \mu^{*})^2} \\ 
\Rightarrow \mu^{*} &\leq 2 - 2\lambda_3(\Hbar). \label{eq:lem3_case2_temp2}
\end{align}
Hence $\mu^{*} \in (0,2 - 2\lambda_3(\Hbar))$ when $\lambda < \frac{1}{4k}$. We can 
now lower bound $\frac{1}{2n} \vecgbarest^T \Hbar \vecgbarest$ in the same manner as before
\begin{align}
\frac{1}{2n} \vecgbarest^T \Hbar \vecgbarest 
&= \frac{2}{n}\veczbar^T (2\Hbar + \mu^{*}\matI)^{-1} \Hbar (2\Hbar + \mu^{*}\matI)^{-1} \veczbar \\
& = \frac{2}{n} \sum_{j=1}^{2n} \frac{\dotprod{\veczbar}{\vecq_j}^2 \lambda_j(\Hbar)}{(2\lambda_j(\Hbar) + \mu^{*})^2} \\
&\geq \frac{2}{n (8 \lambda k + 2 - 2\lambda_3(\Hbar))^2} \sum_{j=1}^{2n}\dotprod{\veczbar}{\vecq_j}^2 \lambda_j(\Hbar) \quad 
(\text{from} \eqref{eq:lem3_temp0}, \eqref{eq:lem3_case2_temp2}) \\
&= \frac{1}{(1 + 4\lambda k - \lambda_3(\Hbar))^2} \left(\frac{1}{2n} \veczbar^T \Hbar \veczbar\right). \label{eq:lem3_case2_temp3}
\end{align}
It remains to lower bound $\lambda_3(\Hbar) = \lambda \beta_2(L)$. We do this by using the following result
of \cite{Fiedler1973} (adjusted to our notation) for lower bounding the second smallest eigenvalue of the Laplacian of a simple graph.
\begin{theorem}[\cite{Fiedler1973}] \label{thm:fiedler_ss_eigval}
Let $G$ be a simple graph of order $n$ other than a complete graph, 
with vertex connectivity $\kappa(G)$ and edge connectivity $\kappa^{\prime}(G)$. It holds true that    
\begin{equation}
2\kappa^{\prime}(G) (1 - \cos(\pi/n)) \leq \beta_2(L) \leq \kappa(G) \leq \kappa^{\prime}(G).
\end{equation}
\end{theorem}
The graph $G$ in our setting has $\kappa(G) = k$. Indeed, there does not exist a vertex cut of size $k-1$ or less, 
but there does exist a vertex cut of size $k$. 
This in turn implies that $\kappa^{\prime}(G) \geq k$, and Theorem \ref{thm:fiedler_ss_eigval}  yields   
\begin{equation} 
\beta_2(L) \geq 2k (1 - \cos(\pi/n)) = 4k \sin^2\left(\frac{\pi}{2n}\right).
\end{equation} 
Hence $\lambda_3(\Hbar) \geq 4 \lambda k \sin^2\left(\frac{\pi}{2n}\right)$. Plugging this in to \eqref{eq:lem3_case2_temp2} completes the proof.
\end{enumerate}
\end{proof}
This completes the proof of Theorem \ref{thm:arb_noise_model}.
\end{proof}

\section{Analysis for random noise models} \label{sec:analysis_rand_noise}
We now analyze two random noise models, namely the Bernoull-Uniform model described in \eqref{eq:bern_unif_noise_model}, 
and the Gaussian noise model described in \eqref{eq:gauss_noise_model}. 

\subsection{Analysis for the Bernoulli-Uniform random noise model} \label{subsec:bern_unif_noise}
Let $u_1,\dots,u_n \sim U[0,1]$ i.i.d, where $U[0,1]$ denotes the uniform distribution over $[0,1]$. 
Also, let $\beta_1,\beta_2,\dots,\beta_n$ be i.i.d Bernoulli random variables where $\beta_i = 1$ with probability $p$, 
and $0$ with probability $1-p$, for some $p \in [0,1]$. Then \eqref{eq:bern_unif_noise_model} is equivalent to 

\begin{equation} \label{eq:bern_unif_noise_model_1}
y_i = \left\{
\begin{array}{rl}
f_i \bmod 1 \quad ; &  \ \text{if} \ \beta_i = 0 \\
u_i \quad ; & \ \text{if} \ \beta_i = 1
\end{array} \right. ; \quad i=1,\dots,n.
\end{equation}
The following theorem is our main result for this noise model.
%
\begin{theorem} \label{thm:bern_unif_noise_model}
Consider the Bernoulli-Uniform noise model in \eqref{eq:bern_unif_noise_model_1} for some $p \in [0,1]$. 
If $\lambda < \frac{1}{4k}$ then $\vecgbarest$ is the unique solution to \eqref{eq:qcqp_denoise_real}. 
Moreover, assuming $n \geq 2$, then for any $\varepsilon \in (0,1/2)$ satisfying $p + \varepsilon \leq 1/2$ and absolute constants $c,c^{\prime} > 0$, 
the following is true. 
\begin{enumerate}
\item\label{thm:bern_unif_1}
\begin{equation} \label{eq:thm_bern_unif_1}
\frac{1}{n}\dotprod{\vechtilbar}{\vecgbarest} \geq 1 - \left(3\sqrt{\frac{p+\varepsilon}{2}} - \frac{\lambda k (1-\varepsilon)}{6(4\lambda k + 1)^2}p \right)
- \frac{\lambda \pi^2 M^2 (2k)^{2\alpha + 1}}{n^{2\alpha}} \left(1 - \frac{(1-p)^2}{(4\lambda k + 1)^2} \right) 
\end{equation}
with probability at least
\begin{equation}
1 - e \cdot \exp\left(-\frac{c^{\prime} (1-p)^2 \varepsilon^2 n}{16} \right) - 4\exp\left(-\frac{c \varepsilon^2 p^2 n}{512} \right) 
- 2e \cdot \exp\left(-\frac{c^{\prime} p^2 n (1-\varepsilon)^2}{2304} \right).
\end{equation}

\item\label{thm:bern_unif_2} 
\begin{equation} \label{eq:thm_bern_unif_2}
\frac{1}{n}\dotprod{\vechtilbar}{\vecgbarest} \geq 1 - \left(3\sqrt{\frac{p+\varepsilon}{2}} \right) - \frac{\lambda \pi^2 M^2 (2k)^{2\alpha + 1}}{n^{2\alpha}}  
\end{equation}
with probability at least
\begin{equation}
1 - e \cdot \exp\left(-\frac{c^{\prime} (1-p)^2 \varepsilon^2 n}{16} \right).
\end{equation}
\end{enumerate}
\end{theorem}
%
Both \eqref{eq:thm_bern_unif_1} and \ref{eq:thm_bern_unif_2} give lower bounds on the correlation between $\vechtilbar, \vecgbarest$, holding with 
high probability over the randomness of the noise samples provided $p$ is small enough. 
However, note that the bound in \eqref{eq:thm_bern_unif_1} is strictly better than 
that in \eqref{eq:thm_bern_unif_2}, albeit with a worse bound on success probability. The usefulness of \eqref{eq:thm_bern_unif_1} is seen when 
$p = \Theta(1)$ and for large $n$, in which case \eqref{eq:thm_bern_unif_1} will hold with a sufficiently large probability, thus giving a better result  than \eqref{eq:thm_bern_unif_2}. In contrast, the bound in \eqref{eq:thm_bern_unif_2} is useful for all values of $p \in [0,1/2)$, and moderately large 
values of $n$. This is especially important when $p$ is close to $0$ and $n$ is moderately large. In that scenario, the bound on the success probability for 
\eqref{eq:thm_bern_unif_1} becomes trivial.

\begin{proof}[Proof of Theorem \ref{thm:bern_unif_noise_model}] Our starting point is the following (slightly restricted) version of 
Theorem \ref{thm:arb_noise_model}, which follows in a straightforward manner from Lemma \ref{lem:lowbd_feas_based} 
and Lemma \ref{lem:lowbd_sol_quad_form}.

\begin{theorem} \label{thm:ar_bd_noise_restrict}
Consider the arbitrary bounded
noise model defined  in \eqref{eq:arb_bd_noise_model}, with $\veczbar$ satisfying 
$\norm{\veczbar - \vechtilbar}_2 \leq \delta \sqrt{n}$ for $\delta \in [0,1]$.
If $\lambda < \frac{1}{4k}$, then $\vecgbarest$ is the unique solution 
to \eqref{eq:qcqp_denoise_real} satisfying 
%
\begin{align}
\frac{1}{n}\dotprod{\vechtilbar}{\vecgbarest} \geq 1 - \frac{3\delta}{2} - \frac{1}{2n} \vechtilbar^T \Hbar \vechtilbar
+ \frac{1}{\left(1 + 4\lambda k \right)^2} \left(\frac{1}{2n} \veczbar^T \Hbar \veczbar\right).
\end{align}
\end{theorem}
Next, let us recall from \eqref{eq:unit_complex_circ_rep} that $z_i = \exp(\iota 2 \pi y_i)$, and so $(z_i)_{i=1}^n$ are 
independent, complex-valued random variables. Consequently, we can upper bound the noise term parameter $\delta$, and lower bound 
the quadratic form $\frac{1}{2n} \veczbar^T \Hbar \veczbar$, w.h.p.
This is stated precisely in the following Proposition.
%
\begin{proposition} \label{prop:bern_unif_conc} 
Consider the Bernoulli-Uniform noise model in \eqref{eq:bern_unif_noise_model_1} for some $p \in [0,1]$. 
For absolute constants $c,c^{\prime} > 0$, the following is true.
\begin{enumerate}
\item \label{prop:quad_bound_bern_unif} For any $\varepsilon \in (0,1)$, 
\begin{align}
\frac{1}{2n} \veczbar^T \Hbar \veczbar \geq \frac{\lambda p k (1-\varepsilon)}{6} + (1-p)^2 \frac{1}{2n} \vechtilbar^T \Hbar \vechtilbar  
\end{align}
holds with probability at least 
$$1 - 4\exp\left(-\frac{c \varepsilon^2 p^2 n}{512} \right) 
- 2e \cdot \exp\left(-\frac{c^{\prime} p^2 n (1-\varepsilon)^2}{2304} \right).$$

\item \label{prop:noise_bound_bern_unif} 
For any $\varepsilon \in (0,1)$, 
\begin{equation} \label{eq:bern_unif_z_err}
\norm{\veczbar - \vechtilbar}_2 \leq \sqrt{2(p + \varepsilon)} \sqrt{n} 
\end{equation}
holds with probability at least $1 - e \cdot \exp\left(-\frac{c^{\prime} (1-p)^2 \varepsilon^2 n}{16} \right)$.  
\end{enumerate}
\end{proposition} 
The proof is deferred to Appendix \ref{sec:proof_prop_bern_unif}. 
Observe from \eqref{eq:bern_unif_z_err} and Theorem \ref{thm:ar_bd_noise_restrict} that $\delta = \sqrt{2(p+\varepsilon)} \in [0,1]$ if $p+\varepsilon \leq 1/2$.
Part \ref{thm:bern_unif_1} of Theorem \ref{thm:bern_unif_noise_model} follows by applying the union bound to the two events in 
Proposition \ref{prop:bern_unif_conc}, and 
combining it with Theorem \ref{thm:ar_bd_noise_restrict} and Lemma \ref{lem:upbd_clean_quad_form}. 
To prove part \ref{thm:bern_unif_2} of Theorem \ref{thm:bern_unif_noise_model}, we simply use the 
fact $\frac{1}{2n} \veczbar^T \Hbar \veczbar \geq 0$ (since $\Hbar$ is p.s.d) in 
Theorem \ref{thm:ar_bd_noise_restrict}, and then combine it with part \ref{prop:noise_bound_bern_unif} of Proposition \ref{prop:bern_unif_conc} 
and Lemma \ref{lem:upbd_clean_quad_form}.  This completes the proof. 
\end{proof} 
\subsection{Analysis for Gaussian noise model} \label{subsec:gaussian_noise}
We now analyze the Gaussian noise model in \eqref{eq:gauss_noise_model}. Recall that  
$y_i = (f(x_i) + \eta_i) \mod 1$ where $\eta_i \sim \mathcal{N}(0,\sigma^2)$ i.i.d for $i=1,\dots,n$.  
The following theorem is our main result for this noise model.
%
\begin{theorem} \label{thm:gauss_noise_model}
Consider the Gaussian noise model in \eqref{eq:gauss_noise_model} for some $\sigma \geq 0$. 
If $\lambda < \frac{1}{4k}$ then $\vecgbarest$ is the unique solution to \eqref{eq:qcqp_denoise_real}. 
Moreover, assuming $n \geq 2$, then for any $\varepsilon \in (0,1/2)$ satisfying $(1-\varepsilon)e^{-2\pi^2\sigma^2} \geq 1/2$, 
and absolute constants $c,c^{\prime} > 0$, the following is true. 
\begin{enumerate}
\item\label{thm:gauss_1}
\begin{align} \label{eq:thm_gauss_1}
\frac{1}{n}\dotprod{\vechtilbar}{\vecgbarest} 
\geq 1 &- \left(3\sqrt{\frac{(1 - (1-\varepsilon)e^{-2\pi^2\sigma^2})}{2}} 
- \frac{\lambda k}{6(4\lambda k + 1)^2} (1-\varepsilon)(1-e^{-4\pi^2\sigma^2})^2 \right) \nonumber \\
&- \frac{\lambda \pi^2 M^2 (2k)^{2\alpha + 1}}{n^{2\alpha}} \left(1-\frac{e^{-4\pi^2\sigma^2}}{(4\lambda k + 1)^2}\right)
\end{align}
with probability at least
\begin{align}
1&-2e\exp\left(\frac{-c'(1-\varepsilon)^2(1-e^{-4\pi^2\sigma^2})^4 e^{4\pi^2\sigma^2}}{2304} n \right) 
	- 4\exp\left(-\frac{c \varepsilon^2 (1-e^{-4\pi^2\sigma^2})^4 n}{1024}\right) \nonumber \\
	&- e\cdot \exp\left(-\frac{c' e^{-4\pi^2\sigma^2} n \varepsilon^2}{4}\right).
\end{align}

\item\label{thm:gauss_2} 
\begin{equation} \label{eq:thm_gauss_2}
\frac{1}{n}\dotprod{\vechtilbar}{\vecgbarest} 
\geq 1 - 3\sqrt{\frac{(1 - (1-\varepsilon)e^{-2\pi^2\sigma^2})}{2}} - \frac{\lambda \pi^2 M^2 (2k)^{2\alpha + 1}}{n^{2\alpha}}
\end{equation}
with probability at least
\begin{equation}
1 - e\cdot \exp\left(-\frac{c' e^{-4\pi^2\sigma^2} n \varepsilon^2}{4}\right).
\end{equation}
\end{enumerate}
\end{theorem}
%
%
The statement of Theorem \ref{thm:gauss_noise_model} is very similar in flavour, to that of Theorem \ref{thm:bern_unif_noise_model}. 
Indeed, the correlation lower bounds hold provided the noise level (dictated by $\sigma$) is sufficiently small. Furthermore, note that 
the bound in \eqref{eq:thm_gauss_1} is strictly better than that in \eqref{eq:thm_gauss_2}, albeit with a worse success probability.
If $\sigma = \Theta(1)$, and $n$ is sufficiently large, then \eqref{eq:thm_gauss_1} is a better bound than \eqref{eq:thm_gauss_2}, 
both holding with respectively high probabilities. On the other hand,  \eqref{eq:thm_gauss_2} is meaningful for all admissible values 
of $\sigma$ (satisfying $(1-\varepsilon)e^{-2\pi^2\sigma^2} \geq 1/2$), and moderately large $n$. In comparison, if for instance 
$\sigma \approx 0$ and $n$ is moderately large, then the success probability lower bound for the event \eqref{eq:thm_gauss_1} becomes trivial.

\begin{proof}[Proof of Theorem \ref{thm:gauss_noise_model}] Our starting point is Theorem \ref{thm:ar_bd_noise_restrict} as before.
The following Proposition is an analogous version of Proposition \ref{prop:bern_unif_conc} for the Gaussian noise model, where 
we upper bound the noise term parameter $\delta$, and lower bound the quadratic form $\frac{1}{2n} \veczbar^T \Hbar \veczbar$, w.h.p.
%
\begin{proposition} \label{prop:gauss_conc} 
Consider the Gaussian noise model in \eqref{eq:gauss_noise_model}. 
For absolute constants $c,c^{\prime} > 0$, the following is true.
\begin{enumerate}
\item \label{prop:quad_bound_gauss} For any $\varepsilon \in (0,1)$, 
\begin{align}
	\frac{1}{2n} \veczbar^T \Hbar \veczbar \geq \frac{\lambda k}{6} (1-\varepsilon)(1-e^{-4\pi^2\sigma^2})^2 + \frac{e^{-4\pi^2\sigma^2}}{2n} \vechtil^T \Hbar \vechtil
\end{align}
holds with probability at least 
\begin{align}
			1-2e\exp\left(\frac{-c'(1-\varepsilon)^2(1-e^{-4\pi^2\sigma^2})^4 e^{4\pi^2\sigma^2}}{2304} n \right) 
	- 4\exp\left(-\frac{c \varepsilon^2 (1-e^{-4\pi^2\sigma^2})^4 n}{1024}\right).
\end{align}

\item \label{prop:noise_bound_gauss} 
For any $\varepsilon \in (0,1)$, 
\begin{equation}
\norm{\veczbar - \vechtilbar}_2 \leq \sqrt{2(1 - (1-\varepsilon)e^{-2\pi^2\sigma^2})} \sqrt{n} 
\end{equation}
holds with probability at least $1 - e\cdot \exp(-\frac{c' e^{-4\pi^2\sigma^2} n \varepsilon^2}{4})$.  
\end{enumerate}
\end{proposition}
The proof is deferred to Appendix \ref{sec:proof_prop_gauss}. 
Observe that $\delta = \sqrt{2(1 - (1-\varepsilon)e^{-2\pi^2\sigma^2})} \in [0,1]$ if $(1-\varepsilon)e^{-2\pi^2\sigma^2} \geq 1/2$.
Part \ref{thm:gauss_1} of Theorem \ref{thm:gauss_noise_model} follows by combining Proposition \ref{prop:gauss_conc} 
with Theorem \ref{thm:ar_bd_noise_restrict} and Lemma \ref{lem:upbd_clean_quad_form}. 
To prove part \ref{thm:gauss_2} of Theorem \ref{thm:gauss_noise_model}, we simply use the 
fact $\frac{1}{2n} \veczbar^T \Hbar \veczbar \geq 0$ (since $\Hbar$ is p.s.d) in 
Theorem \ref{thm:ar_bd_noise_restrict}, and then combine it with part \ref{prop:noise_bound_gauss} of Proposition \ref{prop:gauss_conc} 
and Lemma \ref{lem:upbd_clean_quad_form}.  This  altogether completes the proof. 
\end{proof} 

\section{Analysis for multivariate functions  and  arbitrary bounded noise} \label{sec:multiv_analysis_bdnoise}
We now consider the general setting where $f$ is a $d$-variate function with $d \geq 1$. Specifically, 
we consider $f : [0,1]^d \rightarrow \mathbb{R}$ where $f$ is H\"older continuous, meaning that 
for constants $M > 0$ and $\alpha \in (0,1]$, 
\begin{equation} \label{eq:f_smooth_multi}
\abs{f(\vecx) - f(\vecy)} \leq M \norm{\vecx - \vecy}_2^{\alpha}; \quad \forall \ \vecx,\vecy \in [0,1]^d.
\end{equation}
For a positive integer $m$, we consider $f$ to be sampled everywhere on the regular grid 
$\calG = \set{0,\frac{1}{m-1},\dots,\frac{m-2}{m-1},1}^d$, with $n = \abs{\calG} = m^d$ being the 
total number of samples. Let $V = \set{1,2,\dots,m}^d$ be the index set for the points on the grid. 
We assume for concreteness the same arbitrary bounded noise model as in \eqref{eq:arb_bd_noise_model}, i.e.,
\begin{equation} \label{eq:arb_bd_noise_multiv}
y_{\veci} = (f(\vecx_{\veci}) + \delta_i) \bmod 1; \ \abs{\delta_i} \in (0,1/2); \quad \veci \in V, \vecx_{\veci} \in \calG.
\end{equation} 
We will work with a regularization graph $G = (V,E)$ where the set of edges 
now consists of vertices that are close in the $\ell_{\infty}$ metric (also known as  \textit{Chebychev} distance).  
The edge set $E$ is formally defined as
\begin{equation}
E = \set{\set{\veci,\vecj}: \veci, \vecj \in V, \quad \veci \neq \vecj, \quad \norm{\veci - \vecj}_{\infty} \leq k}.
\end{equation}
Note that the degree of each vertex is less than or equal to $(2k + 1)^d - 1$, and is also at least 
$(k+1)^d - 1$ (with equality achieved by nodes at corners of the grid). As in \eqref{eq:unit_complex_circ_rep}, we denote 
\begin{equation} \label{eq:unit_complex_circ_rep_multi}
\htil_{\veci} := \exp(2 \pi \iota f_{\veci}), \quad \; \ z_{\veci} := \exp(2 \pi \iota y_{\veci}); \quad \veci \in V,
\end{equation}
to be the respective representations of the clean mod 1 and noisy mod 1 samples on the unit  circle in $\mathbb{C}$.
Let $\vechtil, \vecz \in \mathbb{C}^n$ denote the respective corresponding vectors, with entries indexed by $V$, and the 
indices sorted in lexicographic increasing order. Then, we simply consider solving the 
relaxed (trust-region) problem in \eqref{eq:qcqp_denoise_complex}, in particular, its equivalent reformulation 
\eqref{eq:qcqp_denoise_real} with the real-valued representations as in \eqref{eq:real_notat} and \eqref{eq:def_H}.
This altogether leads to the following theorem, which is a generalization of Theorem \ref{thm:arb_noise_model} to the 
multivariate case.
\begin{theorem} \label{thm:arb_noise_multiv} 
Consider the arbitrary bounded noise model defined in \eqref{eq:arb_bd_noise_multiv}, with $\veczbar$ satisfying 
$\norm{\veczbar - \vechtilbar}_2 \leq \delta \sqrt{n}$ for $\delta \in [0,1]$. 
Let $m \geq 2$, and let $\calN(\Hbar)$ denote the null space of $\Hbar$.
\begin{enumerate}
\item If $\veczbar \not\perp \calN(\Hbar)$ then $\vecgbarest$ is the unique solution 
to \eqref{eq:qcqp_denoise_real} satisfying 
%
\begin{align}
\frac{1}{n}\dotprod{\vechtilbar}{\vecgbarest} \geq 1 &- \frac{3\delta}{2} - 
\frac{\lambda \pi^2 M^2 (2k)^{2\alpha} d^{2\alpha} [(2k+1)^d - 1]}{n^{2\alpha/d}}  \nonumber \\ 
&+ \frac{1}{(2\lambda [(2k+1)^d - 1] + 1)^2} \left(\frac{1}{2n} \veczbar^T \Hbar \veczbar\right).
\end{align}
%

\item If $\veczbar \perp \calN(\Hbar)$ and $\lambda < \frac{1}{2((2k+1)^d - 1)}$ then $\vecgbarest$ is the unique solution 
to \eqref{eq:qcqp_denoise_real} satisfying
\begin{align}
\hspace{-6mm} \frac{1}{n}\dotprod{\vechtilbar}{\vecgbarest} &\geq 1 - \frac{3\delta}{2} -
\frac{\lambda \pi^2 M^2 (2k)^{2\alpha} d^{2\alpha} [(2k+1)^d - 1]}{n^{2\alpha/d}} \nonumber \\
 &+  \frac{1}{\left(1 + 2\lambda [(2k+1)^d - 1] - 4 \lambda \kappa_{k,d} \sin^2\left(\frac{\pi}{2n}\right) \right)^2} \left(\frac{1}{2n} \veczbar^T \Hbar \veczbar\right).
\end{align}
Here, $\kappa_{k,d} \leq (k+1)^d - 1$ is the vertex connectivity of the graph $G$.
\end{enumerate}
\end{theorem}
As $\Hbar$ is positive semi-definite, we obtain the following useful Corollary of Theorem \ref{thm:arb_noise_multiv}.
\begin{corollary} \label{cor:main_mainthm_multiv_bd}
Consider the arbitrary bounded noise model defined in \eqref{eq:arb_bd_noise_model}, with $\veczbar$ satisfying 
$\norm{\veczbar - \vechtilbar}_2 \leq \delta \sqrt{n}$ for $\delta \in [0,1]$. 
Let $m \geq 2$. If $\lambda < \frac{1}{2((2k+1)^d - 1)}$ then $\vecgbarest$ is the unique solution 
to \eqref{eq:qcqp_denoise_real} satisfying 
%
\begin{equation}
\frac{1}{n}\dotprod{\vechtilbar}{\vecgbarest} 
\geq 1 - \frac{3\delta}{2} - \frac{\lambda \pi^2 M^2 (2k)^{2\alpha} d^{2\alpha} [(2k+1)^d - 1]}{n^{2\alpha/d}}.
\end{equation}
\end{corollary}
%
It can be easily verified that for $d = 1$, Theorem \ref{thm:arb_noise_multiv} reduces to Theorem \ref{thm:arb_noise_model} 
since $\kappa_{k,1} = k$. 
We see that the bounds on the correlation involve terms depending exponentially in the dimension $d$. This 
is the curse of dimensionality typically arising in high-dimensional non-parametric estimation problems. 
Nevertheless, as noted in Section \ref{sec:Intro}, many interesting applications in phase unwrapping arise 
when $d$ is small, specifically, $d = 2,3$. Furthermore, the lower bound on $\frac{1}{n}\dotprod{\vechtilbar}{\vecgbarest}$ 
readily implies the $\ell_2$ error bound on $\norm{\vecgbarest-\vechtilbar}^2_2$ via 
\begin{equation}
\norm{\vecgbarest-\vechtilbar}^2_2 = \norm{\vecgbarest}^2_2 + \norm{\vechtilbar}^2_2 -  2 n  \frac{\dotprod{ \vecgbarest , \vechtilbar}}{n} \leq 2n - 2n ( 1 - \frac{3\delta}{2} - \frac{ c }{n^{2\alpha/d}} ) = O(\delta n + n^{1-\frac{2\alpha}{d}}).
\end{equation}
  
%
%
\begin{proof}[Proof of Theorem \ref{thm:arb_noise_multiv}] 
The proof follows the same structure as that  of Theorem \ref{thm:arb_noise_model}, 
with minor technical changes. 
To begin with, Lemma \ref{lem:lowbd_feas_based} remains unchanged,  meaning that 
any solution $\vecgbarest$ to \eqref{eq:qcqp_denoise_real} satisfies
\begin{equation} \label{eq:lowbd_ang_feas_form_multi}
\frac{1}{n}\dotprod{\vechtilbar}{\vecgbarest} \geq  1 - \frac{3\delta}{2} 
 - \; \;   \frac{1}{2n} \vechtilbar^T \Hbar \vechtilbar + \; \; \frac{1}{2n} \vecgbarest^T \Hbar \vecgbarest. 
\end{equation}
We then need the following analogue of Lemma \ref{lem:upbd_clean_quad_form} to upper bound the term $\frac{1}{2n} \vechtilbar^T \Hbar \vechtilbar$ 
using the smoothness of the ground truth $\vechtilbar \in \matR^{2n}$.
%
%
\begin{lemma} \label{lem:upbd_clean_quad_form_multi}
For $m \geq 2$, it holds true that 
\begin{equation}
\frac{1}{2n} \vechtilbar^T \Hbar \vechtilbar \leq \frac{\lambda \pi^2 M^2 (2k)^{2\alpha} d^{2\alpha} [(2k+1)^d - 1]}{n^{2\alpha/d}}, 
\end{equation} 
where $\alpha \in (0,1]$ and $M > 0$ are related to the smoothness of $f$ and defined in \eqref{eq:f_smooth_multi}, and  
$\lambda \geq 0$ is the regularization parameter in \eqref{eq:orig_denoise_hard}.
\end{lemma}
%
\begin{proof}
The proof is very similar to that of Lemma \ref{lem:upbd_clean_quad_form}, so we only focus on the 
changes. To begin with, as in \eqref{eq:lem2_temp1}, we have 
\begin{equation} \label{eq:lem2_temp1_multi}
\frac{1}{2n} \vechtilbar^T \Hbar \vechtilbar \leq \frac{\lambda}{2n} \abs{E} \max_{\set{\veci,\vecj} \in E} \abs{\htil_{\veci} - \htil_{\vecj}}^2. 
\end{equation}
For each $\veci \in V$ we have $\text{deg}(\veci) \leq (2k+1)^d-1$, hence $\abs{E} \leq \frac{n}{2}[(2k+1)^d-1]$. 
Next, for any $\set{\veci,\vecj} \in E$, we have
\begin{equation} \label{eq:lem2_temp2_multi}
\abs{f_{\veci} - f_{\vecj}} \leq M \norm{\vecx_{\veci} - \vecx_{\vecj}}_2^{\alpha} 
\leq M \left(\frac{k}{m-1}\right)^{\alpha}d^{\alpha/2} \leq M\left(\frac{2k}{m}\right)^{\alpha} d^{\alpha/2}, 
\end{equation}
if $m \geq 2$ (since then $m-1 \geq m/2$). This leads to
%
%
\begin{align}
\abs{\htil_{\veci} - \htil_{\vecj}} &\leq 2\pi \abs{f_{\vecj} - f_{\veci}} \leq \frac{2\pi M (2k)^{\alpha} d^{\alpha/2}}{m^{\alpha}} 
\quad (\text{using } \eqref{eq:lem2_newround_temp1} \text{ and } \eqref{eq:lem2_temp2_multi}).  \label{eq:lem2_temp3_multi}
\end{align}
Plugging \eqref{eq:lem2_temp3_multi} in \eqref{eq:lem2_temp1_multi} with the bound $\abs{E} \leq \frac{n}{2}[(2k+1)^d-1]$ 
yields the desired bound.
\end{proof} 
Finally, we lower bound the term $\frac{1}{2n} \vecgbarest^T \Hbar \vecgbarest$ in \eqref{eq:lowbd_ang_feas_form_multi} 
via an analogue of Lemma \ref{lem:lowbd_sol_quad_form} outlined below. 
%
%
\begin{lemma} \label{lem:lowbd_sol_quad_form_multi}
Denoting $\calN(\Hbar)$ to be the null space of $\Hbar$, the following holds 
for the solution $\vecgbarest$ to \eqref{eq:qcqp_denoise_real}.
\begin{enumerate}
\item If $\veczbar \not\perp \calN(\Hbar)$ then $\vecgbarest$ is unique and 
\begin{equation}
\frac{1}{2n} \vecgbarest^T \Hbar \vecgbarest \geq \frac{1}{\left(1 + 2\lambda ([2 k + 1]^d-1) \right)^2} \left(\frac{1}{2n} \veczbar^T \Hbar \veczbar\right).
\end{equation}

\item If $\veczbar \perp \calN(\Hbar)$ and $\lambda < \frac{1}{2[(2k+1)^d-1]}$, then $\vecgbarest$ is unique and 
{\small
\begin{equation}
\frac{1}{2n} \vecgbarest^T \Hbar \vecgbarest \geq 
\frac{1}{\left(1 + 2\lambda ([2k+1]^d-1) - 4 \lambda \kappa_{k,d} \sin^2\left(\frac{\pi}{2n}\right)\right)^2} \left(\frac{1}{2n} \veczbar^T \Hbar \veczbar\right).
\end{equation}}
Here, $\kappa_{k,d} \leq (k+1)^d - 1$ is the vertex connectivity of the graph $G$.
\end{enumerate}
\end{lemma}
\begin{proof}
The outline of the proof is the same as the proof of Lemma \ref{lem:lowbd_sol_quad_form}, so we will only 
highlight the technical changes. 
To begin with, the largest eigenvalue of $L$, namely $\beta_n(L)$, now satisfies 
$\beta_n(L) \leq 2([2k+1]^d - 1)$. This is easily verified\footnote{Denote $L_{ij}$ to be the $(i,j)^{\text{th}}$ entry of $L$. 
Using Gershgorin's disk theorem, each eigenvalue of $L$ lies in $\bigcup_{i=1}^{2n}\set{x:\abs{x-L_{ii}} \leq \sum_{j \neq i} \abs{L_{ij}}}$. 
Since $L_{ii} \leq [2k+1]^d - 1$ and $\sum_{j \neq i} \abs{L_{ij}} \leq [2k+1]^d - 1$ holds for each $i$, the claim follows.} through Gershgorin's disk theorem. 
Therefore, the eigenvalues of $\Hbar$ now satisfy
\begin{equation} \label{eq:lem3_temp0_multi}
0 = \lambda_1(\Hbar) = \lambda_2(\Hbar) < \lambda_3(\Hbar) \leq \cdots \leq \lambda_{2n}(\Hbar) \leq 2\lambda ([2k+1]^d - 1)
\end{equation}
Recall that $\set{\vecq_j}_{j=1}^{2n}$ denote the corresponding eigenvectors of $\Hbar$, with 
$\calN(\Hbar) = \text{span}\set{\vecq_1,\vecq_2}$. We now consider the two cases separately below. 
\begin{enumerate}
\item Consider the case where $\veczbar \not\perp \calN(\Hbar)$. 
We know that $\vecgbarest = 2(2\Hbar + \mu^{*}\matI)^{-1}\veczbar$ 
for a unique $\mu^{*} \in (0,\infty)$, and so $\vecgbarest$ is the unique solution to \eqref{eq:qcqp_denoise_real} by 
Lemma \ref{lemma:qcqp_denoise_real} since $2\Hbar + \mu^{*}\matI \succ 0$. We again have 
$\mu^{*} \leq 2$ as in \eqref{eq:lem3_temp2}. Finally, we can lower bound 
$\frac{1}{2n} \vecgbarest^T \Hbar \vecgbarest$ as follows.
\begin{align}
\frac{1}{2n} \vecgbarest^T \Hbar \vecgbarest 
& = \frac{2}{n} \sum_{j=1}^{2n} \frac{\dotprod{\veczbar}{\vecq_j}^2 \lambda_j(\Hbar)}{(2\lambda_j(\Hbar) + \mu^{*})^2} \\
&\geq \frac{2}{n (4\lambda ([2k+1]^d - 1) + 2)^2} \sum_{j=1}^{2n}\dotprod{\veczbar}{\vecq_j}^2 \lambda_j(\Hbar) \quad (\text{from} \ \eqref{eq:lem3_temp0_multi} \ \text{and since} \ \mu^{*} \leq 2) \\ 
&= \frac{1}{(2\lambda ([2k+1]^d - 1) + 1)^2} \left(\frac{1}{2n} \veczbar^T \Hbar \veczbar\right).
\end{align}
%
%
\item Let us now consider the case where $\veczbar \perp \calN(\Hbar)$. The analogous of  \eqref{eq:lem3_case2_temp1} to the multivariate case is given by 
\begin{equation} \label{eq:lem3_case2_temp1_multi}
\phi(0) = \sum_{j=3}^{2n} \frac{\dotprod{\veczbar}{\vecq_j}^2}{\lambda_j(\Hbar)^2} \geq  \frac{n}{4\lambda^2 ([2k+1]^d-1)^2}, 
\end{equation}
since $\lambda_j(\Hbar) \leq 2\lambda ([2k+1]^d - 1)$ for all $j$ (recall \eqref{eq:lem3_temp0_multi}).
Therefore, the choice $\lambda < \frac{1}{2[(2k+1)^d-1]}$  clearly  implies $\phi(0) > n$, and consequently that the solution $\vecgbarest$ is unique. 
Assuming $\lambda < \frac{1}{2[(2k+1)^d-1]}$ holds,  we arrive at  $\mu^{*} \in (0,2 - 2\lambda_3(\Hbar))$ as 
previously shown in \eqref{eq:lem3_case2_temp2}. As a sanity check, note that  $2 - 2\lambda_3(\Hbar) > 0$, in light of 
\eqref{eq:lem3_temp0_multi}  stating that $   \lambda_3(\Hbar)  \leq 2\lambda ([2k+1]^d - 1)$.
Consequently, we can provide the following  lower bound for $\frac{1}{2n} \vecgbarest^T \Hbar \vecgbarest$. 
\begin{align}
\frac{1}{2n} \vecgbarest^T \Hbar \vecgbarest 
& = \frac{2}{n} \sum_{j=1}^{2n} \frac{\dotprod{\veczbar}{\vecq_j}^2 \lambda_j(\Hbar)}{(2\lambda_j(\Hbar) + \mu^{*})^2} \\
&\geq \frac{2}{n (4\lambda ([2k+1]^d - 1) + 2 - 2\lambda_3(\Hbar))^2} \sum_{j=1}^{2n}\dotprod{\veczbar}{\vecq_j}^2 \lambda_j(\Hbar) \\
&= \frac{1}{(1 + 2\lambda ([2k+1]^d - 1) - \lambda_3(\Hbar))^2} \left(\frac{1}{2n} \veczbar^T \Hbar \veczbar\right), \label{eq:lem3_case2_temp3_multi}
\end{align}
where the second inequality follows from \eqref{eq:lem3_temp0_multi} and the fact $\mu^{*} \in (0,2 - 2\lambda_3(\Hbar))$.
Finally, using Theorem \ref{thm:fiedler_ss_eigval} we readily obtain the bound $\lambda_3(\Hbar) \geq 4 \lambda \kappa_{k,d} \sin^2\left(\frac{\pi}{2n}\right)$,   
where $\kappa_{k,d}$ denotes the vertex connectivity of $G$. The minimum degree of $G$ is $(k+1)^d -1$ and so 
$\kappa_{k,d} \leq (k+1)^d-1$. Plugging the lower bound on $\lambda_3(\Hbar)$ into \eqref{eq:lem3_case2_temp3_multi} completes the proof.
\end{enumerate}
\end{proof}
This completes the proof of Theorem \ref{thm:arb_noise_multiv}.
\end{proof}
%
\paragraph{Unwrapping the samples of $f$.} Once we obtain the denoised $f$ mod $1$ samples, we can recover an estimate of the original samples of $f$ (up to 
a global shift) by using, for instance, the OLS method outlined in Section \ref{subsec:unwrap_stage_and_algo}. This is the approach 
we adopt in our simulations, but one could of course consider using more sophisticated unwrapping algorithms. Exploring other approaches for the unwrapping stage is an interesting direction for future work.

\paragraph{Extensions to random noise models.} One could also consider extending the results to random noise models (Bernoulli-Uniform, Gaussian) as was shown in Sections \ref{subsec:bern_unif_noise} and  \ref{subsec:gaussian_noise} for the $d = 1$ case. We expect the proof outline to be very similar to that of Theorems \ref{thm:bern_unif_noise_model} and \ref{thm:gauss_noise_model} with minor technical changes. 

\section{Numerical experiments for the univariate case via TRS-based modulo denoising} \label{sec:num_exps}
%
This section contains numerical experiments\footnote{Code is publicly available online at \url{http://www.stats.ox.ac.uk/~cucuring/ModuloDenoising.htm}} for the univariate case ($d = 1$) for the following three noise models  discussed in Section \ref{sec:ProblemSetup}.
\begin{enumerate}
\item \textbf{Arbitrary bounded} noise model  \eqref{eq:arb_bd_noise_model}, 
analyzed theoretically in Section \ref{sec:bound_noise_analysis}. In particular, we experiment with the Uniform model, with samples generated uniformly at random in $[-\gamma,\gamma]$ for bounded $\gamma$. Results for this model are shown 
in Appendix \ref{sec:num_exps_Bounded}.
\item \textbf{Bernoulli-Uniform} noise model	\eqref{eq:bern_unif_noise_model}, with guarantees put forth in Section \ref{subsec:bern_unif_noise}.
\item \textbf{Gaussian} noise model 	\eqref{eq:gauss_noise_model}, analyzed theoretically  in Section \ref{subsec:gaussian_noise}.
\end{enumerate}

The function $f$ we consider throughout all the experiments in this Section is  given by 
\begin{equation}
f(x) = 4x \cos^2(2\pi x) - 2 \sin^2(2\pi x).
\label{def:f1}
\end{equation}
For each experiment, we average the results over 20 trials,  and show the RMSE error on a log scale, both for the denoised samples of $f$ mod 1  and the final $f$ estimates. For the latter, we compute the RMSE after we mod out the best global shift\footnote{Any algorithm that takes as input the mod 1 samples will be able to recover $f$ only up to a global shift. We mod out the global shift via a simple procedure, that first computes the offset at each sample point $f_i - \hat{f}_i$, and then considers the mode of this distribution, after a discretization step. More specifically, we compute a histogram of the offsets, and consider the center of the bucket that gives the mode of the distribution as our final global shift estimate with respect to the ground truth. }. 
For each noise model, we compare the performance of three methods.
\begin{itemize}
\item 
\textbf{OLS} denotes the algorithm based on the least-squares formulation \eqref{eq:ols_unwrap_lin_system} used to recover samples of $f$, 
and works directly with noisy  mod 1 samples. The final estimated $f$  mod 1 values are then obtained as the corresponding  
mod 1 versions.
\item  \textbf{QCQP} denotes Algorithm \ref{algo:two_stage_denoise} where the unwrapping stage is 
performed via \textbf{OLS} \eqref{eq:ols_unwrap_lin_system}.
\item  \textbf{iQCQP} denotes an iterated version of \textbf{QCQP}, wherein we repeatedly denoise
the noisy $f \mod 1$ estimates (via Stage $1$ of Algorithm \ref{algo:two_stage_denoise}) for $l  \in  \{ 3,5,10 \} $ iterations, 
and finally perform the unwrapping stage via \textbf{OLS} \eqref{eq:ols_unwrap_lin_system}, to recover the final  sample estimates of $f$.
\end{itemize}
%


\subsection{Numerical experiments: Gaussian Model} \label{sec:num_exps_Gaussian}
Figure \ref{fig:instances_f1_Gaussian_delta_cors} shows noisy instances of the Gaussian noise model, for $n=500$ samples. 
The scatter plots on the top row show that, as the noise  level  
increases, the function \eqref{eq:qta_recovery_rule} will produce more and more errors in \eqref{eq:ols_unwrap_lin_system}, while the  remaining plots show the corresponding f mod 1 signal (clean, noisy, and denoised via \textbf{QCQP}) for three levels of noise. 

Figure \ref{fig:instances_f1_Gaussian} depicts instances of the recovery process, highlighting the noise level at which each method shows a significant decrease in performance. \textbf{iQCQP} shows surprisingly good performance even at very high levels of noise, ($\sigma=0.17$). 

Figures \ref{fig:Sims_f1_Gaussian_fmod1}, respectively  \ref{fig:Sims_f1_Gaussian_f},  plot the recovery errors (averaged over 20 runs) for the denoised values of $f$ mod 1,  respectively the estimated $f$ samples, as we scan across the input parameters $k 
\in \{2,3,5\}$ and  $ \lambda   \in  \{ 0.03, 0.1, 0.3, 0.5, 1\}$, at varying levels of noise $ \sigma \in [0, 0.15]$. We remark that $k=5$ and $\lambda=1$ most often lead to the worst performance. For most parameter combinations  \textbf{iQCQP} improves on \textbf{QCQP} (except for very high values of $\lambda$ and $k$).  The best recovery errors are obtained for $k=2$ and higher values of $\lambda$.

\begin{figure}[!ht]
\centering

\subcaptionbox[]{  $\sigma=0.1$
}[ 0.32\textwidth ]
{\includegraphics[width=0.31\textwidth] {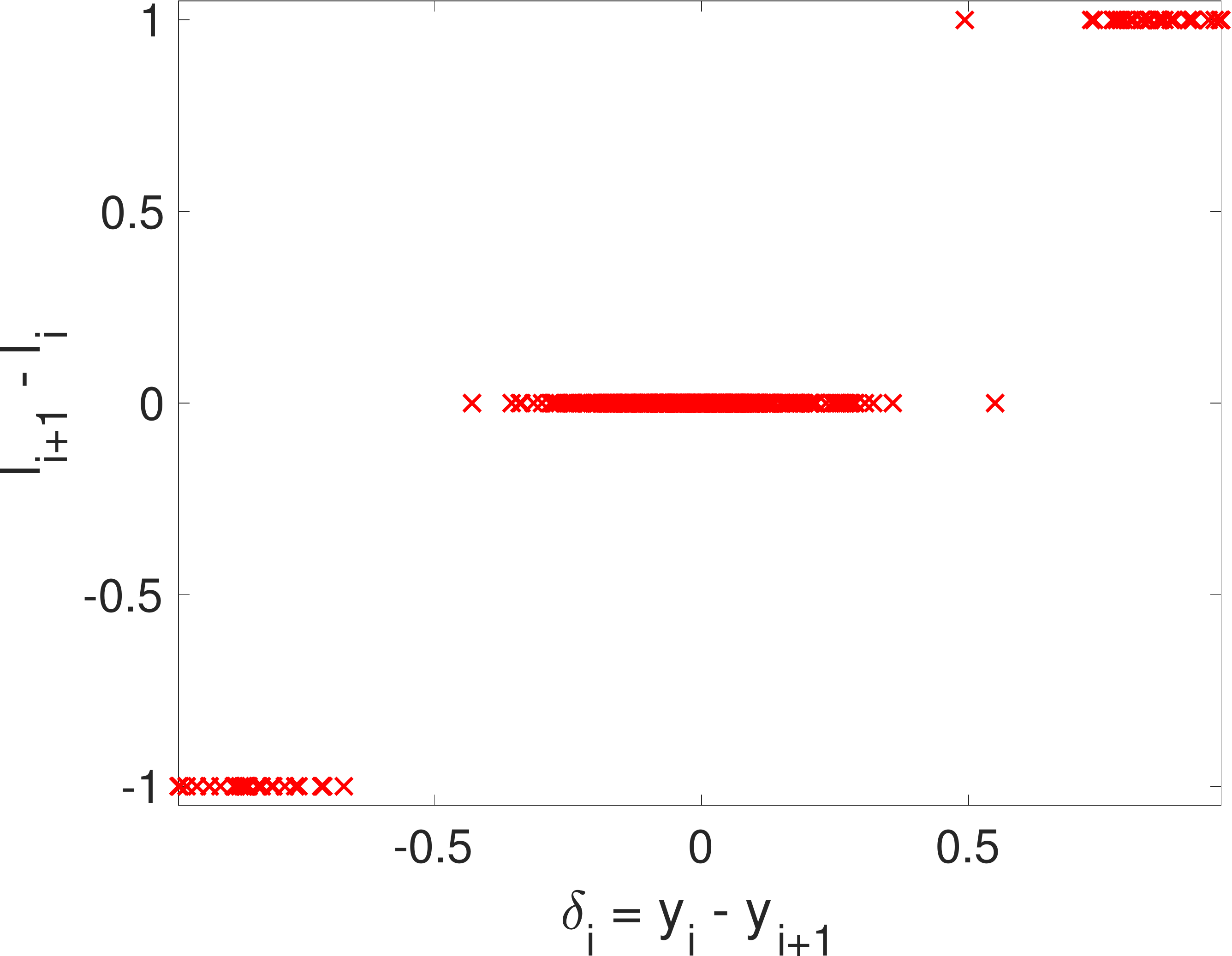} }
%
\subcaptionbox[]{  $\sigma=0.15$
}[ 0.32\textwidth ]
{\includegraphics[width=0.31\textwidth] {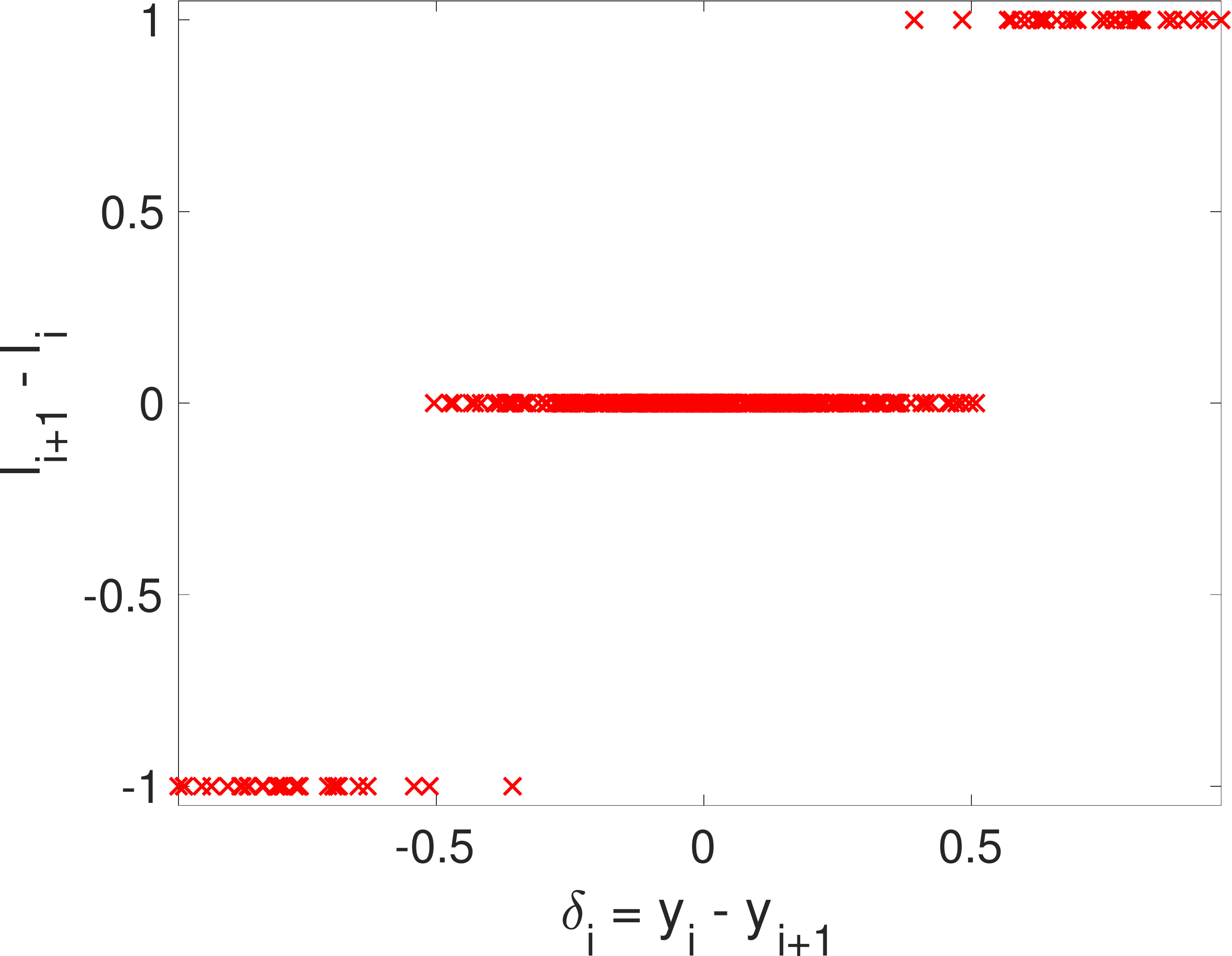} }
%
\subcaptionbox[]{  $\sigma=0.2$
}[ 0.32\textwidth ]
{\includegraphics[width=0.31\textwidth] {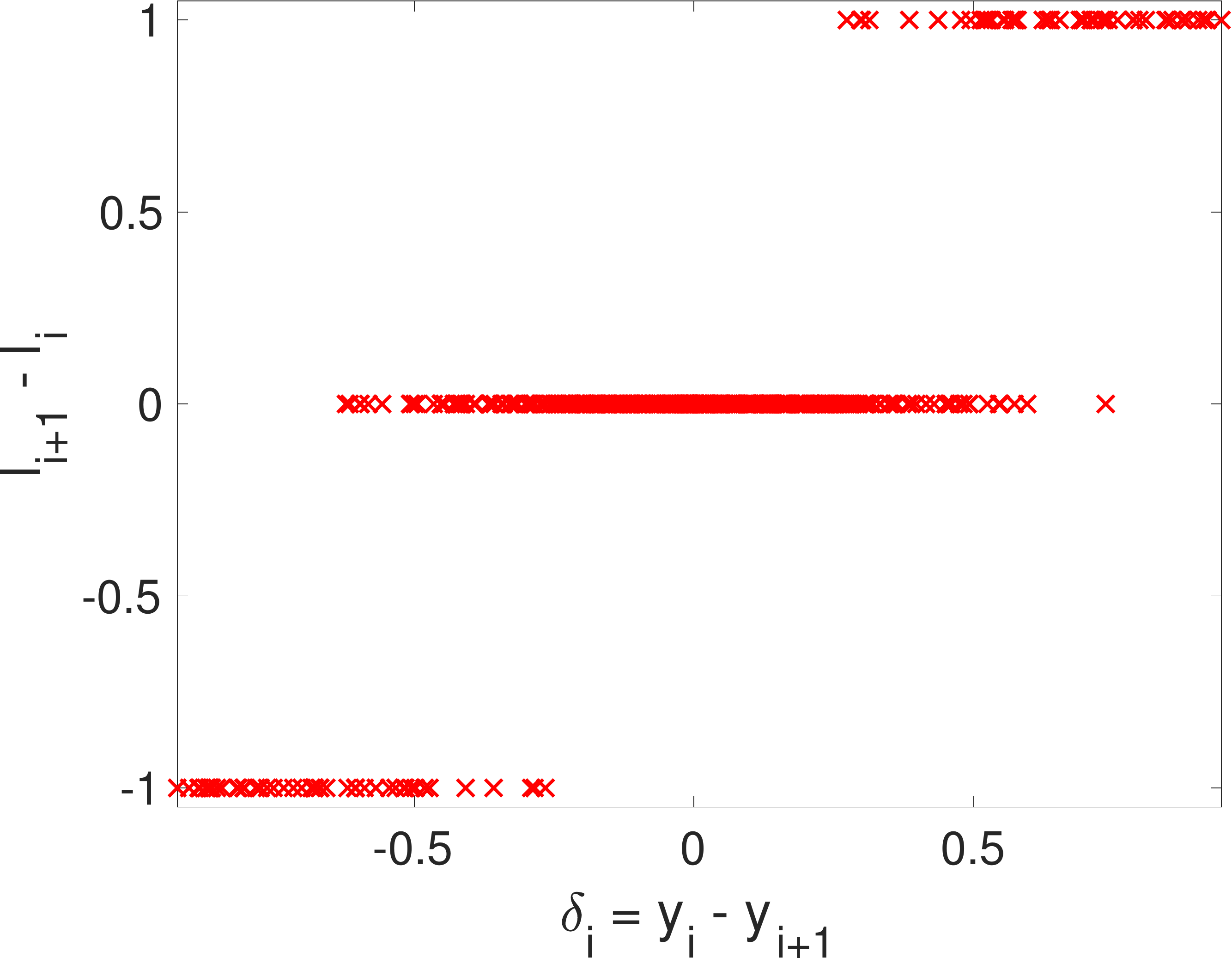} }
\vspace{2mm}

\subcaptionbox[]{  $\sigma=0.1$
}[ 0.96\textwidth ]
{\includegraphics[width=0.95\textwidth] {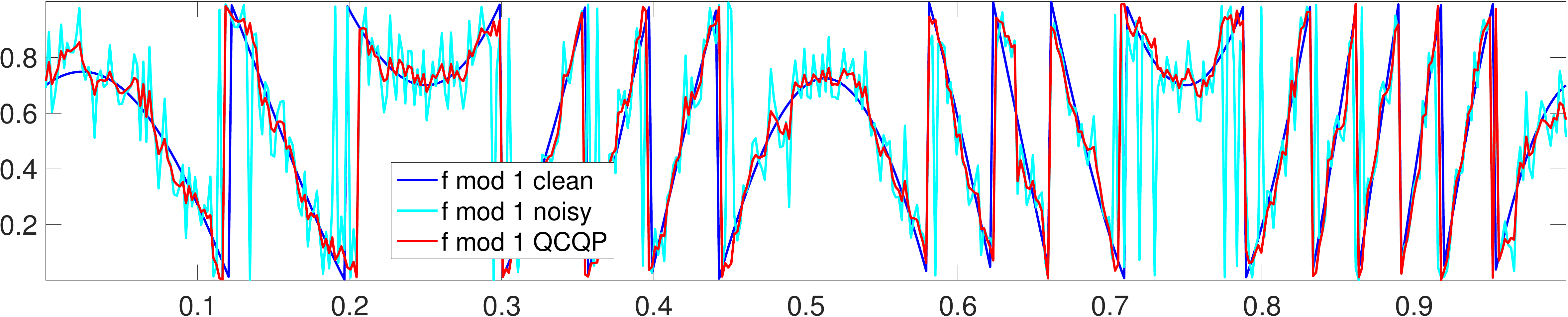} }
%

\vspace{2mm}
\subcaptionbox[]{  $\sigma=0.15$
}[ 0.96\textwidth ]
{\includegraphics[width=0.95\textwidth] {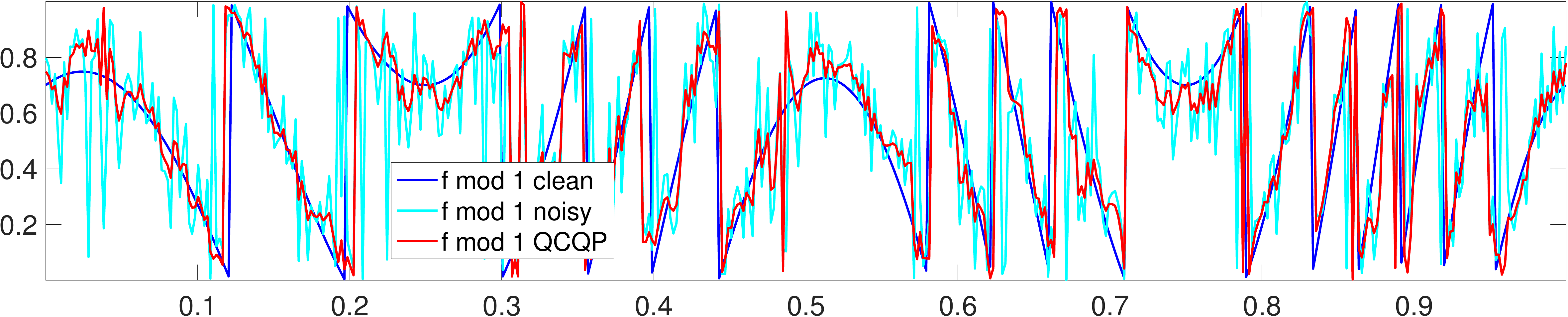} }
%

\vspace{2mm}
\subcaptionbox[]{  $\sigma=0.2$
}[ 0.96\textwidth ]
{\includegraphics[width=0.95\textwidth] {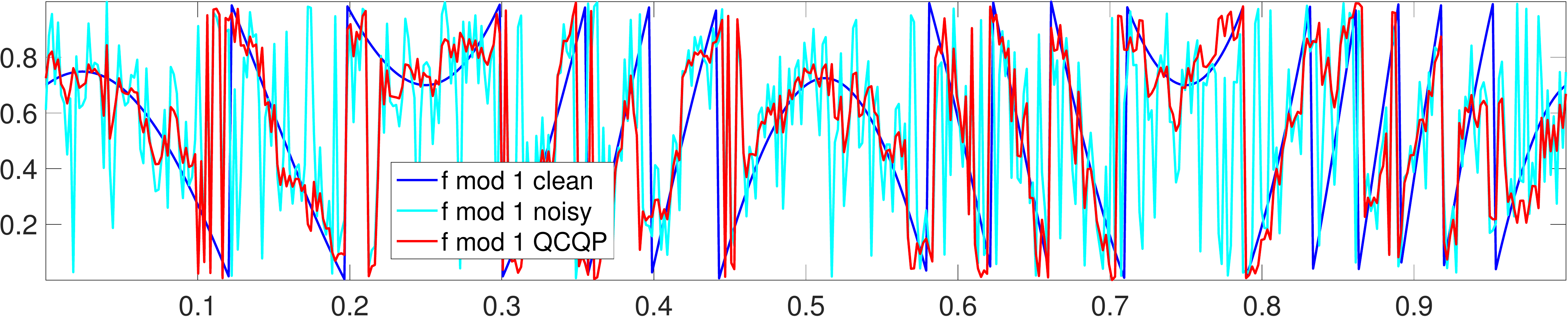} }
%

\captionsetup{width=0.95\linewidth}
\caption[Short Caption]{ Noisy instances of the Gaussian noise model ($n=500$). Top row (a)-(c) shows scatter plots of change in $y$ (the observed noisy f mod 1 values) versus change in $l$ (the noisy quotient). Plots (d)-(f) show  the clean f mod 1 values (blue), the noisy f mod 1 values (cyan) and the denoised    (via \textbf{QCQP}) f mod 1 values (red).     \textbf{QCQP} denotes Algorithm \ref{algo:two_stage_denoise} without  the unwrapping stage performed by \textbf{OLS} \eqref{eq:ols_unwrap_lin_system}.
}
\label{fig:instances_f1_Gaussian_delta_cors}
\end{figure}

\begin{figure}[!ht]
\centering
\subcaptionbox[]{  $\sigma=0.05$, \textbf{OLS}
}[ 0.32\textwidth ]
{\includegraphics[width=0.27\textwidth] {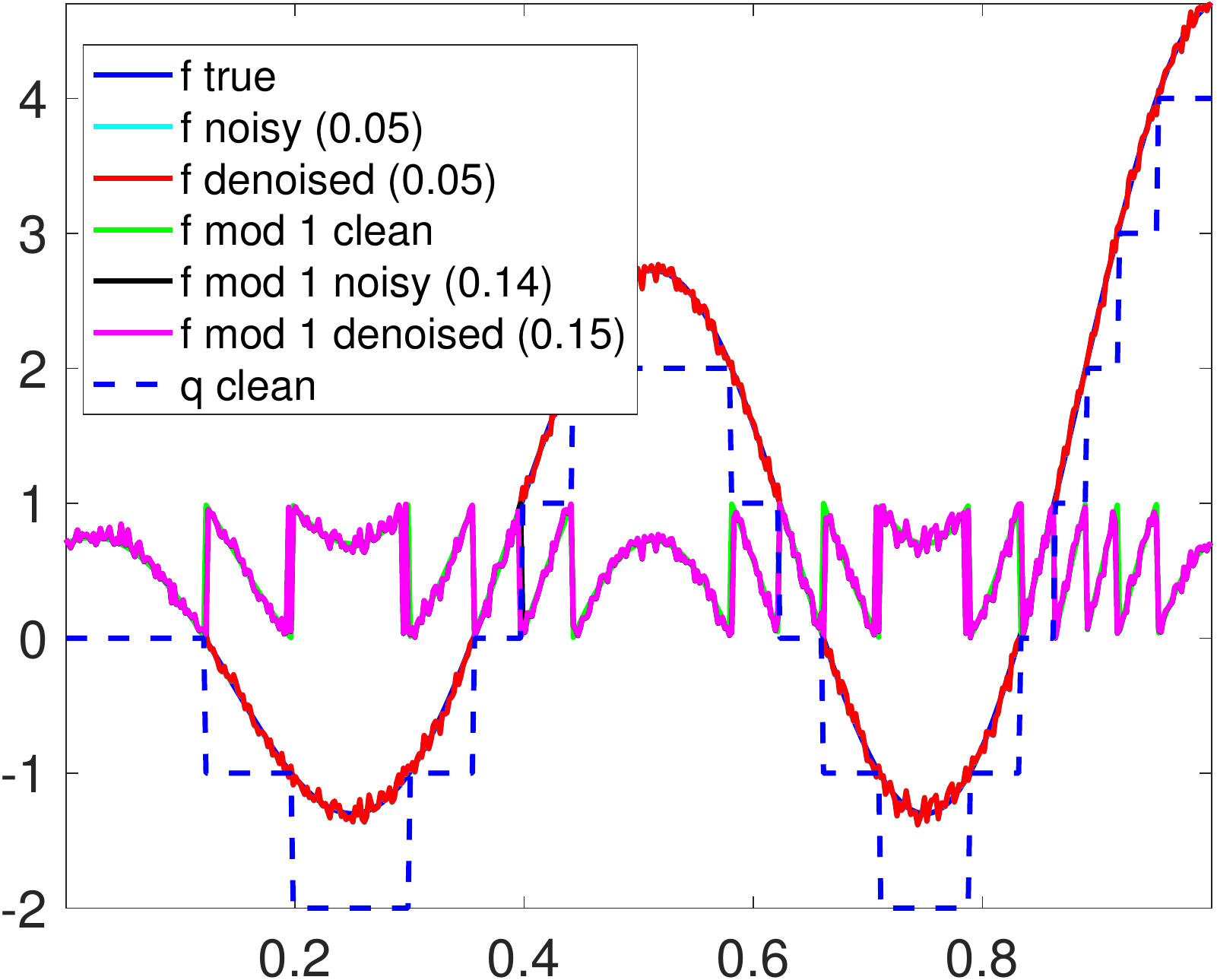} }
%
\subcaptionbox[]{  $\sigma=0.05$, \textbf{QCQP}
}[ 0.32\textwidth ]
{\includegraphics[width=0.27\textwidth] {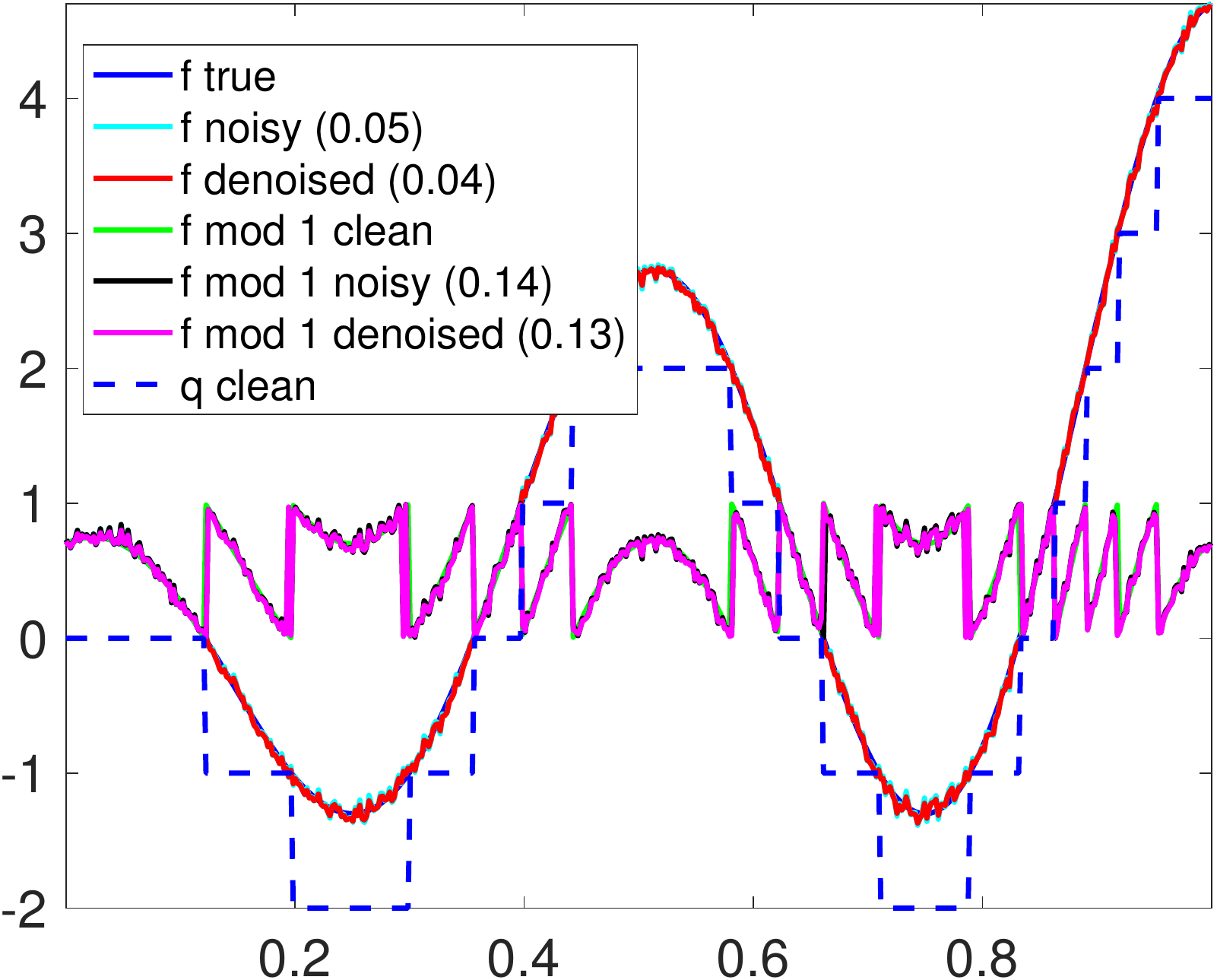} }
%
\subcaptionbox[]{  $\sigma=0.05$, \textbf{iQCQP}  (10  iters.)
}[ 0.32\textwidth ]
{\includegraphics[width=0.27\textwidth] {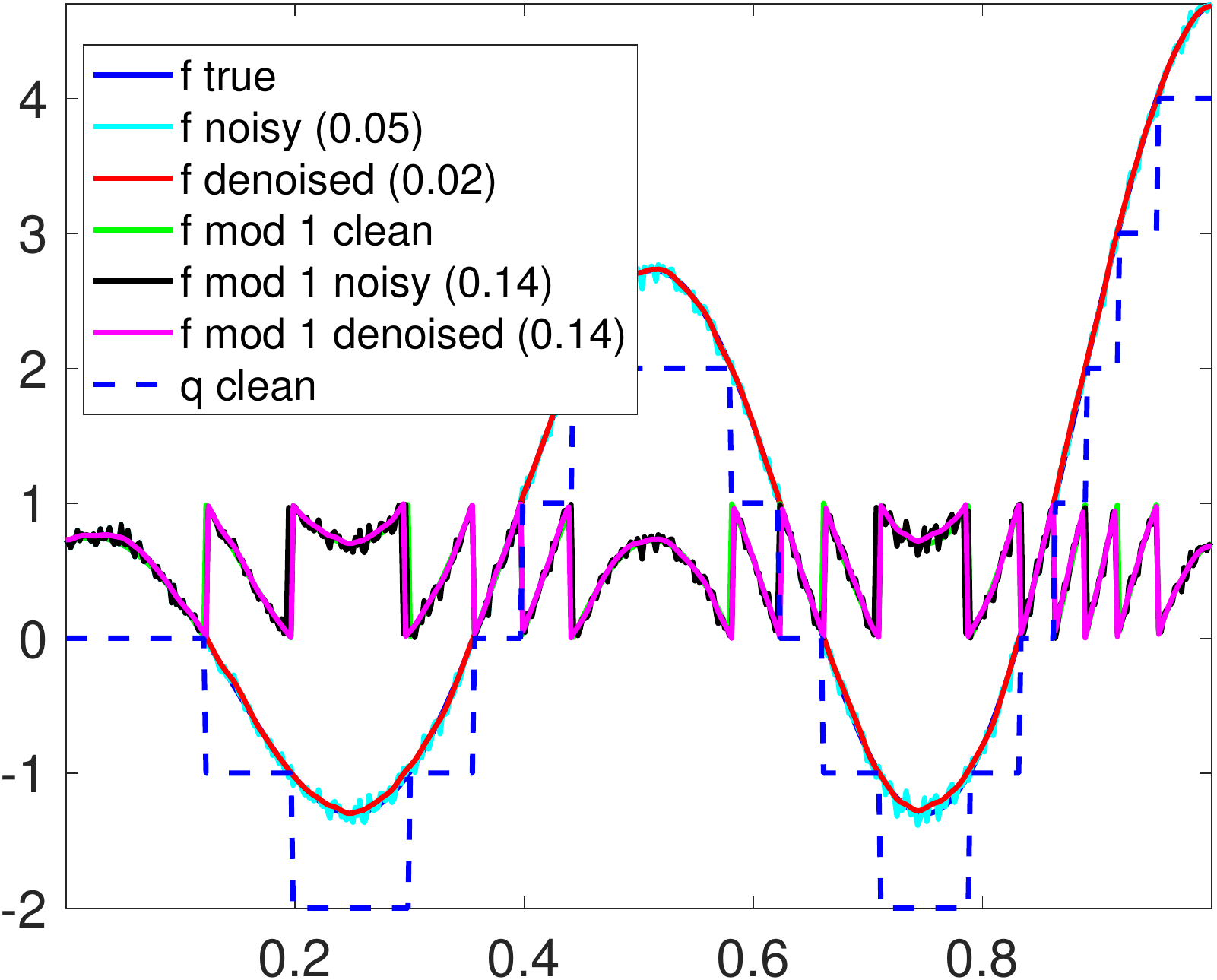} }
%
%
%
\vspace{4mm}

\subcaptionbox[]{  $\sigma=0.13$, \textbf{OLS}
}[ 0.32\textwidth ]
{\includegraphics[width=0.27\textwidth] {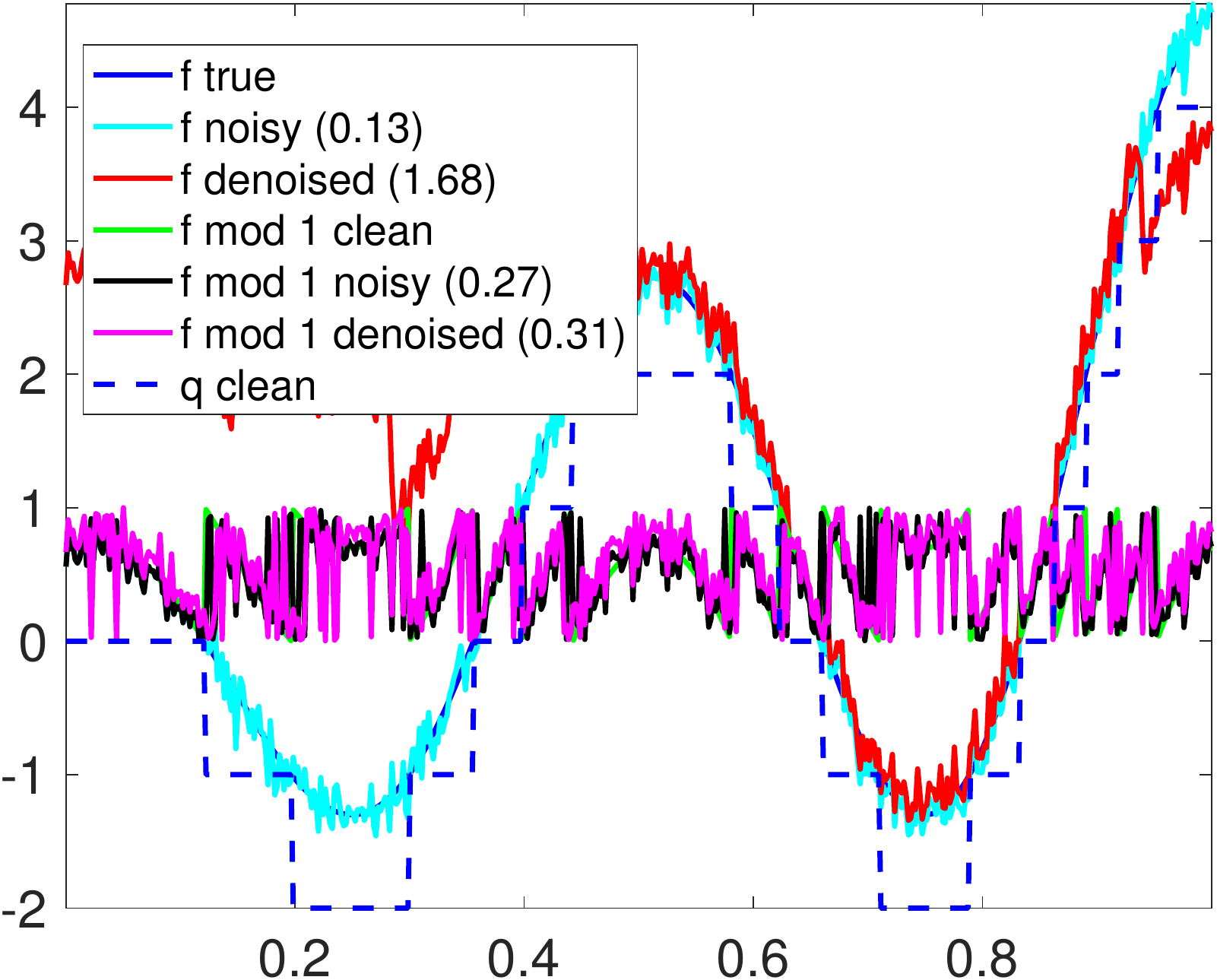} }
%
\subcaptionbox[]{  $\sigma=0.13$, \textbf{QCQP}
}[ 0.32\textwidth ]
{\includegraphics[width=0.27\textwidth] {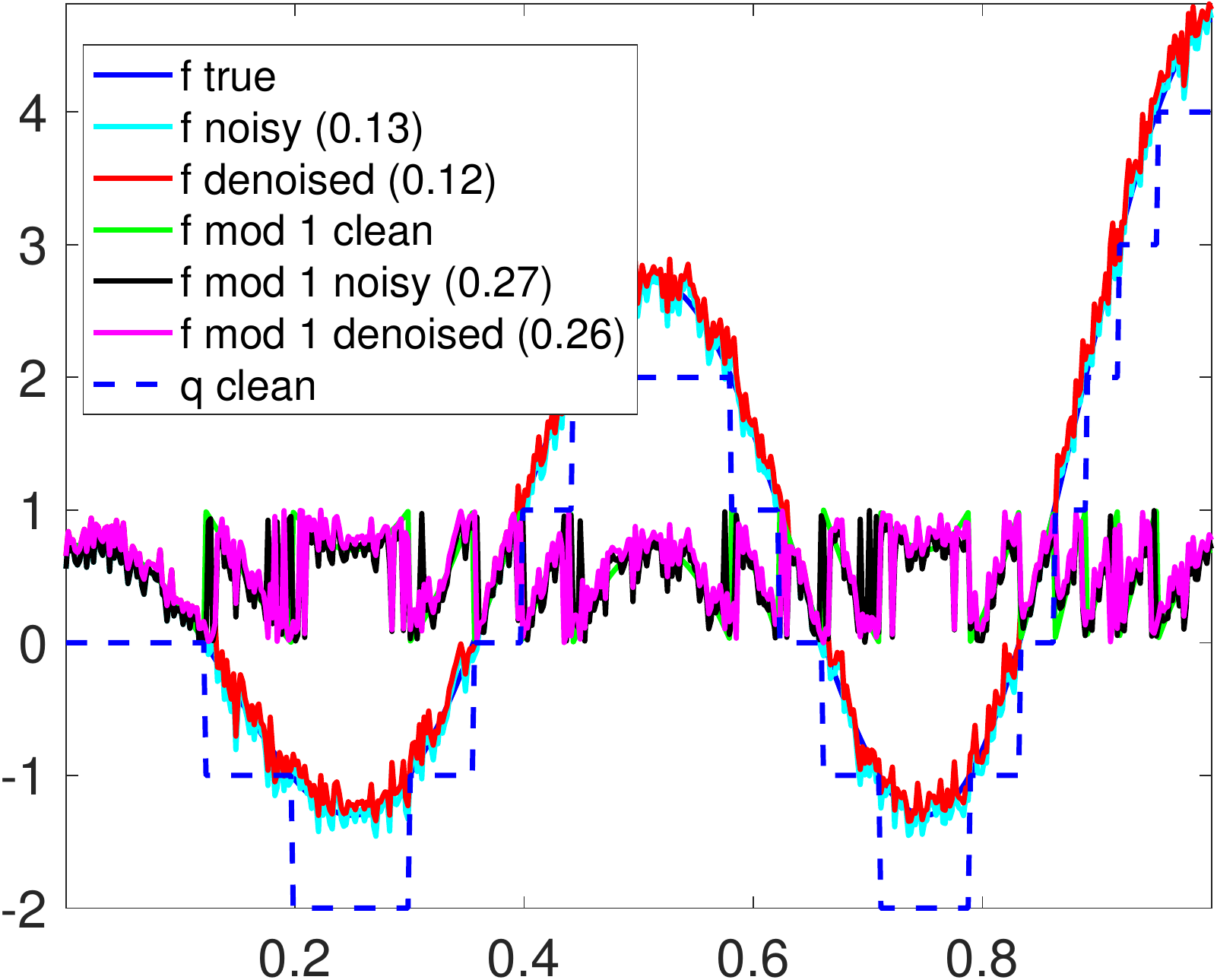} }
%
\subcaptionbox[]{  $\sigma=0.13$, \textbf{iQCQP} (10 iters.)
}[ 0.32\textwidth ]
{\includegraphics[width=0.27\textwidth] {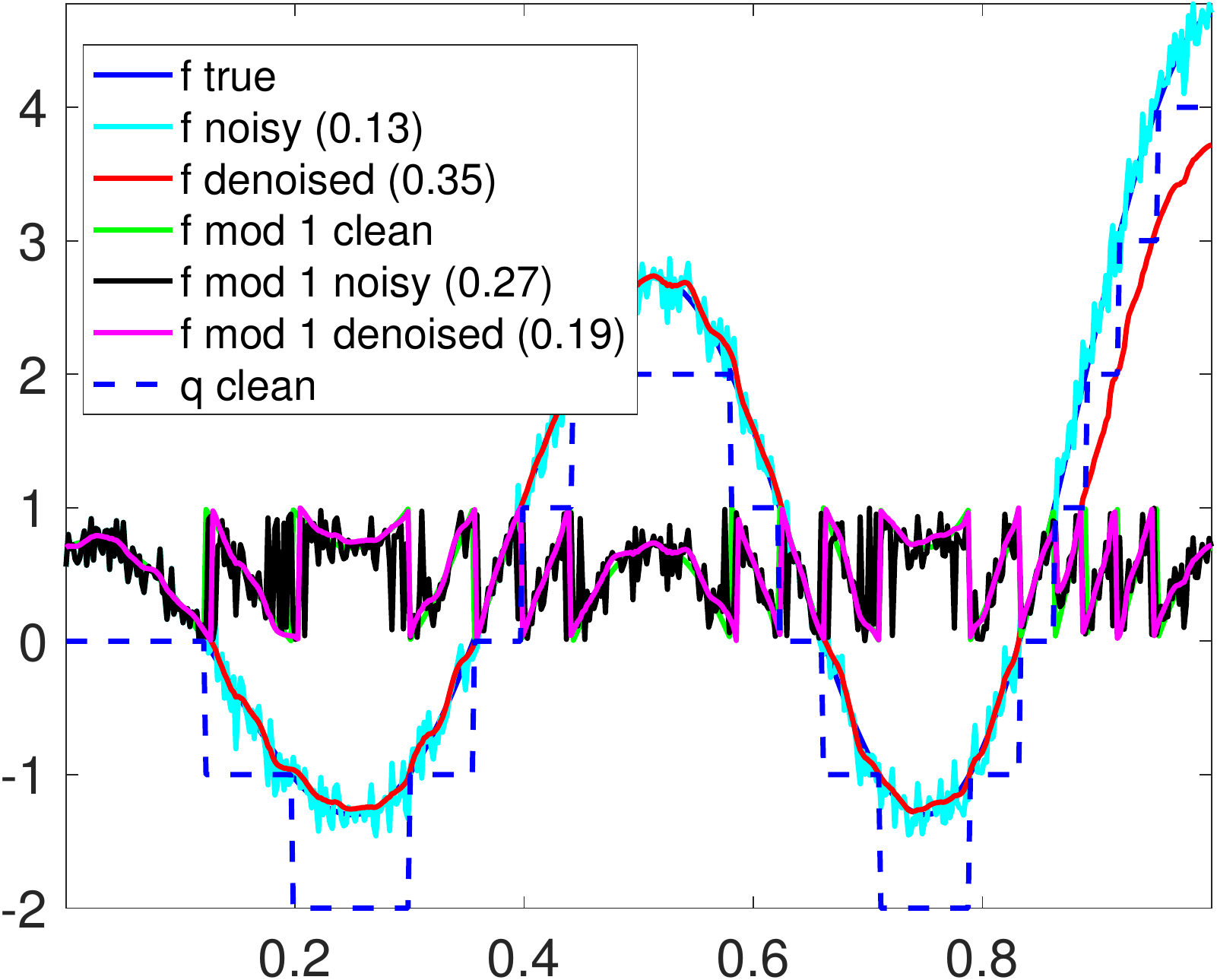} }
%
%

\vspace{4mm}
\subcaptionbox[]{  $\sigma=0.17$, \textbf{OLS}
}[ 0.32\textwidth ]
{\includegraphics[width=0.27\textwidth] {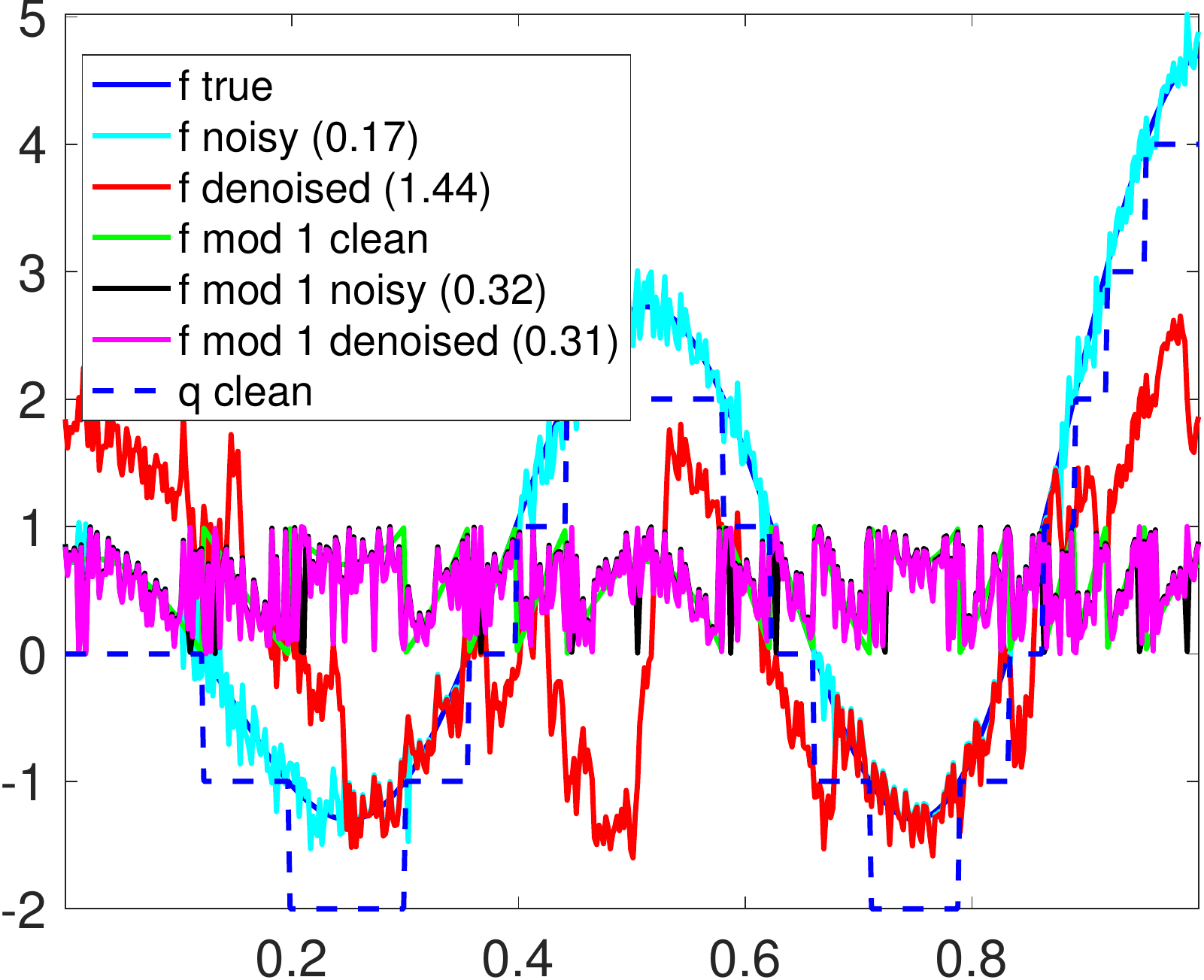} }
%
\subcaptionbox[]{  $\sigma=0.17$, \textbf{QCQP}
}[ 0.32\textwidth ]
{\includegraphics[width=0.27\textwidth] {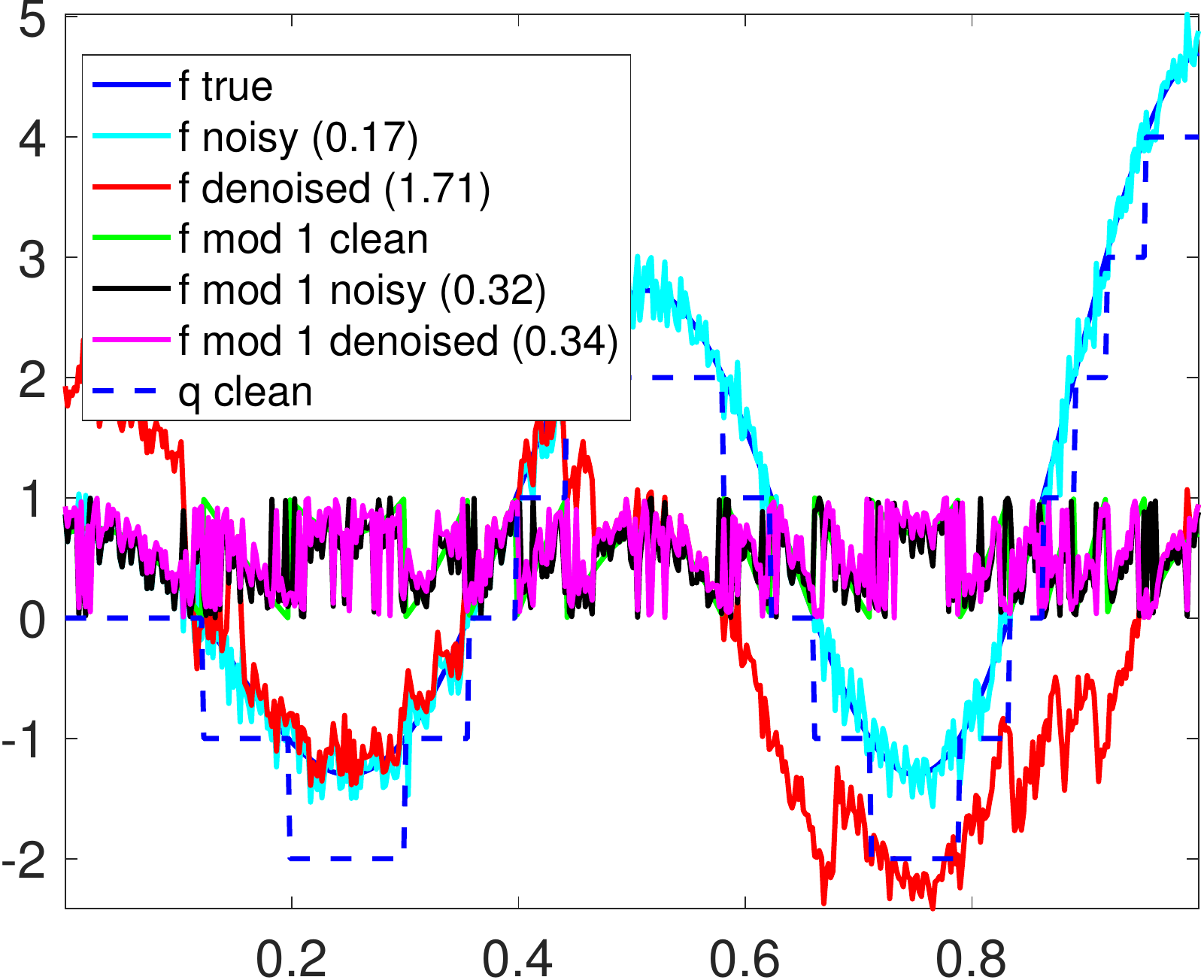} }
%
\subcaptionbox[]{  $\sigma=0.17$, \textbf{iQCQP}(10 iters.)
}[ 0.32\textwidth ]
{\includegraphics[width=0.27\textwidth] {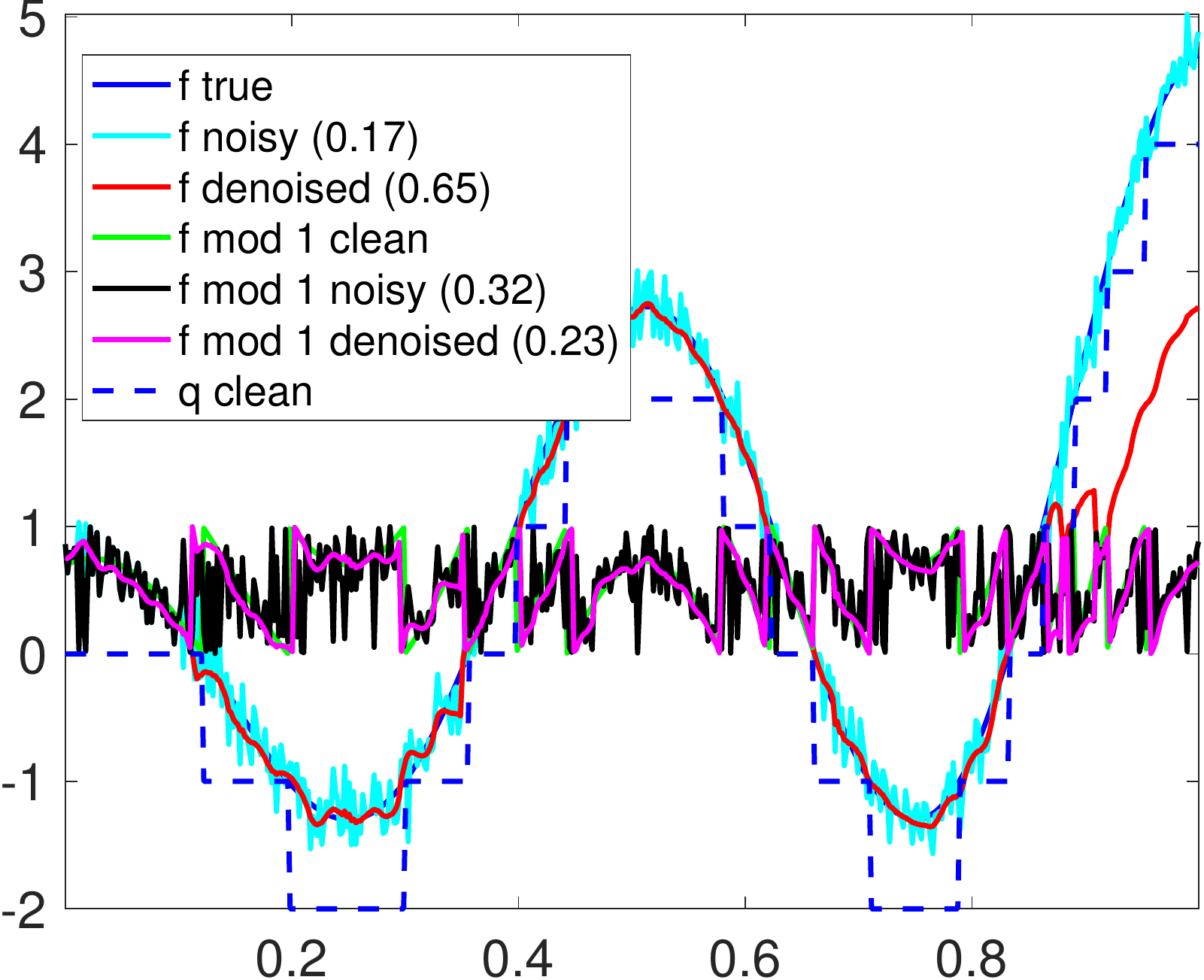} }
%
%
%
\vspace{-2mm}
\captionsetup{width=0.95\linewidth}
\caption[Short Caption]{Denoised instances under the Gaussian noise model, for \textbf{OLS}, \textbf{QCQP} and \textbf{iQCQP},  as we increase the noise level $\sigma$. We keep fixed the parameters $n=500$, $k=2$, $\lambda= 0.1$. \textbf{QCQP} denotes Algorithm \ref{algo:two_stage_denoise} where the unwrapping stage is performed by \textbf{OLS} \eqref{eq:ols_unwrap_lin_system}.
}
\label{fig:instances_f1_Gaussian}
\end{figure}


\begin{figure}[!ht]
\centering
\subcaptionbox[]{ $k=2$, $\lambda= 0.03$}[ 0.19\textwidth ]
{\includegraphics[width=0.19\textwidth] {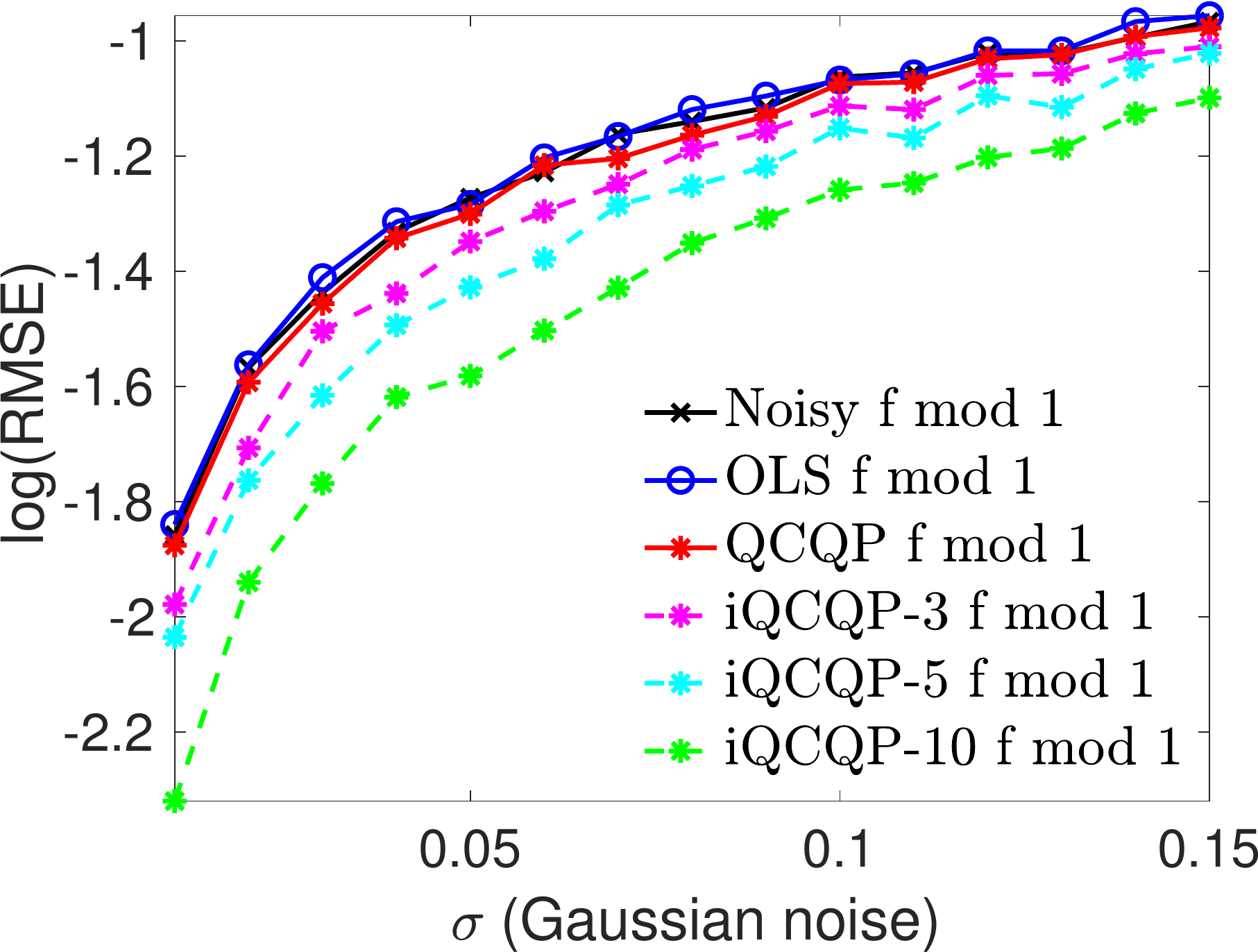} }
\subcaptionbox[]{ $k=2$, $\lambda= 0.1$}[ 0.19\textwidth ]
{\includegraphics[width=0.19\textwidth] {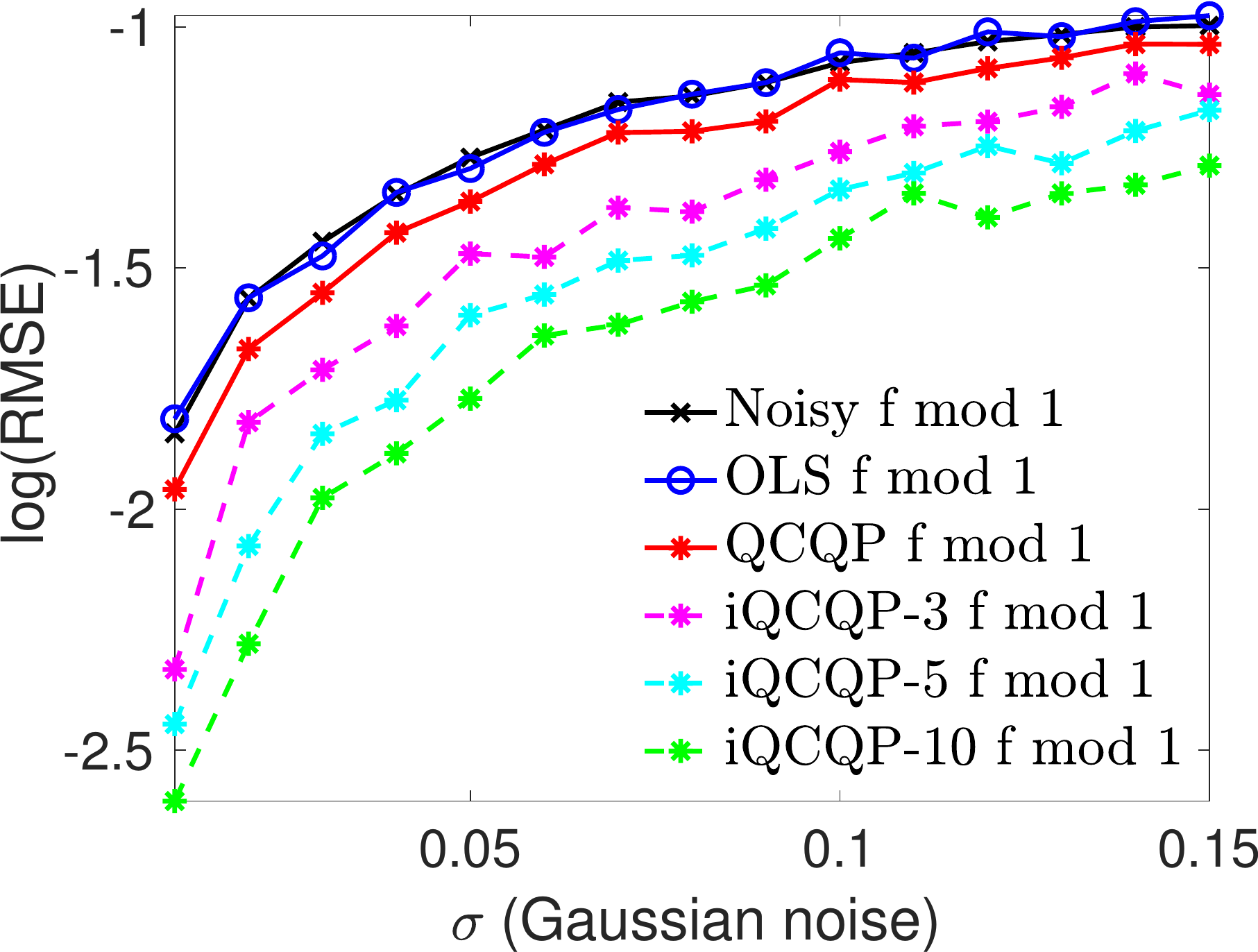} }
%
\subcaptionbox[]{ $k=2$, $\lambda= 0.3$}[ 0.19\textwidth ]
{\includegraphics[width=0.19\textwidth] {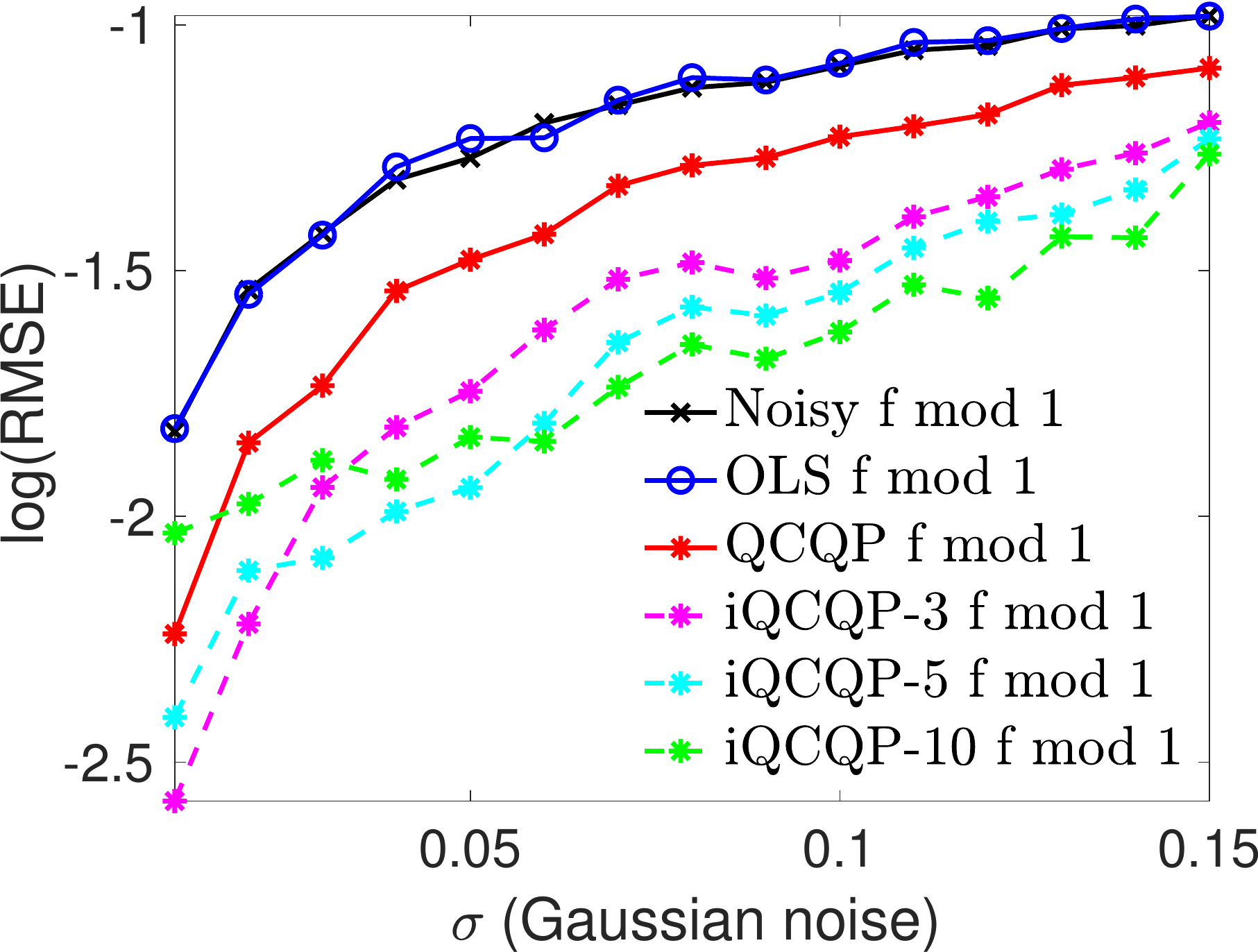} }
%
\subcaptionbox[]{ $k=2$, $\lambda= 0.5$}[ 0.19\textwidth ]
{\includegraphics[width=0.19\textwidth] {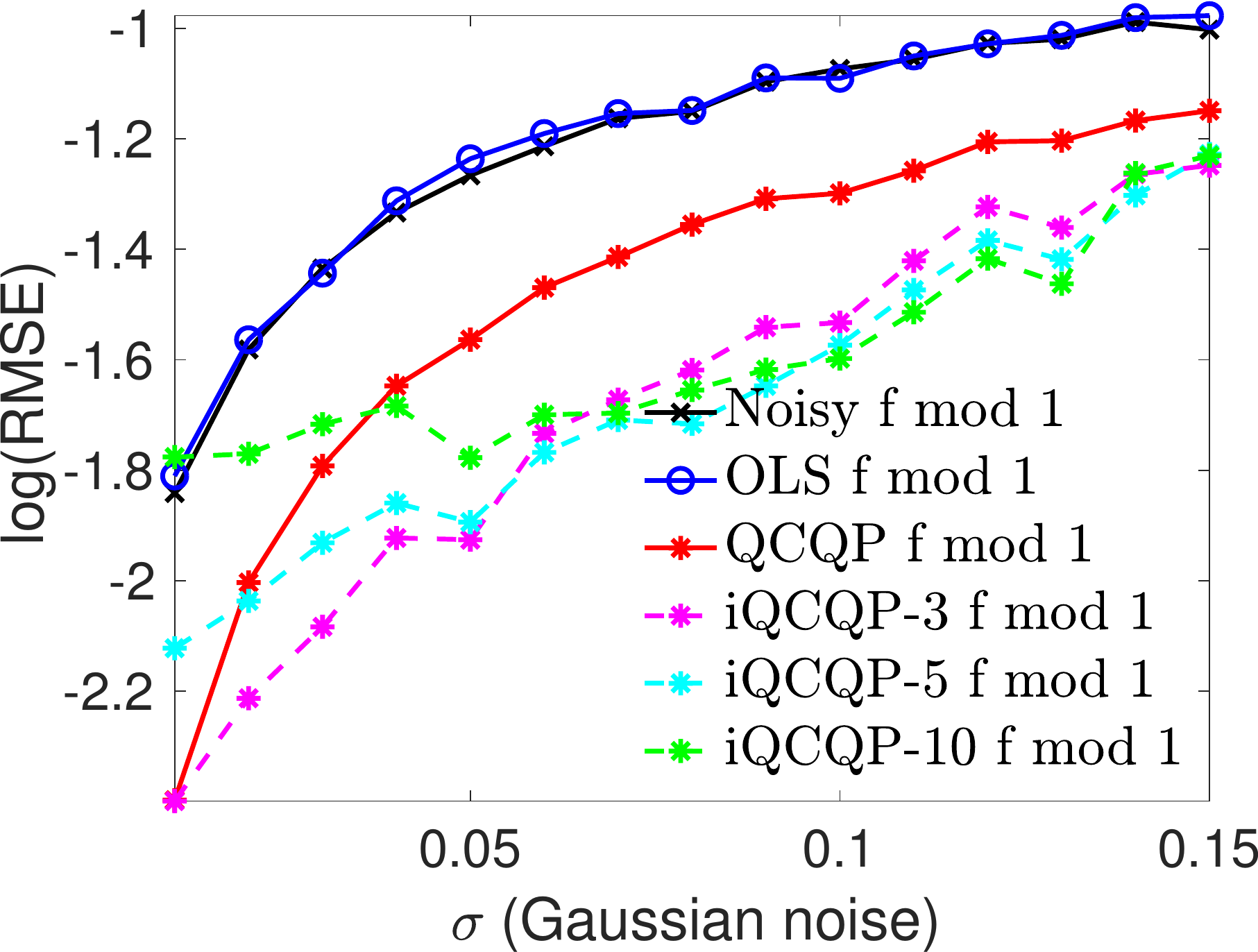} }
%
\subcaptionbox[]{ $k=2$, $\lambda= 1$}[ 0.19\textwidth ]
{\includegraphics[width=0.19\textwidth] {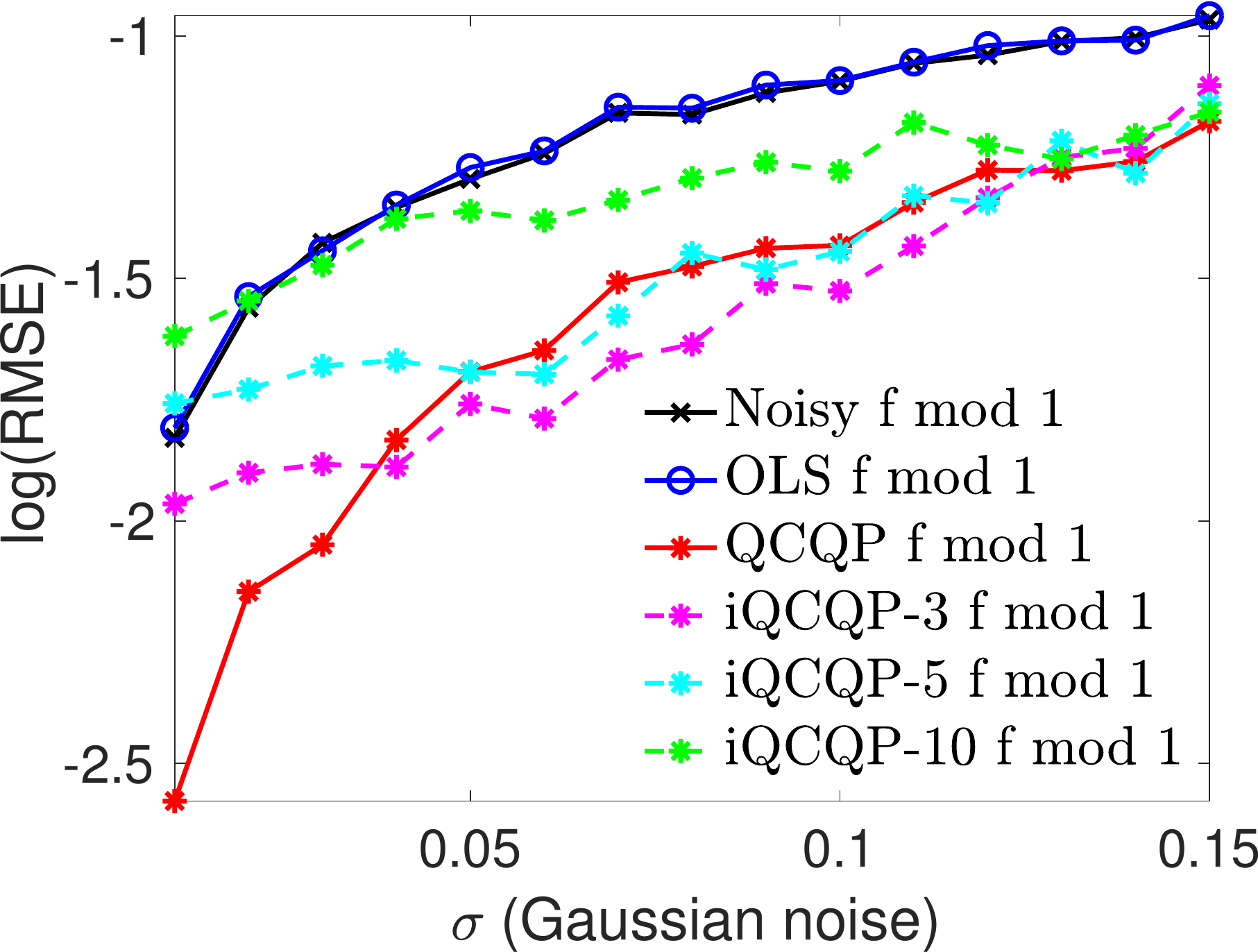} }
\hspace{0.1\textwidth} 
\subcaptionbox[]{ $k=3$, $\lambda= 0.03$}[ 0.19\textwidth ]
{\includegraphics[width=0.19\textwidth] {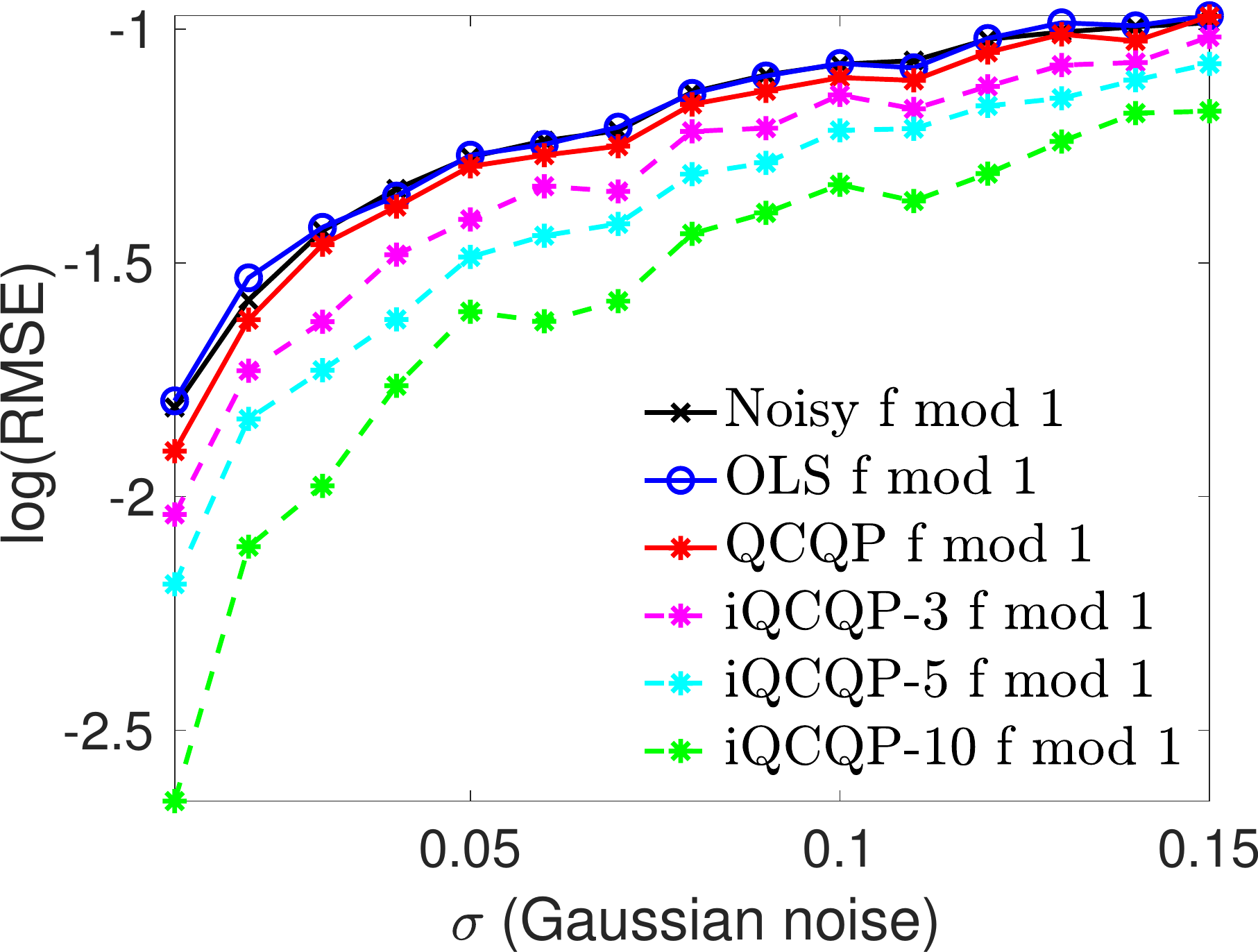} }
\subcaptionbox[]{ $k=3$, $\lambda= 0.1$}[ 0.19\textwidth ]
{\includegraphics[width=0.19\textwidth] {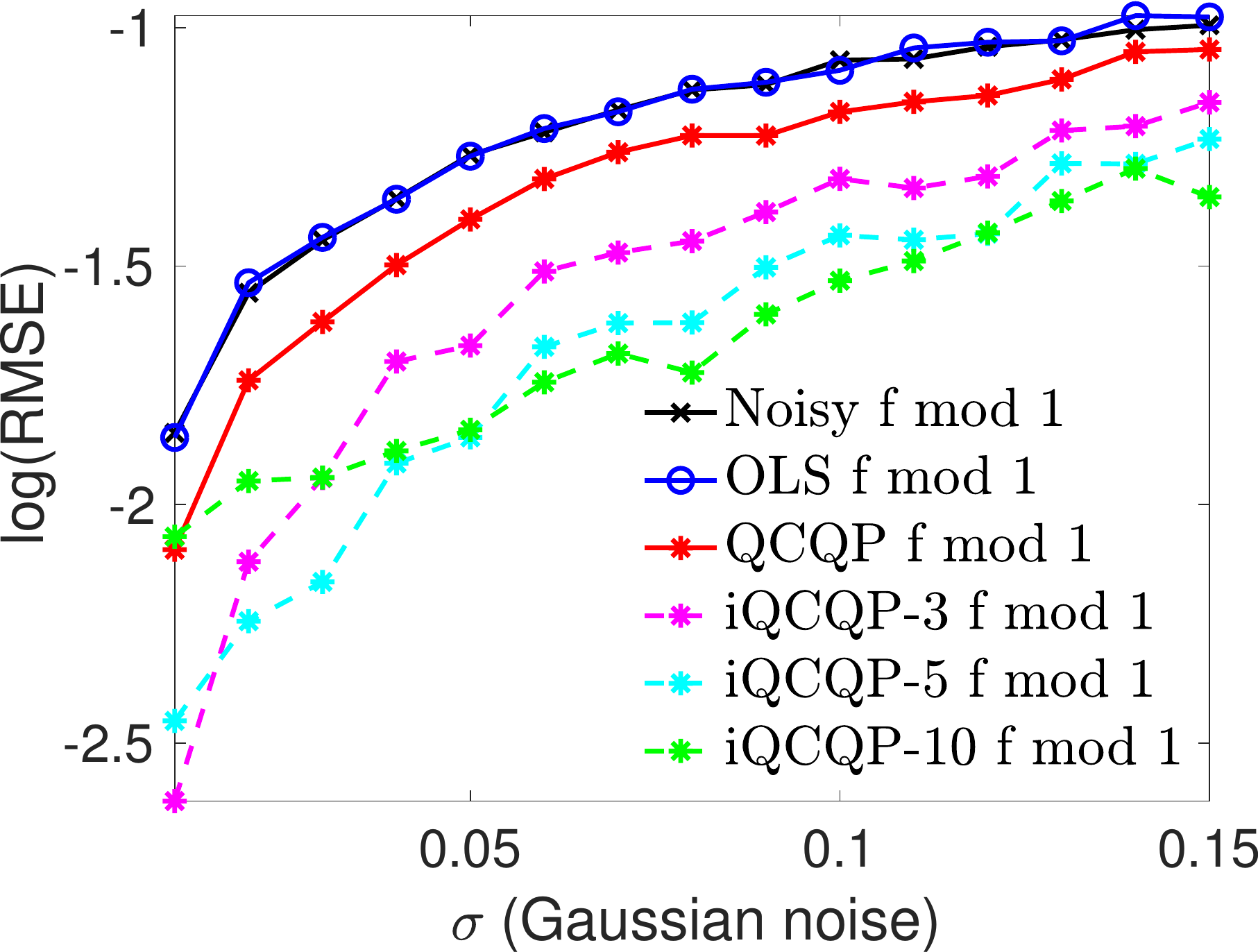} }
%
\subcaptionbox[]{ $k=3$, $\lambda= 0.3$}[ 0.19\textwidth ]
{\includegraphics[width=0.19\textwidth] {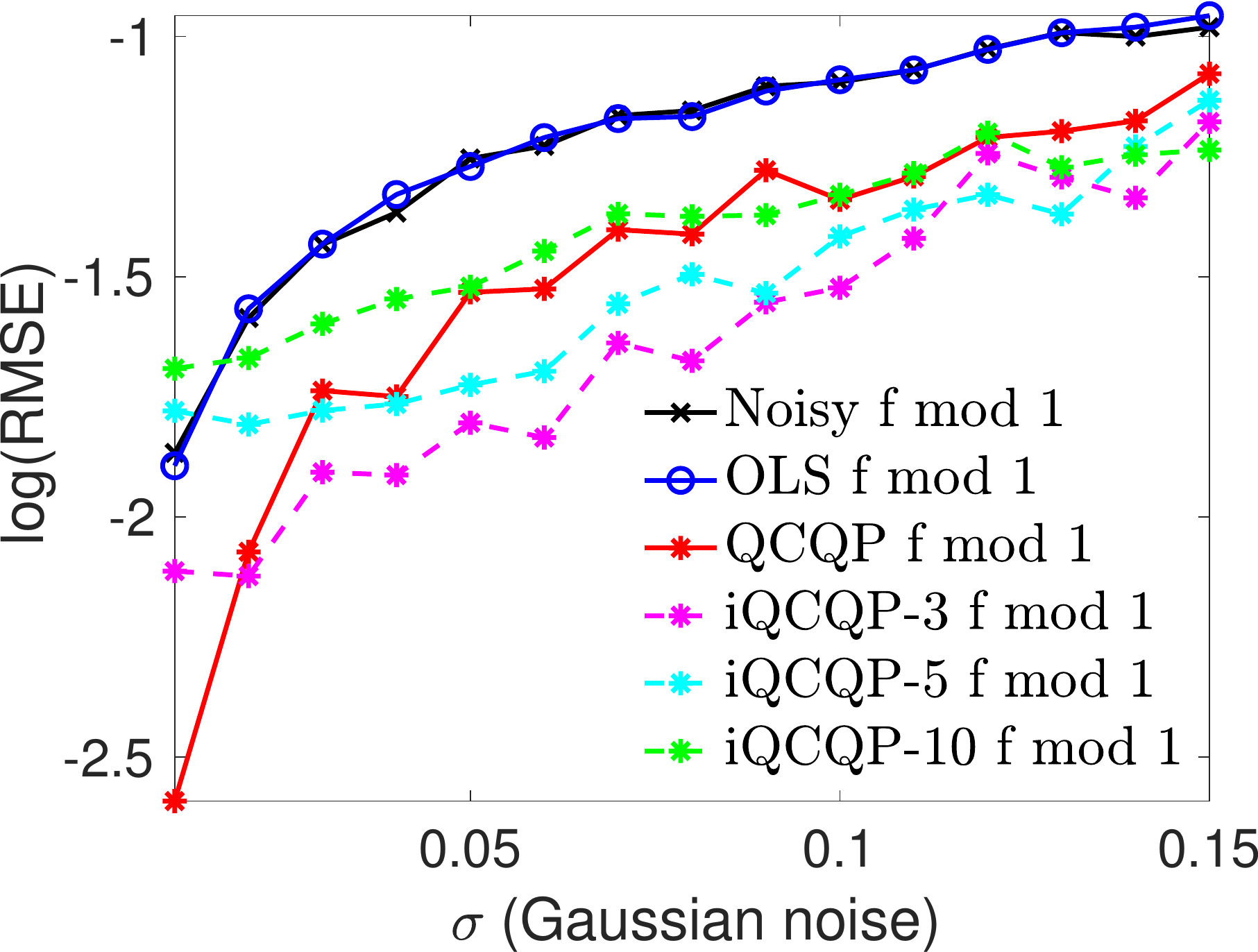} }
%
\subcaptionbox[]{ $k=3$, $\lambda= 0.5$}[ 0.19\textwidth ]
{\includegraphics[width=0.19\textwidth] {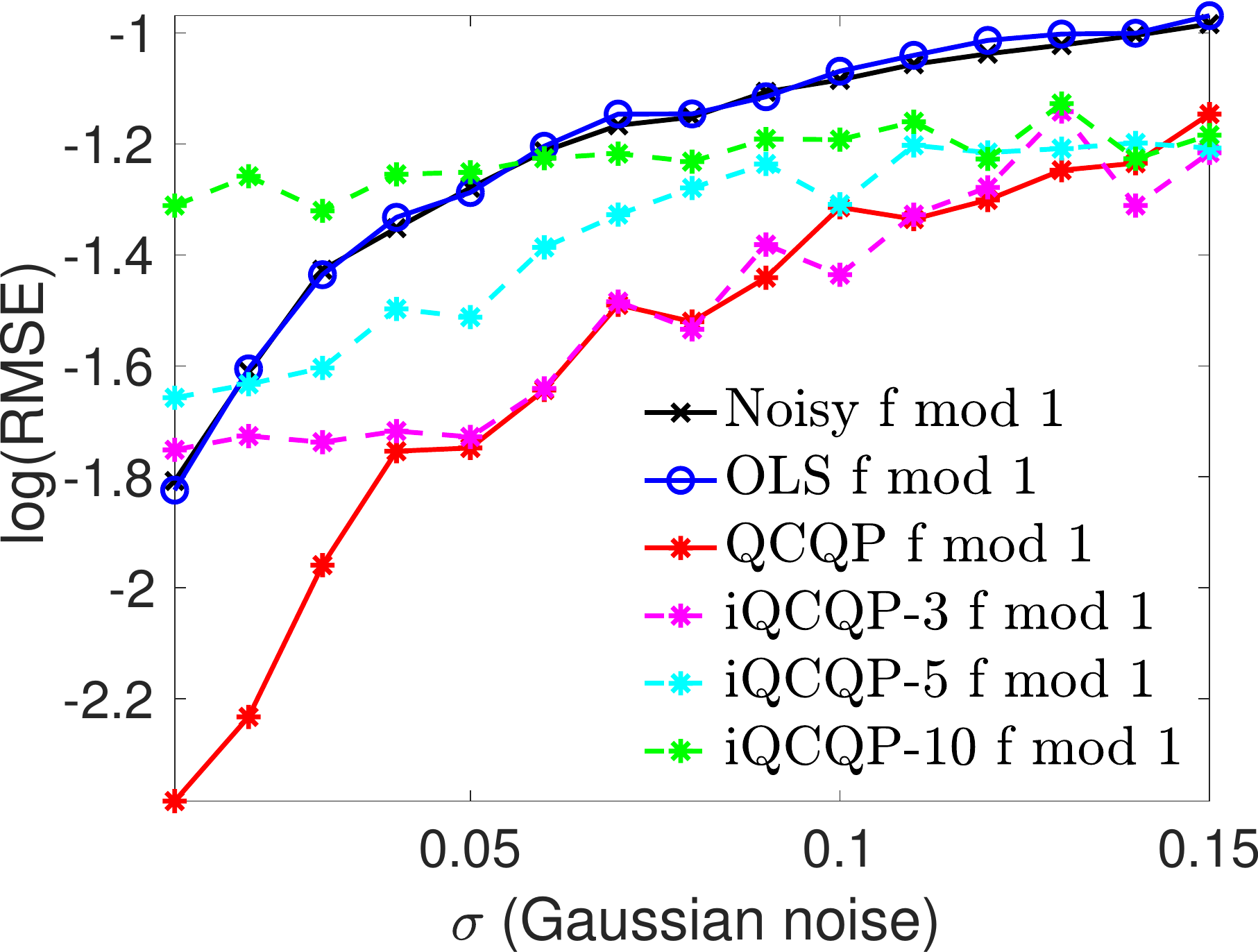} }
%
\subcaptionbox[]{ $k=3$, $\lambda= 1$}[ 0.19\textwidth ]
{\includegraphics[width=0.19\textwidth] {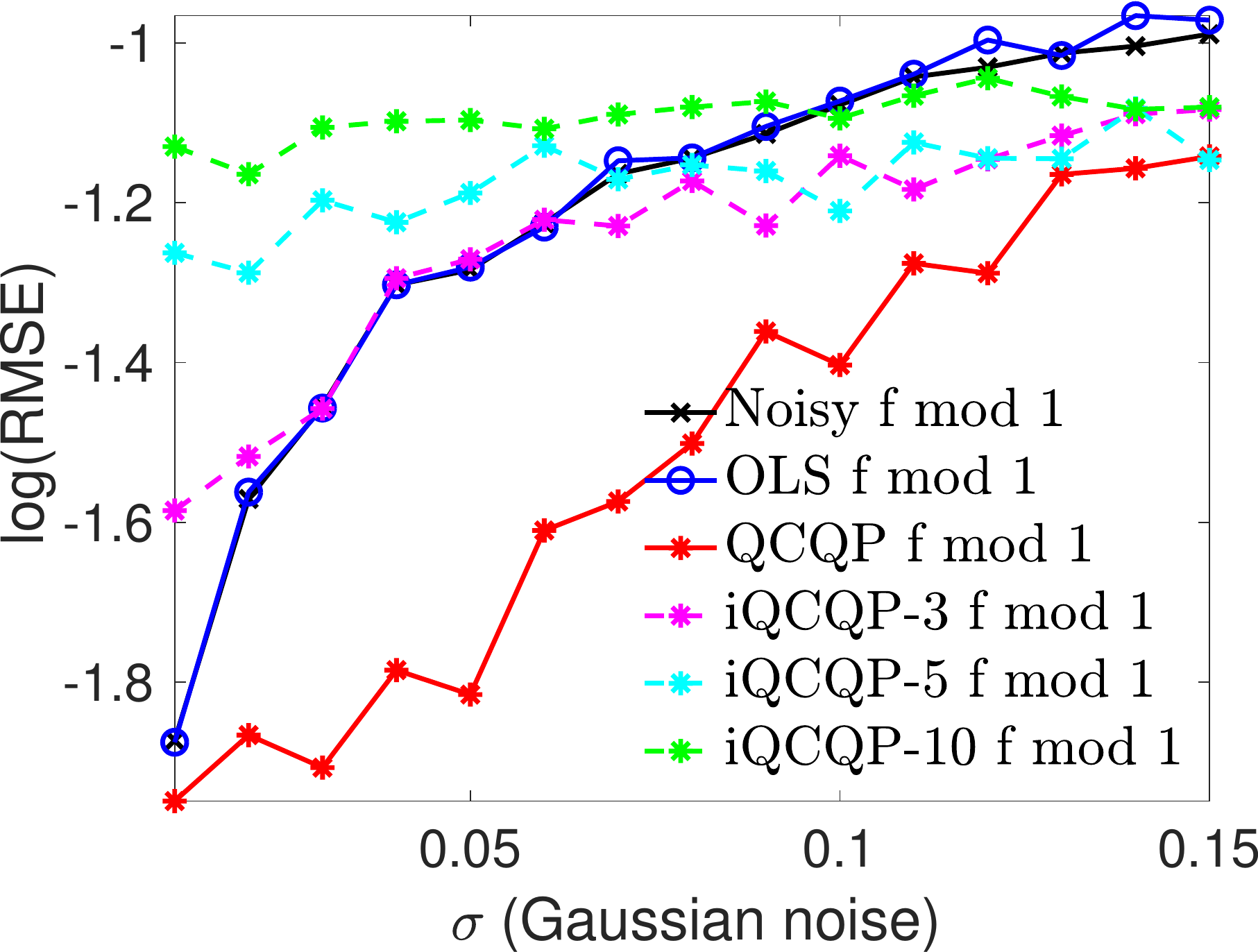} }
\hspace{0.1\textwidth} 
\subcaptionbox[]{ $k=5$, $\lambda= 0.03$}[ 0.19\textwidth ]
{\includegraphics[width=0.19\textwidth] {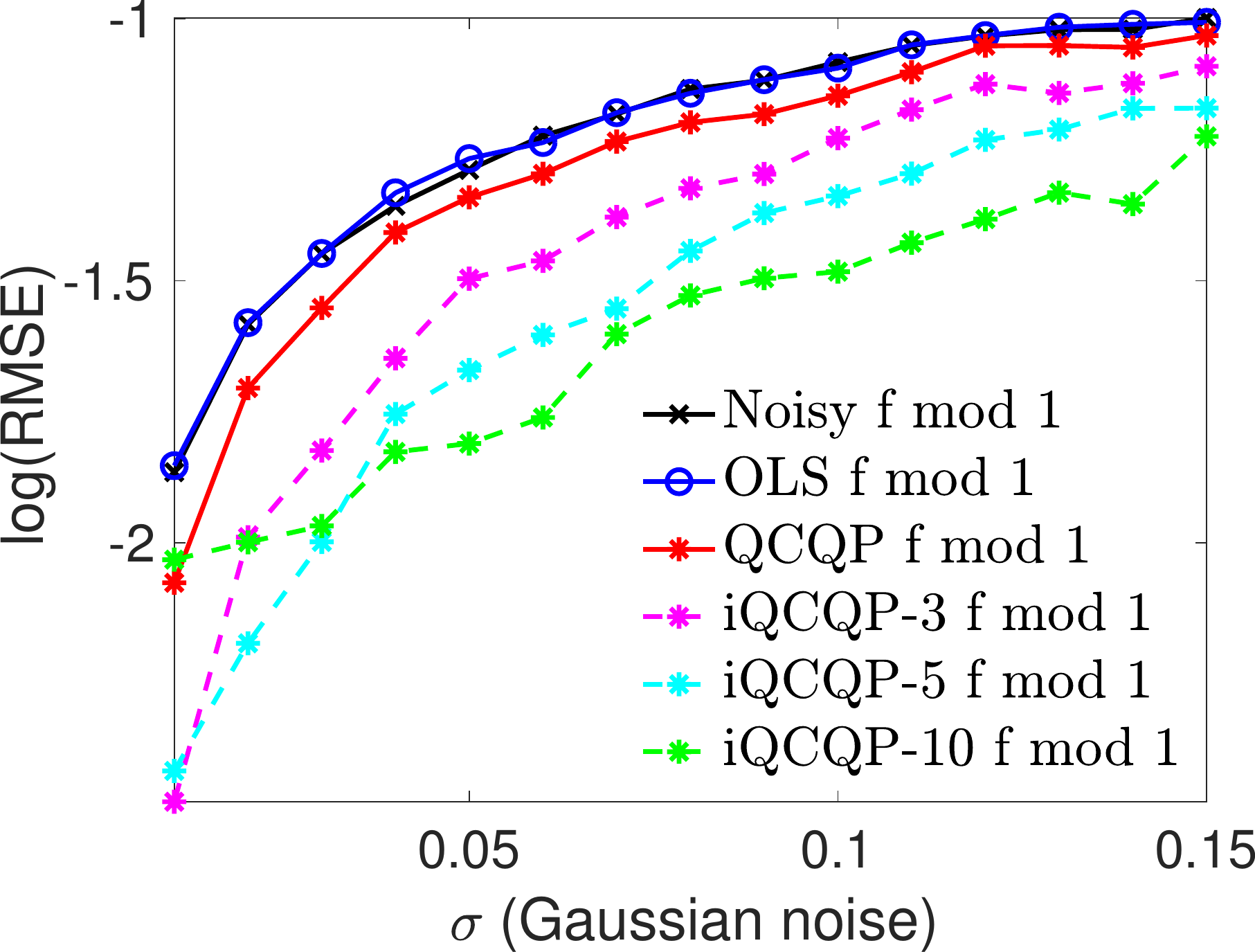} }
\subcaptionbox[]{ $k=5$, $\lambda= 0.1$}[ 0.19\textwidth ]
{\includegraphics[width=0.19\textwidth] {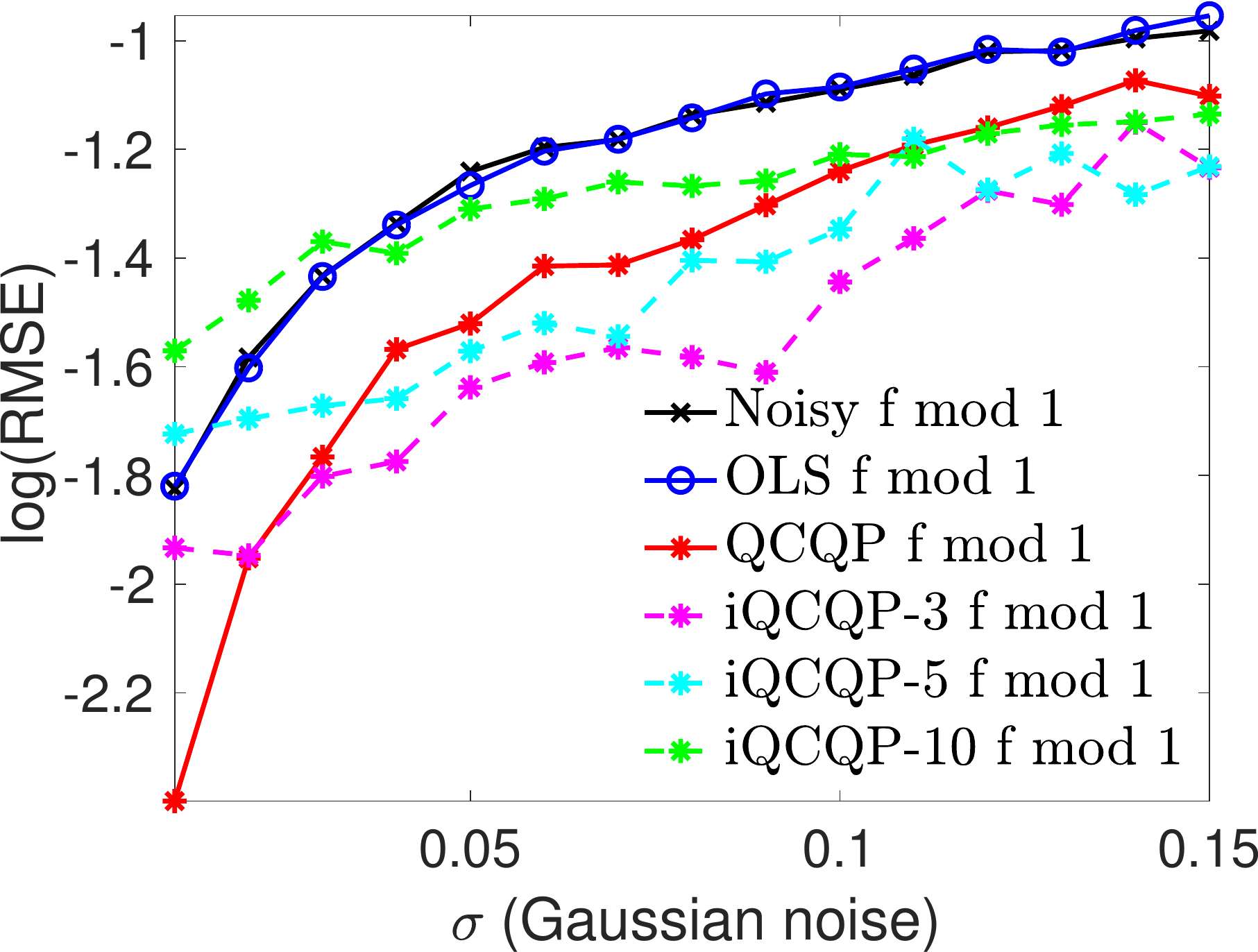} }
%
\subcaptionbox[]{ $k=5$, $\lambda= 0.3$}[ 0.19\textwidth ]
{\includegraphics[width=0.19\textwidth] {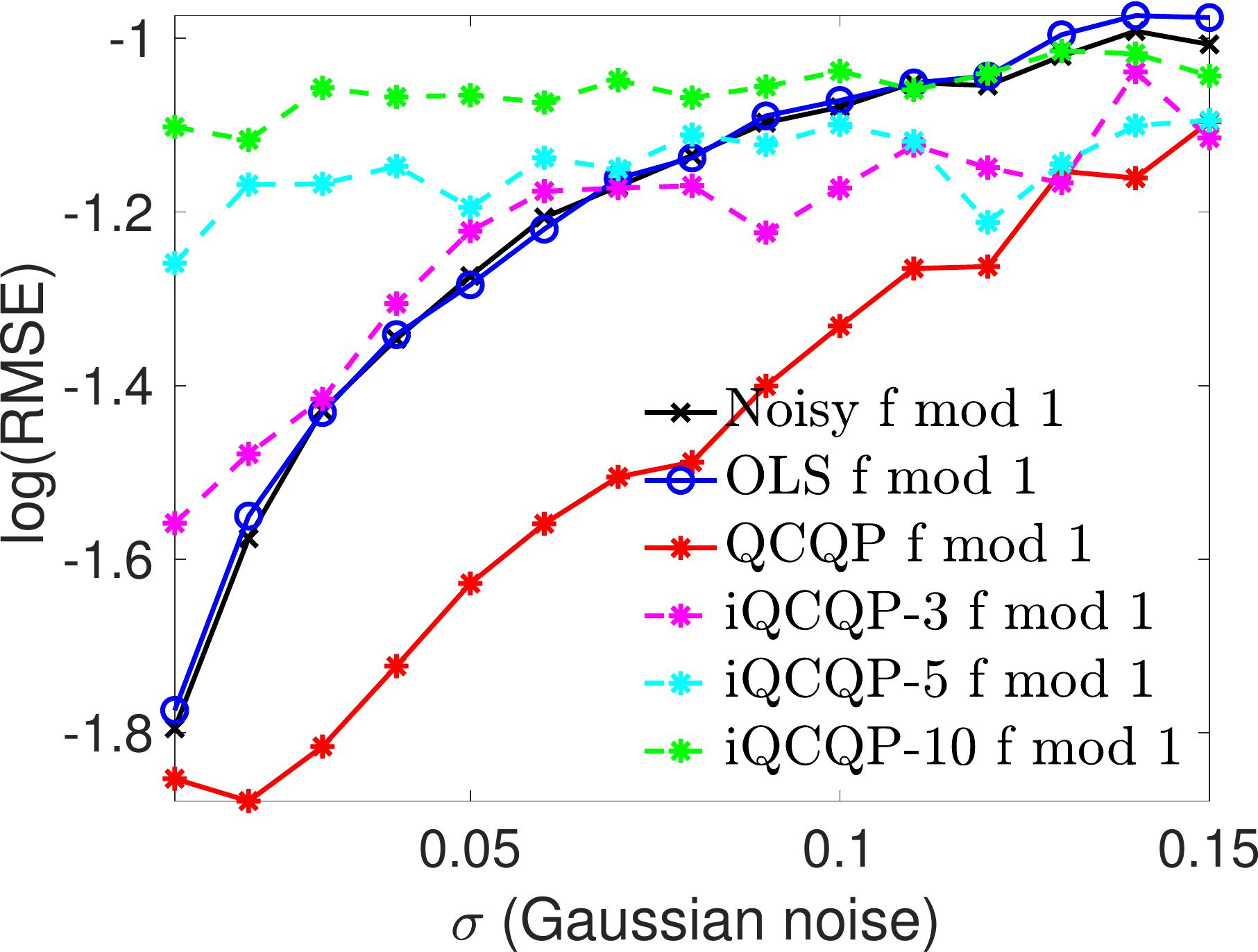} }
%
\subcaptionbox[]{ $k=5$, $\lambda= 0.5$}[ 0.19\textwidth ]
{\includegraphics[width=0.19\textwidth] {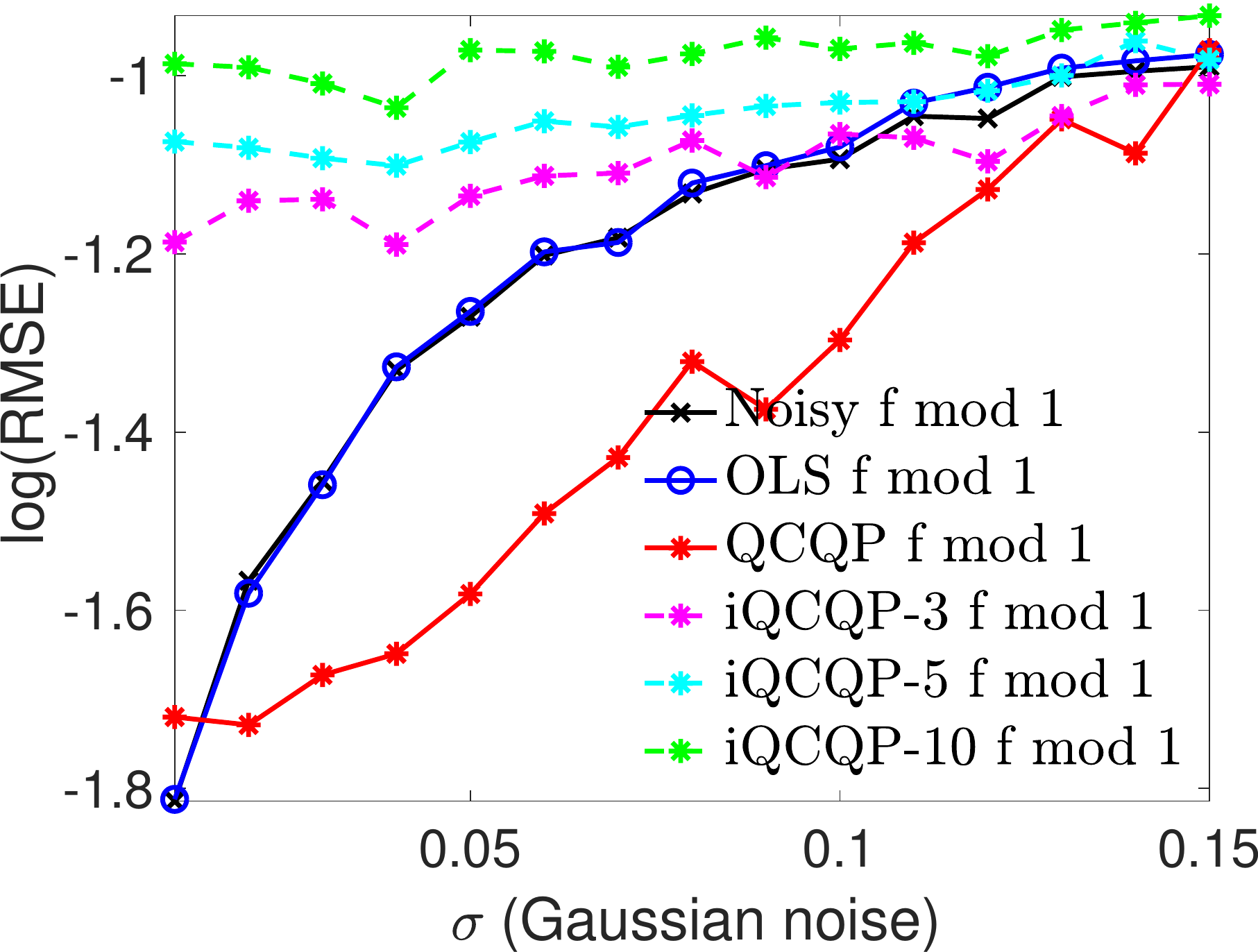} }
%
\subcaptionbox[]{ $k=5$, $\lambda= 1$}[ 0.19\textwidth ]
{\includegraphics[width=0.19\textwidth] {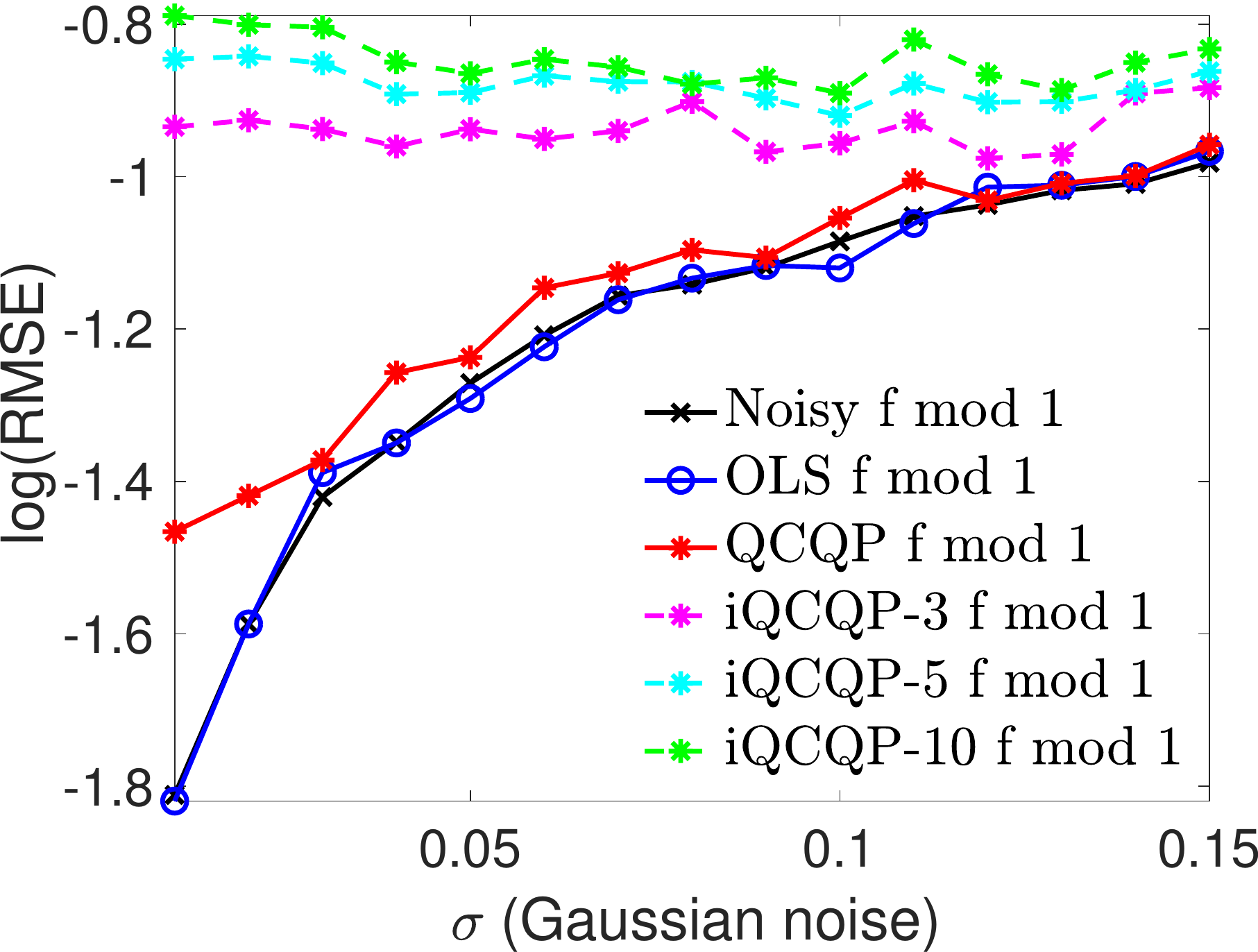} }
\hspace{0.1\textwidth} 
%
%
\vspace{-3mm}
\captionsetup{width=0.98\linewidth}
\caption[Short Caption]{Recovery errors for the denoised $f$ \hspace{-1mm} mod 1 samples, for $n=500$ under the Gaussian noise model (20 trials).  \textbf{QCQP} denotes Algorithm \ref{algo:two_stage_denoise} without  the unwrapping stage performed by \textbf{OLS} \eqref{eq:ols_unwrap_lin_system}.
}
\label{fig:Sims_f1_Gaussian_fmod1}
\end{figure}
 
\vspace{-3mm}
 
\begin{figure}[!ht]
\centering
\subcaptionbox[]{ $k=2$, $\lambda= 0.03$}[ 0.19\textwidth ]
{\includegraphics[width=0.19\textwidth] {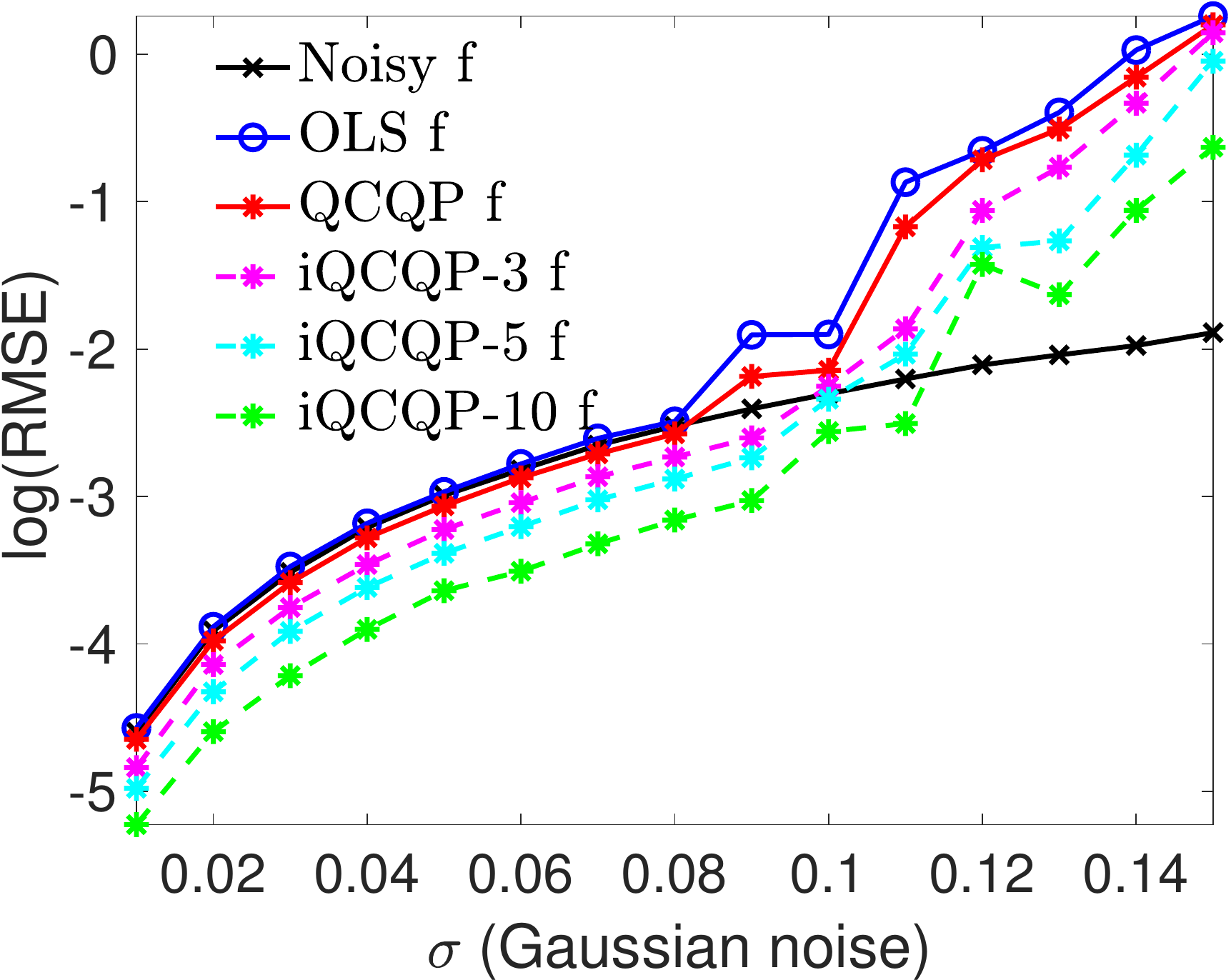} }
\subcaptionbox[]{ $k=2$, $\lambda= 0.1$}[ 0.19\textwidth ]
{\includegraphics[width=0.19\textwidth] {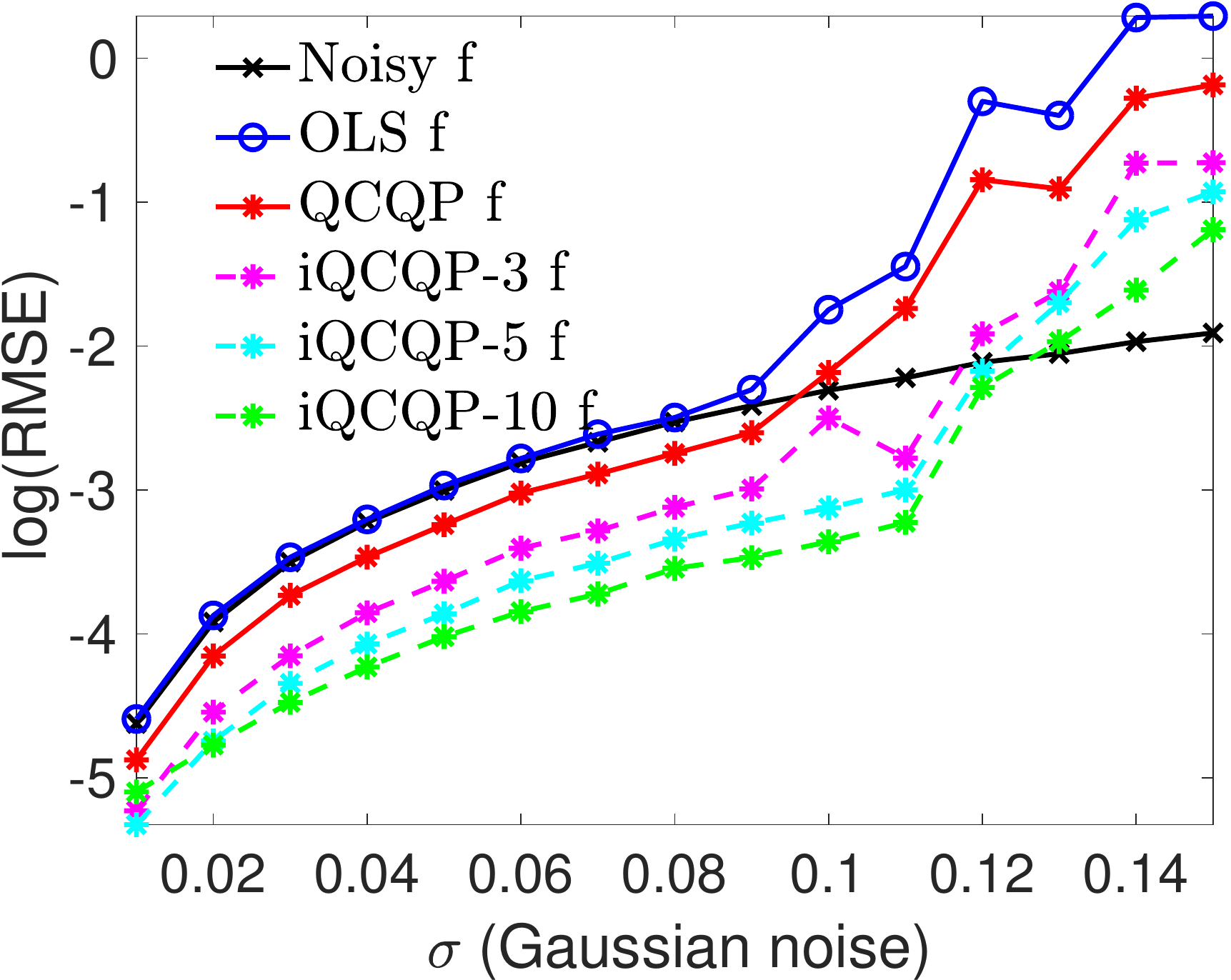} }
%
\subcaptionbox[]{ $k=2$, $\lambda= 0.3$}[ 0.19\textwidth ]
{\includegraphics[width=0.19\textwidth] {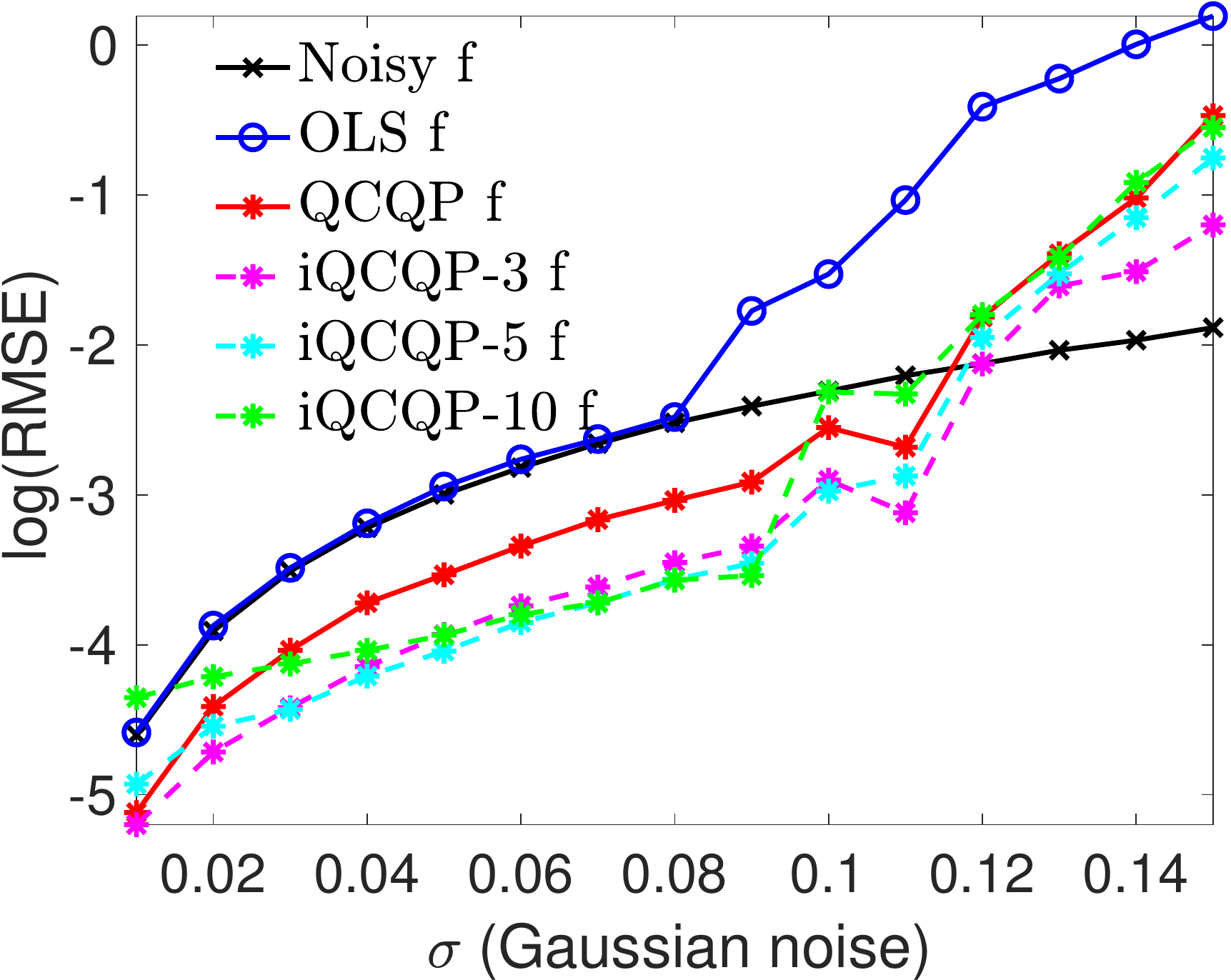} }
%
\subcaptionbox[]{ $k=2$, $\lambda= 0.5$}[ 0.19\textwidth ]
{\includegraphics[width=0.19\textwidth] {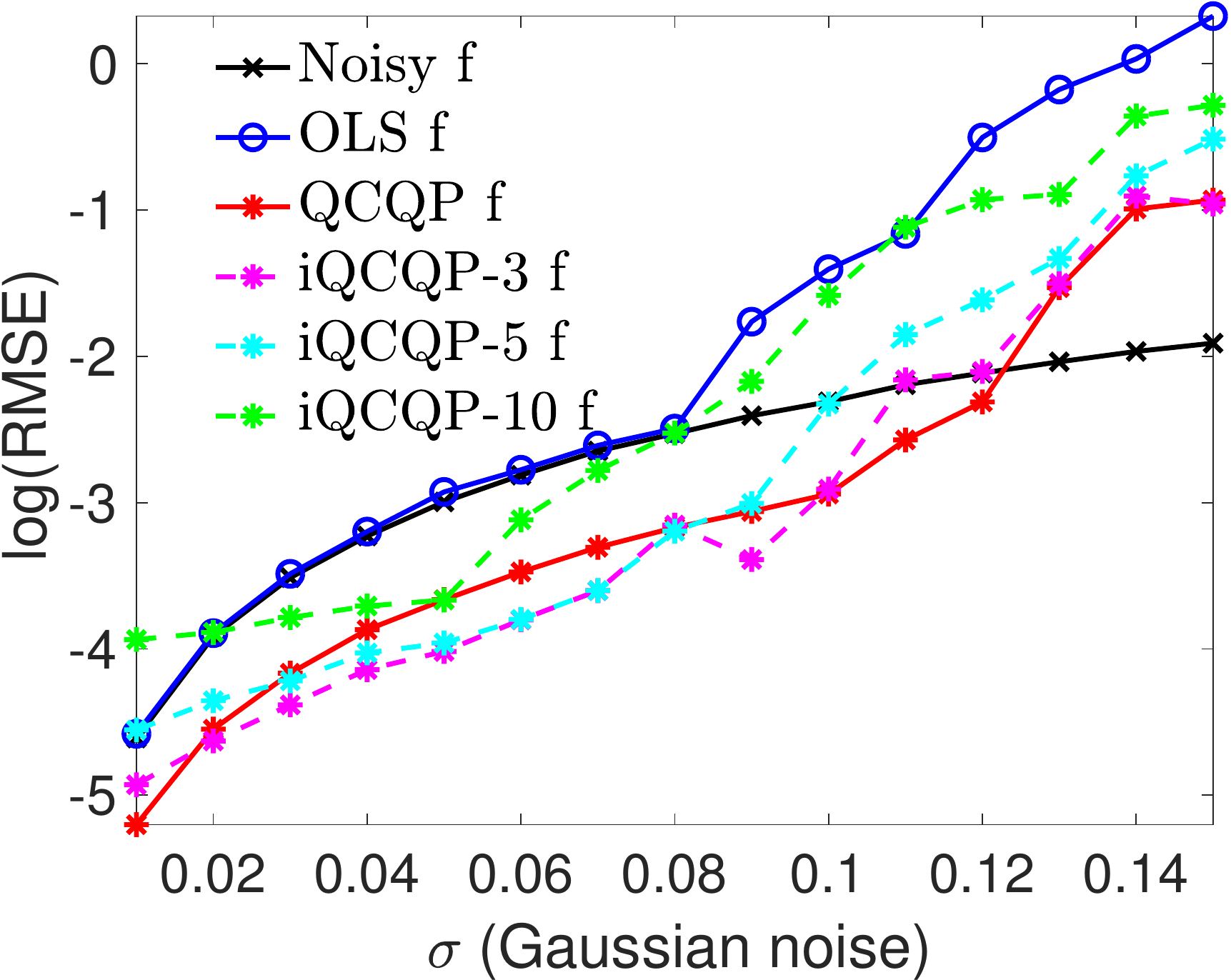} }
%
\subcaptionbox[]{ $k=2$, $\lambda= 1$}[ 0.19\textwidth ]
{\includegraphics[width=0.19\textwidth] {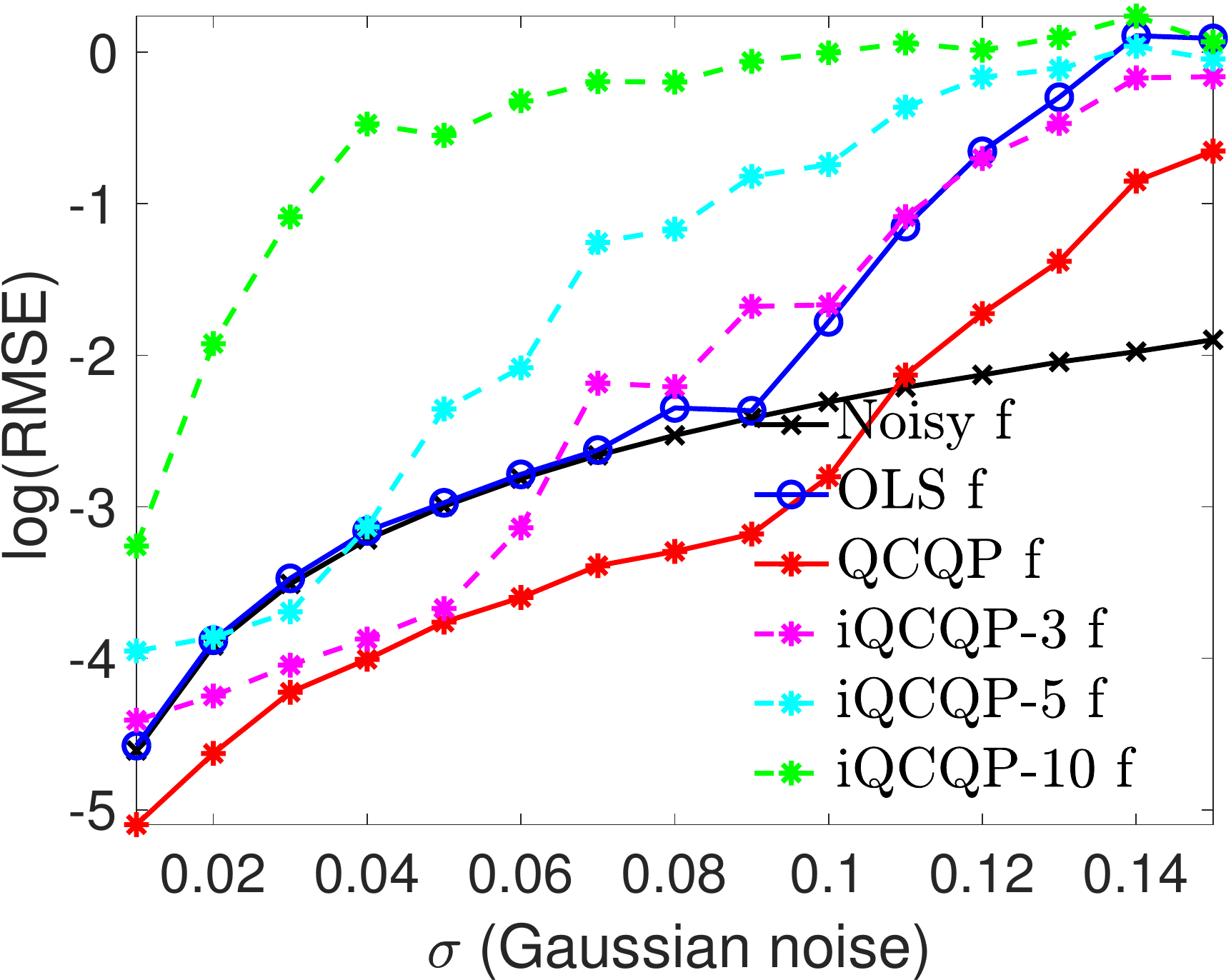} }
%
%
%
\subcaptionbox[]{ $k=3$, $\lambda= 0.03$}[ 0.19\textwidth ]
{\includegraphics[width=0.19\textwidth] {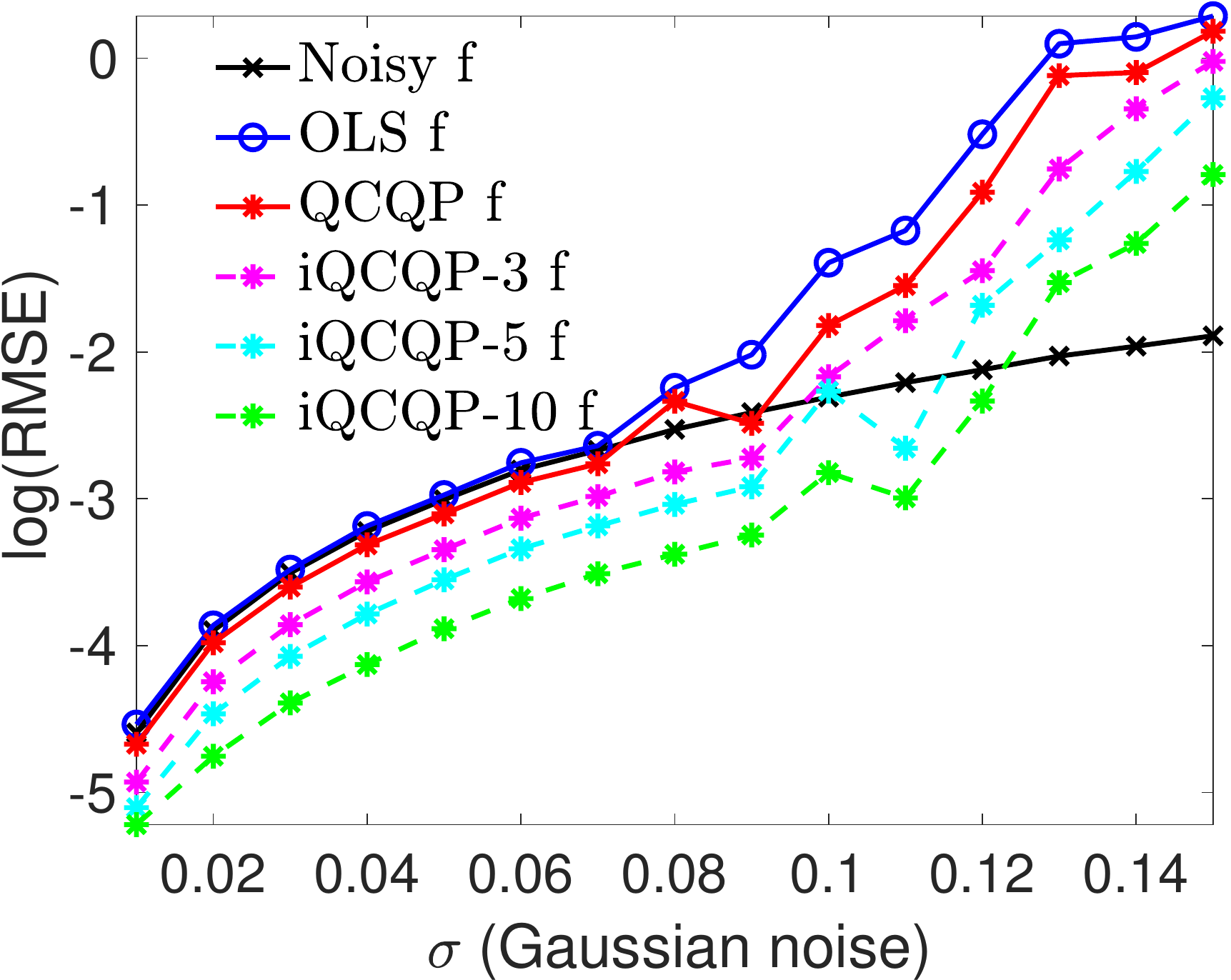} }
\subcaptionbox[]{ $k=3$, $\lambda= 0.1$}[ 0.19\textwidth ]
{\includegraphics[width=0.19\textwidth] {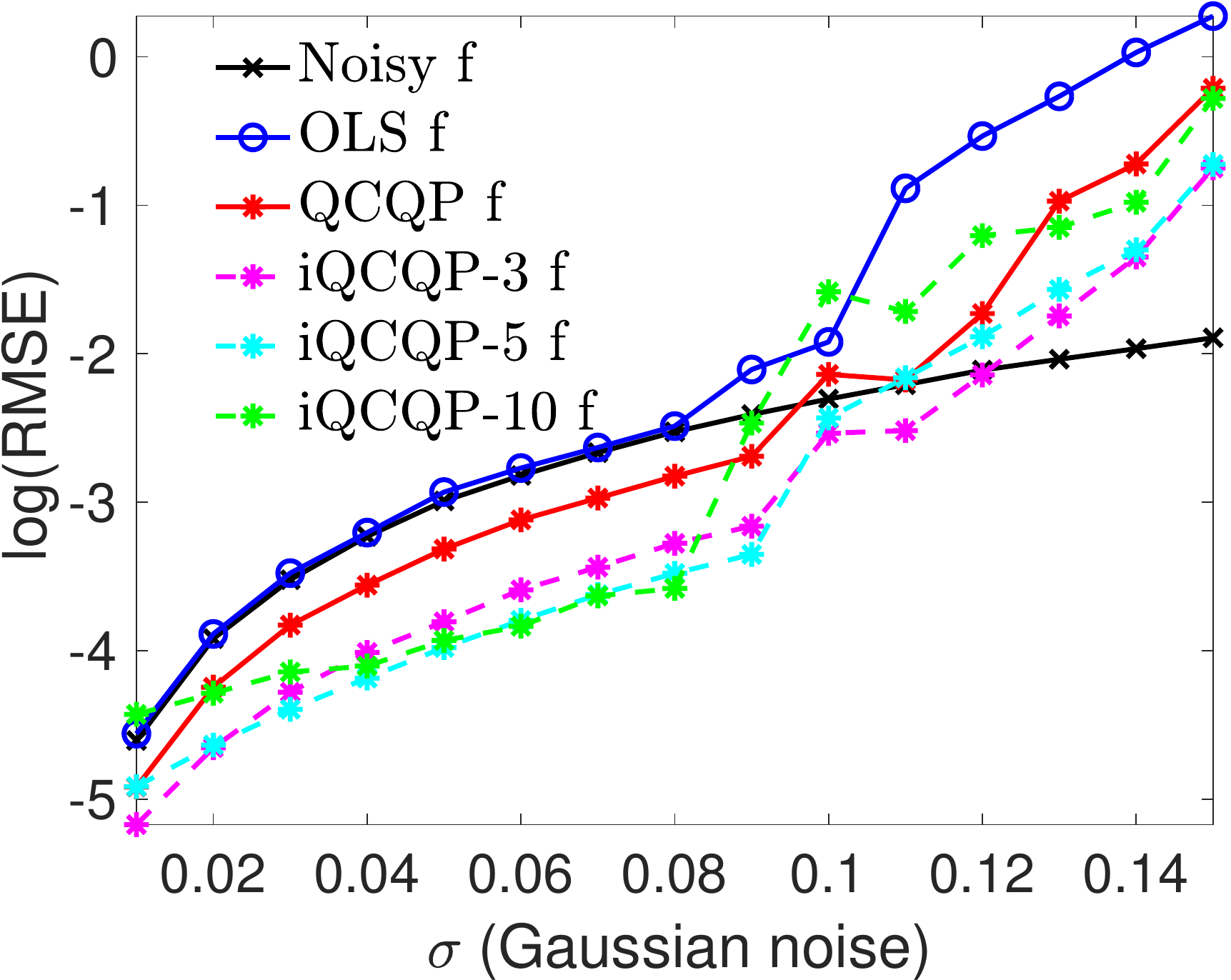} }
%
\subcaptionbox[]{ $k=3$, $\lambda= 0.3$}[ 0.19\textwidth ]
{\includegraphics[width=0.19\textwidth] {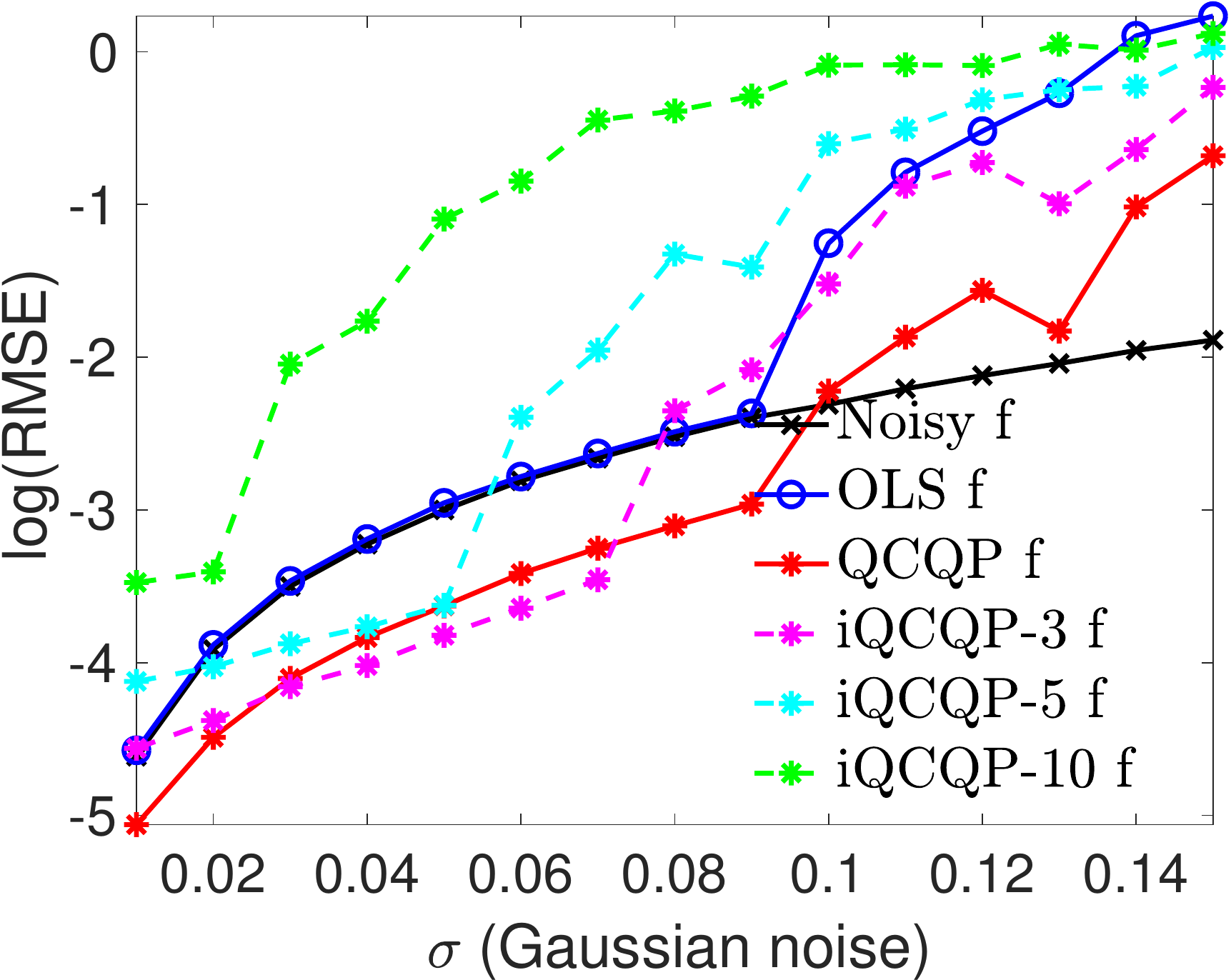} }
%
\subcaptionbox[]{ $k=3$, $\lambda= 0.5$}[ 0.19\textwidth ]
{\includegraphics[width=0.19\textwidth] {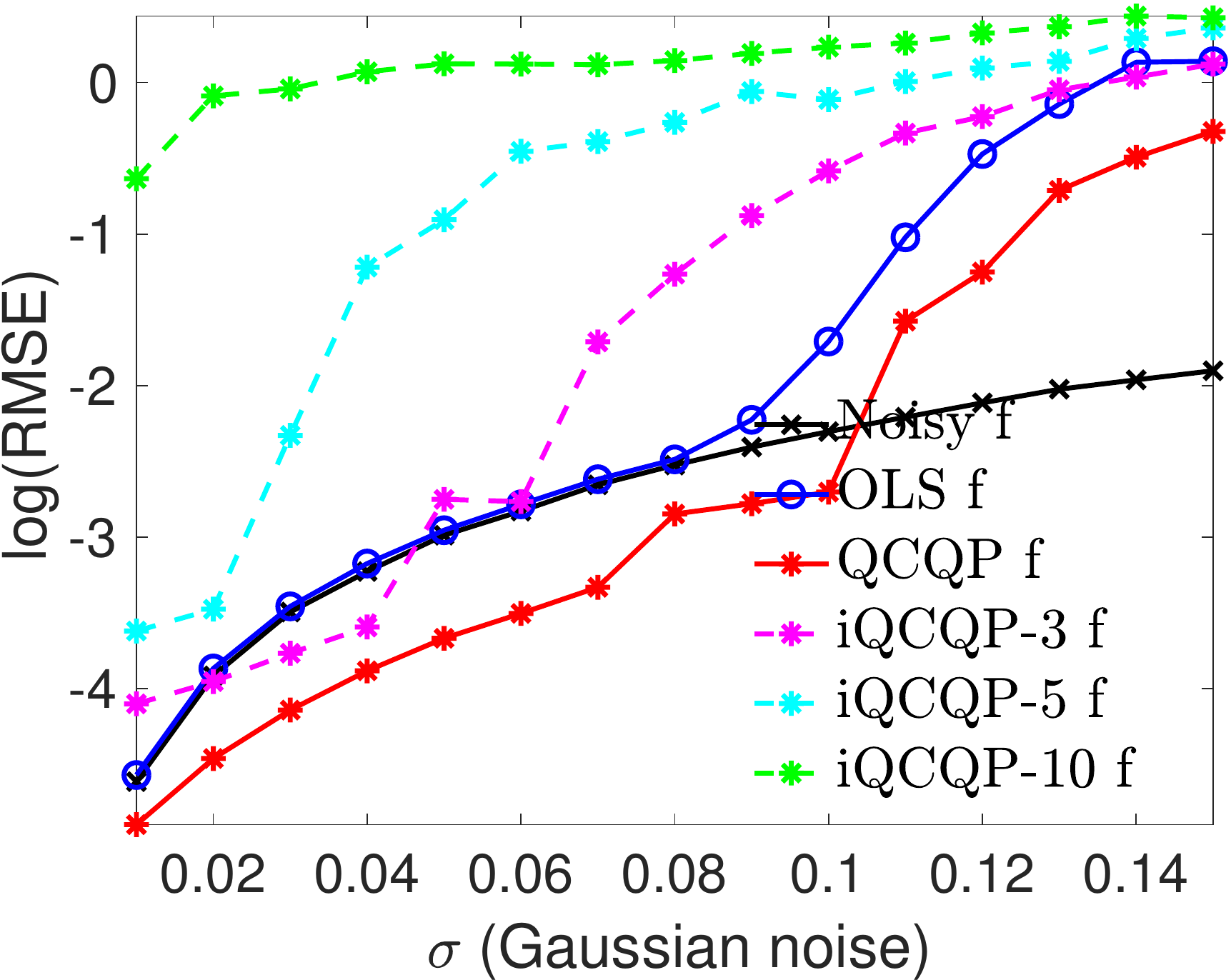} }
%
\subcaptionbox[]{ $k=3$, $\lambda= 1$}[ 0.19\textwidth ]
{\includegraphics[width=0.19\textwidth] {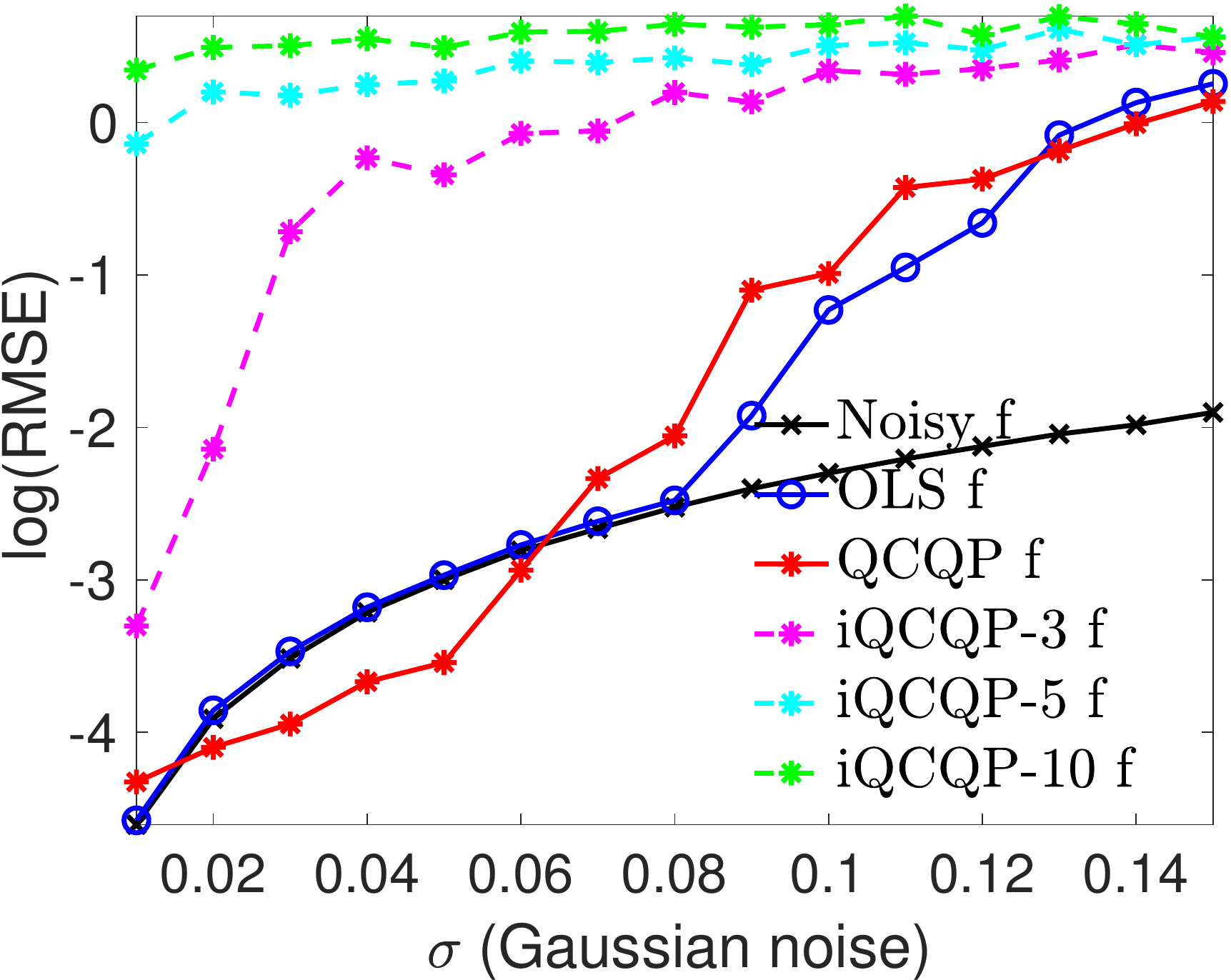} }
%
%
%
\subcaptionbox[]{ $k=5$, $\lambda= 0.03$}[ 0.19\textwidth ]
{\includegraphics[width=0.19\textwidth] {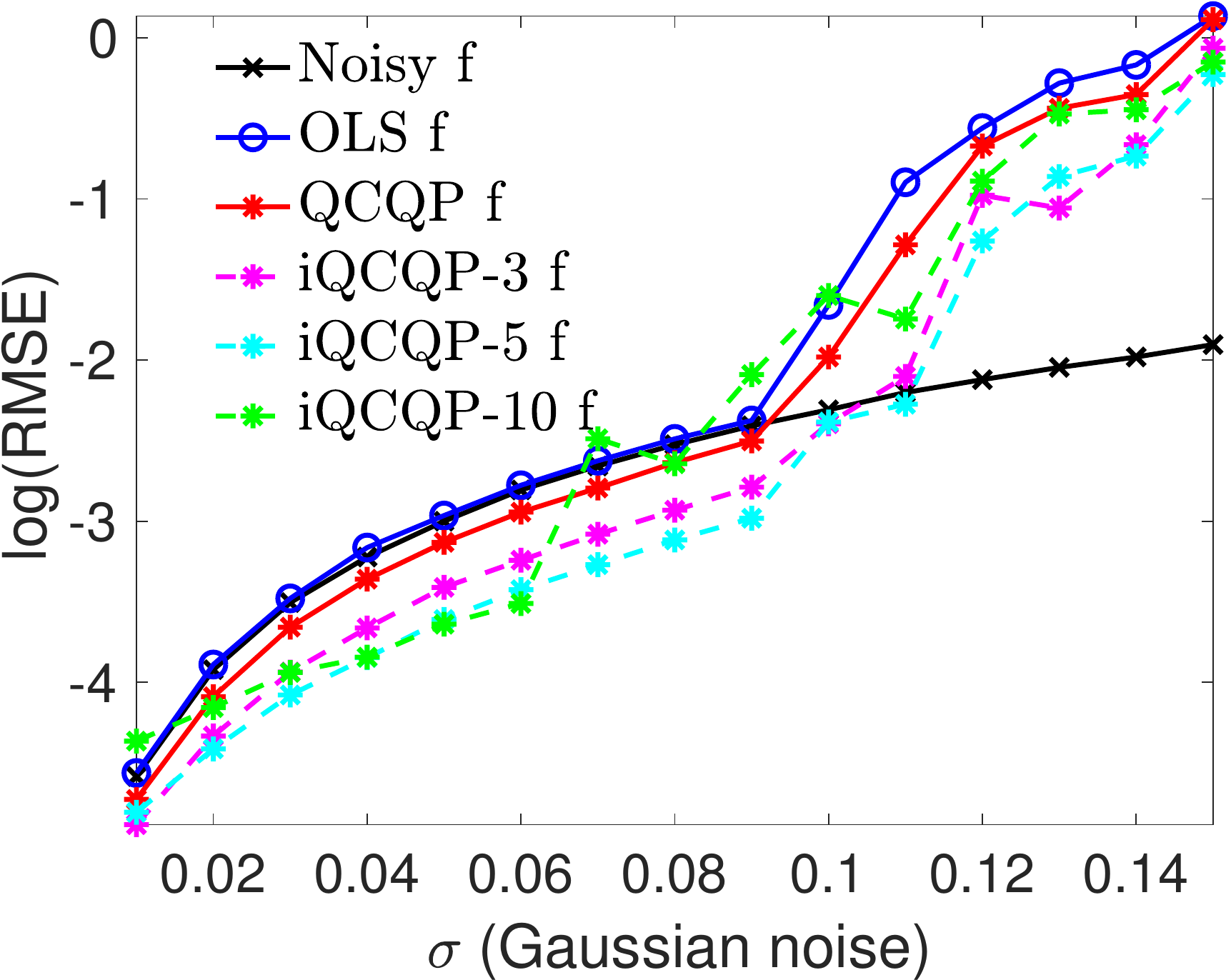} }
\subcaptionbox[]{ $k=5$, $\lambda= 0.1$}[ 0.19\textwidth ]
{\includegraphics[width=0.19\textwidth] {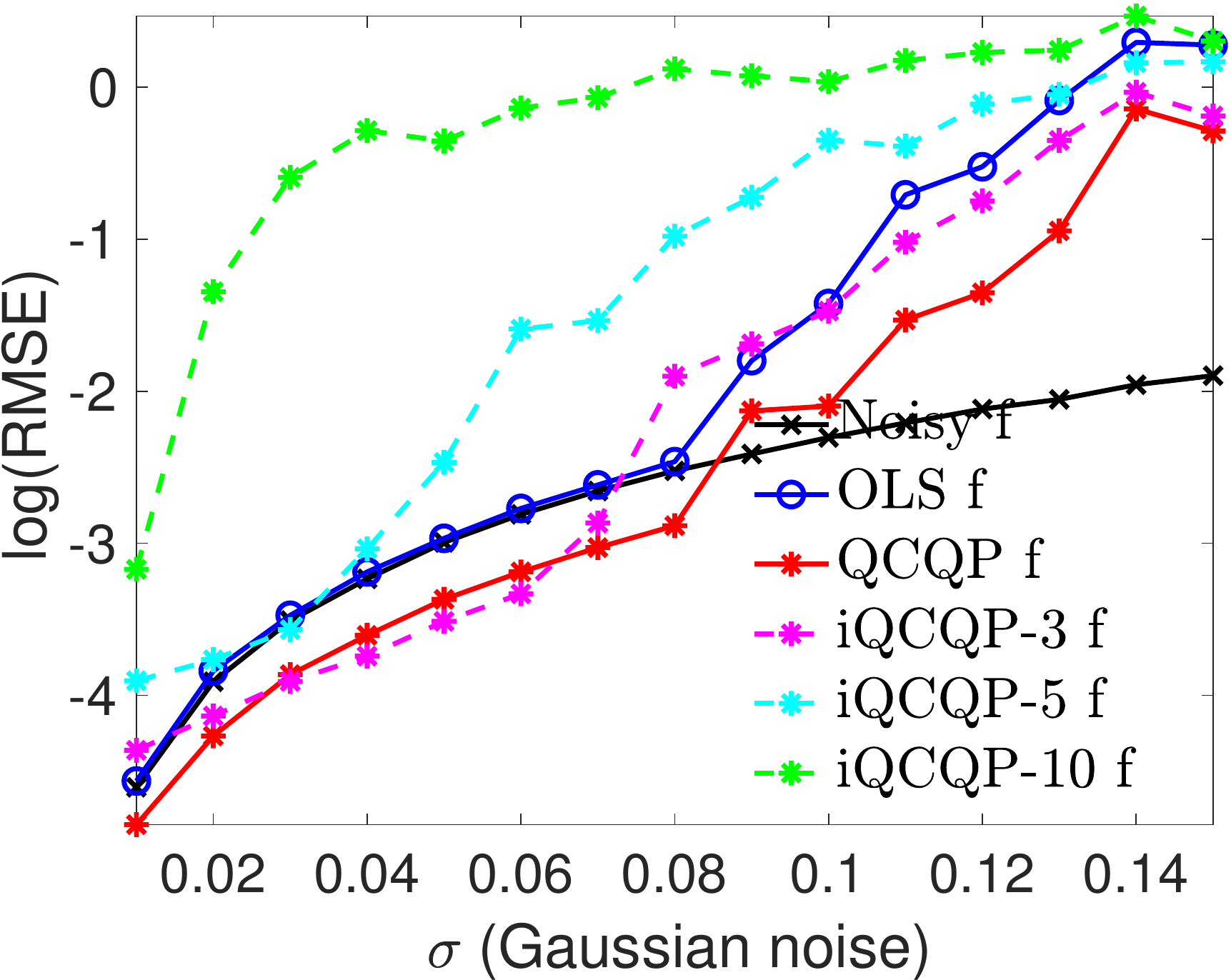} }
%
\subcaptionbox[]{ $k=5$, $\lambda= 0.3$}[ 0.19\textwidth ]
{\includegraphics[width=0.19\textwidth] {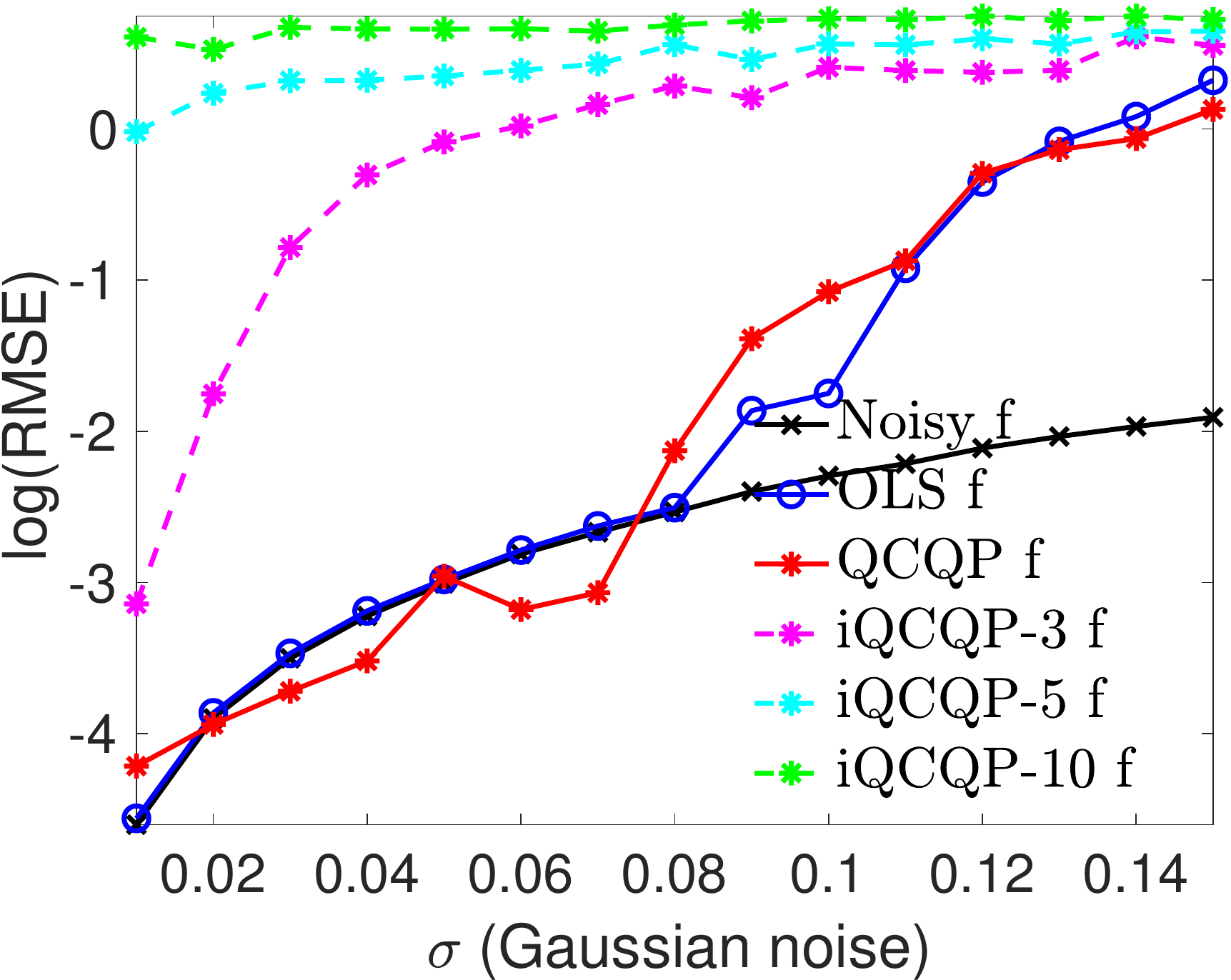} }
%
\subcaptionbox[]{ $k=5$, $\lambda= 0.5$}[ 0.19\textwidth ]
{\includegraphics[width=0.19\textwidth] {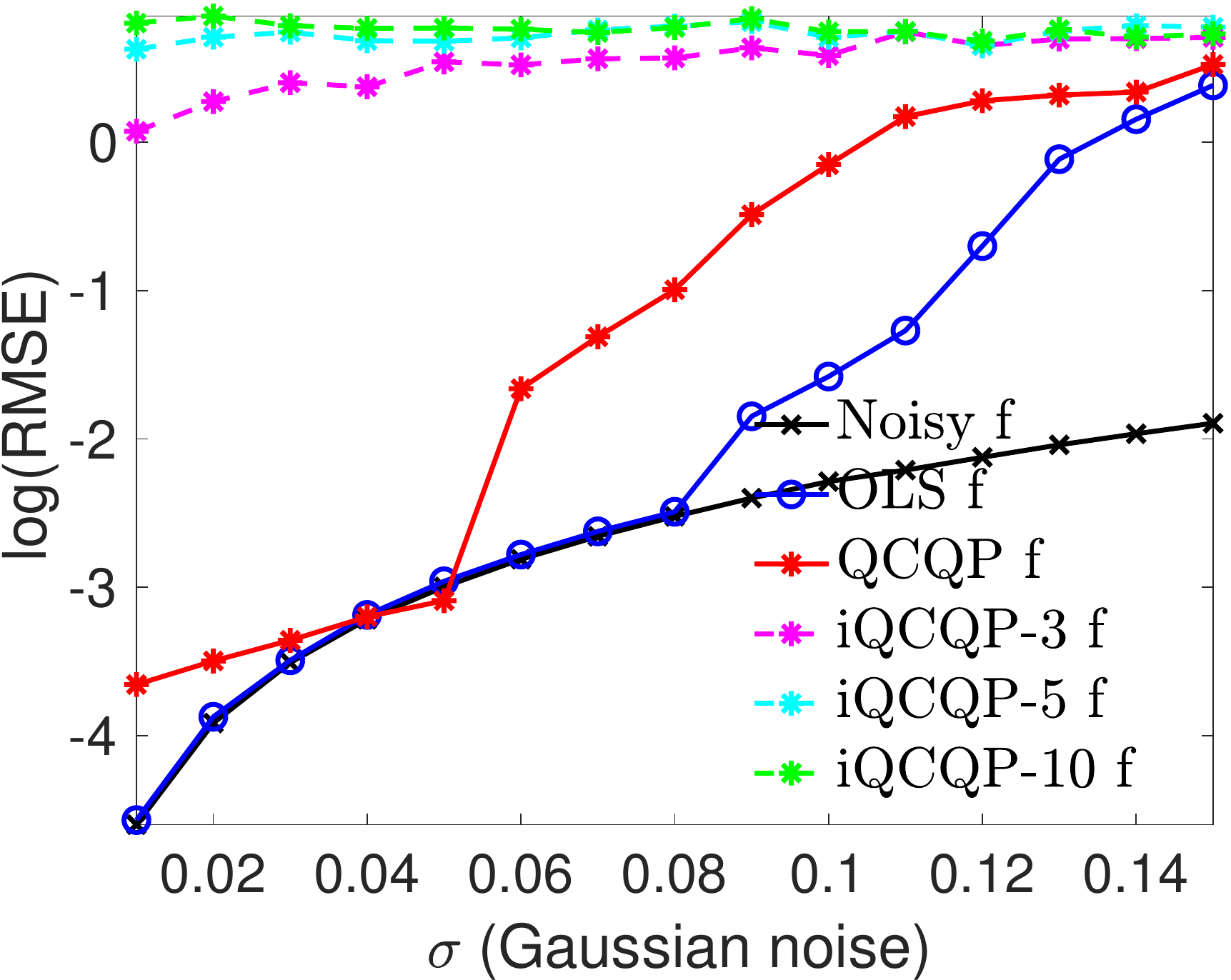} }
%
\subcaptionbox[]{ $k=5$, $\lambda= 1$}[ 0.19\textwidth ]
{\includegraphics[width=0.19\textwidth] {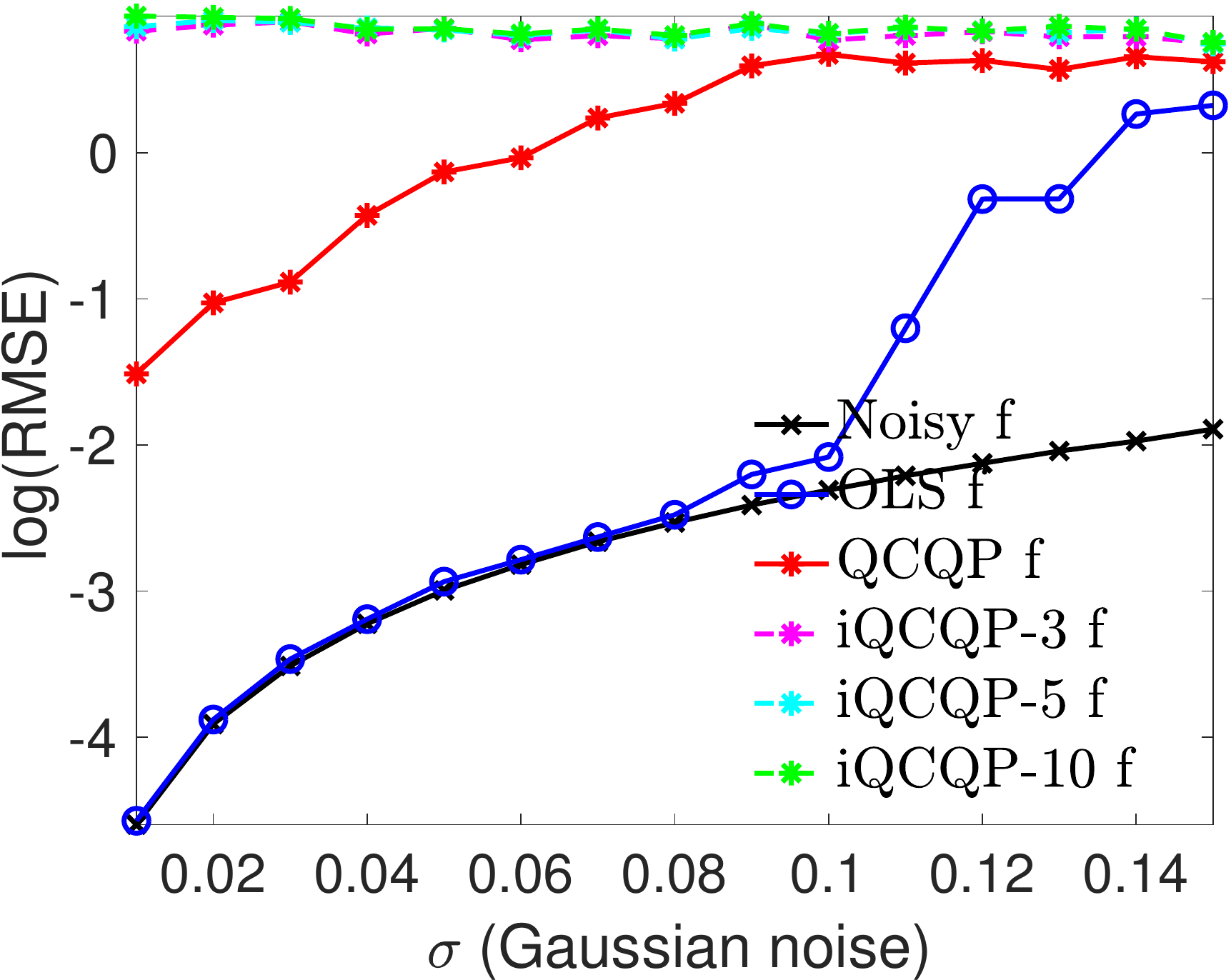} }
%
\vspace{-3mm}
\captionsetup{width=0.98\linewidth}
\caption[Short Caption]{Recovery errors for the final estimated $f$  samples, for $n=500$ under the Gaussian noise model (20 trials).  \textbf{QCQP} denotes Algorithm \ref{algo:two_stage_denoise} where the unwrapping stage is  performed via \textbf{OLS} \eqref{eq:ols_unwrap_lin_system}. 
}
\label{fig:Sims_f1_Gaussian_f}
\end{figure}
 

%
\subsection{Numerical experiments: Bernoulli Model} \label{sec:num_exps_Bernoulli}
%
Figure \ref{fig:instances_f1_Bernoulli_delta_cors} shows noisy instances of the Bernoulli-Uniform  noise model, for $n=500$ samples. 
As shown in the scatter plots (top row), most $\delta$ increments  take value 0 (since only a small fraction of the samples are corrupted by noise). For higher noise levels, the function \eqref{eq:qta_recovery_rule} will produce an increase numbers of errors in \eqref{eq:ols_unwrap_lin_system}.  The  remaining plots show the corresponding f mod 1 signal (clean, noisy, denoised via \textbf{QCQP}) for three levels of noise. 

Figure  \ref{fig:instances_f1_Bernoulli} shows instances of the recovery process, highlighting the noise level at which each method shows a significant decrease in performance. Note that at $p=10\%$ noise level, the simple \textbf{OLS} approach exhibits poor performance, while \textbf{QCQP} is able to denoise most of the  erroneous spikes. 
Interestingly, \textbf{iQCQP} performs worse than \textbf{QCQP} across all levels of noise. 

Finally, Figures \ref{fig:Sims_f1_Bernoulli_fmod1} and \ref{fig:Sims_f1_Bernoulli_f}  show recovery errors (averaged over 20 runs) for the estimated values of $f$ mod 1 and $f$, as we scan across the input parameters $k  \in \{2,3,5\}$ and  $ \lambda   \in  \{ 0.03, 0.1, 0.3, 0.5, 1\}$, at varying levels of noise $p \in [0, 0.15]$. We remark that $k=5$ clearly returns the worst performance, and higher values of $\lambda$ yield better results for \textbf{QCQP}. Interestingly, the behavior of  \textbf{iQCQP}  oscillates, and its performance with respect to \textbf{QCQP}  depends on $k$ and $\lambda$. The best recovery errors are obtained for $k=2$ and $\lambda=1$.

\begin{figure}[!ht]
\centering
\subcaptionbox[]{  $p=0.1$
}[ 0.32\textwidth ]
{\includegraphics[width=0.31\textwidth] {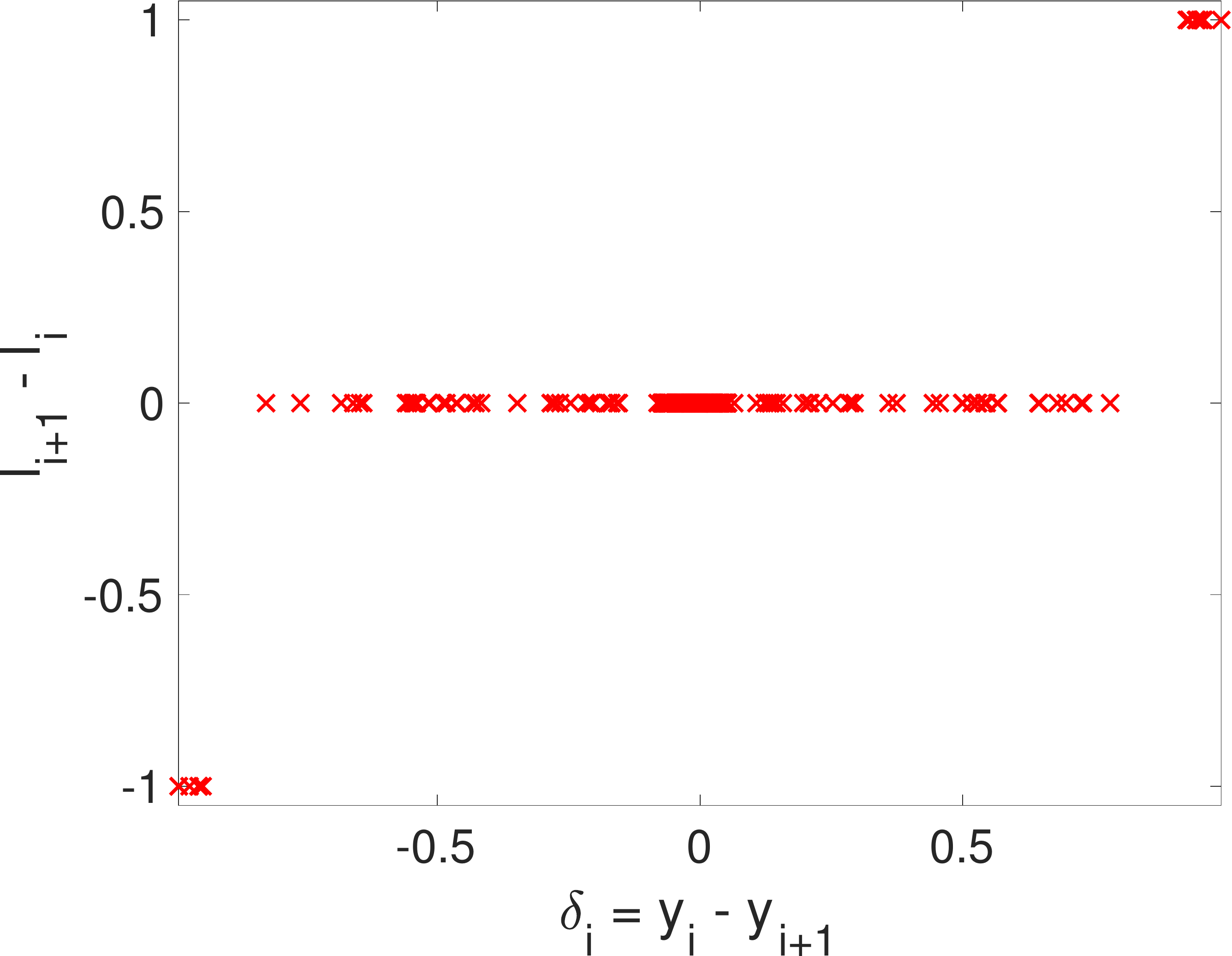} }
%
\subcaptionbox[]{  $p=0.15$
}[ 0.32\textwidth ]
{\includegraphics[width=0.31\textwidth] {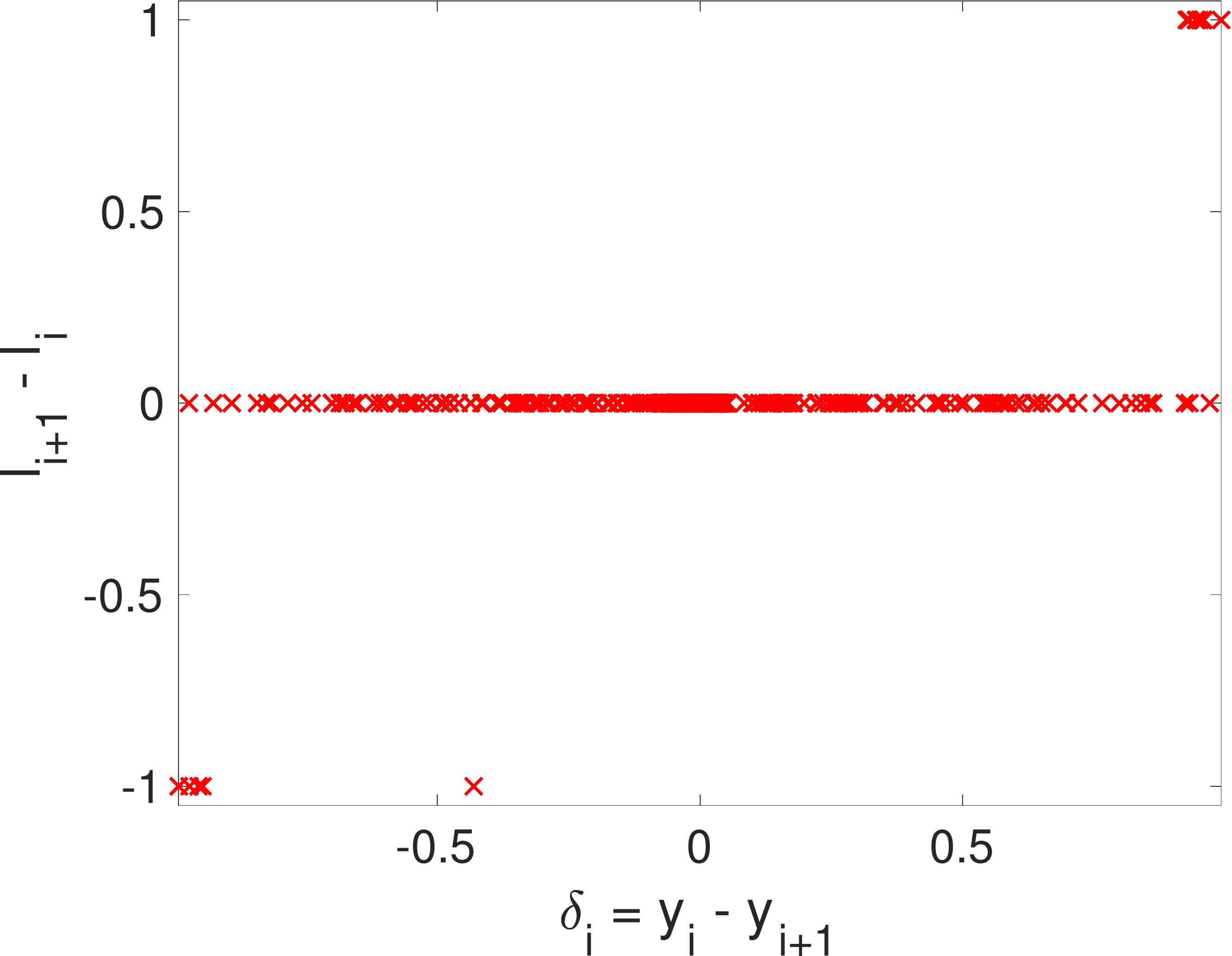} }
%
\subcaptionbox[]{  $p=0.2$
}[ 0.32\textwidth ]
{\includegraphics[width=0.31\textwidth] {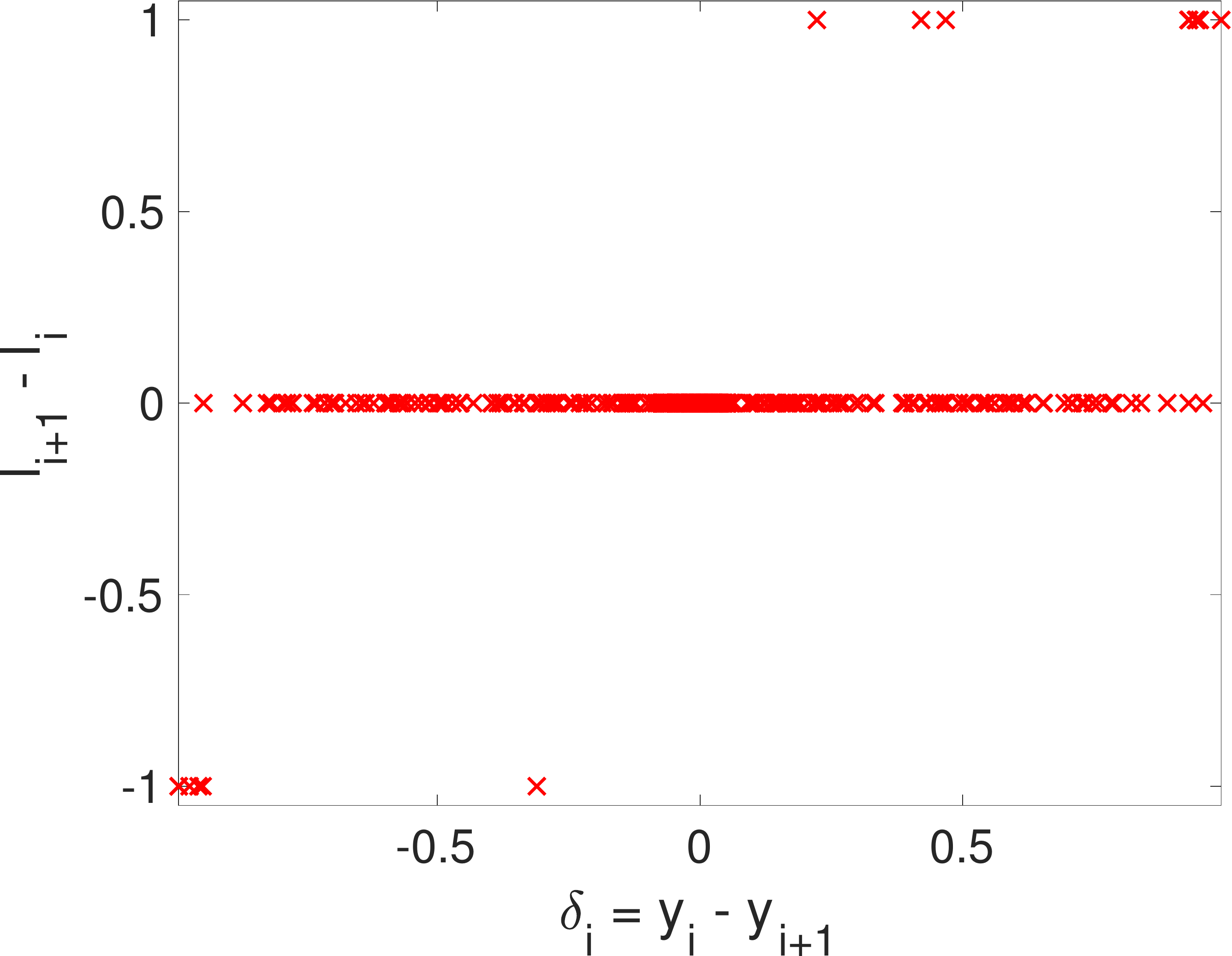} }
\vspace{2mm}

\subcaptionbox[]{  $p=0.1$
}[ 0.96\textwidth ]
{\includegraphics[width=0.95\textwidth] {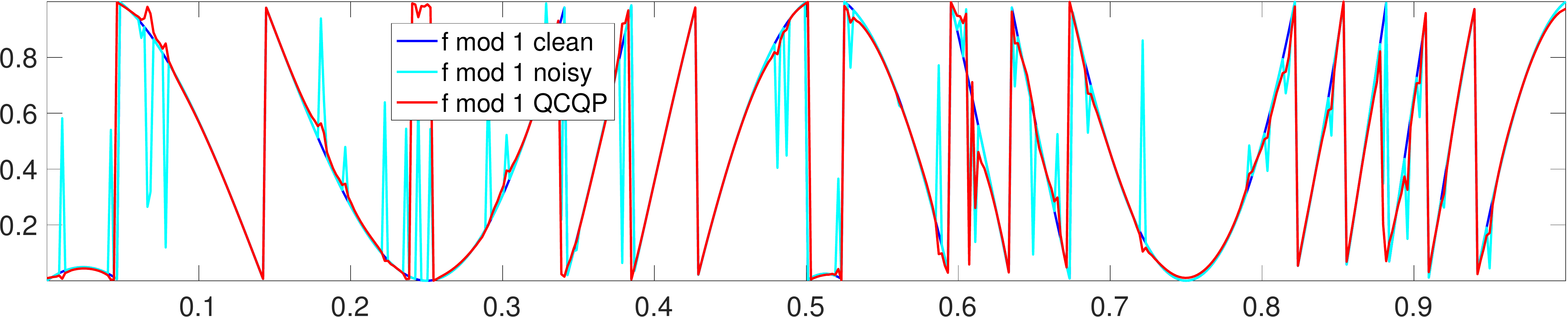} }
%

\vspace{2mm}
\subcaptionbox[]{  $p=0.15$
}[ 0.96\textwidth ]
{\includegraphics[width=0.95\textwidth] {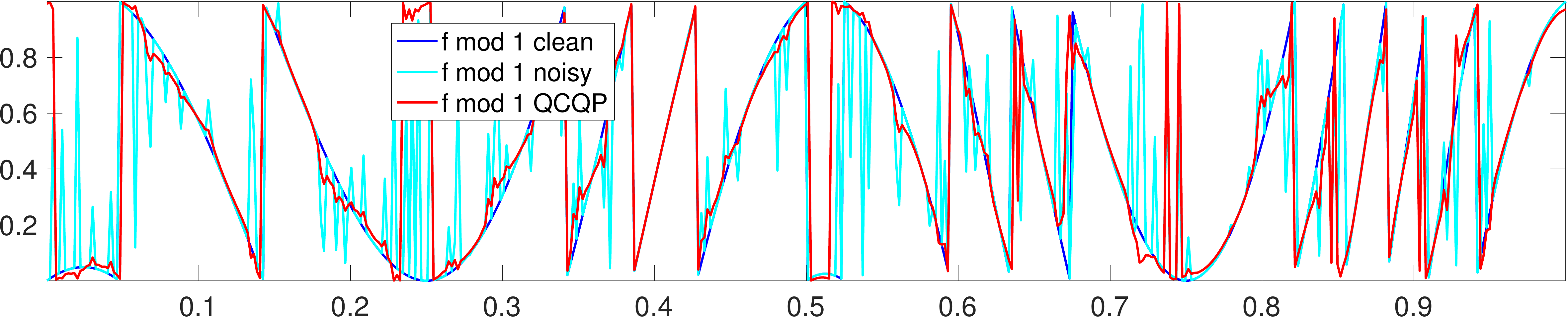} }
%

\vspace{2mm}
\subcaptionbox[]{  $p=0.2$
}[ 0.96\textwidth ]
{\includegraphics[width=0.95\textwidth] {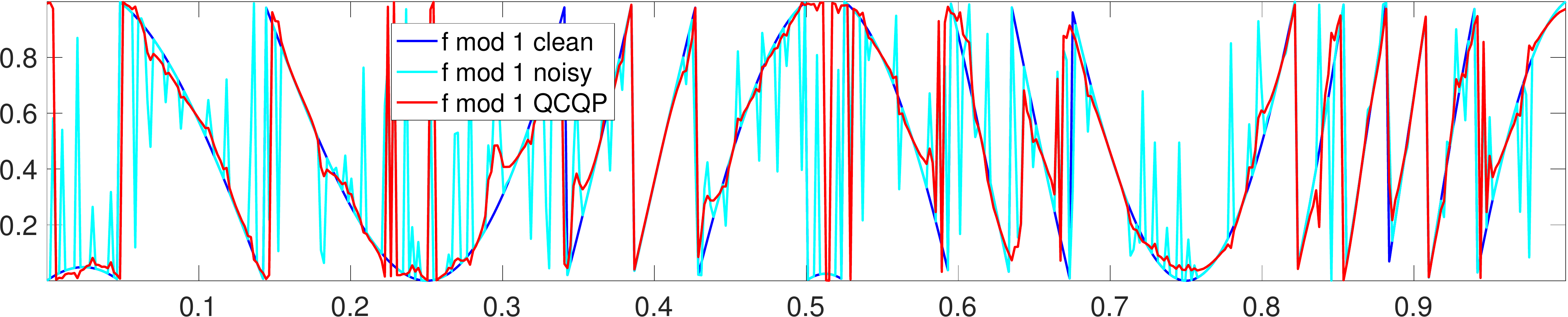} }
%

\captionsetup{width=0.95\linewidth}
\caption[Short Caption]{ Noisy instances of the Bernoulli noise model ($n=500$). Top row (a)-(c) shows scatter plots of change in $y$ (the observed noisy f mod 1 values) versus change in $l$ (the noisy quotient). Plots (d)-(f) show  the clean f mod 1 values (blue), the noisy f mod 1 values (cyan) and the denoised (via \textbf{QCQP}) f mod 1 values (red).  \textbf{QCQP} denotes Algorithm \ref{algo:two_stage_denoise} without  the unwrapping stage performed by \textbf{OLS} \eqref{eq:ols_unwrap_lin_system}.  
}
\label{fig:instances_f1_Bernoulli_delta_cors}
\end{figure}

\begin{figure}[!ht]
\centering
\subcaptionbox[]{  $p=0.10$, \textbf{OLS}
}[ 0.32\textwidth ]
{\includegraphics[width=0.27\textwidth] {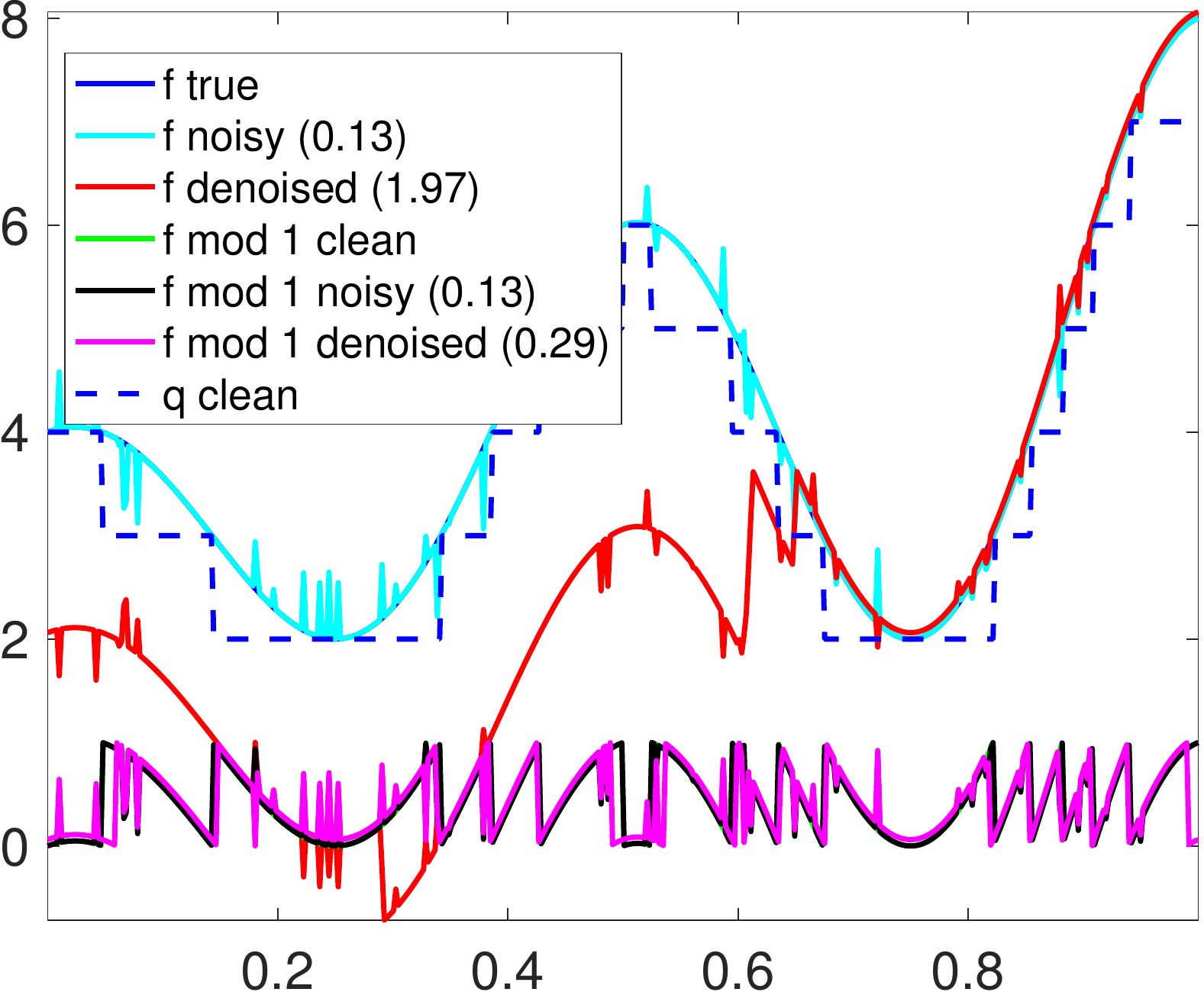} }
%
\subcaptionbox[]{  $p=0.10$, \textbf{QCQP}
}[ 0.32\textwidth ]
{\includegraphics[width=0.27\textwidth] {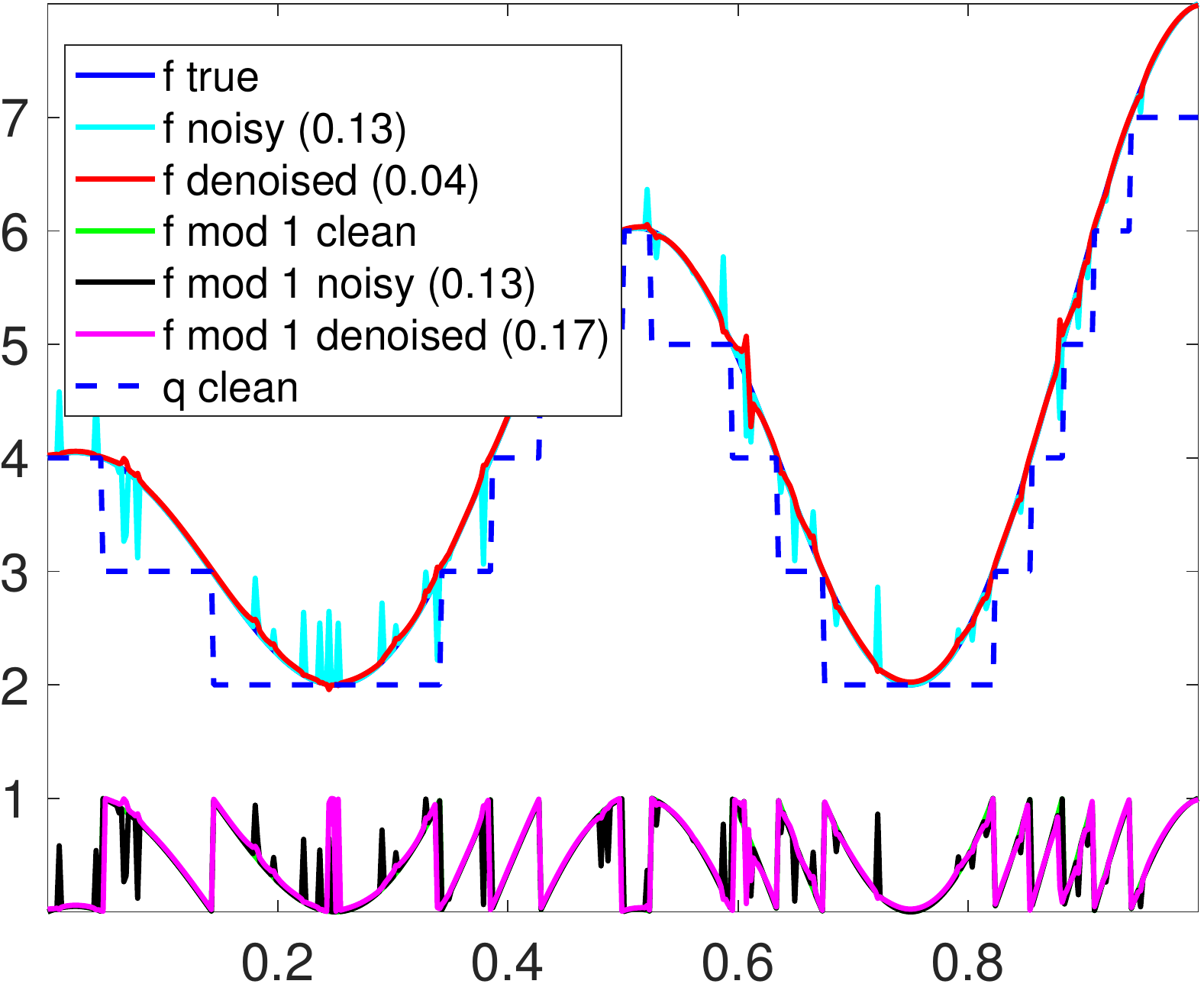} }
%
\subcaptionbox[]{  $p=0.10$, \textbf{iQCQP}  (10  iters.)
}[ 0.32\textwidth ]
{\includegraphics[width=0.27\textwidth] {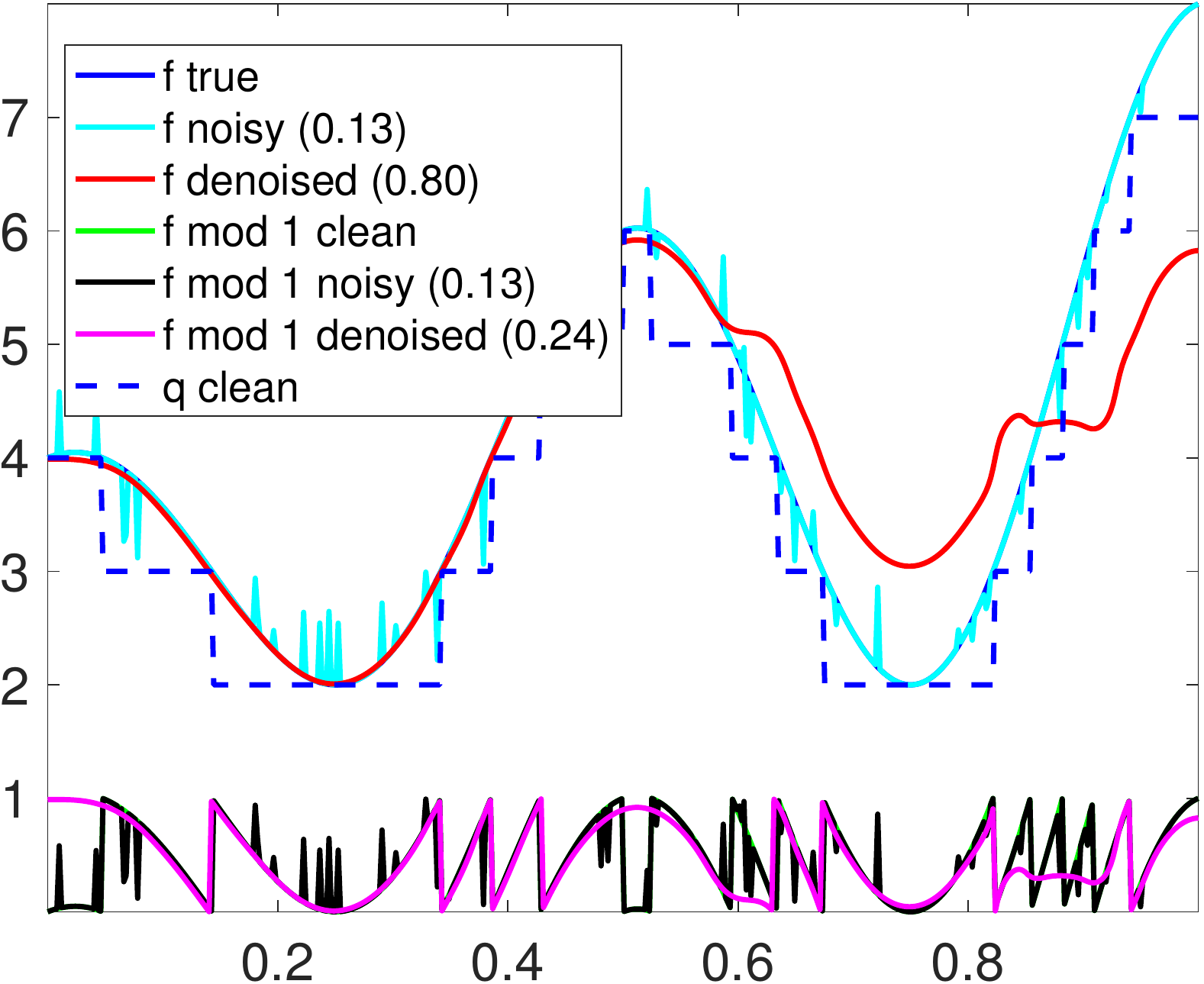} }
%
%
%
\vspace{4mm}
\subcaptionbox[]{  $p=0.2$, \textbf{OLS}
}[ 0.32\textwidth ]
{\includegraphics[width=0.27\textwidth] {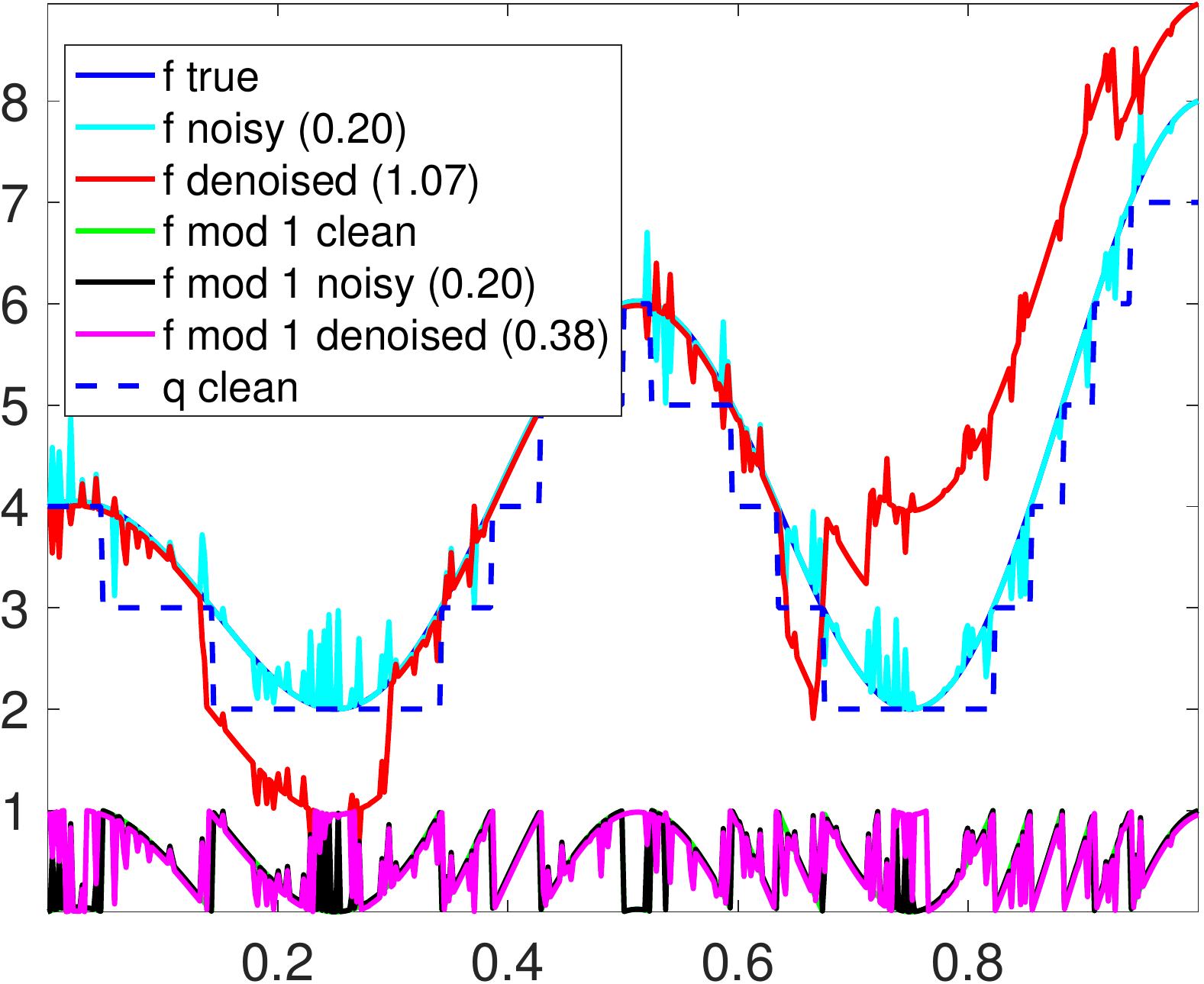} }
%
\subcaptionbox[]{  $p=0.2$, \textbf{QCQP}
}[ 0.32\textwidth ]
{\includegraphics[width=0.27\textwidth] {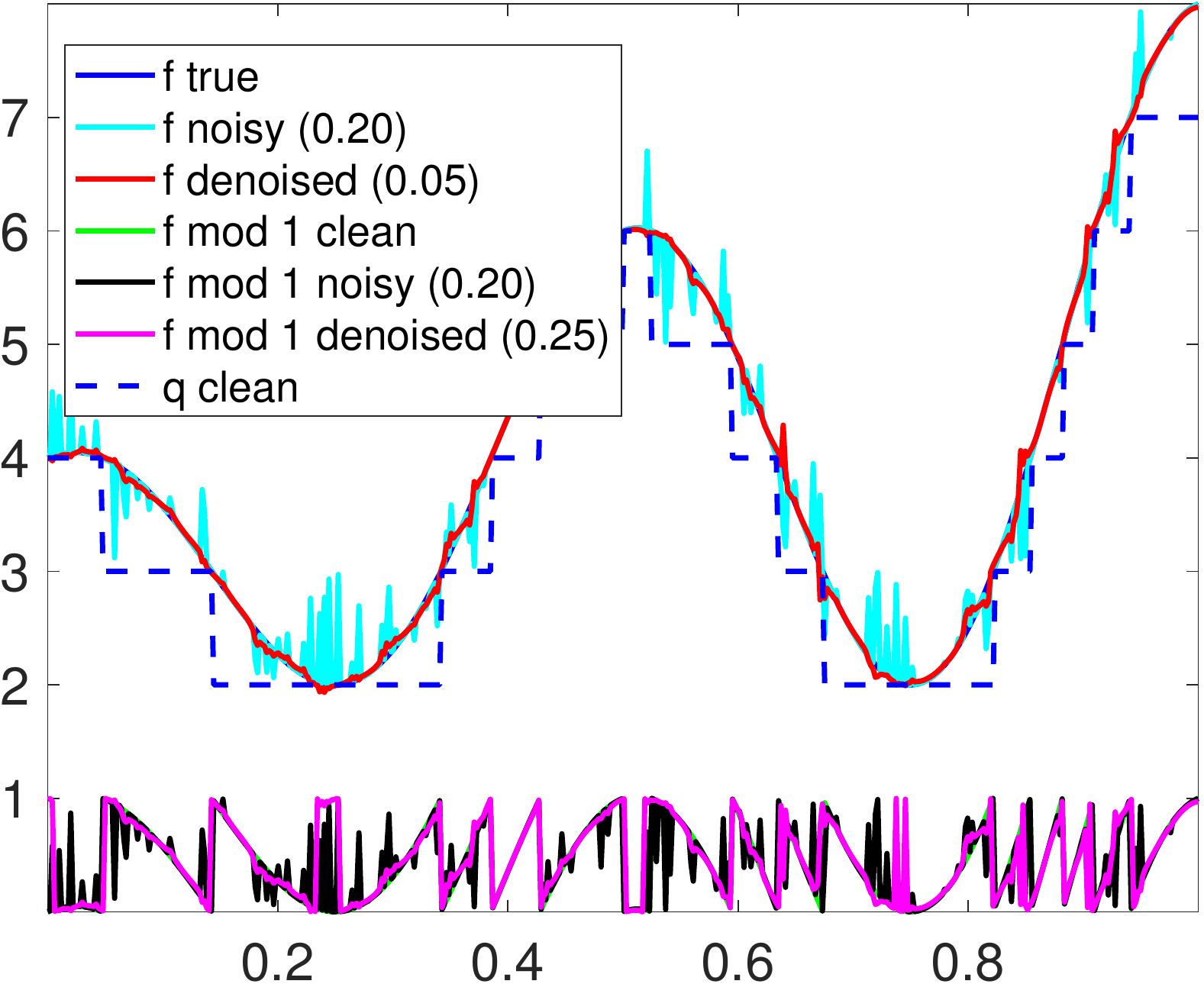} }
%
\subcaptionbox[]{  $p=0.2$, \textbf{iQCQP}  (10  iters.)
}[ 0.32\textwidth ]
{\includegraphics[width=0.27\textwidth] {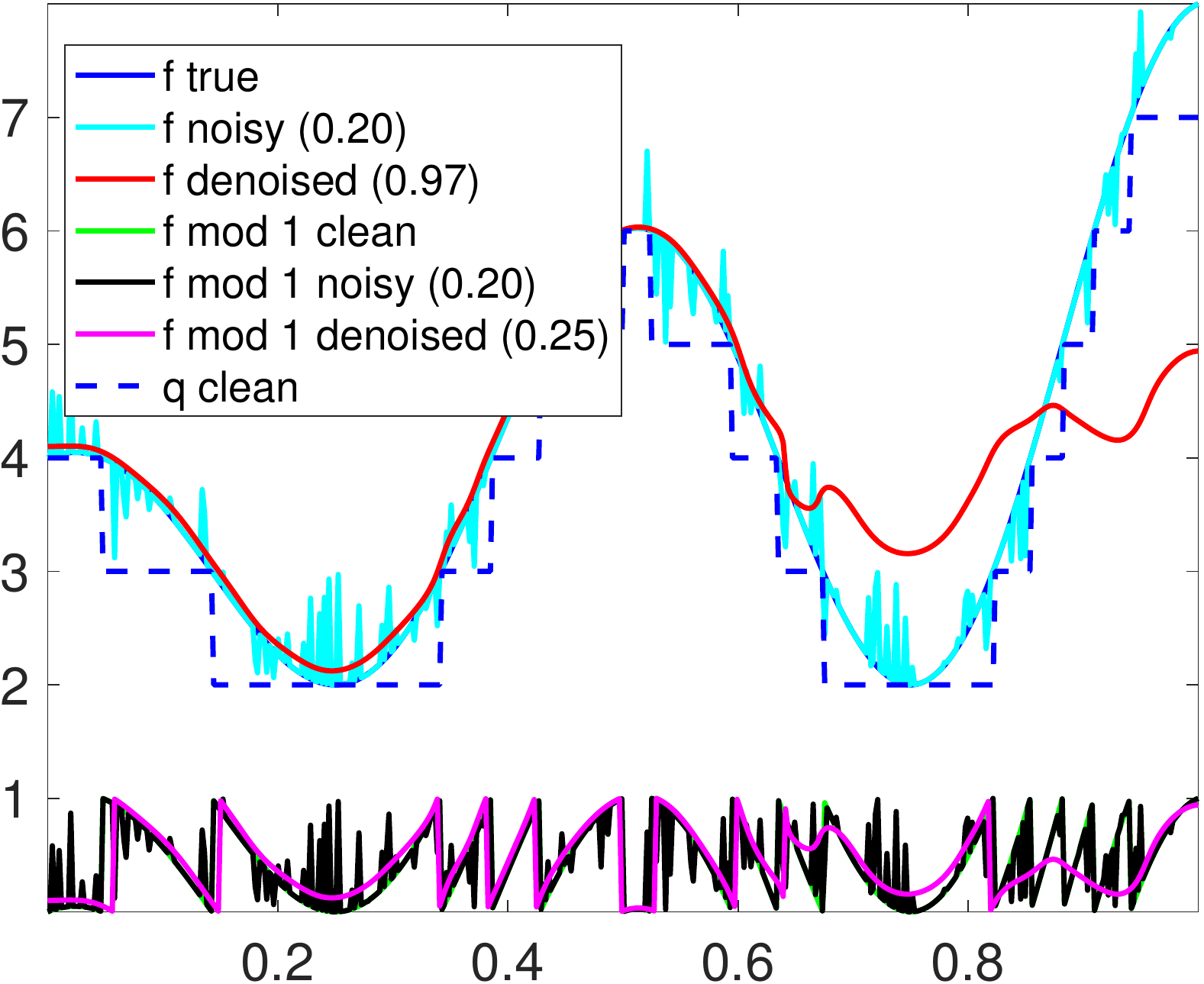} }
%
%
\subcaptionbox[]{  $p=0.25$, \textbf{OLS}
}[ 0.32\textwidth ]
{\includegraphics[width=0.27\textwidth] {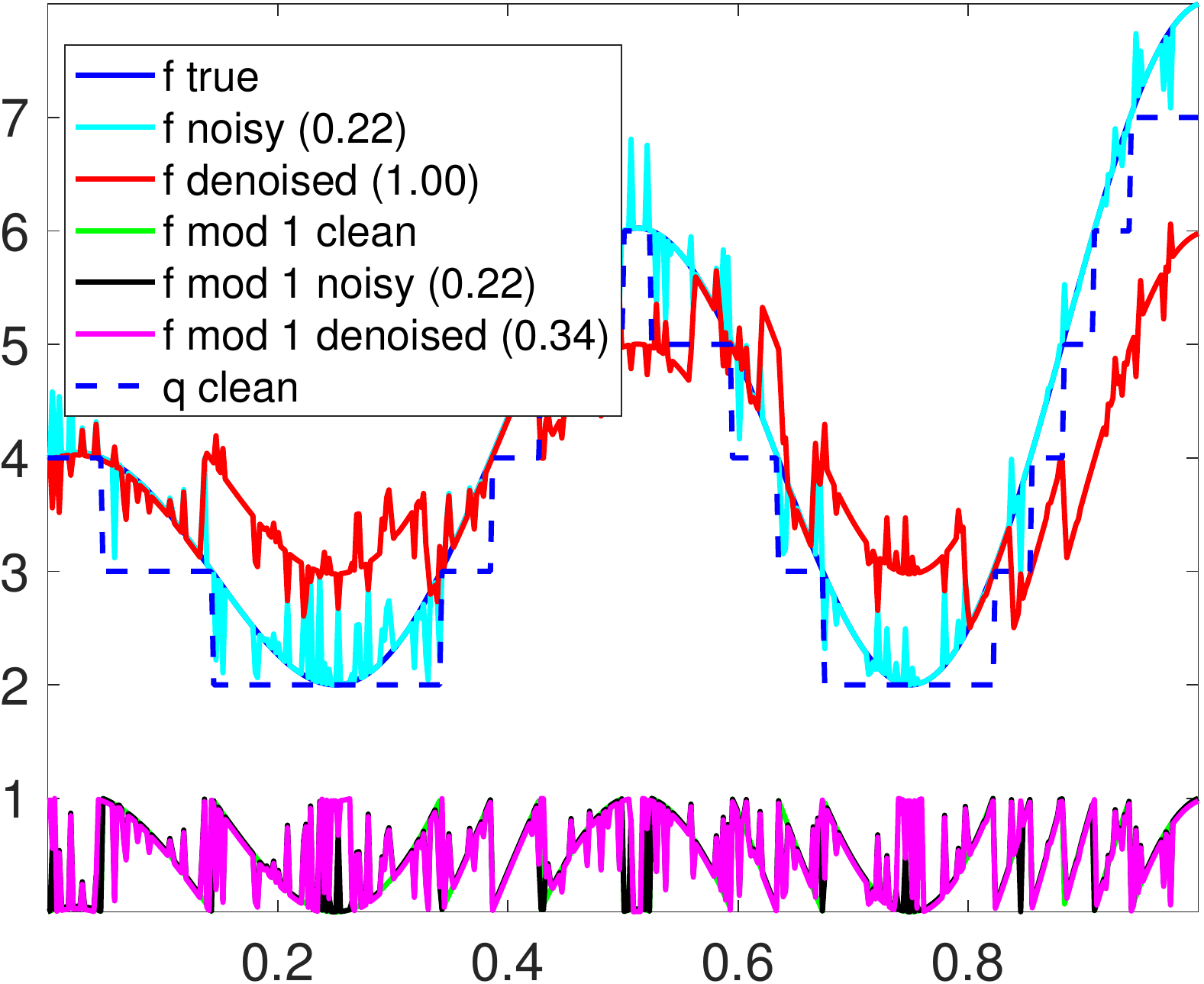} }
%
\subcaptionbox[]{  $p=0.25$, \textbf{QCQP}
}[ 0.32\textwidth ]
{\includegraphics[width=0.27\textwidth] {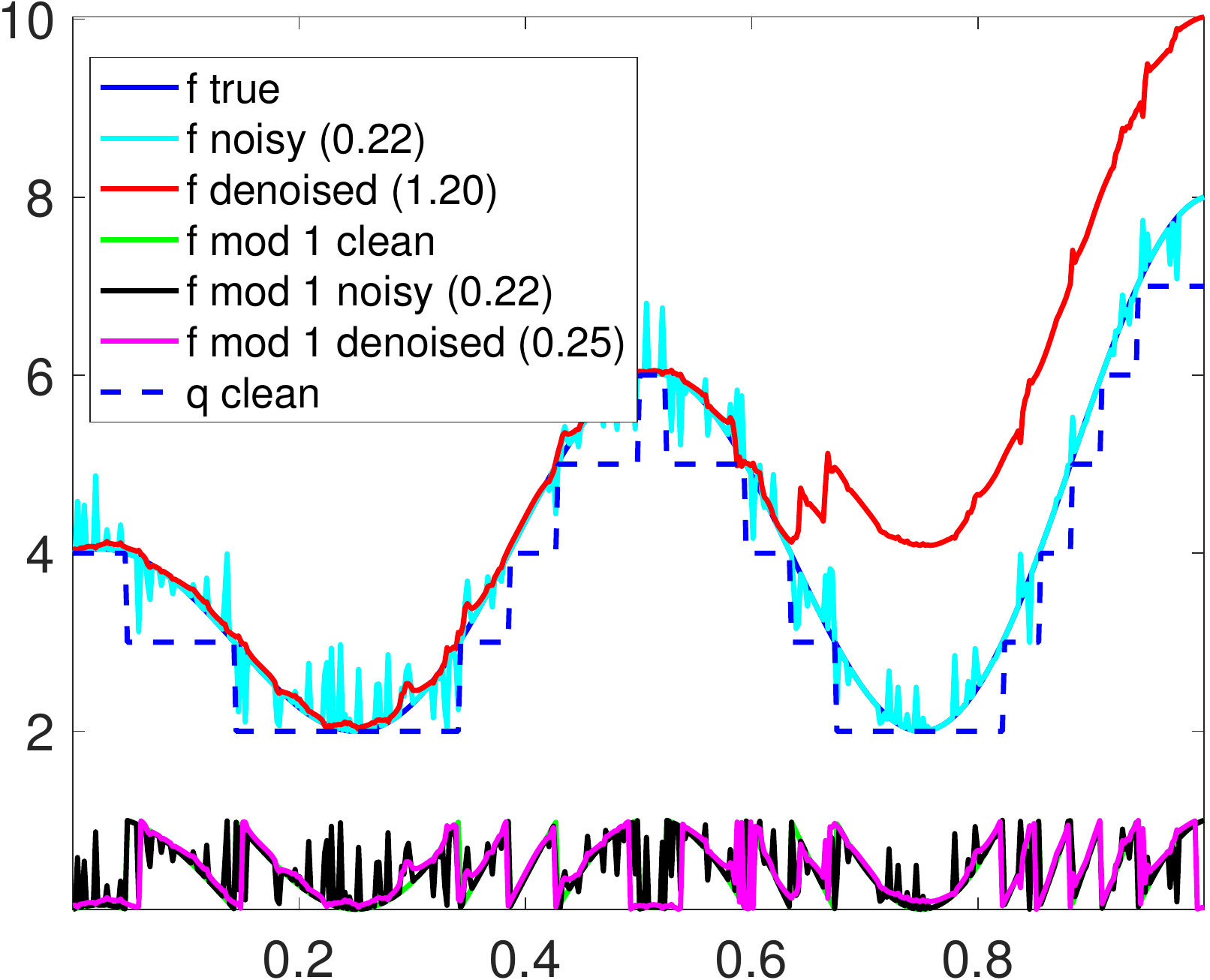} }
%
\subcaptionbox[]{  $p=0.25$, \textbf{iQCQP}  (10  iters.)
}[ 0.32\textwidth ]
{\includegraphics[width=0.27\textwidth] {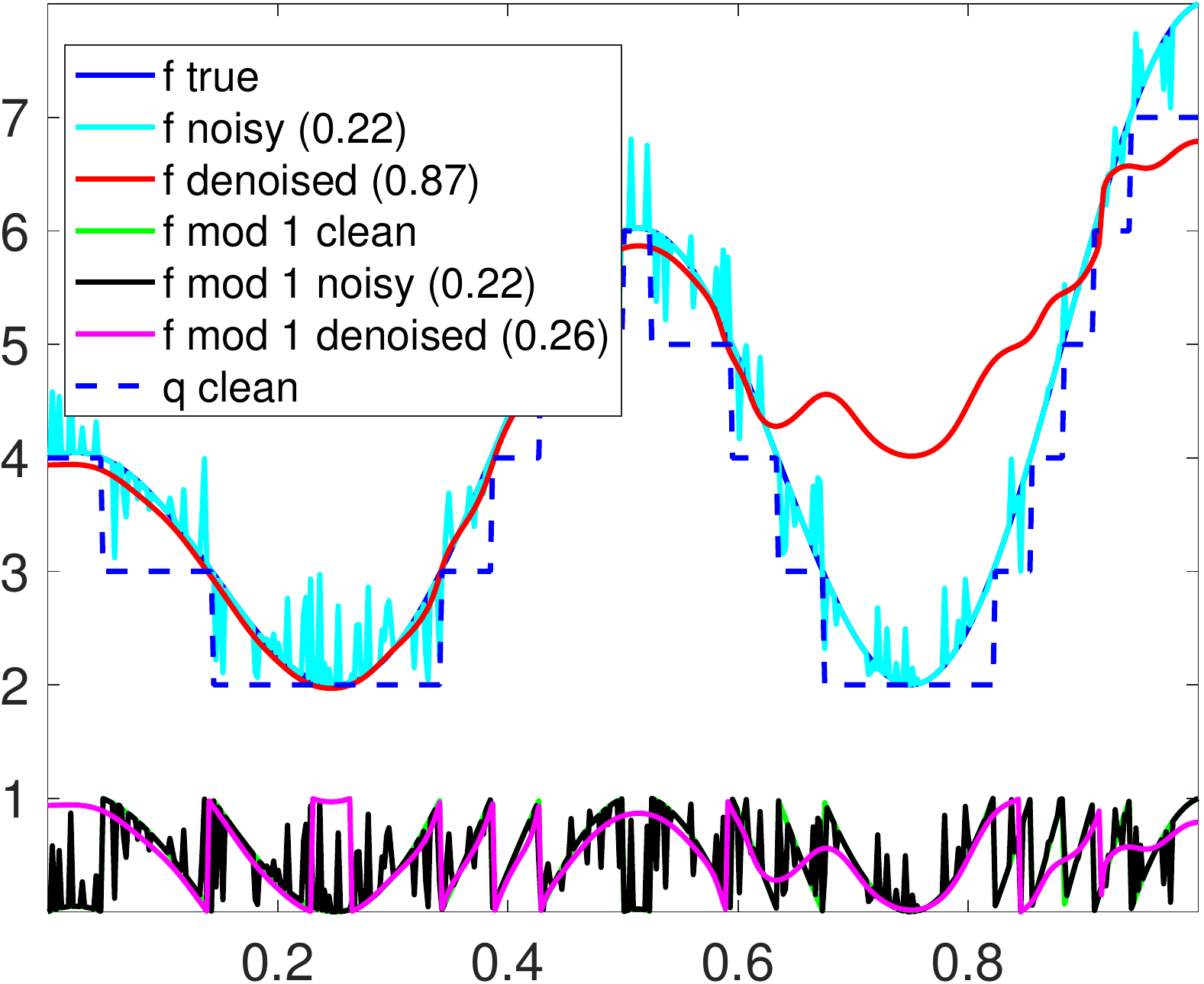} }
\vspace{-2mm}
\captionsetup{width=0.95\linewidth}
\caption[Short Caption]{Denoised instances under the Bernoulli noise model, for \textbf{OLS}, \textbf{QCQP} (Algorithm \ref{algo:two_stage_denoise}) and \textbf{iQCQP},  as we increase the noise level $\sigma$.   \textbf{QCQP} denotes Algorithm \ref{algo:two_stage_denoise}, for which  the unwrapping stage is  performed via \textbf{OLS} \eqref{eq:ols_unwrap_lin_system}.    We keep fixed the parameters $n=500$, $k=2$, $\lambda= 1$. 
}
\label{fig:instances_f1_Bernoulli}
\end{figure}


\begin{figure}[!ht]
\centering
\subcaptionbox[]{ $k=2$, $\lambda= 0.03$}[ 0.19\textwidth ]
{\includegraphics[width=0.19\textwidth] {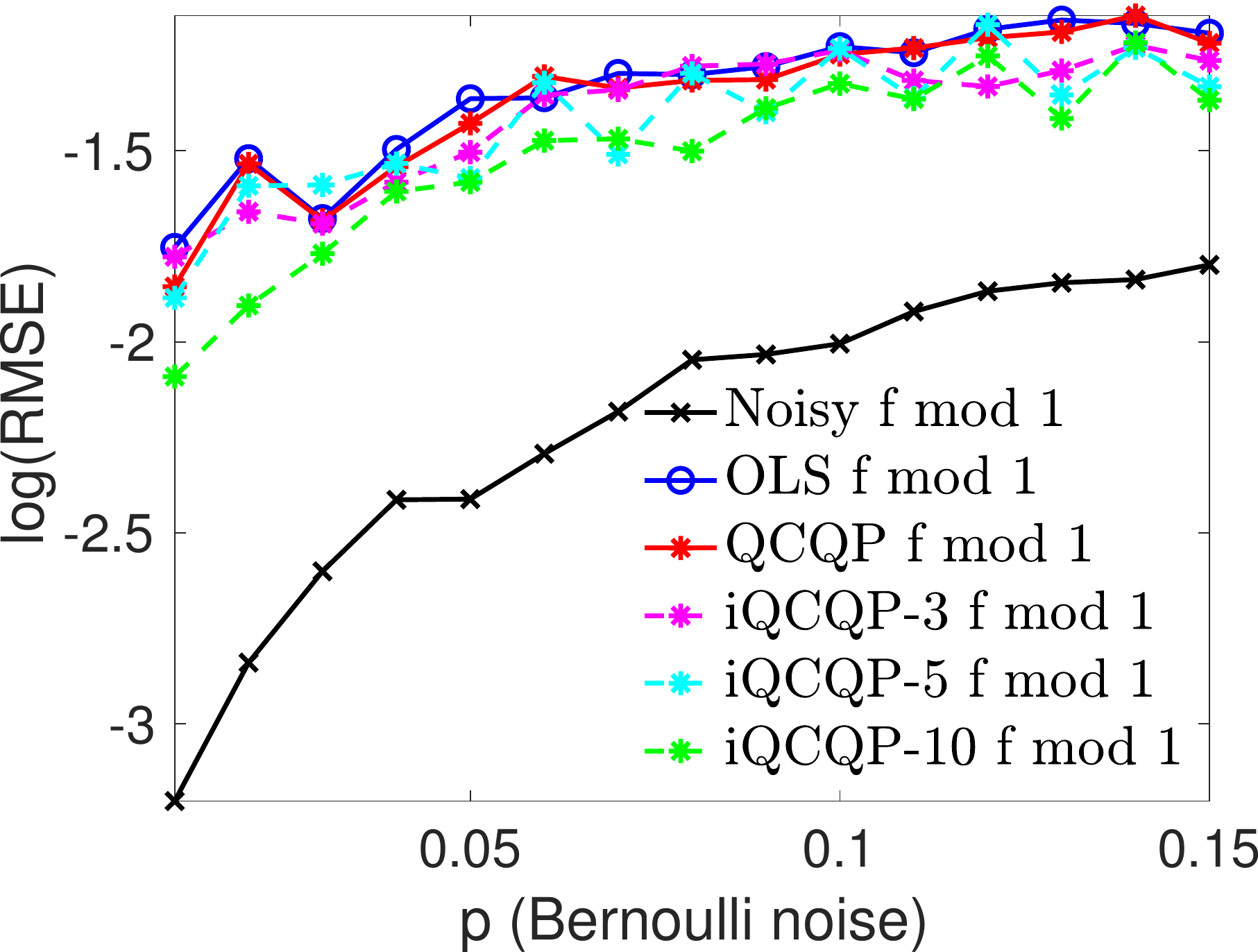} }
\subcaptionbox[]{ $k=2$, $\lambda= 0.1$}[ 0.19\textwidth ]
{\includegraphics[width=0.19\textwidth] {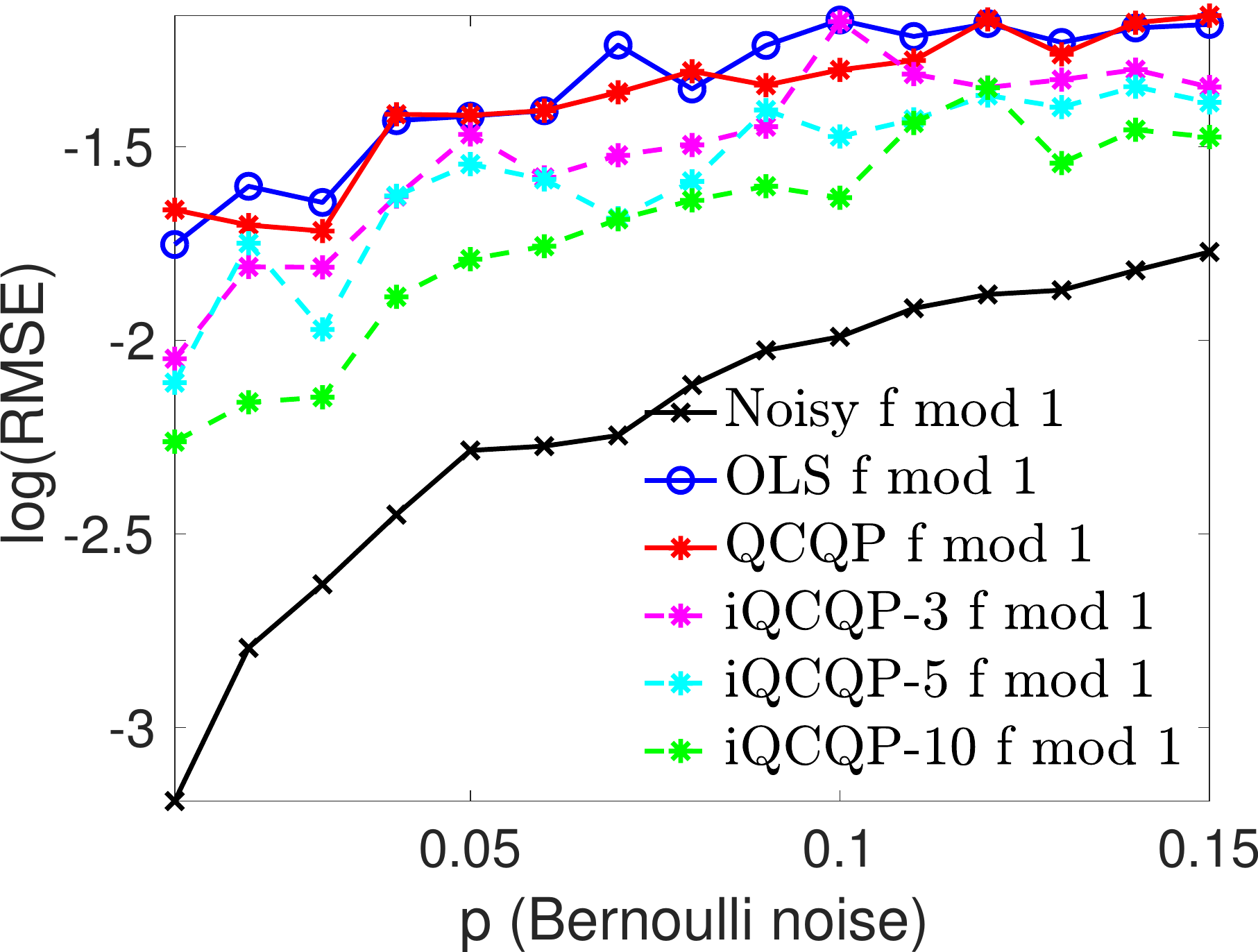} }
%
\subcaptionbox[]{ $k=2$, $\lambda= 0.3$}[ 0.19\textwidth ]
{\includegraphics[width=0.19\textwidth] {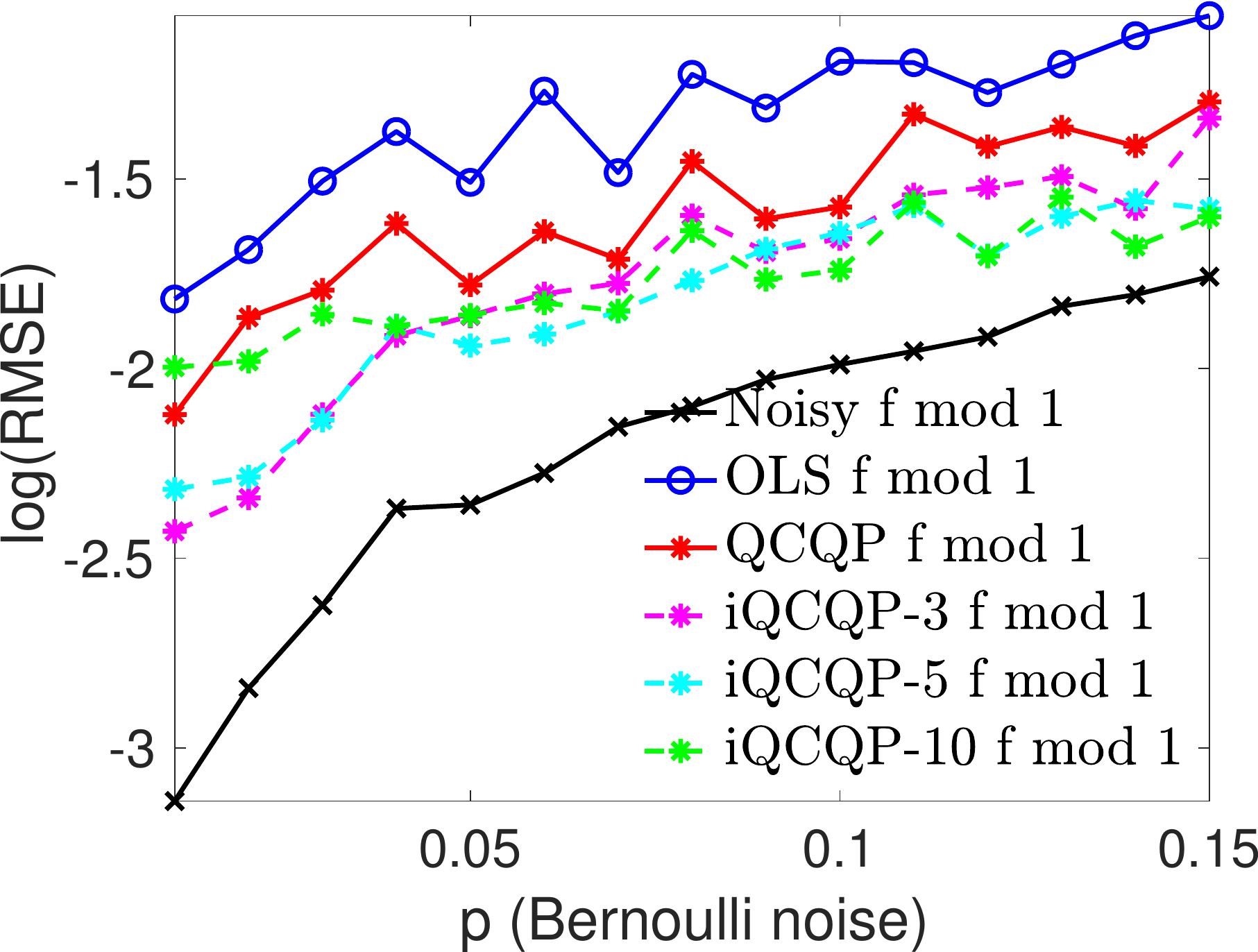} }
%
\subcaptionbox[]{ $k=2$, $\lambda= 0.5$}[ 0.19\textwidth ]
{\includegraphics[width=0.19\textwidth] {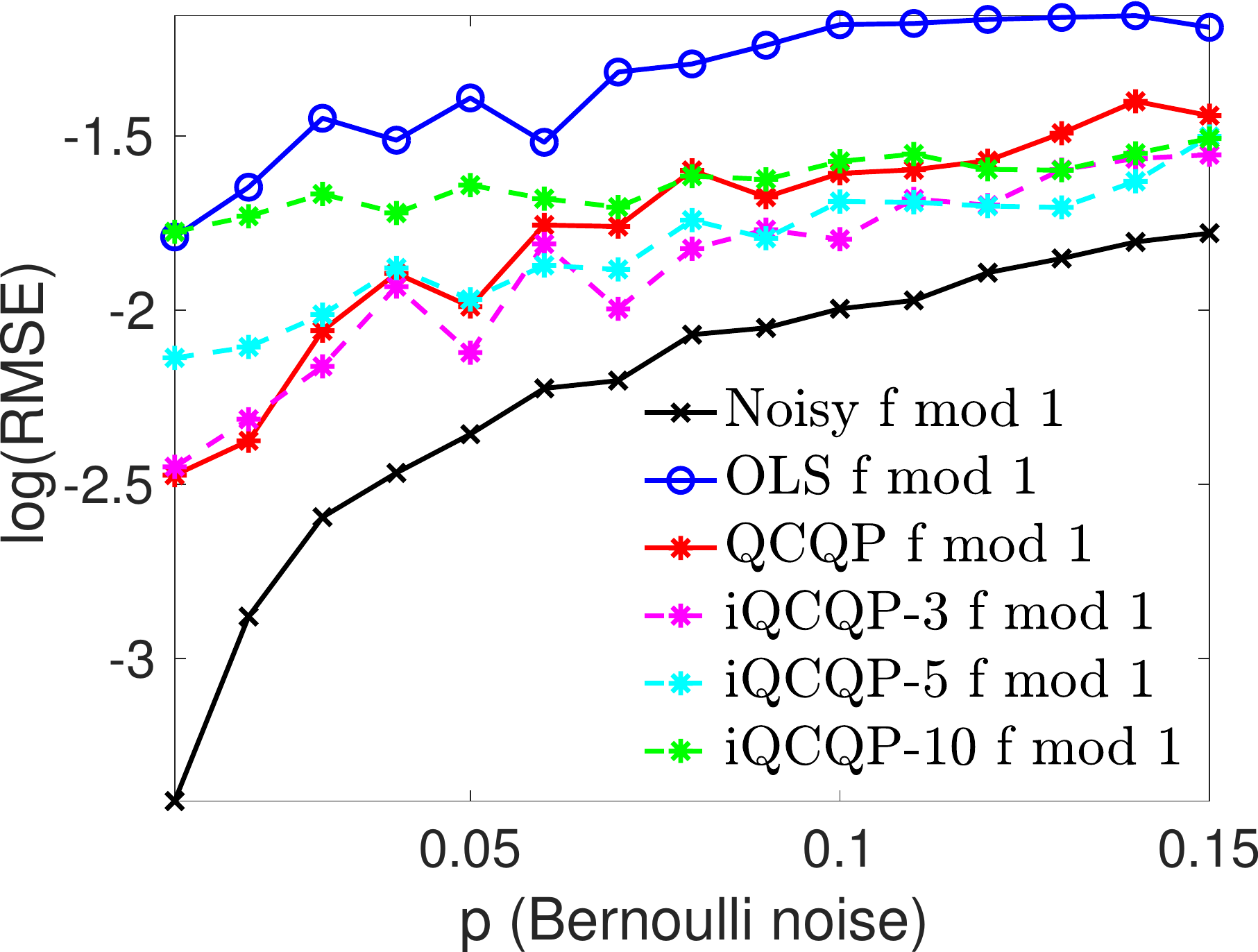} }
%
\subcaptionbox[]{ $k=2$, $\lambda= 1$}[ 0.19\textwidth ]
{\includegraphics[width=0.19\textwidth] {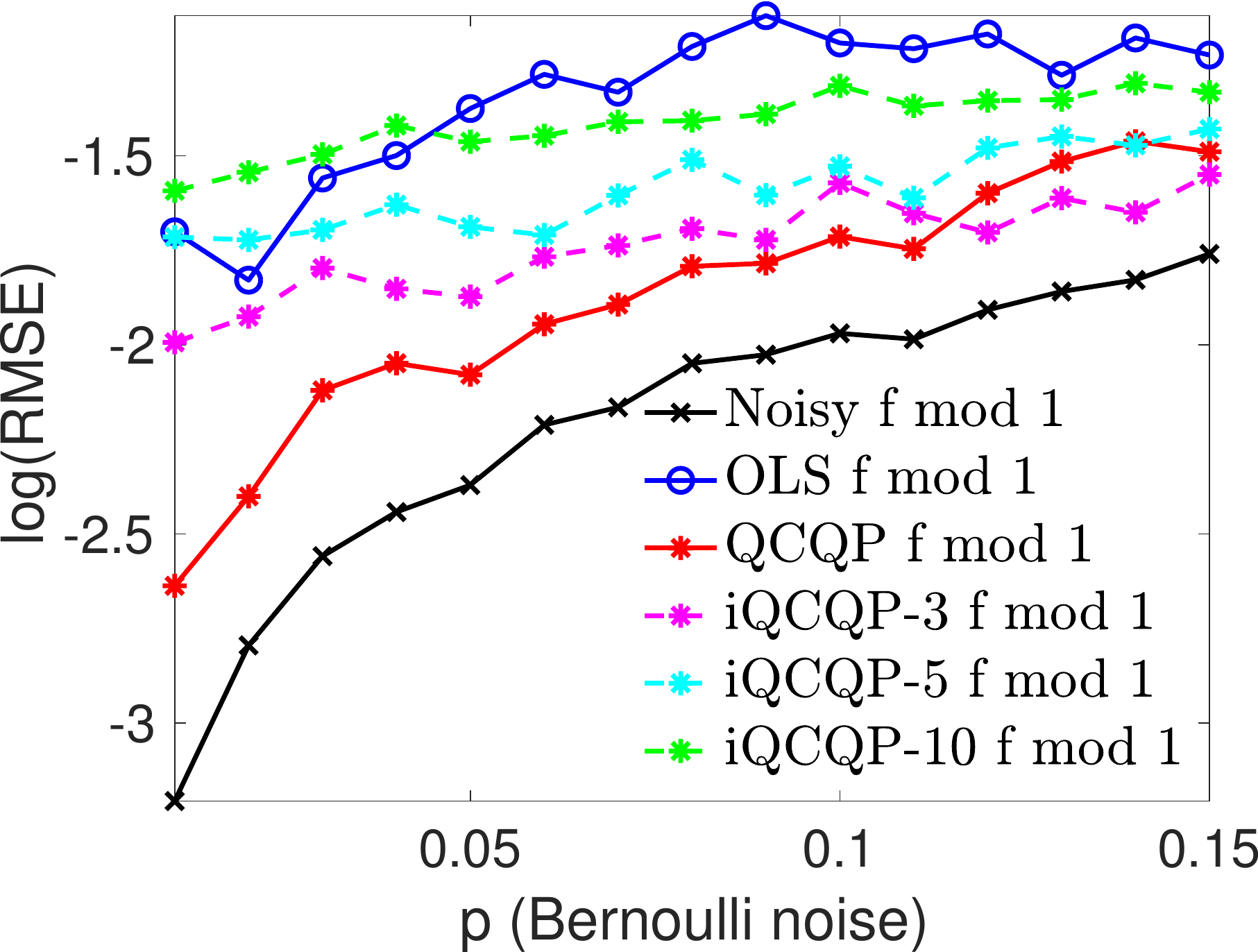} }
\hspace{0.1\textwidth} 
\subcaptionbox[]{ $k=3$, $\lambda= 0.03$}[ 0.19\textwidth ]
{\includegraphics[width=0.19\textwidth] {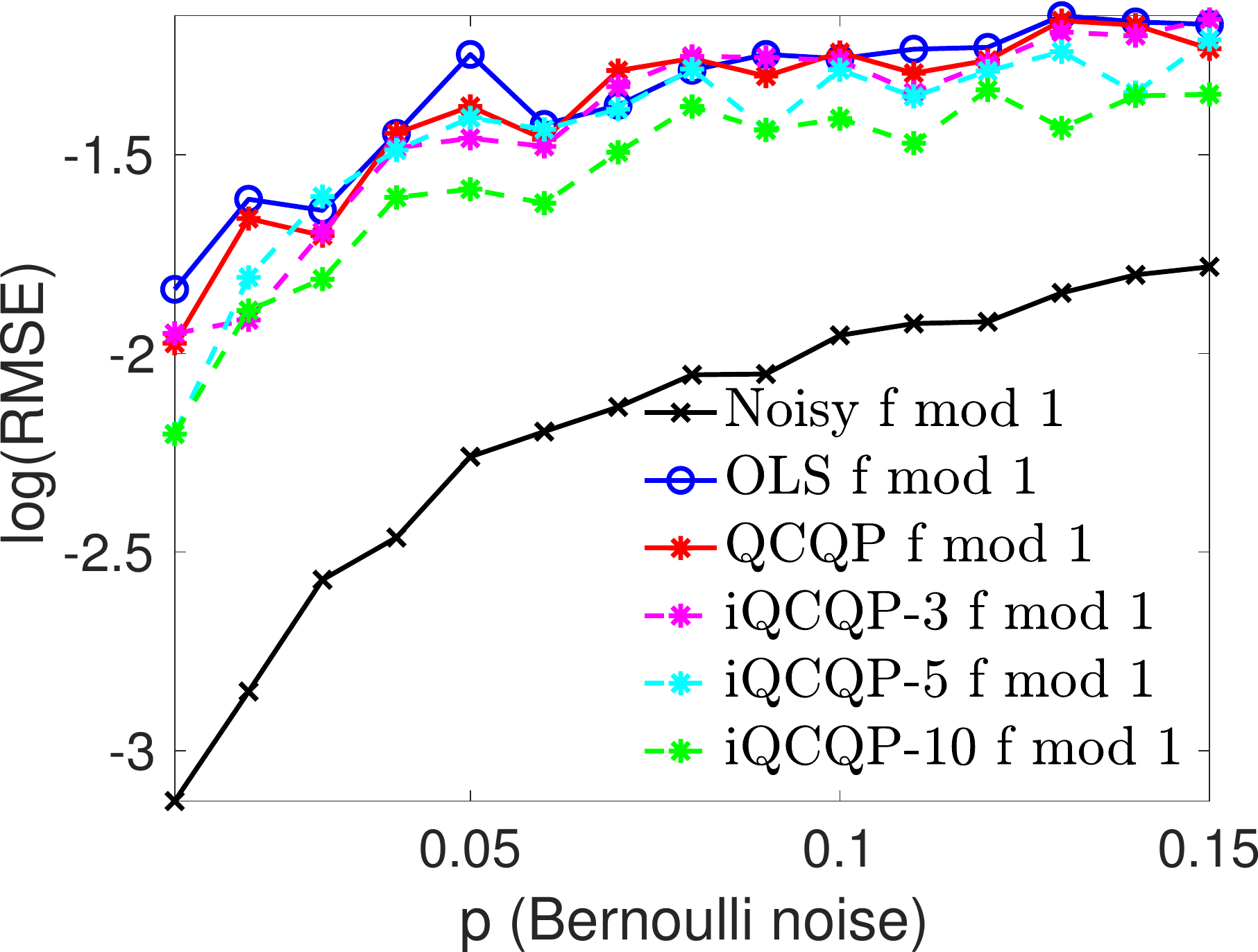} }
\subcaptionbox[]{ $k=3$, $\lambda= 0.1$}[ 0.19\textwidth ]
{\includegraphics[width=0.19\textwidth] {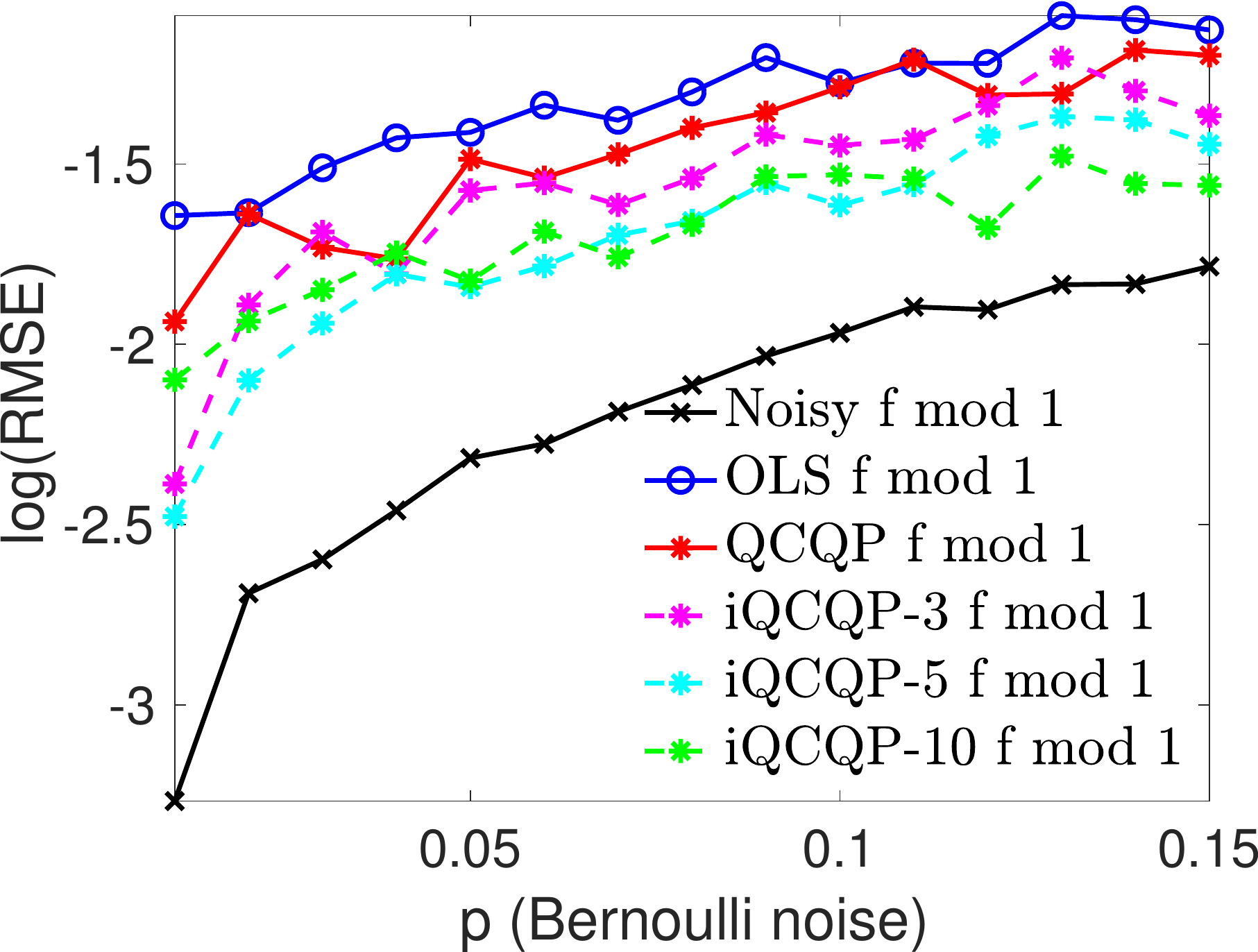} }
%
\subcaptionbox[]{ $k=3$, $\lambda= 0.3$}[ 0.19\textwidth ]
{\includegraphics[width=0.19\textwidth] {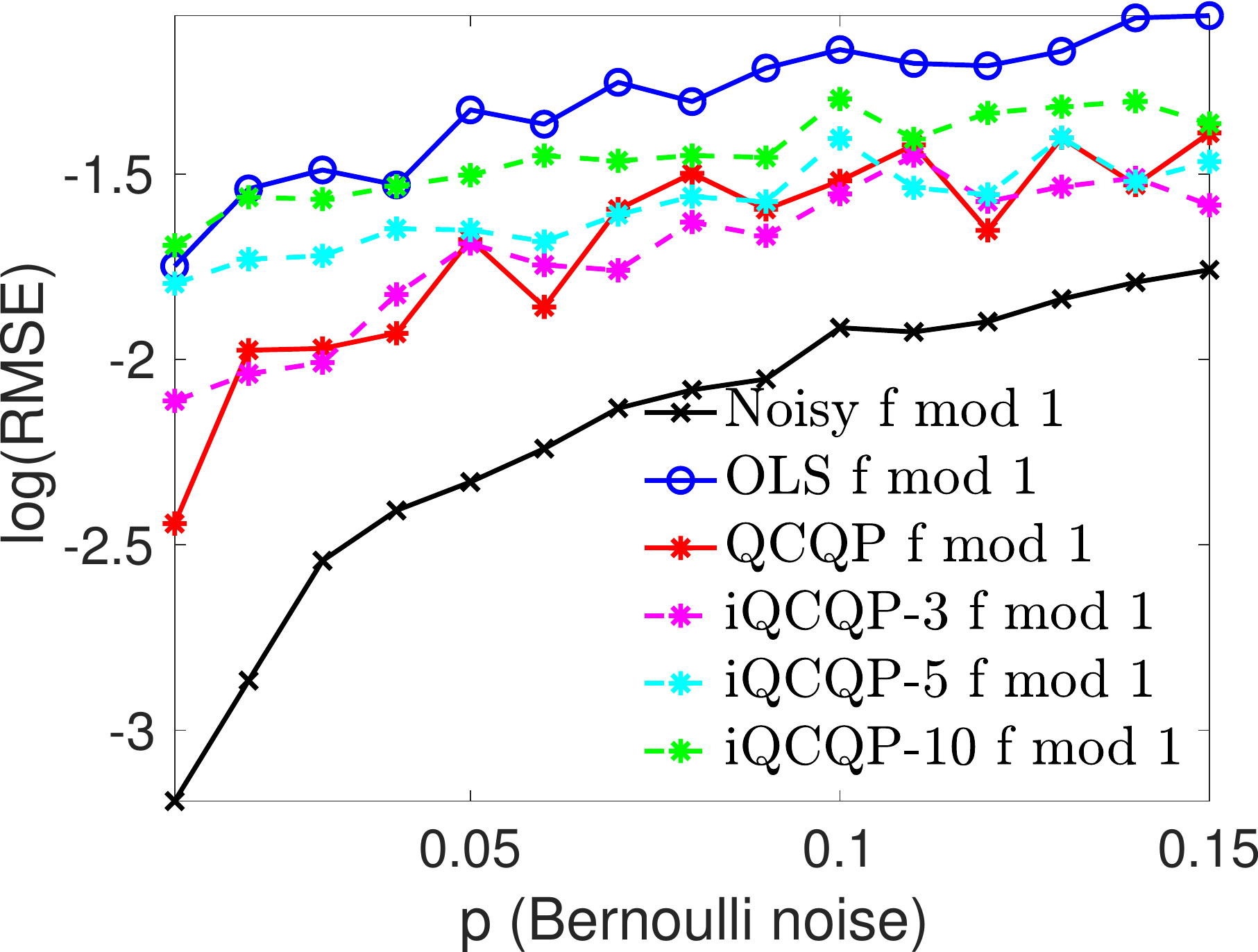} }
%
\subcaptionbox[]{ $k=3$, $\lambda= 0.5$}[ 0.19\textwidth ]
{\includegraphics[width=0.19\textwidth] {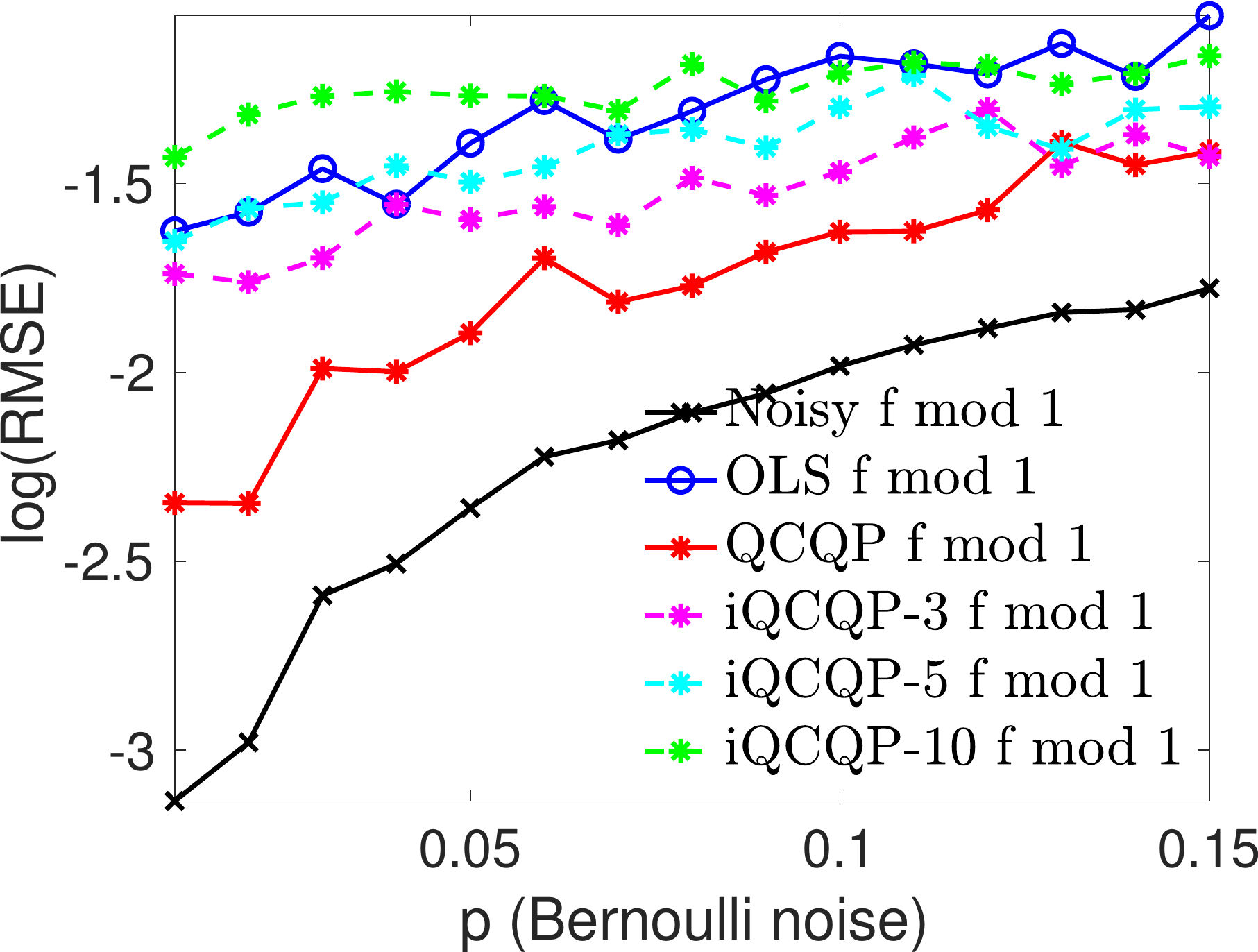} }
%
\subcaptionbox[]{ $k=3$, $\lambda= 1$}[ 0.19\textwidth ]
{\includegraphics[width=0.19\textwidth] {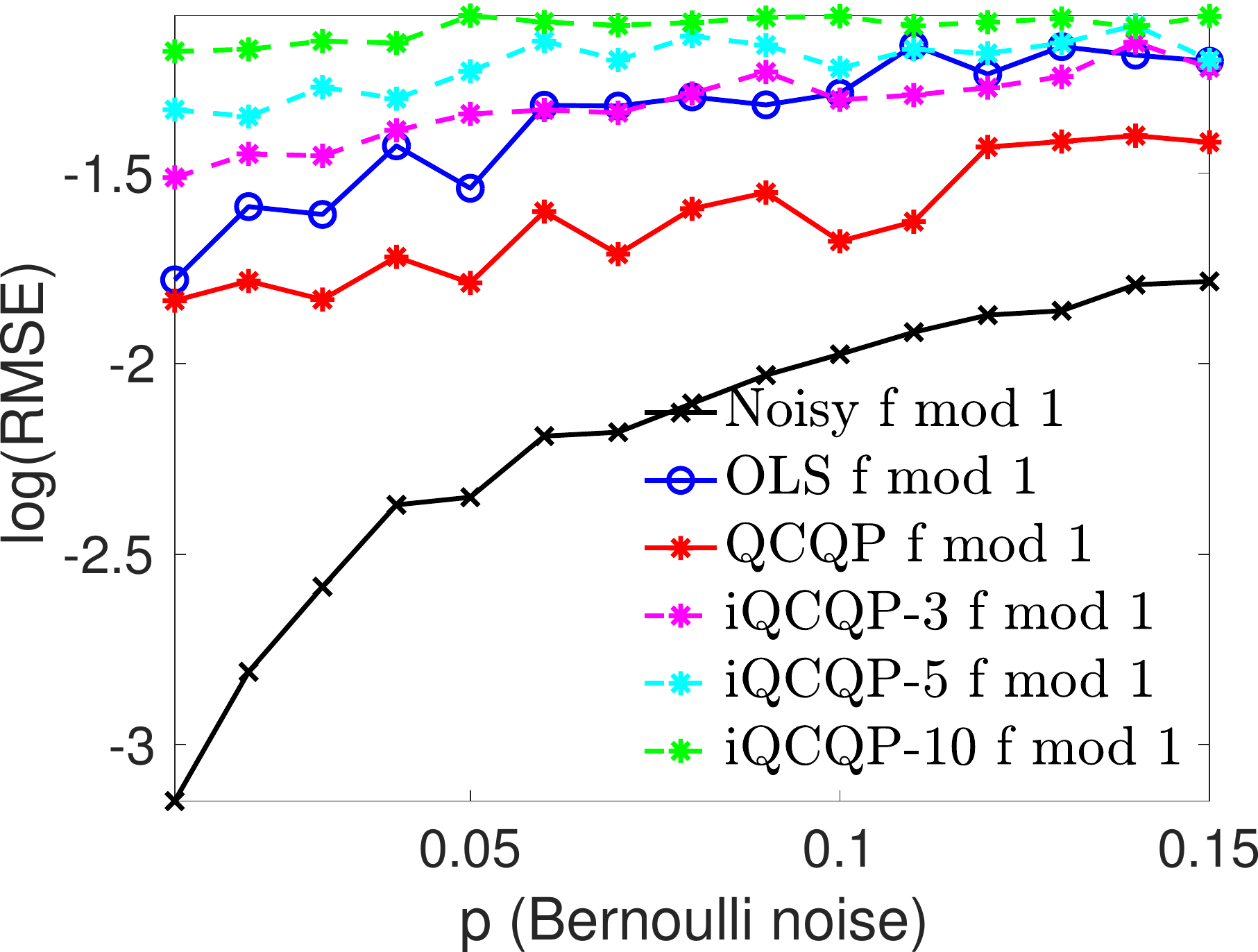} }
\hspace{0.1\textwidth} 
\subcaptionbox[]{ $k=5$, $\lambda= 0.03$}[ 0.19\textwidth ]
{\includegraphics[width=0.19\textwidth] {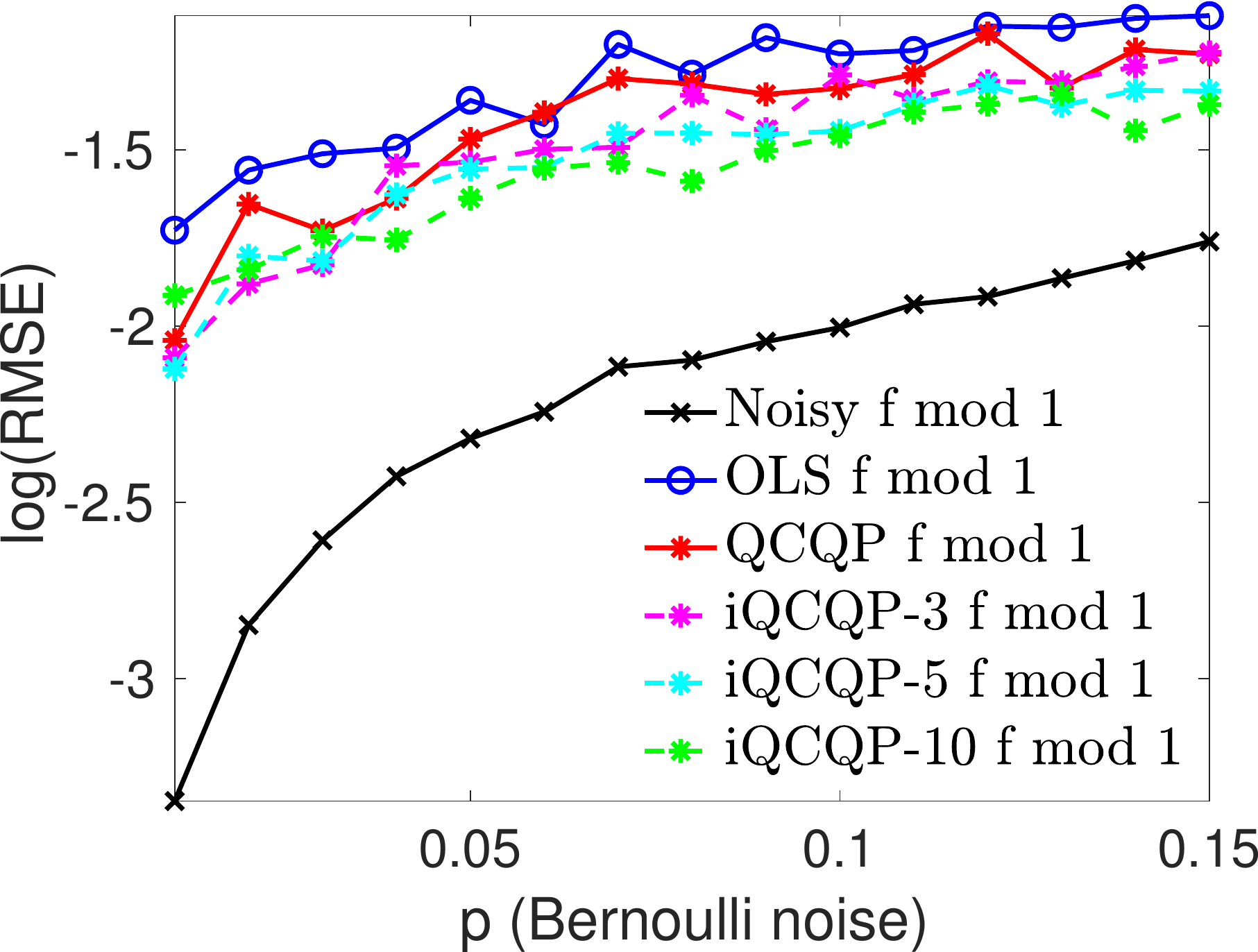} }
\subcaptionbox[]{ $k=5$, $\lambda= 0.1$}[ 0.19\textwidth ]
{\includegraphics[width=0.19\textwidth] {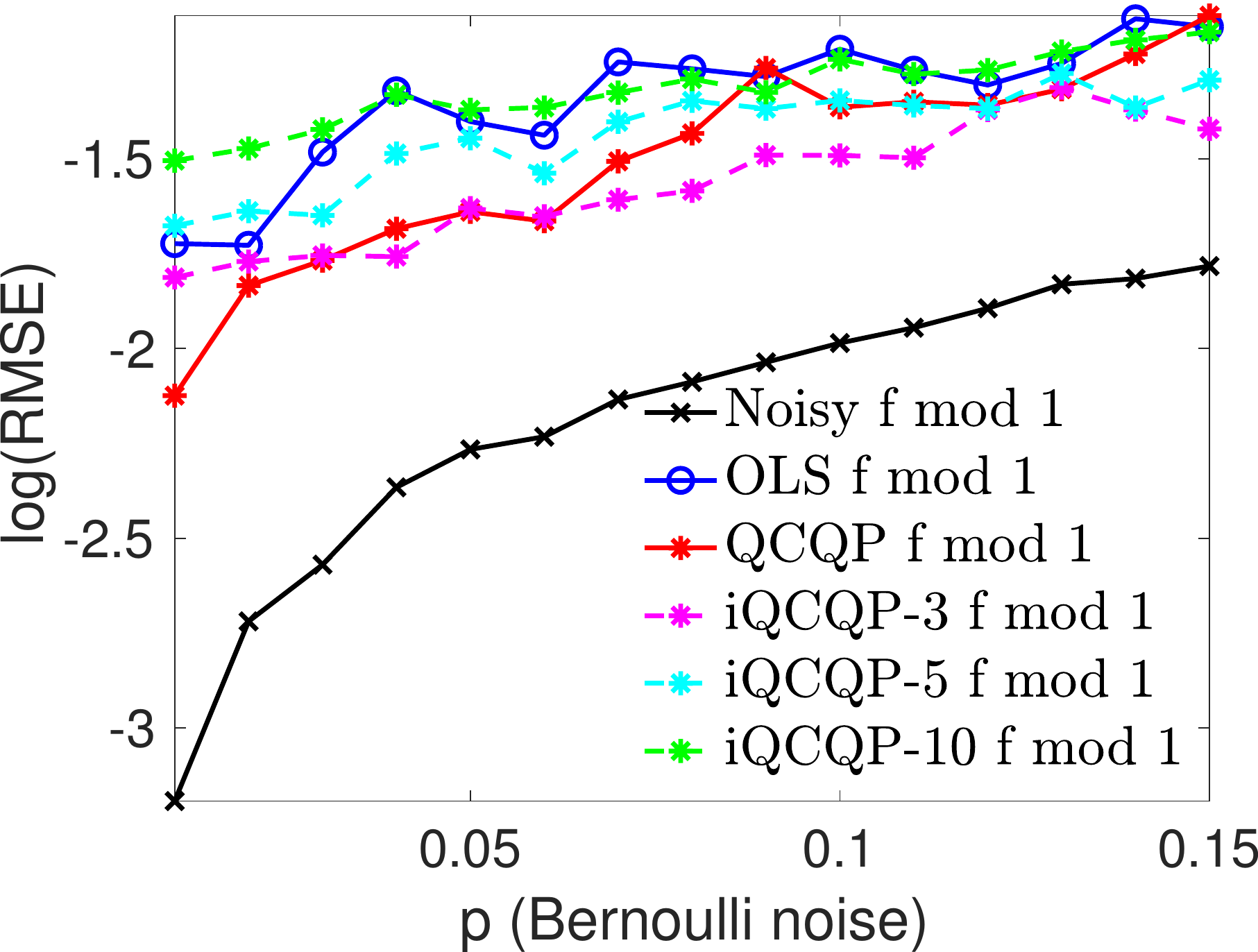} }
%
\subcaptionbox[]{ $k=5$, $\lambda= 0.3$}[ 0.19\textwidth ]
{\includegraphics[width=0.19\textwidth] {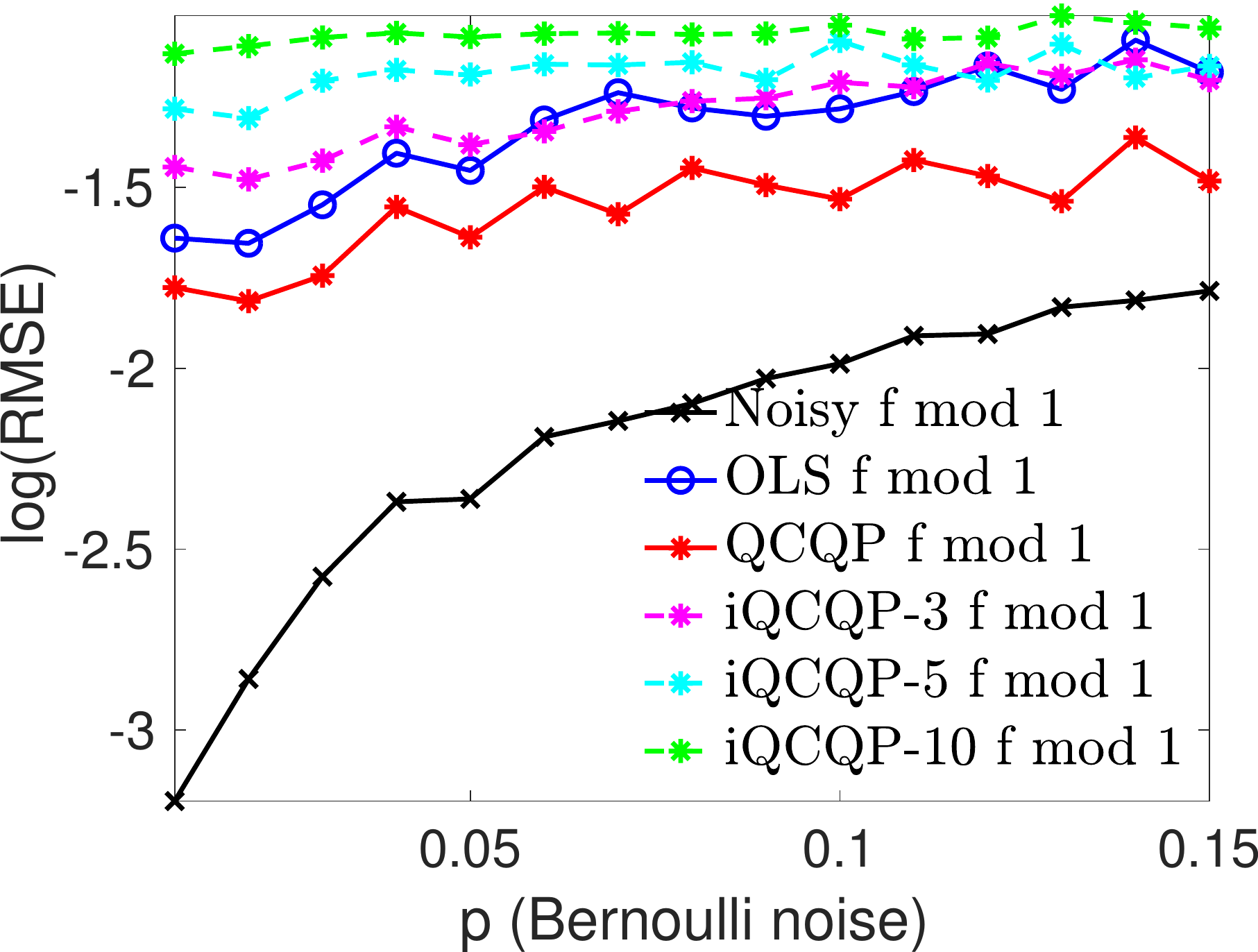} }
%
\subcaptionbox[]{ $k=5$, $\lambda= 0.5$}[ 0.19\textwidth ]
{\includegraphics[width=0.19\textwidth] {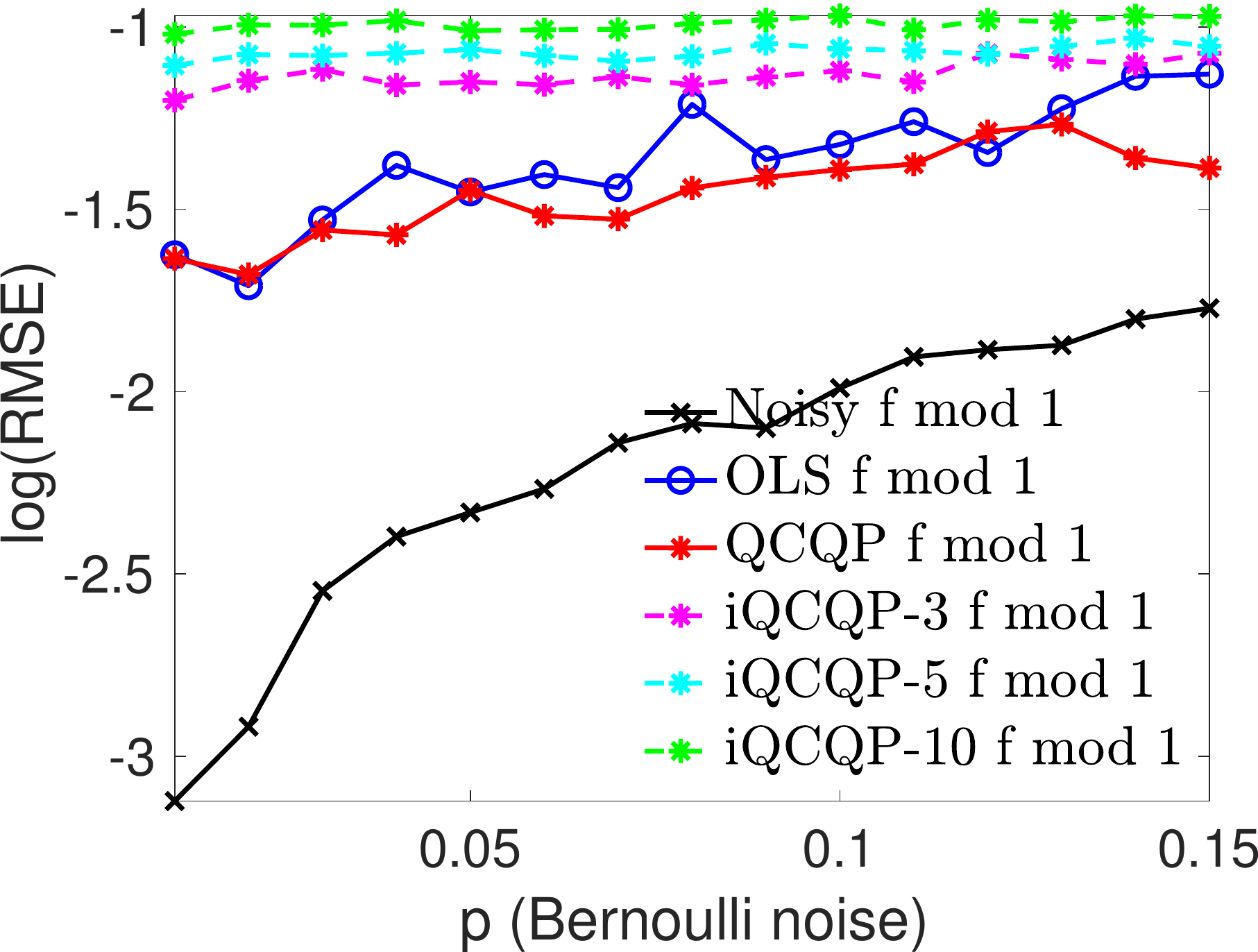} }
%
\subcaptionbox[]{ $k=5$, $\lambda= 1$}[ 0.19\textwidth ]
{\includegraphics[width=0.19\textwidth] {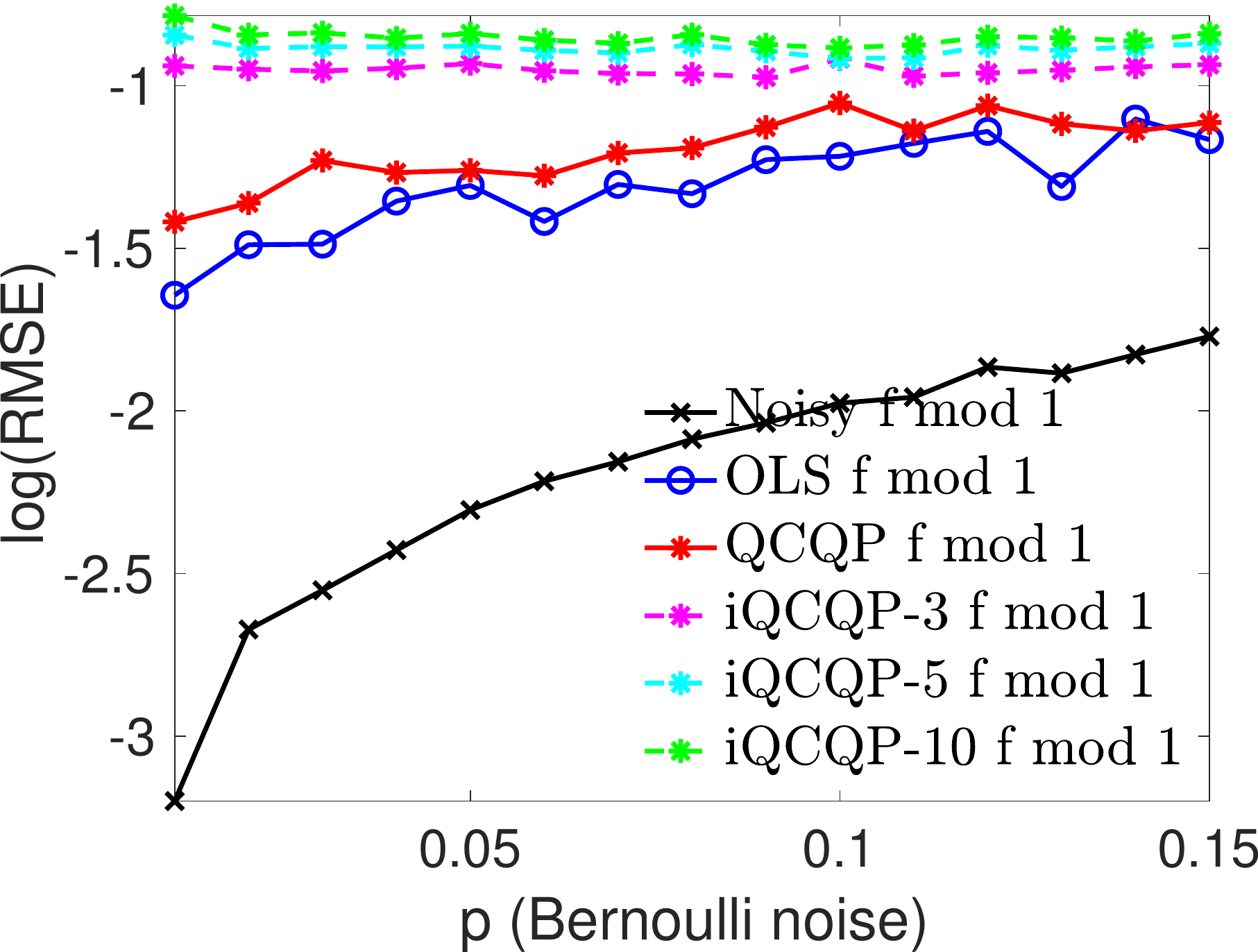} }
\hspace{0.1\textwidth} 
%
%
\vspace{-3mm}
\captionsetup{width=0.98\linewidth}
\caption[Short Caption]{Recovery errors for the denoised $f$ \hspace{-1mm} mod 1 samples, for $n=500$ under the Bernoulli noise model (20 trials).    \textbf{QCQP} denotes Algorithm \ref{algo:two_stage_denoise} without  the unwrapping stage performed by \textbf{OLS} \eqref{eq:ols_unwrap_lin_system}.
}
\label{fig:Sims_f1_Bernoulli_fmod1}
\end{figure}
 
\vspace{-3mm}

\begin{figure}[!ht]
\centering
\subcaptionbox[]{ $k=2$, $\lambda= 0.03$}[ 0.19\textwidth ]
{\includegraphics[width=0.19\textwidth] {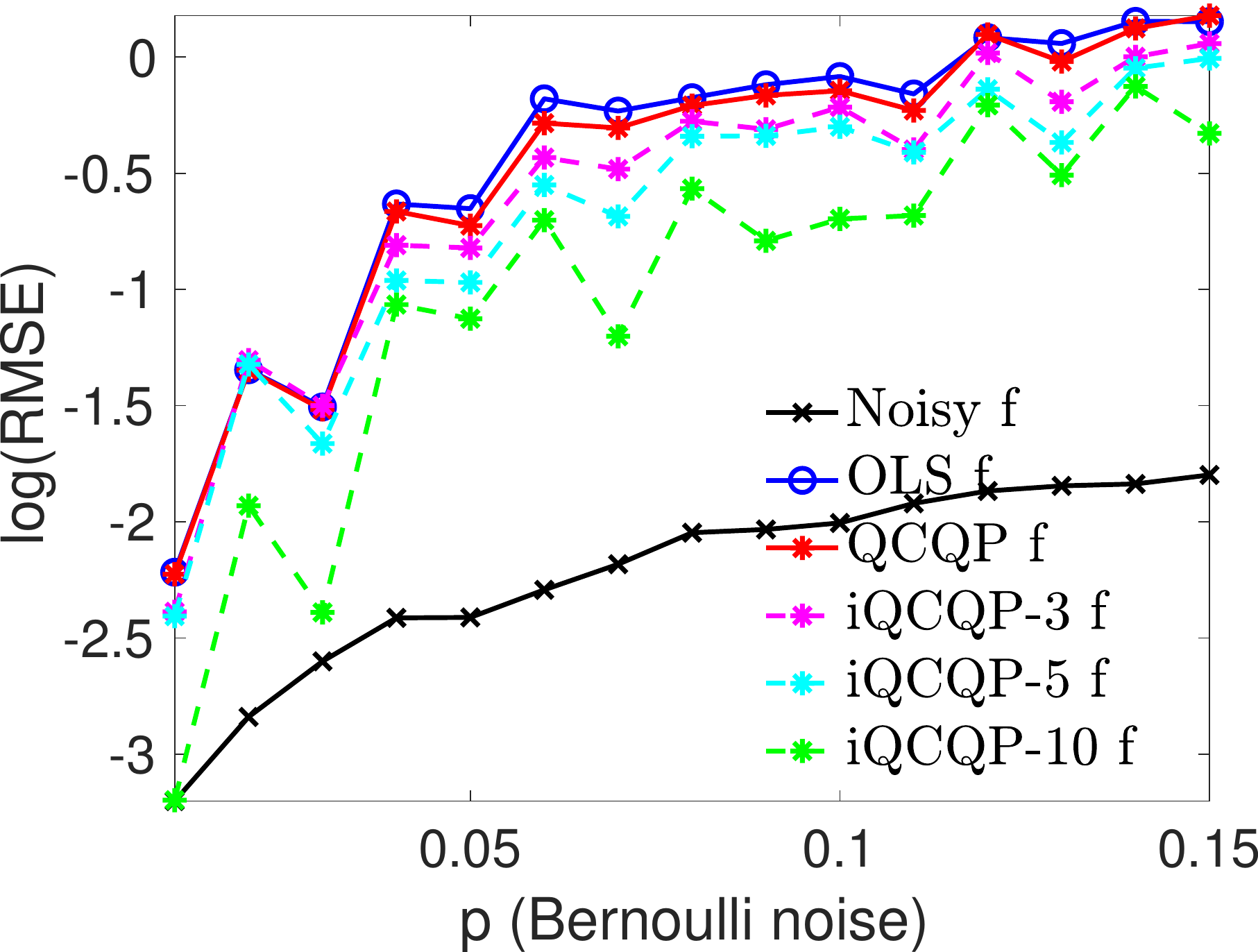} }
\subcaptionbox[]{ $k=2$, $\lambda= 0.1$}[ 0.19\textwidth ]
{\includegraphics[width=0.19\textwidth] {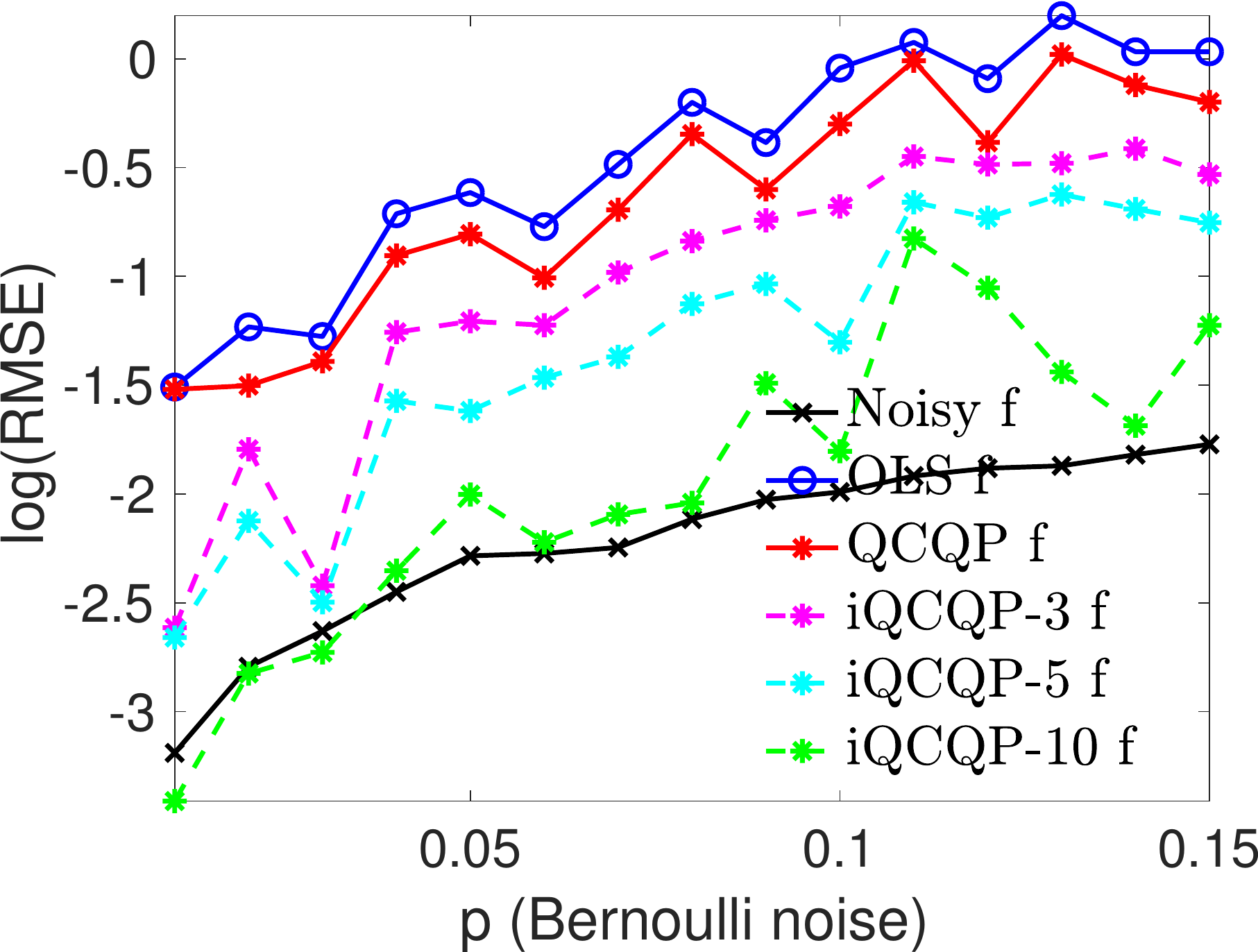} }
%
\subcaptionbox[]{ $k=2$, $\lambda= 0.3$}[ 0.19\textwidth ]
{\includegraphics[width=0.19\textwidth] {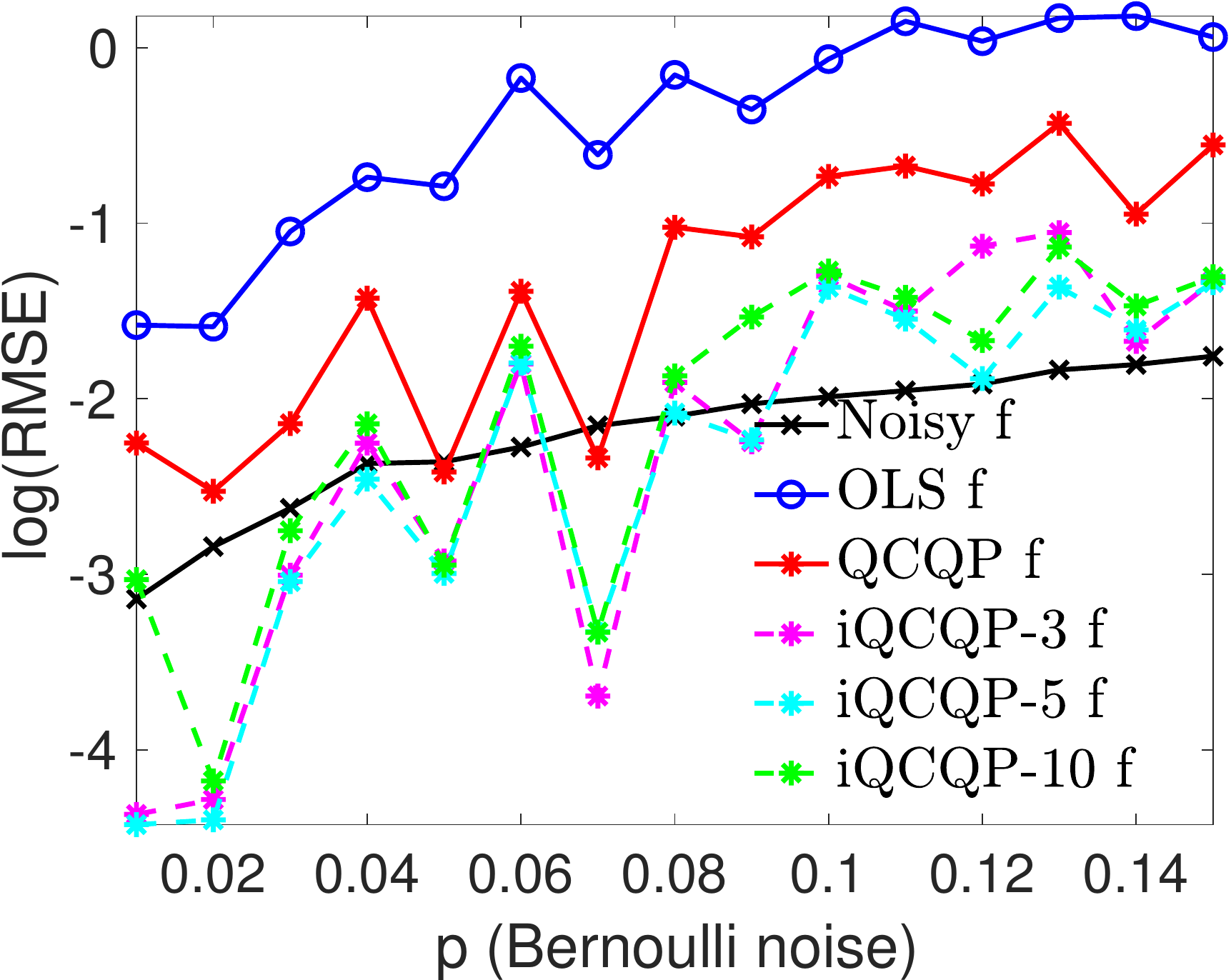} }
%
\subcaptionbox[]{ $k=2$, $\lambda= 0.5$}[ 0.19\textwidth ]
{\includegraphics[width=0.19\textwidth] {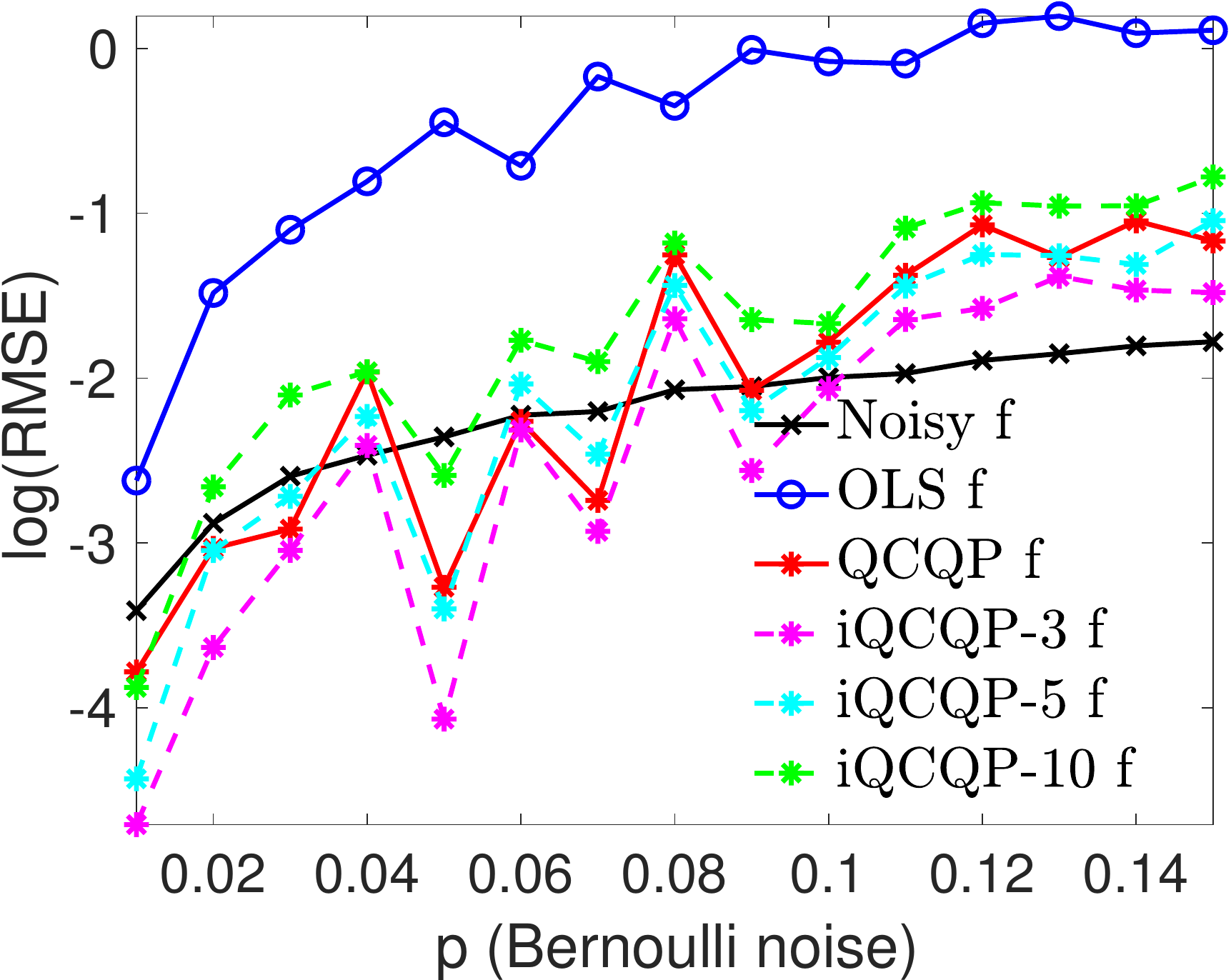} }
%
\subcaptionbox[]{ $k=2$, $\lambda= 1$}[ 0.19\textwidth ]
{\includegraphics[width=0.19\textwidth] {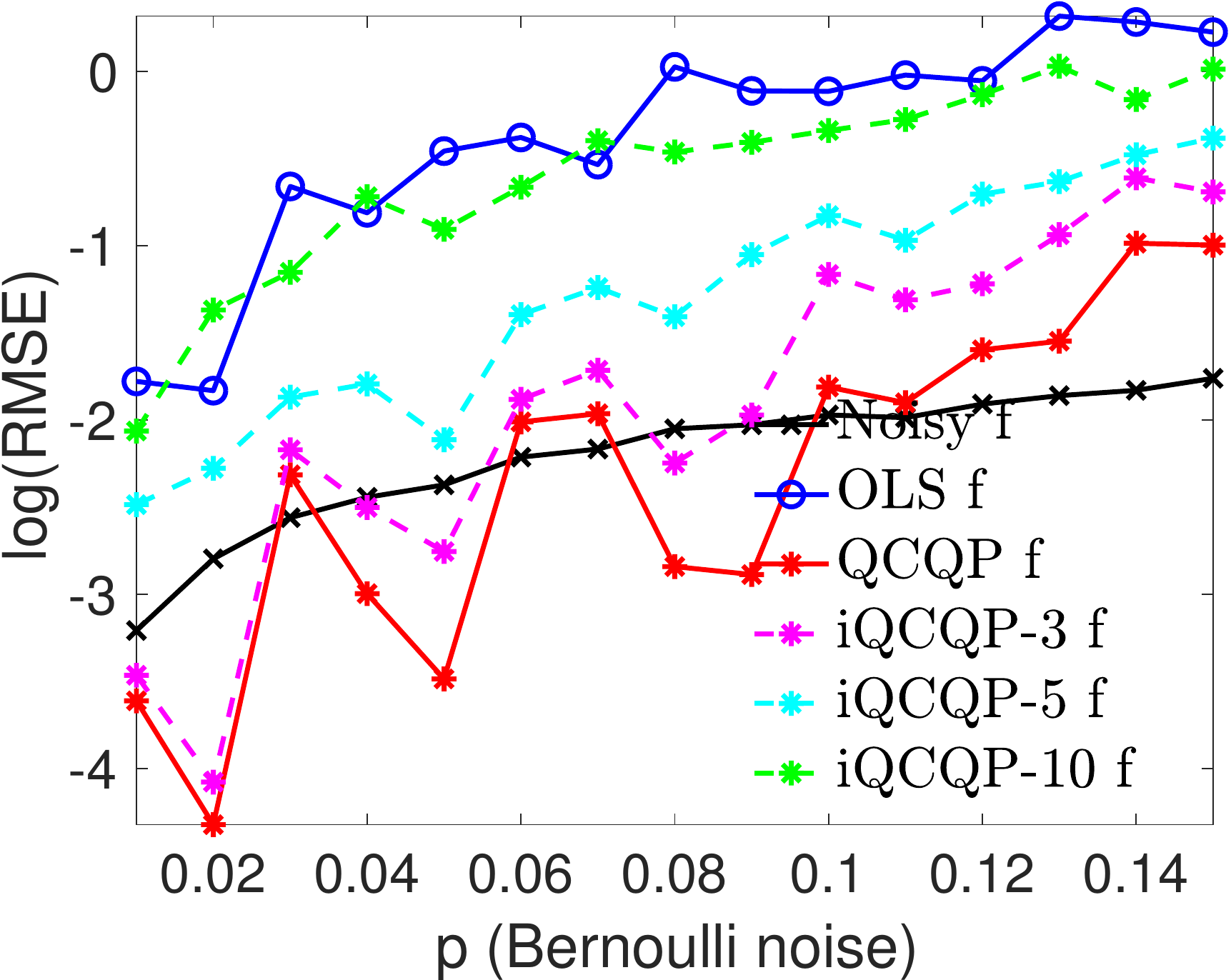} }
%
%
%
\subcaptionbox[]{ $k=3$, $\lambda= 0.03$}[ 0.19\textwidth ]
{\includegraphics[width=0.19\textwidth] {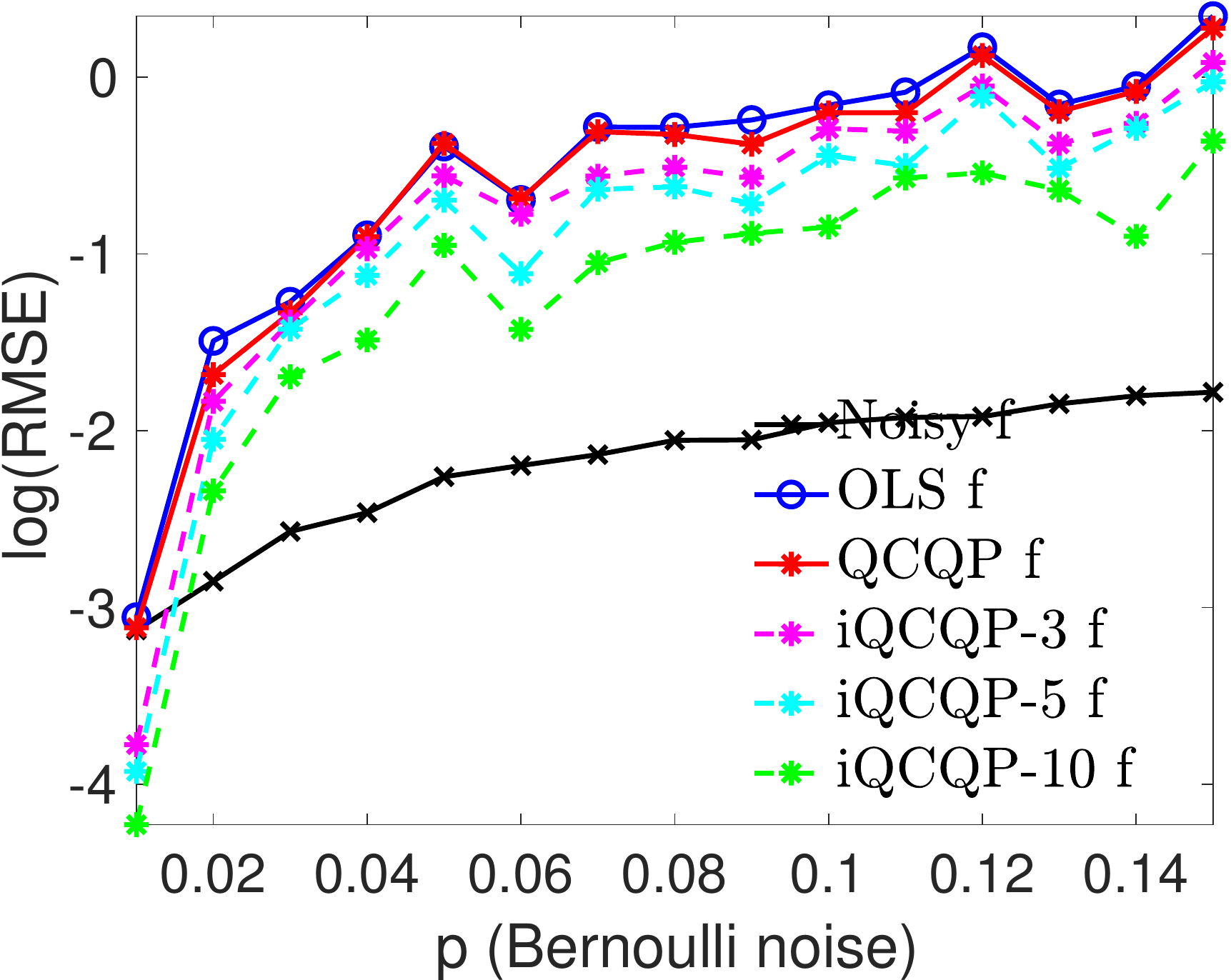} }
\subcaptionbox[]{ $k=3$, $\lambda= 0.1$}[ 0.19\textwidth ]
{\includegraphics[width=0.19\textwidth] {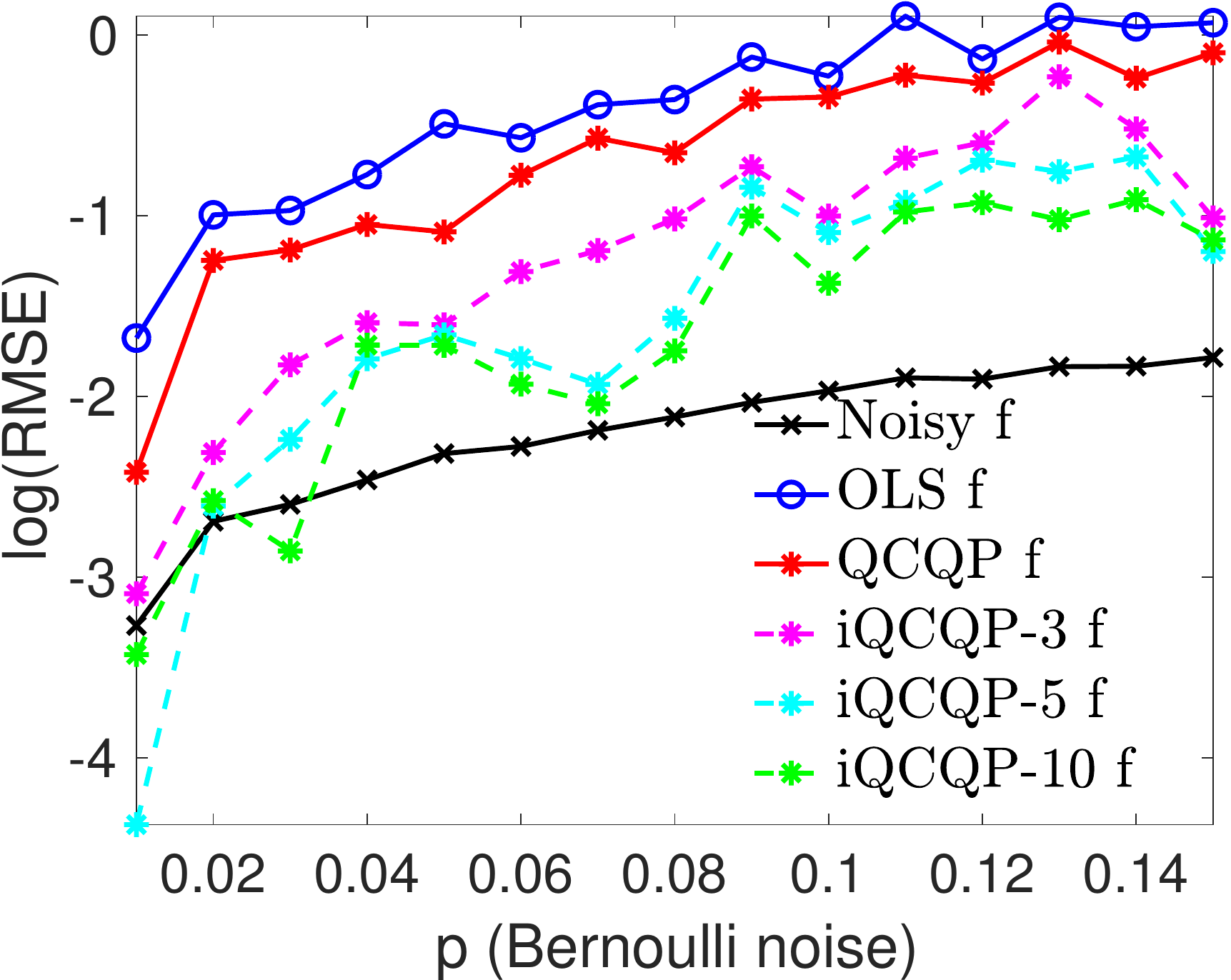} }
%
\subcaptionbox[]{ $k=3$, $\lambda= 0.3$}[ 0.19\textwidth ]
{\includegraphics[width=0.19\textwidth] {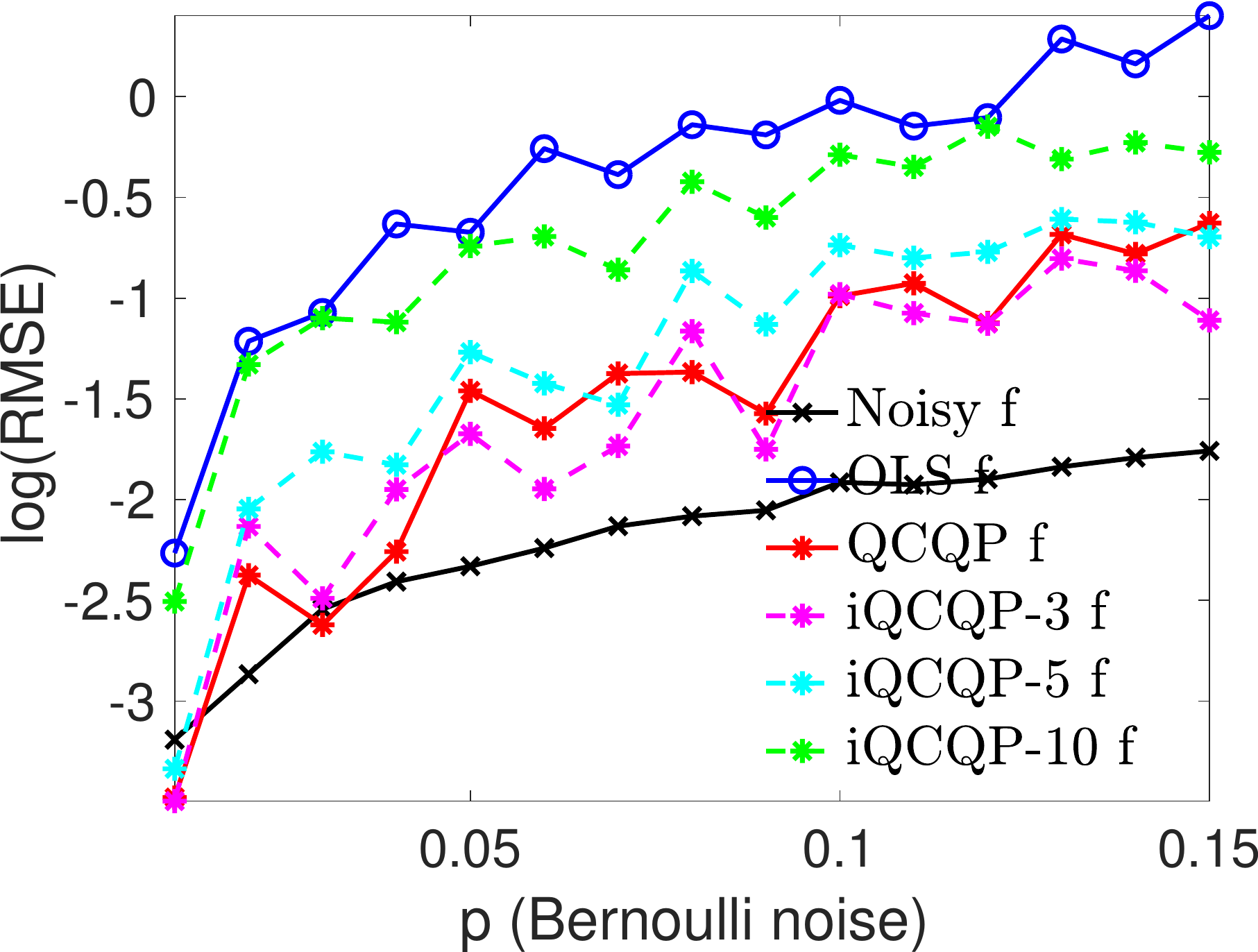} }
%
\subcaptionbox[]{ $k=3$, $\lambda= 0.5$}[ 0.19\textwidth ]
{\includegraphics[width=0.19\textwidth] {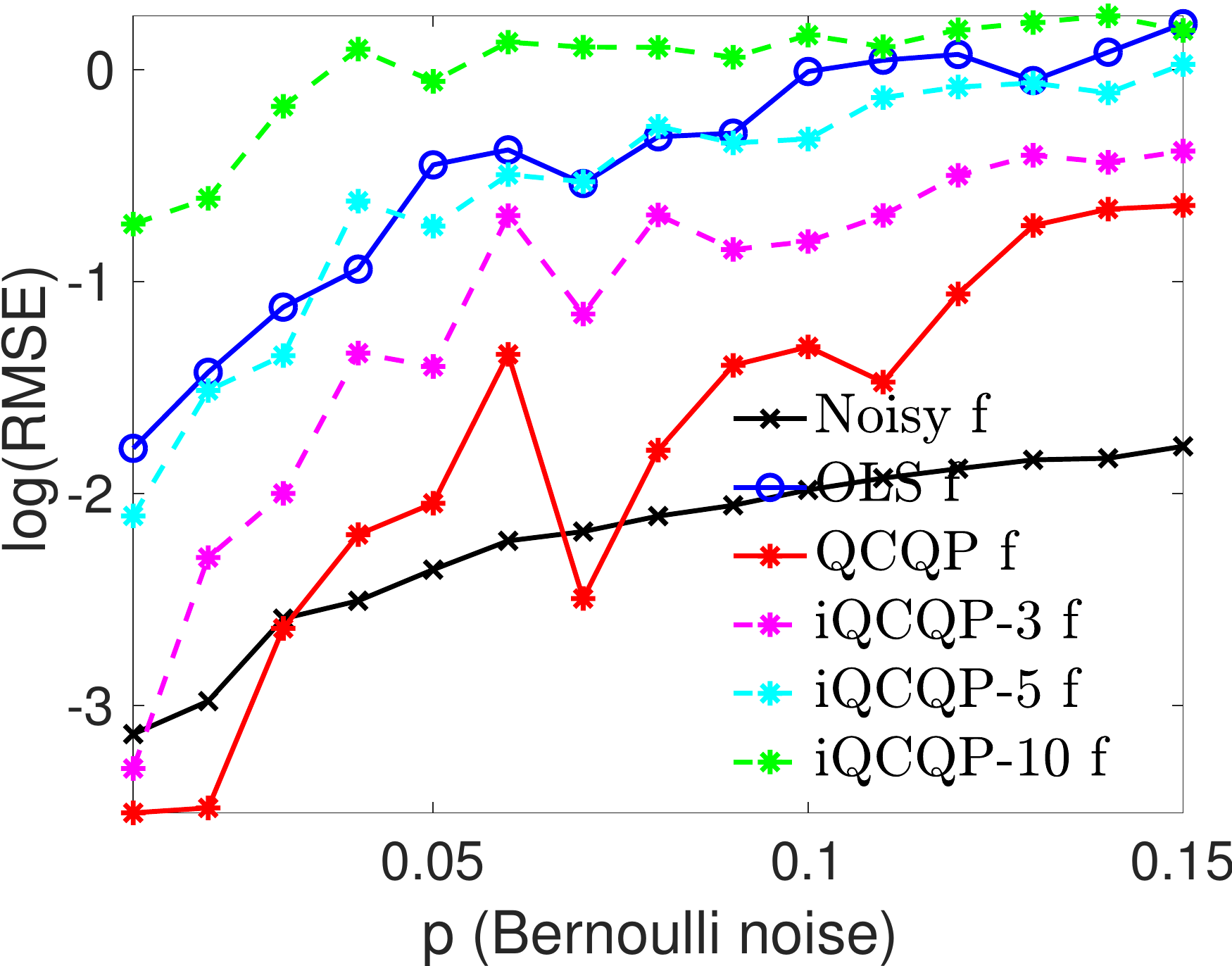} }
%
\subcaptionbox[]{ $k=3$, $\lambda= 1$}[ 0.19\textwidth ]
{\includegraphics[width=0.19\textwidth] {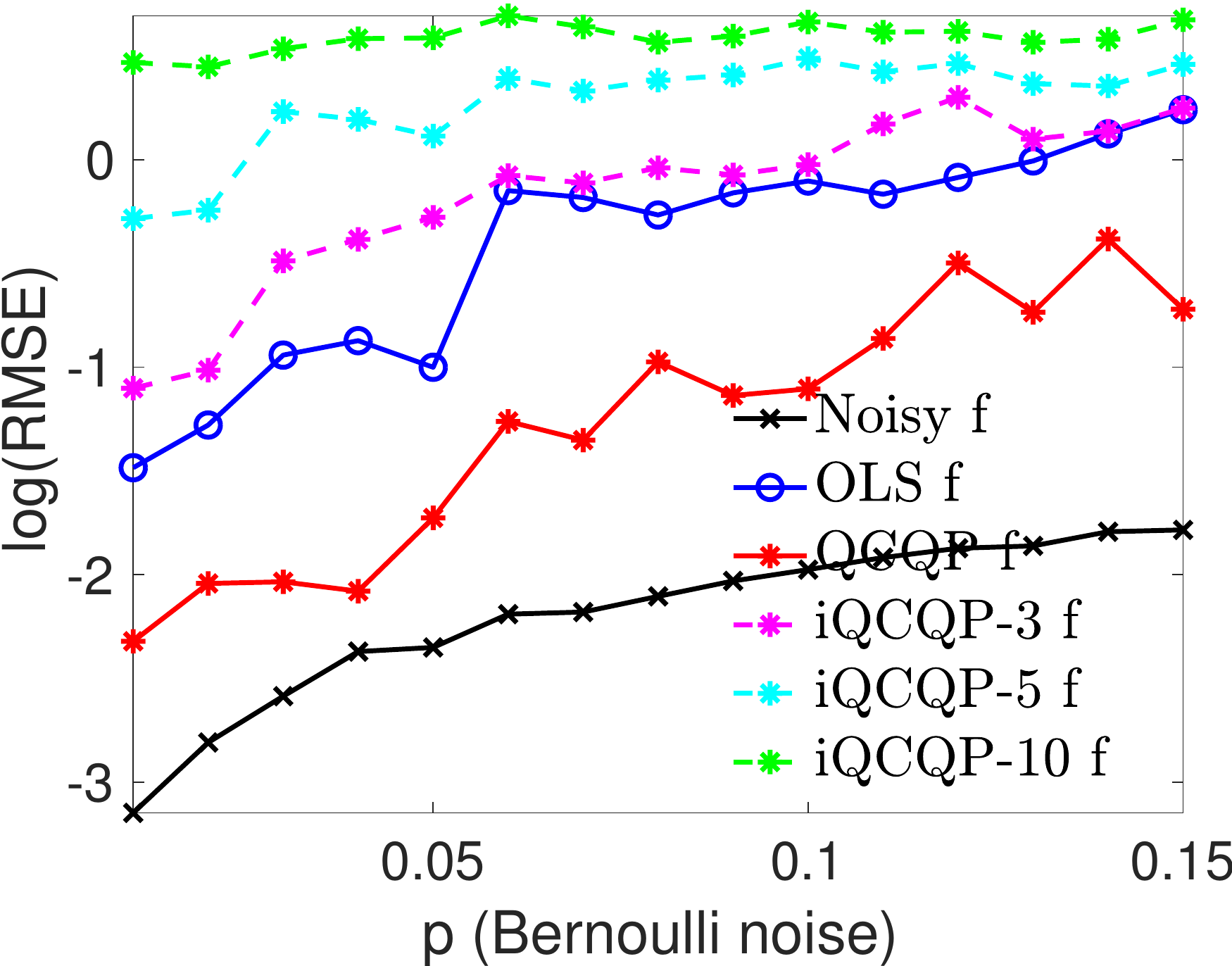} }
%
%
%
\subcaptionbox[]{ $k=5$, $\lambda= 0.03$}[ 0.19\textwidth ]
{\includegraphics[width=0.19\textwidth] {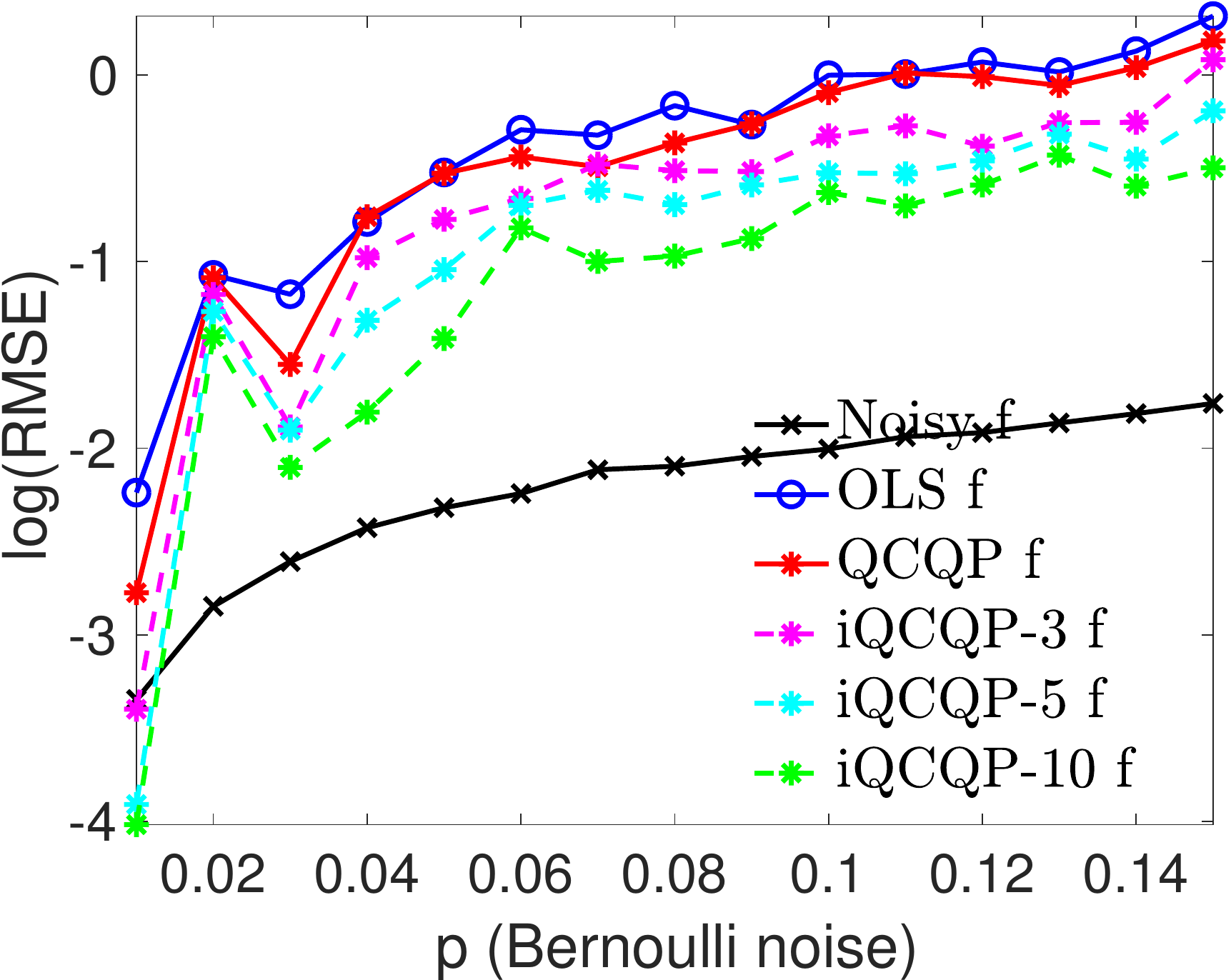} }
\subcaptionbox[]{ $k=5$, $\lambda= 0.1$}[ 0.19\textwidth ]
{\includegraphics[width=0.19\textwidth] {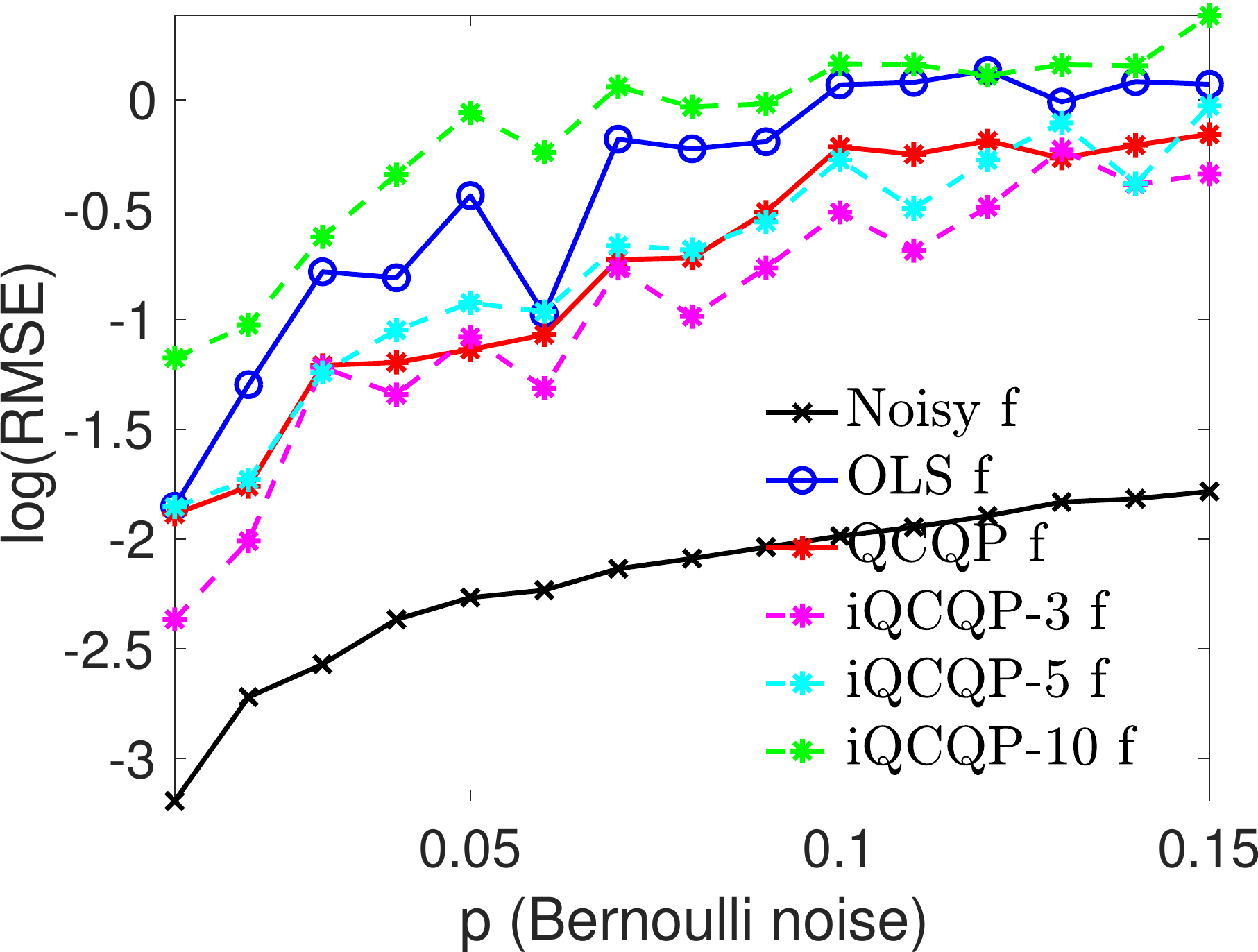} }
%
\subcaptionbox[]{ $k=5$, $\lambda= 0.3$}[ 0.19\textwidth ]
{\includegraphics[width=0.19\textwidth] {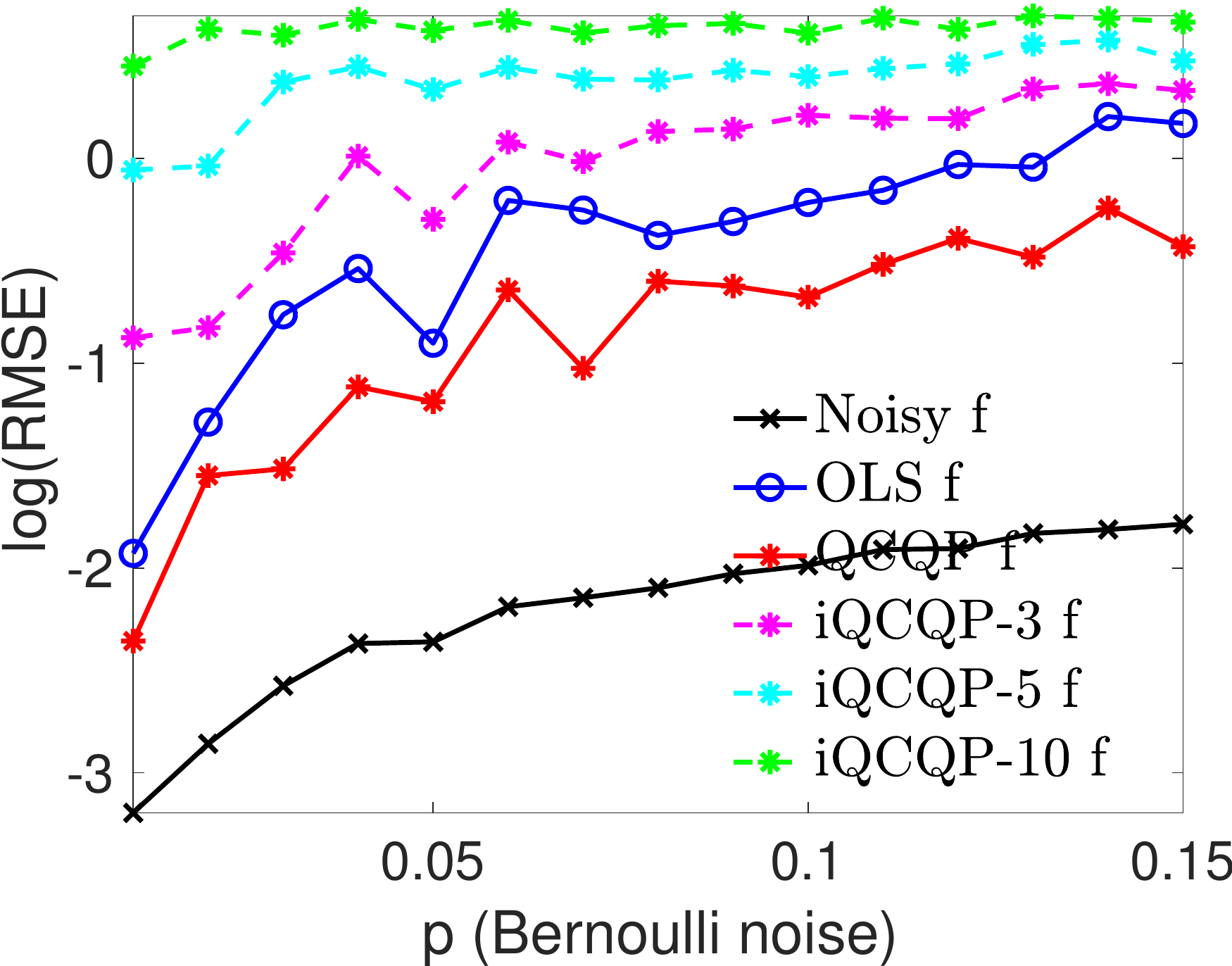} }
%
\subcaptionbox[]{ $k=5$, $\lambda= 0.5$}[ 0.19\textwidth ]
{\includegraphics[width=0.19\textwidth] {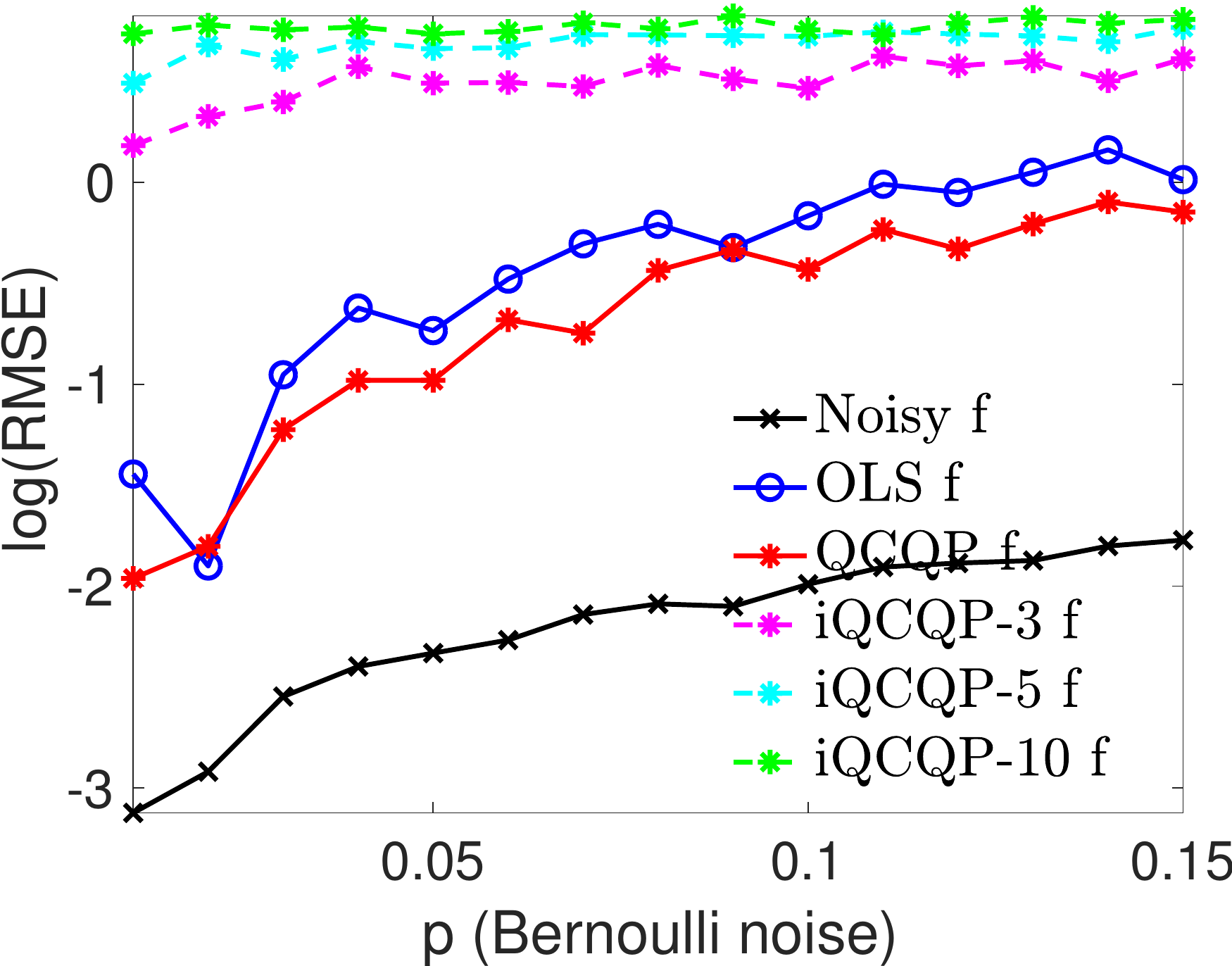} }
%
\subcaptionbox[]{ $k=5$, $\lambda= 1$}[ 0.19\textwidth ]
{\includegraphics[width=0.19\textwidth] {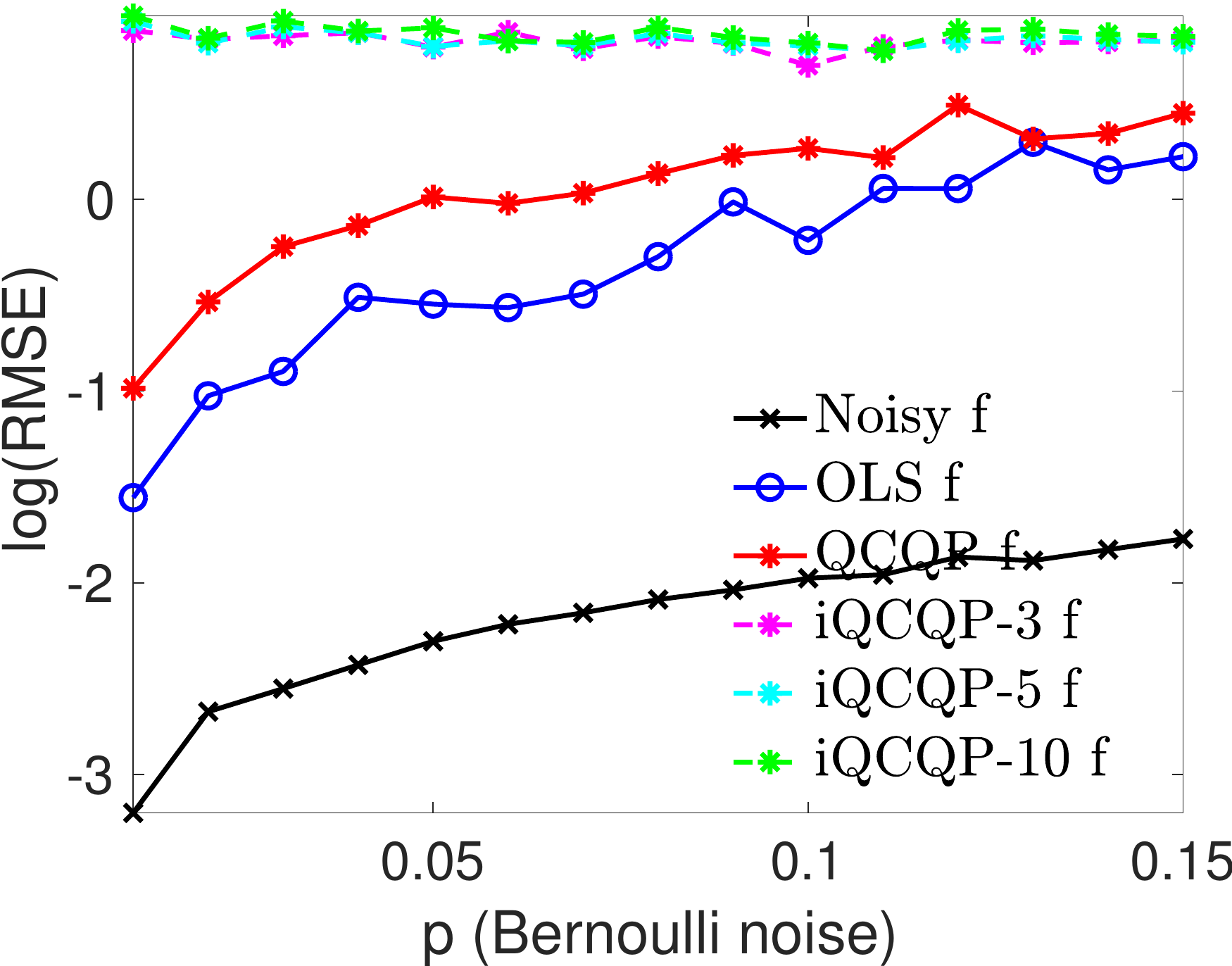} }
%
\captionsetup{width=0.98\linewidth}
\caption[Short Caption]{Recovery errors for the final estimated $f$  samples, for  $n=500$ under  the Bernoulli noise model (20 trials). Here, \textbf{QCQP} denotes Algorithm \ref{algo:two_stage_denoise}, for which  the unwrapping stage is  performed via \textbf{OLS} \eqref{eq:ols_unwrap_lin_system}.   
}
\label{fig:Sims_f1_Bernoulli_f}
\end{figure}




\subsection{Comparison with \cite{bhandari17}}


This section is a comparison of  \textbf{OLS},  \textbf{QCQP}, and \textbf{iQCQP} with the approach of \cite{bhandari17}, whose algorithm we denote by \textbf{BKR} for brevity. We compare all four approaches across two different noise models, Bounded and Gaussian, on two different functions: the function  \eqref{def:f1} used 
in the experiments throughout this paper and a bandlimited  function used by \cite{bhandari17}. 
We defer to the appendix  all the experiments, except for those reported in 
 Figure \ref{fig:instances_f1_Bounded_Sampta}, where we consider the function \eqref{def:f1} under the Bounded noise model.  We  note that at a lower level of noise $\gamma=0.13$, all methods perform  similarly well, with relative performance in the following order: 
\textbf{iQCQP} (RMSE=0.25), 
\textbf{QCQP} (RMSE=0.29), 
\textbf{OLS} (RMSE=0.30),  
\textbf{BKR} (RMSE=0.30).
However, at higher levels of noise,  \textbf{BKR} returns  meaningless results, while \textbf{QCQP}, and especially \textbf{iQCQP}, return more accurate results.

\begin{figure}[!ht]
\centering
\subcaptionbox[]{  $\gamma=0.13$, \textbf{BKR}
}[ 0.24\textwidth ]
{\includegraphics[width=0.24\textwidth] {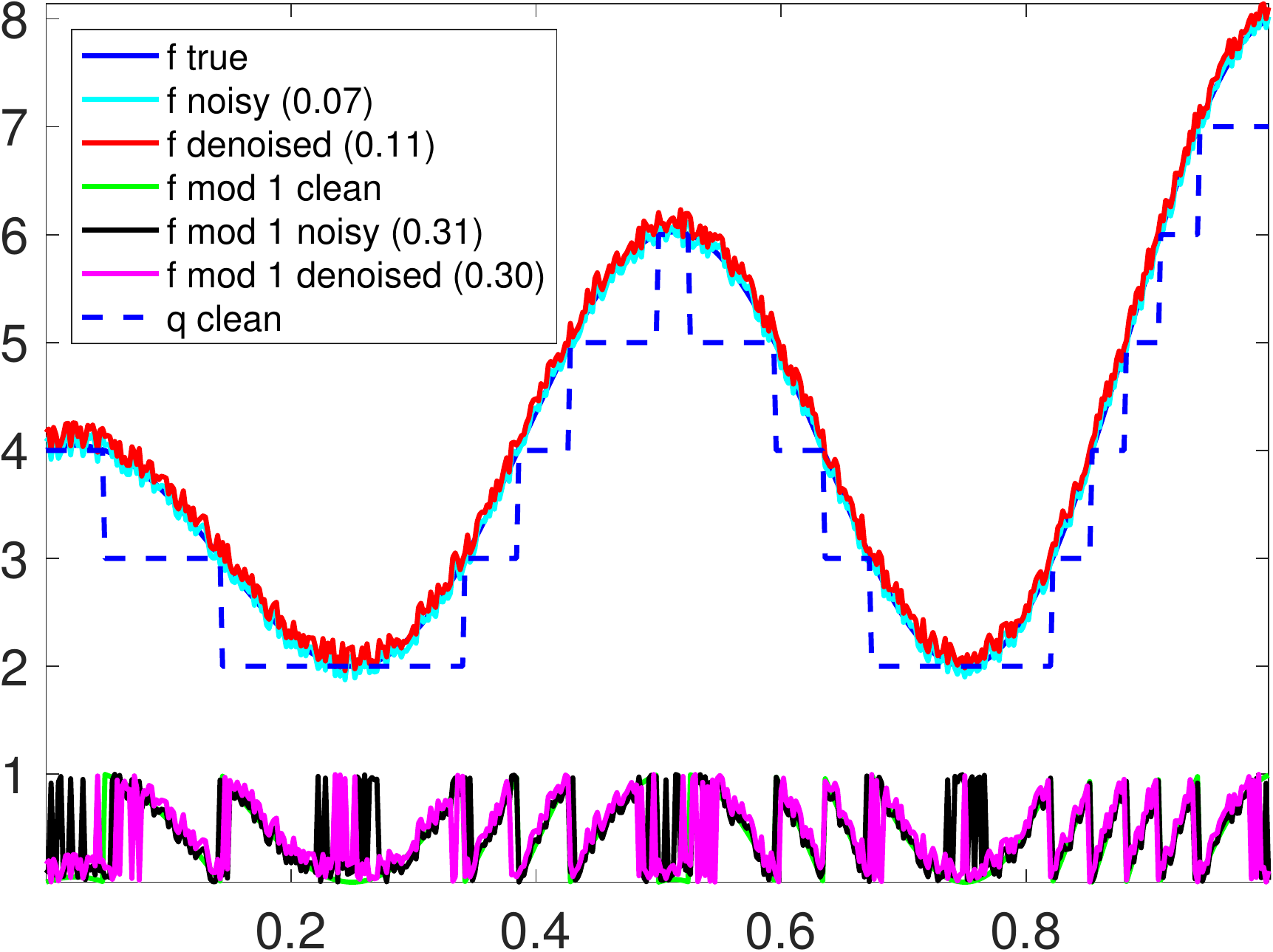} }
%
\subcaptionbox[]{  $\gamma=0.13$, \textbf{OLS}
}[ 0.24\textwidth ]
{\includegraphics[width=0.24\textwidth] {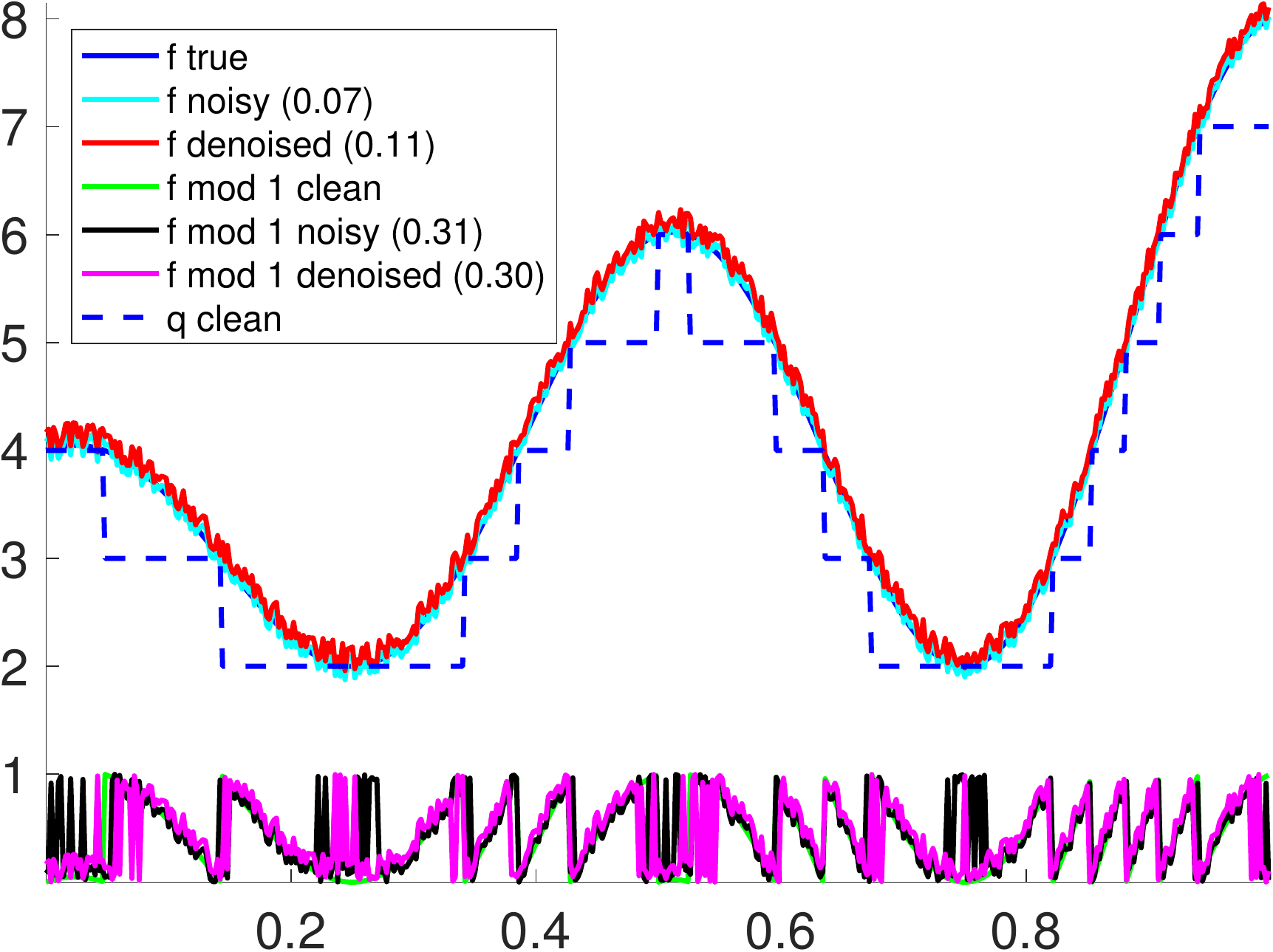} }
%
\subcaptionbox[]{  $\gamma=0.13$, \textbf{QCQP}
}[ 0.24\textwidth ]
{\includegraphics[width=0.24\textwidth] {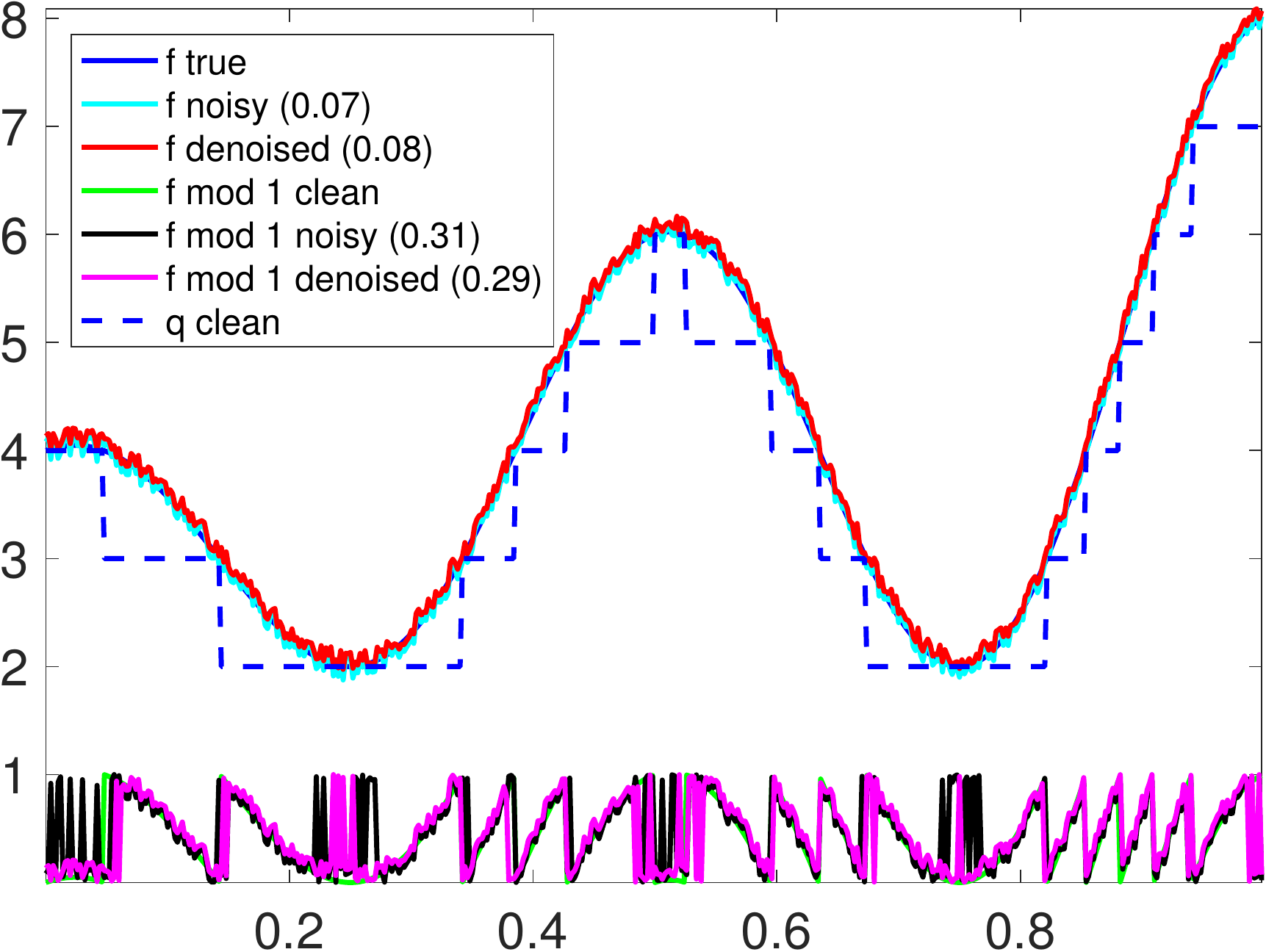} }
%
\subcaptionbox[]{  $\gamma=0.13$, \textbf{iQCQP}
}[ 0.24\textwidth ]
{\includegraphics[width=0.24\textwidth] {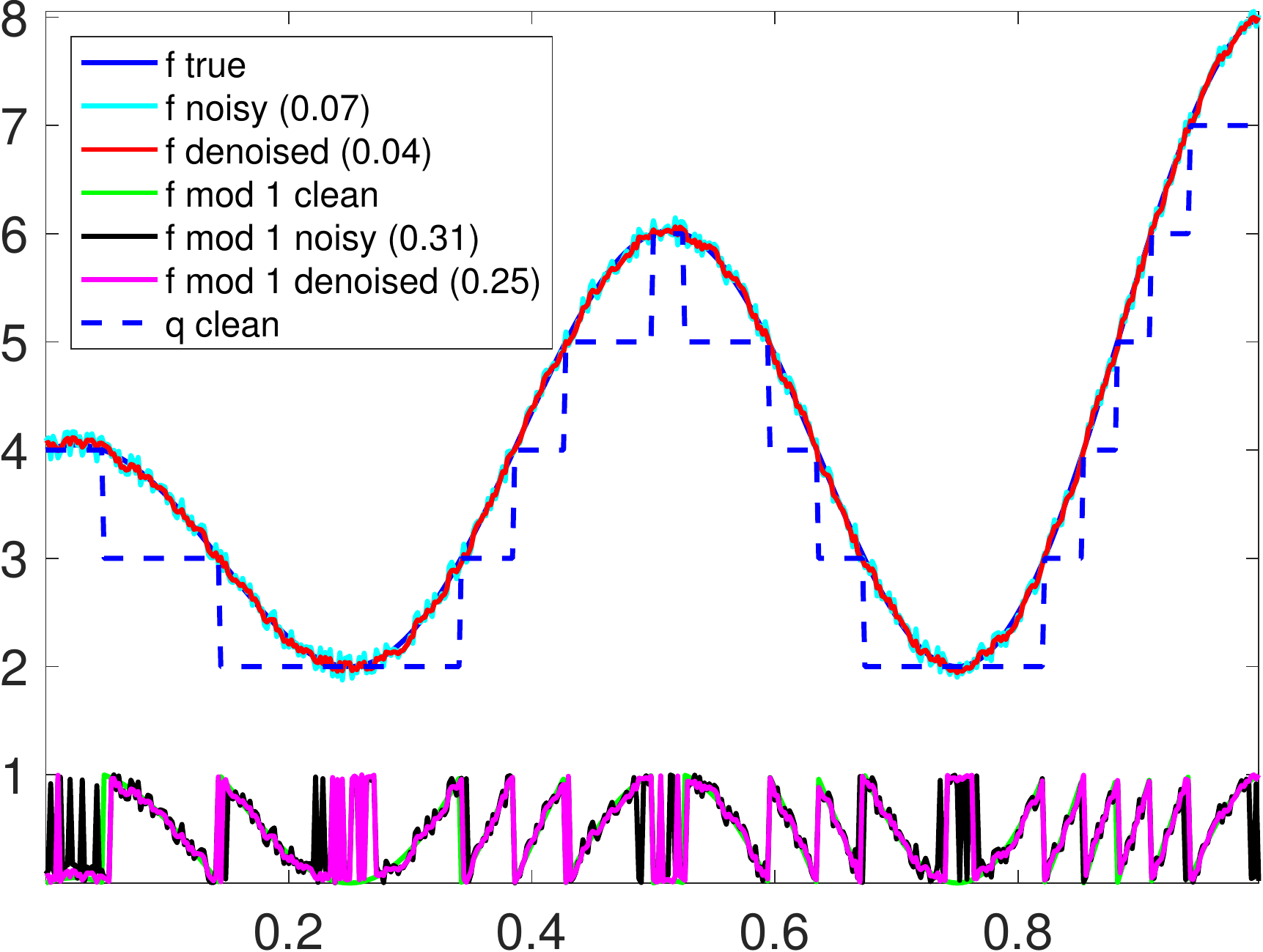} }
%
%
%
\subcaptionbox[]{  $\gamma=0.14$, \textbf{BKR}
}[ 0.24\textwidth ]
{\includegraphics[width=0.24\textwidth] {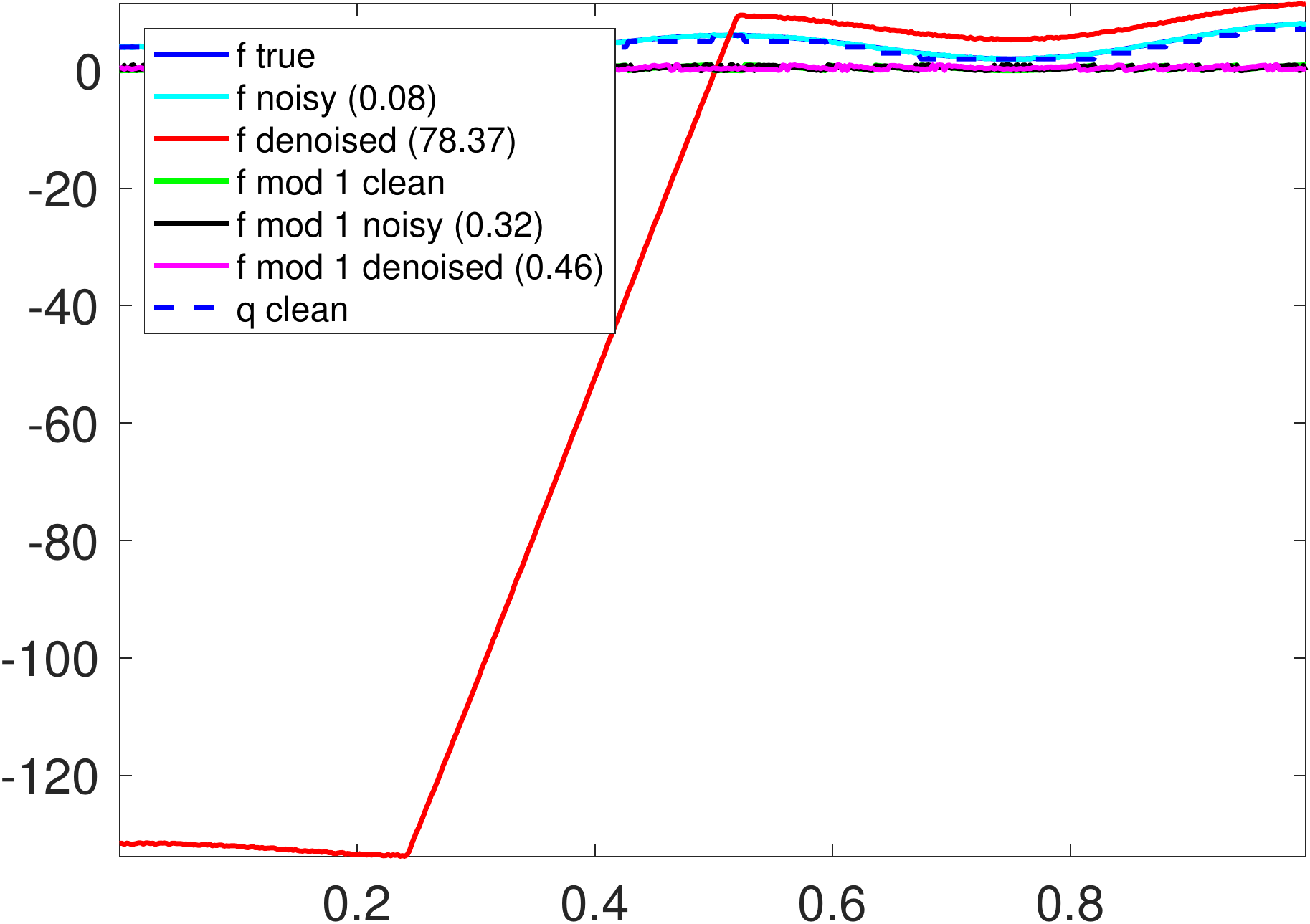} }
%
\subcaptionbox[]{  $\gamma=0.14$, \textbf{OLS}
}[ 0.24\textwidth ]
{\includegraphics[width=0.24\textwidth] {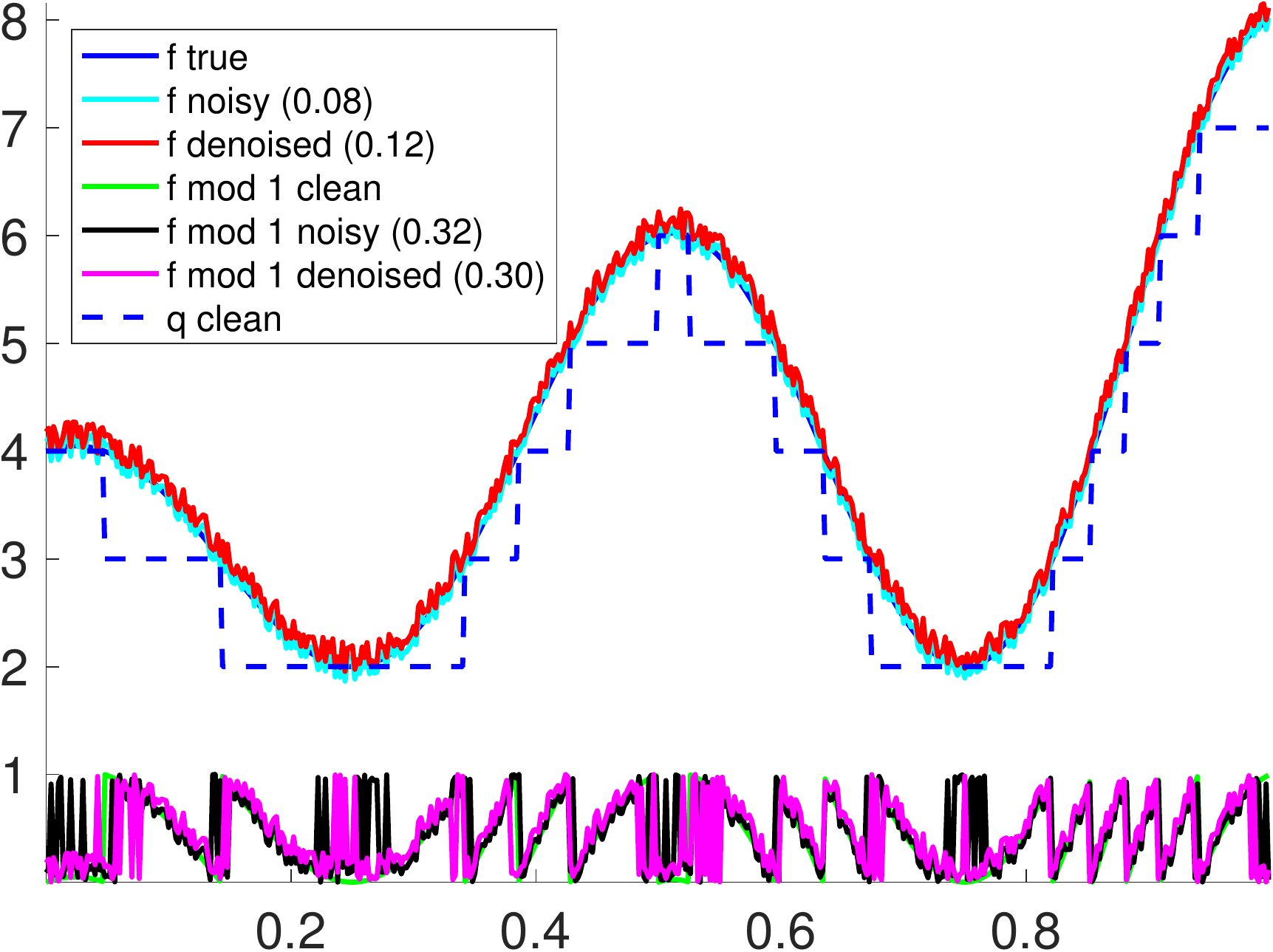} }
%
\subcaptionbox[]{  $\gamma=0.14$, \textbf{QCQP}
}[ 0.24\textwidth ]
{\includegraphics[width=0.24\textwidth] {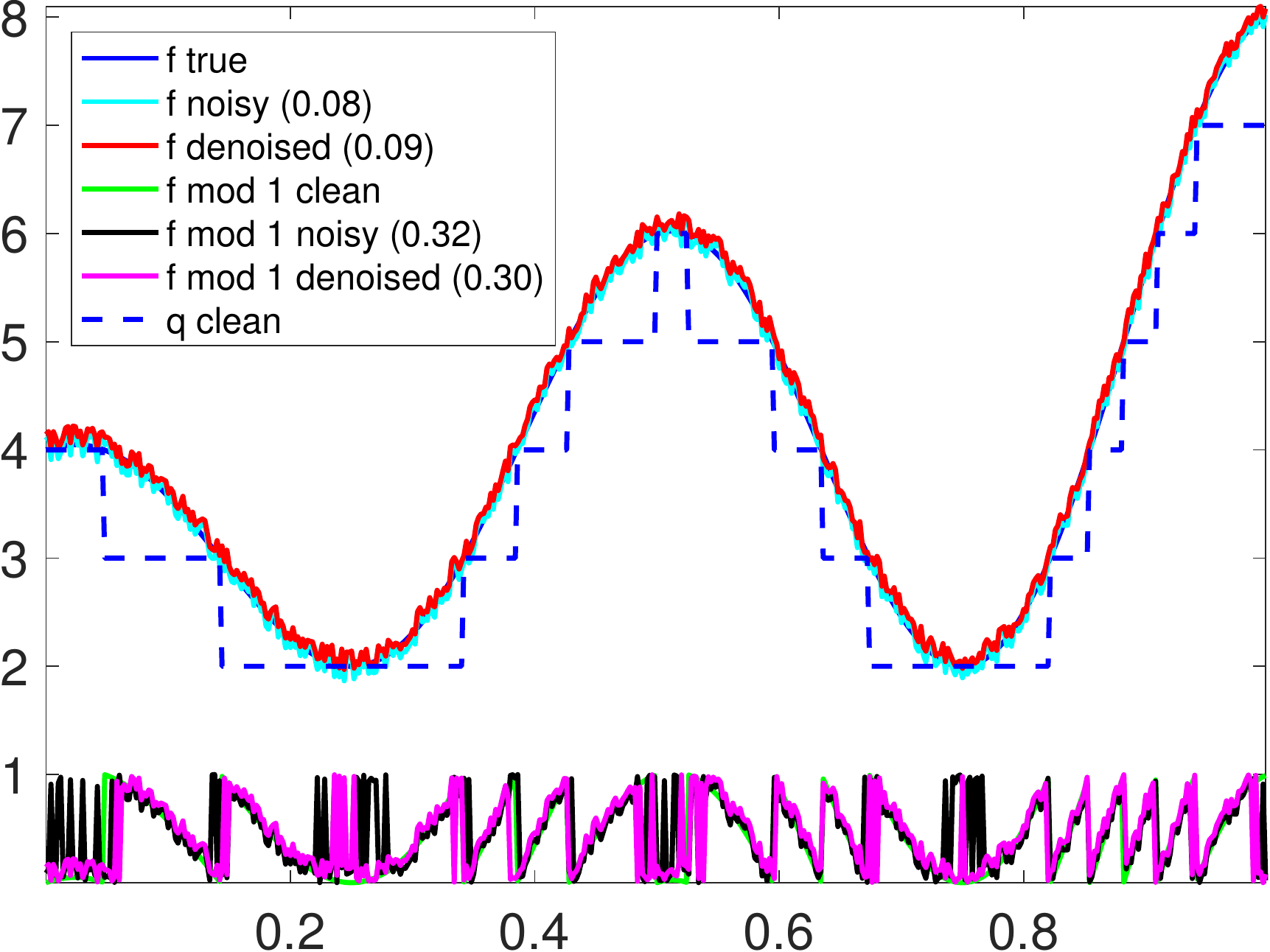} }
%
\subcaptionbox[]{  $\gamma=0.14$, \textbf{iQCQP}
}[ 0.24\textwidth ]
{\includegraphics[width=0.24\textwidth] {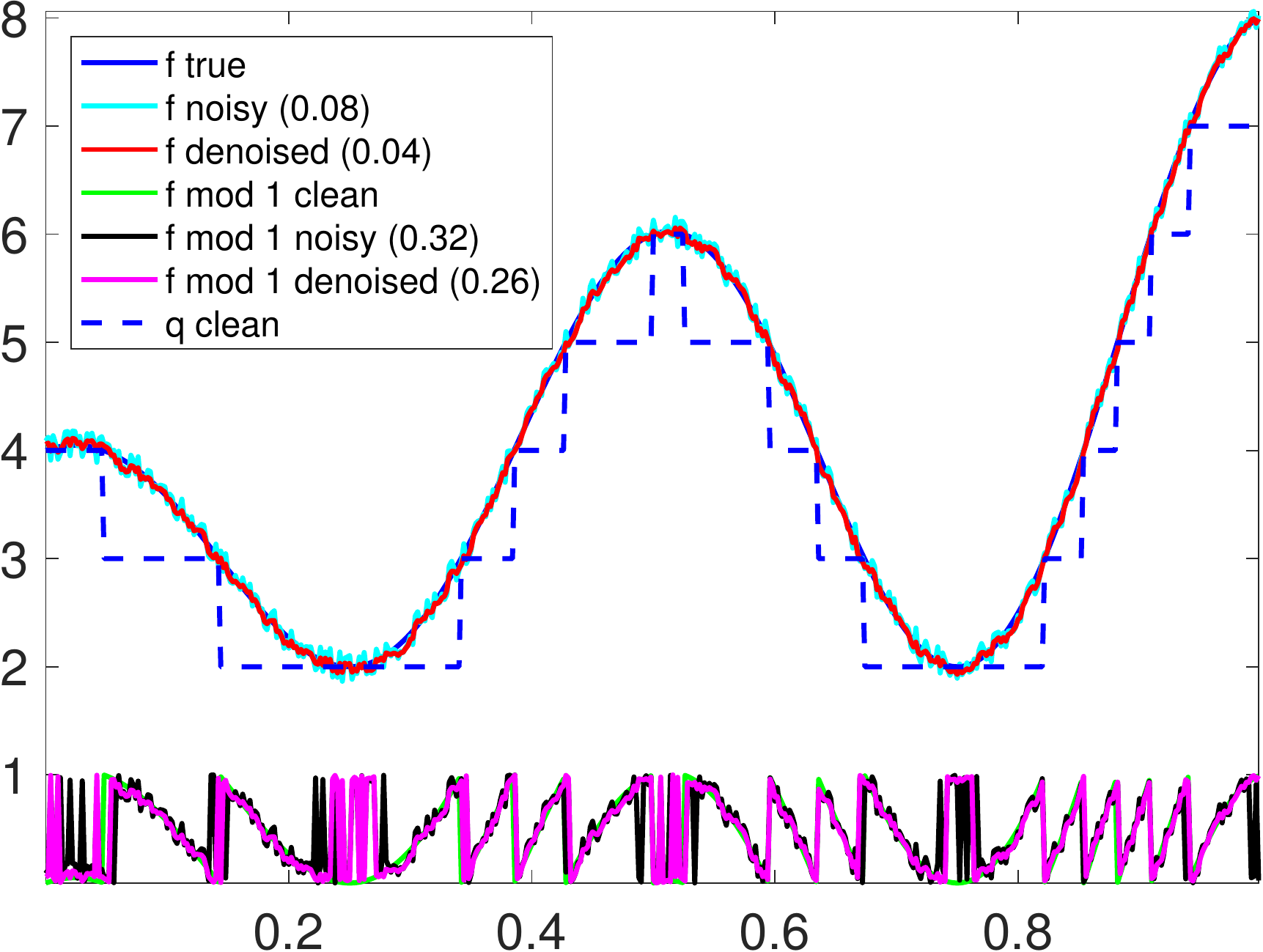} }
%
%
%
%
%
%
\subcaptionbox[]{  $\gamma=0.27$, \textbf{BKR}
}[ 0.24\textwidth ]
{\includegraphics[width=0.24\textwidth] {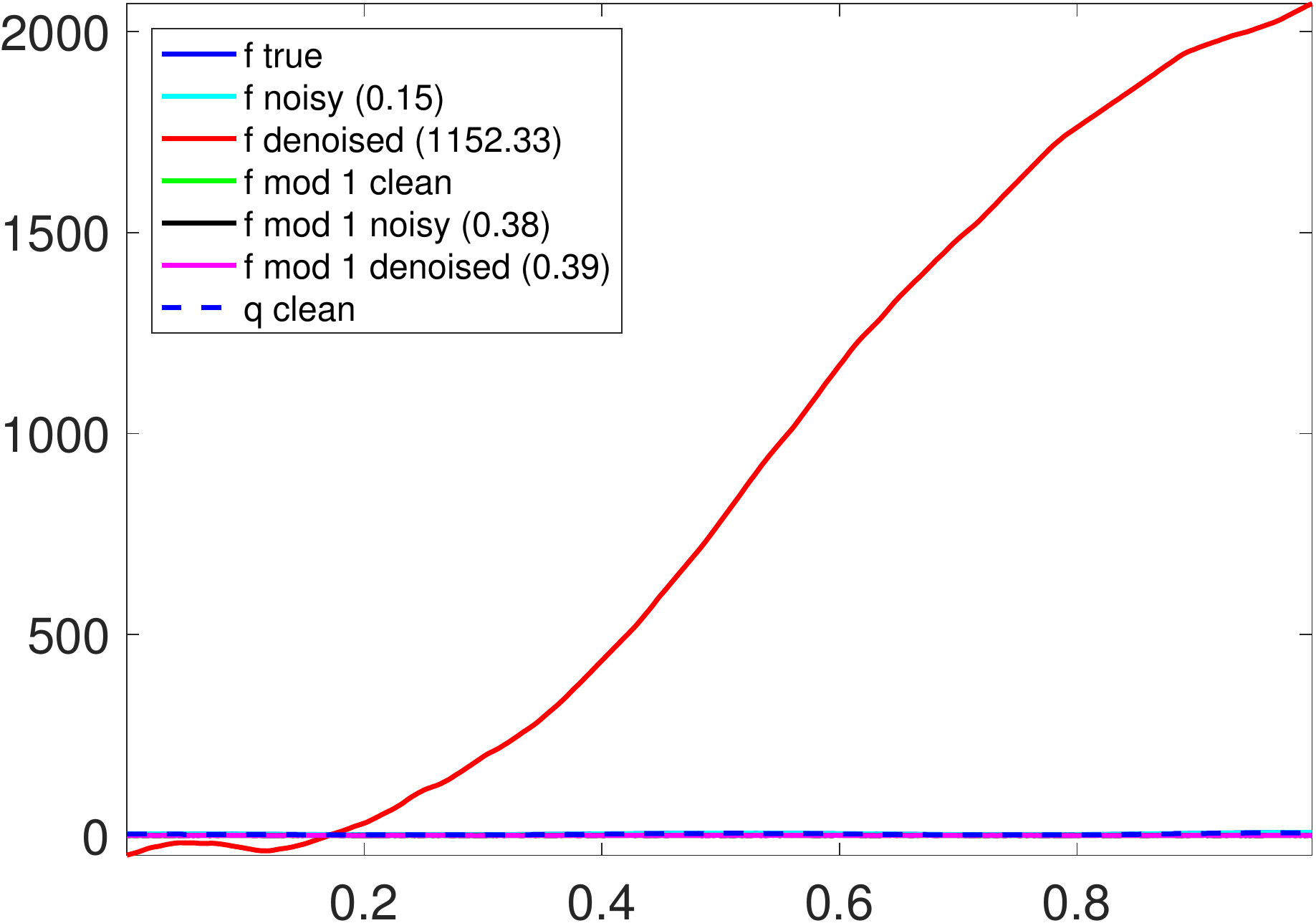} }
%
\subcaptionbox[]{  $\gamma=0.27$, \textbf{OLS}
}[ 0.24\textwidth ]
{\includegraphics[width=0.24\textwidth] {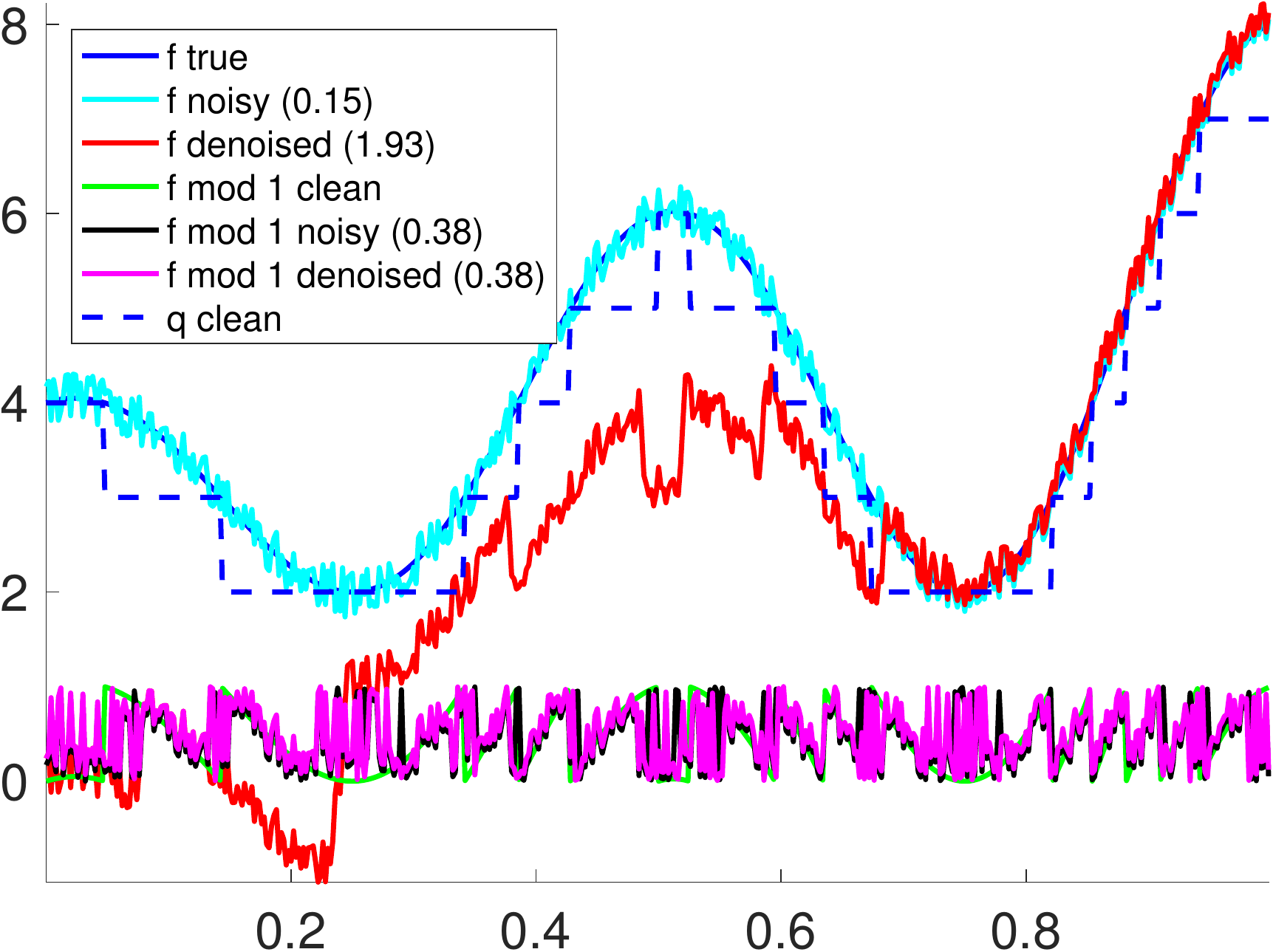} }
%
\subcaptionbox[]{  $\gamma=0.27$, \textbf{QCQP}
}[ 0.24\textwidth ]
{\includegraphics[width=0.24\textwidth] {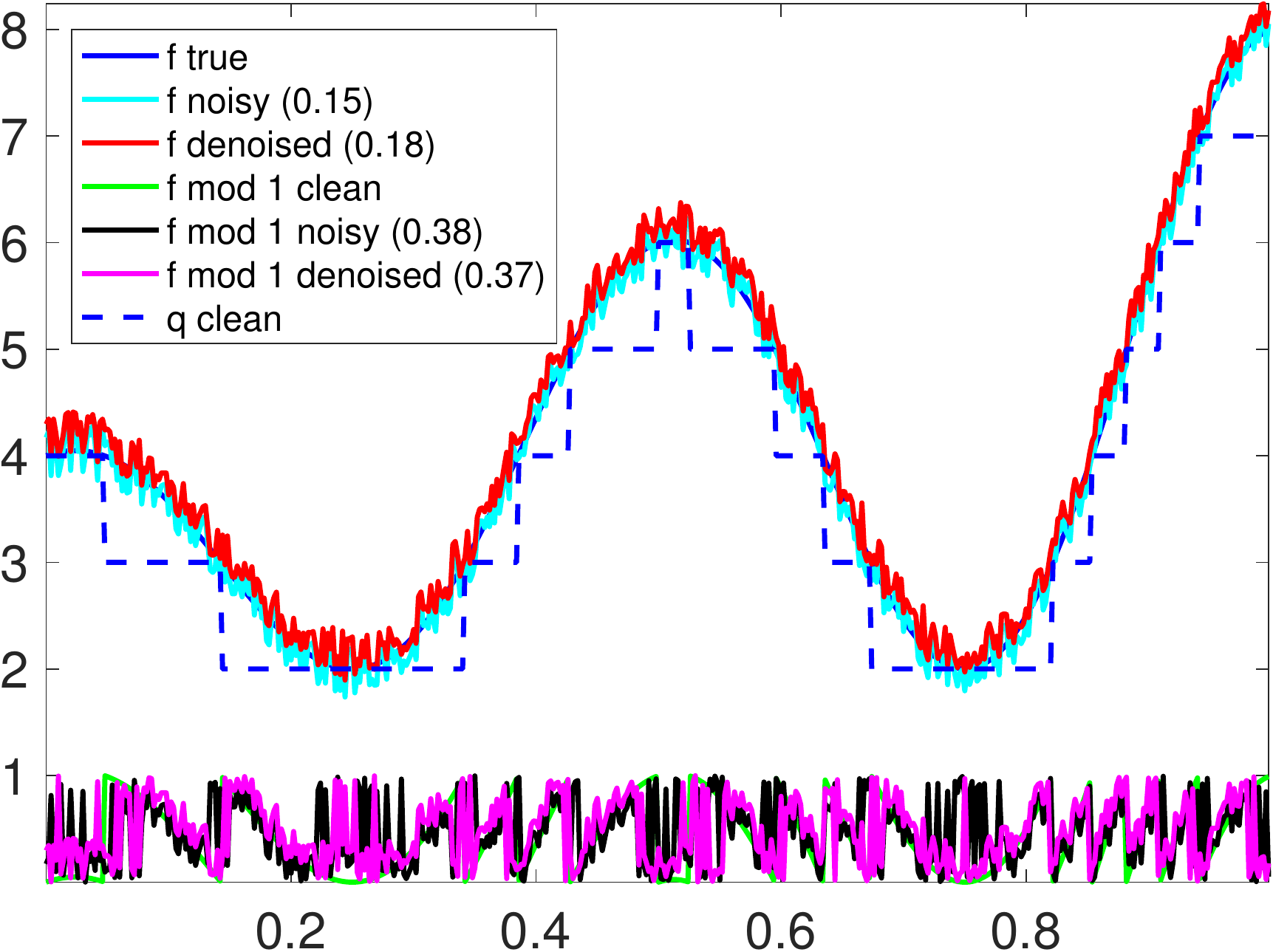} }
%
\subcaptionbox[]{  $\gamma=0.27$, \textbf{iQCQP}
}[ 0.24\textwidth ]
{\includegraphics[width=0.24\textwidth] {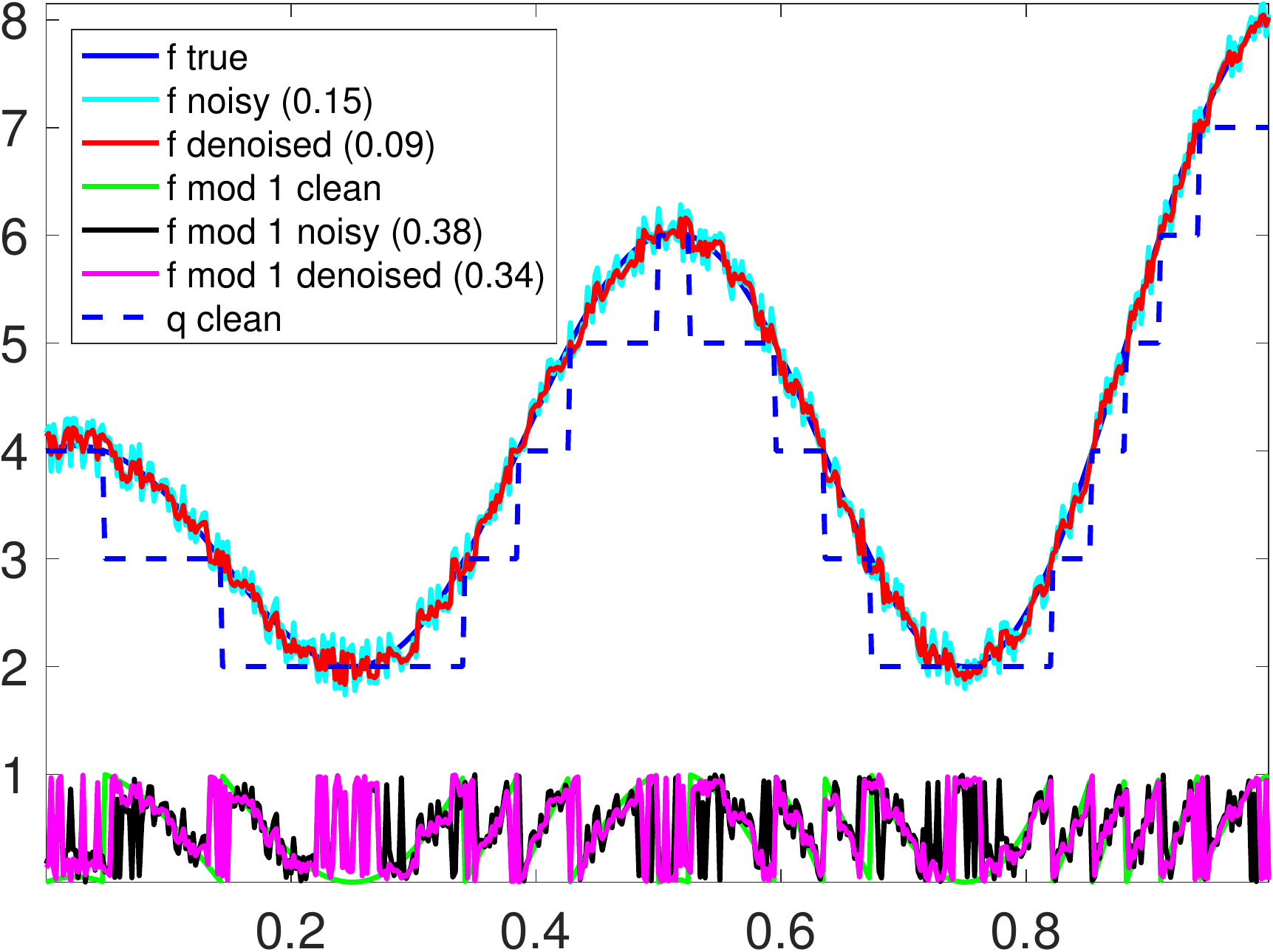} }
%
%
\captionsetup{width=0.98\linewidth}
\caption[Short Caption]{Denoised instances for both the $f$ mod 1 and $f$ values, under the Bounded-Uniform noise model for function \eqref{def:f1}, for  \textbf{BKR},  \textbf{OLS}, \textbf{QCQP} (Algorithm \ref{algo:two_stage_denoise}) and \textbf{iQCQP},  as we increase the noise level $\gamma$.  \textbf{QCQP} denotes Algorithm \ref{algo:two_stage_denoise}, for which  the unwrapping stage is  performed via \textbf{OLS} \eqref{eq:ols_unwrap_lin_system}.  
   We keep fixed the parameters $n=500$, $k=2$, $\lambda= 0.1$. The numerical values in the legend denote the RMSE.
}
\label{fig:instances_f1_Bounded_Sampta}
\end{figure}


 \FloatBarrier

\section{Modulo 1 denoising via optimization on manifolds} \label{sec:opt_manifolds}
Section \ref{subsec:alter_methods_manifolds} outlines two alternative  approaches for denoising the 
modulo 1 samples, both of which involve optimization on a smooth manifold. Section \ref{sec:manifold_num_sims} 
contains numerical simulations involving these approaches for the setting $d =2$, on both synthetic and 
real data.

\subsection{Two formulations   for denoising modulo $1$ samples} \label{subsec:alter_methods_manifolds}
We now describe two approaches for denoising the modulo $1$ samples, which are based on optimization over a smooth manifold. 
Both methods are implemented using Manopt (\cite{manopt}), a free Matlab toolbox for optimization on manifolds. 
Among others, the toolbox features a Riemannian trust-region solver based on the algorithm originally 
developed by \cite{absil2007trust}, who proposed a truncated conjugate-gradient algorithm to solve  trust-region 
subproblems. The  algorithm enjoys global convergence guarantees, i.e., it will converge to stationary points for any initial guess. 
Under certain assumptions on the smoothness of the manifold and cost function, the method enjoys quadratic local convergence if the true Hessian is used, superlinear  if the approximate Hessian is accurate enough, and linear otherwise. We refer the reader to \cite{absil2007trust, AbsMahSep2008} for an in-depth description of the algorithm, along with its theoretical convergence guarantees, numerical results on several problems  popular in the  numerical linear algebra literature, and comparison with other state-of-the art methods.
\cite{boumal2016globalrates} recently provided the first global rates of convergence to approximate first and second-order KKT points on the manifold.


\paragraph{(i) Solving original QCQP via optimization over $\calC_n$.} 
Our first approach involves solving the original QCQP \eqref{eq:orig_denoise_hard_1}, i.e., 
\begin{eqnarray}   \label{eq:orig_denoise_hard_1_bis}
\min_{ \vecg \in \calC_n} \lambda \vecg^{*}  L \vecg - 2Re(\vecg^{*}\vecz).
\end{eqnarray}
The feasible set $\calC_n := \set{\vecg \in \mathbb{C}^n: \abs{g_i} = 1}$ is a manifold (\cite{AbsMahSep2008}). It is in fact a submanifold 
of the embedding space $R^2 \times \ldots  \times R^2$, where the complex circle is identified with the unit circle in the real plane. 
We employ the \textbf{complexcirclefactory} structure of Manopt that returns a manifold structure to optimize over unit-modulus complex numbers.
This constitutes the input to the \textbf{trustregions} function within Manopt, along with our objective function, the gradient and the Hessian operator. In our experiments, we will denote this approach by Manopt-Phases. 

\paragraph{(ii) SDP relaxation via Burer-Monteiro.} An alternative direction we consider is to 
cast \eqref{eq:orig_denoise_hard_1_bis}  as a semidefinite program (SDP), and solve it efficiently via a Burer-Monteiro approach. Denoting 
\begin{equation}
	T  \mydef  \left[
  \begin{array}{cc}
 \lambda L & -\vecz \\ 
       - \vecz^* & 0
\end{array}
 \right]  \in \mathbb{C}^{(n+1) \times (n+1)}
\end{equation}
to account for the linear term, the SDP relaxation of \eqref{eq:orig_denoise_hard_1_bis} is given by 

\noindent\begin{minipage}{.48\linewidth}

\begin{equation}
	\begin{aligned}
& \underset{ \vecg \in  \mathbb{C}^n; \;\;     \Upsilon \in \mathbb{C}^{n \times n}}{\text{minimize}}
& & \lambda \vecg^{*}  L \vecg - 2Re(\vecg^{*}\vecz) \\
& \text{subject to} 
  & & \Upsilon_{ii} =1, i=1,\ldots,n \\
  & & &    
   \underbrace{ \left[
  \begin{array}{cc}
 \Upsilon & \vecg \\ 
       \vecg^* & 1
\end{array}
 \right]}_{W} \succeq 0 
	\end{aligned}
\label{SDP_phase_unwrapping_1}
\end{equation}

\end{minipage}
\noindent\begin{minipage}{.48\linewidth}

\begin{equation}
	\begin{aligned}
& \underset{W \in  \mathbb{C}^{(n+1) \times (n+1)}}{\text{minimize}}
& & \tr(T W) \\
& \text{subject to} 
  & & W_{ii} =1, i=1,\ldots,n+1 \\  
	\vspace{2mm}
  & & &    
    W \succeq 0.
	\end{aligned}
\label{SDP_phase_unwrapping_2}
\end{equation}
\vspace{3mm}
\end{minipage}

Note that, for \eqref{SDP_phase_unwrapping_1}, if  $\Upsilon = \vecg \vecg^* $, and using properties of the trace (invariance  under cyclic permutations), it can be easily verified that $\tr(T W)  = \lambda \vecg^{*}  L \vecg - 2Re(\vecg^{*}\vecz)$.  
Ideally, we would like to enforce the constraint $ \Upsilon = \vecg \vecg^* $, which guarantees that $\Upsilon$ is indeed a rank-1 solution. However, since rank constrains do not lead to convex optimization problems, we relax this constraint to $\Upsilon \succeq \vecg \vecg^{*}$, which via Schur's lemma is equivalent 
to  $W \succeq 0 $ (see for eg., \cite{boyd1994linear}). 

Since semidefinite programs are computationally  expensive to solve, we resort to the Burer-Monteiro approach, which consists of replacing optimization of a linear function $\langle C,X \rangle$ over the convex set $ \mathcal{X} = \{ X \succeq 0: \mathcal{A}(X)=b \}$ with optimization of the quadratic function $\langle CY,Y \rangle $ over the non-convex set $ \mathcal{Y} = \{ Y \in \mathbb{R}^{n \times p}: \mathcal{A}(YY^T)=b \}$.

If the convex set $ \mathcal{X}$  is compact, and $m$ denotes the number of constraints, it holds true that whenever  $p$ satisfies $ \frac{p(p+1)}{2} \geq m$, the two problems share the same global optimum \cite{Barvinok1995, Burer2005}.  Note that over the complex domain, an analogous statement holds true as soon as  $p^2 \geq m$. Building on earlier work of \cite{Burer2005}, \cite{boumal2016bmapproach} show that if the set $  \mathcal{X}  $ is compact and the set $ \mathcal{Y} $ is a smooth manifold,  then $ \frac{p(p+1)}{2} \geq m$ implies that, for almost all cost matrices, global optimality is achieved by any $Y$ satisfying second-order necessary optimality conditions.  

The problem we consider in \eqref{SDP_phase_unwrapping_1} and \eqref{SDP_phase_unwrapping_2} falls in the above setting 
with the set $\mathcal{X} := \{W \succeq 0: \mathcal{A}(W)=1\}$ being compact and involving $m = n+1$ diagonal constraints, 
and $\mathcal{Y}$ (defined below) being a smooth manifold. Note that in theory, choosing $p \geq \lfloor\sqrt{n+1}\rfloor + 1$ 
is sufficient, however in our simulations the choice $p=3$ already returned excellent results. 
In the Burer-Monteiro approach, we replace $\Upsilon$ with $Y Y^{*}$ in \eqref{SDP_phase_unwrapping_1} which amounts  to the following optimization problem over $Y \in  \mathbb{C}^{n \times p}$ and $\vecv \in  \mathbb{C}^{p}$ (the latter of which arises due to the linear term in the objective function)   
%
\begin{equation}
	\begin{aligned}
& \underset{Y \in  \mathbb{C}^{n \times p}; \;\;  \vecv \in  \mathbb{C}^{p} } {\text{minimize}}
& & \lambda \langle L Y, Y  \rangle - 2 \real(\vecz^{*} Y \vecv) \\
& \text{subject to} 
  & & \mbox{diag}(Y Y^*) =1 \\
  & & &   || \vecv ||_2^2 = 1. 
	\end{aligned}
\label{SDP_phase_unweerapping_BM}
\end{equation}
Therefore, we optimize over the product  manifold 
$\mathcal{Y} = \mathcal{M}_1  \times   \mathcal{M}_2$, where $\mathcal{M}_1$ denotes the manifold of $n \times p$ complex matrices with unit $2$-norm rows, and   $\mathcal{M}_2$ the unit sphere in  $\mathbb{C}^p$. 
Within Manopt,  $\mathcal{M}_1$   is  captured  by the \textbf{obliquecomplexfactory(n,p)} structure.  
Its underlying geometry is a product geometry of $n$ unit spheres in $\mathbb{C}^p$, with the real and imaginary  parts being treated separately as $2p$ real coordinates, thus making the complex oblique manifold a Riemannian submanifold of $(\mathbb{R}^2)^{p \times n}$.  
Furthermore,  $\mathcal{M}_2$ is captured within Manopt by the \textbf{spherecomplexfactory} structure.
In our experiments, we will denote this approach by Manopt-Burer-Monteiro.

    
%

\subsection{Numerical simulations} \label{sec:manifold_num_sims}
This section details the application of our approach to the two-dimensional phase unwrapping problem, on a variety of  both synthetic and    real-world examples.  We empirically evaluate the performance of Algorithm \ref{algo:two_stage_denoise} with three different methods for denoising the modulo 1 samples (Stage $1$): 
Trust-region subproblem (denoted by \textbf{QCQP} for consistency of notation with previous sections), Manopt-Phases, and  Manopt-Burer-Monteiro. 
Throughout all our experiments, we fix the neighborhood radius $k=1$ so that each node has on average 
$(2k+1)^2 - 1 = 8$ neighbors. The noise model throughout this section is fixed to be the Gaussian noise model.  Note that 
the running times reported throughout this section are exclusively for the optimization problem itself, and exclude the time needed to 
build the measurement graph and its Laplacian.


\paragraph{Two-dimensional synthetic example.}
We consider a synthetic example given by the multivariate function 
\begin{equation}
	  f(x,y) = 6 x e^{- x^2 - y^2}, 
\label{def:fxy}
\end{equation}	
with $n \approx 15,000$ (i.e., a square grid of size $\sqrt{n} \times \sqrt{n}$), which we almost perfectly recover at $\sigma = 10\%$ Gaussian noise level, as depicted in Figure \ref{fig:MV_Synt_sigma0p1}.  
Figure \ref{fig:MV_Manopts_Comp_Times} is a comparison of the computational running times for the Manopt-Phases and Manopt-Burer-Monteiro approaches, as well \textbf{OLS}  and \textbf{QCQP}, for varying parameters $\lambda \in \{0.01, 0.1,1\}$, noise level $\sigma \in \{0,0.1\}$, and $n \in \{1,2,4,8,\ldots,512,1024\} \times  1000$. 
For the same set of parameters, Figure \ref{fig:MV_Manopts_Comp_fmod1} shows the $\log(\mbox{RMSE})$ errors  for the recovered $f$ mod 1 values, while Figure  \ref{fig:MV_Manopts_Comp_f} shows similar  recovery errors for the final estimate of $f$. For large values of $n$, Manopt-Phases is significantly faster than Manopt-Burer-Monteiro and \textbf{QCQP}, with the recovery errors being almost identical for all three approaches. For the noisy setup, \textbf{OLS}  consistently returns much worse results compared to the other three methods, both in terms of denoising the $f$ mod 1 values and estimating the final $f$ samples.

In the noiseless setting, for mod 1 recovery, \textbf{OLS}  consistently returns an RMSE close to machine precision as shown in the top row of  Figure \ref{fig:MV_Manopts_Comp_fmod1}    (in other words, \textbf{OLS} preserves the mod 1 values when recovering the function $f$).  	
\textbf{OLS} has a mixed performance in terms of recovering the actual function $f$, shown in the top row of Figure \ref{fig:MV_Manopts_Comp_f}, where it can be seen that it either has a performance similar to that of the other three methods, or achieves an RMSE that is close to machine precision.

We remark that for the largest problem we consider, with over one million sample points, Manopt-Phases returns an almost perfect solution ($\log(\mbox{RMSE}) \approx -11$  for the mod 1 estimate with $\lambda=0.01$) for the noiseless case in under 2 seconds, for a fraction of the  cost of running \textbf{OLS}  (Figure  \ref{fig:MV_Manopts_Comp_Times} (a)). For the noisy example ($\sigma=0.10$, $\lambda = 1$), the same algorithm achieves an error of $\log(\mbox{RMSE}) \approx -4$ in under 10 seconds (Figure  \ref{fig:MV_Manopts_Comp_Times} (f)).

\paragraph{Two-dimensional real elevation maps.}
We then consider a variety of  elevation maps\footnote{Data acquired via the Digital Elevation Model (DEM) of the Earth using radar interferometry. We relied on a Matlab script  developed by  Fran\c{c}ois Beauduce (Institut de Physique du Globe de Paris, \url{https://dds.cr.usgs.gov/srtm/version2_1}), which downloads SRTM data and plots a map corresponding to the geographic area defined by latitude and longitude limits (in decimal degrees).}, in particular those surrounding the  Vesuvius and Etna volcanoes and Mont Blanc, at various resolution parameters, all of which we solve by Manopt-Phases.  
Figure \ref{fig:Vesuvius_LOW_Noisy} (shown in the Appendix) pertains to the elevation map of  Vesuvius, using  $n = 3600$ sample points, and noise level  $\sigma=0.05$. 
Figures  \ref{fig:Vesuvius_HIGH_Noisy}, respectively \ref{fig:Vesuvius_HIGH_VeryNoisy},  also concern Mount Vesuvius, but at a higher resolution, using $n = 32400$ sample points with noise level  $\sigma=0.10$, respectively     $\sigma=0.25$.
Figure  \ref{fig:Etna_Low_Noisy}, respectively  \ref{fig:ETNA_HIGH_Noisy} (deferred to the Appendix), pertain to noisy measurements of Mount Etna at $\sigma=0.10$ Gaussian noise added, where the number of sample points is  $n = 4100$, respectively  $n = 19000$. 
Finally, Figure  \ref{fig:MontBlanc_clean}, respectively Figure \ref{fig:MontBlanc_noisy}, show the recovered elevation map of Mont Blanc from over 1 million sample points in under 18 seconds, for clean, respectively noisy ($\gamma=10 \%$ Bounded uniform noise), measurements.


\ifthenelse{\boolean{ShowNoiselessMultivar}}{  
\begin{figure*}
\centering

\subcaptionbox[]{  Noiseless function
}[ 0.24\textwidth ]
{\includegraphics[width=0.24\textwidth] {figs/PLOTS_MULTIVARIATE/MultivarSynthetic_1/Gaussian_sigma_0_n_14884_kDist_1_lambda_2_scale_6_fClean_3d.png} }
%
\subcaptionbox[]{  Noiseless function
}[ 0.24\textwidth ]
{\includegraphics[width=0.24\textwidth] {figs/PLOTS_MULTIVARIATE/MultivarSynthetic_1/Gaussian_sigma_0_n_14884_kDist_1_lambda_2_scale_6_fClean_TOP.png} }
%
\subcaptionbox[]{  Clean $f$ mod 1
}[ 0.24\textwidth ]
{\includegraphics[width=0.24\textwidth] {figs/PLOTS_MULTIVARIATE/MultivarSynthetic_1/Gaussian_sigma_0_n_14884_kDist_1_lambda_2_scale_6_f_mod1_clean_TOP.png} }
\subcaptionbox[]{    Denoised $f$ mod 1
}[ 0.24\textwidth ]
{\includegraphics[width=0.24\textwidth] {figs/PLOTS_MULTIVARIATE/MultivarSynthetic_1/Gaussian_sigma_0_n_14884_kDist_1_lambda_2_scale_6_f_mod1_denoised_TOP.png} }

\subcaptionbox[]{  Recovered $f$
}[ 0.24\textwidth ]
{\includegraphics[width=0.24\textwidth] {figs/PLOTS_MULTIVARIATE/MultivarSynthetic_1/Gaussian_sigma_0_n_14884_kDist_1_lambda_2_scale_6_f_denoised_3d.png} }
%
\subcaptionbox[]{  Error  $|f - \hat{f}|$ 
}[ 0.24\textwidth ]
{\includegraphics[width=0.24\textwidth] {figs/PLOTS_MULTIVARIATE/MultivarSynthetic_1/Gaussian_sigma_0_n_14884_kDist_1_lambda_2_scale_6_f_delta_TOP.png} }
%
\captionsetup{width=0.95\linewidth}
\caption[Short Caption]{  Synthetic example  $f(x,y) = 6 x e^{- x^2 - y^2}$, with $n = 15000$,     $k=1$ (Chebychev distance), $\lambda = 2$,  and noise level $\sigma=0$.
}
\label{fig:MV_Synt_sigma0}
\end{figure*}
}{}

\begin{figure*}
\centering
  
\subcaptionbox[]{  Noiseless function
}[ 0.24\textwidth ]
{\includegraphics[width=0.24\textwidth] {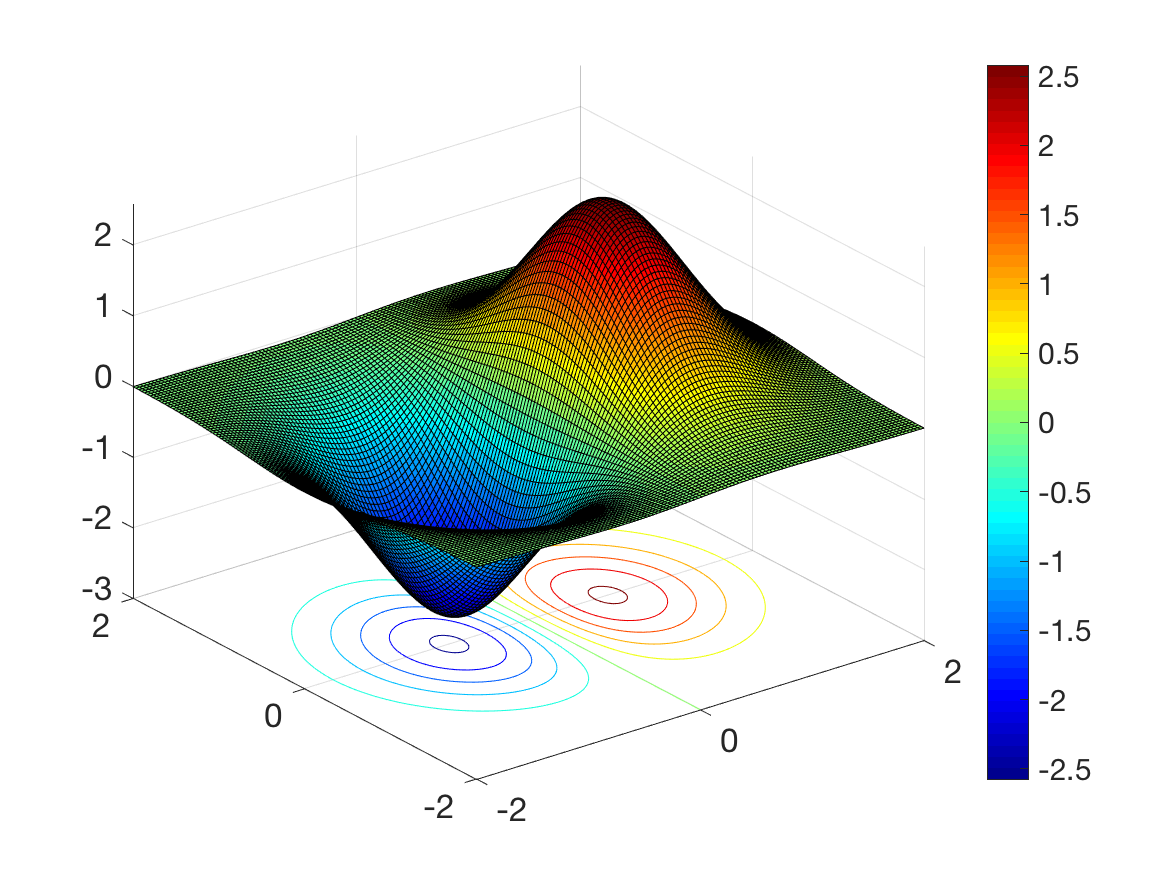} }
%
\subcaptionbox[]{  Noiseless function (depth)
}[ 0.24\textwidth ]
{\includegraphics[width=0.24\textwidth] {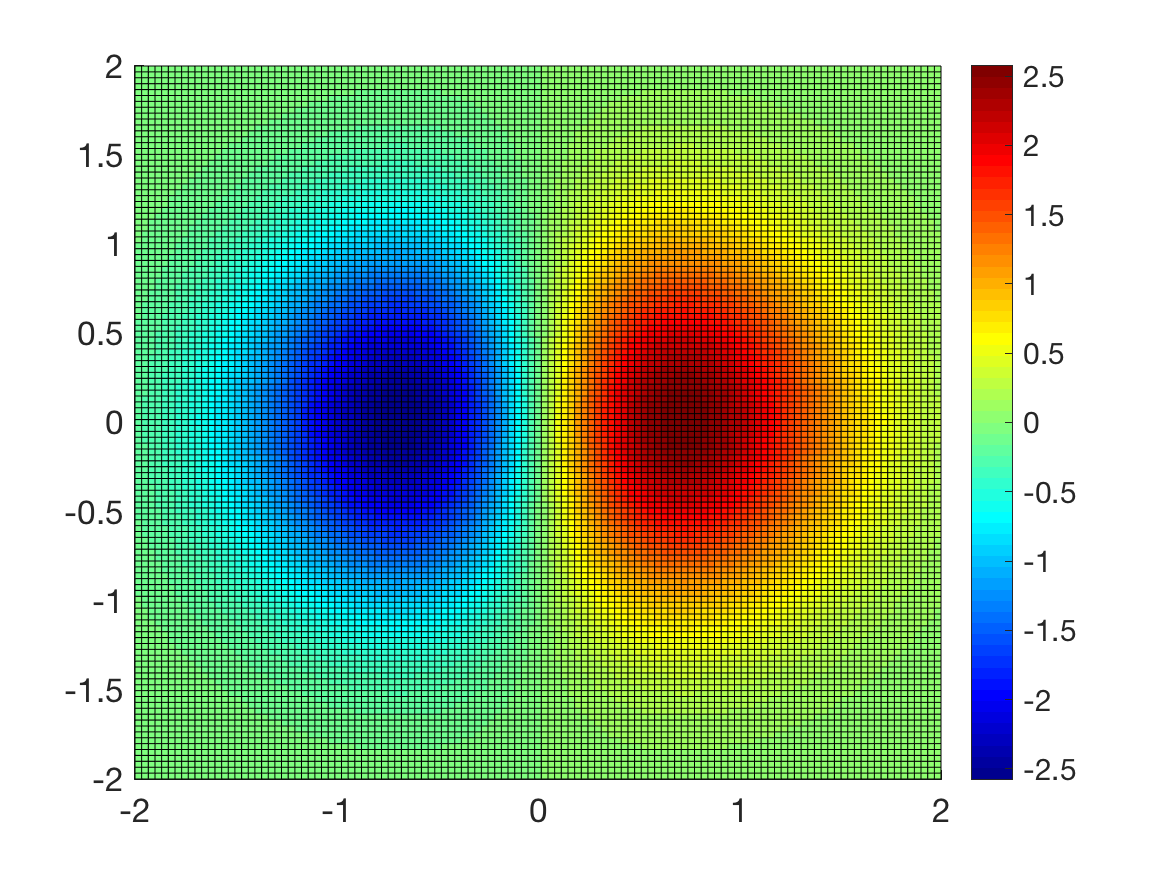} }
%
\subcaptionbox[]{   Clean $f$ mod 1
}[ 0.24\textwidth ]
{\includegraphics[width=0.24\textwidth] {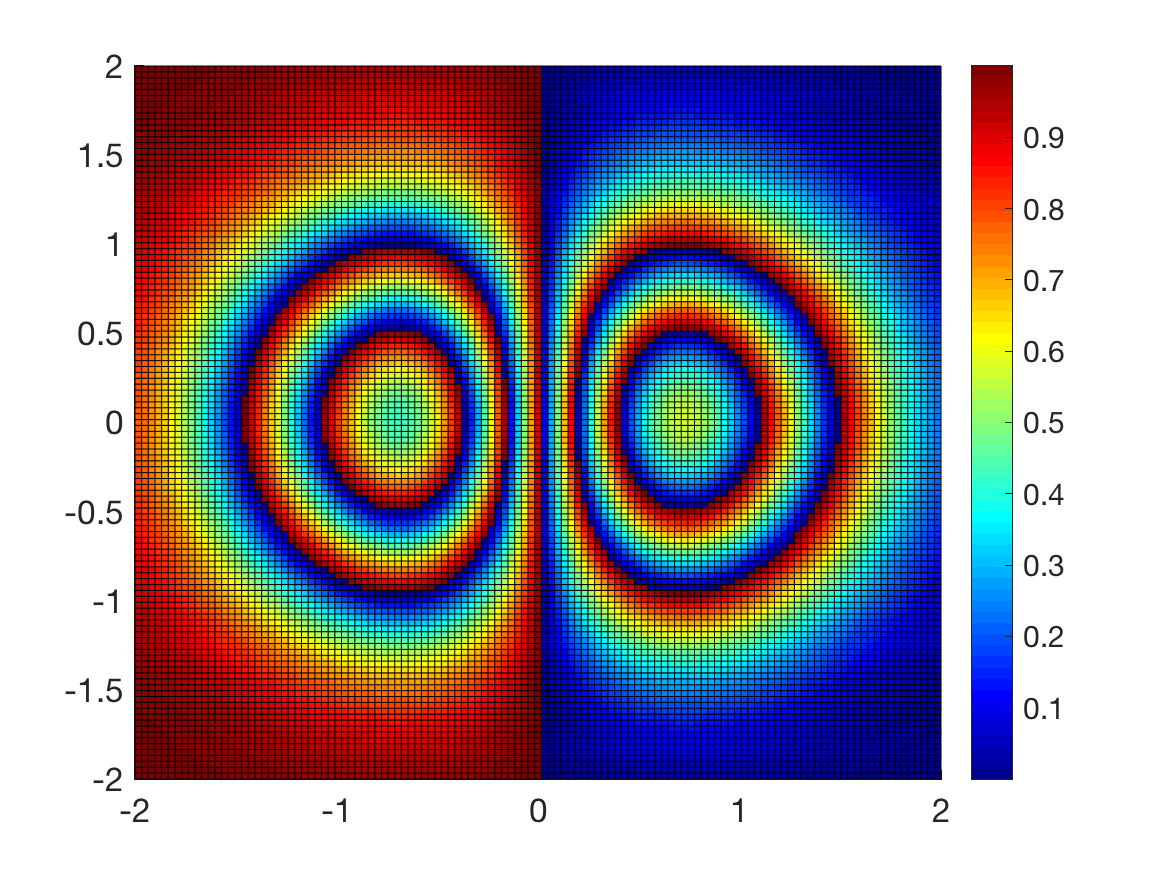} }
\subcaptionbox[]{    Noisy $f$ mod 1
}[ 0.24\textwidth ]
{\includegraphics[width=0.24\textwidth] {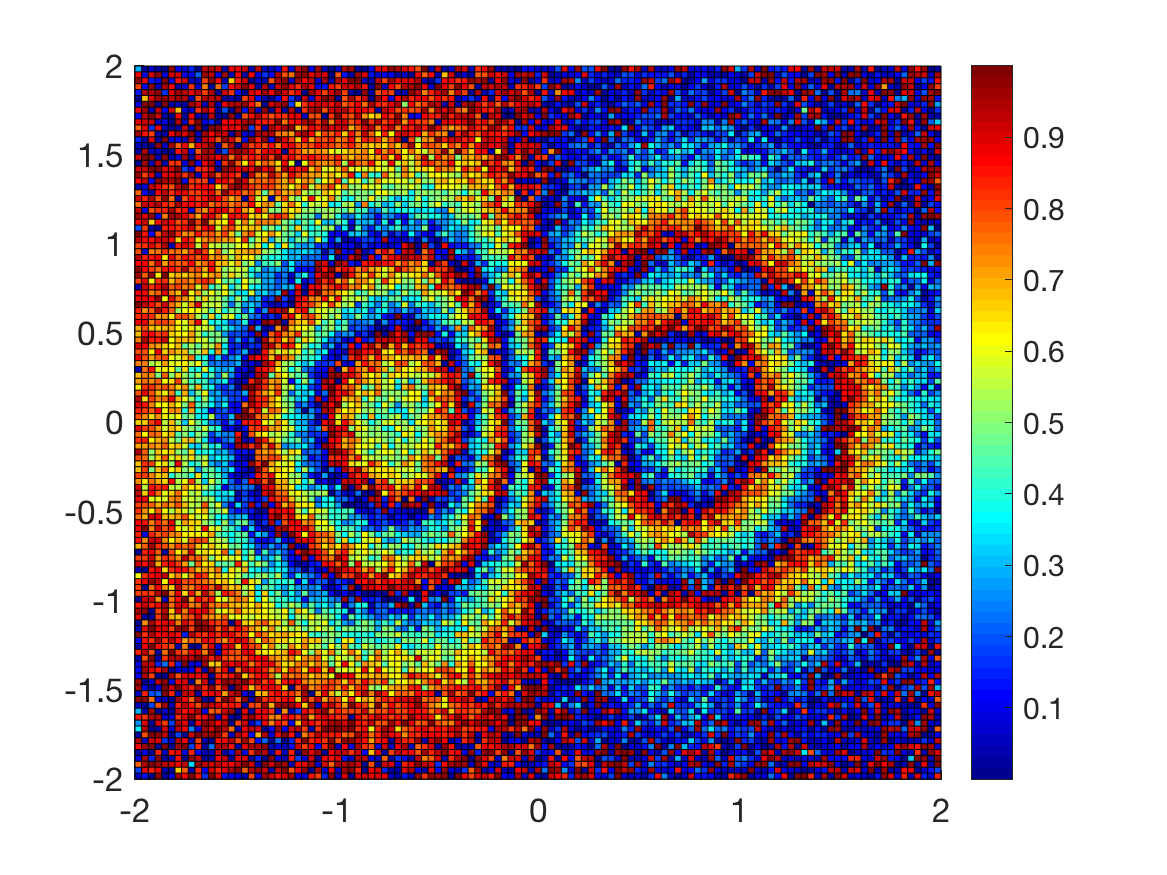} }
 
\subcaptionbox[]{      Denoised $f$ mod 1 \\  (RMSE=0.135)
}[ 0.24\textwidth ]
{\includegraphics[width=0.24\textwidth] {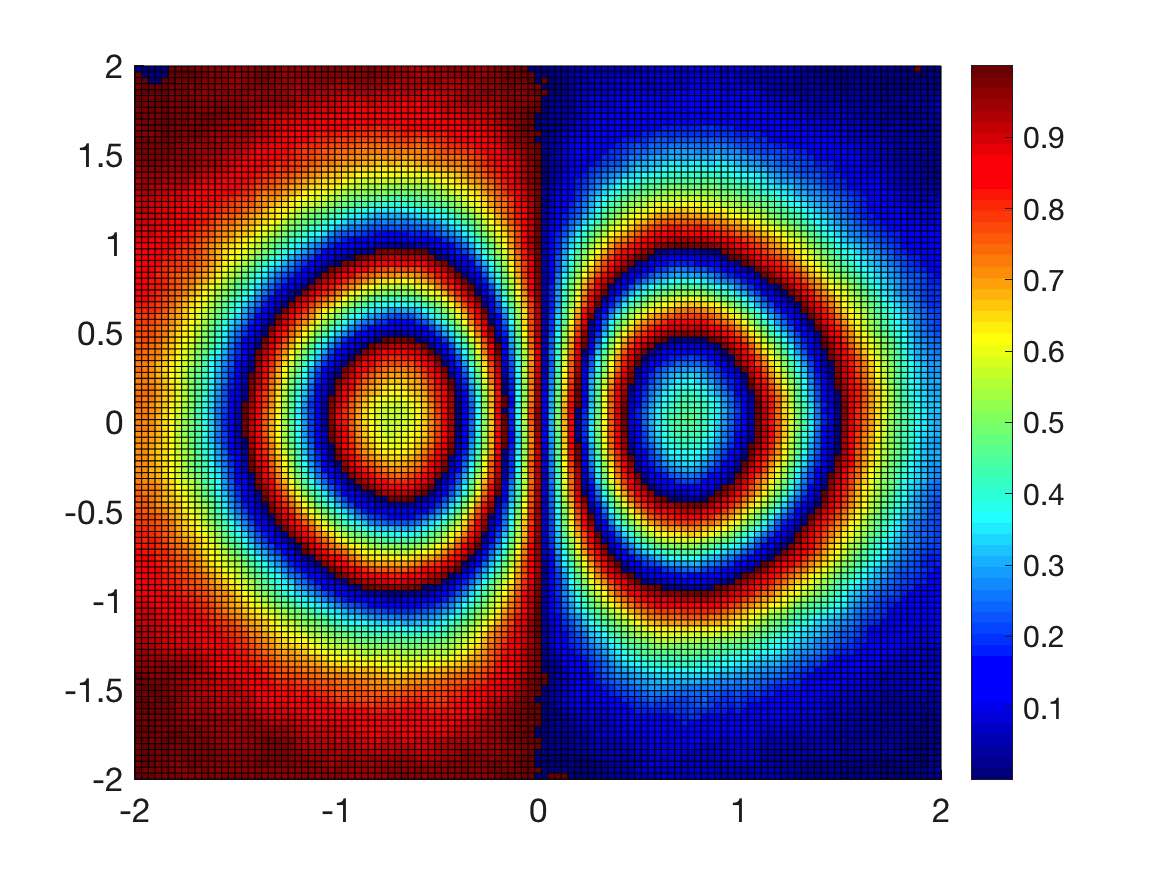} }
%
\subcaptionbox[]{   Denoised $f$  \\   (RMSE=0.035)
}[ 0.24\textwidth ]
{\includegraphics[width=0.24\textwidth] {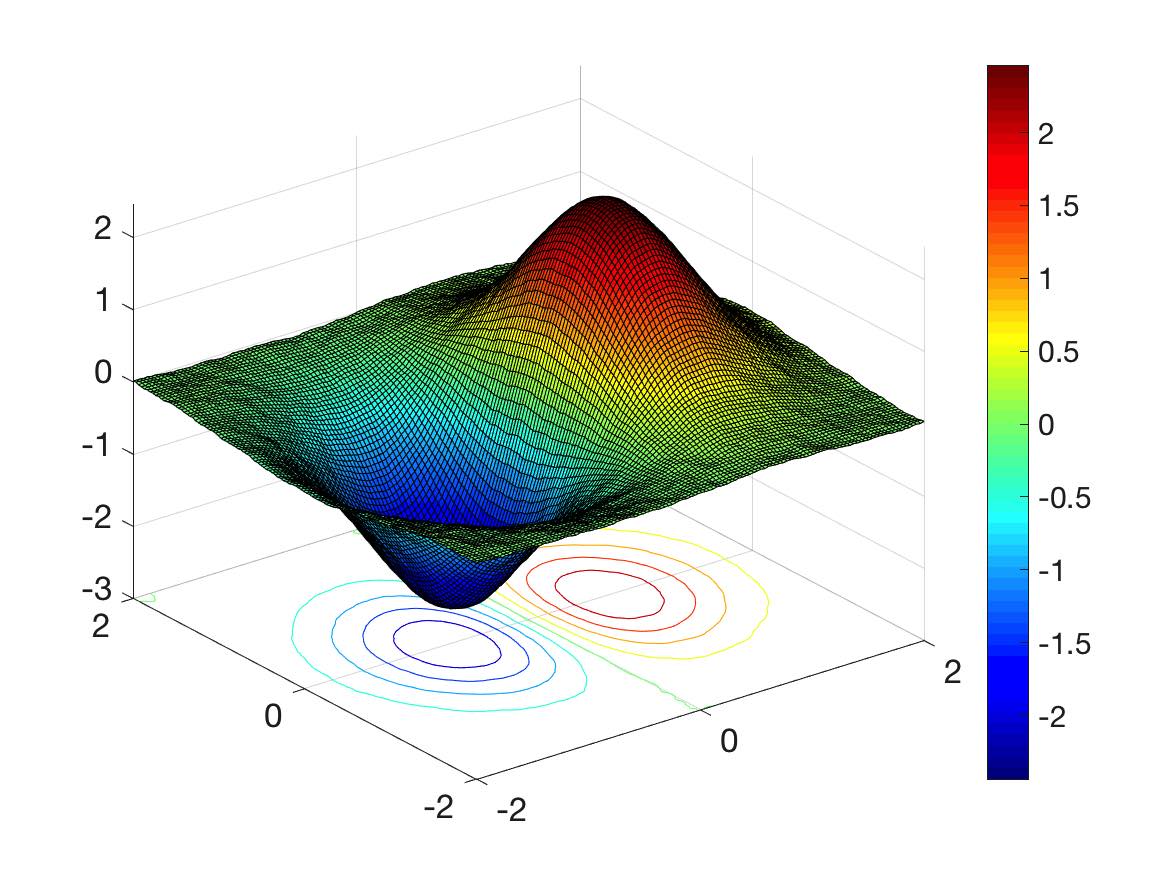} }
%
%
\subcaptionbox[]{   Denoised $f$ (depth)
}[ 0.24\textwidth ]
{\includegraphics[width=0.24\textwidth] {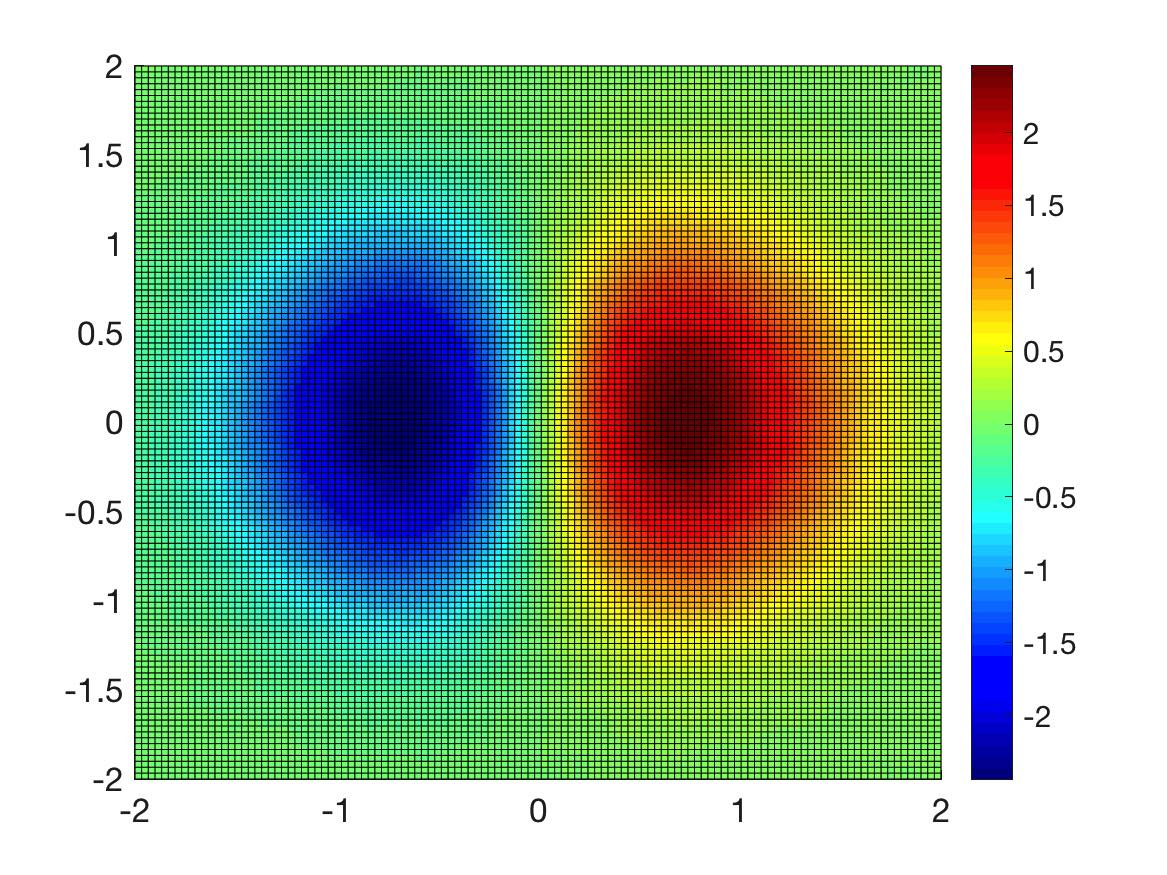} }
%
%
\subcaptionbox[]{  Error  $|f - \hat{f}|$ 
}[ 0.24\textwidth ]
{\includegraphics[width=0.24\textwidth] {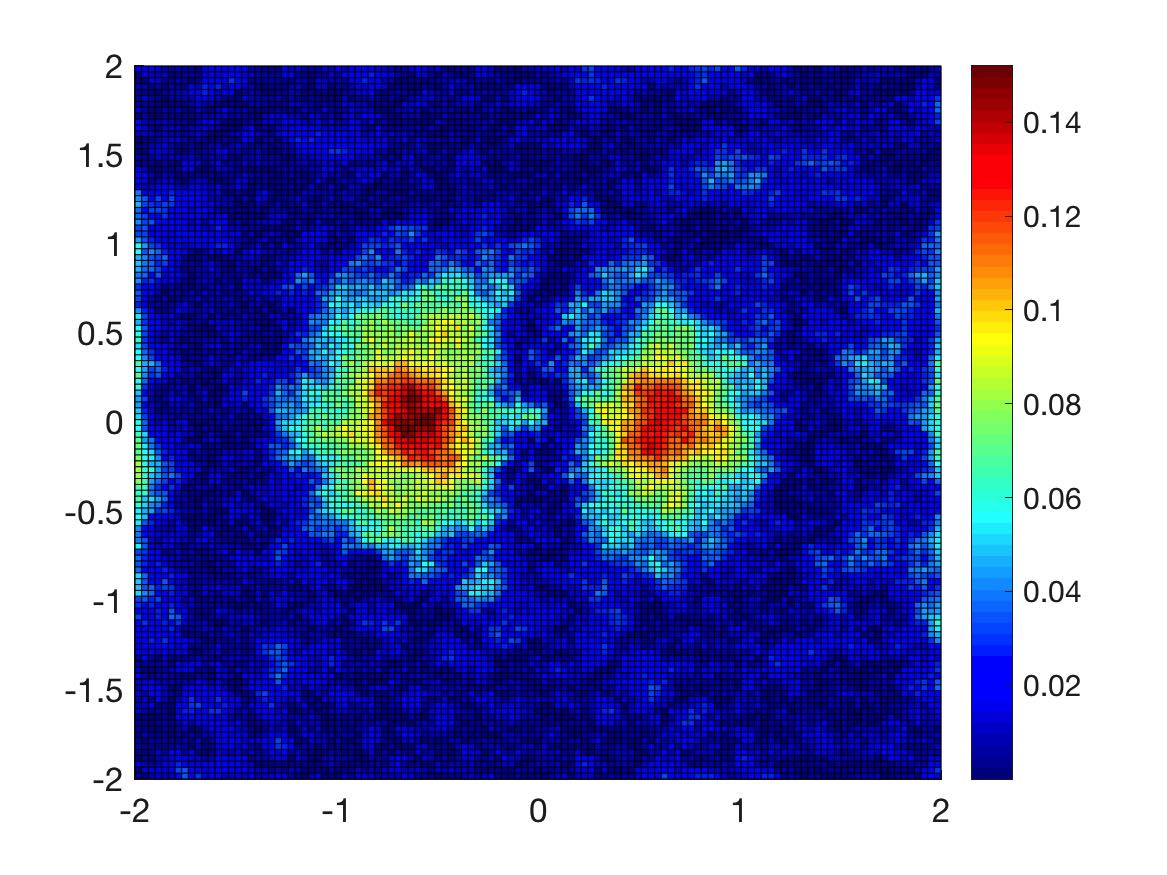}}   
\captionsetup{width=0.95\linewidth}
\caption[Short Caption]{ Synthetic example  $f(x,y) = 6 x e^{- x^2 - y^2}$, with $n = 15000$,     $k=1$ (Chebychev distance), $\lambda = 2$,  and noise level $\sigma=0.10$ under the Gaussian model, as recovered by Manopt-Phases.}
\label{fig:MV_Synt_sigma0p1}
\end{figure*}

\begin{figure*}
\centering

\subcaptionbox[]{   $ \sigma=0, \lambda = 0.01$
}[ 0.30\textwidth ]
{\includegraphics[width=0.30\textwidth] {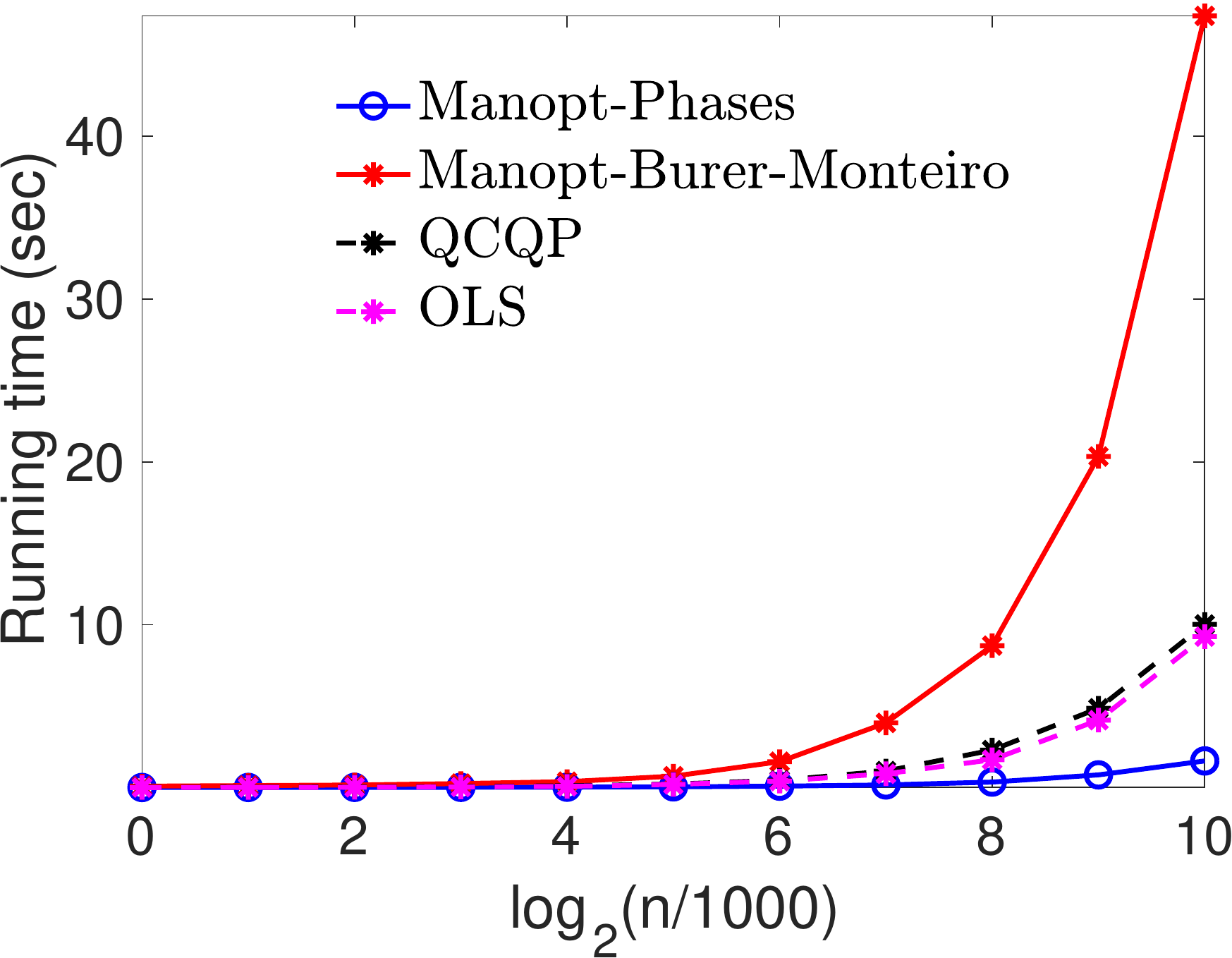} }
%
\subcaptionbox[]{   $ \sigma=0, \lambda = 0.1$
}[ 0.30\textwidth ]
{\includegraphics[width=0.30\textwidth] {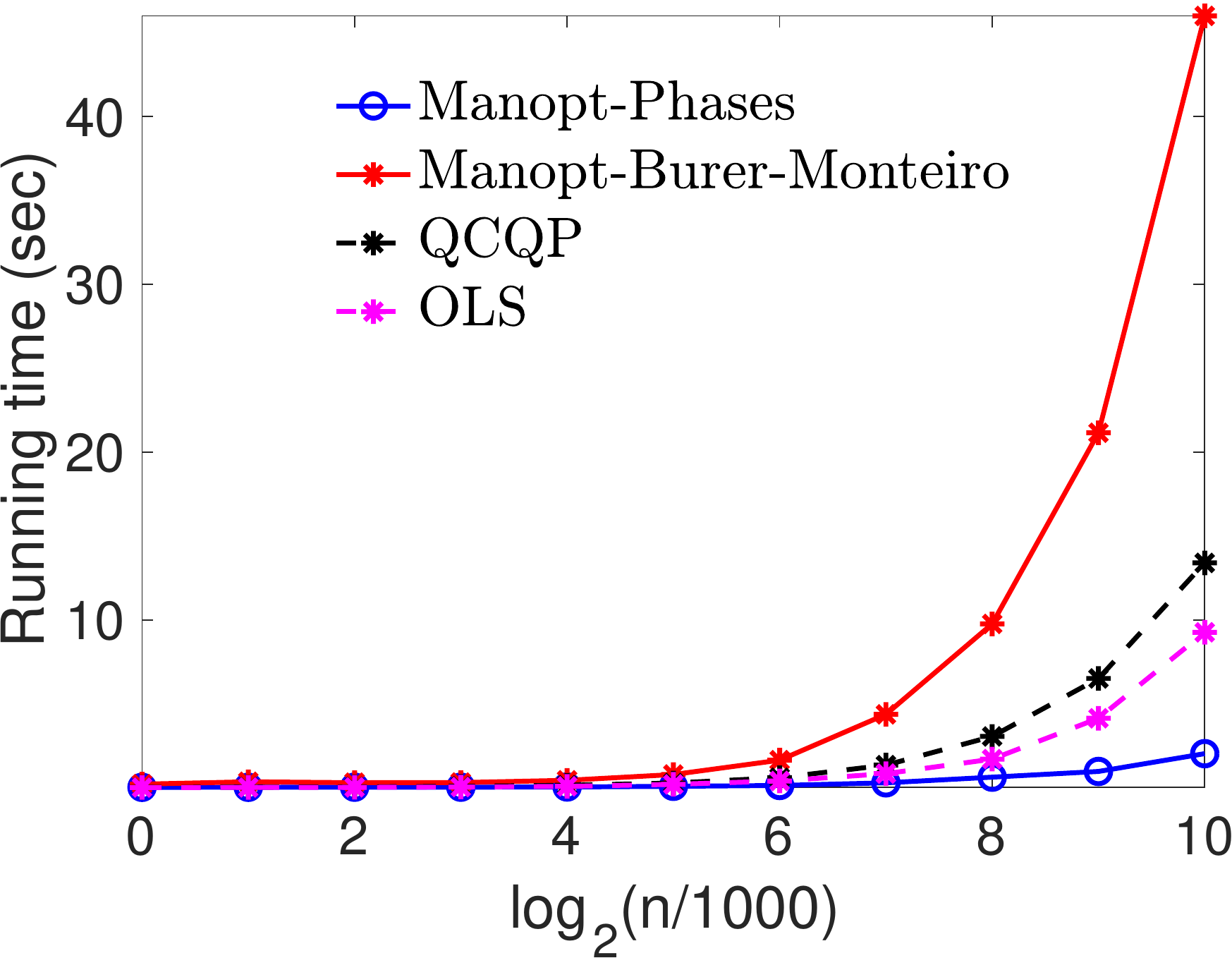} }
%
\subcaptionbox[]{   $ \sigma=0, \lambda = 1$
}[ 0.30\textwidth ]
{\includegraphics[width=0.30\textwidth] {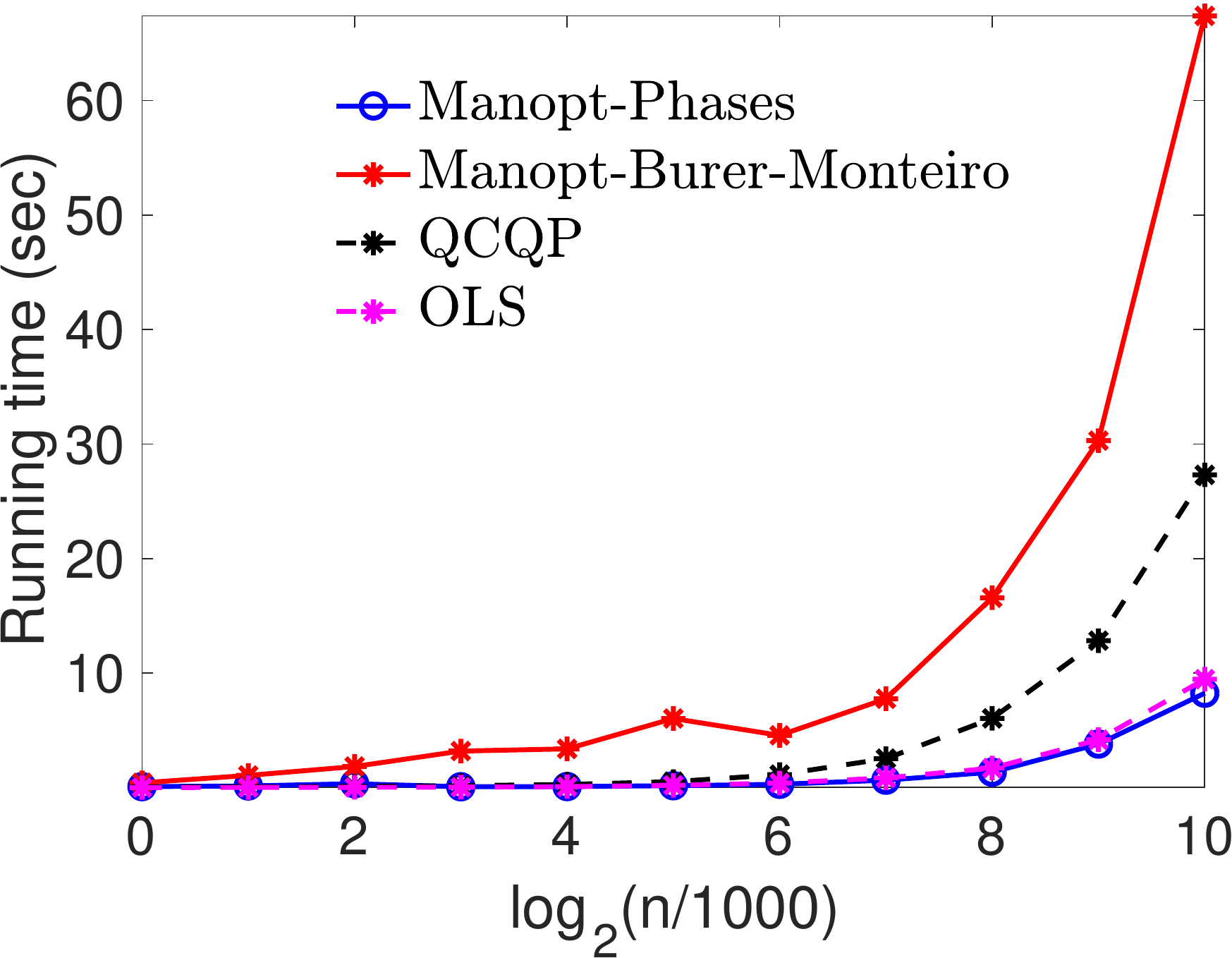} }

\subcaptionbox[]{    $ \sigma=0.1, \lambda = 0.01$
}[ 0.30\textwidth ]
{\includegraphics[width=0.30\textwidth] {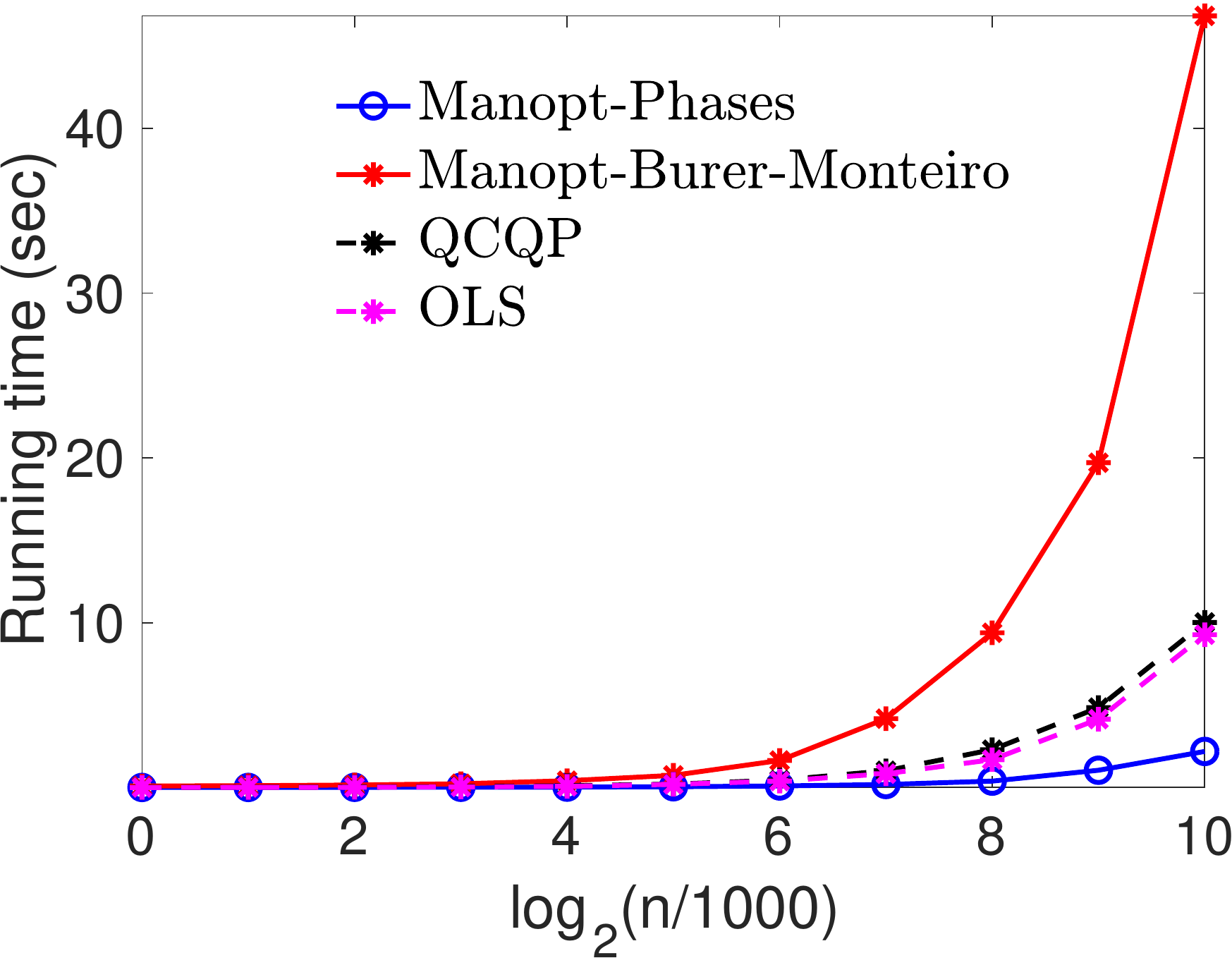} }
\subcaptionbox[]{   $ \sigma=0.1, \lambda = 0.1$
}[ 0.30\textwidth ]
{\includegraphics[width=0.30\textwidth] {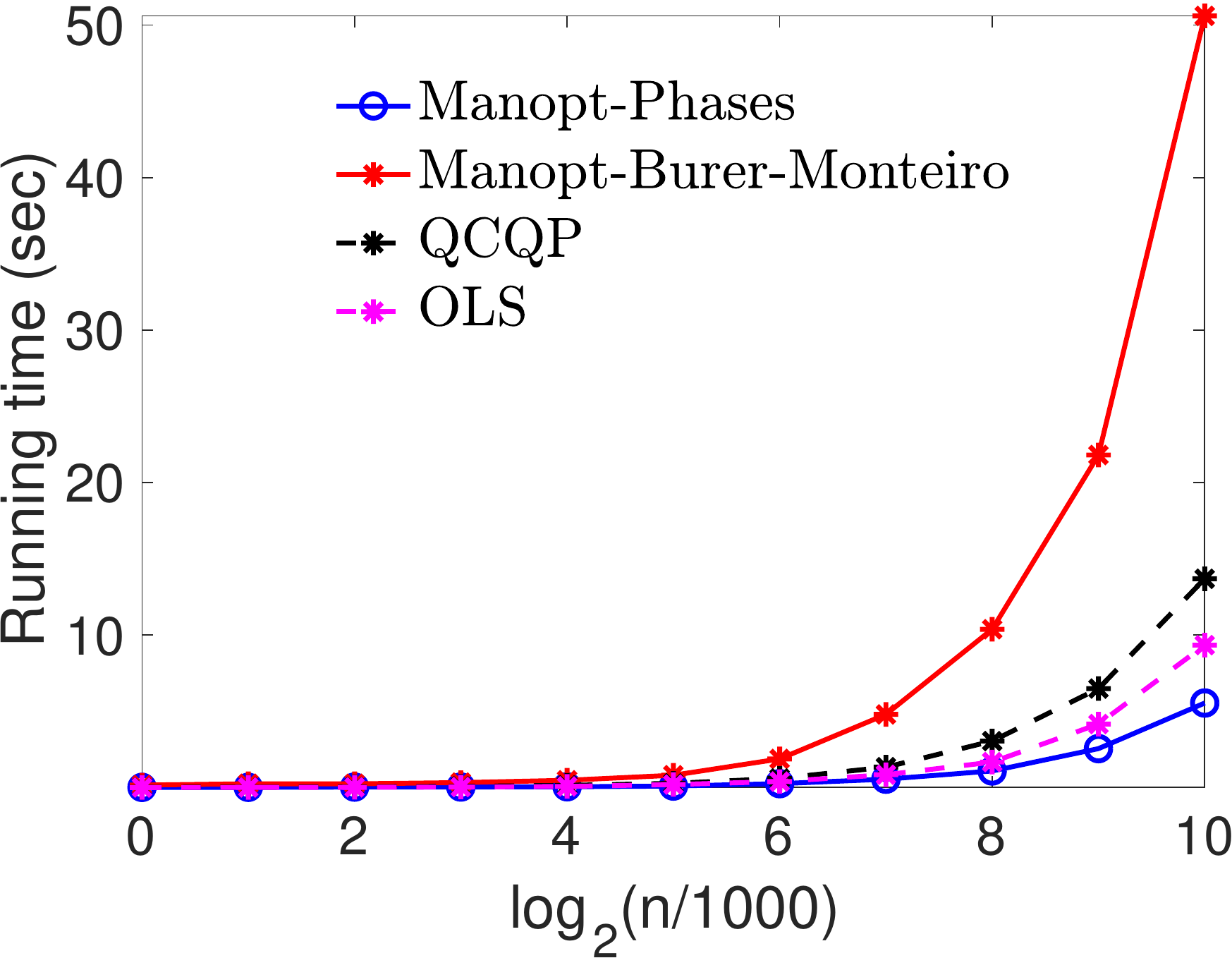} }
%
\subcaptionbox[]{  $ \sigma=0.1, \lambda = 1$
}[ 0.30\textwidth ]
{\includegraphics[width=0.30\textwidth] {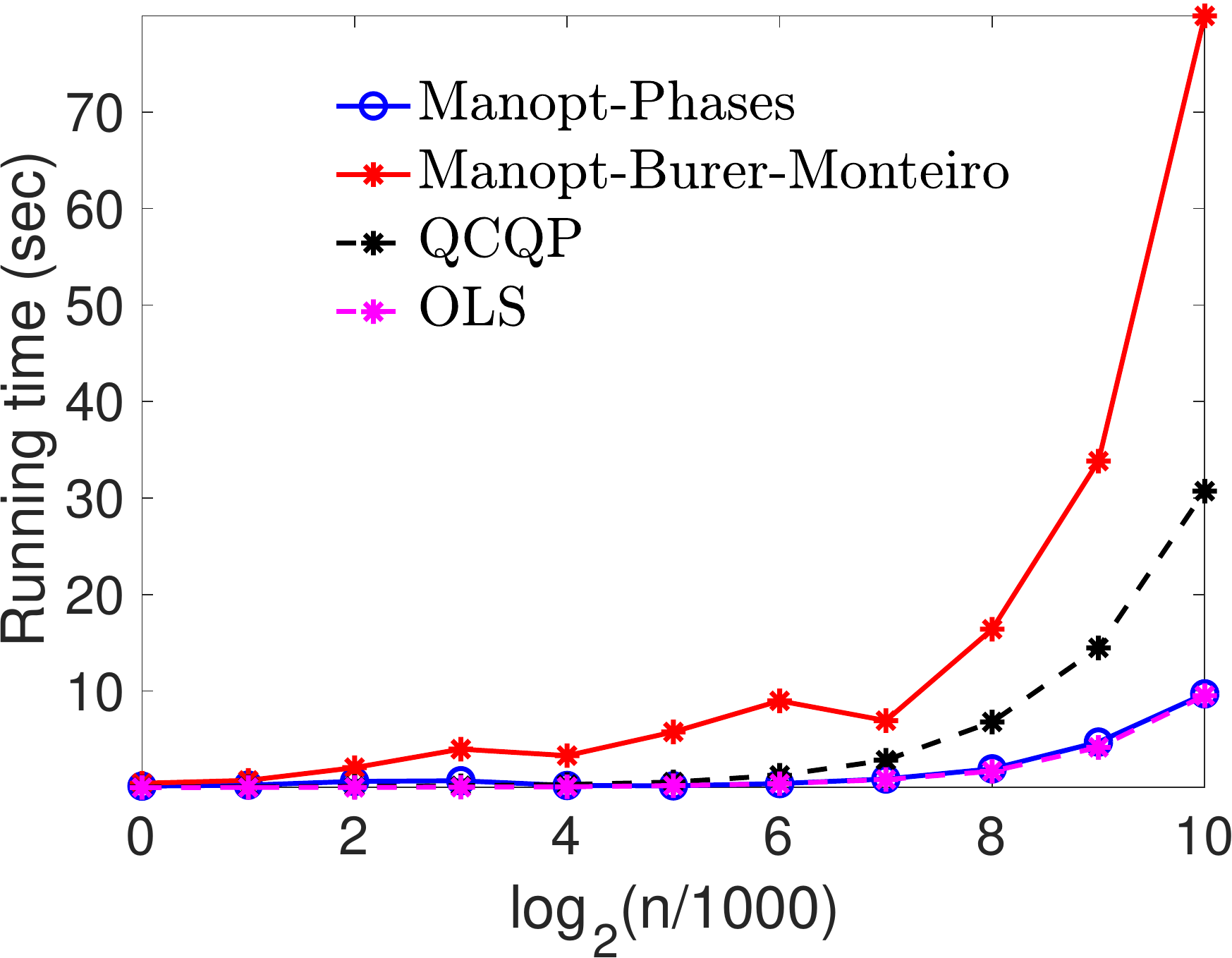} }
%
\captionsetup{width=0.95\linewidth}
\caption[Short Caption]{Comparison of running times  
for the synthetic example  $f(x,y) = 6 x e^{- x^2 - y^2}$,  $k=1$ (Chebychev distance), for several noise levels $\sigma$ and values of $\lambda$. 
\textbf{QCQP} denotes Algorithm \ref{algo:two_stage_denoise}, where  the unwrapping stage is  performed by \textbf{OLS} \eqref{eq:ols_unwrap_lin_system}.  
}         
\label{fig:MV_Manopts_Comp_Times}
\end{figure*}
 
\vspace{-2mm}

\begin{figure*}
\centering
\vspace{-2mm}
\subcaptionbox[]{   $ \sigma=0, \lambda = 0.01$
}[ 0.30\textwidth ]
{\includegraphics[width=0.30\textwidth] {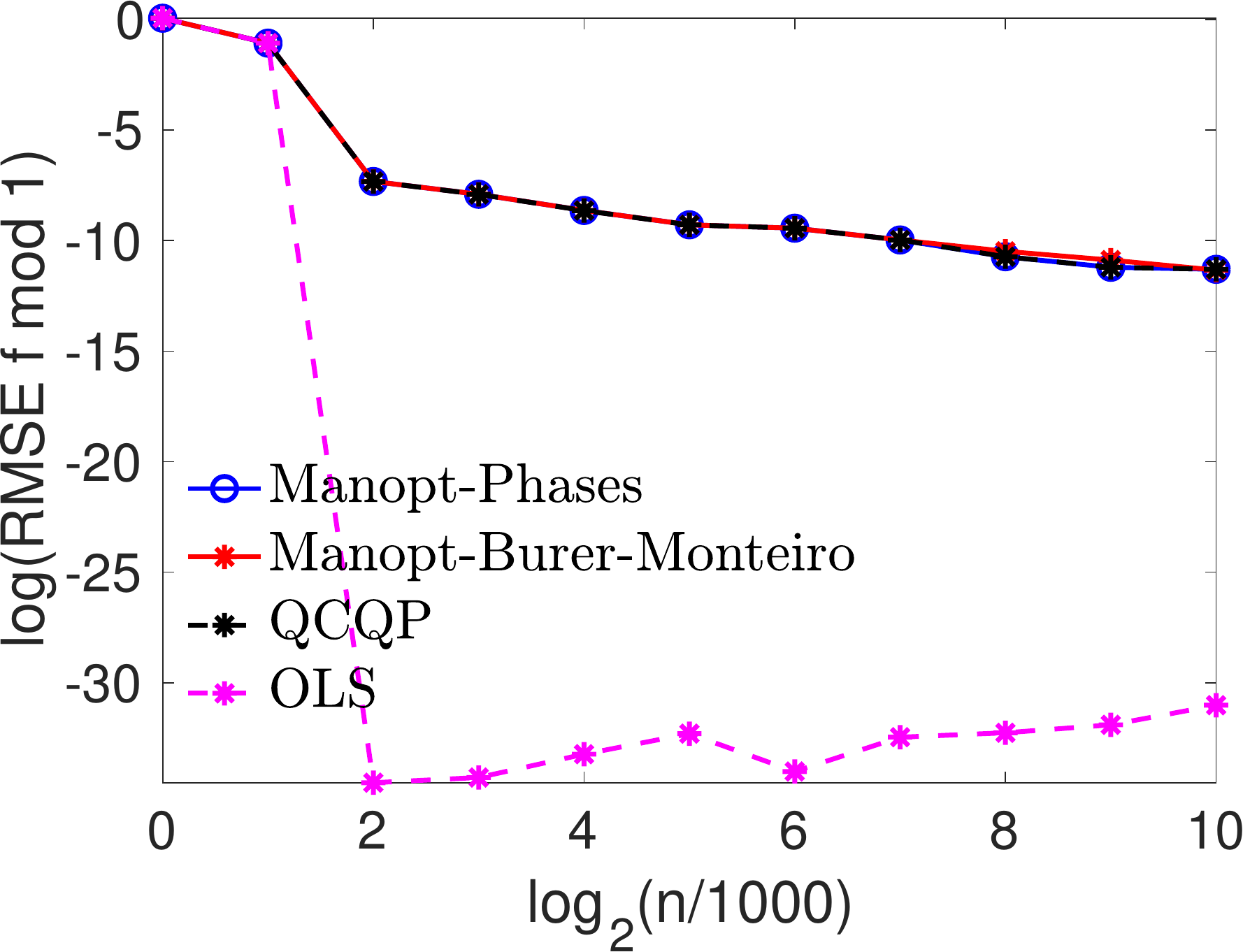} }
%
\subcaptionbox[]{   $ \sigma=0, \lambda = 0.1$
}[ 0.30\textwidth ]
{\includegraphics[width=0.30\textwidth] {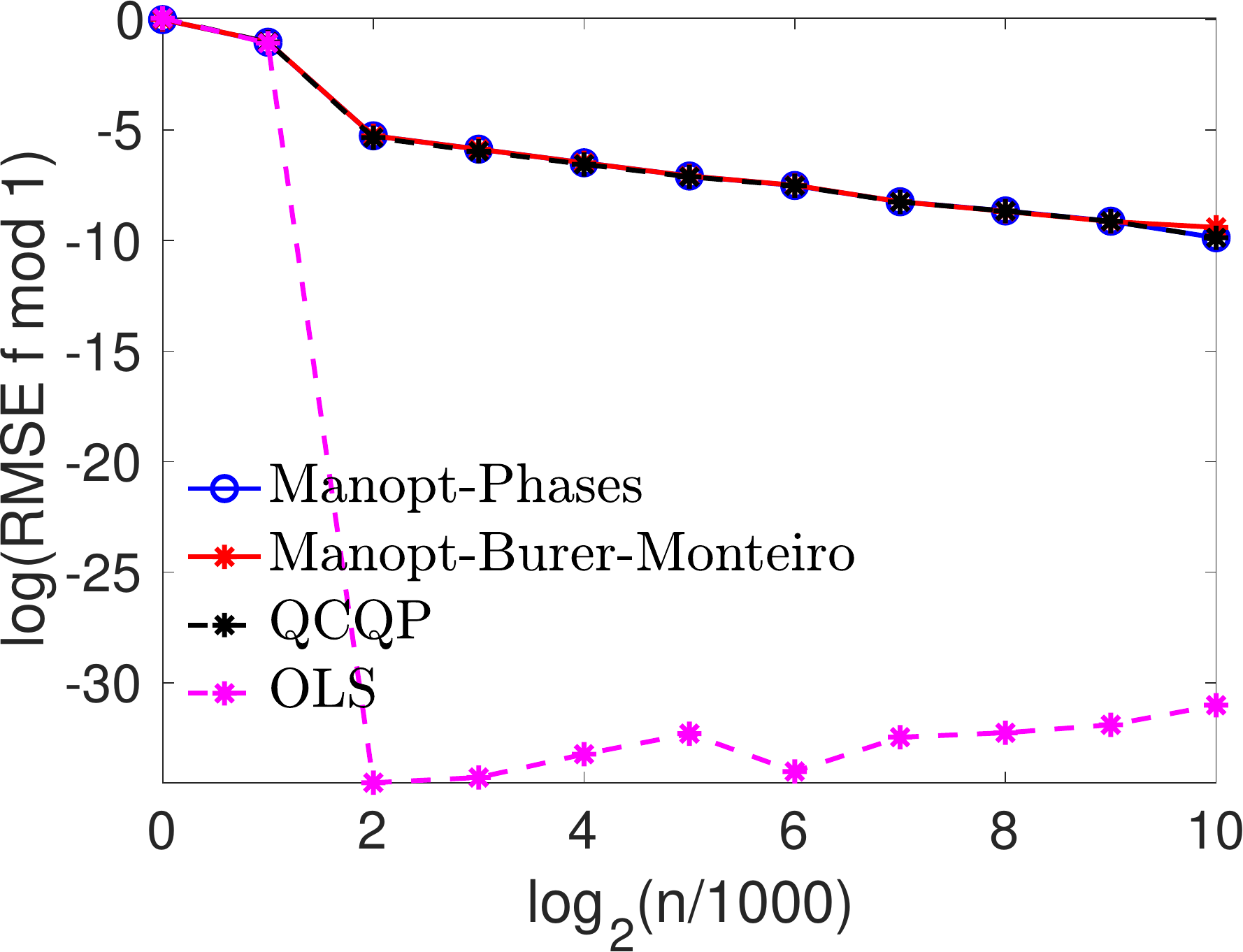} }
%
\subcaptionbox[]{   $ \sigma=0, \lambda = 1$
}[ 0.30\textwidth ]
{\includegraphics[width=0.30\textwidth] {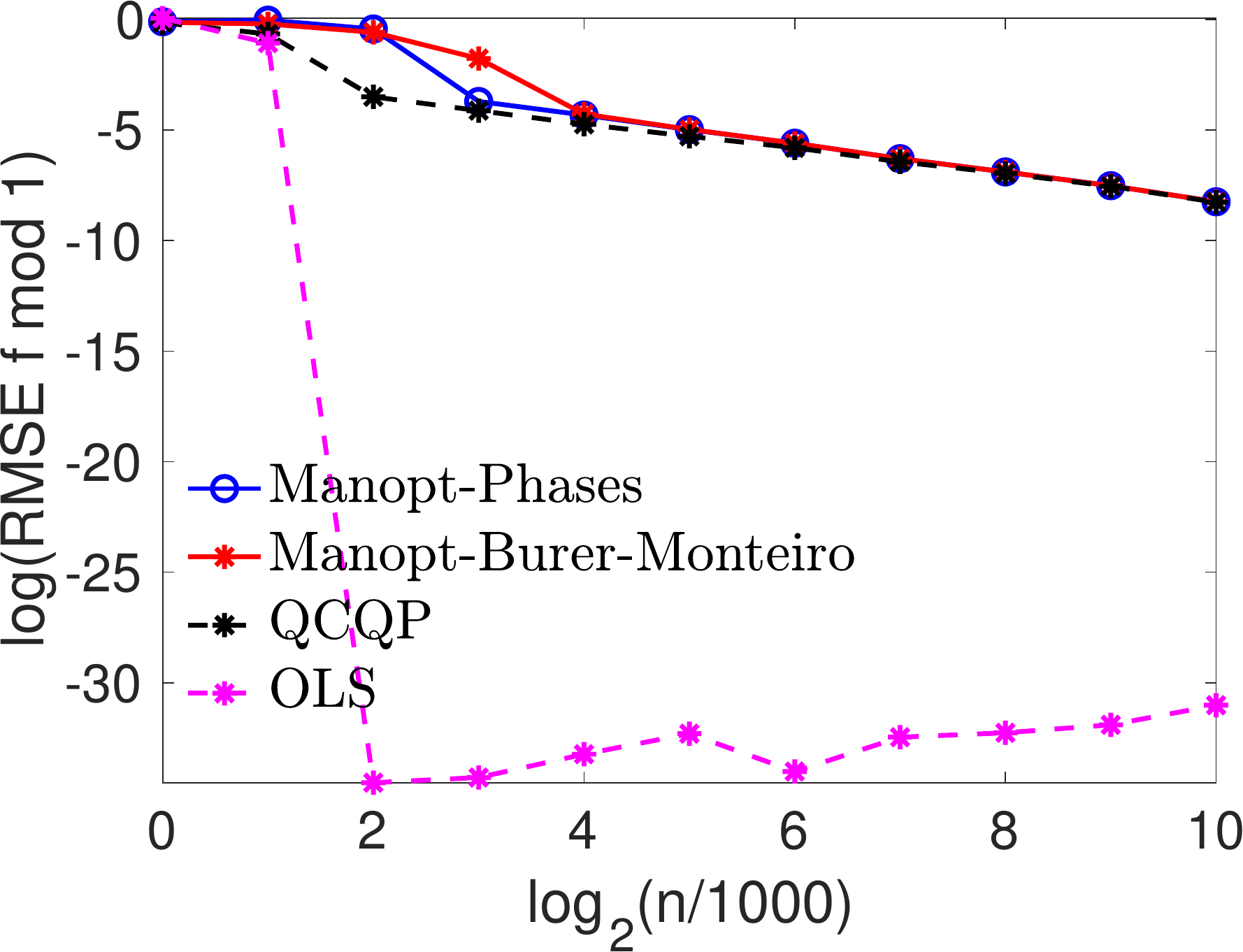} }

\subcaptionbox[]{    $ \sigma=0.1, \lambda = 0.01$
}[ 0.30\textwidth ]
{\includegraphics[width=0.30\textwidth] {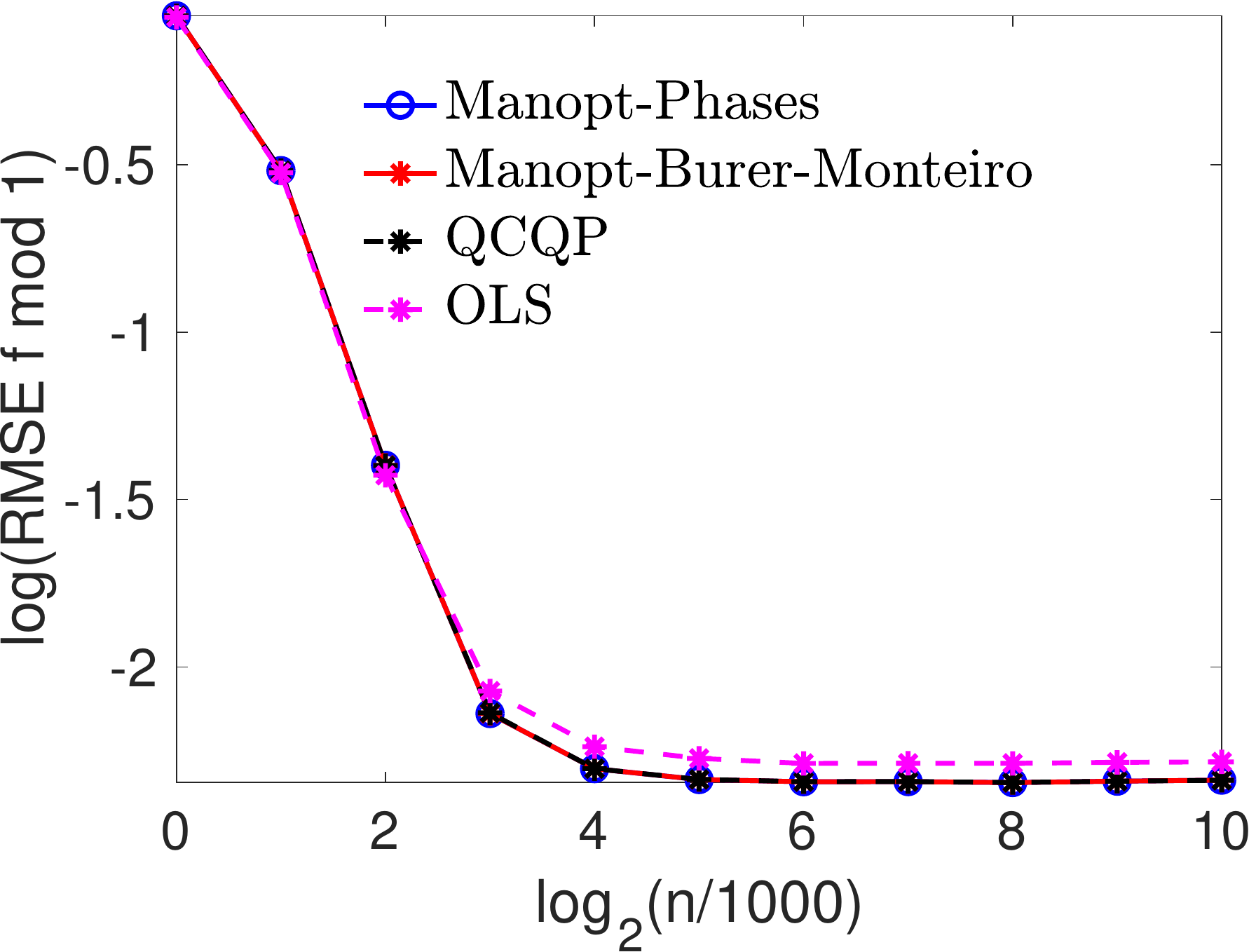} }
\subcaptionbox[]{   $ \sigma=0.1, \lambda = 0.1$
}[ 0.30\textwidth ]
{\includegraphics[width=0.30\textwidth] {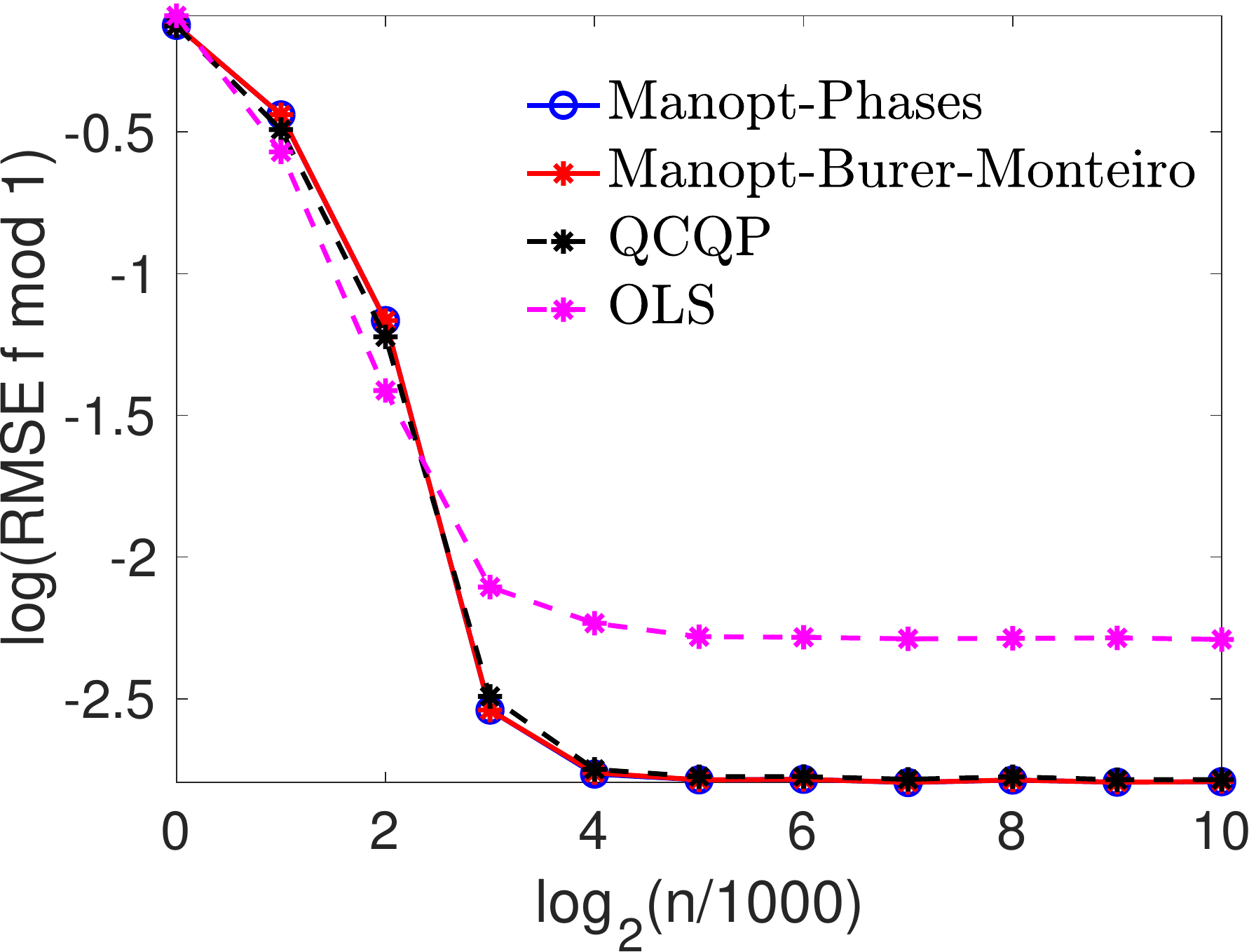} }
%
\subcaptionbox[]{  $ \sigma=0.1, \lambda = 1$
}[ 0.30\textwidth ]
{\includegraphics[width=0.30\textwidth] {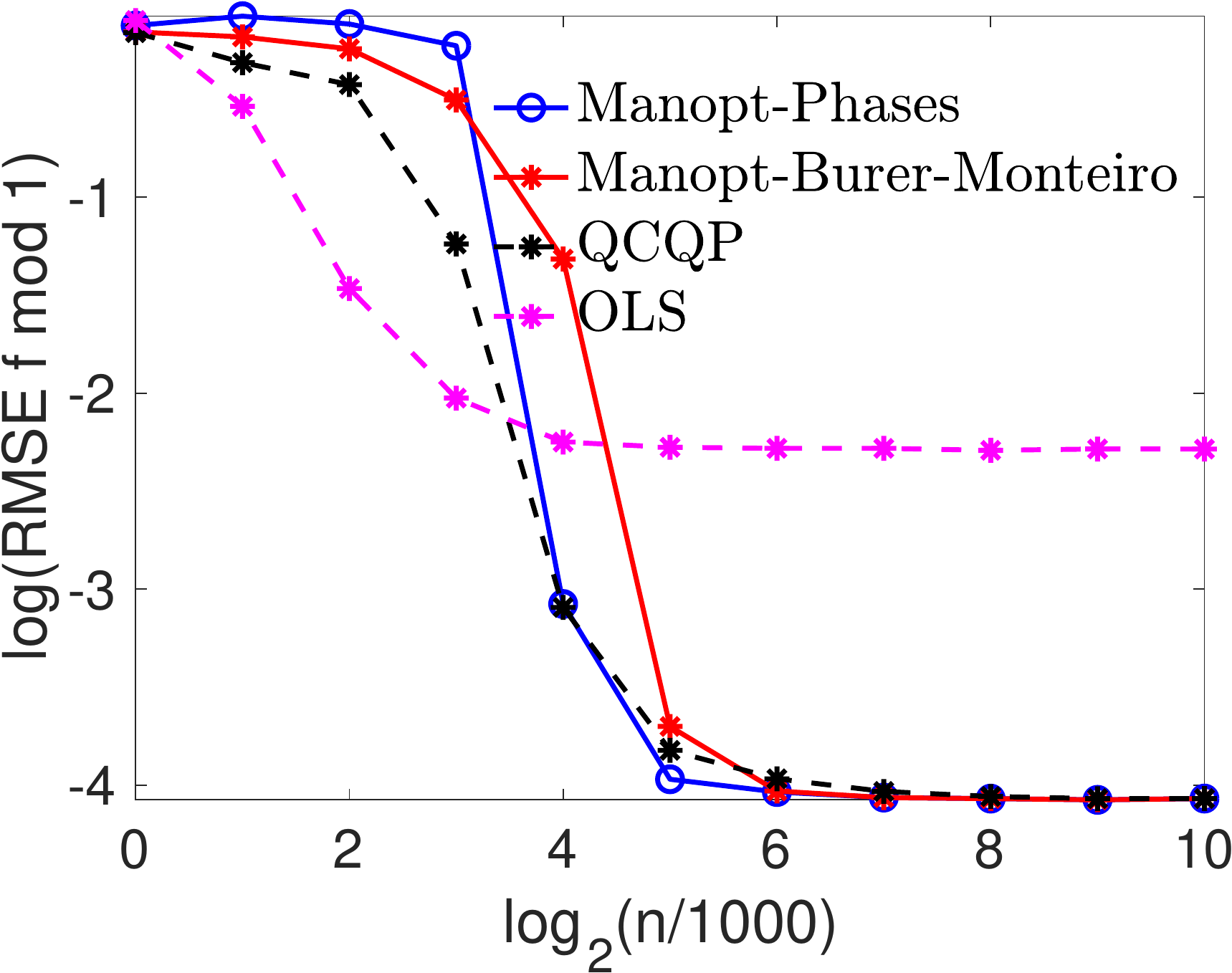} }
%
\captionsetup{width=0.95\linewidth}
\caption[Short Caption]{RMSE comparison for  $f$ mod 1 denoising, 
for the synthetic example  $f(x,y) = 6 x e^{- x^2 - y^2}$, with  $k=1$ (Chebychev distance), while varying $\sigma$ and $\lambda$. \textbf{QCQP} denotes Algorithm \ref{algo:two_stage_denoise} where the unwrapping stage is  performed via \textbf{OLS} \eqref{eq:ols_unwrap_lin_system}.
} 
\label{fig:MV_Manopts_Comp_fmod1}
\end{figure*}

 \begin{figure*}
\centering

\subcaptionbox[]{   $ \sigma=0, \lambda = 0.01$
}[ 0.30\textwidth ]
{\includegraphics[width=0.30\textwidth] {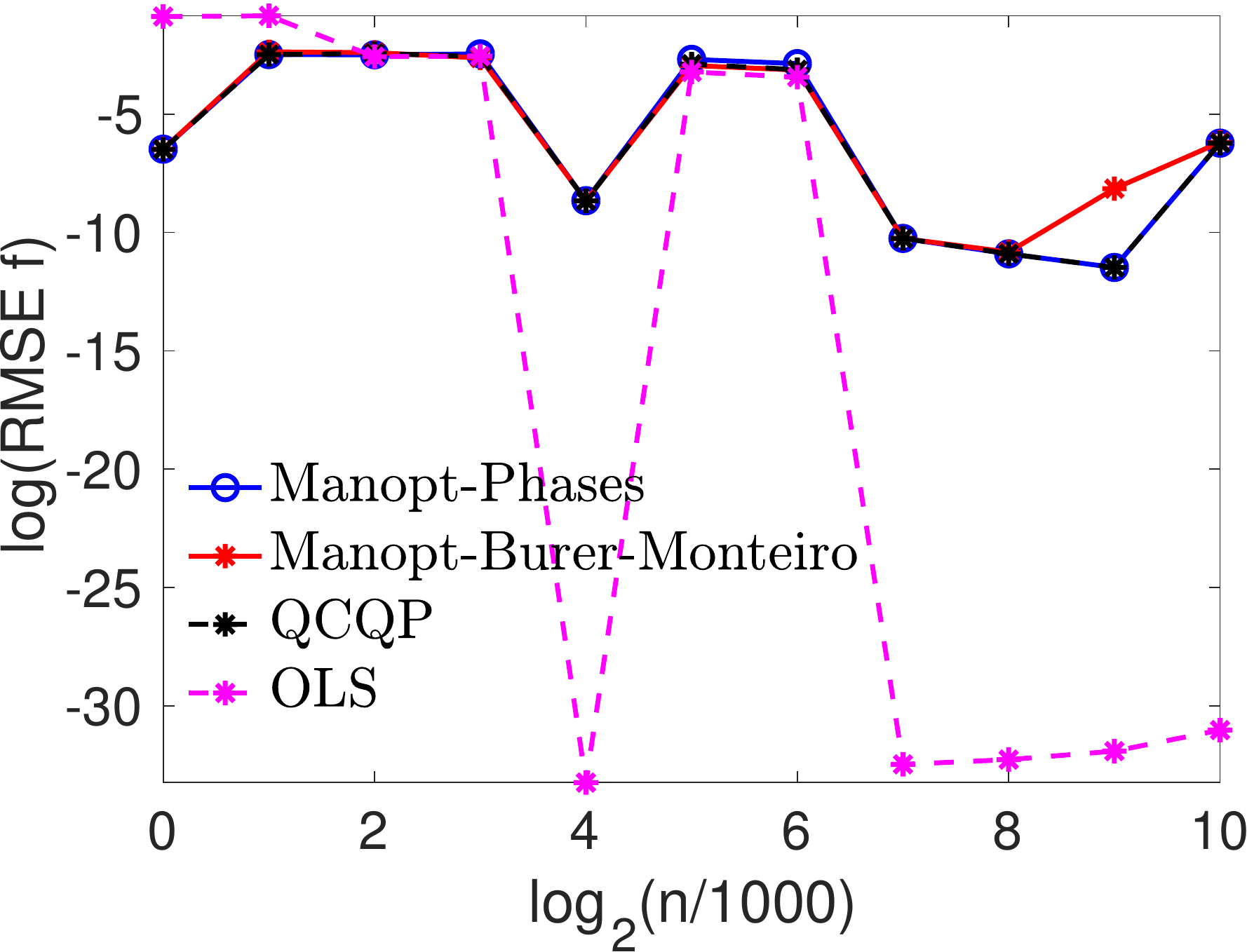} }
%
\subcaptionbox[]{   $ \sigma=0, \lambda = 0.1$
}[ 0.30\textwidth ]
{\includegraphics[width=0.30\textwidth] {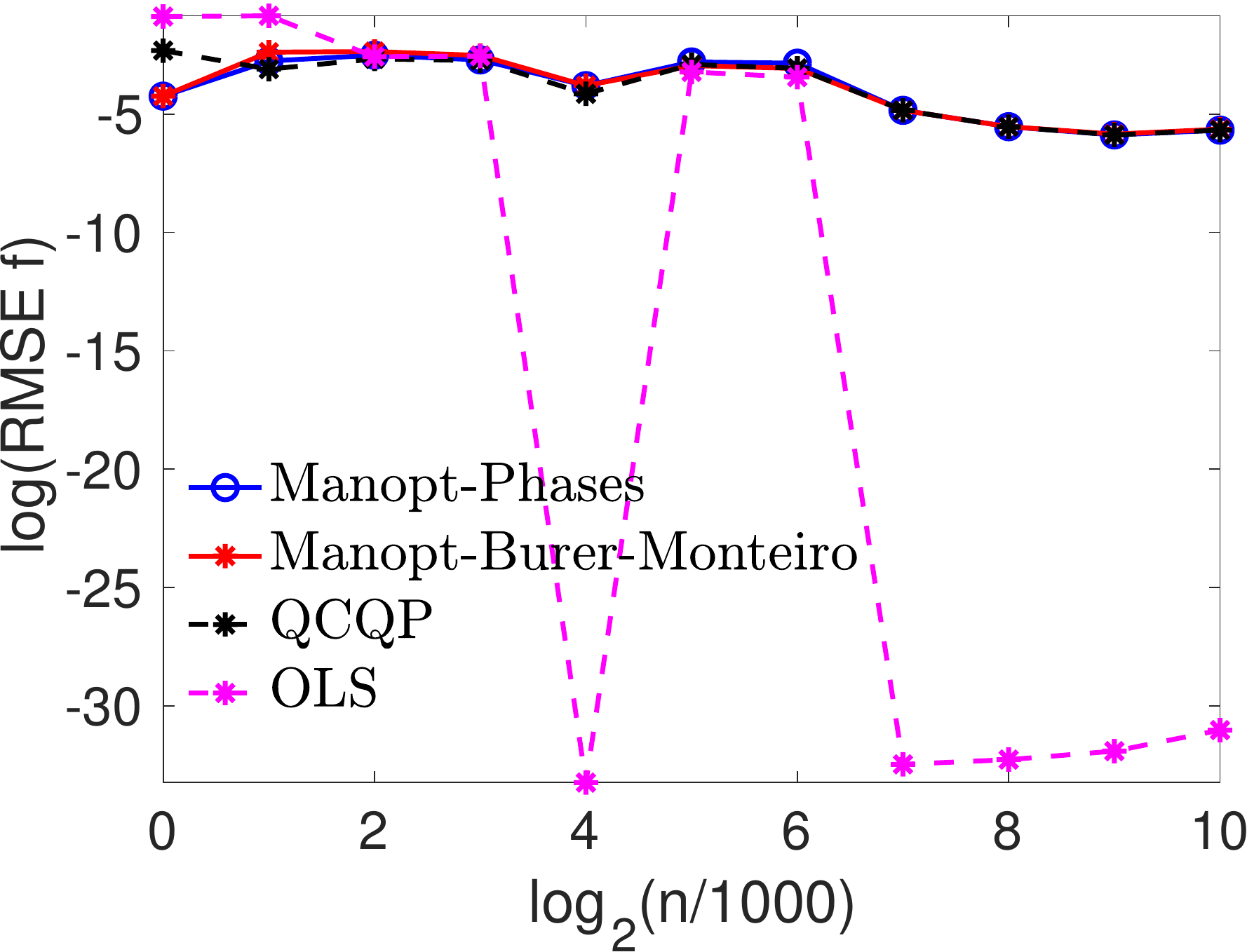} }
%
\subcaptionbox[]{   $ \sigma=0, \lambda = 1$
}[ 0.30\textwidth ]
{\includegraphics[width=0.30\textwidth] {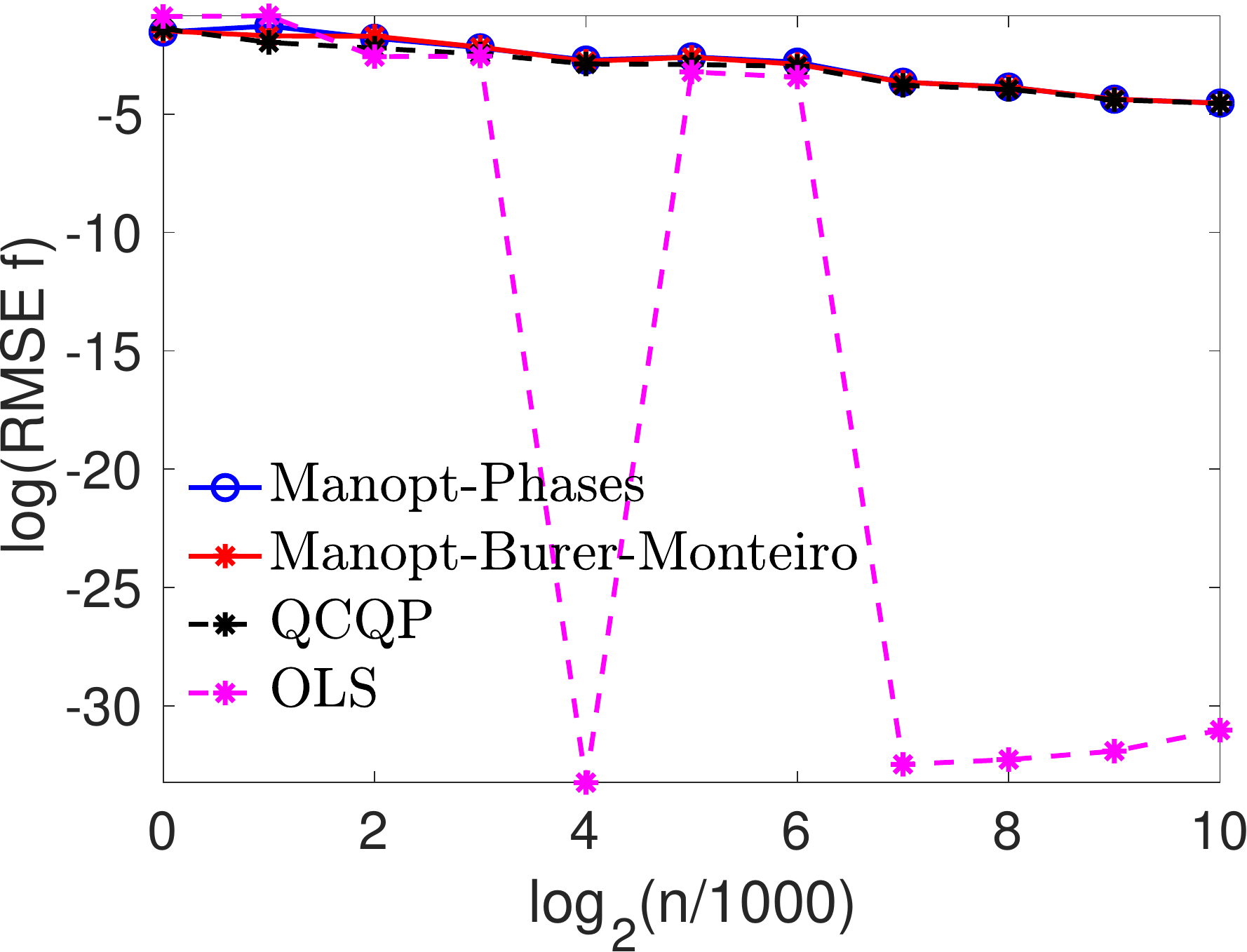} }

\subcaptionbox[]{    $ \sigma=0.1, \lambda = 0.01$
}[ 0.30\textwidth ]
{\includegraphics[width=0.30\textwidth] {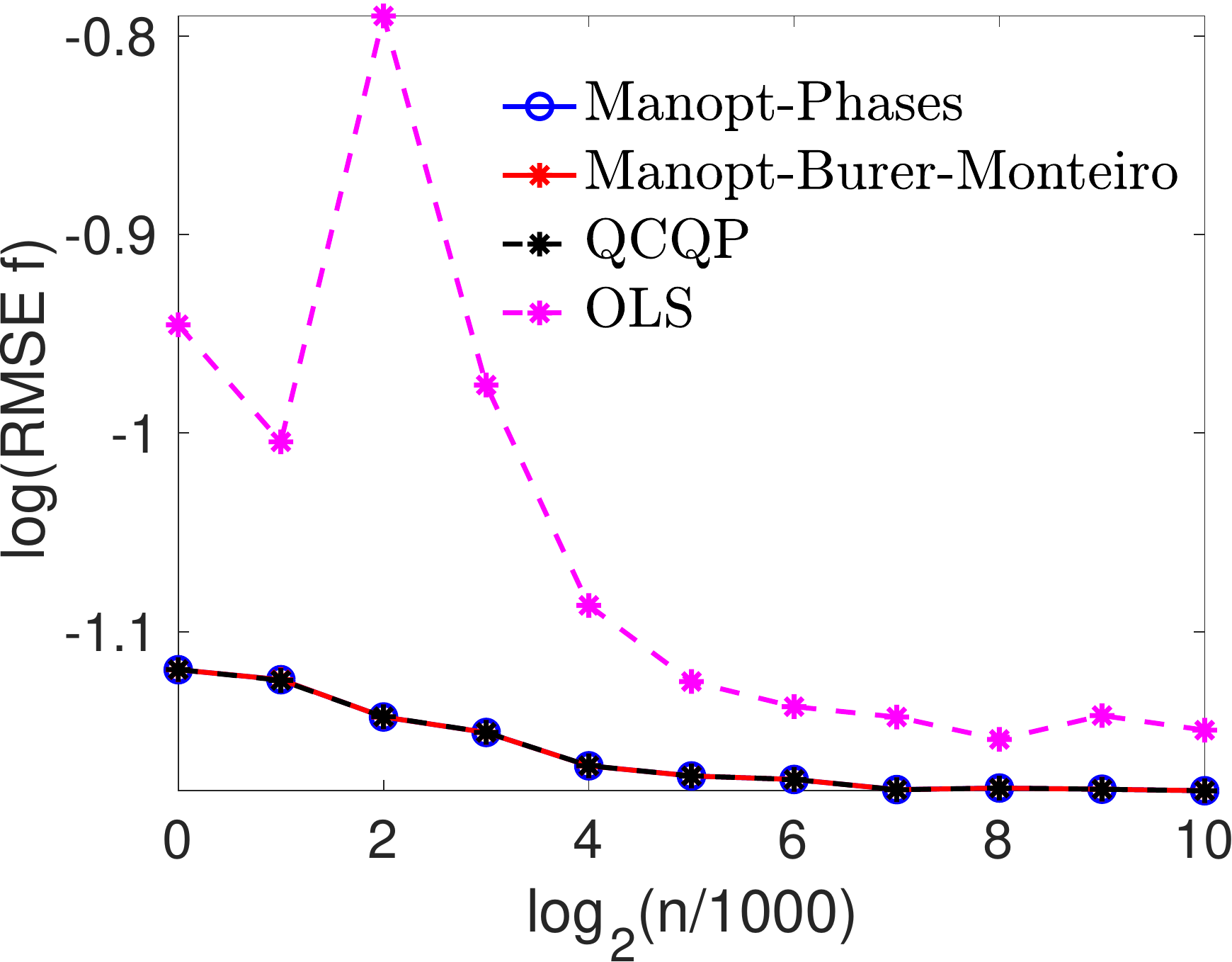} }
\subcaptionbox[]{   $ \sigma=0.1, \lambda = 0.1$
}[ 0.30\textwidth ]
{\includegraphics[width=0.30\textwidth] {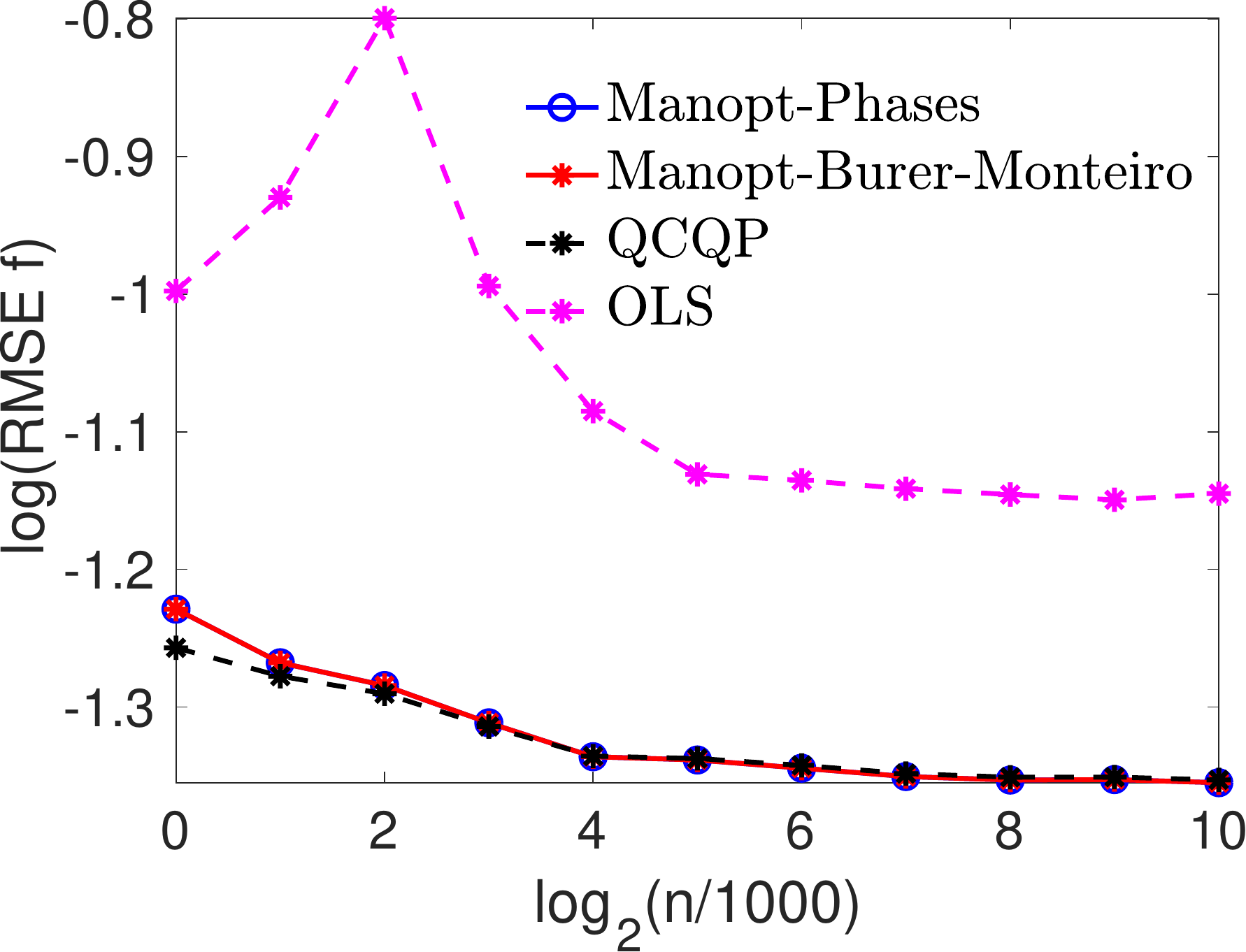} }
%
\subcaptionbox[]{  $ \sigma=0.1, \lambda = 1$
}[ 0.30\textwidth ]
{\includegraphics[width=0.30\textwidth] {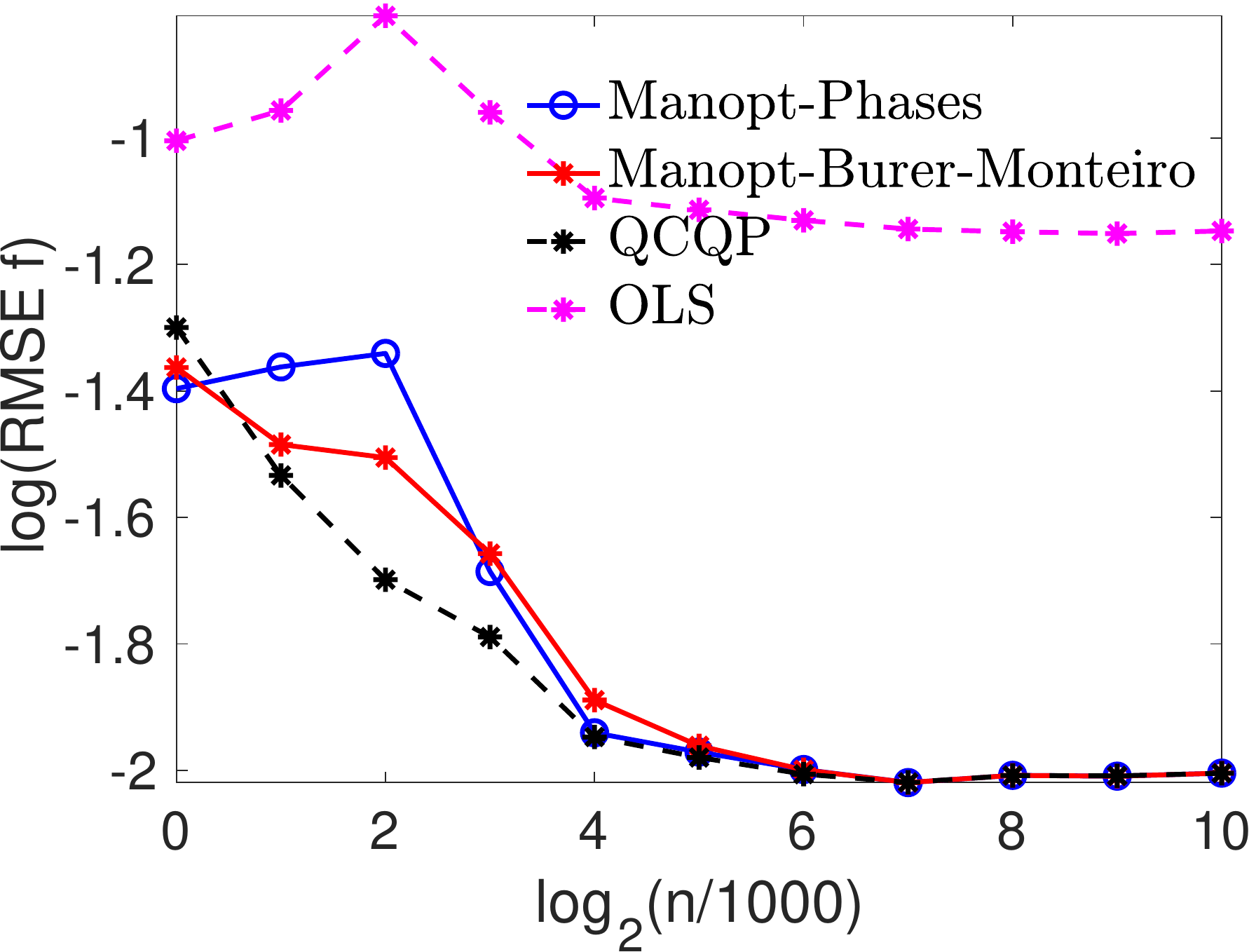} }
%
\captionsetup{width=0.95\linewidth}
\caption[Short Caption]{RMSE comparison 
for the final estimate $f$, for the synthetic example  $f(x,y) = 6 x e^{- x^2 - y^2}$, with  $k=1$ (Chebychev distance), across several noise levels $\sigma$ and values of $\lambda$. Here, \textbf{QCQP} denotes Algorithm \ref{algo:two_stage_denoise}, where  the unwrapping stage is  performed by  \textbf{OLS} \eqref{eq:ols_unwrap_lin_system}.
}
\label{fig:MV_Manopts_Comp_f}
\end{figure*}


\begin{figure*}
\centering
 
\subcaptionbox[]{  Noiseless function
}[ 0.24\textwidth ]
{\includegraphics[width=0.24\textwidth] {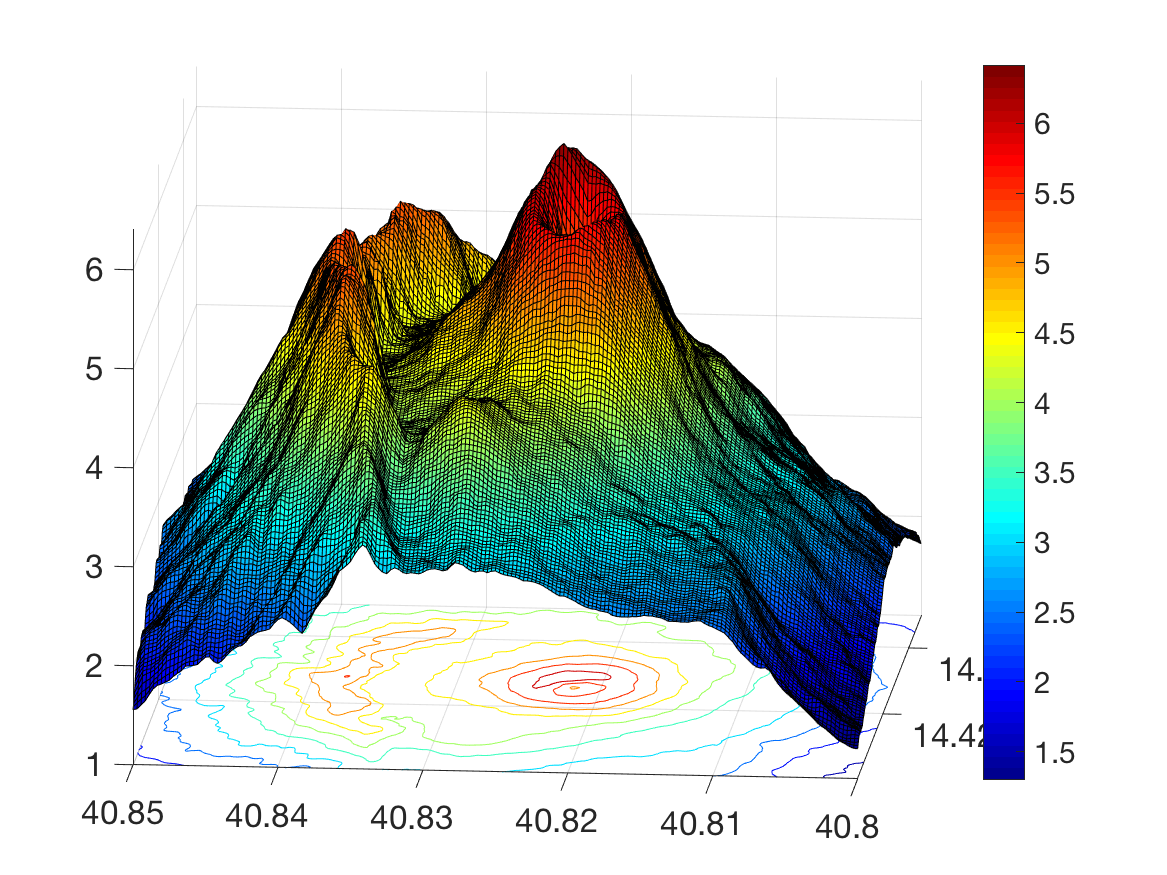} }
%
\subcaptionbox[]{  Noiseless function (depth)
}[ 0.24\textwidth ]
{\includegraphics[width=0.24\textwidth] {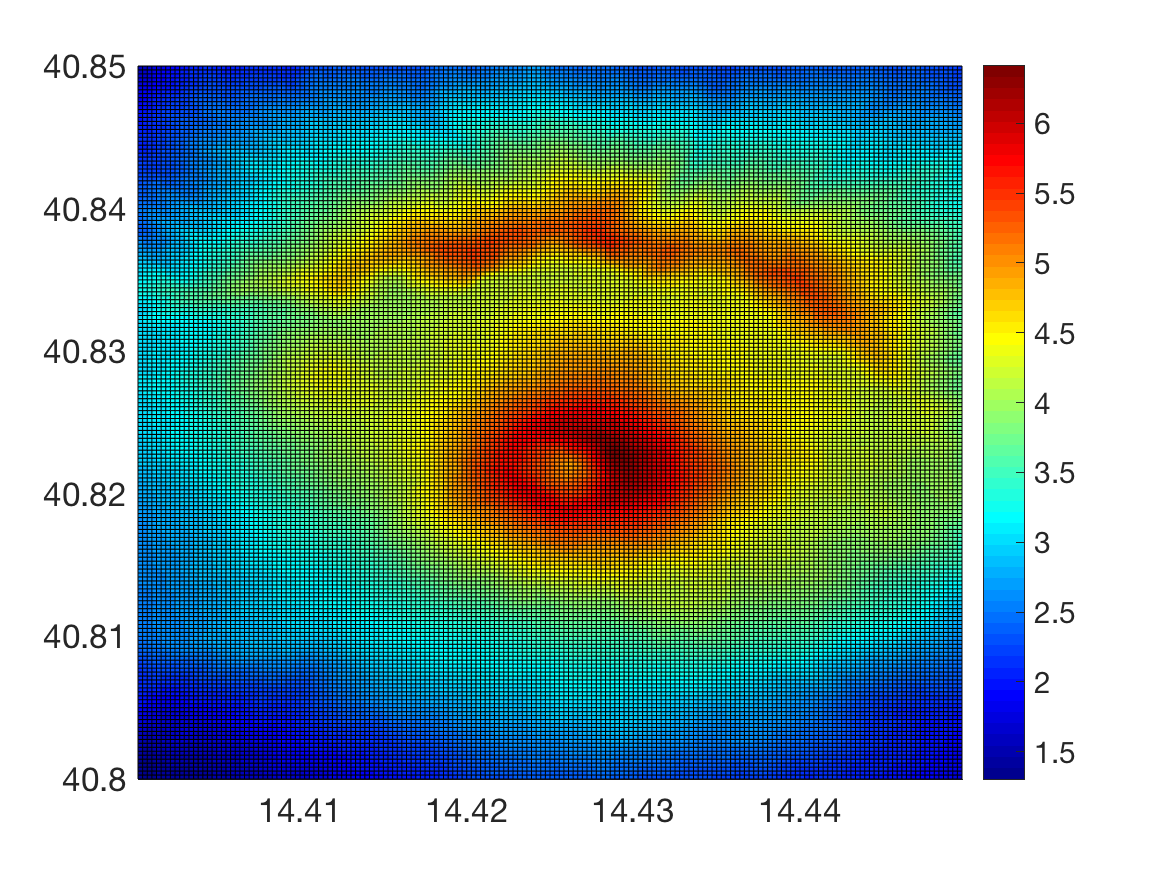} }
%
\subcaptionbox[]{   Clean $f$ mod 1
}[ 0.24\textwidth ]
{\includegraphics[width=0.24\textwidth] {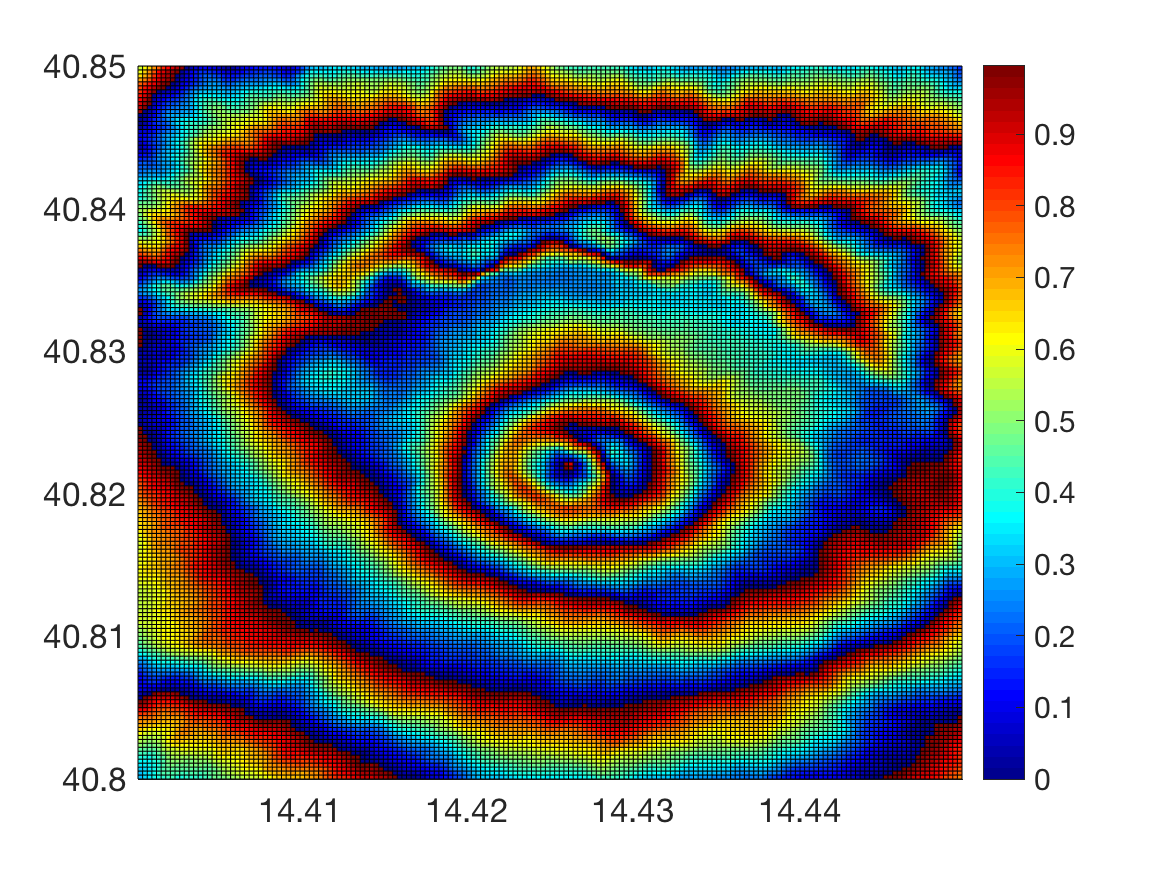} }
\subcaptionbox[]{    Noisy $f$ mod 1
}[ 0.24\textwidth ]
{\includegraphics[width=0.24\textwidth] {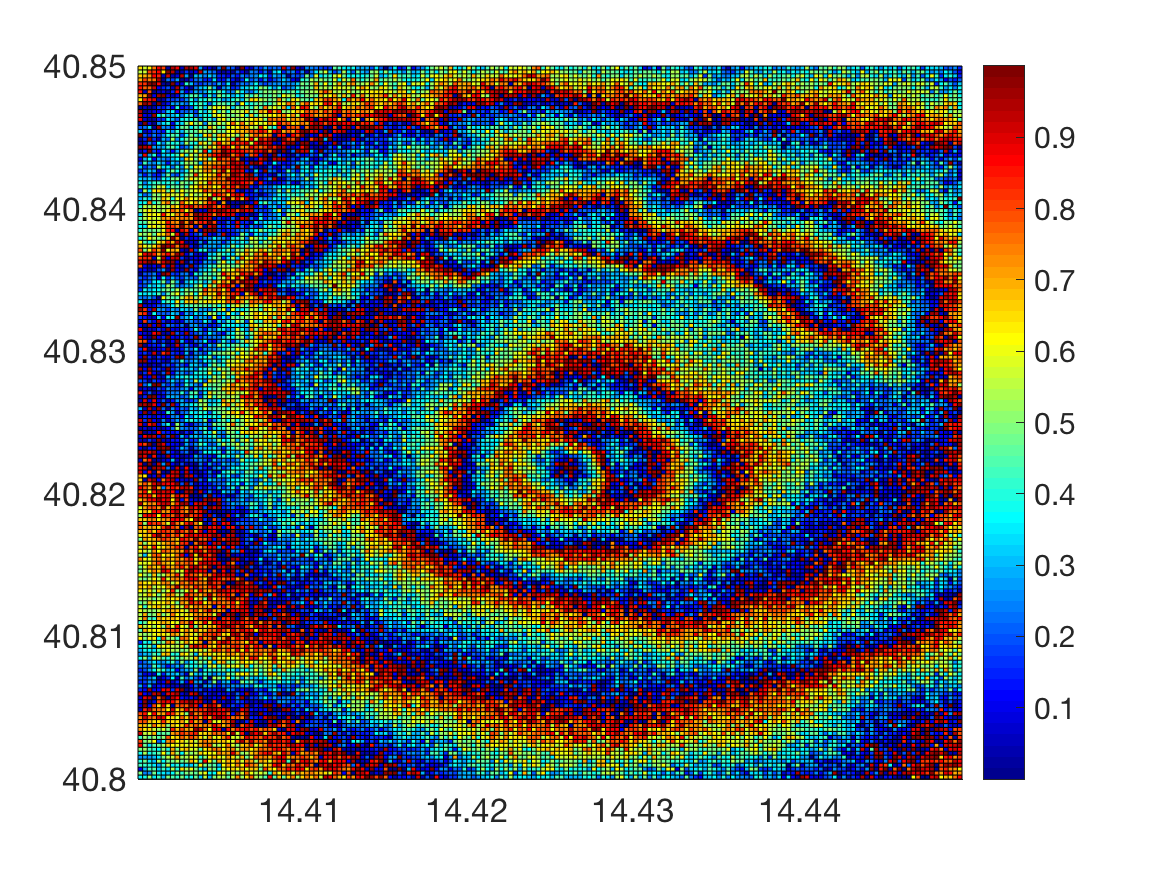} }

\subcaptionbox[]{      Denoised $f$ mod 1   \\ (RMSE=0.155)
}[ 0.24\textwidth ]
{\includegraphics[width=0.24\textwidth] {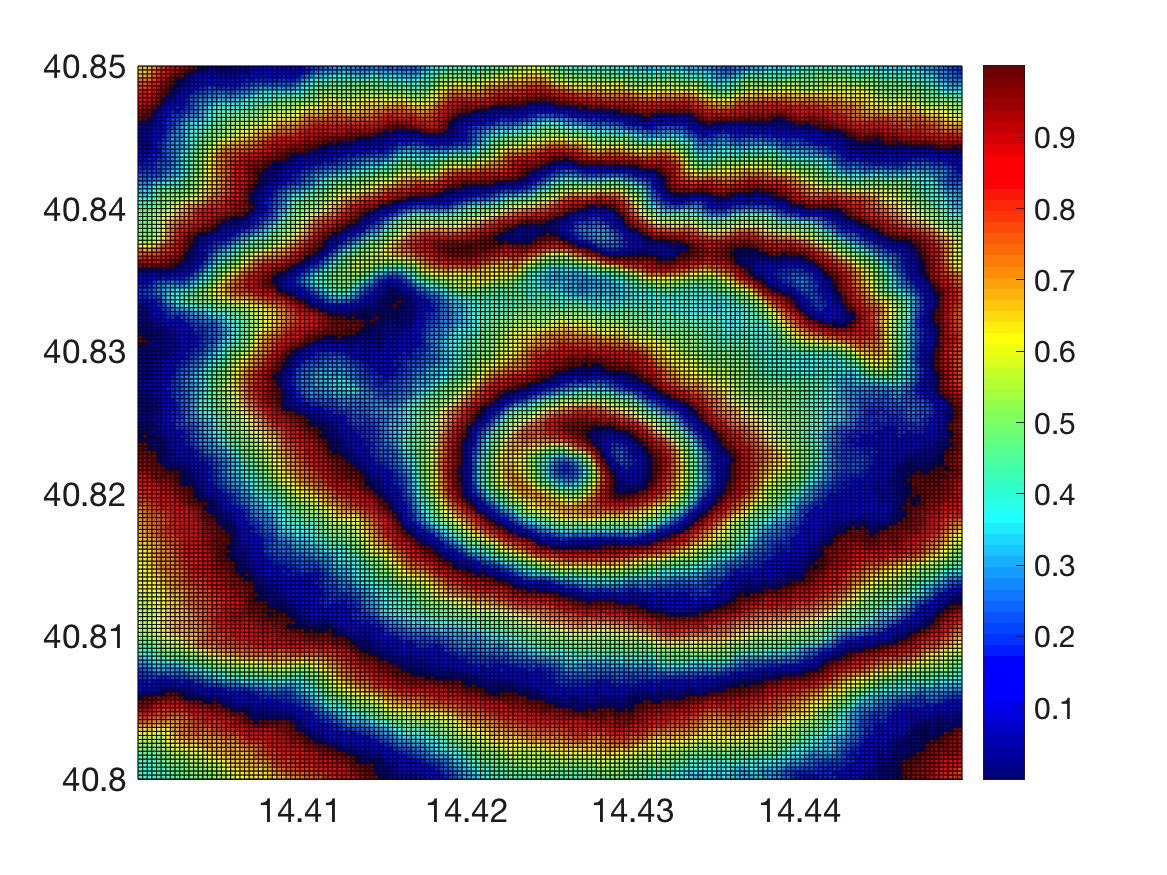} }
%
\subcaptionbox[]{   Denoised $f$   \\ (RMSE=0.062)
}[ 0.24\textwidth ]
{\includegraphics[width=0.24\textwidth] {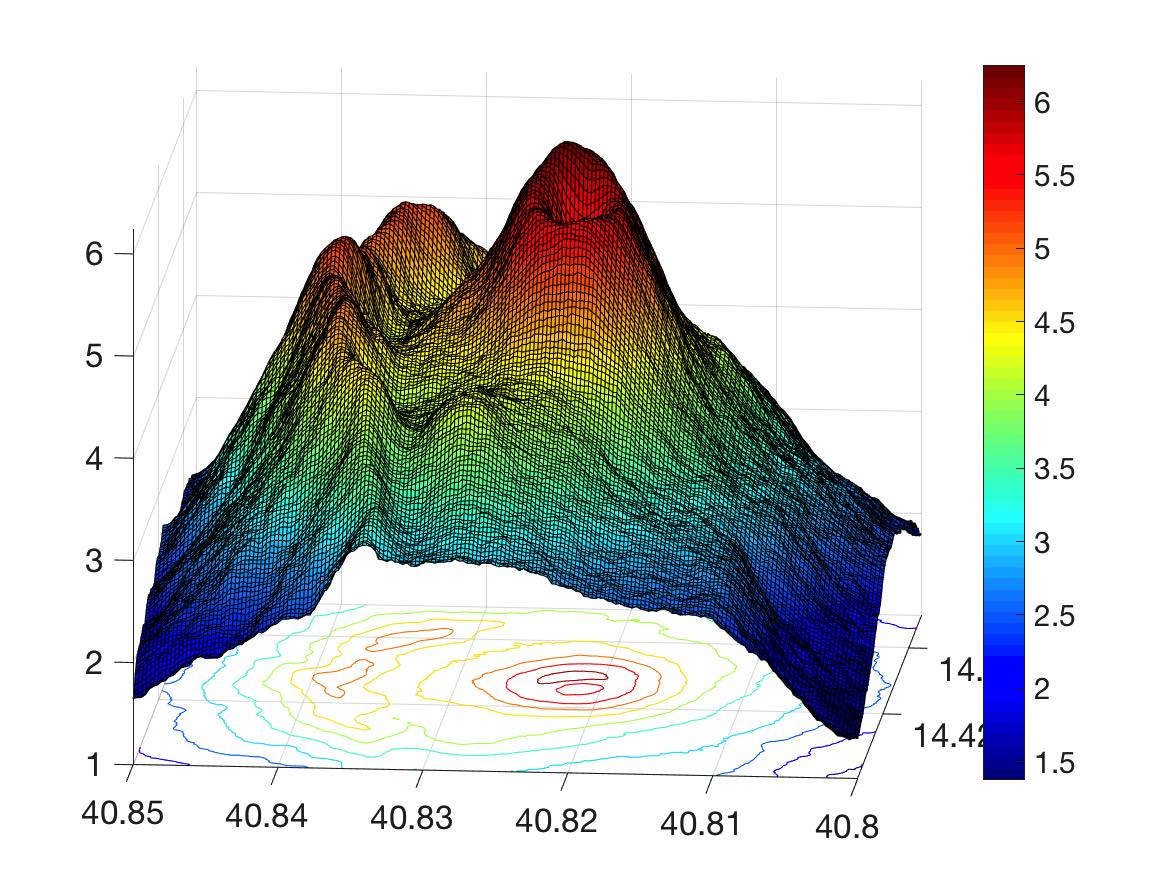} }
%
%
\subcaptionbox[]{   Denoised $f$ (depth)
}[ 0.24\textwidth ]
{\includegraphics[width=0.24\textwidth] {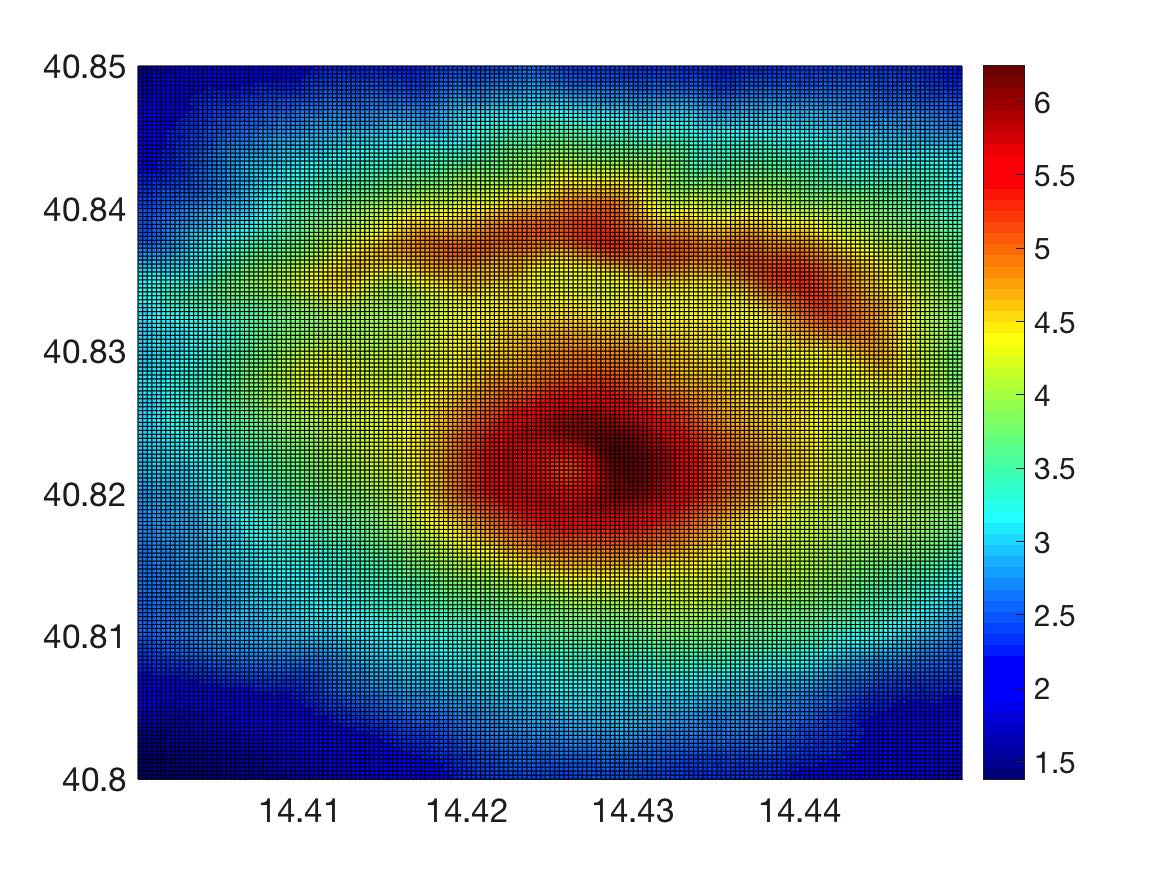} }
%
%
\subcaptionbox[]{  Error  $|f - \hat{f}|$ 
}[ 0.24\textwidth ]
{\includegraphics[width=0.24\textwidth] {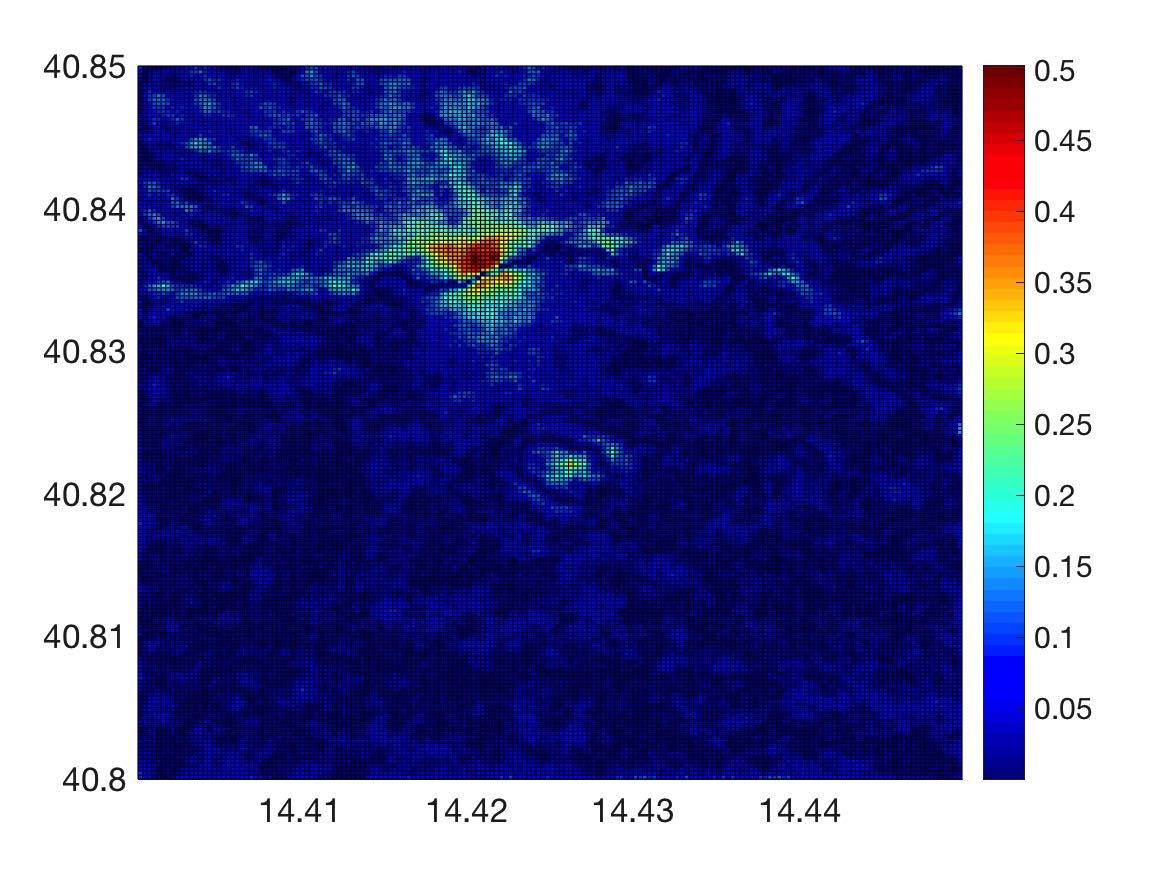} }
\captionsetup{width=0.95\linewidth}
\caption[Short Caption]{Elevation map of Mount Vesuvius with  $n = 32400$,  recovered by Manopt-Phases with  $k=1$ (Chebychev distance), $\lambda = 1$,  and noise level $\sigma=0.10$ under the Gaussian model.}
\label{fig:Vesuvius_HIGH_Noisy}
\end{figure*}

\begin{figure*}
\centering
 
\subcaptionbox[]{  Noiseless function
}[ 0.24\textwidth ]
{\includegraphics[width=0.24\textwidth] {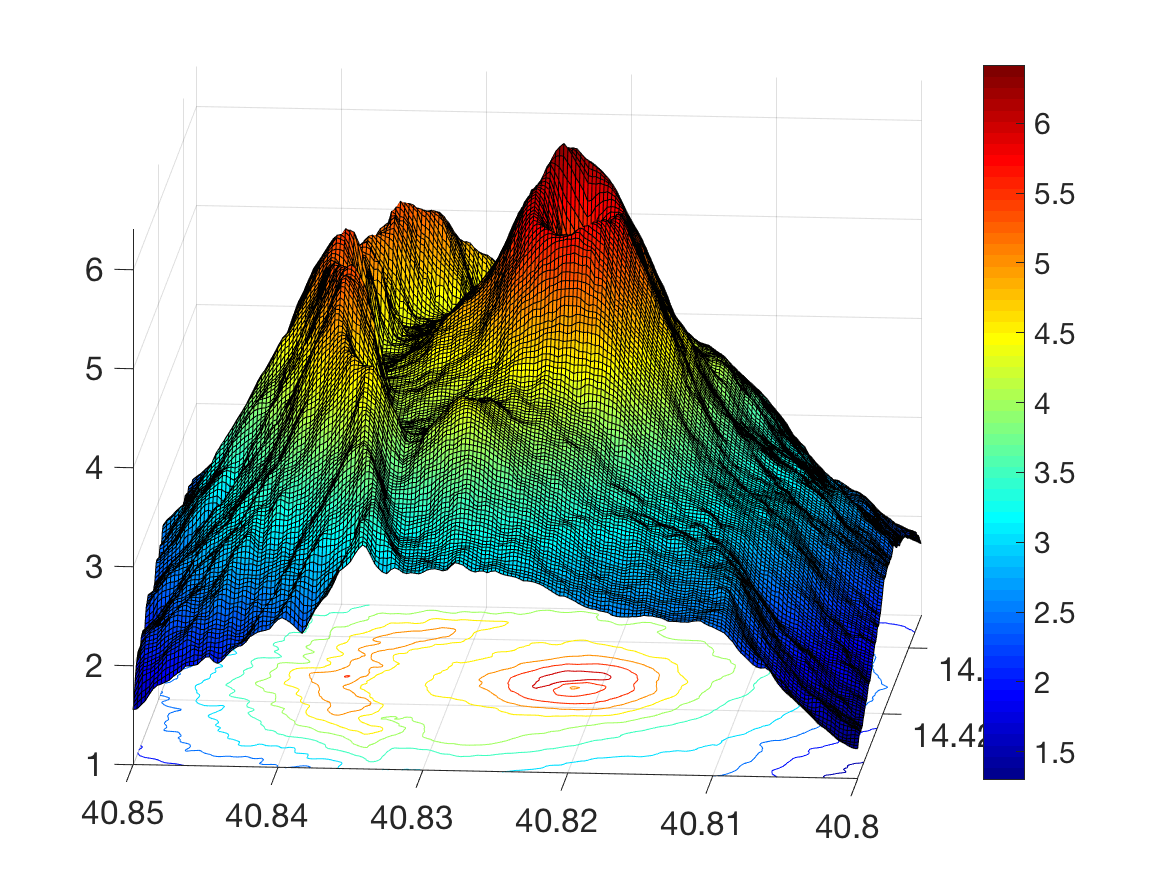} }
%
\subcaptionbox[]{  Noiseless function (depth)
}[ 0.24\textwidth ]
{\includegraphics[width=0.24\textwidth] {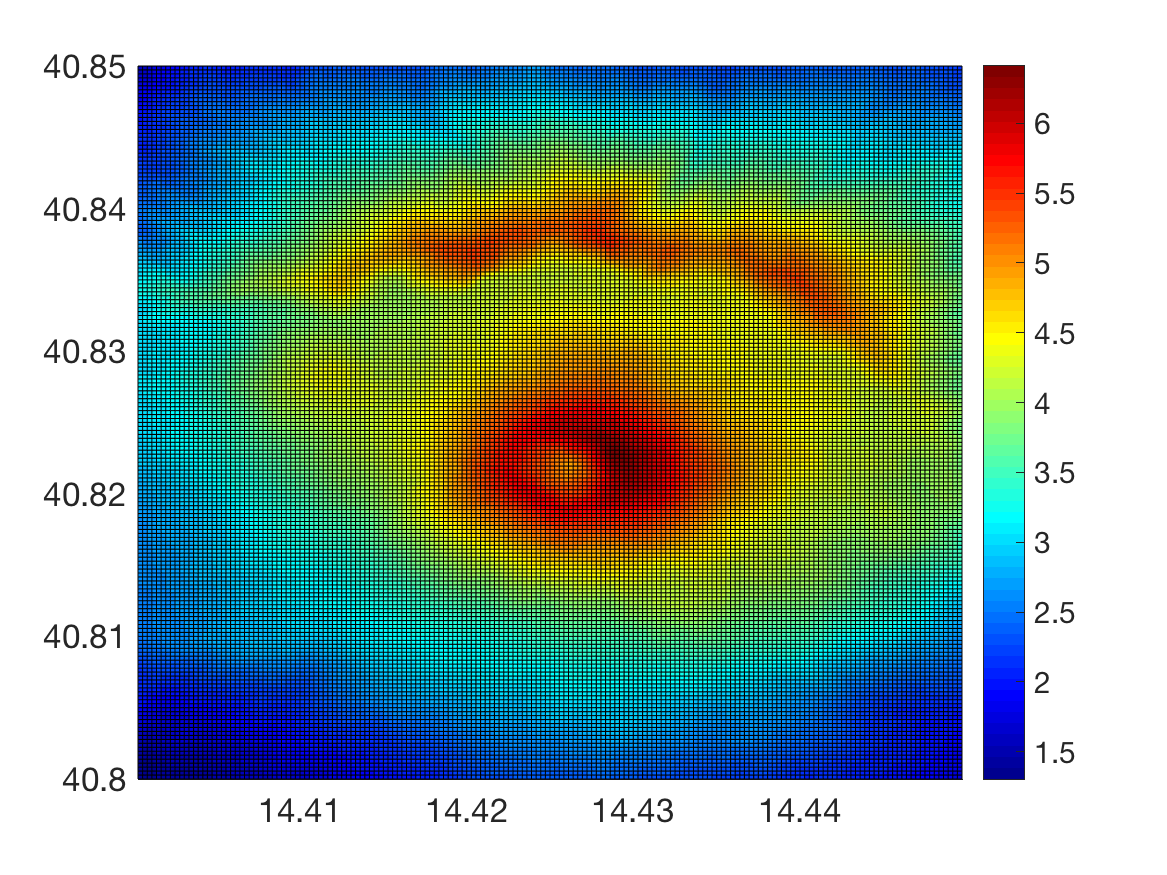} }
%
\subcaptionbox[]{   Clean $f$ mod 1
}[ 0.24\textwidth ]
{\includegraphics[width=0.24\textwidth] {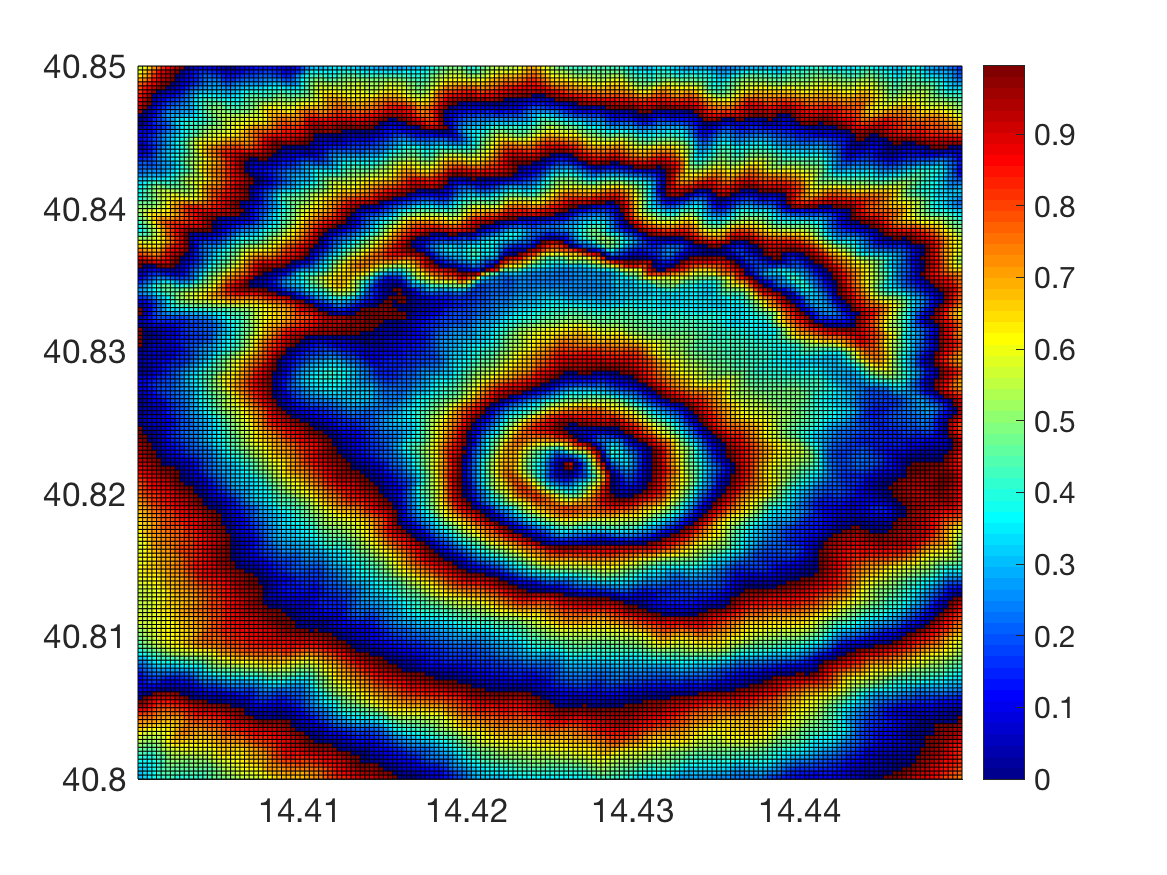} }
\subcaptionbox[]{    Noisy $f$ mod 1
}[ 0.24\textwidth ]
{\includegraphics[width=0.24\textwidth] {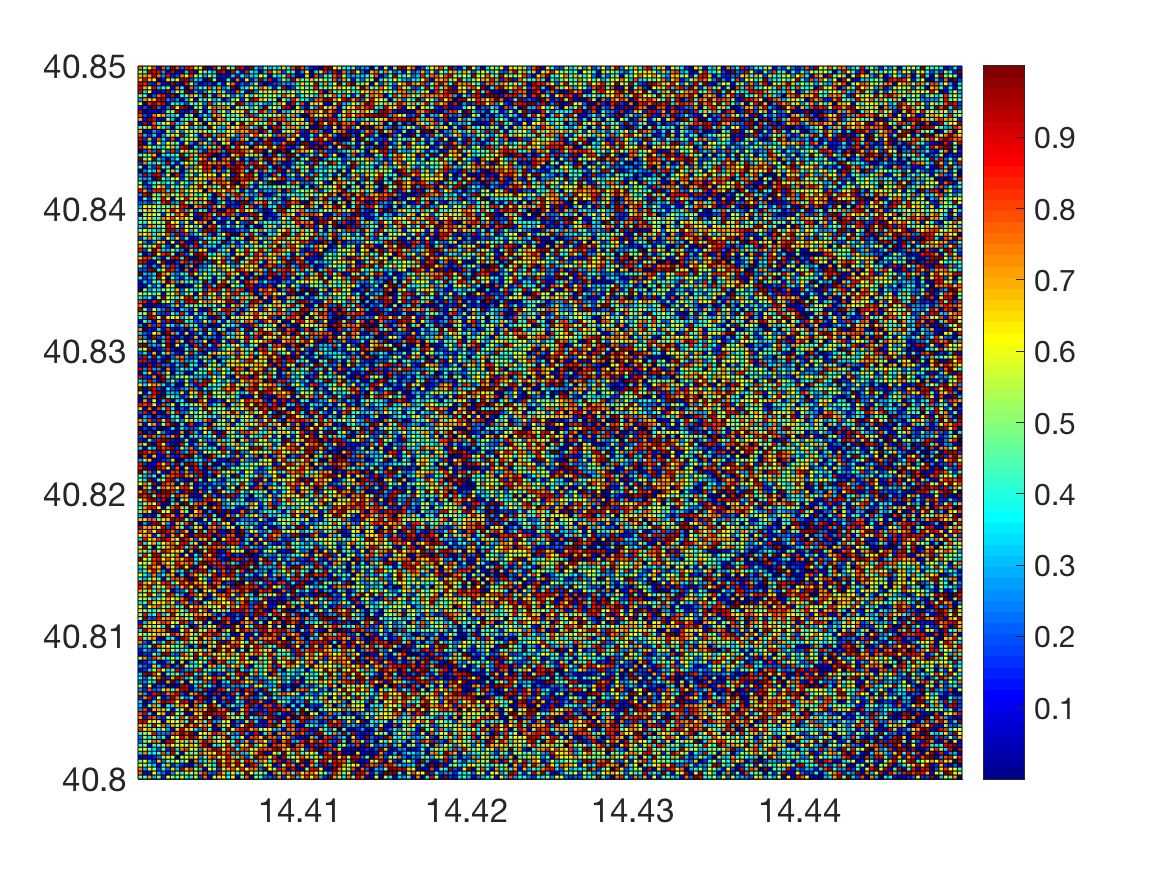} }

\subcaptionbox[]{      Denoised $f$ mod 1  \\    (RMSE = 0.23)     
}[ 0.24\textwidth ]
{\includegraphics[width=0.24\textwidth] {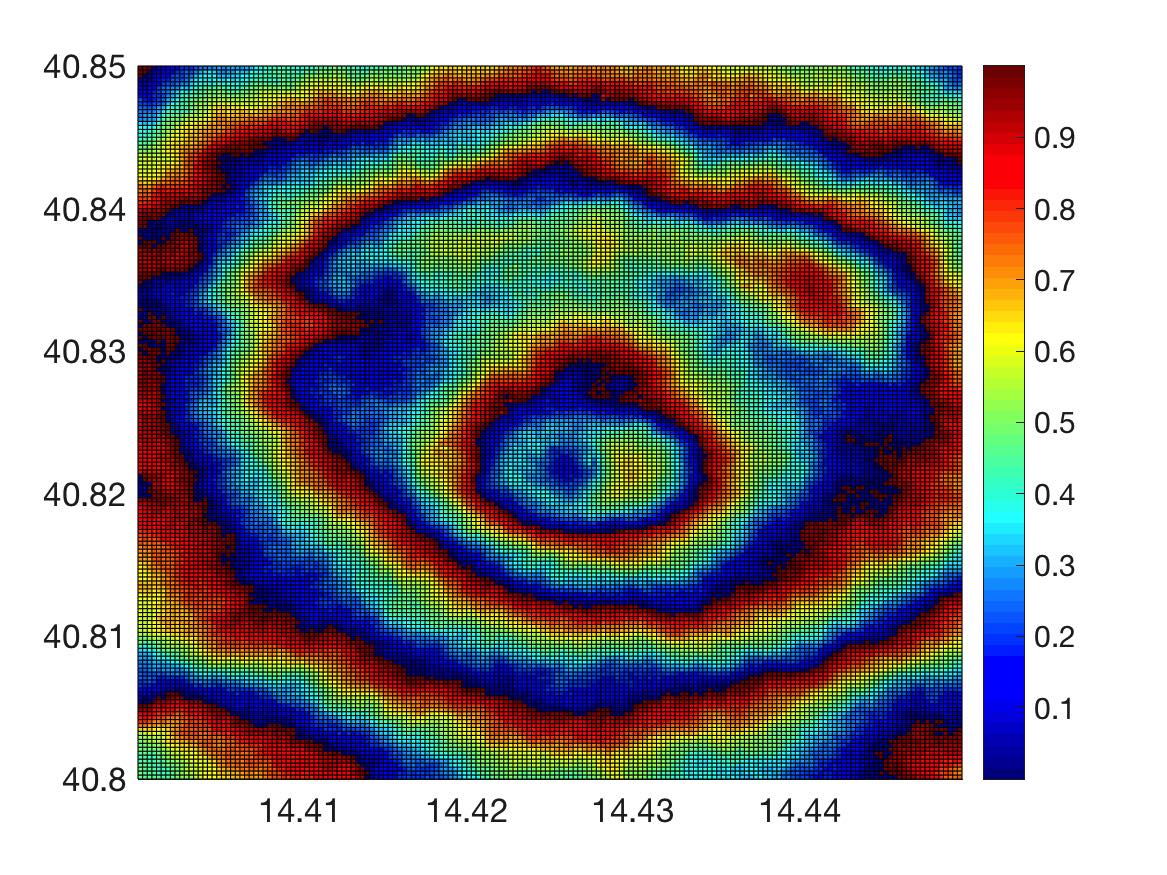} }
%
\subcaptionbox[]{   Denoised $f$   \\  (RMSE = 0.186)
}[ 0.24\textwidth ]
{\includegraphics[width=0.24\textwidth] {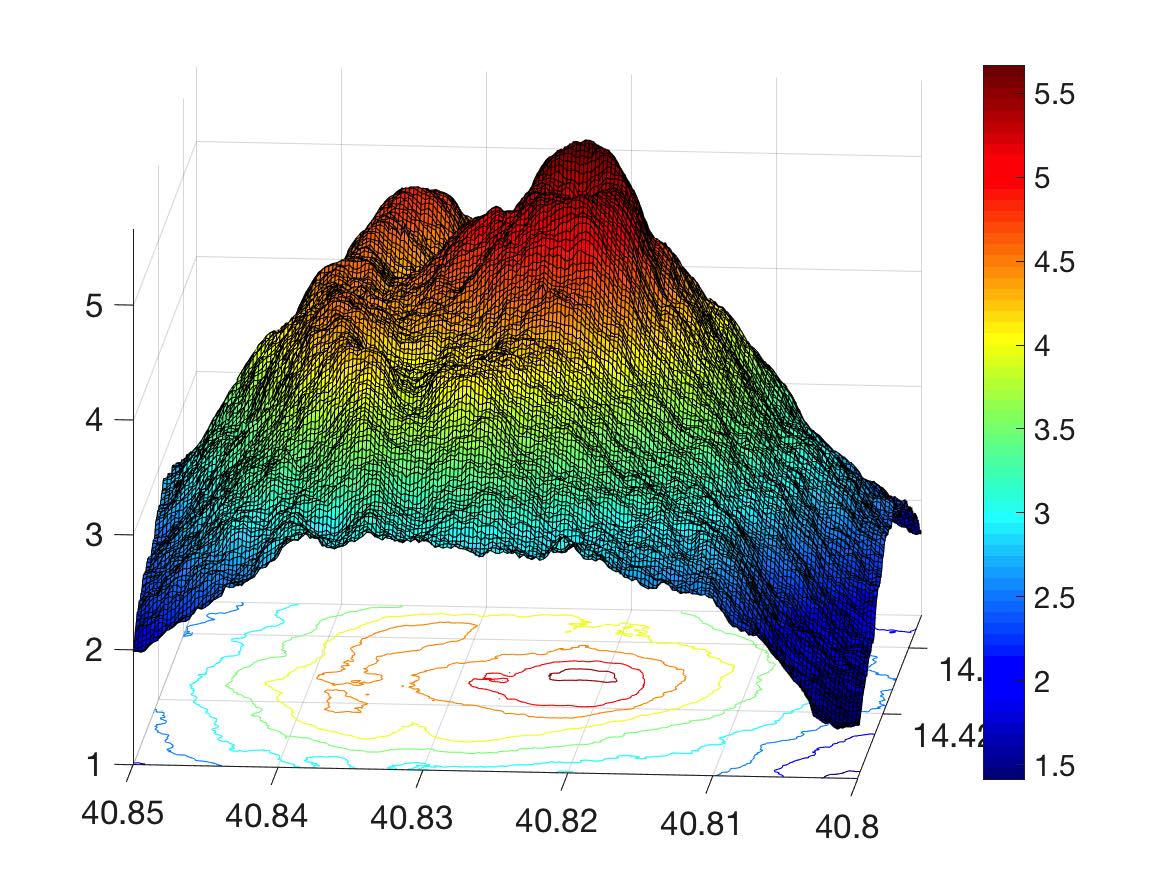} }
%
%
\subcaptionbox[]{   Denoised $f$ (depth)
}[ 0.24\textwidth ]
{\includegraphics[width=0.24\textwidth] {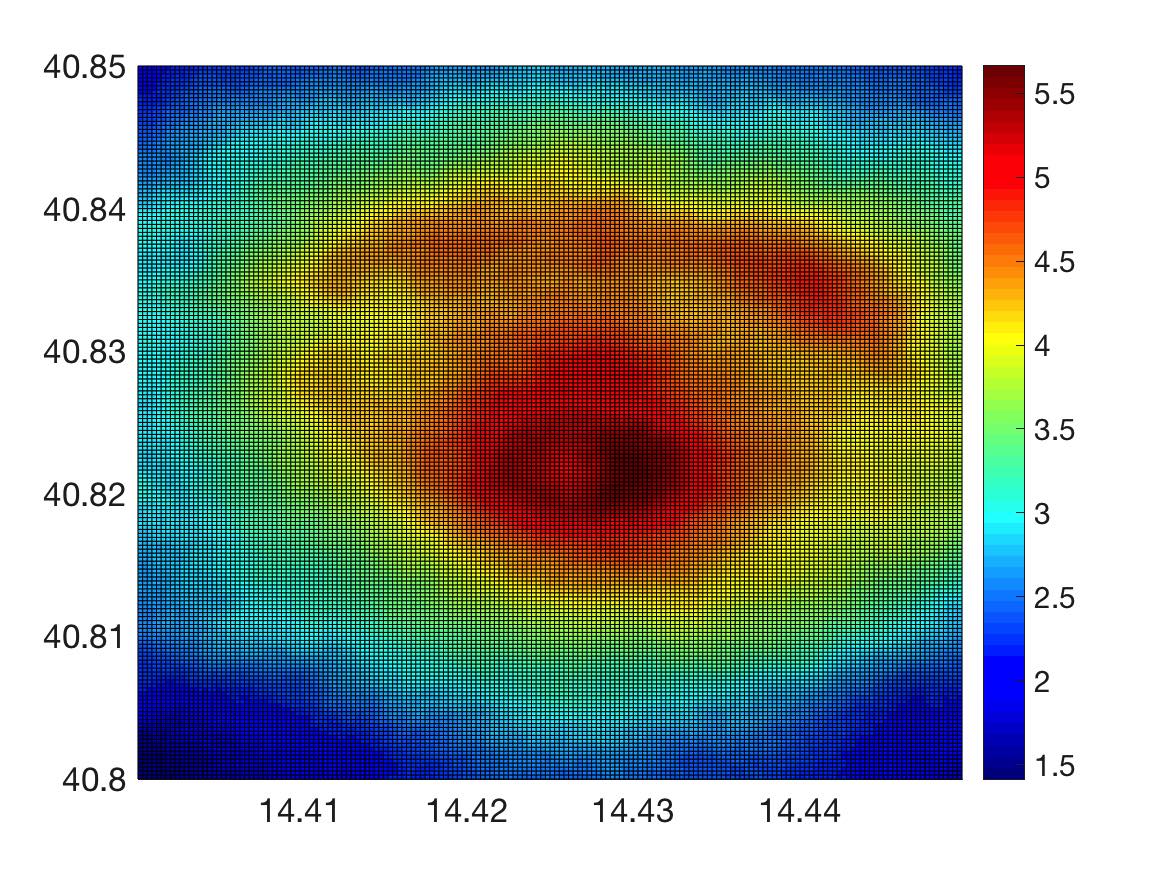} }
%
%
\subcaptionbox[]{  Error  $|f - \hat{f}|$ 
}[ 0.24\textwidth ]
{\includegraphics[width=0.24\textwidth] {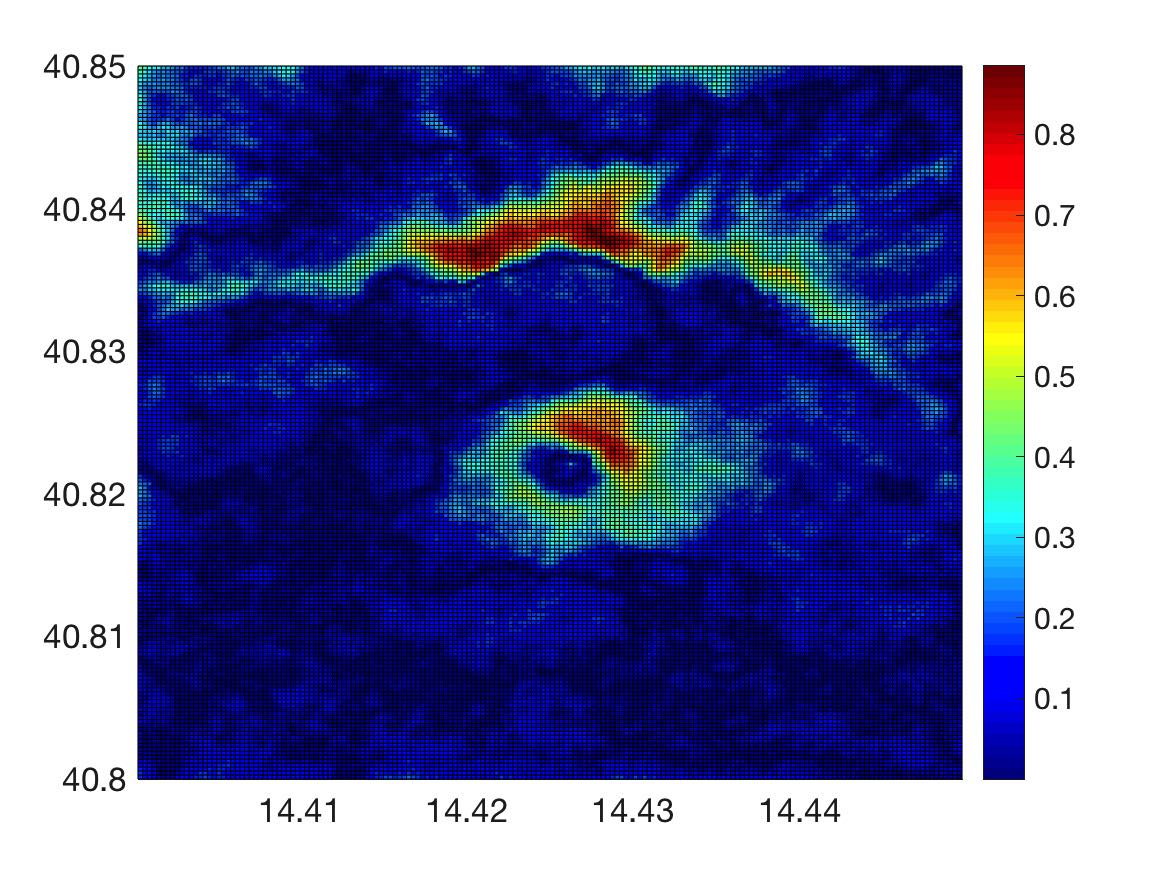} }
\captionsetup{width=0.95\linewidth}
\caption[Short Caption]{Elevation map of Mount Vesuvius with  $n = 32400$, $k=1$ (Chebychev distance), $\lambda = 1$,  and noise level $\sigma=0.25$ under the Gaussian model, as recovered by Manopt-Phases.} 
\label{fig:Vesuvius_HIGH_VeryNoisy}
\end{figure*}

\begin{figure*}
\centering
 
\subcaptionbox[]{  Noiseless function
}[ 0.24\textwidth ]
{\includegraphics[width=0.24\textwidth] {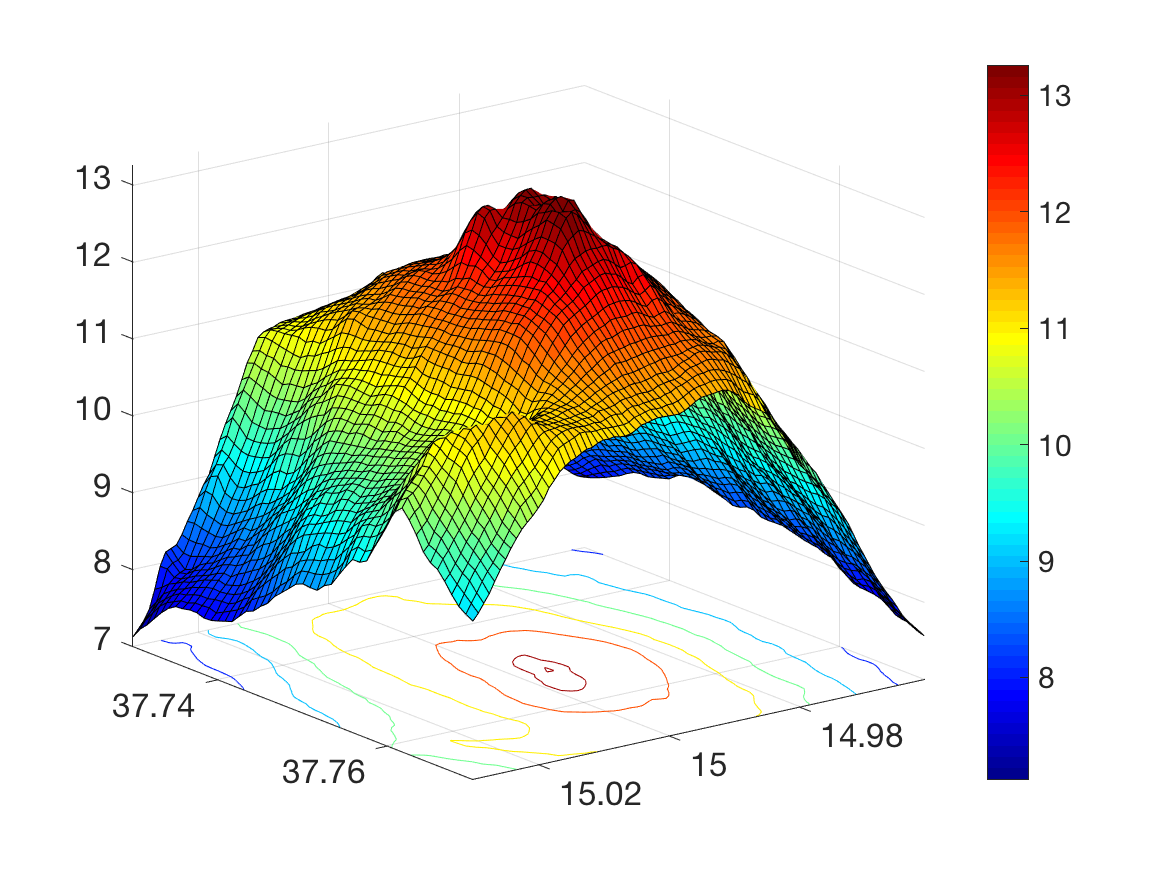} }
%
\subcaptionbox[]{  Noiseless function (depth)
}[ 0.24\textwidth ]
{\includegraphics[width=0.24\textwidth] {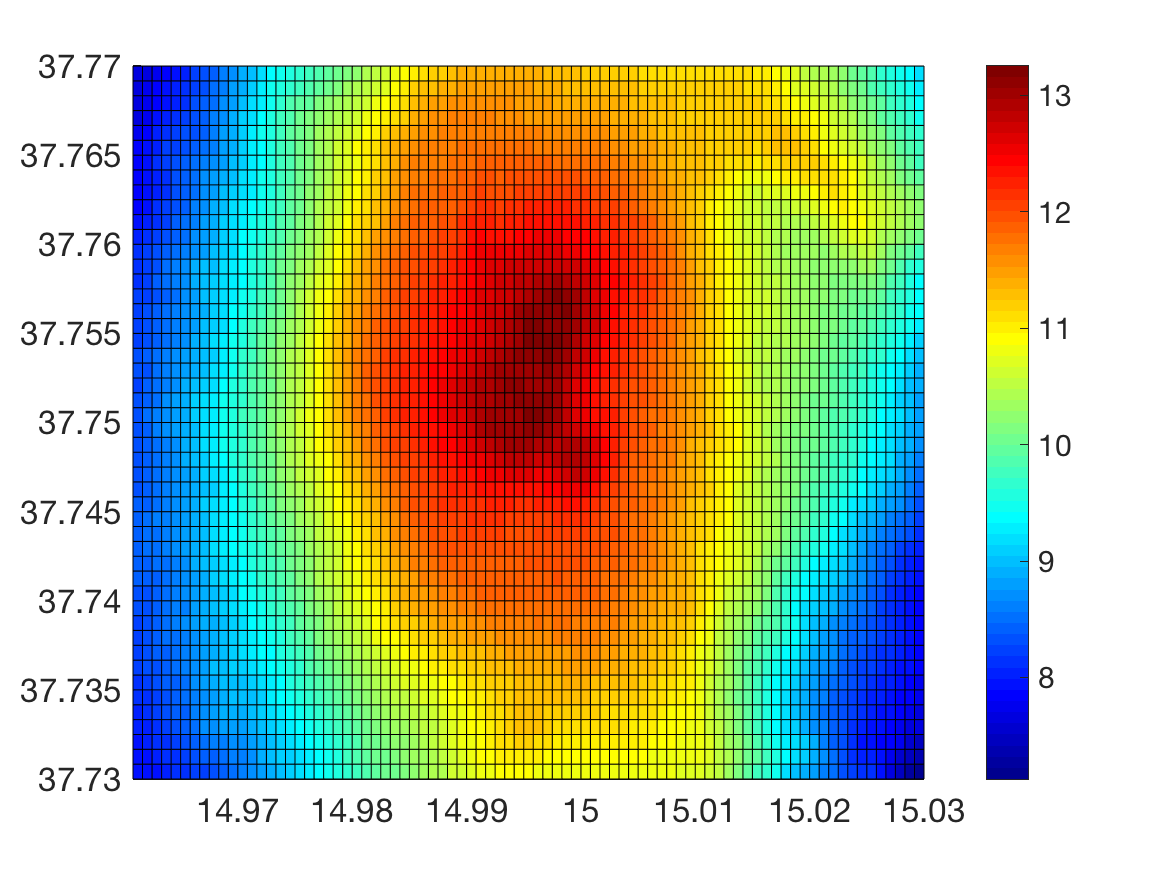} }
%
\subcaptionbox[]{   Clean $f$ mod 1
}[ 0.24\textwidth ]
{\includegraphics[width=0.24\textwidth] {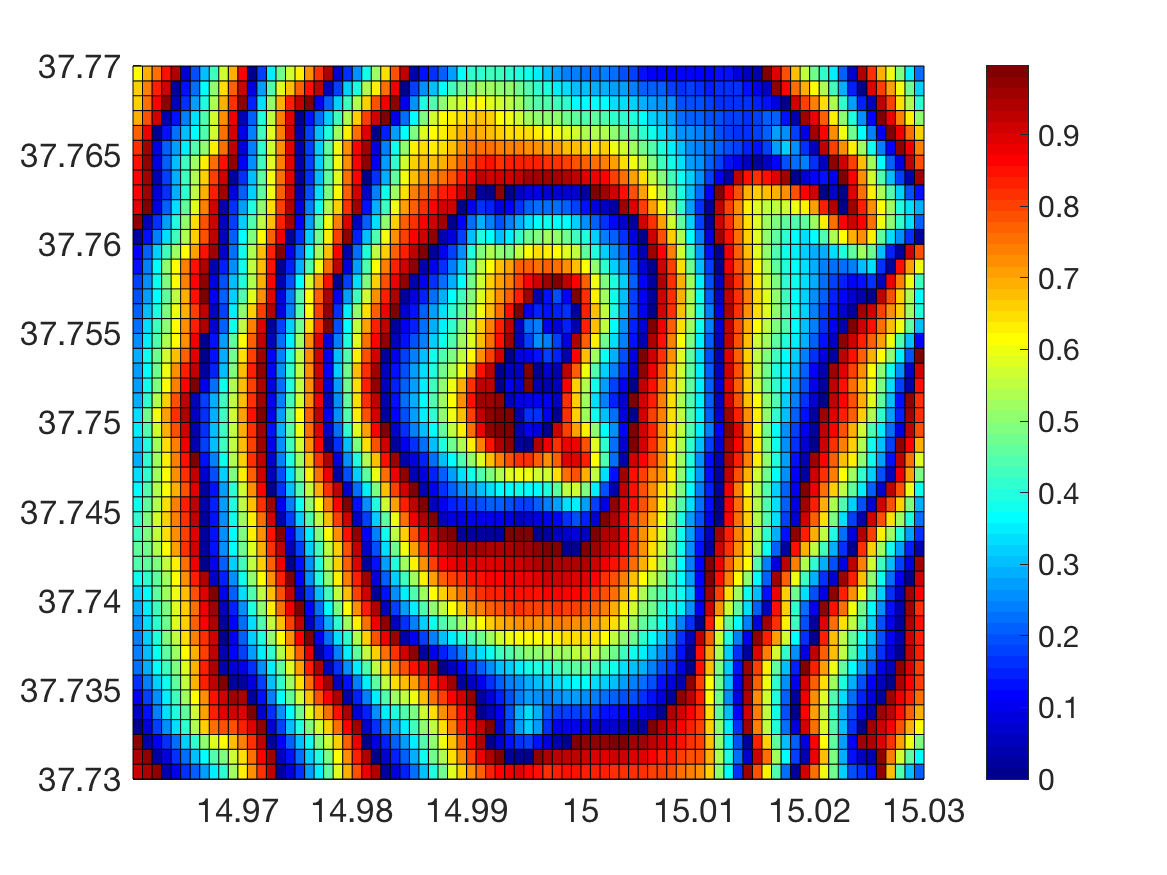} }
\subcaptionbox[]{    Noisy $f$ mod 1
}[ 0.24\textwidth ]
{\includegraphics[width=0.24\textwidth] {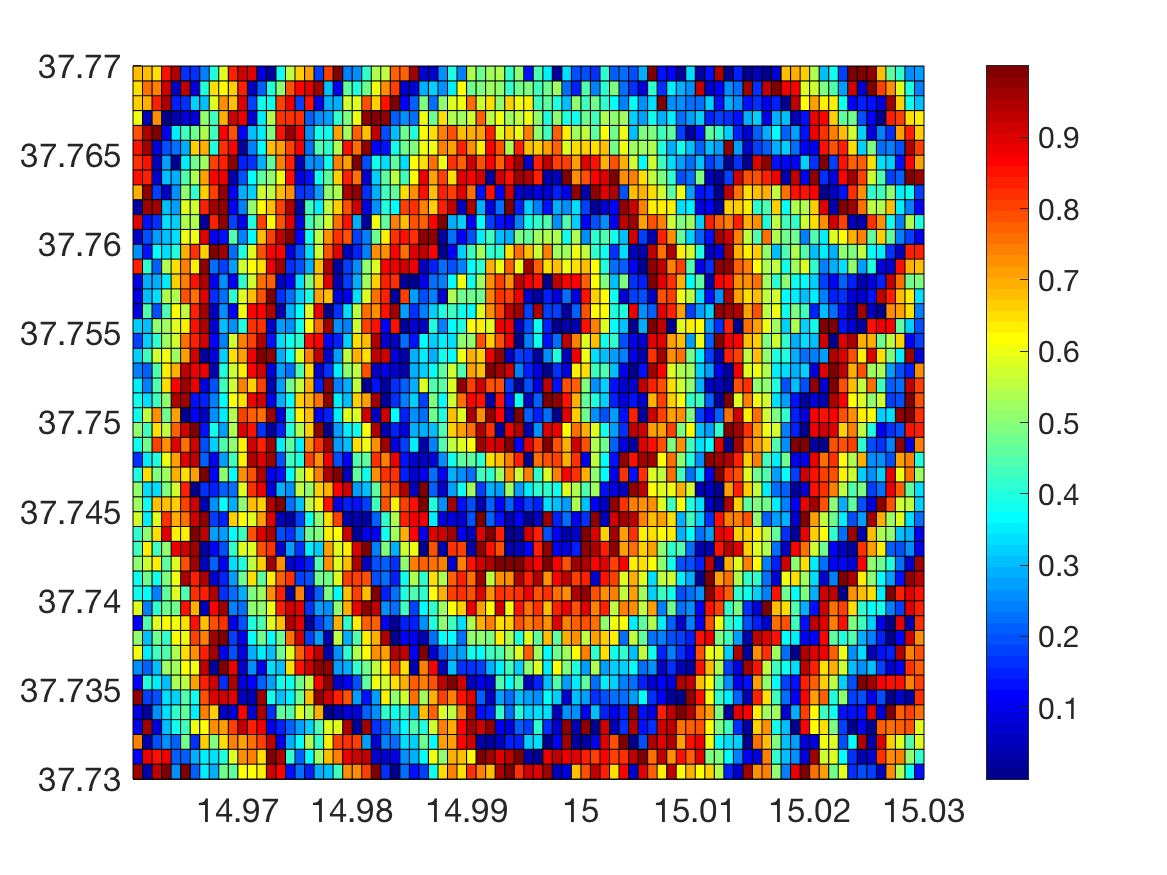} }

\subcaptionbox[]{      Denoised $f$ mod 1  \\(RMSE = 0.194)  
}[ 0.24\textwidth ]
{\includegraphics[width=0.24\textwidth] {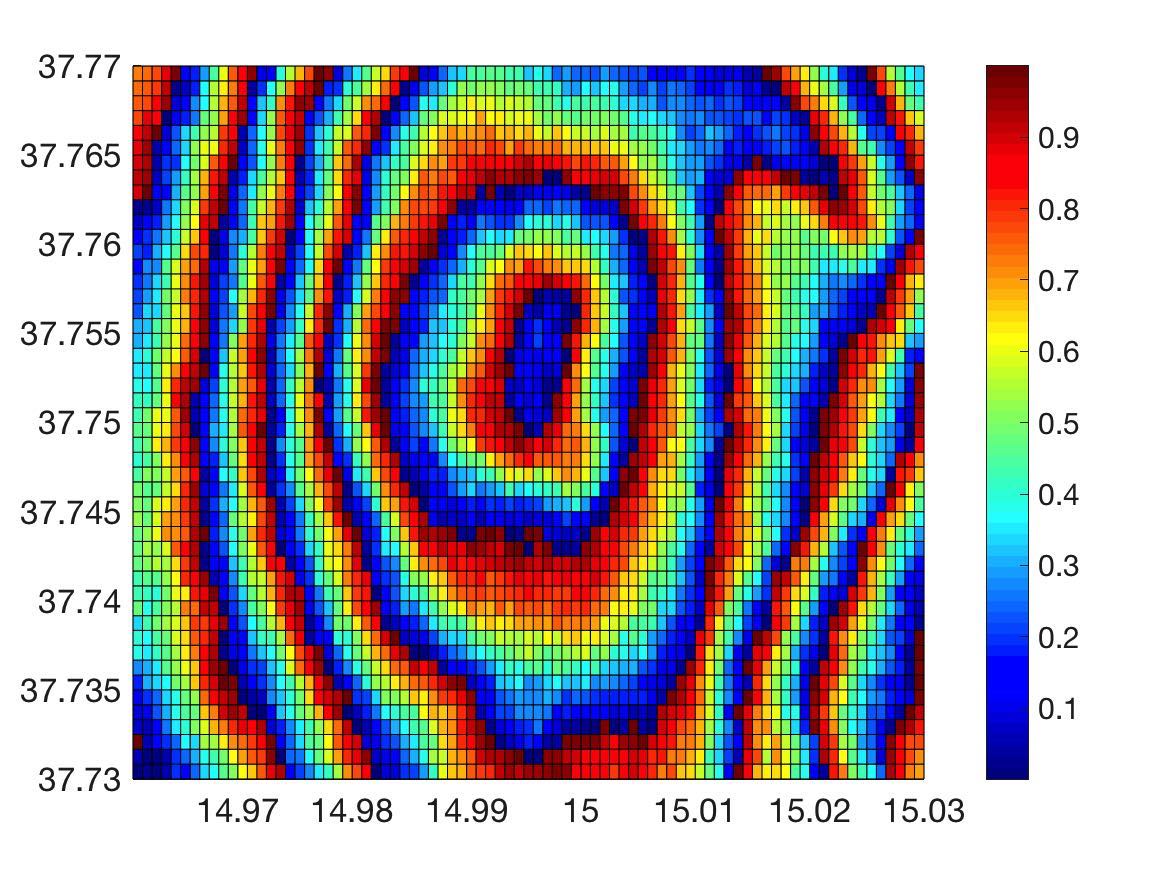} }
%
\subcaptionbox[]{   Denoised $f$    \\(RMSE = 0.08)
}[ 0.24\textwidth ]
{\includegraphics[width=0.24\textwidth] {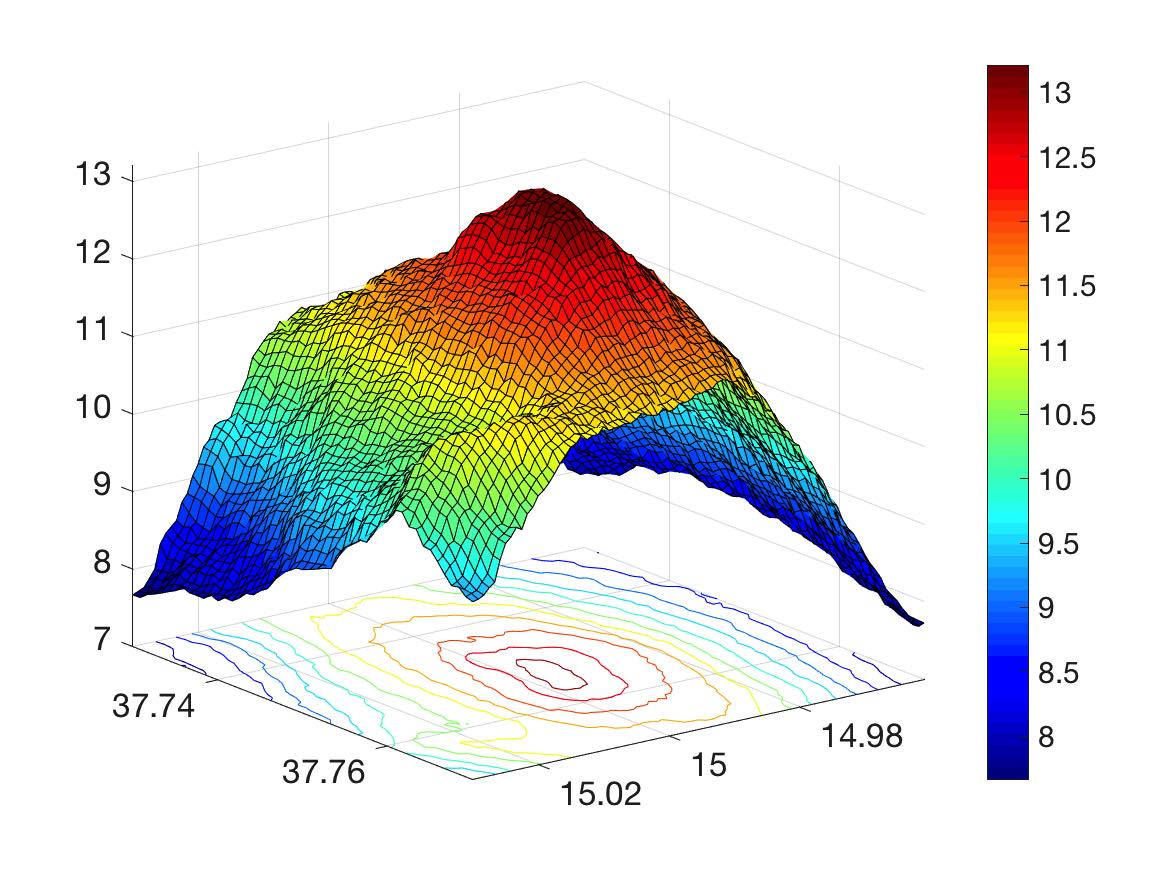} }
%
%
\subcaptionbox[]{   Denoised $f$ (depth)
}[ 0.24\textwidth ]
{\includegraphics[width=0.24\textwidth] {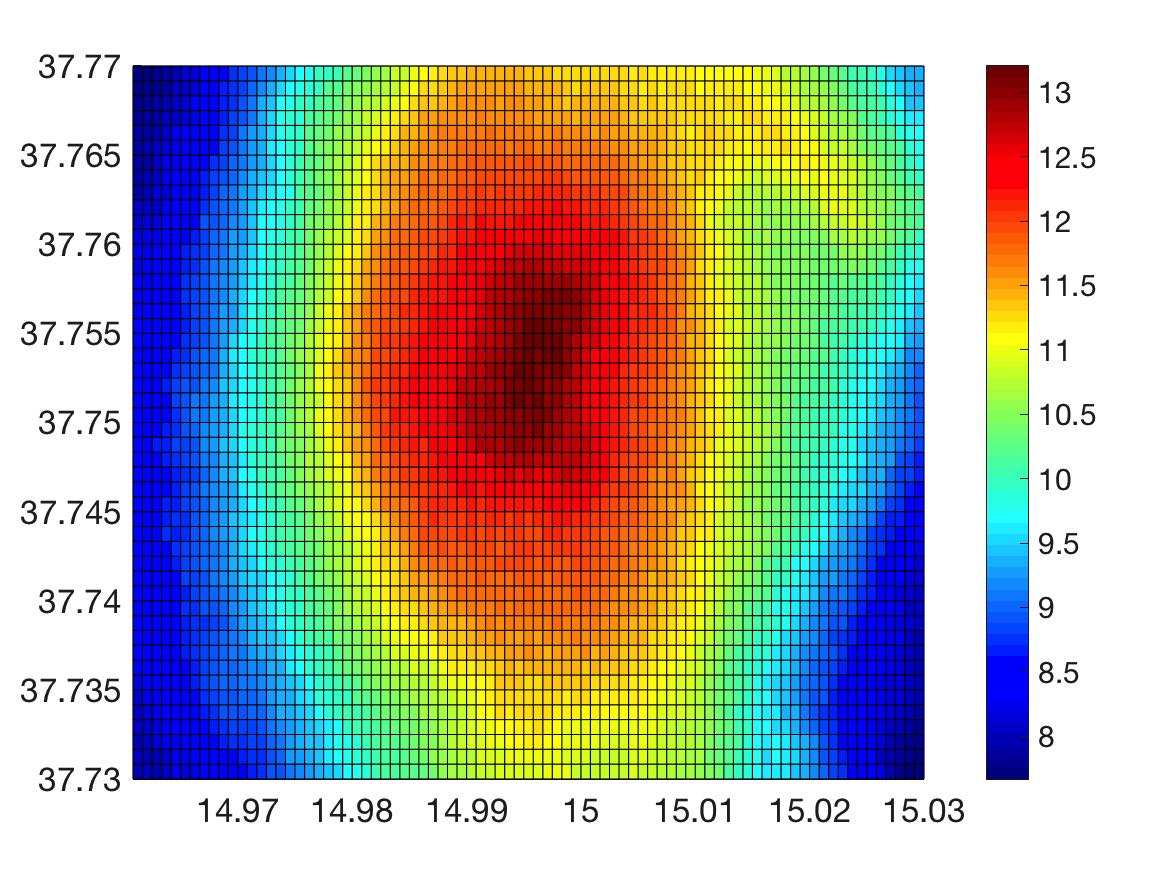} }
%
%
\subcaptionbox[]{  Error  $|f - \hat{f}|$ 
}[ 0.24\textwidth ]
{\includegraphics[width=0.24\textwidth] {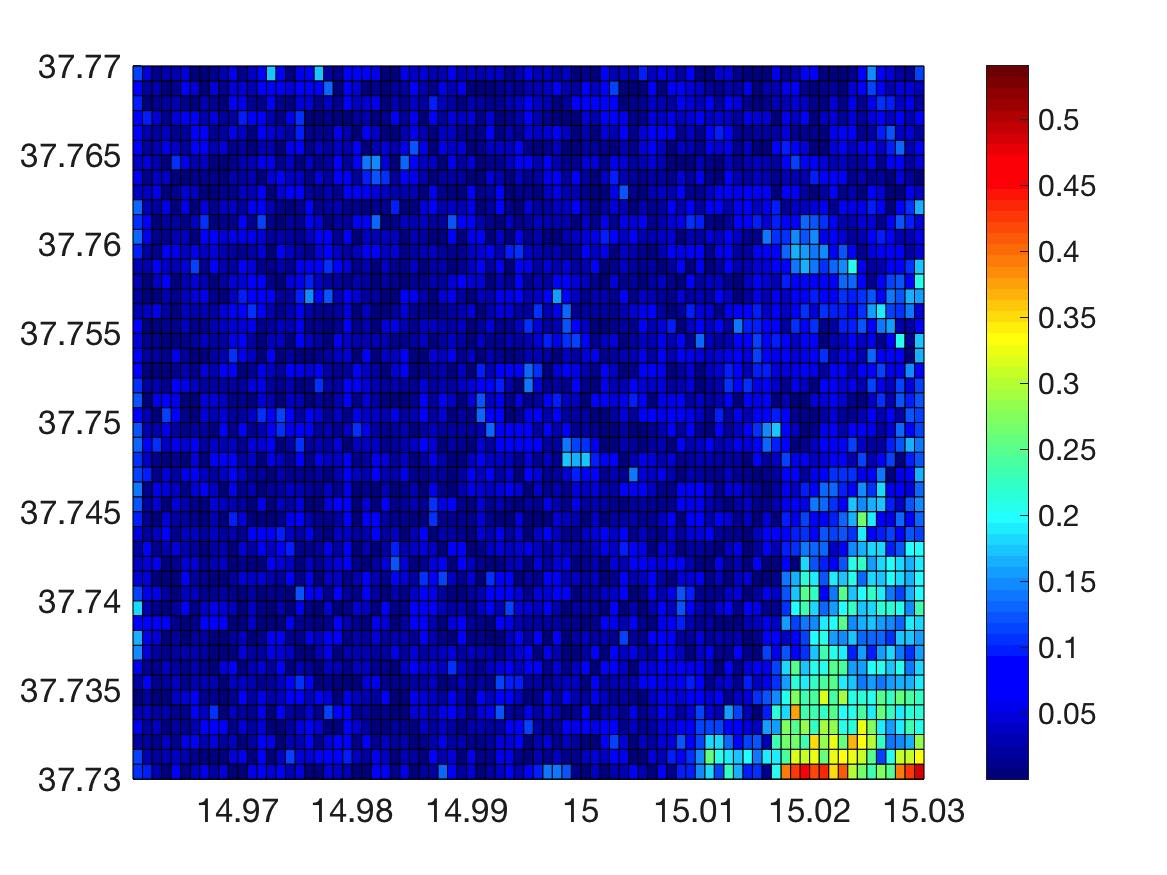} }
\captionsetup{width=0.95\linewidth}
\caption[Short Caption]{ Elevation map of Mount  Etna  with   $n = 4100$, recovered by Manopt-Phases with $k=1$ (Chebychev distance), $\lambda = 0.3$,  and noise level $\sigma=0.10$ under the Gaussian model.} 
\label{fig:Etna_Low_Noisy} 
\end{figure*}

\begin{figure}[!ht]
\centering
\subcaptionbox[]{   Clean $f$ mod 1.
}[ 0.43\textwidth ]
{\includegraphics[width=0.43\textwidth] {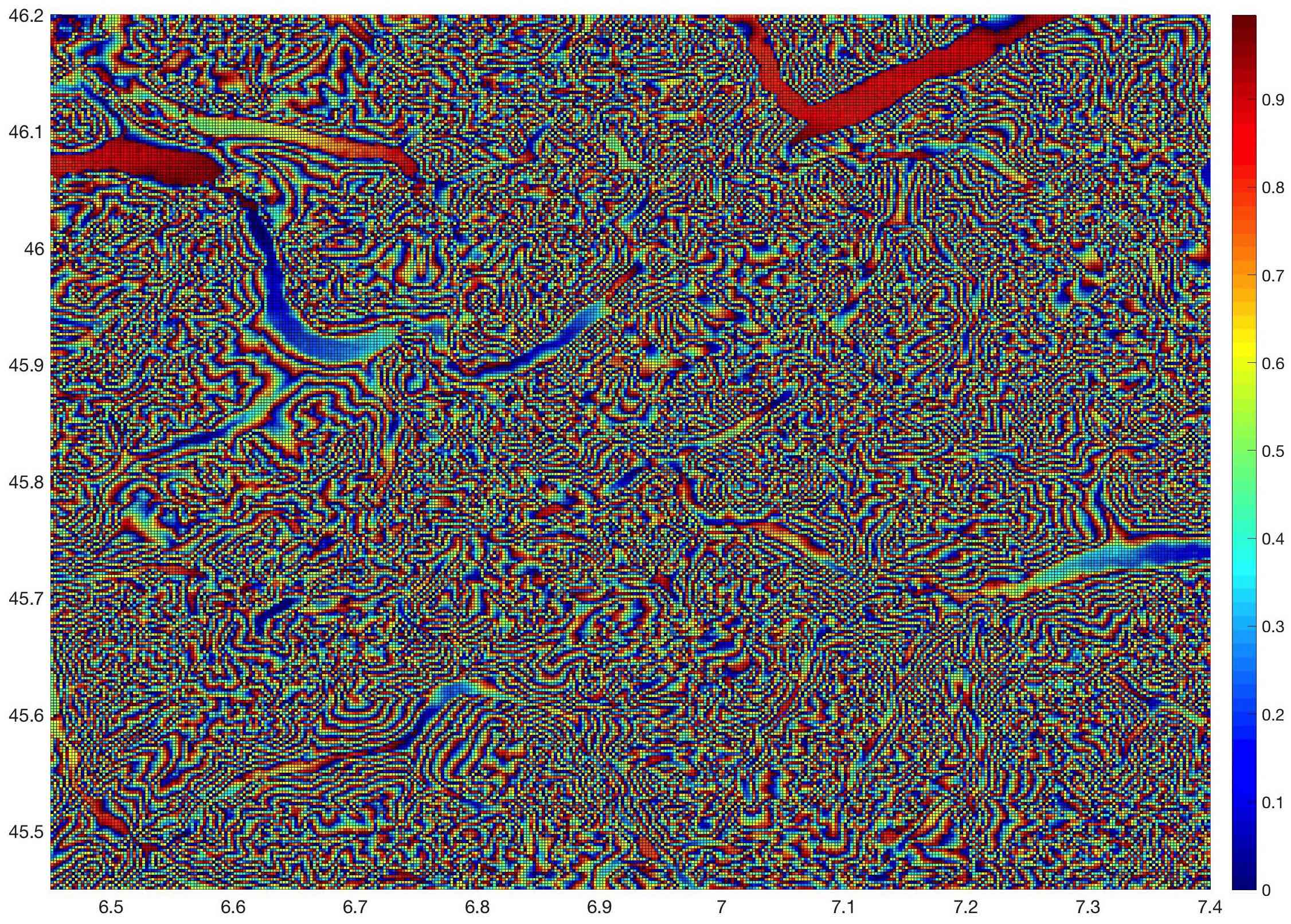} }
 \hspace{-0.08\textwidth} 
\subcaptionbox[]{ Recovery   $ \lambda = 0.01$.
}[ 0.6\textwidth ]
{\includegraphics[width=0.6\textwidth] {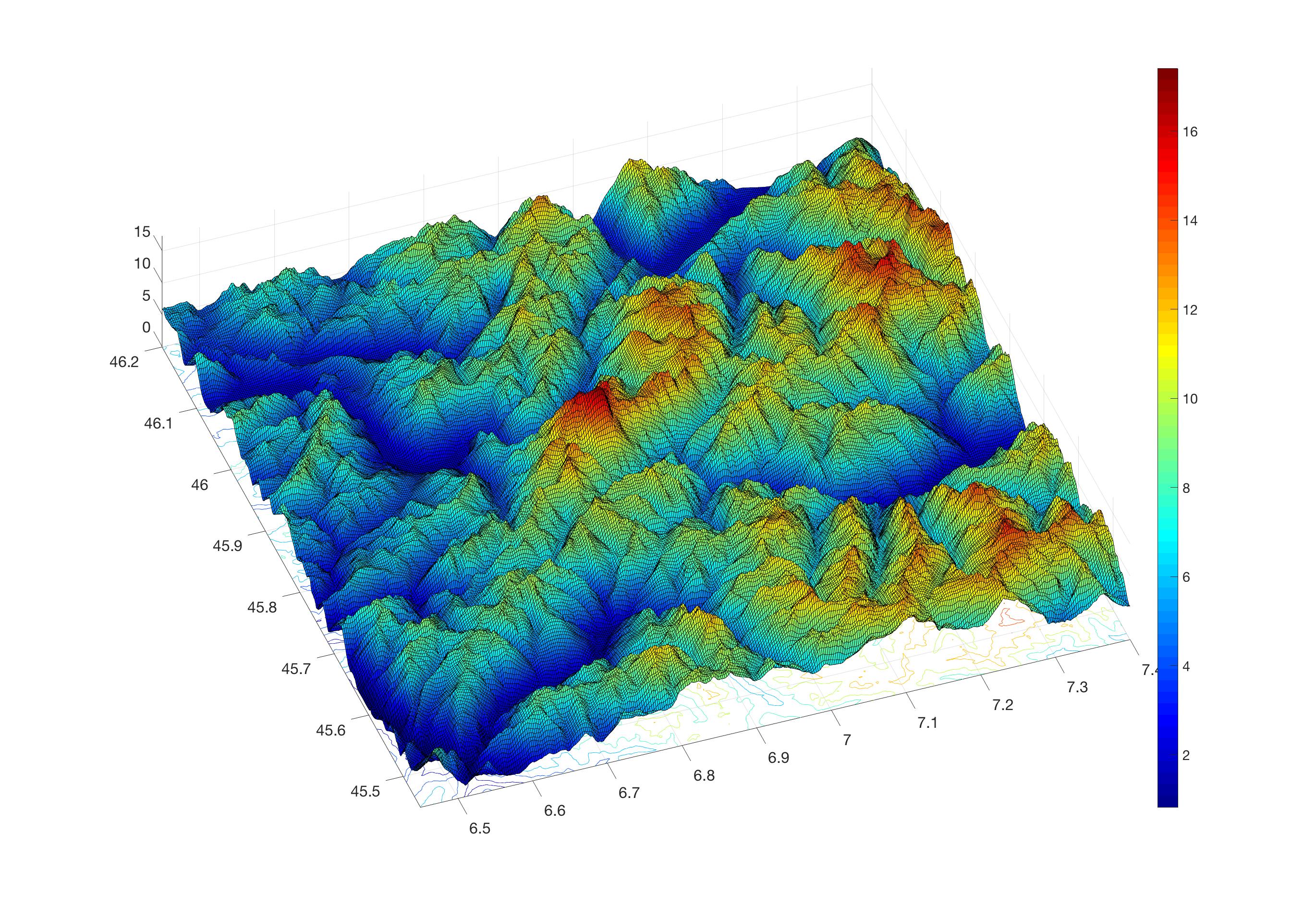} }
\captionsetup{width=0.95\linewidth}
\caption[Short Caption]{Recovery of the elevation map of Mont Blanc from  over 1 million  clean sample points, as recovered by Manopt-Phases. Despite the challenging topology, we are able to almost perfectly recover (unwrap) the shape of the mountain, in just under 18 seconds.   RMSE for the mod 1 recovery is $0.0457 $, while the RMSE for the final $f$ estimates was $0.69$.
}
\label{fig:MontBlanc_clean}
\end{figure}

\begin{figure}[!ht]
\centering
\hspace{-2mm}
\subcaptionbox[]{   Noisy $f$ mod 1, $ \sigma=0.10$.
}[ 0.43\textwidth ]
{\includegraphics[width=0.43\textwidth] {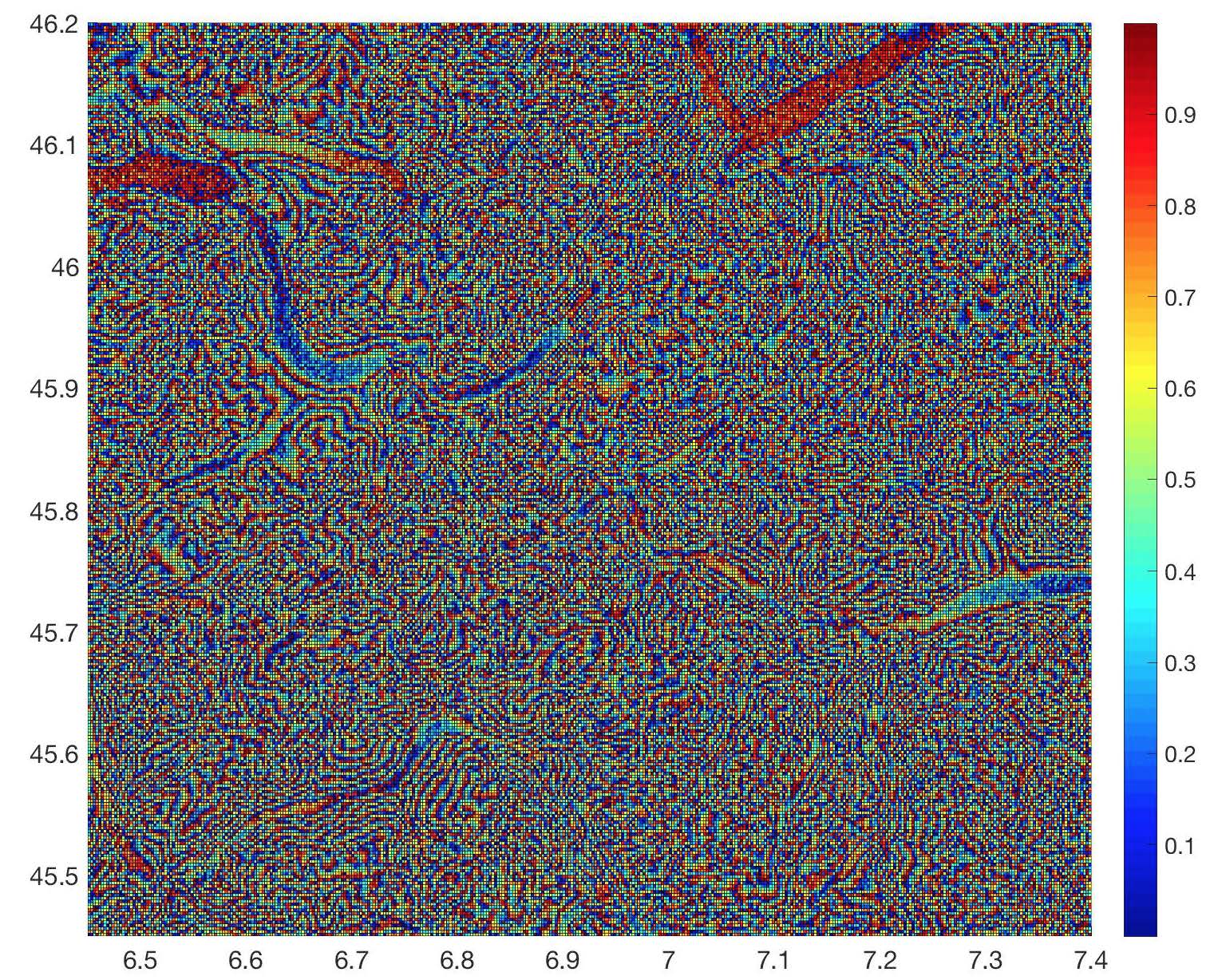} }
\hspace{-0.08\textwidth} 
\subcaptionbox[]{ Recovery   $\lambda = 0.01$.
}[ 0.6\textwidth ]
{\includegraphics[width=0.56\textwidth] {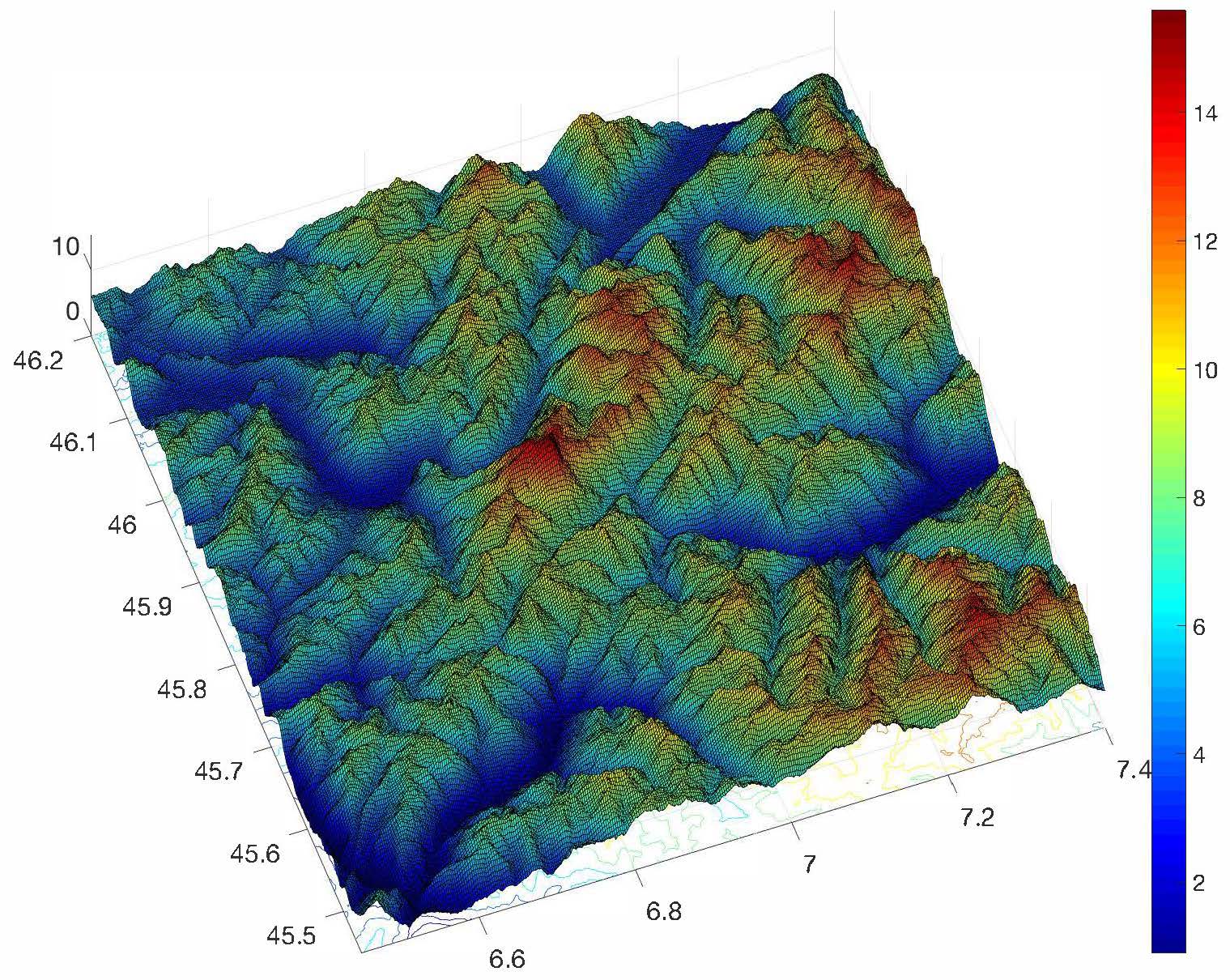} }
\captionsetup{width=0.95\linewidth}
\caption[Short Caption]{Recovery of the elevation map of Mont Blanc from  over 1 million sample points under the Gaussian noise model with $\sigma = 10 \%$, as recovered by Manopt-Phases. Despite the noise level and the topology, our approach  recovers the details of the topographical relief, in under 20 seconds.   RMSE for the mod 1 recovery is $0.26 $, while the RMSE for the final $f$ estimates equals $1.06$.
}
\label{fig:MontBlanc_noisy}
\end{figure}



\section{Related work}  \label{sec:related_work}
This section provides a discussion of the recent closely related work of \cite{bhandari17}, 
details the connection with the problem of  group synchronization with anchors, and finally surveys some 
of the approaches in the phase unwrapping literature. 

\subsection{Function recovery from modulo samples} 
The recent work of  \cite{bhandari17} 
considers the problem of recovering a bandlimited function $g$ (the spectrum is limited to $[-\pi,\pi]$) from its centered modulo samples  
defined using the non-linear map $M_{\lambda} : \matR \rightarrow [-\lambda,\lambda]$
\begin{equation} \label{eq:centered_mod_fn}
 M_{\lambda}(g(t)) := 2\lambda\left(\left[\frac{g(t)}{2\lambda} + \frac{1}{2} \right] - \frac{1}{2} \right), 
\end{equation}
where $[x]$ denotes the fractional part of $x$, and $t \in \matR$. By considering a regular sampling 
on $\matR$ with step length $T$, it is shown \cite[Theorem 1]{bhandari17} that if $T \leq 1/(2\pi e)$, then their 
Algorithm recovers the samples of $g$ exactly (assuming no noise) from $N^{th}$ order finite differences of the modulo samples (with 
$N$ suitably large). Consequently, as $g$ is bandlimited, one can recover the function $g$ itself via a low pass filter. The analysis for unwrapping the  samples mainly relies on the additional assumption that $g \in C^{\infty}(\matR) \cap L^{\infty}(\matR)$. This  assumption  is required  in order to be able  to control the $N^{th}$ order finite differences of the modulo samples for appropriately large $N$.
Moreover, Theorem $1$ only holds in the case of \emph{noiseless} modulo samples, while the conditions under which the method stably unwraps the samples  are unclear. As shown in the experiments, our finding is that their  algorithm is  stable only for very low levels of noise. Finally, note that 
for $\lambda = 0.5$, \eqref{eq:centered_mod_fn} corresponds to a centered modulo $1$ function, with range in $[-0.5,0.5)$. 
It is easy to verify that 
\begin{equation} \label{eq:rel_cenmod_noncenmod}
M_{0.5}(g(t)) 
= \left(\left[g(t) + \frac{1}{2} \right] - \frac{1}{2} \right) 
= \left\{
\begin{array}{rl}
g(t) \bmod 1 \quad ; & \text{if} \ g(t) \bmod 1 < 1/2  \\
g(t) \bmod 1 - 1 \quad ; & \text{if} \ g(t) \bmod 1 \geq 1/2
\end{array} \right..
\end{equation}
Therefore, \eqref{eq:rel_cenmod_noncenmod}  implies  that there is a one-to-one correspondence between $M_{0.5}(g(t))$ 
and $g(t) \bmod 1$, the latter of which we adopt throughout our paper.
  

\subsection{Relation to group  synchronization} 

The problem we consider is closely related to the well-known  \emph{group synchronization} problem introduced  by  \cite{sync}. 
The synchronization of clocks in a distributed network from noisy pairwise measurements of their time offsets is one such example of synchronization, where the underlying group is the real line  $\mathbb{R}$.
The eigenvector and semidefinite programming methods for solving an instance of the synchronization problem were initially  introduced by  \cite{sync} in the context of angular synchronization (over the group SO(2) of planar rotations)  where one is asked to estimate $n$ unknown angles $\theta_1,\ldots,\theta_n \in [0,2\pi)$ given $m$ noisy measurements $ \delta_{ij} $ of their offsets $\theta_i - \theta_j \mod 2\pi$.  The difficulty of the problem is amplified on one hand by the amount of noise in the pairwise  offset measurements, and on the other hand by the fact that $m \ll {n \choose 2}$, i.e., only a very small subset of all possible pairwise offsets are measured.
Given the noisy measurements, \cite{sync}  introduced an  approach that first embeds the noisy angle differences  
on the unit complex circle via $Z_{ij} = \exp(\iota 2\pi  \delta_{ij})$, leading to the Hermitian $n \times n$ matrix $\mathbf{Z}$. 
Thereafter, one recovers the angle estimates by maximizing a quadratic  form  $\vecx^{*} \mathbf{Z} \vecx$, subject to $\vecx$ 
lying on a sphere. This is similar to the formulation in \eqref{eq:qcqp_denoise_complex}, but without the linear term.  
However, there is a \emph{fundamental difference} between the two problems. The matrix in the quadratic term in \eqref{eq:qcqp_denoise_complex} 
is formed using the Laplacian of the smoothness regularization graph, and thus is independent of the data (given noisy mod 1 samples). 
On the other hand, as described above, the matrix $\mathbf{Z}$ in synchronization is formed using the given noisy pairwise angle offsets embedded on the unit circle, 
and hence depends on the data. Additionally, the angular representations of the modulo samples bear a smoothness property arising due to the function $f$, while   the angle offsets in synchronization  are typically arbitrary. 

We point out that a similar objective function encompassing both a quadratic and a linear term 
arises in the setting of synchronization with \textit{anchors}, which are nodes whose corresponding group element is known a priori. 
The goal is to combine information given by the anchors with the matrix of noisy pairwise group ratios, in order to estimate the 
corresponding group elements of the non-anchor nodes.
Such a variation of the synchronization problem has been explored in prior work by a subset of the authors (see \cite{asap3d, sync_congress}), who  introduced several methods for incorporating anchor information in the  synchronization problem over $\mathbb{Z}_2$, 
in the context of the molecule problem from structural biology. The first approach relies on casting the problem as a quadratically 
constrained quadratic program (QCQP) similar to the approach we pursue for denoising the modulo 1 samples, while a second one relies 
on an SDP formulation. 

\subsection{Phase unwrapping}  
We now discuss methods from the phase unwrapping literature to put our methodology in context. 
Phase unwrapping is a field in its own with a vast body of work, therefore  this section is by no means a comprehensive 
overview of the literature.  

Phase unwrapping is a classical and challenging problem in signal processing with a long line of work wherein 
one recovers the original phase values (in radians) from their wrapped (modulo $2\pi$) versions. Formally, 
for an unknown sequence of phase values $(\phi_i)_{i=1}^n$ (in radians), one is given the wrapped 
values $\psi_i = w(\phi_i)$,  where $w:\matR \rightarrow [-\pi,\pi)$ is the wrap function that outputs 
centered modulo $2\pi$ values. We consider $\phi_i = \phi(\vecx_i)$ for some unknown function $\phi : \matR^d \rightarrow \matR$, 
with $\vecx_i \in \matR^d$. The goal is to recover the original $\phi_i$'s from the $\psi_i$'s.   
The problem is, of course, ill posed in general so one typically needs to make some additional smoothness assumptions on $\phi$, 
although some methods are developed with an eye towards dealing with \emph{abrupt} phase changes.  
For $d = 1$, \cite{Itoh82} showed in $1982$ that if the condition 
\begin{equation} \label{eq:itoh_cond}
\abs{\phi_i-\phi_{i-1}} \leq \pi  
\end{equation}
holds for each $i$, then this implies
$\phi_i - \phi_{i-1} = w(\psi_i - \psi_{i-1})$ for all $i = 1,\dots,n$. Thereafter, 
one can easily recover the $\phi_i$'s in a sequential manner by integration, up to a global 
shift by an integer multiple of $2\pi$. We remark that this approach is essentially the Quotient Tracker Algorithm (QTA), discussed 
in Section \ref{subsec:unwrap_stage_and_algo}. A similar argument extends to the 
case $d = 2$ (cf., \cite[Lemma 1.2]{Ying06}). Of course, \eqref{eq:itoh_cond} will typically not hold in the presence of noise. 
Therefore, several robust alternatives have been proposed to  date, and we briefly review some of them for the 
$d=2$ case (which is also one of the most studied cases due to the many applications). The methods can be broadly classified 
into the following categories. We refer the reader to \cite[Section 2.1]{Ying06} for a more detailed overview of these methods. 
\begin{itemize}
\item \emph{Least squares methods.} This  class of methods seeks to find the phase 
function $\phi : \matR^2 \rightarrow \matR$ that minimizes a cost function of the form
\begin{equation} \label{eq:lsm_phase_unwrap}
 \norm{\partial_x \phi - w(\partial_x \psi)}_p + \norm{\partial_y \phi - w(\partial_y \psi)}_p.
\end{equation}
Here, $\norm{\cdot}_p$ is the $L_p$ norm and $\partial_x, \partial_y$ denote partial derivatives w.r.t $x,y$ respectively. 
Specifically, one solves a discrete version of \eqref{eq:lsm_phase_unwrap} with discrete partial derivatives.
For $p=2$, the solution actually has an analytical form
given by the discrete form of Poisson's equation with Neumann boundary conditions. This is known to be efficiently 
implementable,  using for instance methods based on fast Fourier transform (FFT) (\cite{Pritt94}), discrete cosine transform (DCT) (\cite{Kerr96}), 
multigrid techniques (\cite{pritt1996}) etc. More broadly, the least-squares formulation of phase unwrapping (both the weighted and unweighted versions) 
lead to discretized partial differentiable equations, which are solvable efficiently by standard methods from the numerical 
PDE literature (see for eg., \cite{Takajo88,Takajo88_non,Hunt79,Ghiglia89,Ghiglia94,buzbee1970direct}). 
We remark that smoothness constraints have been incorporated by introducing a Tikhonov regularization term, which increases 
robustness to noise and missing data (\cite{Marroquin95,NumericalReceipesArt}). 
\cite{huang2012path} proposed a TV (Total variation) minimization based method wherein the unwrapped phase gradients are estimated from the 
wrapped counterparts via a TV norm regularized least-squares method. \cite{Rivera04half} proposed solving a ``half quadratic'' 
regularized-least squares problem, and proposed convex as well as non-convex formulations of the problem. 
Finally, we note that \cite{Ghiglia96_lp} proposed an algorithm to minimize a more general $\ell_p$ norm based objective function for $2$D phase unwrapping; 
they showed that this framework is equivalent to weighted least squares phase unwrapping where the weights are data dependent. 

\item \emph{Branch cut methods.} A common problem that arises in 2D phase unwrapping is the so-called \textit{path dependent problem},  wherein 
the integration result depends on the path chosen between two points. In particular, it is known (\cite{Kreysig66}) that if the unwrapping is 
path dependent then it implies that there exists a closed path $C$ such that $\sum_{C} [w(\triangle_x \psi_{m,n}) + w(\triangle_y \psi_{m,n})]$ 
is not zero; here $\triangle_x, \triangle_y$ denote discrete partial derivatives. This non-zero value is referred to as the \emph{residue}, and arises 
for instance due to noise or (near) discontinuities in the original phase signal (see for eg., \cite[Section 1.1.2]{Ying06}).   
Branch cut methods rely on the assumption that paths from positive residues to negative residues (referred to as branch cuts) 
cross the regions corresponding to the discontinuities. To this end, many algorithms have been proposed to find the 
optimal branch cuts in the phase image (\cite{Huntley89,Bone_91,Prati90,Chavez02}).

\item \emph{Network flow methods.} These methods are similar to branch-cut methods and rely on the same assumptions. The difference is 
that they explicitly try to quantify the discontinuities at each sample and then find the optimal unwrapped phases that 
minimize the overall discontinuities. Broadly speaking, the idea is to construct a network with the nodes defined by the residues, 
and with a directed arc from the positive residues to the negative ones. With a suitable cost per unit flow assigned to each arc, the 
algorithms then seek to find a flow that minimizes the cumulative cost over all the arcs. (\cite{Towers91,ching92,takeda96}). 
\end{itemize}
%
%
%
 

More recently, \cite{dias07} proposed a method that involves the minimization of an energy functional 
based on a generalization of the classical $\ell_p$ norm, using graph cut algorithms. 
\cite{Dias08,Vala09} proposed more noise tolerant versions 
of PUMA, which together with PUMA constitute the current state of art in phase unwrapping. \cite{dias07} proposed a two-stage algorithm 
similar in spirit to ours. Specifically, the first stage involves obtaining the denoised 
modulo $2\pi$ samples via a local adaptive denoising scheme; the second stage then takes as input 
these denoised samples and unwraps them using PUMA. Assuming the original phase to be piecewise smooth  
that is well approximated by a polynomial in the neighborhood of the estimation point, their denoising scheme is based on 
local polynomial approximations in sliding windows of varying sizes. \cite{Vala09} adopted a Bayesian framework assuming 
a first order Markov random field prior. The proposed algorithm first unwraps the phase image using PUMA, and then in the second 
stage, it denoises the unwrapped phase to form the final estimate.

\cite{Gon14} considered placing a sparsity prior on the phase signal 
in the wavelet domain, to model discontinuities. They proposed an unwrapping algorithm based on solving a $\ell_1$ minimization problem. 
\cite{Kamilov15} proposed a method based on the minimization of an energy functional that includes
a weighted $\ell_1$ norm based data fidelity term, along with a regularizer term based on the 
higher order total variation (TV) norm.
 

Extensions to higher dimensions have also attracted significant interest from the community, in light of numerous applications. In SAR interferometry, the data typically consists of a sequence of time-dependent consecutive interferograms.  Existing approaches utilize the time dimension to perform 1-D phase unwrapping along the time axis component (as opposed to performing standard 2-D phrase unwrapping for each time stamped interferogram), an approach which increases robustness of the recovery process for scenarios of steep gradients  in  the phase maps (\cite{Huntley_93_temporal}). This approach has been subsequently extended to the case of 3-D phrase unwrapping in a 3-D  space-time domain  by \cite{Costantini_3D_2002}. They formulate the phase unwrapping problem as a linear integer minimization, and report increased robustness when compared to standard 2-D phase unwrapping on simulated and real SAR images.
More recently, \cite{Osmanoglu_3D_2014} introduced a novel 3-D phase unwrapping approach for generating digital elevation models, 
by combining multiple SAR acquisitions using an extended Kalman filter.
%
%

Very recently, \cite{Dardikman_4D_2018}  introduced a new 4-D phase unwrapping approach for time-lapse quantitative phase microscopy, which allows for the reconstruction of optically thick objects that are optically thin in a certain temporal point and angular view. The authors leverage both the angular dimension and the temporal dimension, in addition to the usual $x-y$ dimensions, to enhance the reconstruction process. They report improved numerical results over other state-of-the-art methods.  
Finally, we remark that \cite{jenkinson03}  introduced a fast algorithm for the $N$-dimensional phase unwrapping problem, built around a cost function that leads to an integer programming problem, solved via a \textit{greedy} optimization approach, and tested on 3D MRI medical data for venogram studies.


\section{Concluding remarks  and future work}  \label{sec:conclusion}
We considered the problem of denoising noisy modulo 1 samples of a smooth  function, arising in the context of the popular phase unwrapping problem.  We proposed a method centered around solving a trust-region subproblem with a sphere constraint, and provided extensive empirical evidence as well as theoretical analysis, that altogether highlight its robustness to noise. We mostly focused on the one-dimensional case, but also provided extensions to the two-dimensional setting, whose applications include the phase unwrapping problem.  In addition, we have also explored  two formulations that leverage tools from Riemannian optimization on manifolds, including a relaxation based on semidefinite programing, solvable fast via a Burer-Monteiro approach.
 
%

There are several possible research directions worth exploring in future work, that we detail below. 
\begin{itemize}
\item An interesting direction would be to better understand the unwrapping stage of our approach, either via the 
simple OLS (potentially with an additional smoothness regularization term), or by proposing a new method based on group synchronization (\cite{sync}). 
For the latter direction, finding the right global shift of each sample point (i.e., the quotient $q_i$) can be cast a synchronization 
problem over the real line $ \mathbb{R} $, given pairwise (potentially noisy) offset measurements $q_i - q_j$ between nearby points. 
The non-compactness of $\mathbb{R}$ renders the eigenvector or semidefinite programming relaxations no longer applicable directly, 
thus motivating  an approach that first compactifies the real line by wrapping it over the unit circle, as explored recently 
by \cite{syncRank} in the context of ranking from noise incomplete pairwise information. 

\item Another potentially interesting direction to explore would be a patch-based divide-and-conquer method that first decomposes the graph 
into many (either overlapping or non-overlapping) subgraphs (potentially with an eye towards the \textit{jumps} in the mod 1 measurements), 
solves the problem locally and then integrates the local solutions into a globally consistent framework, in the spirit of existing methods 
from the group synchronization literature (see for eg., \cite{asap2d, asap3d}). 

\item As discussed briefly at the start of Section \ref{sec:trust_reg_relax},  an interesting approach for solving \eqref{eq:orig_denoise_hard_1} would be to first discretize the angular domain, and then solve \eqref{eq:orig_denoise_hard_1} approximately via dynamic programming. In general, the graph $G$ has tree-width $\Omega(k^d)$, rendering the computational cost of naive dynamic programming to be exponential in $k^d$.  

\item  Another potential direction would be to explore a variation of \eqref{eq:orig_denoise_hard}, with the last term replaced by a TV norm
\begin{eqnarray}
\min_{ g_1,\ldots,g_n \in \mathbb{C}; |g_i|=1 } \sum_{i=1}^{n} |g_i - z_i|^2 + \lambda  \sum_{\set{i,j} \in E} |g_i - g_{j}|. 
\label{eq:orig_denoise_hard_TV}
\end{eqnarray}
At first sight, the TV norm may not seem beneficial in light of the smoothness assumptions in our setting; however, in certain imaging applications (\cite{fang06}), it is desirable to preserve any sharp discontinuities that may arise in the boundaries between classes within the image.
\cite{Rudin_Osher_Fatemi} introduced \textbf{ROF}, a total variation (TV)-based approach as a regularizing functional for image denoising. The intuition is that the TV term in the minimization discourages the solution from having an oscilatory behavior, while allowing it to have discontinuities.
%
Inspired by the above seminal variational framework, and building on their prior work, such as  \cite{Ronny_BLSW14},   \cite{Ronny_BW16}  introduced  variational models for various nonlinear data spaces. In particular, they developed proximal point algorithms for the solution of the second order TV-based problems for denoising, inpainting, and a combination of both, for combined  \textit{cyclic} (i.e., modulo measurements) and linear space valued images, along with a convergence analysis and a number of applications to real-world problems.

\item Another  natural direction to explore is the development and theoretical analysis  of  a ``single-stage method''  that directly outputs denoised estimates of the original real-valued samples. 
%
\item Finally, an interesting question would be to consider the regression problem, where one attempts to learn a smooth, robust estimate to the underlying 
mod 1 function itself (not just the samples)  and/or the original $f$, from noisy  mod 1 samples of $f$.


\end{itemize}

%
%
%
%
%
%
%
%


\section{Acknowledgements}

We thank Joel Tropp for bringing this problem to our attention. We are grateful to Afonso Bandeira and Joel Tropp for  insightful discussions on the topic and pointing out the relevant literature;  Nicolas Boumal for helping us in using the Manopt optimization toolbox and providing useful comments; Yuji Nakatsukasa  for many helpful discussions on the trust-region subproblem, and for kindly providing us the code for the algorithm in his paper \cite{Naka17}; Ayush Bhandari for sharing the code from his paper \cite{bhandari17}. This work was done while HT was affiliated to the Alan Turing Institute, London, and the School of Mathematics, University of Edinburgh, UK. Both MC and HT were supported by EPSRC grant EP/N510129/1 at the Alan Turing Institute.  

\bibliographystyle{apalike}

\bibliography{funcmod}

\newpage
\onecolumn

\appendix 

\begin{centering}
\large \bf Supplementary Material: Provably robust estimation of modulo $1$ samples of a smooth function with applications to phase unwrapping
\end{centering}

%
%

 \renewcommand\appendix{\par
        \renewcommand\thesection{A}
       \renewcommand\thesubsection{A\arabic{subsection}}
        \renewcommand\thetable{A\arabic{table} } }        

\appendix

\section{Rewriting the QCQP in real domain} \label{sec:qcqp_compl_to_real}
We first show that $\lambda \vecg^{*} L \vecg = \vecgbar^T \Hbar \vecgbar$. Indeed, it holds true that 
\begin{eqnarray}
\vecgbar^T \Hbar \vecgbar
&=& (\real(\vecg)^T \imag(\vecg)^T) 
\begin{pmatrix}
  \lambda L \quad & 0 \\ 0 \quad & \lambda L  
 \end{pmatrix}  
\begin{pmatrix}
  \real(\vecg) \\ \imag(\vecg) 
 \end{pmatrix} \\
&=& (\real(\vecg)^T \imag(\vecg)^T) 
\begin{pmatrix}
  \lambda L \real(\vecg) \\ \lambda L \imag(\vecg) 
 \end{pmatrix} \\
&=& \real(\vecg)^T (\lambda L) \real(\vecg) + \imag(\vecg)^T (\lambda L) \imag(\vecg) \\
&=& \lambda (\real(\vecg) - \iota \imag(\vecg))^T L (\real(\vecg) + \iota \imag(\vecg)) \\
&=& \lambda \vecg^{*} L \vecg.
\end{eqnarray}
Next, we can verify that 
\begin{eqnarray}
\real(\vecg^{*}\vecz) 
&=& \real((\real(\vecg) - \iota \imag(\vecg))^T (\real(\vecz) + \iota \imag(\vecz))) \\ 
&=& \real(\vecg)^T \real(\vecz) + \imag(\vecg)^T \imag(\vecz) \\
&=& \vecgbar^T \veczbar.
\end{eqnarray}
Lastly, it is trivially seen that $\norm{\vecgbar}_2^2 = \norm{\vecg}_2^2 = n$. 
Hence \eqref{eq:qcqp_denoise_complex} and \eqref{eq:qcqp_denoise_real} are equivalent. 


 \renewcommand\appendix{\par
        \renewcommand\thesection{B}
       \renewcommand\thesubsection{B\arabic{subsection}}
        \renewcommand\thetable{B\arabic{table} } }        

\appendix

\section{Trust-region sub-problem with $\ell_2$ ball/sphere constraint} \label{sec:trust_region_discuss}
Consider the following two optimization problems
\begin{equation}
\begin{split}
\begin{rcases}
\min_{\vecx} \quad \vecb^T \vecx + \frac{1}{2} \vecx^T \matP \vecx \\
\text{s.t} \norm{\vecx}_2 \leq r
\end{rcases} \text{(P1)}
\end{split}
\qquad \vline \qquad
\begin{split}
\begin{rcases}
\min_{\vecx} \quad \vecb^T \vecx + \frac{1}{2} \vecx^T \matP \vecx \\
\text{s.t} \norm{\vecx}_2 = r
\end{rcases} \text{(P2)}
\end{split}
\label{eq:trust_region_opt} 
\end{equation}
with $\matP \in \matR^{n \times n}$ being a symmetric matrix. 
(P1) is known as the trust-region sub-problem in the optimization literature and has been 
studied extensively with a rich body of work. (P2) is closely related to (P1), and has a non-convex equality constraint. 
There exist algorithms that efficiently find the global solution of (P1) and (P2), to arbitrary accuracy. 
In this section, we provide a discussion on the characterization of the solution of these two problems. 

To begin with, it is useful to note for (P1) that
\begin{itemize}
\item If the solution lies in the interior of the feasible domain, then it implies 
$\matP \succeq 0$. This follows from the second order necessary condition for a local minimizer. 


\item In the other direction, if $\matP \not\succeq 0$ then the solution will always lie on the boundary.
\end{itemize}
Surprisingly, we can characterize the solution of (P1), as shown in the following lemma. 
\begin{lemma}[\cite{Sorensen82}] \label{lemma:trs_ball}
$\vecx^{*}$ is a solution to (P1) iff $\norm{\vecx^{*}}_2 \leq r$ and $\exists \mu^{*} \geq 0$ such that 
(a) $\mu^{*}(\norm{\vecx^{*}}_2 - r) = 0$, (b) $(\matP + \mu^{*}\matI)\vecx^{*} = -\vecb$ 
and (c) $\matP + \mu^{*}\matI \succeq 0$. Moreover, if $\matP + \mu^{*}\matI \succ 0$, then the solution is unique.
\end{lemma}
Note that if the solution lies in the interior, and if $\matP$ is p.s.d and singular, then 
there will also be a pair of solutions on the boundary of the ball. This is easily verified. 
The solution to (P2) is characterized by the following lemma. 
\begin{lemma}[\cite{Hager01,Sorensen82}] \label{lemma:trs_sphere}
$\vecx^{*}$ is a solution to (P2) iff $\norm{\vecx^{*}}_2 = r$ and $\exists \mu^{*}$ such that
(a) $\matP + \mu^{*}\matI \succeq 0$ and (b) $(\matP + \mu^{*}\matI) \vecx^{*} = -\vecb$. 
Moreover, if $\matP + \mu^{*}\matI \succ 0$, then the solution is unique.
\end{lemma}
The solution to (P1) and  (P2) is closely linked to solving the nonlinear equation $\norm{(\matP + \mu\matI)^{-1} \vecb}_2 = r$. Let
\begin{equation}
\vecx(\mu) = -(\matP + \mu\matI)^{-1} \vecb = -\sum_{j=1}^{n} \frac{\dotprod{\vecb}{\vecq_j}}{\mu + \mu_j} \vecq_j, 
\end{equation}
where $\mu_1 \leq \mu_2 \leq \dots \leq \mu_n$ and $\set{\vecq_j}$ are the eigenvalues and eigenvectors of 
$\matP$ respectively. Let us define $\phi(\mu)$ as  
\begin{equation}
\phi(\mu) := \norm{\vecx(\mu)}_2^2 = \sum_{j=1}^{n} \frac{\dotprod{\vecb}{\vecq_j}^2}{(\mu + \mu_j)^2}.
\end{equation}
Denoting $\calS = \set{\vecq \in \matR^n : \matP\vecq = \mu_1 \vecq}$, there are two cases to consider.
\begin{enumerate}
%
%
\item \underline{\textit{$\dotprod{\vecb}{\vecq} \neq 0$ for some $\vecq \in \calS$}}

This is the easy case. $\phi(\mu)$ has a pole at $-\mu_1$ and is monotonically decreasing 
in ($-\mu_1,\infty$) with $\lim_{\mu \rightarrow \infty} \phi(\mu) = 0$ and 
$\lim_{\mu \rightarrow -\mu_1} \phi(\mu) = \infty$. Hence there is a unique 
$\mu^{*} \in (-\mu_1,\infty)$ such that $\phi(\mu) = r^2$, 
and $\vecx(\mu^{*}) = -\sum_{j=1}^{n} \frac{\dotprod{\vecb}{\vecq_j}}{\mu^{*} + \mu_j} \vecq_j$ 
will be the unique solution to (P2). Some remarks are in order. 

\begin{itemize}
\item If $\matP$ was p.s.d and singular, then there is no solution to $\matP\vecx = -\vecb$, 
since $\vecb \not\in$ colspan($\matP$). Also, since $\mu_1 = 0$, we would 
have $\mu^{*} \in (0,\infty)$. Hence the corresponding solution $\vecx(\mu^{*})$ would be the same for 
(P1), (P2) and would lie on the boundary. Moreover, the solution will be unique due to Lemma \ref{lemma:trs_sphere}
since $\matP + \mu^{*}\matI \succ 0$. 

\item If $\matP$ was p.d and $\phi(0) < r^2$, then 
this would mean that the global solution to the unconstrained problem is a feasible point for (P1). 
In other words, $\vecx(0) = -\matP^{-1}\vecb$ would be the unique solution to (P1) with $\mu^{*} = 0$. 
Moreover, $\vecx(\mu^{*})$ with $\mu^{*} \in (-\mu_1,\infty)$ satisfying $\phi(\mu^{*}) = r^2$,  
would be the unique solution to (P2) (with $\mu^{*} < 0$); the uniqueness follows from Lemma \ref{lemma:trs_sphere} since 
$\matP + \mu^{*}\matI \succ 0$.
\end{itemize}
\item \underline{\textit{$\dotprod{\vecb}{\vecq} = 0$ for all $\vecq \in \calS$}}

This is referred to as the ``hard'' case in the literature - $\phi(\mu)$ does not  have a pole at $-\mu_1$, so $\phi(-\mu_1)$ is well defined.
There are two possibilities.
\begin{enumerate}
\item If $\phi(-\mu_1) \geq r^2$ then the solution is straightforward -- simply 
find the unique $\mu^{*} \in [-\mu_1,\infty)$ such that $\phi(\mu^{*}) = r^2$. 
This is possible since $\phi(\mu)$ is monotonically decreasing in $[-\mu_1,\infty)$. 
Hence,  
$$\vecx(\mu^{*}) = -\sum_{j: \mu_j \neq \mu_1} \frac{\dotprod{\vecb}{\vecq_j}}{\mu^{*} + \mu_j} \vecq_j$$
is the unique solution to (P2). If $\matP$ was p.s.d and singular, then $\mu_1 = 0$, and 
so $\mu^{*} \geq 0$. Hence, $\vecx(\mu^{*})$ would be the solution to both (P1) and (P2) and 
would lie on the boundary.

\item If $\phi(-\mu_1) < r^2$, then slightly more work is needed. For $\theta \in \matR$ and 
any $\vecz \in \calS$ with $\norm{\vecz}_2 = 1$, define 
$$\vecx(\theta) := \underbrace{-\sum_{j: \mu_j \neq \mu_1} \frac{\dotprod{\vecb}{\vecq_j}}{\mu_j - \mu_1} \vecq_j}_{\vecx(-\mu_1)} + \theta \vecz.$$
Then, it holds true that 
\begin{align}
\norm{\vecx(\theta)}_2^2 
&= \sum_{j: \mu_j \neq \mu_1} \frac{\dotprod{\vecb}{\vecq_j}^2}{(\mu_j - \mu_1)^2} + \theta^2 \\
&= \phi(-\mu_1) + \theta^2.
\end{align}
Solving $\norm{\vecx(\theta)}_2^2 = r^2$ for $\theta$, we see that for any solution $\theta^{*}$, we will also have 
$-\theta^{*}$ as a solution. Hence, $\vecx(\theta^{*}), \vecx(-\theta^{*})$ will be solutions to (P2) with $\mu^{*} = -\mu_1$. 
If $\matP$ was p.s.d and singular, then $\mu^{*} = 0$, and $\vecx(\pm \theta^{*}) = \vecx(0) \pm \theta^{*}\vecz$ would 
be solutions to both (P1) and (P2). Note that $\vecx(-\mu_1)$ is a solution to (P1). In fact, any 
point in the interior of the form $\vecx(-\mu_1) + \theta\vecz$ is a solution to (P1).
\end{enumerate}
\end{enumerate}


%

 \renewcommand\appendix{\par
        \renewcommand\thesection{C}
       \renewcommand\thesubsection{C\arabic{subsection}}
        \renewcommand\thetable{C\arabic{table} } }        

\appendix

\section{Useful concentration inequalities} \label{sec:app_conc_ineq}
We present in this section some useful concentration inequalities, that will be employed 
as a tool to prove our other results. Recall that for a random variable $X$, its sub-Gaussian norm 
$\norm{X}_{\psi_2}$ is defined as 
\begin{equation}
\norm{X}_{\psi_2} := \sup_{p \geq 1} \frac{(\expec \abs{X}^p)^{1/p}}{\sqrt{p}}.
\end{equation}
Moreover, $X$ is a sub-Gaussian random variable if $\norm{X}_{\psi_2}$ is finite. 
For instance, consider a bounded random variable $X$ with $\abs{X} \leq M$. Then, $X$ 
is a sub-Gaussian random variable with $\norm{X}_{\psi_2} \leq M$ \cite[Example 5.8]{vershynin2012}.

We begin with the well-known Hanson-Wright inequality (\cite{Hanson71}) for concentration of random quadratic forms. 
The following version is taken from \cite{rudelson2013}.
%
\begin{theorem}[\cite{rudelson2013}] \label{thm:hanson_wright} 
Let $(X_1 \ X_2 \ \cdots \ X_n) \in \matR^{n}$ be a random vector with independent components $X_i$ 
satisfying $\expec[X_i] = 0$ and $\norm{X_i}_{\psi_2} \leq K$. Let $\matA$ be a $n \times n$ matrix. 
Then for every $t \geq 0$
\begin{equation}
\mathbb{P}(\abs{X^T \matA X - \expec[X^T \matA X]} \geq t) \leq 
2\exp\left( -c \min \left(\frac{t^2}{K^4\norm{\matA}_F^2} , \frac{t}{K^2\norm{\matA}}\right)  \right)
\end{equation}
for an absolute constant $c > 0$.
\end{theorem}
Next, we recall the following Hoeffding type inequality for sums of independent sub-Gaussian random variables.
%
\begin{proposition} \cite[Proposition 5.10]{vershynin2012} \label{prop:hoeff_subgauss_conc}
Let $X_1,\dots,X_n$ be independent  centered sub-Gaussian random variables and let $K = \max_i \norm{X_i}_{\psi_2}$. 
Then for every $\veca \in \matR^n$, and every $t \geq 0$, we have
\begin{equation}
\prob(\abs{\sum_{i=1}^n a_i X_i} \geq t) \leq e \cdot \exp\left(-\frac{c^{\prime} t^2}{K^2 \norm{\veca}_2^2}\right), 
\end{equation}
where $c^{\prime} > 0$ is an absolute constant.
\end{proposition}


 \renewcommand\appendix{\par
        \renewcommand\thesection{D}
       \renewcommand\thesubsection{D\arabic{subsection}}
        \renewcommand\thetable{D\arabic{table} } }        

\appendix

\section{Proof of Proposition \ref{prop:bern_unif_conc}} \label{sec:proof_prop_bern_unif}
Let us first recall the definition of $\veczbar$ from \eqref{eq:real_notat}. For clarity of notation, 
we will denote $\veczbar_R = \real(\vecz) \in \matR^n$ and $\veczbar_I = \imag(\vecz) \in \matR^n$, and thus 
$\veczbar = [\veczbar_R^T \quad \veczbar_I^T]^T \in \matR^{2n}$. Clearly, $(z_i)_{i=1}^n$ are 
independent, complex-valued random variables. Note that 
\begin{equation}
(\veczbar_R)_i = \cos(2\pi y_i) = \left\{
\begin{array}{rl}
\cos(2\pi (f_i\bmod 1)) \quad ; & \text{if} \ \beta_i = 0 \\
\cos(2\pi u_i) \quad ; & \text{if} \ \beta_i = 1
\end{array} \right. \quad ; \quad i=1,\dots,n.
\end{equation}
and 
\begin{equation}
(\veczbar_I)_i = \sin(2\pi y_i) = \left\{
\begin{array}{rl}
\sin(2\pi (f_i\bmod 1)) \quad ; & \text{if} \ \beta_i = 0 \\
\sin(2\pi u_i) \quad ; & \text{if} \ \beta_i = 1
\end{array} \right. \quad ; \quad i=1,\dots,n.
\end{equation}
Since $(\beta_i)_{i=1}^{n}$ and $(u_i)_{i=1}^{n}$ are i.i.d random variables, hence the components of $\veczbar_R$ are independent 
real-valued random variables. The same is true for the components of $\veczbar_I$.
\begin{enumerate}
\item \underline{\textbf{Lower bounding $\frac{1}{2n} \veczbar^T \Hbar \veczbar$}} 

To begin with, note that $\veczbar^T \Hbar \veczbar = \veczbar_R^T (\lambda L) \veczbar_R + \veczbar_I^T (\lambda L) \veczbar_I$. 
Denote $\mean_R = \expec[\veczbar_R] \in \matR^n$, and $\mean_I = \expec[\veczbar_I] \in \matR^n$. 
We see that 
\begin{align}
(\mean_R)_i &= \expec[(\veczbar_R)_i] \\
&= (1-p)\cos(2\pi (f_i\bmod 1)) + p\expec[\cos(2\pi u_i)] \\
&= (1-p)\cos(2\pi (f_i\bmod 1)) ,
\end{align}
since $\expec[\cos(2\pi u_i)] = 0$ for $i=1,\dots,n$. 
Similarly, we have that  $(\mean_I)_i = (1-p)\sin(2\pi (f_i\bmod 1))$. Hence, we may write 
\begin{equation} \label{eq:bern_unif_temp0}
\mean_R = (1-p) \real(\vechtil) \quad \text{and} \quad \mean_I = (1-p) \imag(\vechtil). 
\end{equation}
We now focus on lower bounding the term $\veczbar_R^T L \veczbar_R$. We observe that 
\begin{align}
\veczbar_R^T L \veczbar_R
&= (\veczbar_R - \mean_R + \mean_R)^T L (\veczbar_R - \mean_R + \mean_R) \\
&= (\veczbar_R - \mean_R)^T L (\veczbar_R - \mean_R) + 2(\veczbar_R - \mean_R)^T L \mean_R + \mean_R^T L \mean_R. \label{eq:bern_unif_temp1}
\end{align}
%
The first two terms in \eqref{eq:bern_unif_temp1} are random, and we proceed to lower bound them w.h.p starting with the first term.  Let us note that 
\begin{align}
\expec[(\veczbar_R - \mean_R)^T L (\veczbar_R - \mean_R)] 
&= \sum_{i=1}^n \expec[(\veczbar_R)_i - (\mean_R)_i]^2 L_{ii} \nonumber \\
		&+  \sum_{i \neq j} \underbrace{\expec[((\veczbar_R)_i - (\mean_R)_i)((\veczbar_R)_j - (\mean_R)_j)]}_{ = 0} L_{ij} \label{eq:bern_unif_tempA} \\
&= \sum_{i=1}^n (\expec[(\veczbar_R)_i]^2 - (\mean_R)_i^2) \deg(i),  \label{eq:bern_unif_temp2}
\end{align}
since the cross terms in \eqref{eq:bern_unif_tempA} are zero. We now obtain
%
%
\begin{align}
\expec[(\veczbar_R)_i]^2 &= \expec\left[\frac{1 + \cos(4\pi y_i)}{2}\right] \nonumber \\
&=  \frac{1}{2} + \frac{1}{2}\left[(1-p)\cos(4\pi (f_i \bmod 1)) + p\expec[\cos (4\pi u_i)] \right] \nonumber \\
&= \frac{1}{2} + \frac{1}{2}(1-p) \cos(4\pi(f_i\bmod 1)). \label{eq:bern_unif_temp3}
\end{align}
Concerning the second term in \eqref{eq:bern_unif_temp2}, we note that 
\begin{align} 
 (\mean_R)_i^2 &= (1-p)^2\cos^2(2\pi(f_i \bmod 1)) = \frac{(1-p)^2}{2} + \frac{(1-p)^2}{2} \cos(4\pi(f_i \bmod 1)), \label{eq:bern_unif_temp4}
\end{align}
%
%
for $i=1,\dots,n$. Plugging \eqref{eq:bern_unif_temp3}, \eqref{eq:bern_unif_temp4} in \eqref{eq:bern_unif_temp2}, and 
observing that $k \leq \deg(i) \leq 2k$, one can easily verify that 
\begin{align}
\frac{pnk}{2} \leq \expec[(\veczbar_R - \mean_R)^T L (\veczbar_R - \mean_R)] \leq 3pnk.  \label{eq:bern_unif_temp5}
\end{align}
%
%
Next, for each $i = 1,\dots,n$, the random variables $(\veczbar_R)_i - (\mean_R)_i$ are zero mean, and  also uniformly bounded since 
\begin{align}
\abs{(\veczbar)_i - (\mean_R)_i} = \abs{\cos(2\pi y_i) - (1-p)\cos(2\pi(f_i \bmod 1))} \leq 2.
\end{align}
Hence $\norm{(\veczbar)_i - (\mean_R)_i}_{\psi_2} \leq 2$ for each $i$. Therefore, applying Hanson-Wright inequality to 
$(\veczbar_R - \mean_R)^T L (\veczbar_R - \mean_R)$ yields
%
\begin{align}
\prob(\abs{(\veczbar_R - \mean_R)^T L (\veczbar_R - \mean_R) - \expec[(\veczbar_R - \mean_R)^T L (\veczbar_R - \mean_R)]} \geq t) \nonumber \\
\leq 2\exp\left( -c \min \left(\frac{t^2}{16\norm{L}_F^2}, \frac{t}{4\norm{L}}\right) \right). \label{eq:bern_unif_temp6}
\end{align}
Since $\deg(i) \leq 2k$ for each $i=1,\dots,n$, therefore Gershgorins disk theorem yields the estimate $\norm{L} \leq 4k$. 
Moreover, $\norm{L}_F^2 \leq \sum_i ((2k)^2 + 2k) \leq 8k^2n$. Plugging $t = \frac{\varepsilon p n k}{2}$ for $\varepsilon \in (0,1)$, 
we observe that
\begin{align} 
\frac{t^2}{16\norm{L}_F^2} &\geq \frac{(\varepsilon^2 p^2 n^2 k^2)/4}{128 k^2 n} = \frac{\varepsilon^2 p^2 n}{512}, \\
\text{and} \quad \frac{t}{4\norm{L}} &\geq \frac{(\varepsilon p n k)/2}{16 k} = \frac{\varepsilon p n}{32}. 
\end{align}
Thus, $\min\set{\frac{t^2}{16\norm{L}_F^2}, \frac{t}{4\norm{L}}} \geq \frac{\varepsilon^2 p^2 n}{512}$. Plugging this estimate 
in \eqref{eq:bern_unif_temp6} for the aforementioned choice of $t$, and using the bounds in \eqref{eq:bern_unif_temp5}, 
we have with probability at least $1 - 2\exp\left(-\frac{c \varepsilon^2 p^2 n}{512}\right)$ that
%
\begin{align}
\frac{pnk}{2}(1-\varepsilon) \leq (\veczbar_R - \mean_R)^T L (\veczbar_R - \mean_R) \leq pnk\left(3 + \frac{\varepsilon}{2}\right). \label{eq:bern_unif_temp7}
\end{align}
We now turn our attention to the second term in \eqref{eq:bern_unif_temp1} namely $2(\veczbar_R - \mean_R)^T L \mean_R$.
Recall that $(\veczbar_R)_i - (\mean_R)_i$ are independent, zero-mean sub-Gaussian random variables with $\norm{(\veczbar_R)_i - (\mean_R)_i}_{\psi_2} \leq 2$, 
for each $i$. Hence,  invoking Proposition \ref{prop:hoeff_subgauss_conc} yields 
\begin{align}
\prob(\abs{\dotprod{\veczbar_R - \mean_R}{2 L \mean_R}} \geq t) 
&\leq e\cdot \exp\left(-\frac{c^{\prime} t^2}{4 \norm{2 L \mean_R}_2^2}\right) \\ 
&\leq e\cdot \exp\left(-\frac{c^{\prime} t^2}{16 \norm{L}^2 \norm{\mean_R}_2^2}\right) \\ 
&\leq e\cdot \exp\left(-\frac{c^{\prime} t^2}{256 k^2 n}\right), \label{eq:bern_unif_temp8}
\end{align}
where in \eqref{eq:bern_unif_temp8}, we used the bounds $\norm{L} \leq 4k$ (as shown earlier) and $\norm{\mean_R}_2^2 \leq n$ 
(using \eqref{eq:bern_unif_temp0}). 
Plugging $t = \frac{pnk(1-\varepsilon)}{3}$ in \eqref{eq:bern_unif_temp8}, we have that the following 
holds with probability at least $1 - e \cdot \exp\left(-\frac{c^{\prime} p^2 n (1-\varepsilon)^2}{2304}\right)$
%
\begin{align}
\abs{\dotprod{\veczbar_R - \mean_R}{2 L \mean_R}} \leq \frac{pnk(1-\varepsilon)}{3}. \label{eq:bern_unif_temp9}
\end{align}
Combining \eqref{eq:bern_unif_temp7}, \eqref{eq:bern_unif_temp9} and applying the union bound, we have with 
probability at least $1 - e \cdot\exp\left(-\frac{c^{\prime} p^2 n (1-\varepsilon)^2}{2304}\right) - 2\exp\left(-\frac{c \varepsilon^2 p^2 n}{512}\right)$
that the following bound holds
\begin{equation} 
\veczbar_R^T L \veczbar_R \geq \frac{pnk}{6}(1-\varepsilon) + \mean_R^T L \mean_R. \label{eq:bern_unif_temp10}
\end{equation}
By proceeding as above, one can verify that with 
probability at least $1 - e \cdot \exp\left(-\frac{c^{\prime} p^2 n (1-\varepsilon)^2}{2304}\right)$ $- 2\exp\left(-\frac{c \varepsilon^2 p^2 n}{512}\right)$, 
the following bound holds
\begin{equation} 
\veczbar_I^T L \veczbar_I \geq \frac{pnk}{6}(1-\varepsilon) + \mean_I^T L \mean_I. \label{eq:bern_unif_temp11}
\end{equation}
Combining \eqref{eq:bern_unif_temp10}, \eqref{eq:bern_unif_temp11} and applying the union bound, we have 
with probability at least $1 - 2e\cdot\exp\left(-\frac{c^{\prime} p^2 n (1-\varepsilon)^2}{2304}\right) - 4\exp\left(-\frac{c \varepsilon^2 p^2 n}{512}\right)$ 
that 
%
\begin{align}
\frac{1}{2n} \veczbar^T \Hbar \veczbar 
&= \frac{1}{2n} \lambda(\veczbar_R^T L \veczbar_R + \veczbar_I^T L \veczbar_I) \\
&\geq \frac{\lambda p n k}{6n}(1-\varepsilon) + \frac{1}{2n}\left(\mean_R^T (\lambda L) \mean_R + \mean_I^T (\lambda L) \mean_I \right) \label{eq:bern_unif_temp12}  \\
&= \frac{\lambda p k}{6}(1-\varepsilon) + (1-p)^2 \frac{1}{2n} \vechtilbar \Hbar \vechtilbar \label{eq:bern_unif_temp13}
\end{align}
holds true. To go from \eqref{eq:bern_unif_temp12} to \eqref{eq:bern_unif_temp13}, we used \eqref{eq:bern_unif_temp0} 
along with the definition of $\Hbar$ (see also Appendix \ref{sec:qcqp_compl_to_real}). 
This completes the derivation of the lower bound on $\frac{1}{2n} \veczbar^T \Hbar \veczbar$.

%
\item \underline{\textbf{Upper bounding $\norm{\veczbar - \vechtilbar}_2$}}
By recalling the definition of $\veczbar,\vechtilbar \in \matR^{2n}$ from \eqref{eq:real_notat}, we note that
\begin{align}
\norm{\veczbar - \vechtilbar}_2^2 
&= \left|\left|\begin{pmatrix}
  \real(\vecz) \\ \imag(\vecz) 
 \end{pmatrix} 
- \begin{pmatrix}
  \real(\vechtil) \\ \imag(\vechtil) 
 \end{pmatrix}\right|\right|_2^2 \\
&= \norm{\vecz - \vechtil}_2^2 \\
&= 2n - (\vecz^{*}\vechtil + \vechtil^{*}\vecz) \qquad (\text{Since } \abs{z_i}, \abs{h_i} = 1, \text{for each } i) \\
&= 2n - \sum_{i=1}^{n}(\exp(\iota 2\pi(f_i\bmod 1 - y_i)) + \exp(-\iota 2\pi(f_i\bmod 1 - y_i))) \\
&= 2n - \sum_{i=1}^{n}\underbrace{(2\cos (2\pi(f_i \bmod 1 - y_i)))}_{M_i}. \label{eq:nois_conc_temp0}
\end{align}
Clearly, $(M_i)_{i=1}^{n}$ are independent sub-Gaussian random variables with $\abs{M_i} \leq 2$;  
hence $\norm{M_i}_{\psi_2} \leq 2$ for each $i$. Moreover,
\begin{align}
\expec[M_i] 
&= 2(1-p) + p\expec[(2\cos (2\pi(f_i \bmod 1 - u_i)))] \\
&= 2(1-p) + p\int_{0}^{1} 2\cos (2\pi(f_i \bmod 1 - u_i)) du_i \\
&= 2(1-p) + p \left[\frac{2\sin (2\pi u_i - 2\pi (f_i \bmod 1))}{2\pi} \right]_{0}^{1} \\
&= 2(1-p).
\end{align}
Therefore,  $(M_i - \expec[M_i])_{i=1}^n$ are centered, independent sub-Gaussian random variables with $\norm{M_i - \expec[M_i]}_{\psi_2} \leq 4$. 
Thus,  by applying Proposition \ref{prop:hoeff_subgauss_conc} to $\sum_{i=1}^n (M_i - \expec[M_i])$, we obtain
\begin{equation} \label{eq:nois_conc_temp1}
\prob(\abs{\sum_i M_i - 2n(1-p)} \geq t) \leq e\cdot\exp\left(-\frac{c^{\prime} t^2}{16n}\right).
\end{equation}
Plugging $t = 2(1-p)n\varepsilon$ in \eqref{eq:nois_conc_temp1} for $\varepsilon \in (0,1)$, we have with probability 
at least $1 - e\cdot\exp\left(-\frac{c^{\prime} (1-p)^2 \varepsilon^2}{4n}\right)$ that the following bound holds
\begin{equation} \label{eq:nois_conc_temp2}
2(1-p)n(1-\varepsilon) \leq \sum_{i=1}^n M_i \leq 2(1-p)n(1+\varepsilon). 
\end{equation}
Conditioning on the event in \eqref{eq:nois_conc_temp2}, we finally obtain 
from \eqref{eq:nois_conc_temp0} the bound
\begin{align}
\norm{\veczbar - \vechtilbar}_2^2 \ 
\leq 2n - 2(1-p)n(1-\varepsilon) \ 
\leq 2n(p + \varepsilon).
\end{align}
\end{enumerate}

%

 \renewcommand\appendix{\par
        \renewcommand\thesection{E}
       \renewcommand\thesubsection{E\arabic{subsection}}
        \renewcommand\thetable{E\arabic{table} } }        

\appendix

\section{Proof of Proposition \ref{prop:gauss_conc}} \label{sec:proof_prop_gauss}
We consider the noise model  $y_i = (f_i + \eta_i) \mod 1$, where $\eta_i \sim \calN(0,\sigma^2)$ are i.i.d. 
Recall the definition of $\veczbar$ from \eqref{eq:real_notat}, where upon denoting 
$\veczbar_R = \real(\vecz) \in \matR^n$, $\veczbar_I = \imag(\vecz) \in \matR^n$, we have  
$\veczbar = [\veczbar_R^T \quad \veczbar_I^T]^T \in \matR^{2n}$. Clearly $(\veczbar_R)_i = \cos(2\pi((f_i + \eta_i) \bmod 1))$, 
$i=1,\dots,n$ are independent random variables, the same being true for $(\veczbar_I)_i = \sin(2\pi((f_i + \eta_i) \bmod 1))$. 

\begin{enumerate}
\item \underline{\textbf{Lower bounding $\frac{1}{2n} \veczbar^T \Hbar \veczbar$.}} 
To begin with, note that $\veczbar^T \Hbar \veczbar = \veczbar_R^T (\lambda L) \veczbar_R + \veczbar_I^T (\lambda L) \veczbar_I$. 
Denote $\mean_R = \expec[\veczbar_R] \in \matR^n$, and $\mean_I = \expec[\veczbar_I] \in \matR^n$. We then have
\begin{align}
(\mean_R)_i &= \expec[\cos(2\pi((f_i + \eta_i) \bmod 1))] \\
&= \expec[\frac{e^{\iota 2\pi((f_i + \eta_i) \bmod 1)} + e^{-\iota 2\pi((f_i + \eta_i) \bmod 1)}}{2}] \\
&= \frac{e^{\iota 2\pi(f_i \bmod 1)}\expec[e^{\iota 2\pi \eta_i}] + e^{-\iota 2\pi(f_i \bmod 1)} \expec[e^{-\iota 2\pi \eta_i}]}{2}.  \label{eq:app_gauss_temp1}
\end{align}
Using the fact\footnote{Indeed, for $X \sim \calN(\mu,\sigma^2)$, we have 
$\expec[e^{\iota t X}] = e^{\iota t \mu - \frac{\sigma^2 t^2}{2}}$ for any $t \in \matR$.} 
$\expec[e^{\iota 2\pi \eta_i}], \expec[e^{-\iota 2\pi \eta_i}] = e^{-2\pi^2 \sigma^2}$ in \eqref{eq:app_gauss_temp1}, we get 
$(\mean_R)_i = e^{-2\pi^2 \sigma^2} \cos(2\pi(f_i \bmod 1))$. In an analogous manner, one can show that 
$(\mean_I)_i = e^{-2\pi^2 \sigma^2} \sin(2\pi(f_i \bmod 1))$. To summarize, we have shown that 
\begin{align} \label{eq:app_gauss_temp2}
\mean_R = e^{-2\pi^2 \sigma^2} \real(\vechtil), \quad \mean_I = e^{-2\pi^2 \sigma^2} \imag(\vechtil). 
\end{align}
Now recall that $\veczbar^T \Hbar \veczbar = \veczbar_R^T (\lambda L) \veczbar_R + \veczbar_I^T (\lambda L) \veczbar_I$. 
We will lower bound $\veczbar_R^T (\lambda L) \veczbar_R$, respectively $\veczbar_I^T (\lambda L) \veczbar_I$ individually, w.h.p., via similar arguments as for the Bernoulli noise, via a Hanson-Wright inequality, respectively, a Hoeffding type inequality arguments.  

Consider $\veczbar_R^T L \veczbar_R$, and recall its decomposition as in \eqref{eq:bern_unif_temp1}. We will 
lower bound $(\veczbar_R - \mean_R)^T L (\veczbar_R - \mean_R)$ and  $2(\veczbar_R - \mean_R)^T L \mean_R$ w.h.p. 

To begin with, recall from \eqref{eq:bern_unif_temp2} that 
\begin{align}
\expec[(\veczbar_R - \mean_R)^T L (\veczbar_R - \mean_R)] = \sum_{i=1}^n (\expec[(\veczbar_R)_i]^2 - (\mean_R)_i^2) \deg(i).
\label{eg:quad_form_in_L_centered}
\end{align}
Next, using the following observation
\begin{align}
	\expec[(\veczbar_R)_i]^2 = \expec[  \cos^2( 2\pi ((f_i + \eta_i) \bmod  1) )  ] = \expec[  \cos^2( 2\pi (f_i + \eta_i) )  ] 
	= \expec \left[  \frac{ 1+ \cos ( 4\pi (f_i + \eta_i) ) }{2} \right]
\end{align}
together with the fact that  $ \expec [ \cos ( 4\pi (f_i + \eta_i) )] = e^{-8 \pi ^2 \sigma^2} \cos ( 4\pi (f_i  \bmod 1) )  $, 
we remark that, for each $i = 1,\dots,n$, it holds true that 
%
%
\begin{align}
	\expec[(\veczbar_R)_i]^2 
	= \frac{1 + e^{-8\pi^2 \sigma^2} \cos(4\pi(f_i \bmod 1))}{2}. 
	\label{eq:expec_veczbar_R}
\end{align}
Note that the last equality follows along similar steps as those in the derivation of $(\mean_R)_i$ earlier. A similar calculation yields that
\begin{align}
(\mean_R)_i^2 = e^{-4 \pi^2 \sigma^2}  \cos^2( 2\pi ( f_i \bmod 1 )  ) = 
e^{-4 \pi^2 \sigma^2} \left[  \frac{1 + \cos ( 4 \pi ( f_i \bmod 1 ) )}{2}  \right].
\label{eq:mean_R_i_sq}
\end{align}
Combining   \eqref{eq:expec_veczbar_R} with  \eqref{eq:mean_R_i_sq} yields
\begin{align}
	\expec[(\veczbar_R)_i]^2 - (\mean_R)_i^2
	&= \frac{(1 - e^{-4\pi^2 \sigma^2})  + (e^{-8\pi^2 \sigma^2} - e^{-4\pi^2 \sigma^2})\cos(4\pi(f_i \bmod 1))}{2} \\
	&= \frac{(1 - e^{-4\pi^2 \sigma^2})(1 - e^{-4\pi^2 \sigma^2}\cos(4\pi(f_i \bmod 1))) }{2}.
\end{align}
Since $k \leq \deg(i) \leq 2k$, we readily conclude that \eqref{eg:quad_form_in_L_centered} amounts to  
\begin{align} \label{eq:app_gauss_temp3}
\frac{(1-e^{-4\pi^2 \sigma^2})^2}{2}kn \leq \expec[(\veczbar_R - \mean_R)^T L (\veczbar_R - \mean_R)] \leq (1-e^{-8\pi^2 \sigma^2})kn.	
\end{align}
Note that for each $i = 1,\dots,n$, the random variables $(\veczbar_R)_i - (\mean_R)_i$ are zero-mean, and are also uniformly bounded  as 
\begin{align}
\abs{(\veczbar)_i - (\mean_R)_i} = \abs{\cos(2\pi ((f_i + \eta_i)\bmod 1)) - e^{-2\pi^2\sigma^2}\cos(2\pi(f_i \bmod 1))} \leq 2.
\end{align}
Hence $\norm{(\veczbar)_i - (\mean_R)_i}_{\psi_2} \leq 2$ for each $i$, therefore allowing us to  apply the Hanson-Wright inequality to 
$(\veczbar_R - \mean_R)^T L (\veczbar_R - \mean_R)$ yields
\begin{align}
\prob(\abs{(\veczbar_R - \mean_R)^T L (\veczbar_R - \mean_R) - \expec[(\veczbar_R - \mean_R)^T L (\veczbar_R - \mean_R)]} \geq t) \nonumber \\
\leq 2\exp\left( -c \min \left(\frac{t^2}{16\norm{L}_F^2}, \frac{t}{4\norm{L}}\right) \right). \label{eq:app_gauss_temp4}
\end{align}
We saw earlier that $\norm{L} \leq 4k$ (see the proof of Lemma \ref{lem:lowbd_sol_quad_form}).  Since $L_{ii} \leq 2k $, and $L_{i,j} \leq 1,   \mbox{ for } i,j \in E$, it holds true that $\norm{L}_F^2 \leq 8k^2n$.
Plugging in  $t = \varepsilon k n \frac{(1-e^{-4\pi^2\sigma^2})^2}{2}$ for $\varepsilon \in (0,1)$
we observe that
\begin{align} 
\frac{t^2}{16\norm{L}_F^2} &\geq \frac{(\varepsilon^2 (1-e^{-4\pi^2\sigma^2})^4  n^2 k^2)/4}{128 k^2 n} = \frac{\varepsilon^2 (1-e^{-4\pi^2\sigma^2})^4 n}{1024}, \\
\text{and} \quad \frac{t}{4\norm{L}} &\geq \frac{(\varepsilon (1-e^{-4\pi^2\sigma^2})^2 n k)/2}{16 k} = \frac{\varepsilon (1-e^{-4\pi^2\sigma^2})^2 n}{32}. 
\end{align}
Thus, $\min\set{\frac{t^2}{16\norm{L}_F^2}, \frac{t}{4\norm{L}}} \geq \frac{\varepsilon^2 (1-e^{-4\pi^2\sigma^2})^4 n}{1024}$. 
Plugging this estimate in \eqref{eq:app_gauss_temp4} for the aforementioned choice of $t$, and using the bounds in \eqref{eq:app_gauss_temp3}, 
we have with probability at least $1 - 2\exp\left(-\frac{c \varepsilon^2 (1-e^{-4\pi^2\sigma^2})^4 n}{1024}\right)$ that
\begin{align} \label{eq:app_gauss_temp5}
	(1-\varepsilon)\frac{(1-e^{-4\pi^2\sigma^2})^2}{2} kn 
	\leq (\veczbar_R - \mean_R)^T L (\veczbar_R - \mean_R) 
	\leq kn\left[(1-e^{-8\pi^2\sigma^2}) + \frac{\varepsilon}{2}(1-e^{-4\pi^2\sigma^2})^2\right]. 
\end{align}
Next, we turn our attention to lower-bounding the second term, namely $2(\veczbar_R - \mean_R)^T L \mean_R$.
Recall that $(\veczbar_R)_i - (\mean_R)_i$ are independent, zero-mean sub-Gaussian random variables with 
$\norm{(\veczbar_R)_i - (\mean_R)_i}_{\psi_2} \leq 2$, 
for each $i$. Hence,  invoking Proposition \ref{prop:hoeff_subgauss_conc} along with the facts $\norm{L} \leq 4k$ and 
$\norm{\mean_R}_2^2 \leq e^{-4\pi^2\sigma^2} n$, we obtain
\begin{align}
\prob(\abs{\dotprod{\veczbar_R - \mean_R}{2 L \mean_R}} \geq t) 
\leq e\cdot \exp\left(-\frac{c^{\prime} t^2}{256 k^2 \norm{\mean_R}_2^2}\right) 
\leq  e\cdot \exp\left(-\frac{c^{\prime} t^2 e^{4\pi^2\sigma^2}}{256 k^2 n}\right).
\end{align}
Plugging in $t = \frac{(1-\varepsilon)(1-e^{-4\pi^2\sigma^2})^2}{3}kn$, we have with probability at least 
$$1-e\exp\left(\frac{-c'(1-\varepsilon)^2(1-e^{-4\pi^2\sigma^2})^4 e^{4\pi^2\sigma^2}}{2304} n \right)$$ that 
\begin{align} \label{eq:app_gauss_temp6}
	\abs{\dotprod{\veczbar_R - \mean_R}{2 L \mean_R}} \leq \frac{(1-\varepsilon)(1-e^{-4\pi^2\sigma^2})^2}{3} kn.
\end{align}
Recalling the decomposition of $\veczbar_R^T L \veczbar_R$ as in \eqref{eq:bern_unif_temp1}, 
combining the lower bound from \eqref{eq:app_gauss_temp5} with the lower bound implied by \eqref{eq:app_gauss_temp6}, and applying the union bound, we have that, with probability at least 
\begin{align}
	1-e\exp\left(\frac{-c'(1-\varepsilon)^2(1-e^{-4\pi^2\sigma^2})^4 e^{4\pi^2\sigma^2}}{2304} n \right) 
	- 2\exp\left(-\frac{c \varepsilon^2 (1-e^{-4\pi^2\sigma^2})^4 n}{1024}\right)
\end{align}
the following bound holds
\begin{align} \label{eq:gauss_1d_real_quad_bd}
\veczbar_R^T L \veczbar_R \geq \frac{(1-\varepsilon)(1-e^{-4\pi^2\sigma^2})^2}{6} kn + \mean_R^T L \mean_R.
\end{align}
By proceeding as above, we obtain the same lower bound on $\veczbar_I^T L \veczbar_I$ with the same lower bound on the 
success probability. Hence, by applying the union bound to \eqref{eq:gauss_1d_real_quad_bd} and its imaginary counterpart, we 
finally have with probability at least 
\begin{align} \label{eq:app_gauss_temp7}
		1-2e\exp\left(\frac{-c'(1-\varepsilon)^2(1-e^{-4\pi^2\sigma^2})^4 e^{4\pi^2\sigma^2}}{2304} n \right) 
	- 4\exp\left(-\frac{c \varepsilon^2 (1-e^{-4\pi^2\sigma^2})^4 n}{1024}\right)
\end{align}

that the following bound holds 
\begin{align}
	\frac{1}{2n} \veczbar^T \Hbar \veczbar 
	&= \frac{1}{2n} \lambda(\veczbar_R^T L \veczbar_R + \veczbar_I^T L \veczbar_I) \\
	&\geq \frac{\lambda}{2n} \left[\frac{(1-\varepsilon)(1-e^{-4\pi^2\sigma^2})^2}{3} kn \right] 
	+ \frac{1}{2n} (\mean_R^T (\lambda L) \mean_R + \mean_I^T (\lambda L) \mean_I) \\
	&=  \frac{\lambda k}{6} (1-\varepsilon)(1-e^{-4\pi^2\sigma^2})^2 + \frac{e^{-4\pi^2\sigma^2}}{2n} \vechtil^T \Hbar \vechtil. \label{eq:app_gauss_temp8}
\end{align}
In the last step, we used \eqref{eq:app_gauss_temp2} along with the definition of $\Hbar$ (see also Appendix \ref{sec:qcqp_compl_to_real}). 
This completes the derivation of the lower bound on $\frac{1}{2n} \veczbar^T \Hbar \veczbar$. 

\item \underline{\textbf{Upper bounding $\norm{\veczbar - \vechtilbar}_2$.}}
By recalling the definition of $\veczbar,\vechtilbar \in \matR^{2n}$ from \eqref{eq:real_notat}, we observe that
\begin{align}
\norm{\veczbar - \vechtilbar}_2^2 
&= \norm{\vecz - \vechtil}_2^2 \\
&= 2n - (\vecz^{*}\vechtil + \vechtil^{*}\vecz) \\
&= 2n - \sum_{i=1}^{n}[e^{-\iota 2\pi((f_i + \eta_i)\bmod 1 - f_i\bmod 1)} + e^{\iota 2\pi((f_i + \eta_i)\bmod 1 - f_i\bmod 1)}] \\
&= 2n - \sum_{i=1}^{n}[e^{\iota 2\pi(\eta_i \bmod 1)} + e^{-\iota 2\pi (\eta_i \bmod 1)}] 
= 2n - \sum_{i=1}^{n}\underbrace{[e^{\iota 2\pi\eta_i} + e^{-\iota 2\pi \eta_i}]}_{M_i}. \label{eq:app_gauss_temp9}
\end{align}
Clearly $(M_i)_{i=1}^{n}$ are i.i.d real-valued sub-Gaussian random variables with $\abs{M_i} \leq 2$;  
hence $\norm{M_i}_{\psi_2} \leq 2$ for each $i$. Moreover, we readily obtain $\expec[M_i] = 2e^{-2\pi^2\sigma^2}$, and consequently, 
$(M_i - \expec[M_i])_{i=1}^n$ are centered, i.i.d sub-Gaussian random variables with $\norm{M_i - \expec[M_i]}_{\psi_2} \leq 4$. 
Therefore, by applying Proposition \ref{prop:hoeff_subgauss_conc} to $\sum_{i=1}^n (M_i - \expec[M_i])$, we obtain
\begin{equation} \label{eq:app_gauss_temp10}
\prob(\abs{\sum_i M_i - 2ne^{-2\pi^2\sigma^2}} \geq t) \leq e\cdot\exp\left(-\frac{c^{\prime} t^2}{16n}\right).
\end{equation}
Plugging in $t = 2 n \varepsilon e^{-2\pi^2\sigma^2}$ for $\varepsilon \in (0,1)$, we have with probability 
at least $1 - e\cdot \exp(-\frac{c' e^{-4\pi^2\sigma^2} n \varepsilon^2}{4})$ that the following bound holds
\begin{equation} \label{eq:app_gauss_temp11}
 2 n \varepsilon e^{-2\pi^2\sigma^2} (1-\varepsilon) \leq \sum_{i=1}^n M_i \leq  2 n \varepsilon e^{-2\pi^2\sigma^2} (1+\varepsilon). 
\end{equation}
Conditioning on the event in \eqref{eq:app_gauss_temp11}, we finally obtain from \eqref{eq:app_gauss_temp9} the bound
\begin{align}
	\norm{\veczbar - \vechtilbar}_2^2 \leq 2n - 2n(1-\varepsilon)e^{-2\pi^2\sigma^2} = 2n(1 - (1-\varepsilon)e^{-2\pi^2\sigma^2}).
\end{align}

\end{enumerate}

 \renewcommand\appendix{\par
        \renewcommand\thesection{F}
       \renewcommand\thesubsection{F\arabic{subsection}}
        \renewcommand\thetable{F\arabic{table} } }        

\appendix

\section{Additional numerical experiments}

\subsection{Numerical experiments: Bounded Model} \label{sec:num_exps_Bounded}

Figure \ref{fig:instances_f1_Bounded} shows several denoising instances as we increase the noise level in the Uniform noise model ($\gamma \in \{0.15, 0.27, 0.30\}$). Note that \textbf{OLS} starts failing at $\gamma = 0.27$, 
while \textbf{QCQP} still estimates the samples of $f$ well. Interestingly, \textbf{iQCQP} performs quite well, 
even for $\gamma = 0.30$ (where \textbf{QCQP} starts failing) and produces highly smooth, and accurate estimates.
It would be interesting to investigate the properties of \textbf{iQCQP} in future work.

Figures \ref{fig:Sims_f1_Bounded_fmod1}, \ref{fig:Sims_f1_Bounded_f} plot RMSE (on a log scale) for denoised 
$f$ mod 1 and $f$ samples versus the noise level, for the Uniform noise model. They illustrate the importance of 
the choice of the regularization parameters $\lambda, k$. If $\lambda$ is too small (eg., $\lambda = 0.03$), 
then \textbf{QCQP} has negligible improvement in performance, and sometimes also has worse RMSE than the raw noisy samples. 
However, for a larger $\lambda$ ($\lambda \in \set{0.3,0.5}$), \textbf{QCQP} has a strictly smaller error than $\textbf{OLS}$ and the raw noisy 
samples. Interestingly, \textbf{iQCQP} typically performs very well, even for $\lambda = 0.03$.


Figure \ref{fig:Sims_f1_Bounded_ScanID2_ffmod1} plots the RMSE (on a log scale) for both the denoised $f$ mod 1 samples, and samples of $f$, 
versus $n$ (for Uniform noise model). Observe that for large enough $n$, \textbf{QCQP} shows strictly smaller RMSE than both the initial input noisy data, and \textbf{OLS}. Furthermore, we remark that \textbf{iQCQP} typically has superior performance to \textbf{QCQP} except for small values of $n$.

\begin{figure}[!ht]
\centering
\subcaptionbox[]{  $\gamma=0.15$, \textbf{OLS}
}[ 0.32\textwidth ]
{\includegraphics[width=0.27\textwidth] {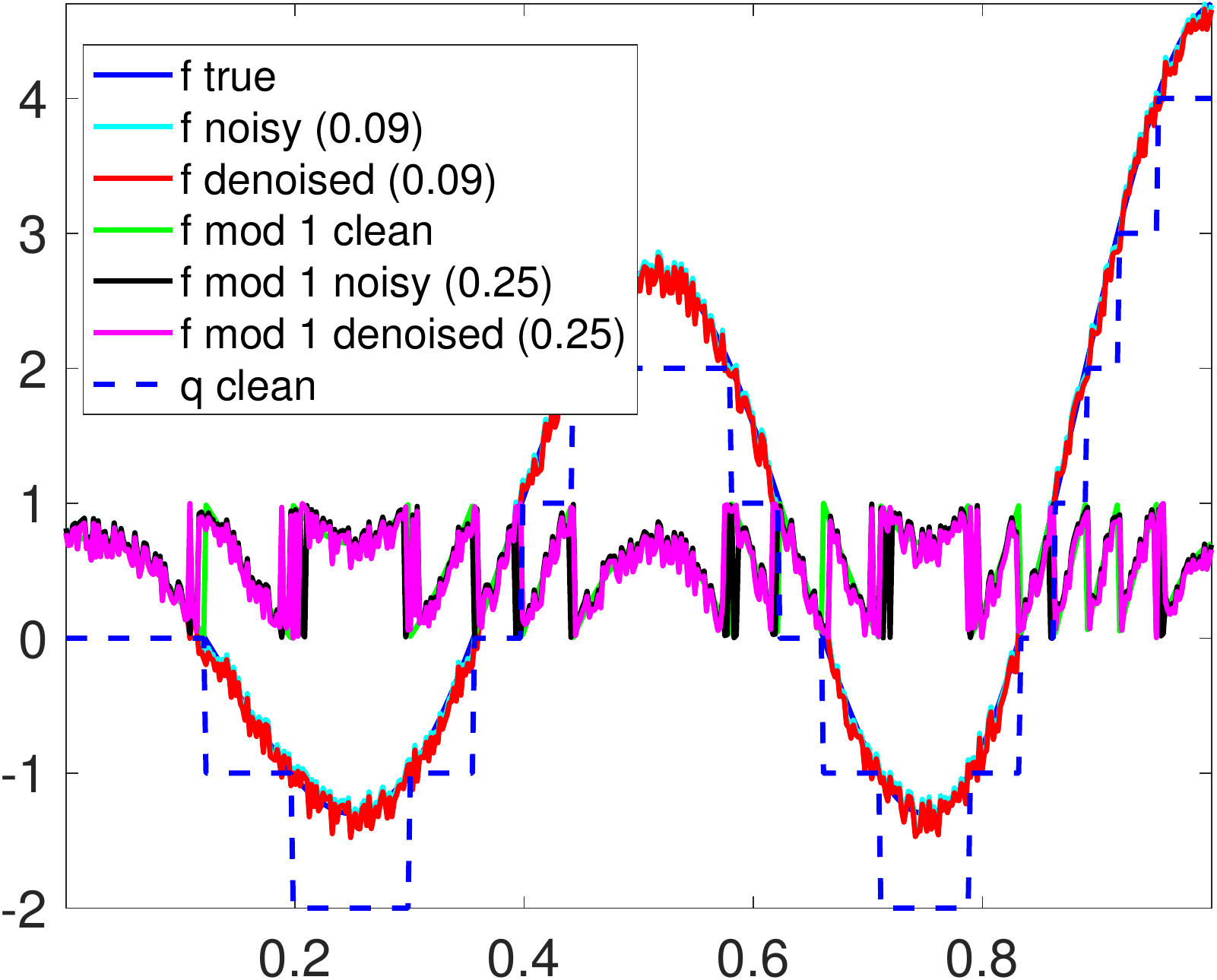} }
%
\subcaptionbox[]{  $\gamma=0.15$, \textbf{QCQP}
}[ 0.32\textwidth ]
{\includegraphics[width=0.27\textwidth] {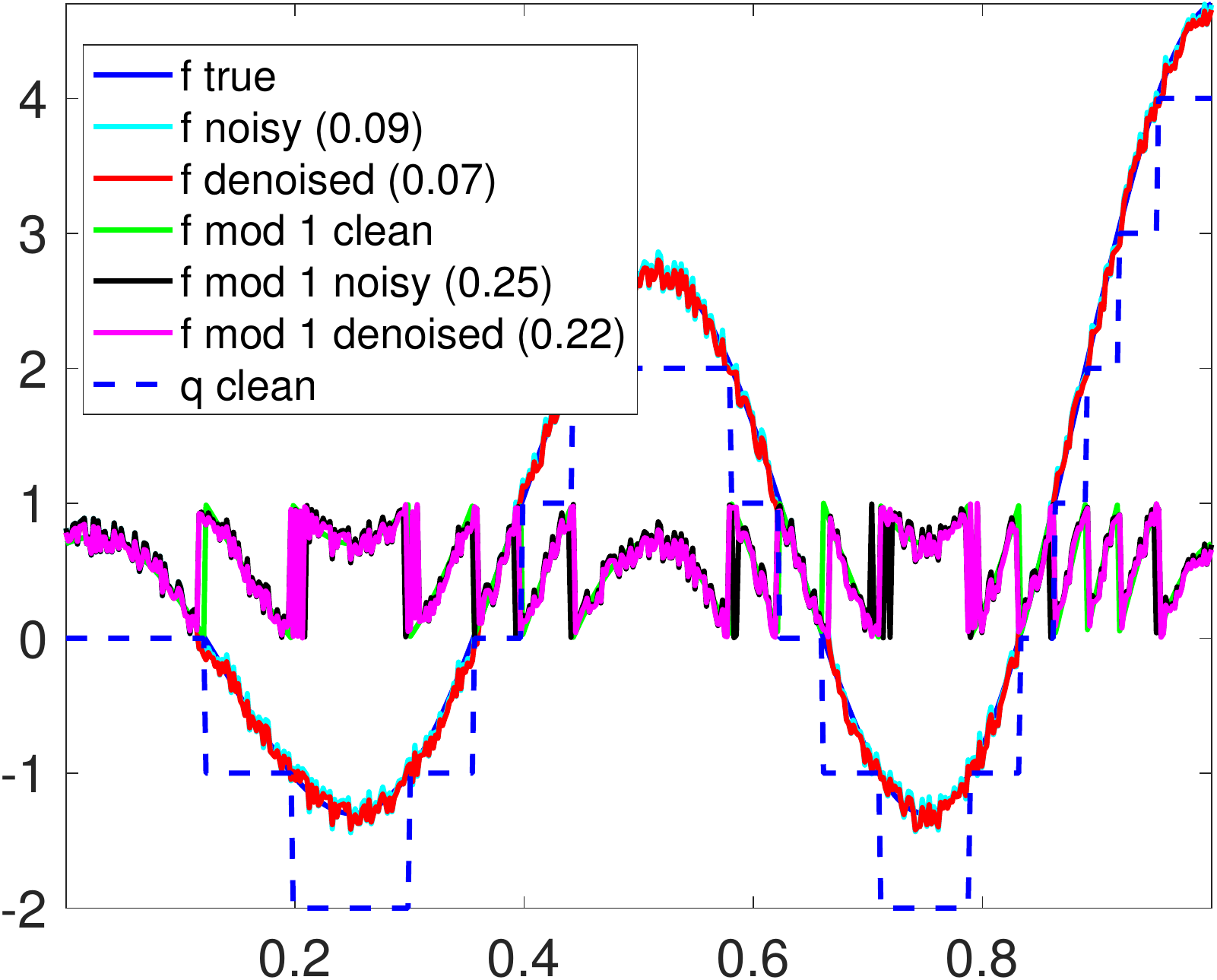} }
%
\subcaptionbox[]{  $\gamma=0.15$, \textbf{iQCQP}  (10  iters.)
}[ 0.32\textwidth ]
{\includegraphics[width=0.27\textwidth] {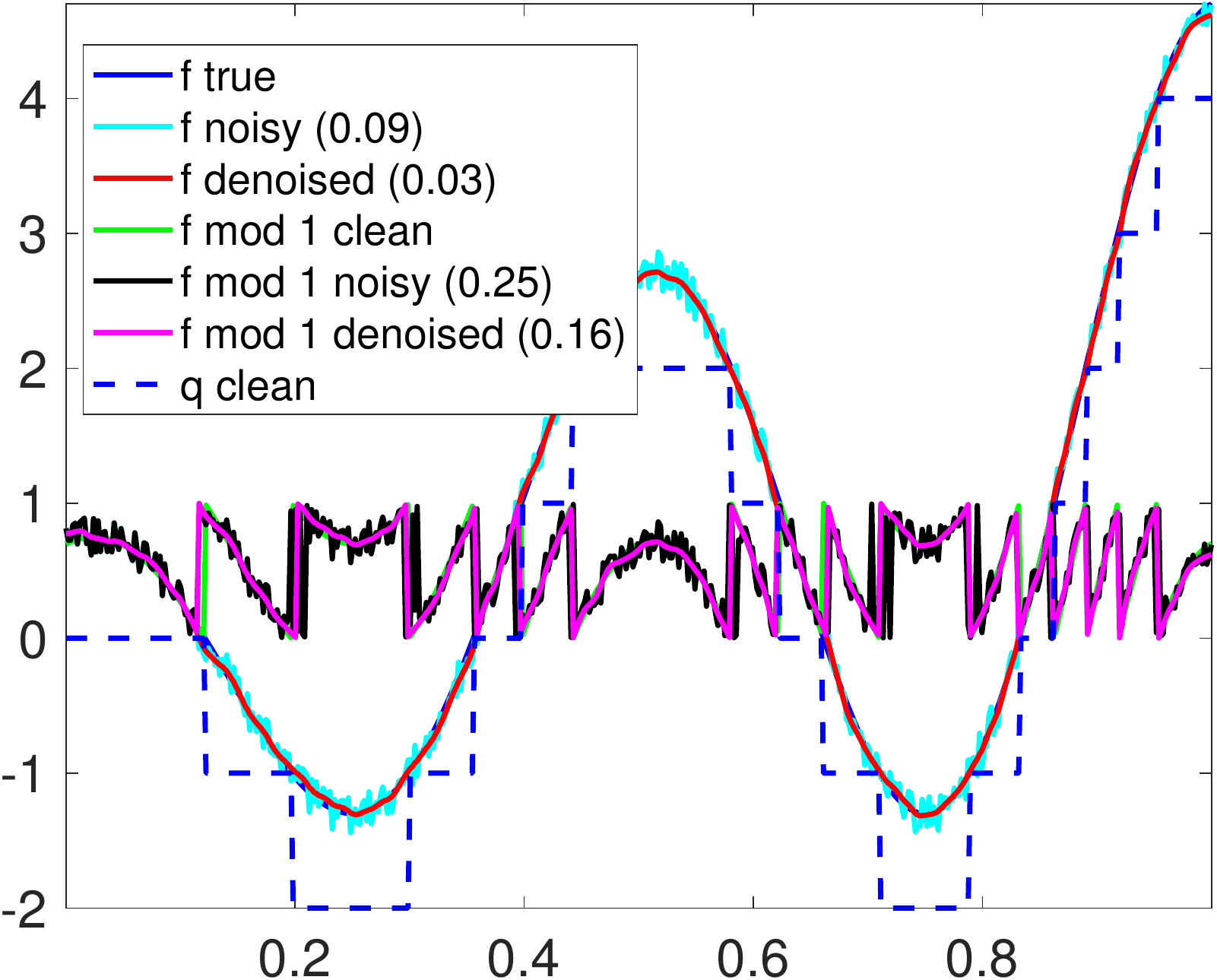} }
%
%

\vspace{4mm}
\subcaptionbox[]{  $\gamma=0.27$, \textbf{OLS}
}[ 0.32\textwidth ]
{\includegraphics[width=0.27\textwidth] {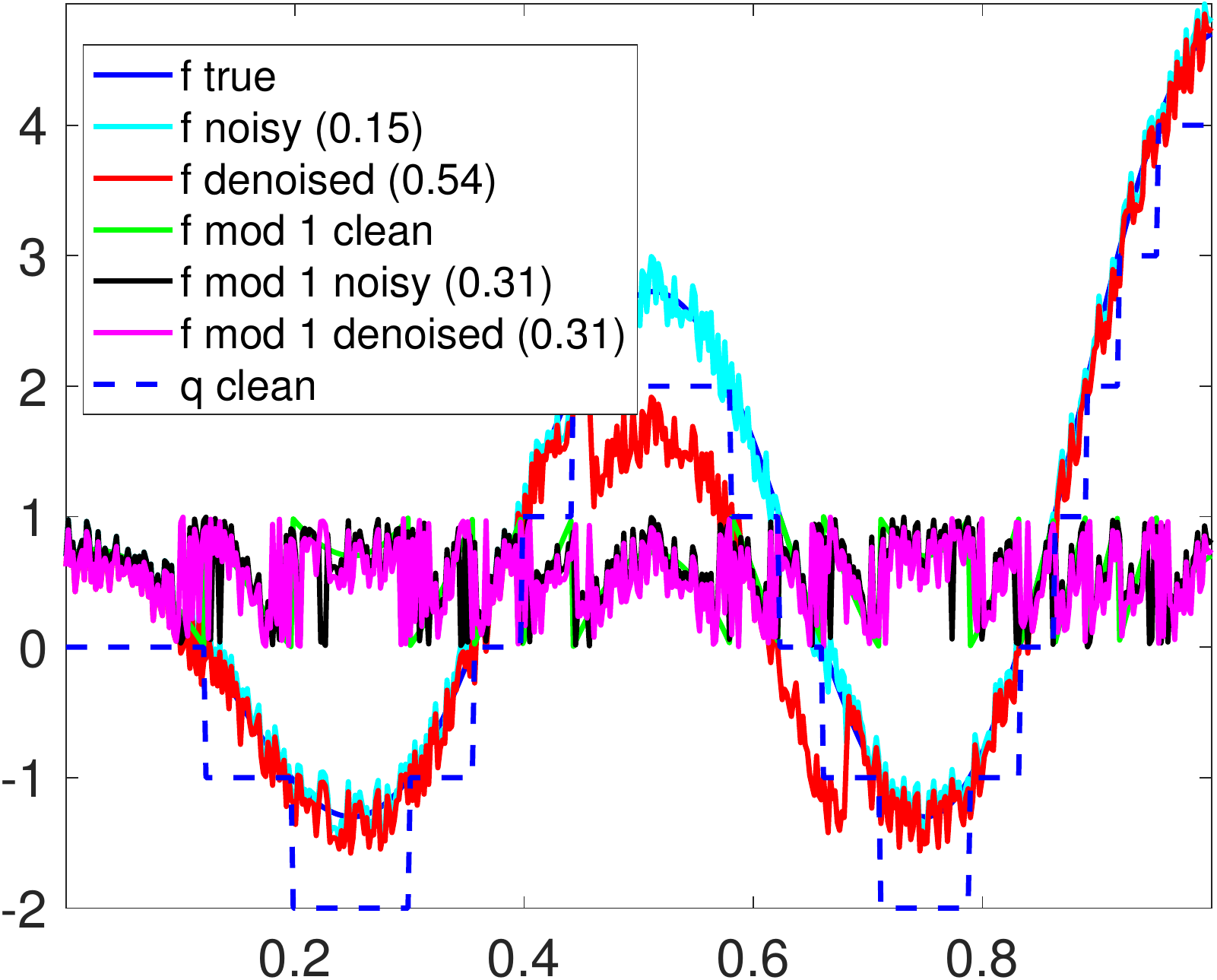} }
%
\subcaptionbox[]{  $\gamma=0.27$, \textbf{QCQP}
}[ 0.32\textwidth ]
{\includegraphics[width=0.27\textwidth] {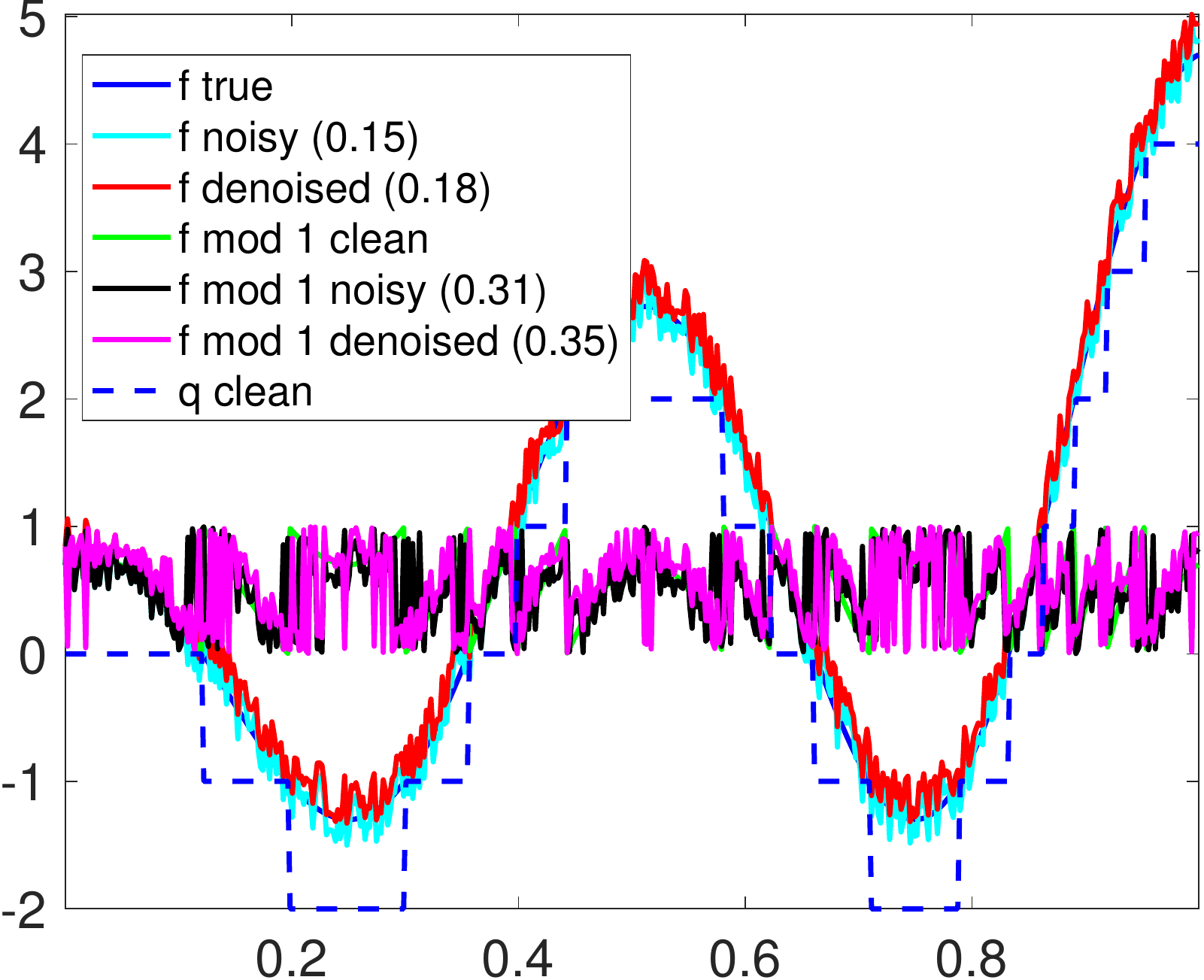} }
%
\subcaptionbox[]{  $\gamma=0.27$, \textbf{iQCQP} (10 iters.)
}[ 0.32\textwidth ]
{\includegraphics[width=0.27\textwidth] {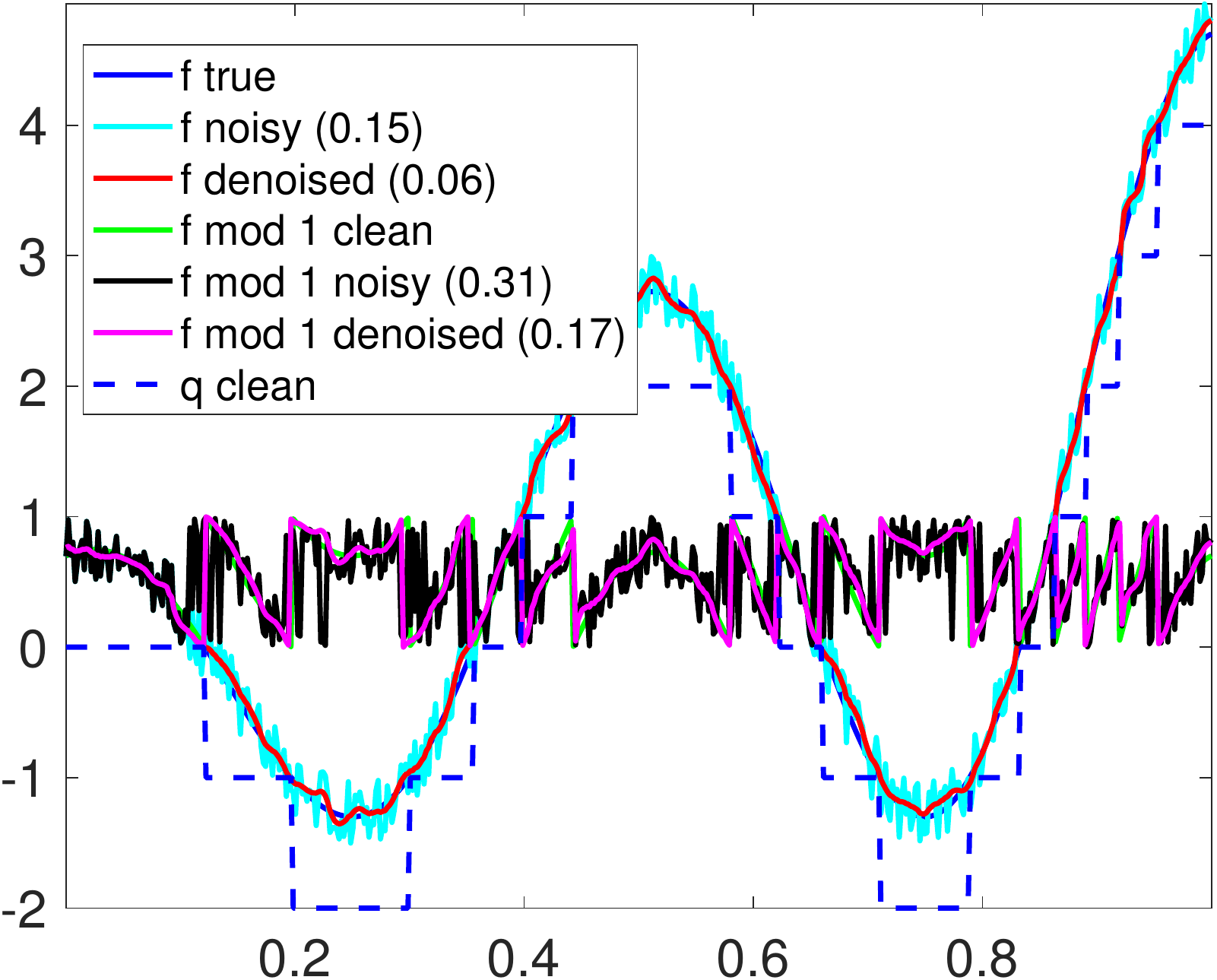} }
%
%

\vspace{4mm}
\subcaptionbox[]{  $\gamma=0.30$, \textbf{OLS}
}[ 0.32\textwidth ]
{\includegraphics[width=0.27\textwidth] {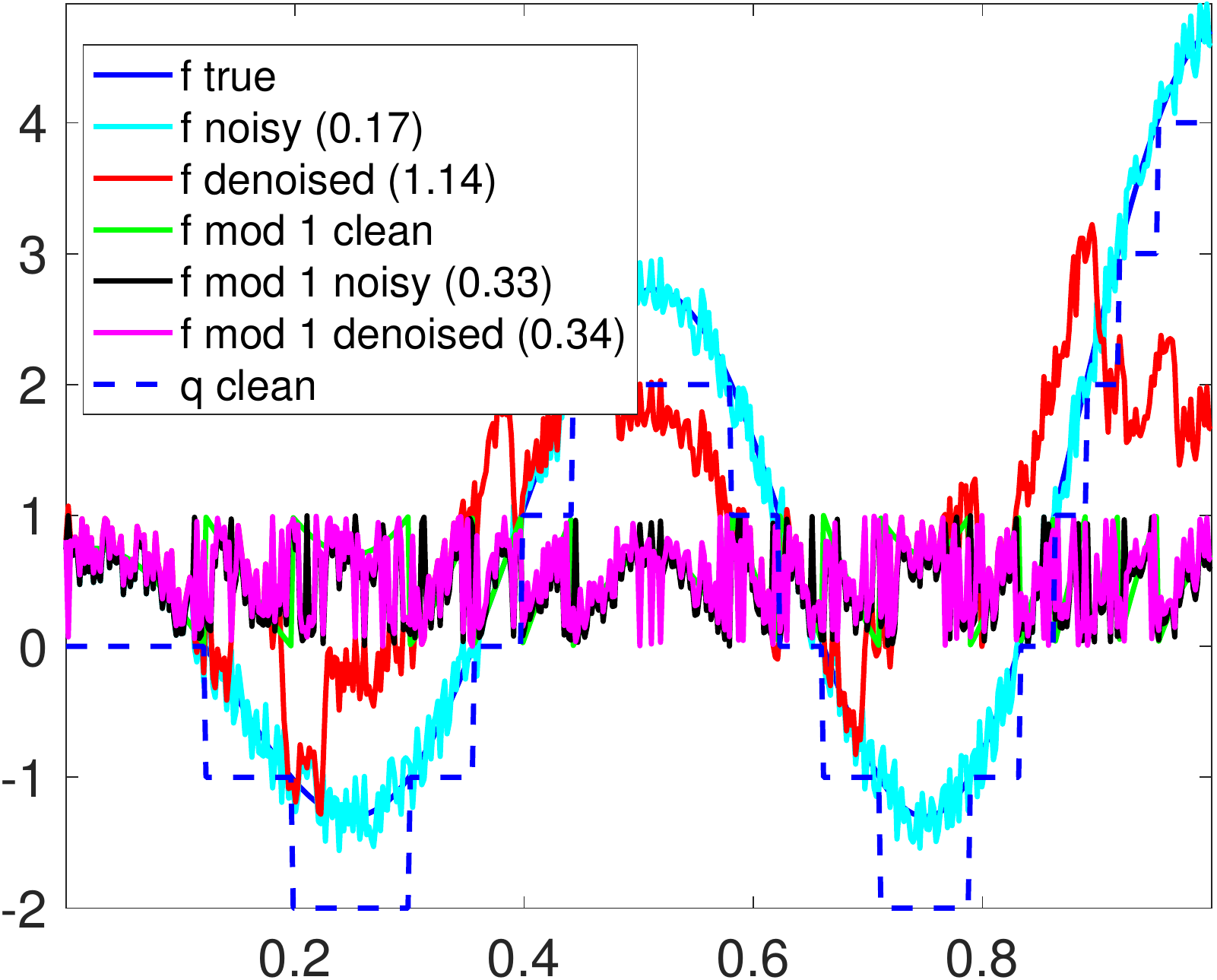} }
%
\subcaptionbox[]{  $\gamma=0.30$, \textbf{QCQP}
}[ 0.32\textwidth ]
{\includegraphics[width=0.27\textwidth] {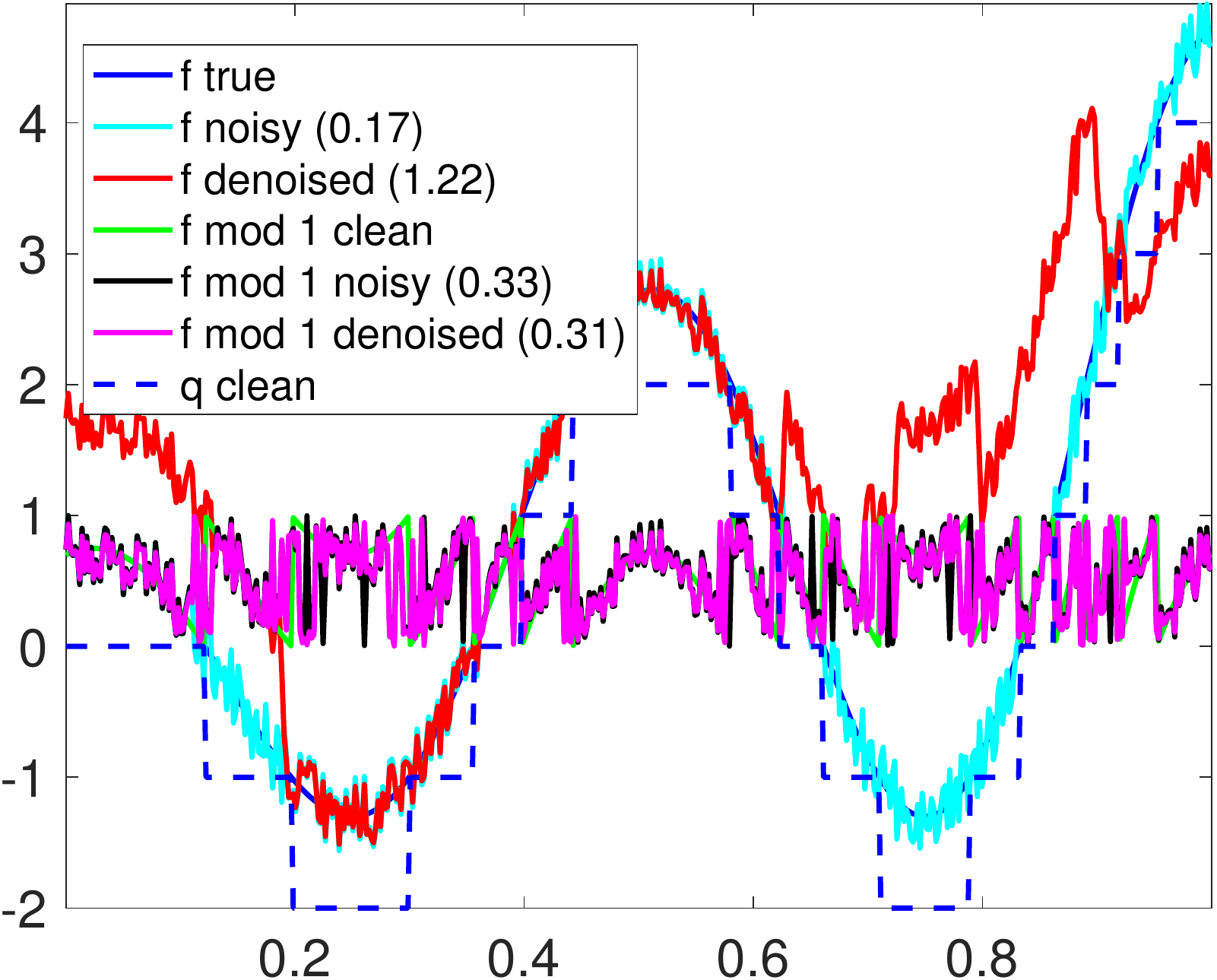} }
%
\subcaptionbox[]{  $\gamma=0.30$, \textbf{iQCQP}  (10 iters.)
}[ 0.32\textwidth ]
{\includegraphics[width=0.27\textwidth] {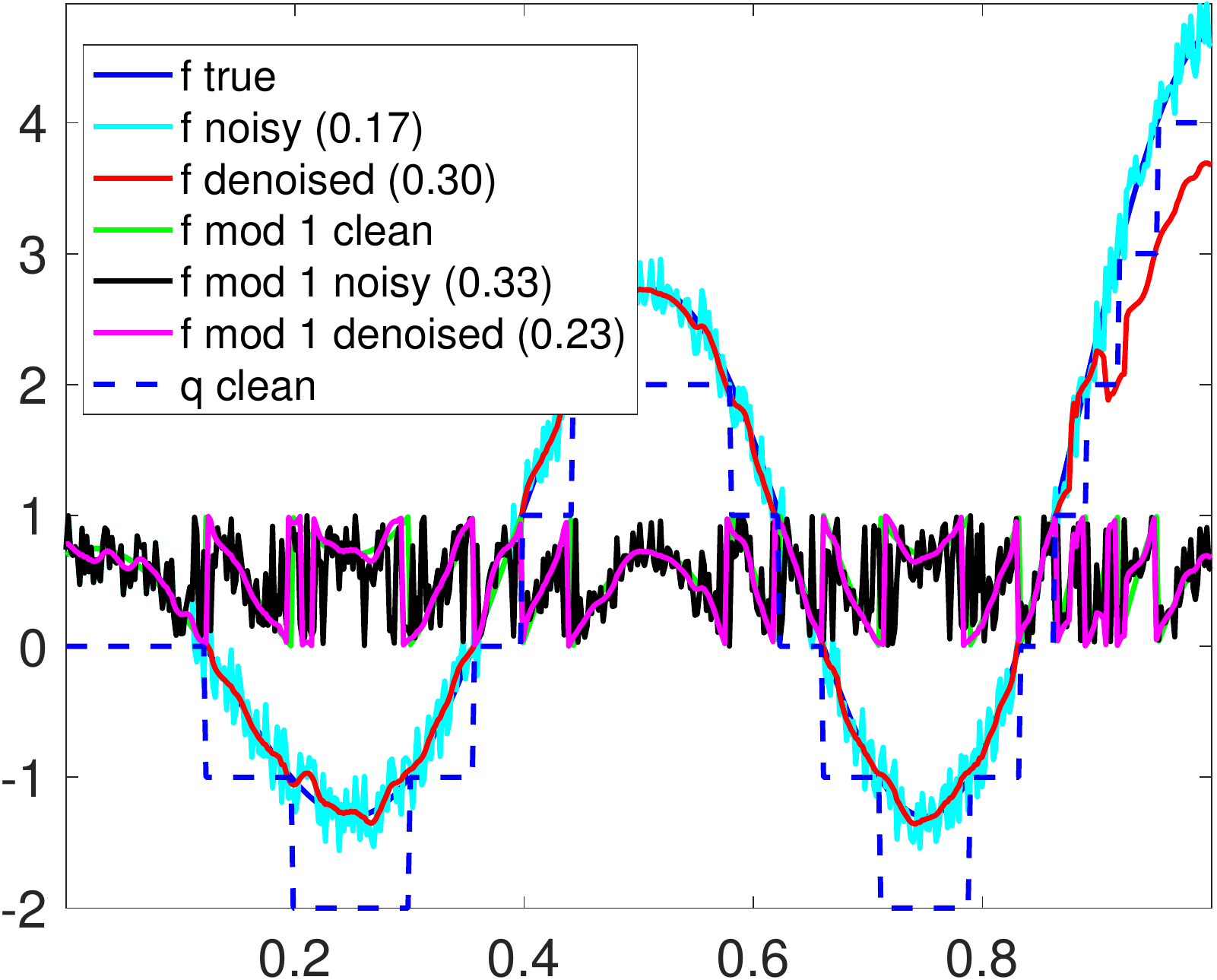} }
%
%
%
%
%
\vspace{-2mm}
\captionsetup{width=0.98\linewidth}
\caption[Short Caption]{Denoised instances under the Uniform noise model, for \textbf{OLS}, \textbf{QCQP} and \textbf{iQCQP},  as we increase the noise level $\gamma$. We keep fixed the parameters $n=500$, $k=2$, $\lambda= 0.1$. The numerical values in the legend denote the RMSE. \textbf{QCQP} denotes Algorithm \ref{algo:two_stage_denoise}, for which  the unwrapping stage is  performed via \textbf{OLS} \eqref{eq:ols_unwrap_lin_system}.
}
\label{fig:instances_f1_Bounded}
\end{figure}

\begin{figure}[!ht]
\centering
\subcaptionbox[]{ $k=2$, $\lambda= 0.03$}[ 0.19\textwidth ]
{\includegraphics[width=0.19\textwidth] {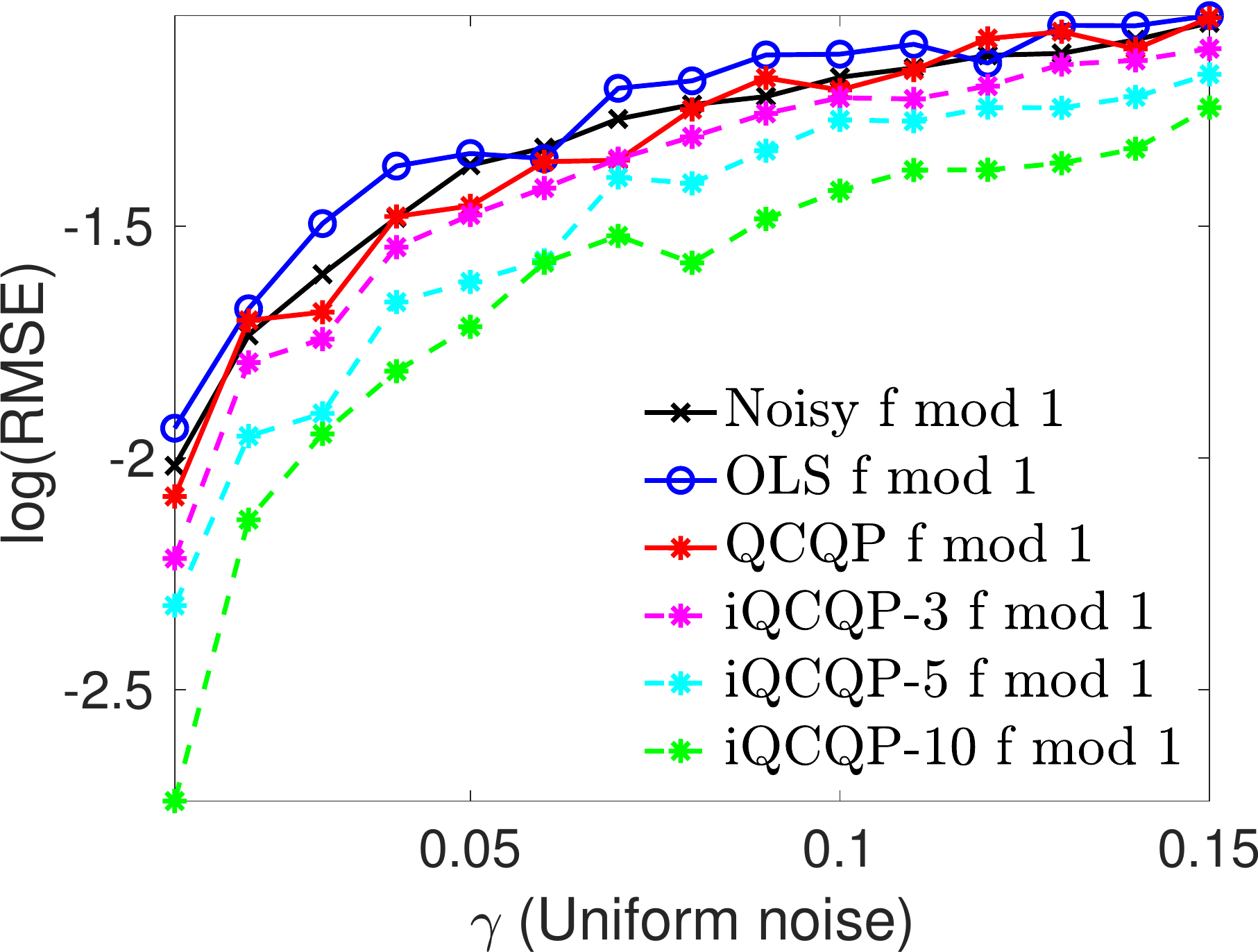} }
\subcaptionbox[]{ $k=2$, $\lambda= 0.1$}[ 0.19\textwidth ]
{\includegraphics[width=0.19\textwidth] {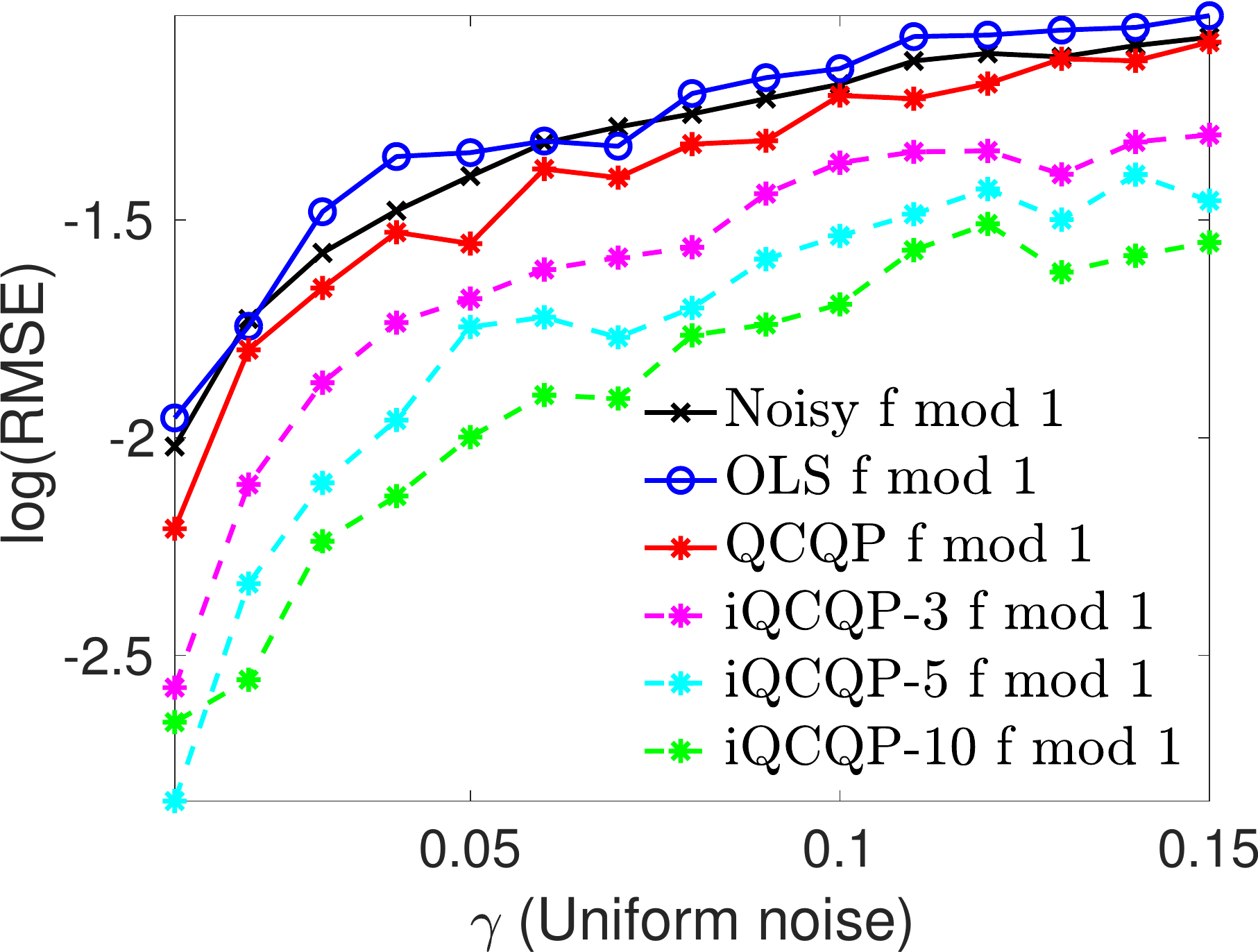} }
%
\subcaptionbox[]{ $k=2$, $\lambda= 0.3$}[ 0.19\textwidth ]
{\includegraphics[width=0.19\textwidth] {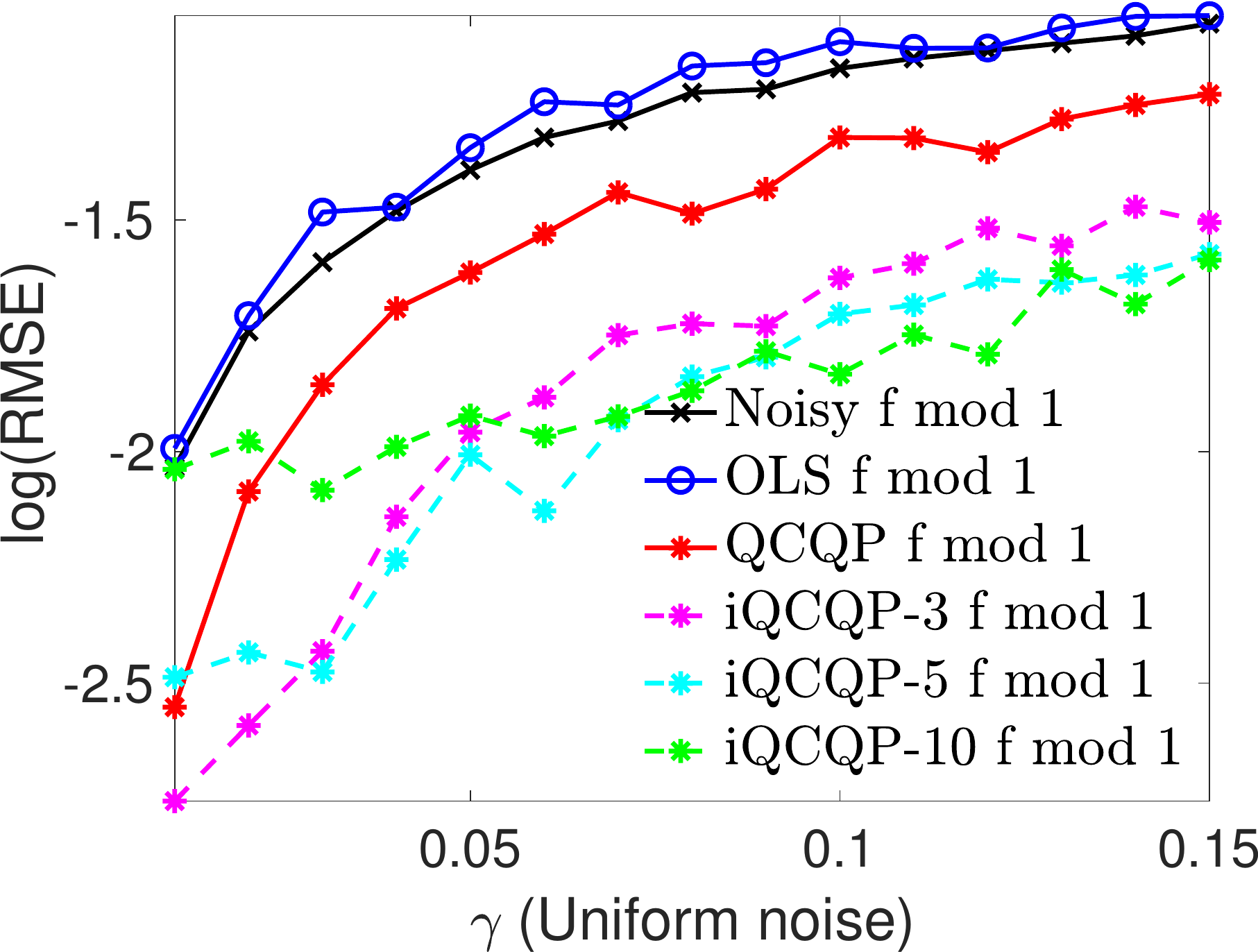} }
%
\subcaptionbox[]{ $k=2$, $\lambda= 0.5$}[ 0.19\textwidth ]
{\includegraphics[width=0.19\textwidth] {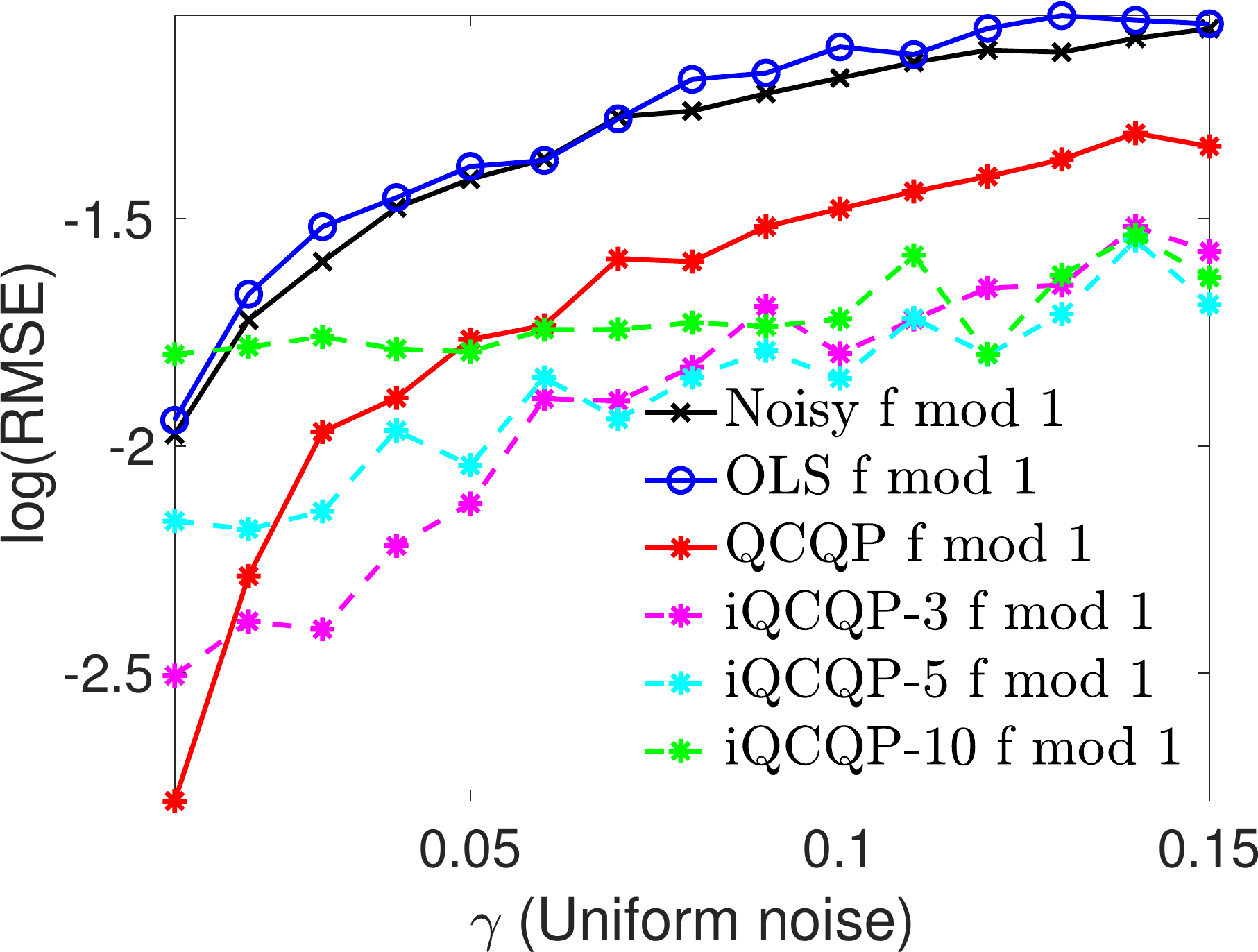} }
%
\subcaptionbox[]{ $k=2$, $\lambda= 1$}[ 0.19\textwidth ]
{\includegraphics[width=0.19\textwidth] {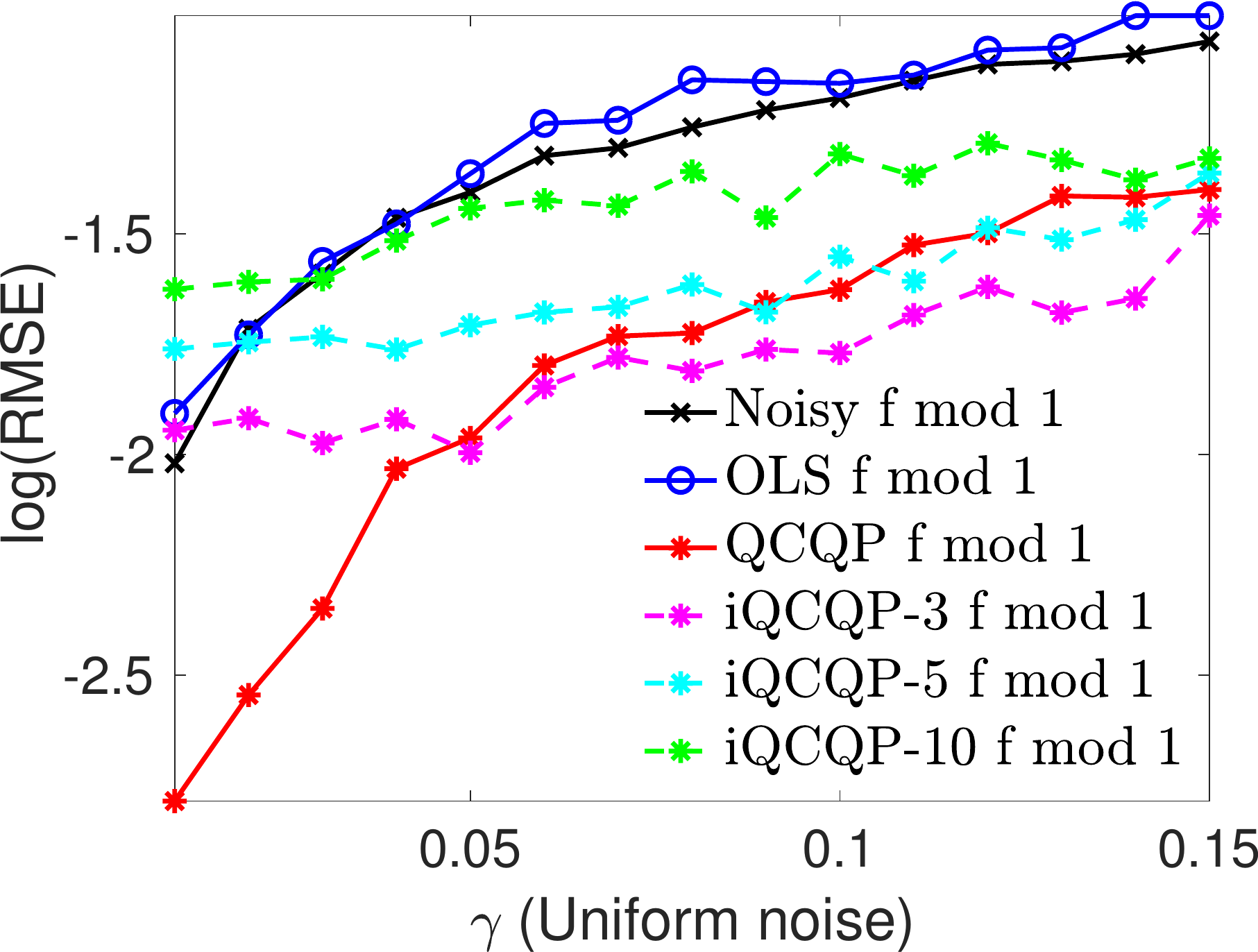} }
\hspace{0.1\textwidth} 
\subcaptionbox[]{ $k=3$, $\lambda= 0.03$}[ 0.19\textwidth ]
{\includegraphics[width=0.19\textwidth] {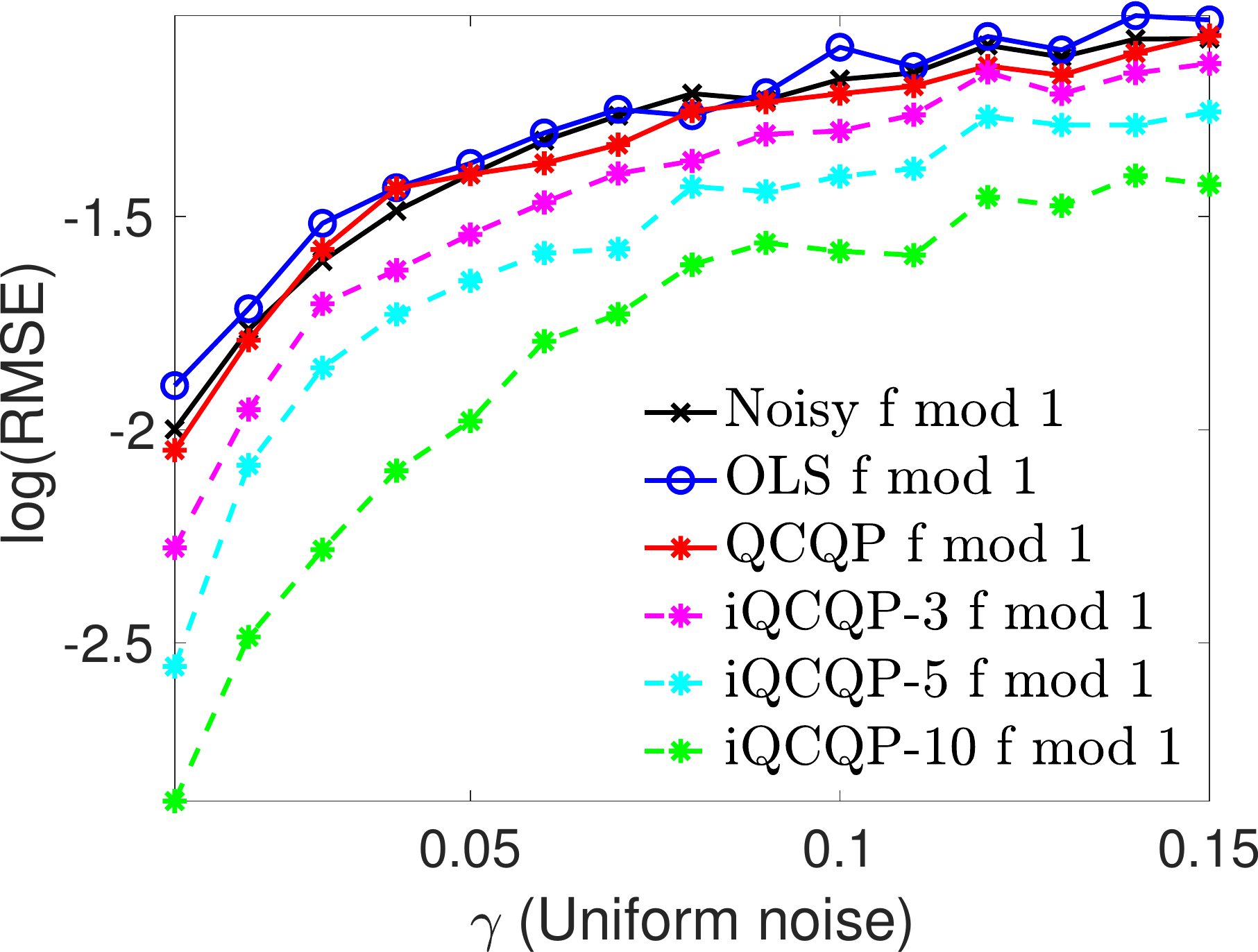} }
\subcaptionbox[]{ $k=3$, $\lambda= 0.1$}[ 0.19\textwidth ]
{\includegraphics[width=0.19\textwidth] {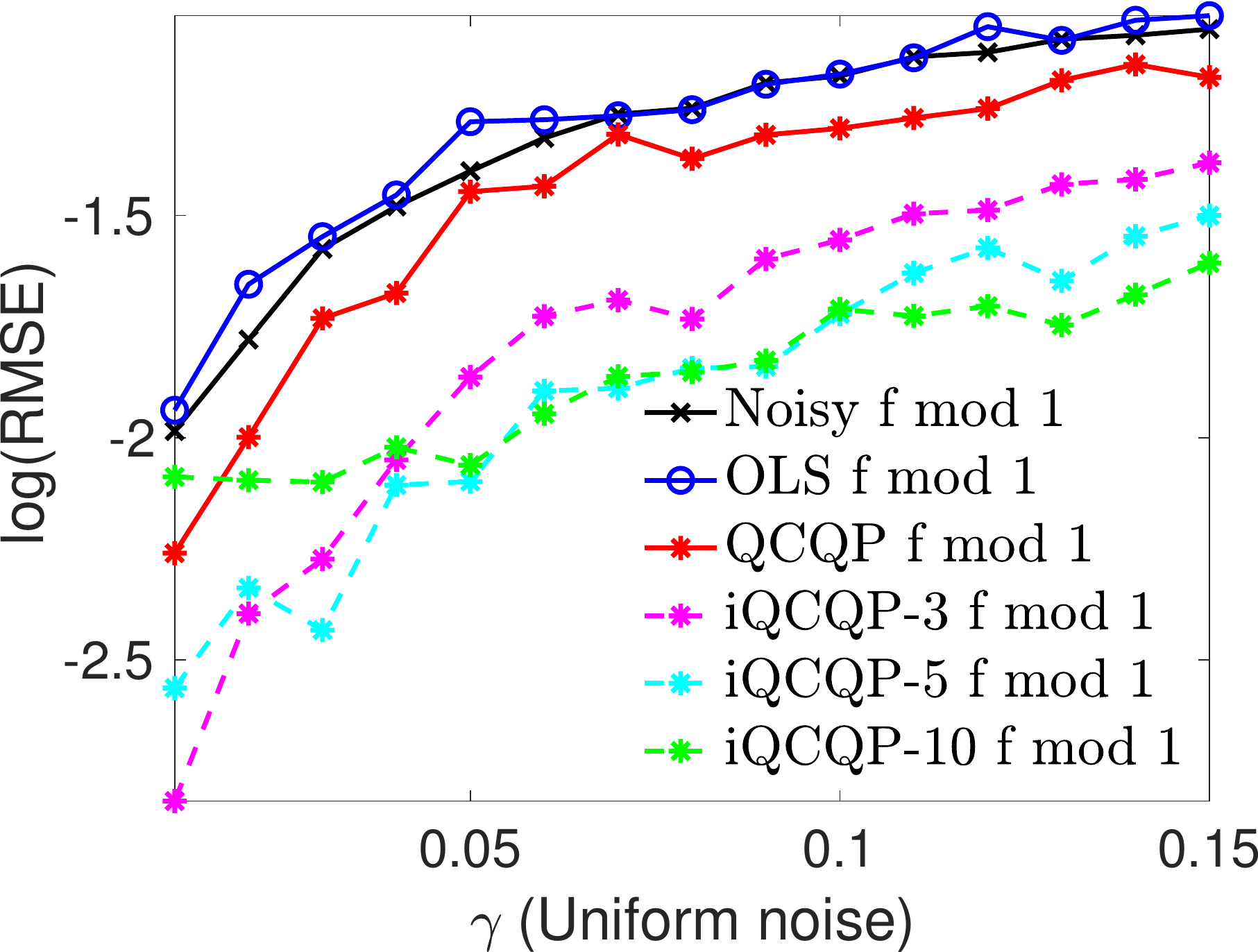} }
%
\subcaptionbox[]{ $k=3$, $\lambda= 0.3$}[ 0.19\textwidth ]
{\includegraphics[width=0.19\textwidth] {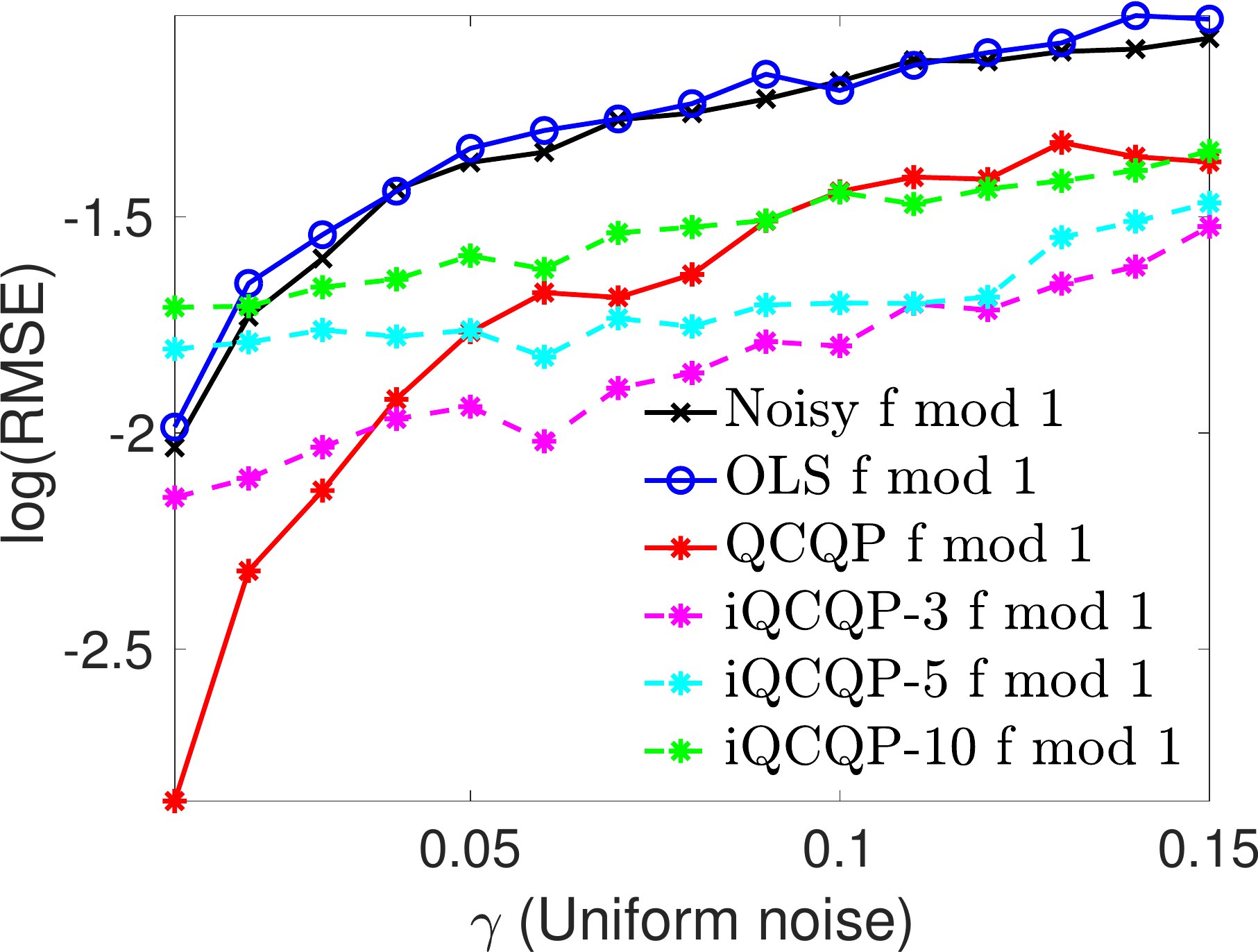} }
%
\subcaptionbox[]{ $k=3$, $\lambda= 0.5$}[ 0.19\textwidth ]
{\includegraphics[width=0.19\textwidth] {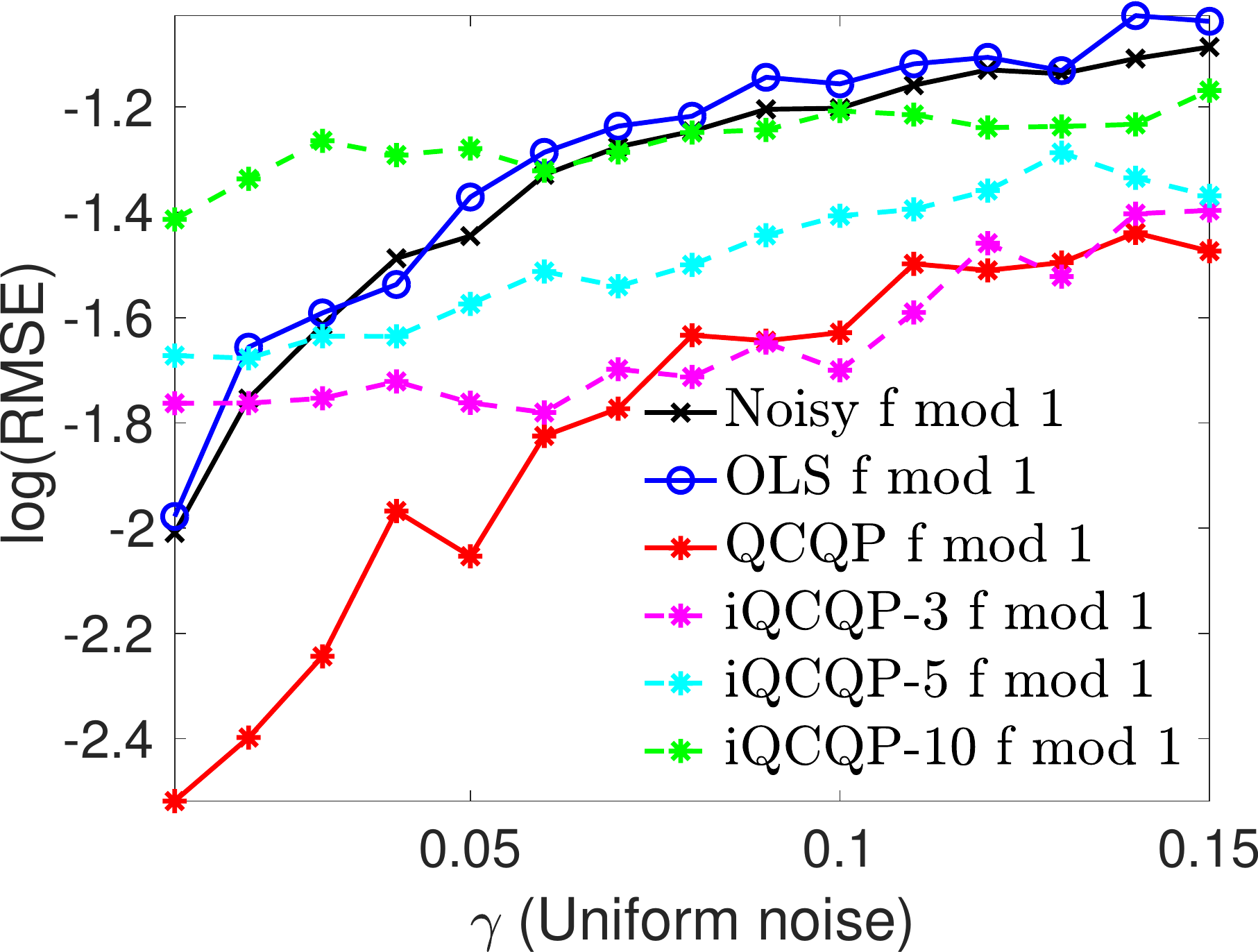} }
%
\subcaptionbox[]{ $k=3$, $\lambda= 1$}[ 0.19\textwidth ]
{\includegraphics[width=0.19\textwidth] {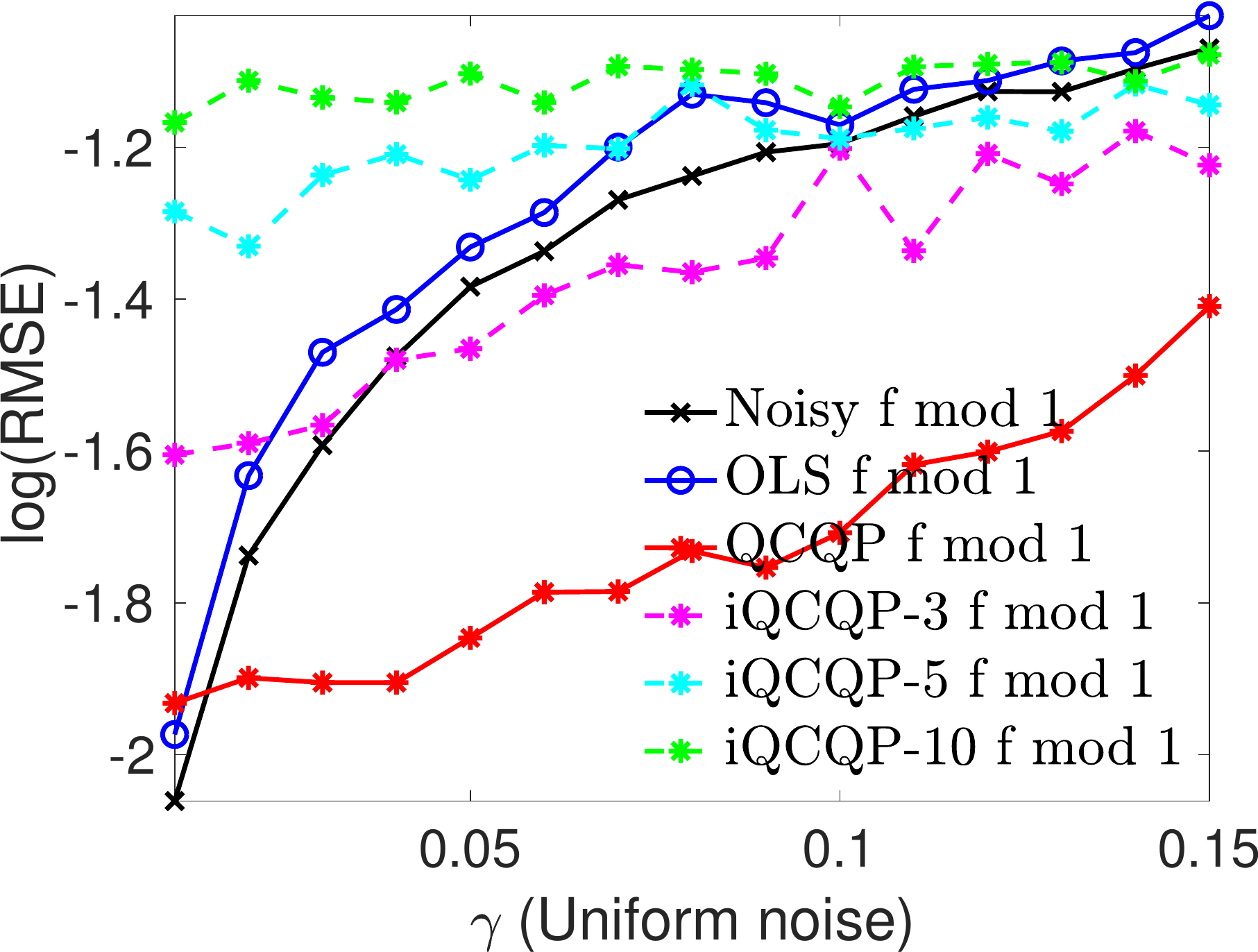} }
\hspace{0.1\textwidth} 
\subcaptionbox[]{ $k=5$, $\lambda= 0.03$}[ 0.19\textwidth ]
{\includegraphics[width=0.19\textwidth] {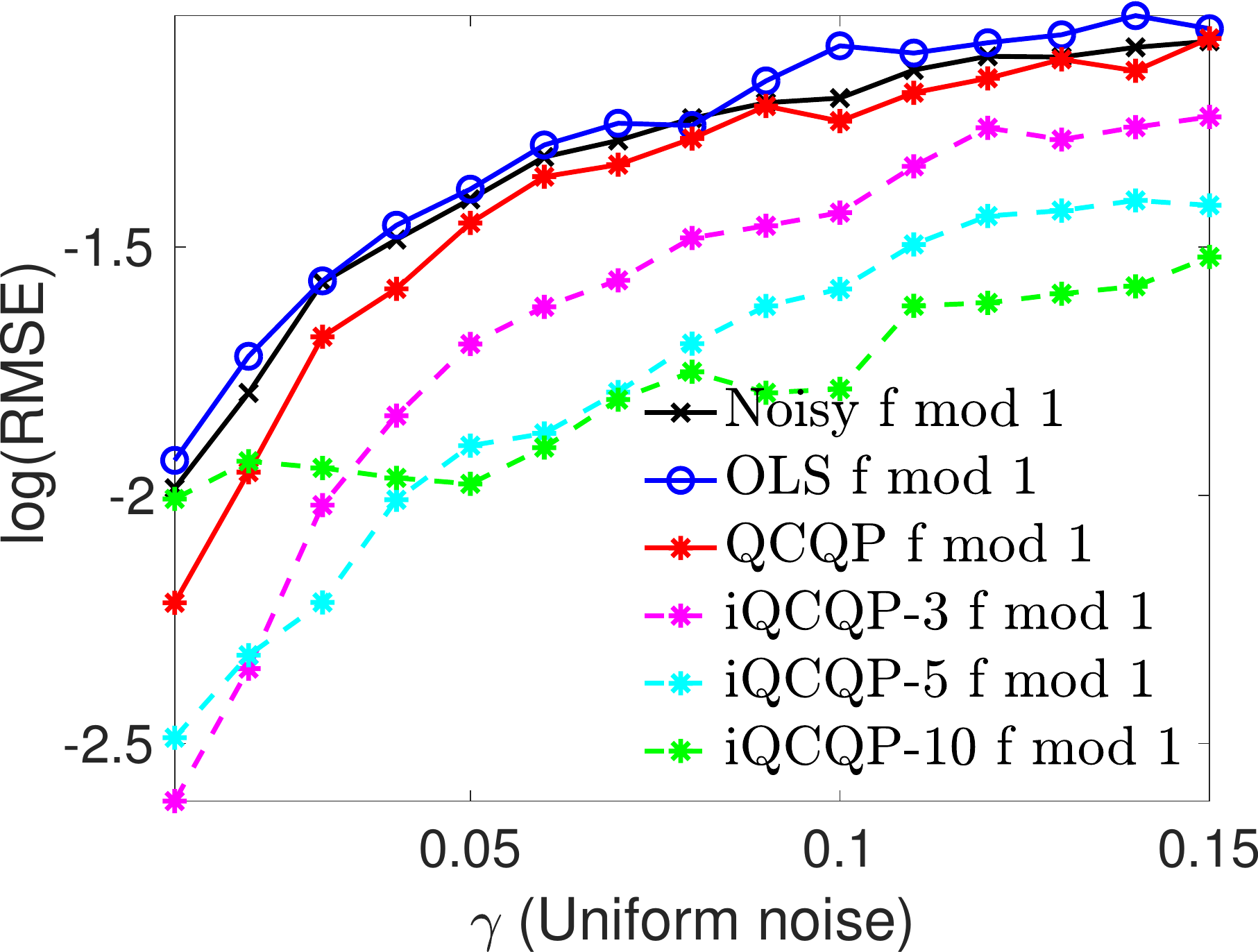} }
\subcaptionbox[]{ $k=5$, $\lambda= 0.1$}[ 0.19\textwidth ]
{\includegraphics[width=0.19\textwidth] {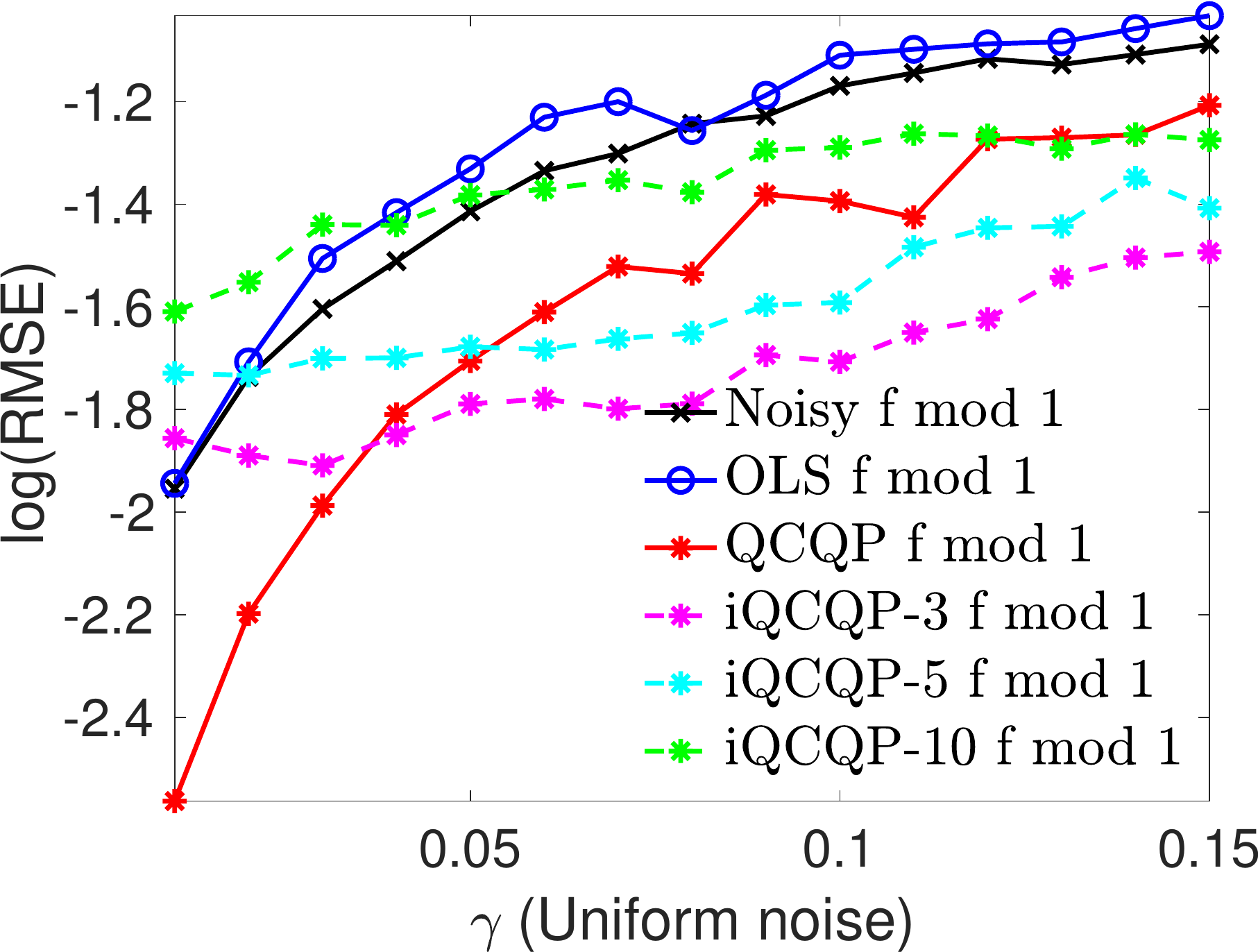} }
%
\subcaptionbox[]{ $k=5$, $\lambda= 0.3$}[ 0.19\textwidth ]
{\includegraphics[width=0.19\textwidth] {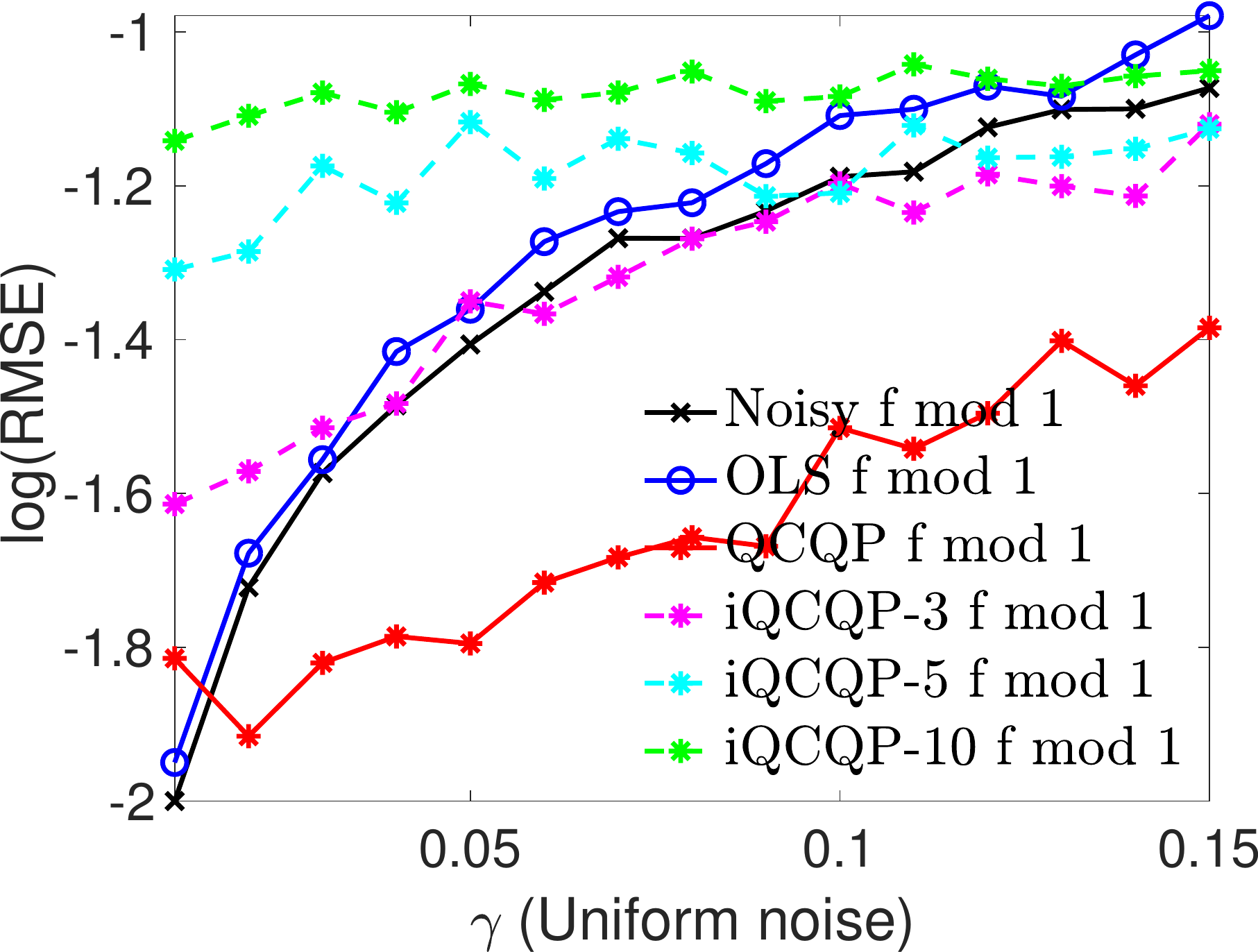} }
%
\subcaptionbox[]{ $k=5$, $\lambda= 0.5$}[ 0.19\textwidth ]
{\includegraphics[width=0.19\textwidth] {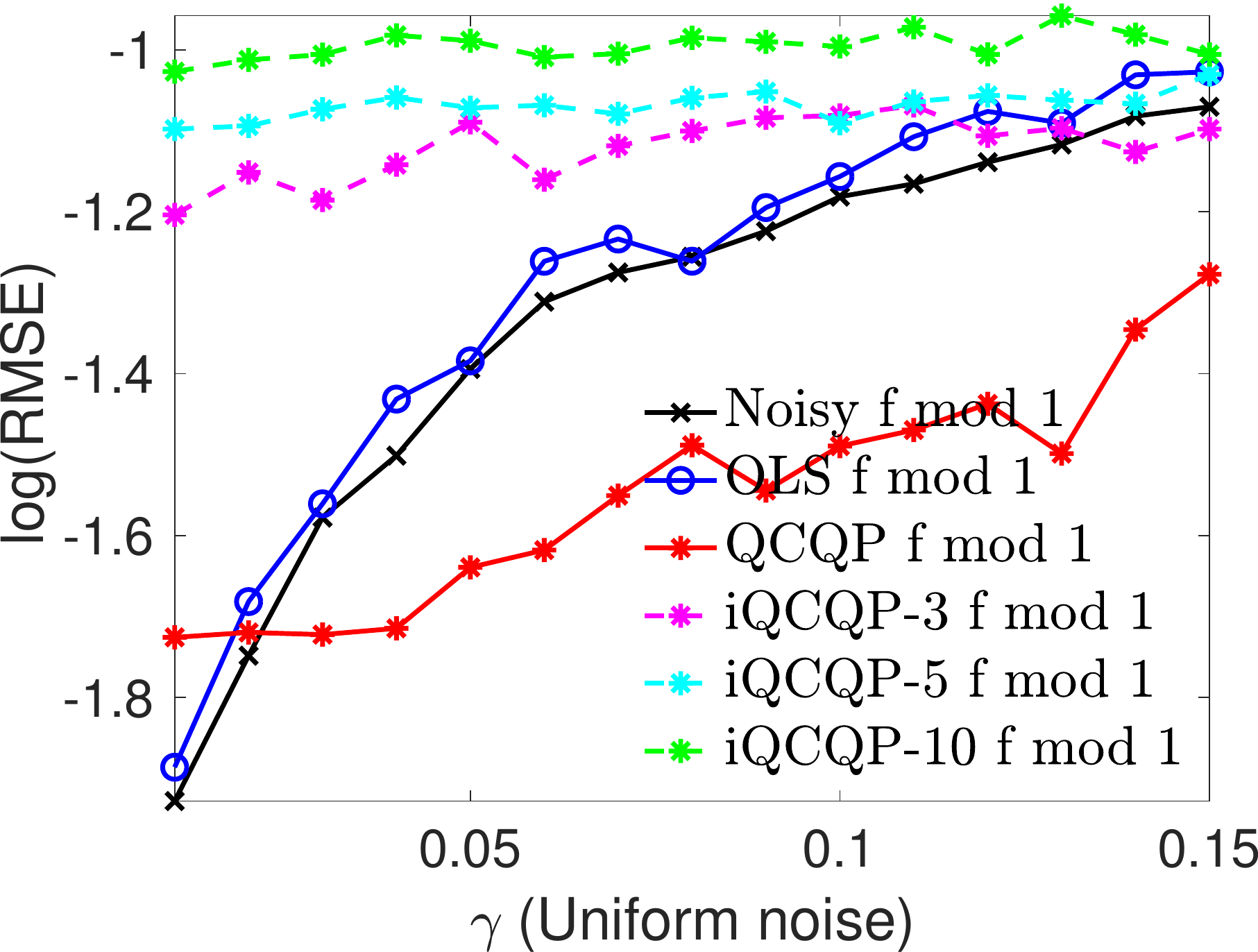} }
%
\subcaptionbox[]{ $k=5$, $\lambda= 1$}[ 0.19\textwidth ]
{\includegraphics[width=0.19\textwidth] {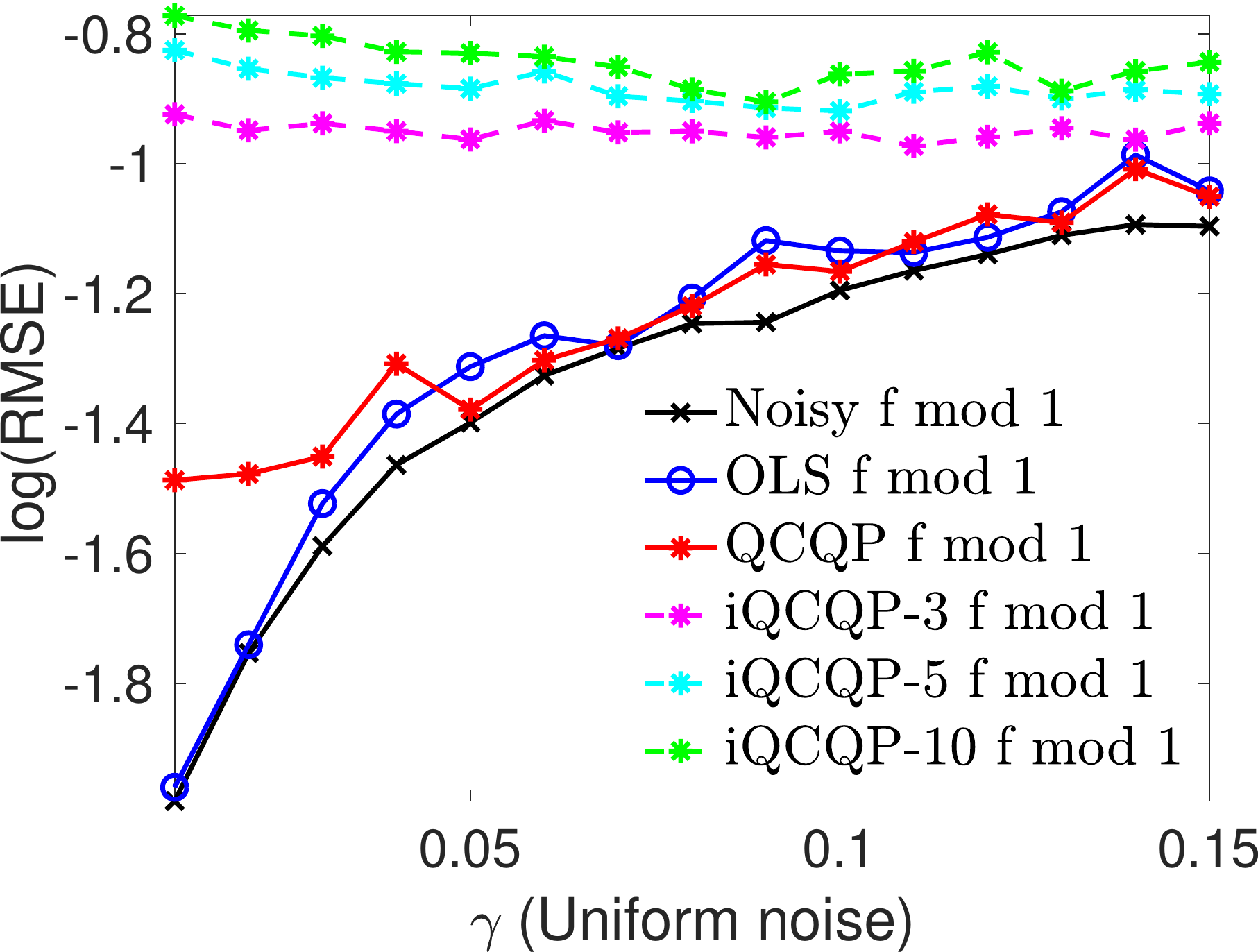} }
\hspace{0.1\textwidth} 
%
%
\vspace{-3mm}
\captionsetup{width=0.98\linewidth}
\caption[Short Caption]{Recovery errors for the denoised $f$ \hspace{-2mm} mod 1 samples, for the Bounded noise model (20 trials). 
\textbf{QCQP} denotes Algorithm \ref{algo:two_stage_denoise} without  the unwrapping stage performed by \textbf{OLS} \eqref{eq:ols_unwrap_lin_system}.
}
\label{fig:Sims_f1_Bounded_fmod1}
\end{figure}
 
\vspace{-3mm}
 
\begin{figure}[!ht]
\centering
\subcaptionbox[]{ $k=2$, $\lambda= 0.03$}[ 0.19\textwidth ]
{\includegraphics[width=0.19\textwidth] {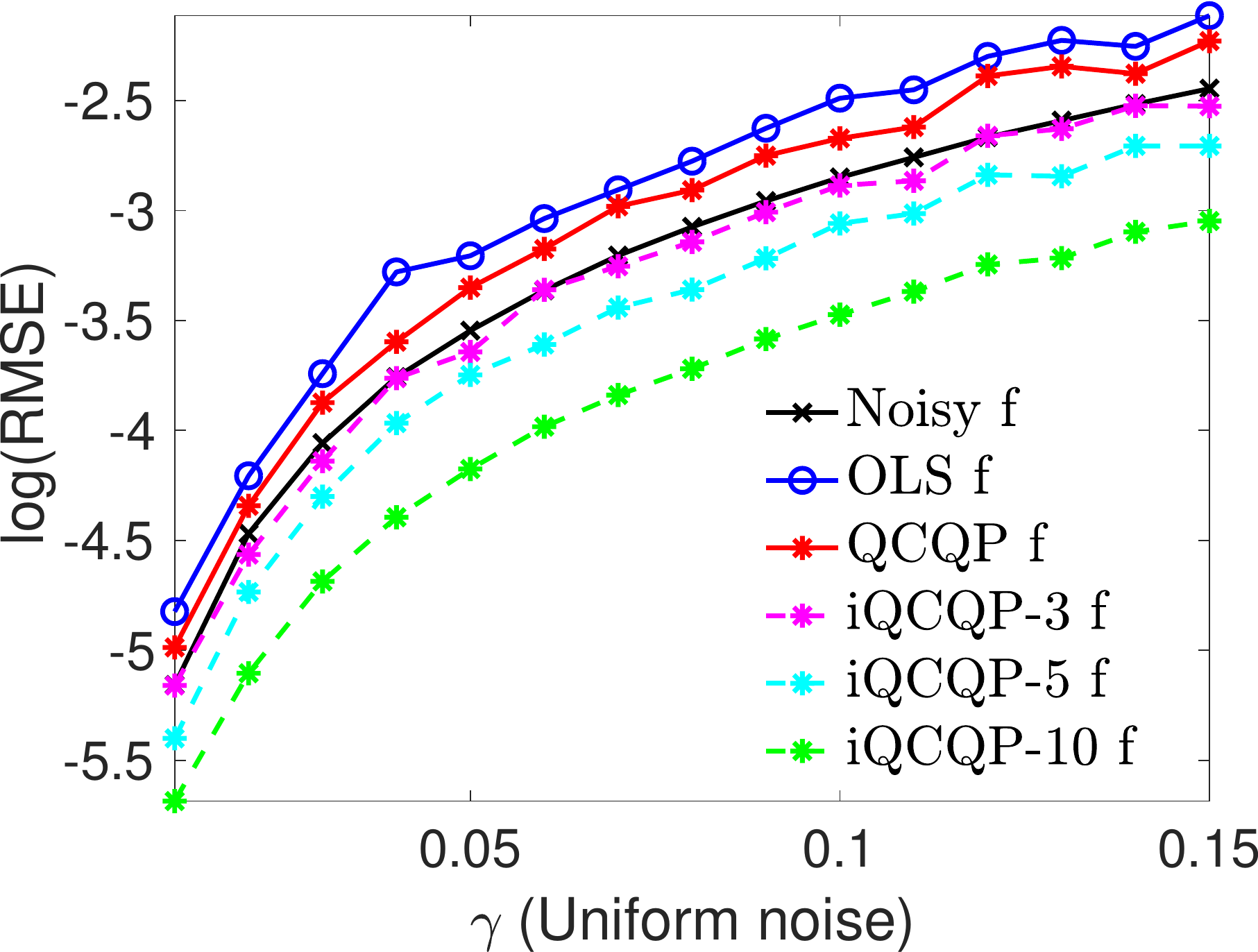} }
\subcaptionbox[]{ $k=2$, $\lambda= 0.1$}[ 0.19\textwidth ]
{\includegraphics[width=0.19\textwidth] {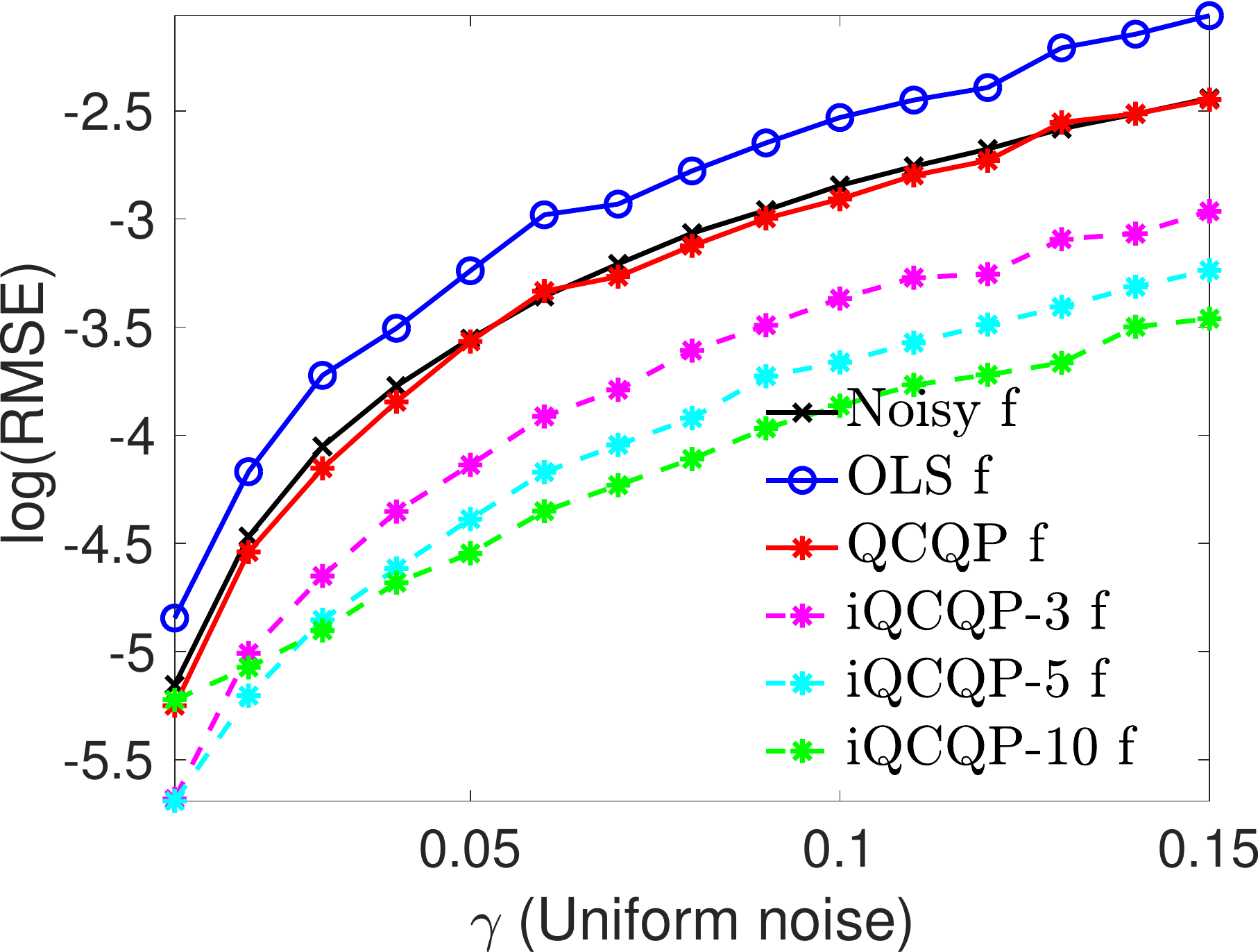} }
%
\subcaptionbox[]{ $k=2$, $\lambda= 0.3$}[ 0.19\textwidth ]
{\includegraphics[width=0.19\textwidth] {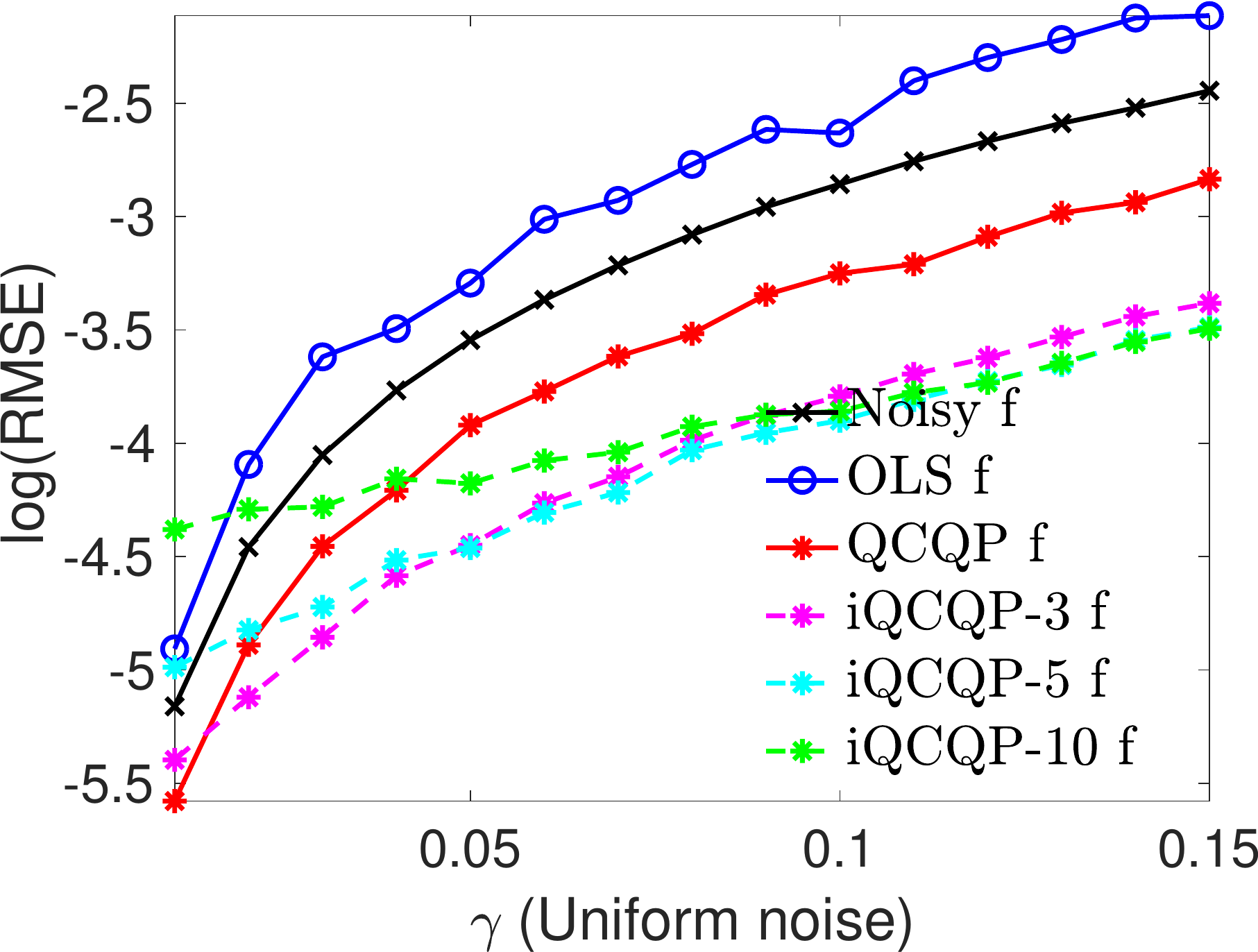} }
%
\subcaptionbox[]{ $k=2$, $\lambda= 0.5$}[ 0.19\textwidth ]
{\includegraphics[width=0.19\textwidth] {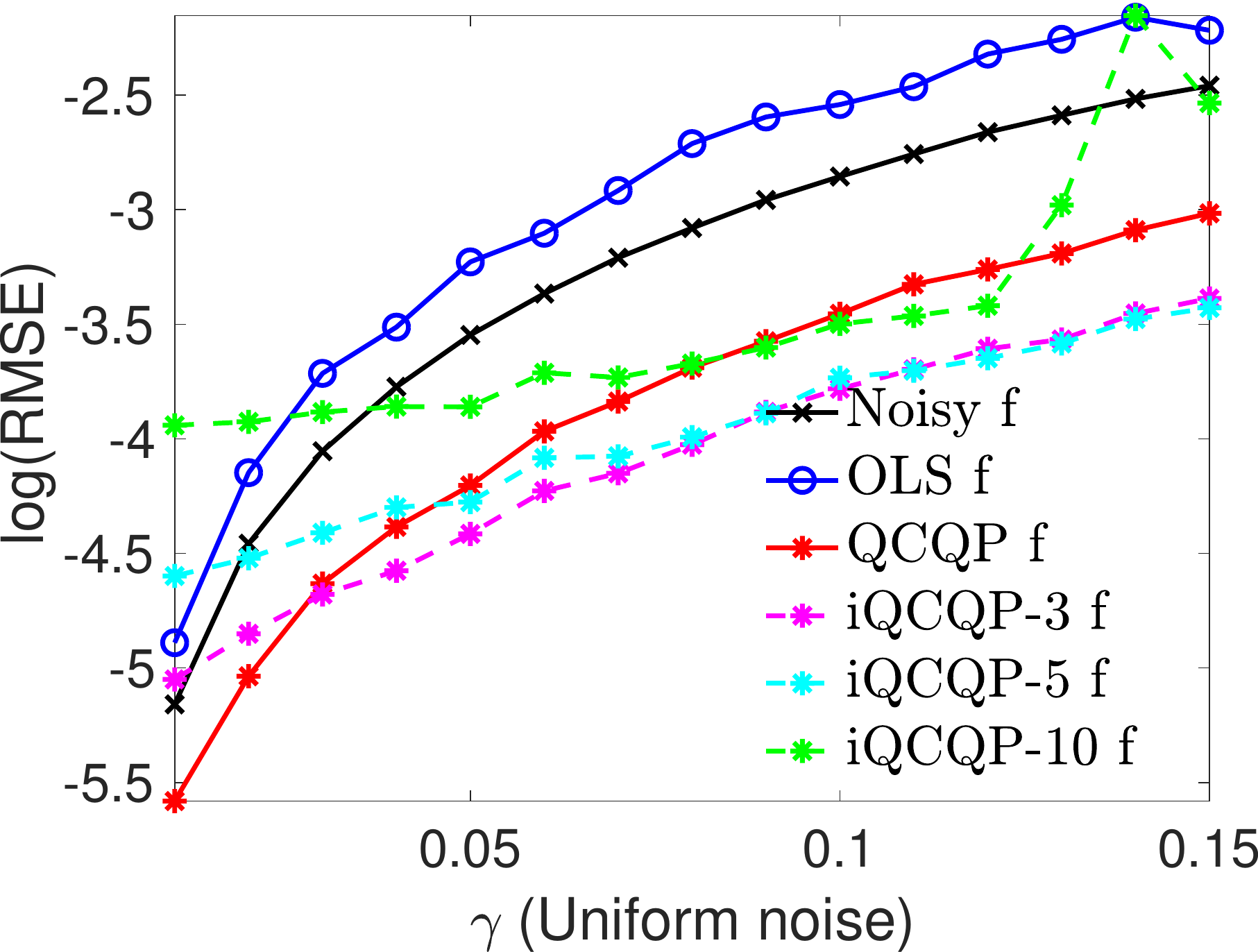} }
%
\subcaptionbox[]{ $k=2$, $\lambda= 1$}[ 0.19\textwidth ]
{\includegraphics[width=0.19\textwidth] {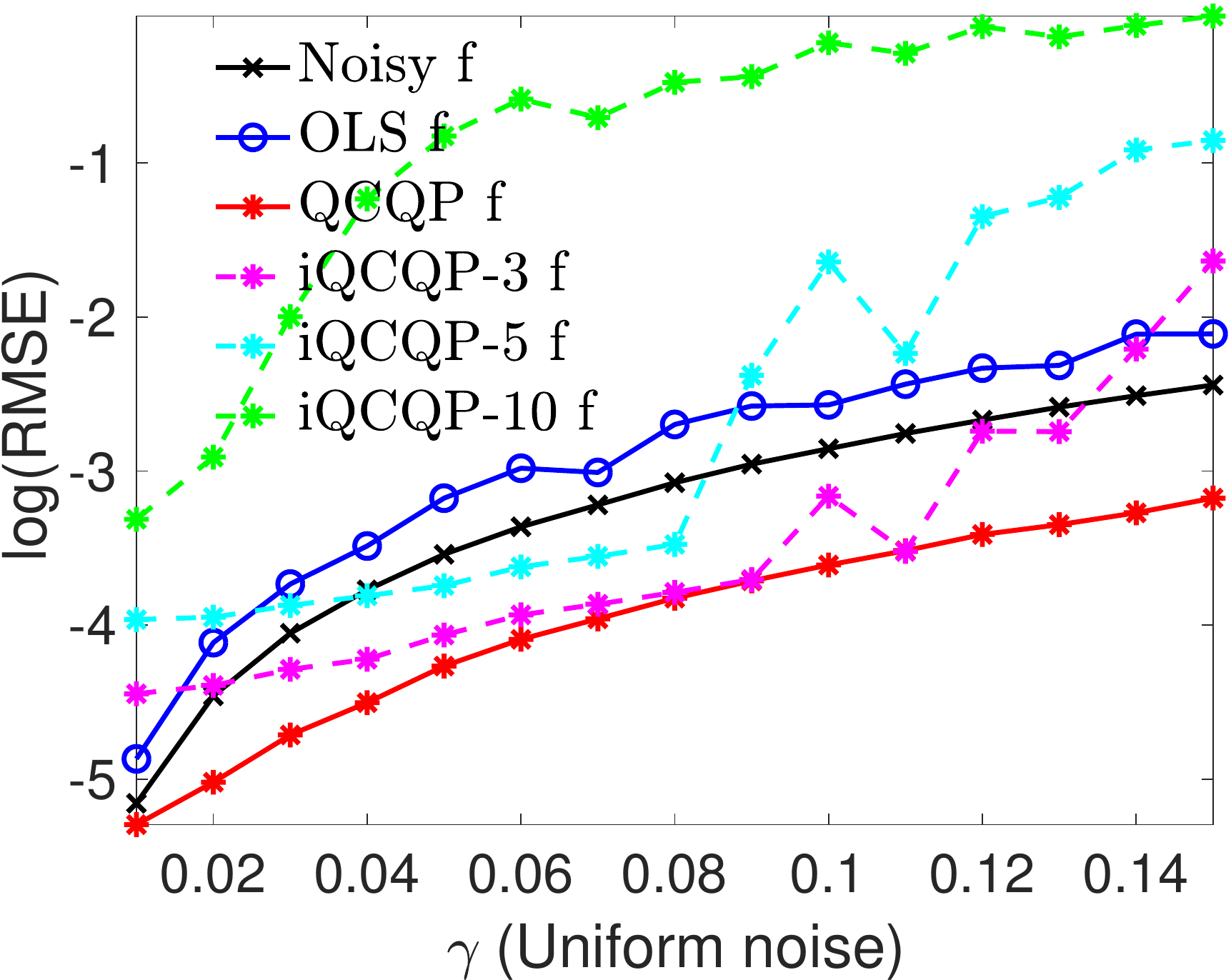} }
%
%
%
\subcaptionbox[]{ $k=3$, $\lambda= 0.03$}[ 0.19\textwidth ]
{\includegraphics[width=0.19\textwidth] {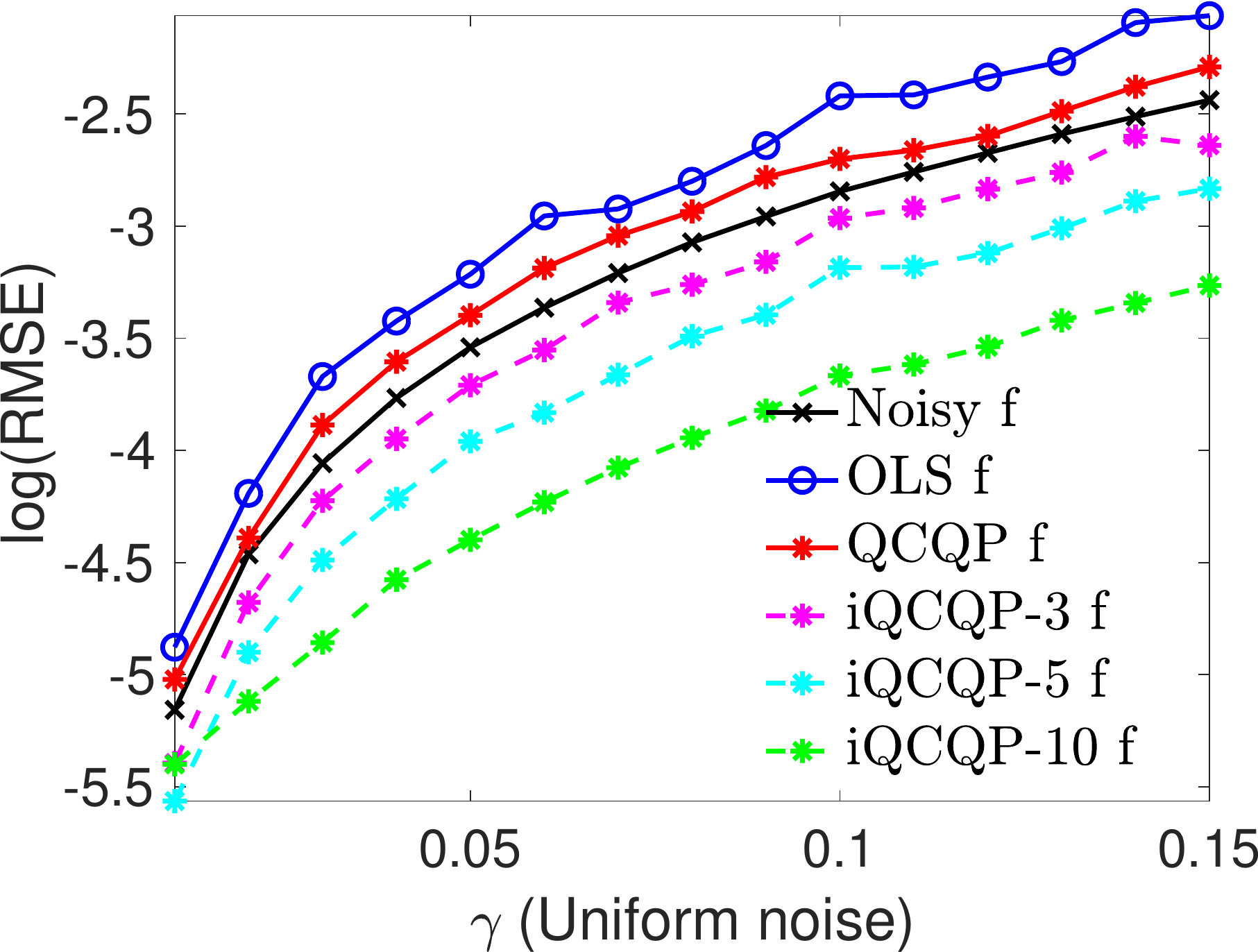} }
\subcaptionbox[]{ $k=3$, $\lambda= 0.1$}[ 0.19\textwidth ]
{\includegraphics[width=0.19\textwidth] {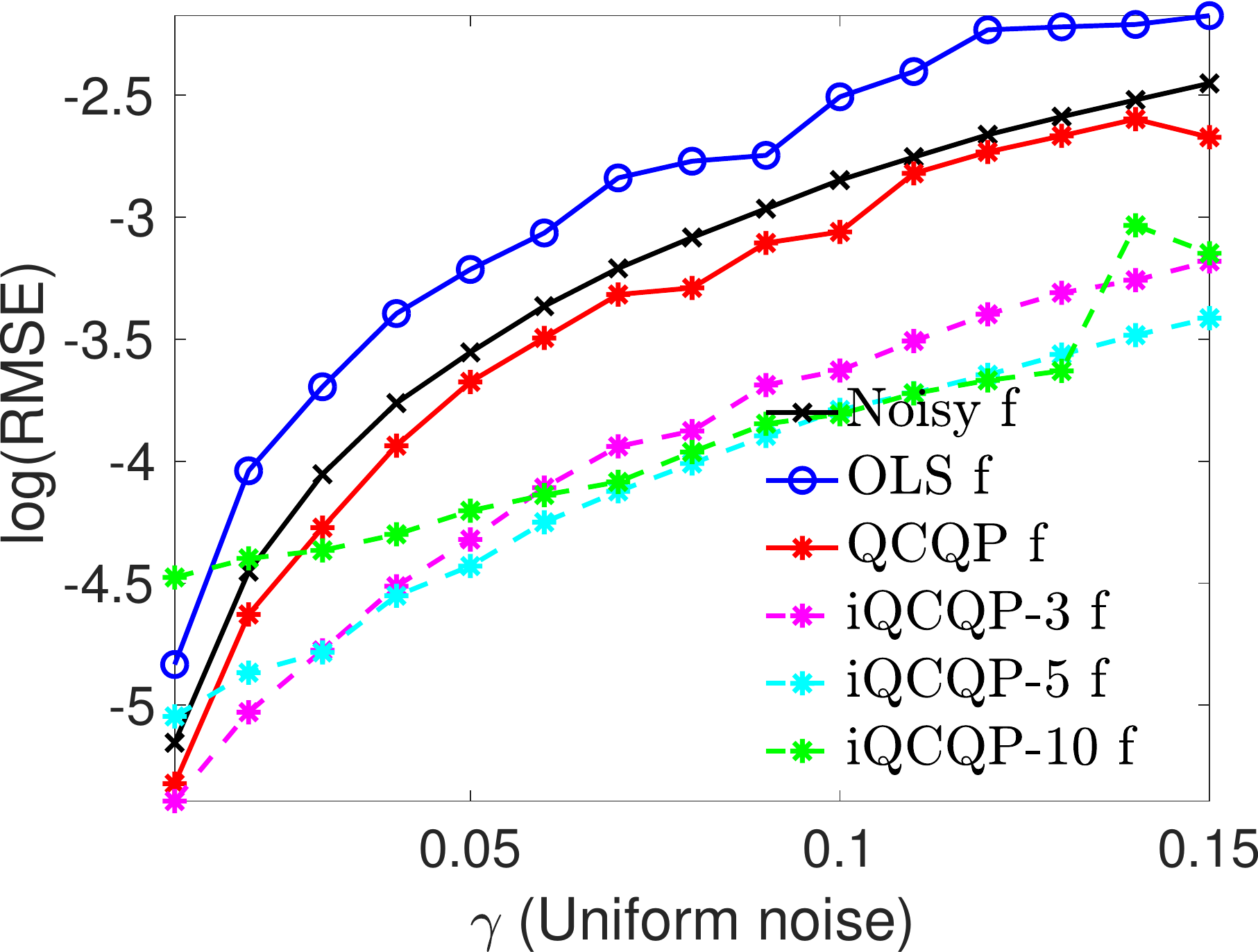} }
%
\subcaptionbox[]{ $k=3$, $\lambda= 0.3$}[ 0.19\textwidth ]
{\includegraphics[width=0.19\textwidth] {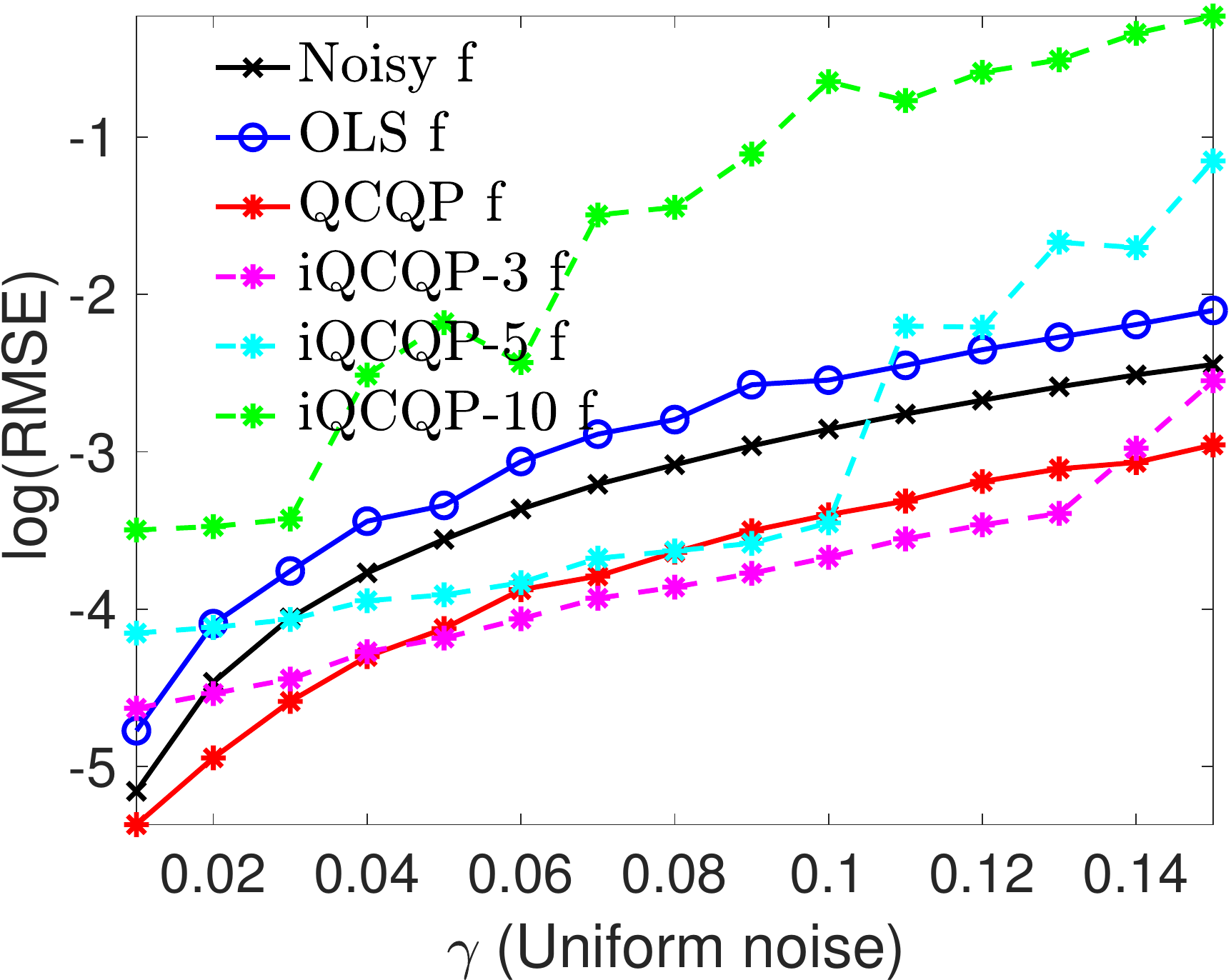} }
%
\subcaptionbox[]{ $k=3$, $\lambda= 0.5$}[ 0.19\textwidth ]
{\includegraphics[width=0.19\textwidth] {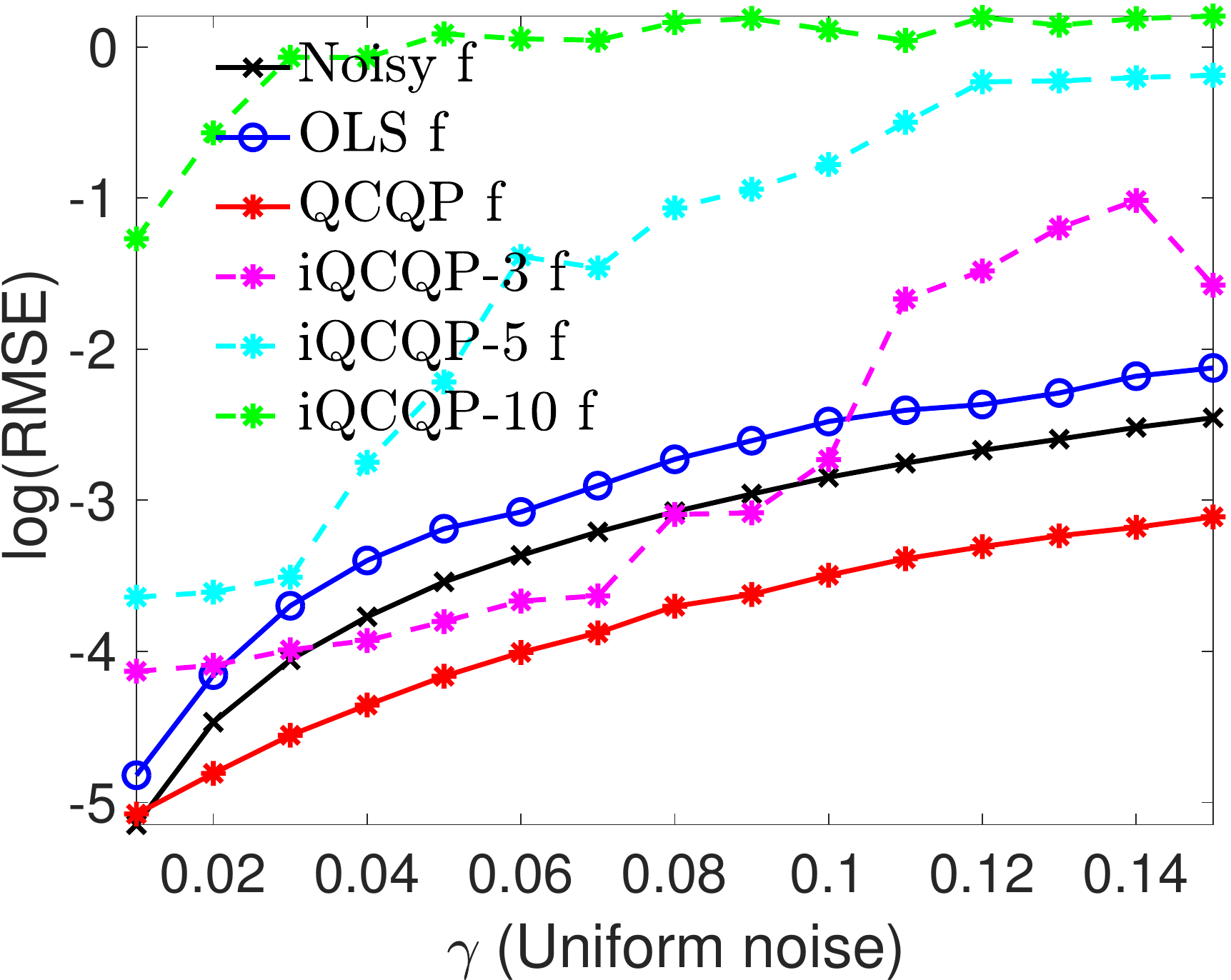} }
%
\subcaptionbox[]{ $k=3$, $\lambda= 1$}[ 0.19\textwidth ]
{\includegraphics[width=0.19\textwidth] {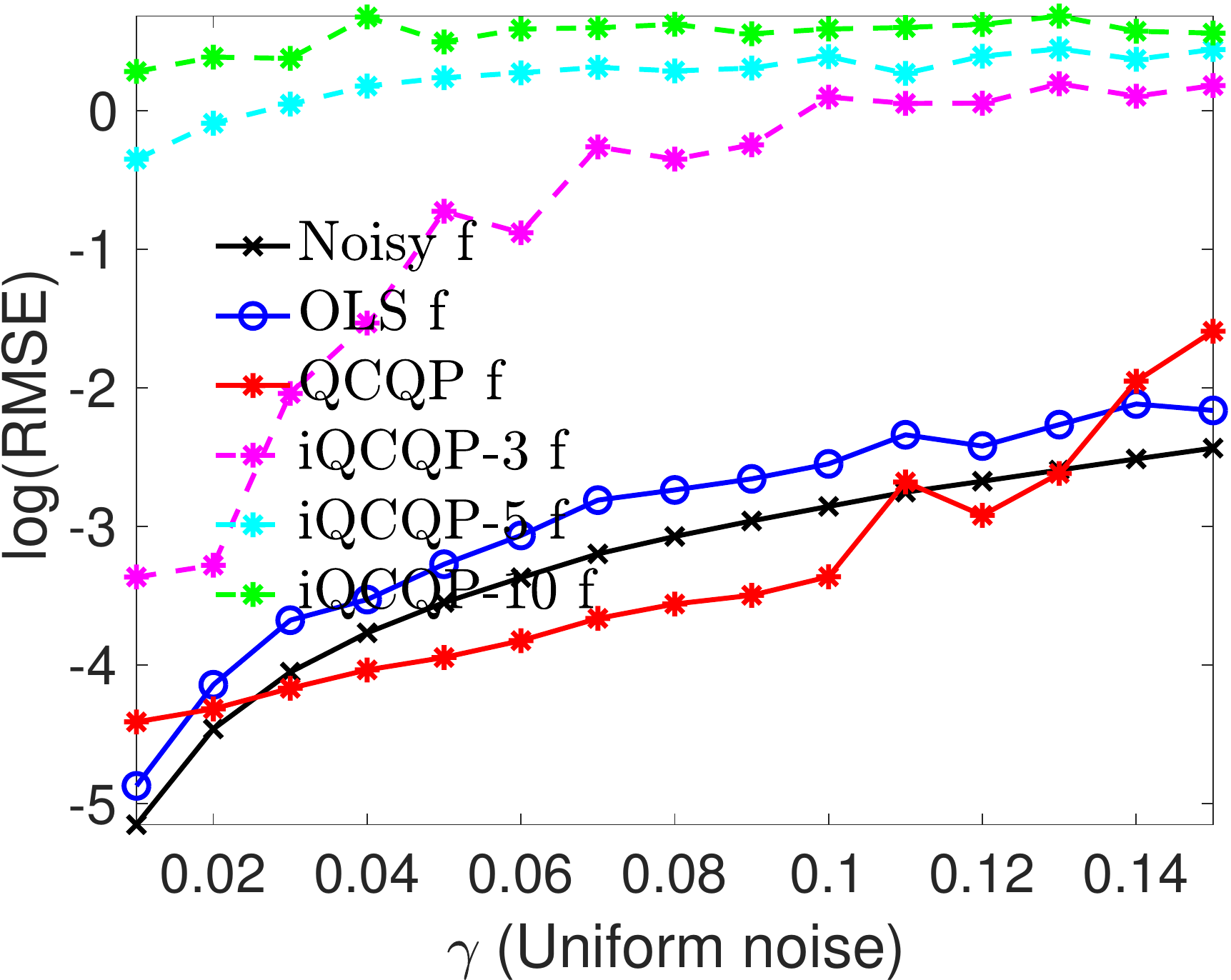} }
%
%
%
\subcaptionbox[]{ $k=5$, $\lambda= 0.03$}[ 0.19\textwidth ]
{\includegraphics[width=0.19\textwidth] {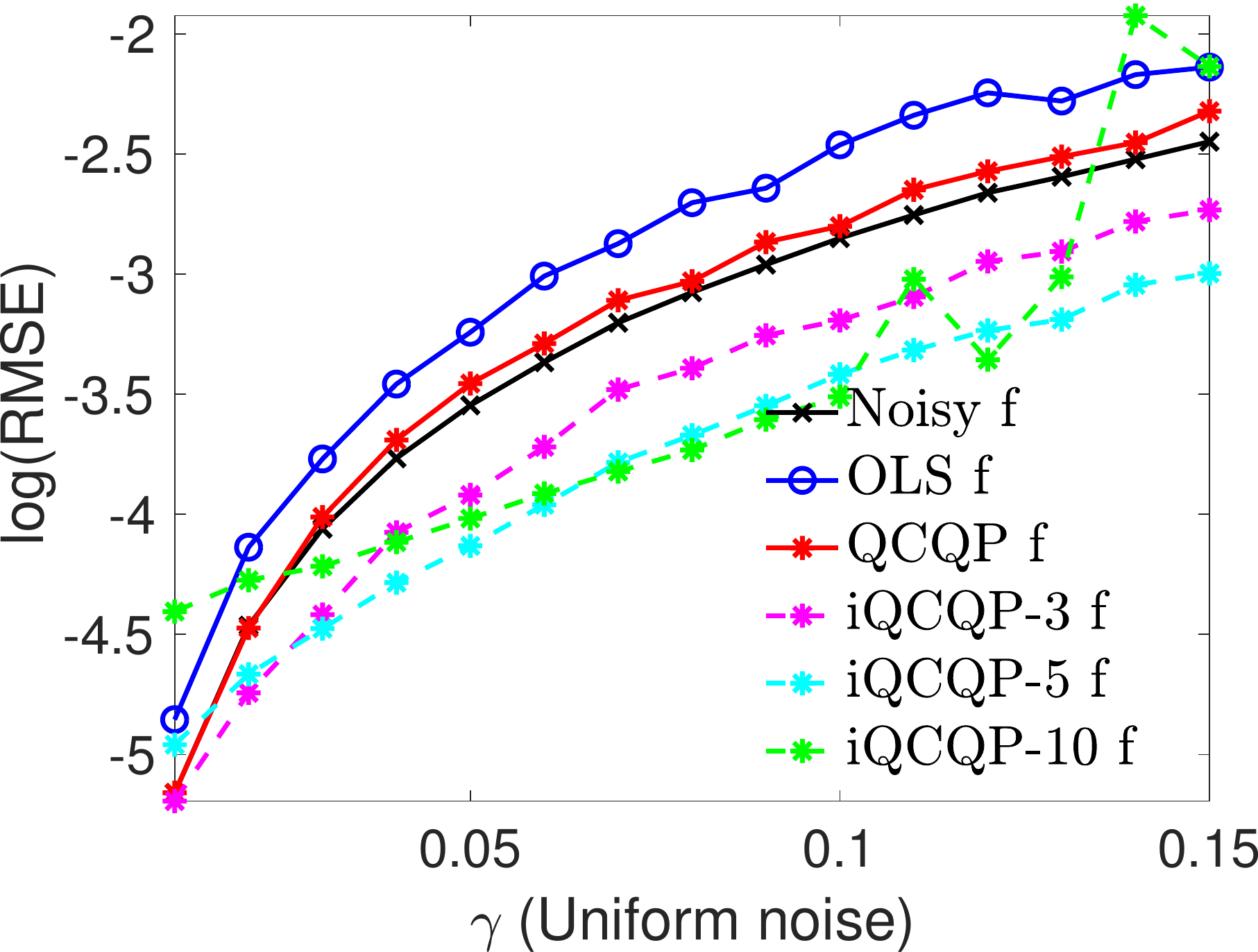} }
\subcaptionbox[]{ $k=5$, $\lambda= 0.1$}[ 0.19\textwidth ]
{\includegraphics[width=0.19\textwidth] {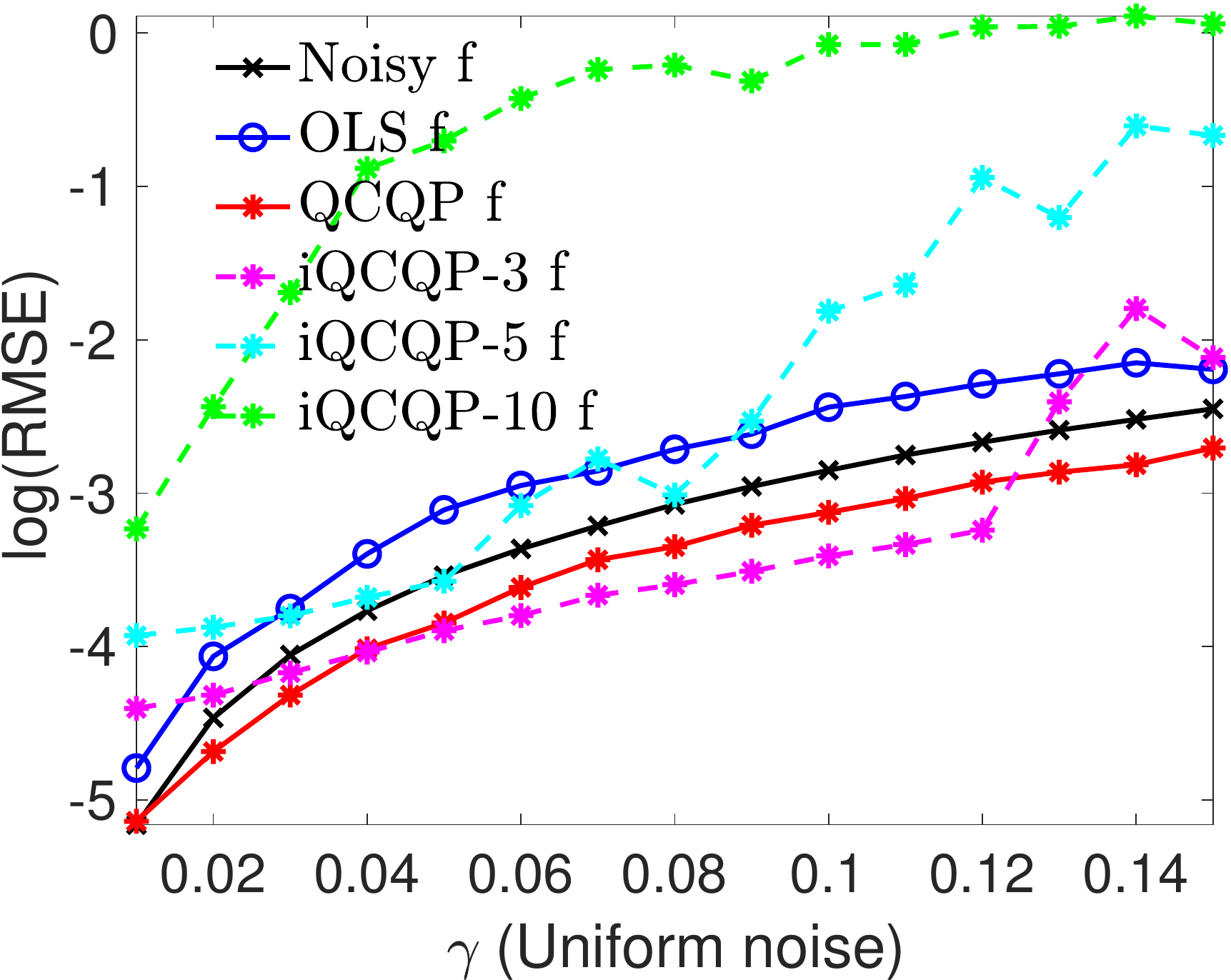} }
%
\subcaptionbox[]{ $k=5$, $\lambda= 0.3$}[ 0.19\textwidth ]
{\includegraphics[width=0.19\textwidth] {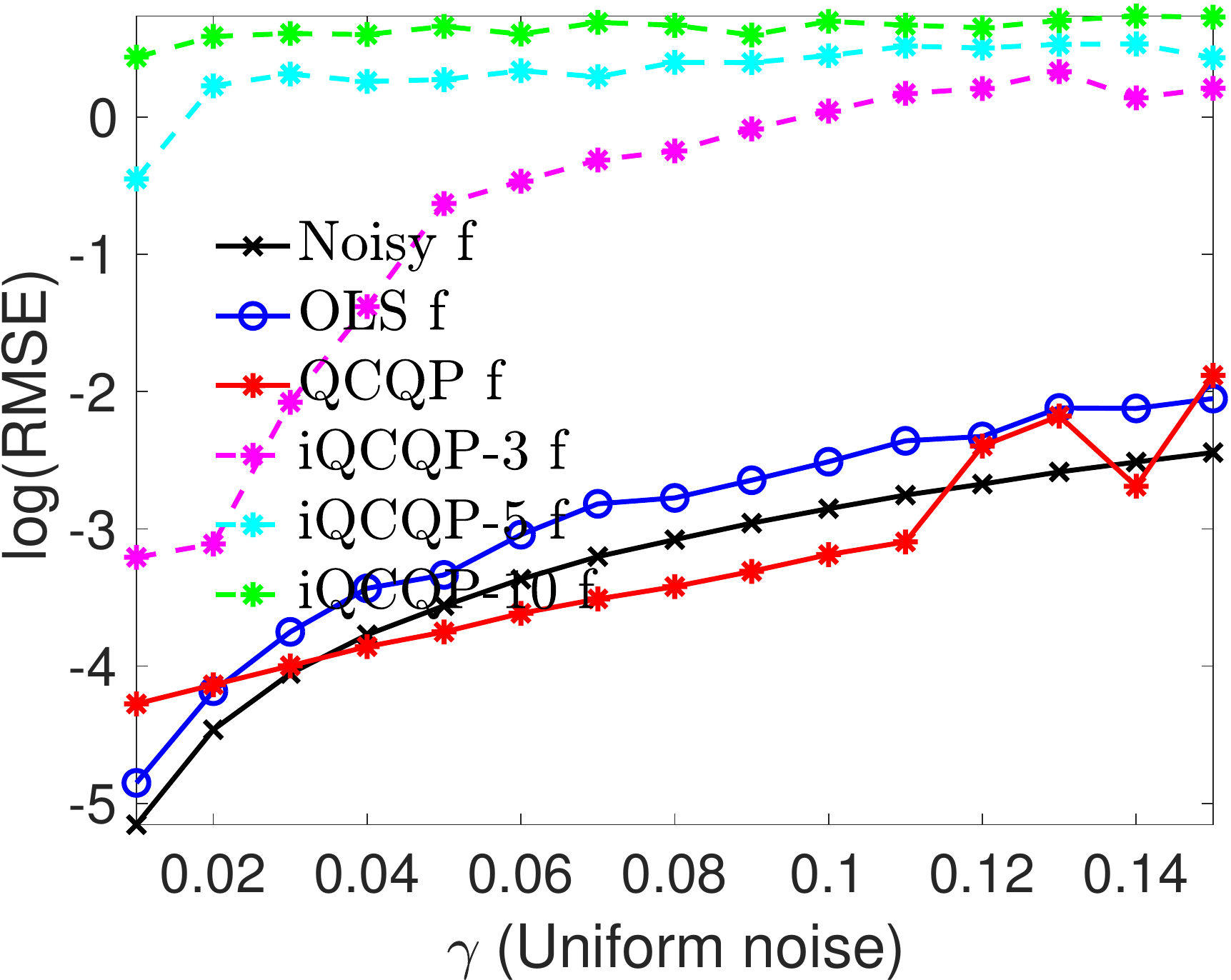} }
%
\subcaptionbox[]{ $k=5$, $\lambda= 0.5$}[ 0.19\textwidth ]
{\includegraphics[width=0.19\textwidth] {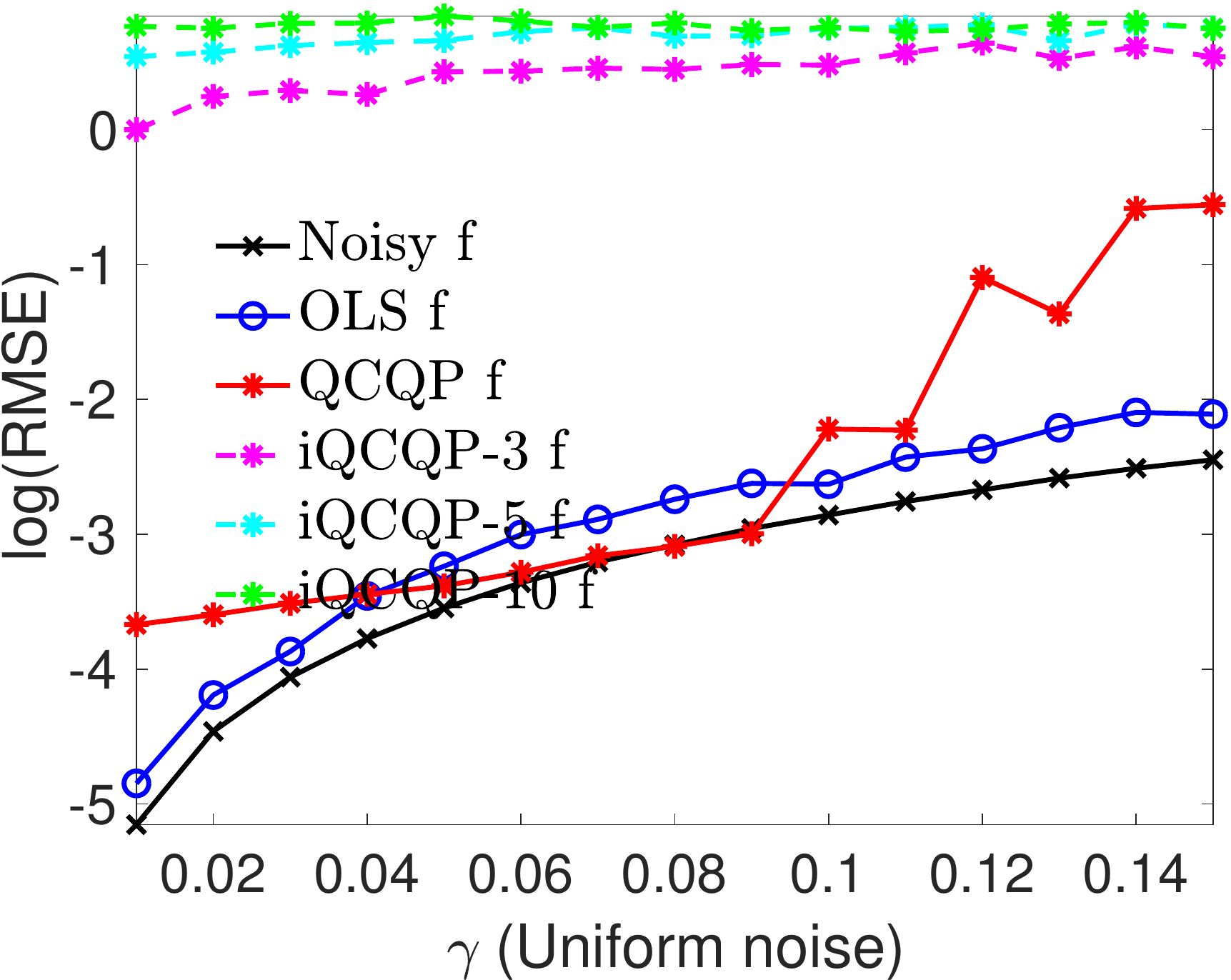} }
%
\subcaptionbox[]{ $k=5$, $\lambda= 1$}[ 0.19\textwidth ]
{\includegraphics[width=0.19\textwidth] {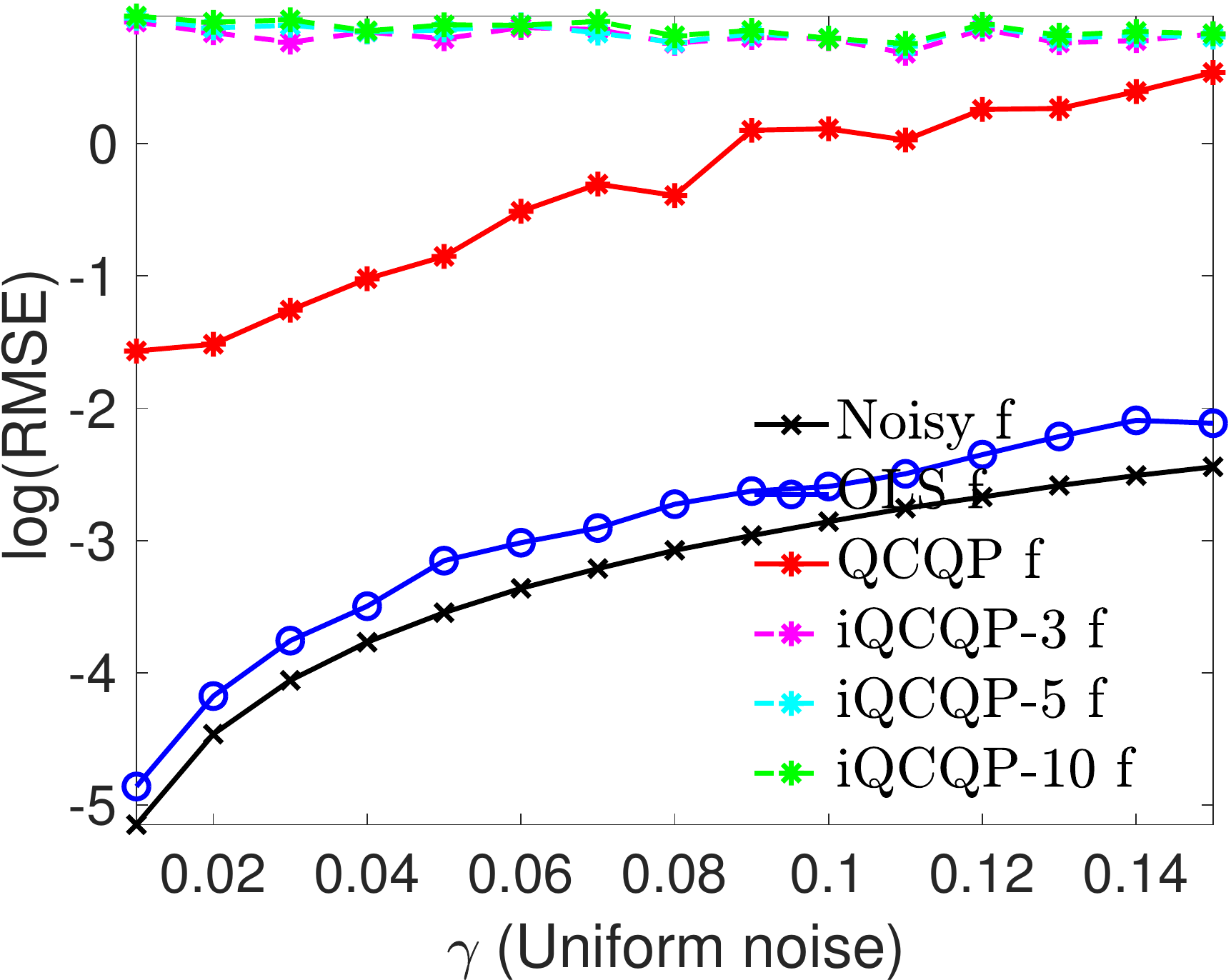} }
%
\vspace{-3mm}
\captionsetup{width=0.98\linewidth}
\caption[Short Caption]{Recovery errors for the final estimated $f$  samples, for the Bounded noise model (20 trials). 
\textbf{QCQP} denotes Algorithm \ref{algo:two_stage_denoise}, for which  the unwrapping stage is  performed via \textbf{OLS} \eqref{eq:ols_unwrap_lin_system}. 
}
\label{fig:Sims_f1_Bounded_f}
\end{figure}

\begin{figure}[!ht]
\centering
\subcaptionbox[]{ $k=2$, $\lambda= 0.3$, $\gamma = 0.01$
}[ 0.24\textwidth ]
{\includegraphics[width=0.24\textwidth] {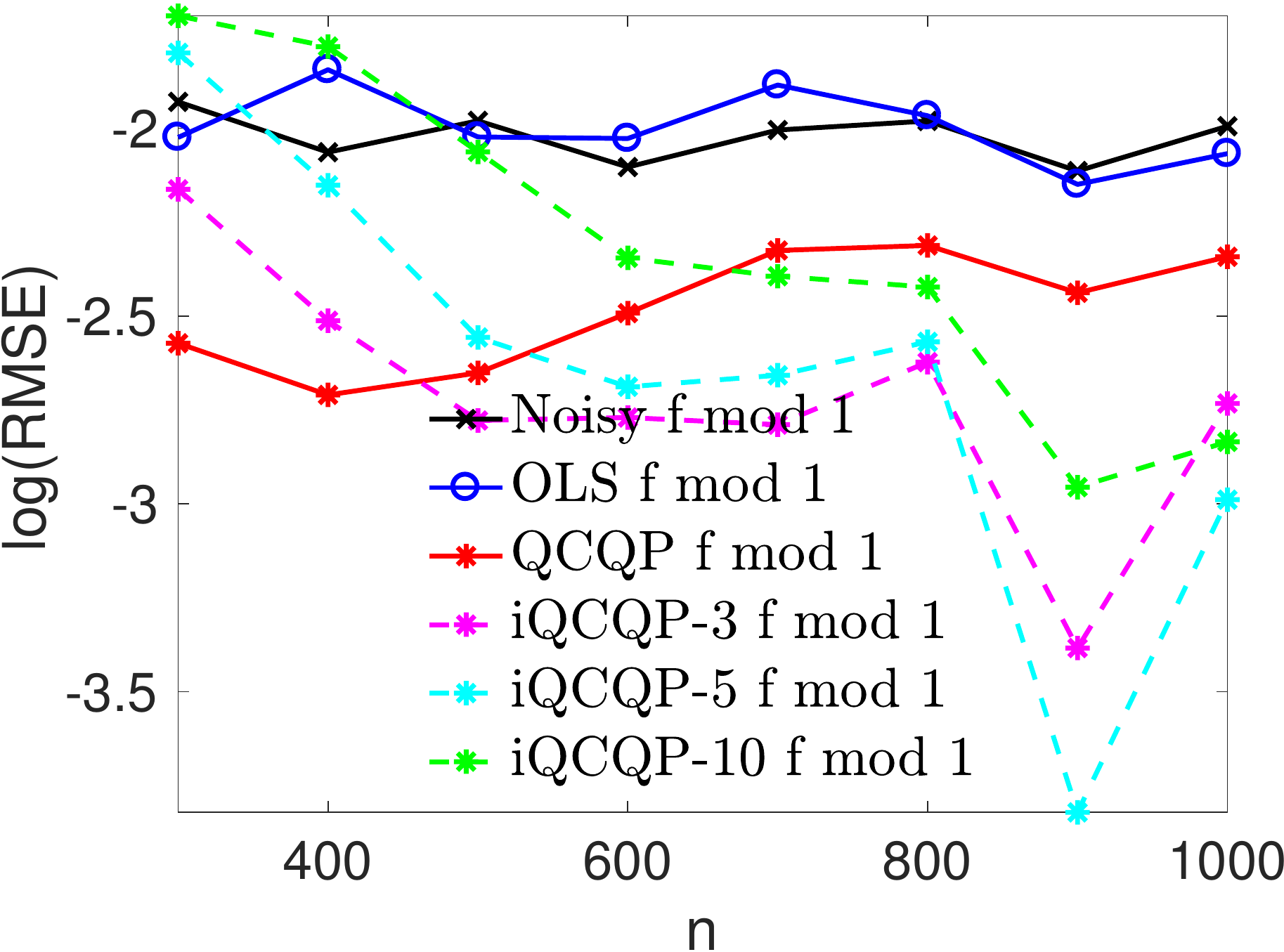} }
%
\subcaptionbox[]{ $k=3$, $\lambda= 0.3$, $\gamma = 0.25$
}[ 0.24\textwidth ]
{\includegraphics[width=0.24\textwidth] {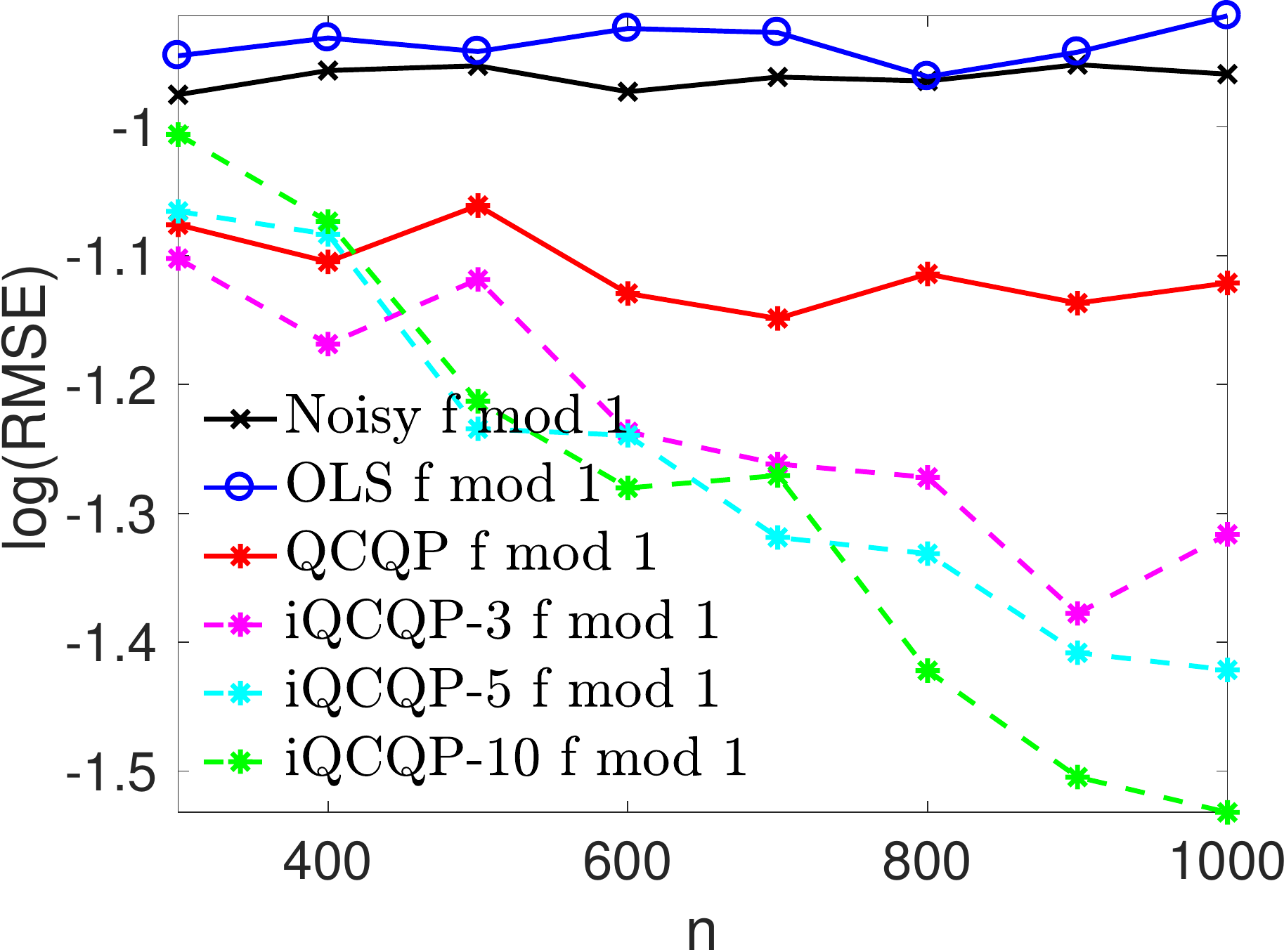} }
%
%
\subcaptionbox[]{ $k=2$, $\lambda= 0.3$, $\gamma = 0.01$
}[ 0.24\textwidth ]
{\includegraphics[width=0.24\textwidth] {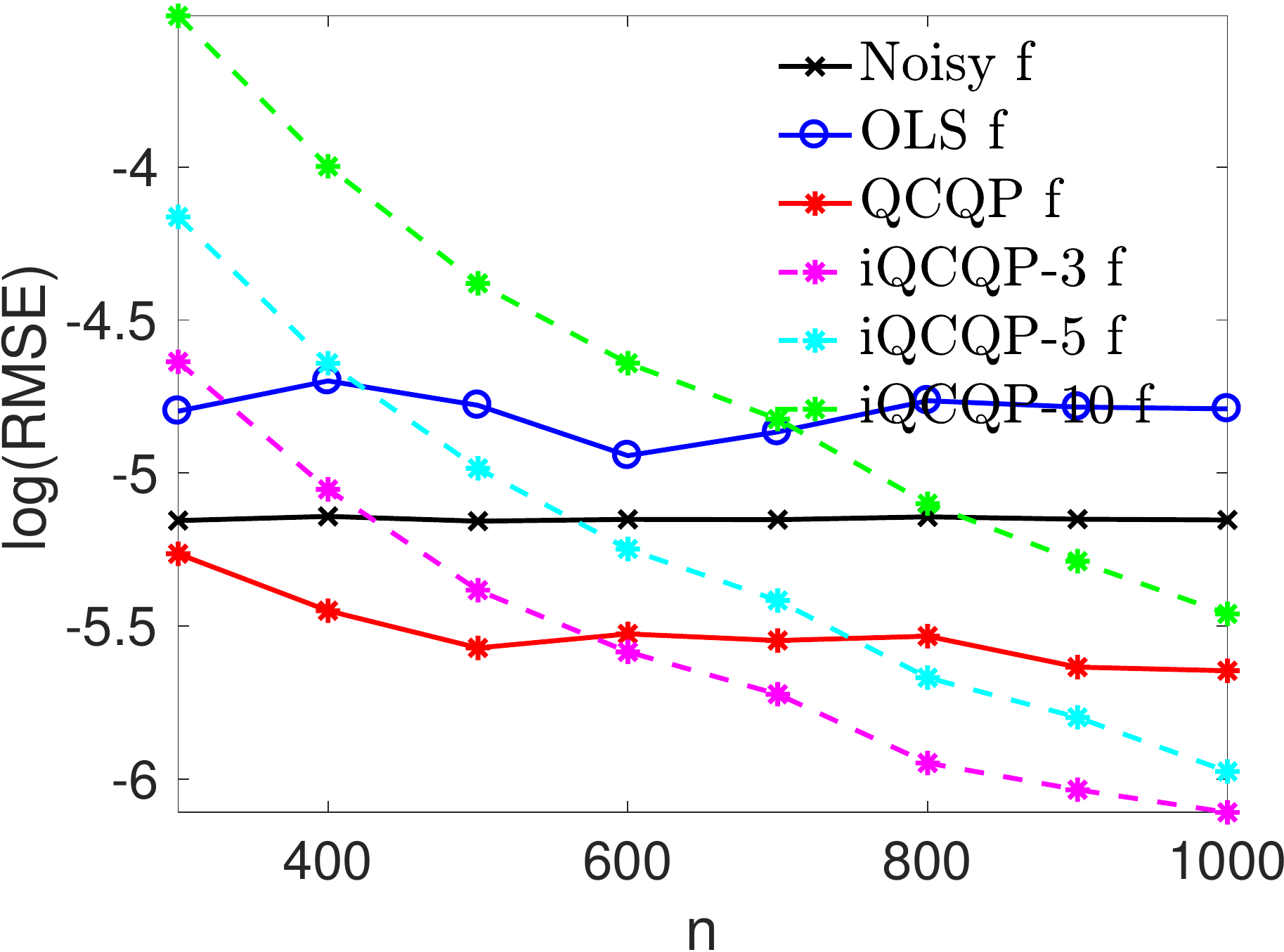} }
%
\subcaptionbox[]{ $k=3$, $\lambda= 0.3$, $\gamma = 0.25$
}[ 0.24\textwidth ]
{\includegraphics[width=0.24\textwidth] {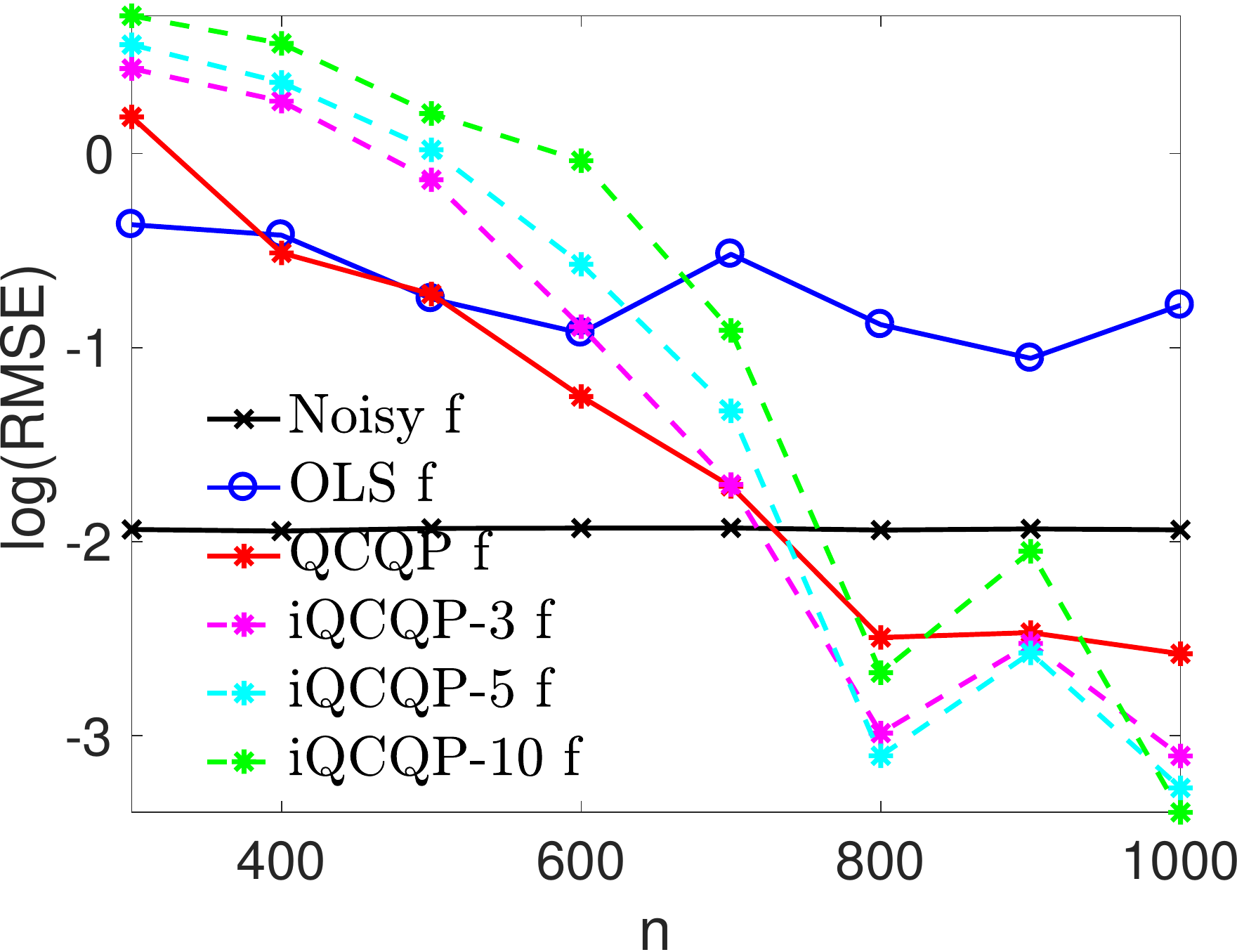} }
\captionsetup{width=0.98\linewidth}
\caption[Short Caption]{Recovery errors for \textbf{OLS}, \textbf{QCQP}, and  \textbf{iQCQP} 
as a function of $n$ (number of samples), for both the $f$ mod 1 samples (leftmost two plots) and the final $f$ estimates (rightmost two plots) under the Uniform noise model, for different values of $k$, $\lambda$ and $\gamma$. Results are averaged over 20 runs. \textbf{QCQP} denotes Algorithm \ref{algo:two_stage_denoise}, for which  the unwrapping stage is  performed via \textbf{OLS} \eqref{eq:ols_unwrap_lin_system}. 
}
\label{fig:Sims_f1_Bounded_ScanID2_ffmod1}
\end{figure}


\subsection{Comparison with Bhandari et al. (2017)} \label{sec:apx_Bhandari}  




This section is a comparison of  \textbf{OLS},  \textbf{QCQP}, and \textbf{iQCQP} with the approach introduced by \cite{bhandari17}, whose 
algorithm we denote by \textbf{BKR} for brevity. We compare all four approaches  across two different noise models, Bounded and Gaussian, 
on two different functions: the function  \eqref{def:f1} used thus  far in the experiments throughout this paper, and a bandlimited  function used 
by \cite{bhandari17}.

\begin{figure}[!ht]
\centering
\subcaptionbox[]{  $\sigma=0.05$, \textbf{BKR}
}[ 0.24\textwidth ]
{\includegraphics[width=0.24\textwidth] {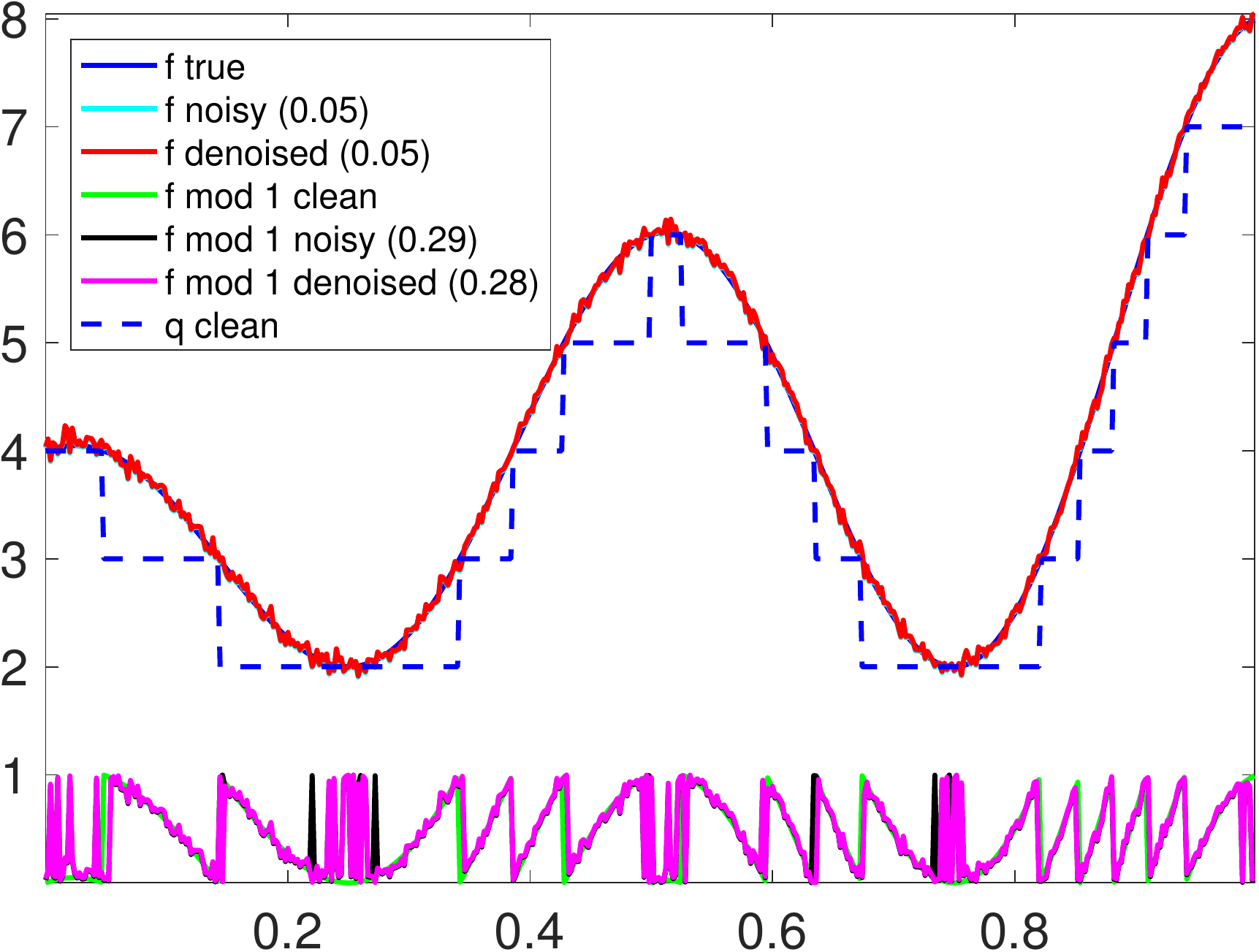} }
%
\subcaptionbox[]{  $\sigma=0.05$, \textbf{OLS}
}[ 0.24\textwidth ]
{\includegraphics[width=0.24\textwidth] {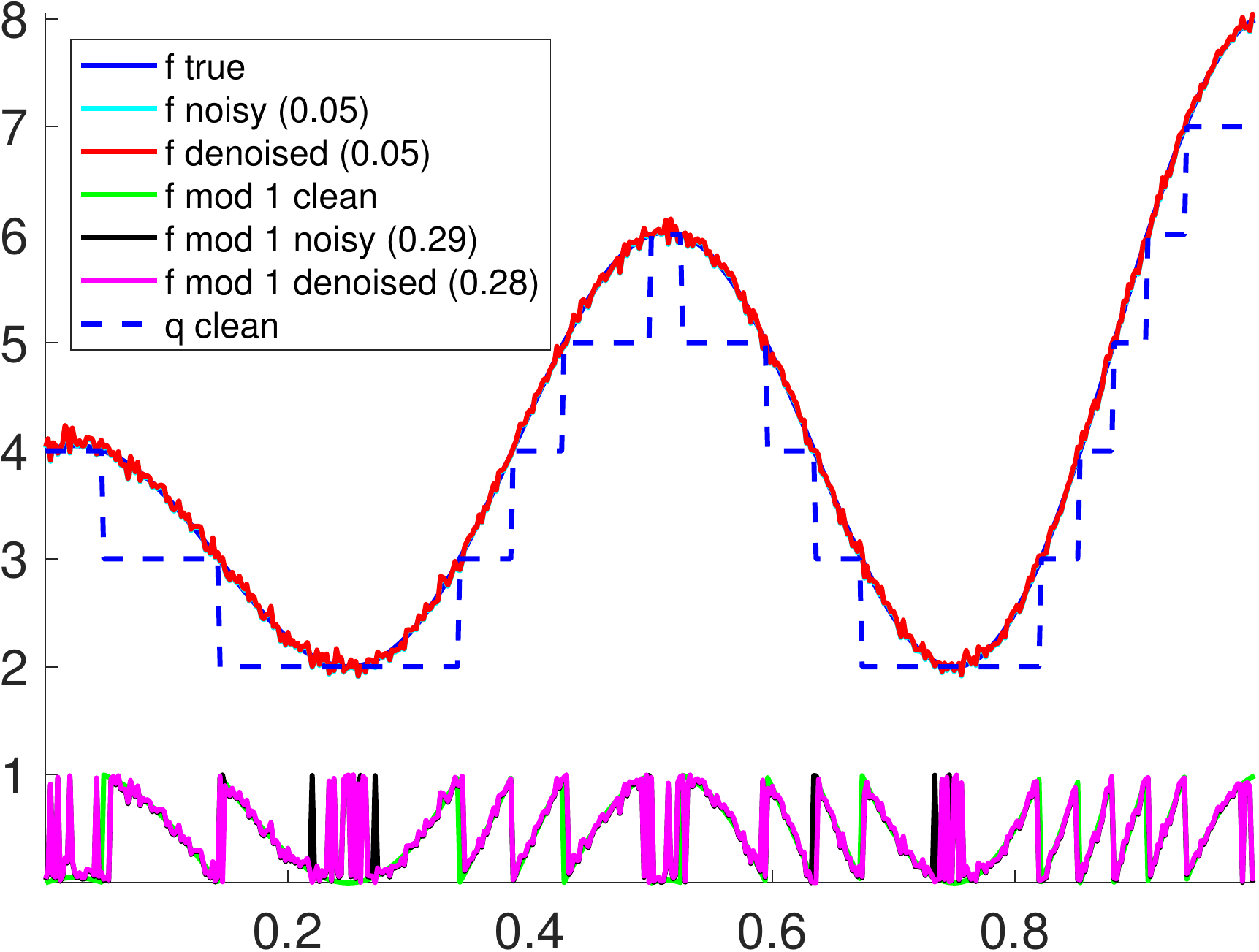} }
%
\subcaptionbox[]{  $\sigma=0.05$, \textbf{QCQP}
}[ 0.24\textwidth ]
{\includegraphics[width=0.24\textwidth] {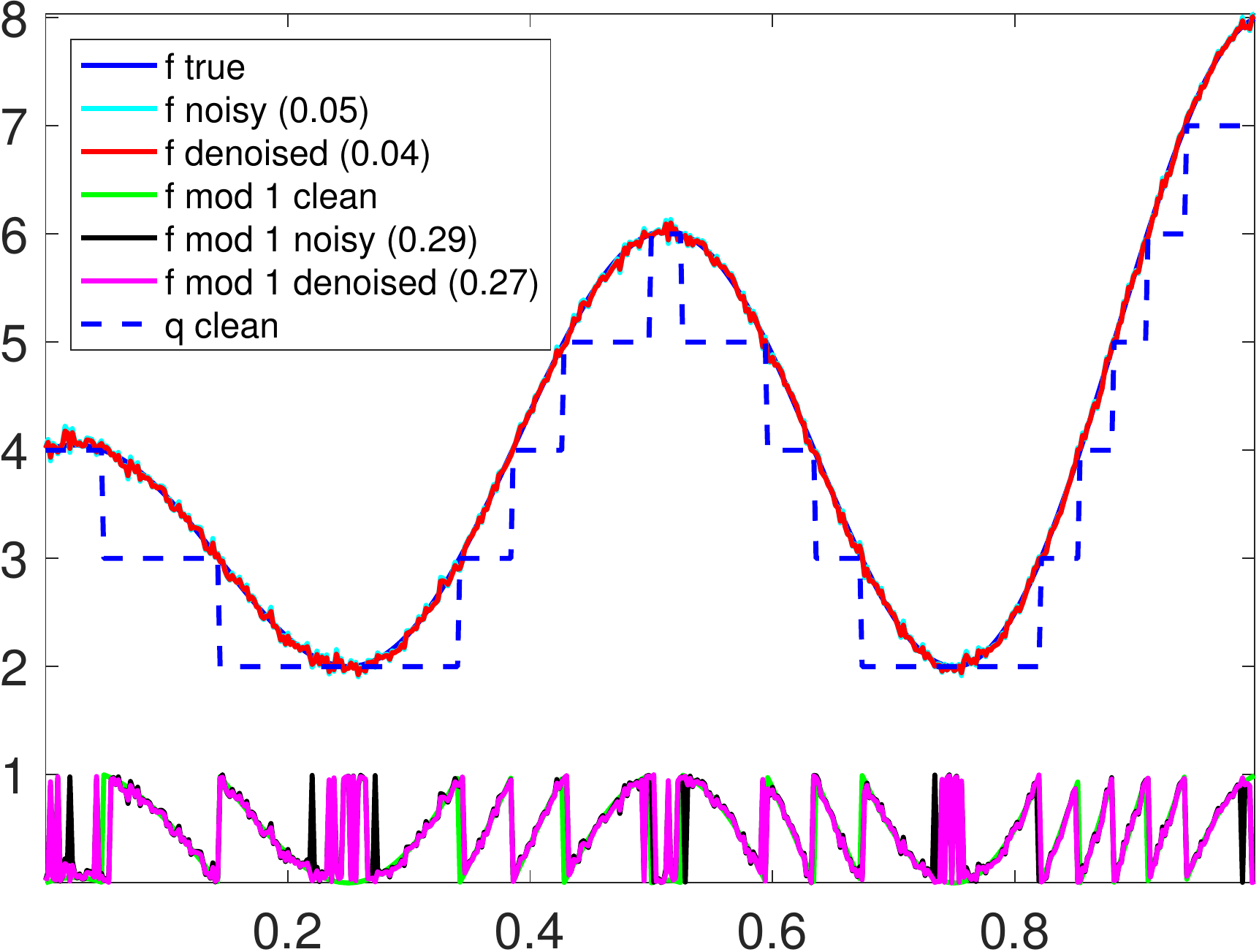} }
%
\subcaptionbox[]{  $\sigma=0.05$, \textbf{iQCQP}
}[ 0.24\textwidth ]
{\includegraphics[width=0.24\textwidth] {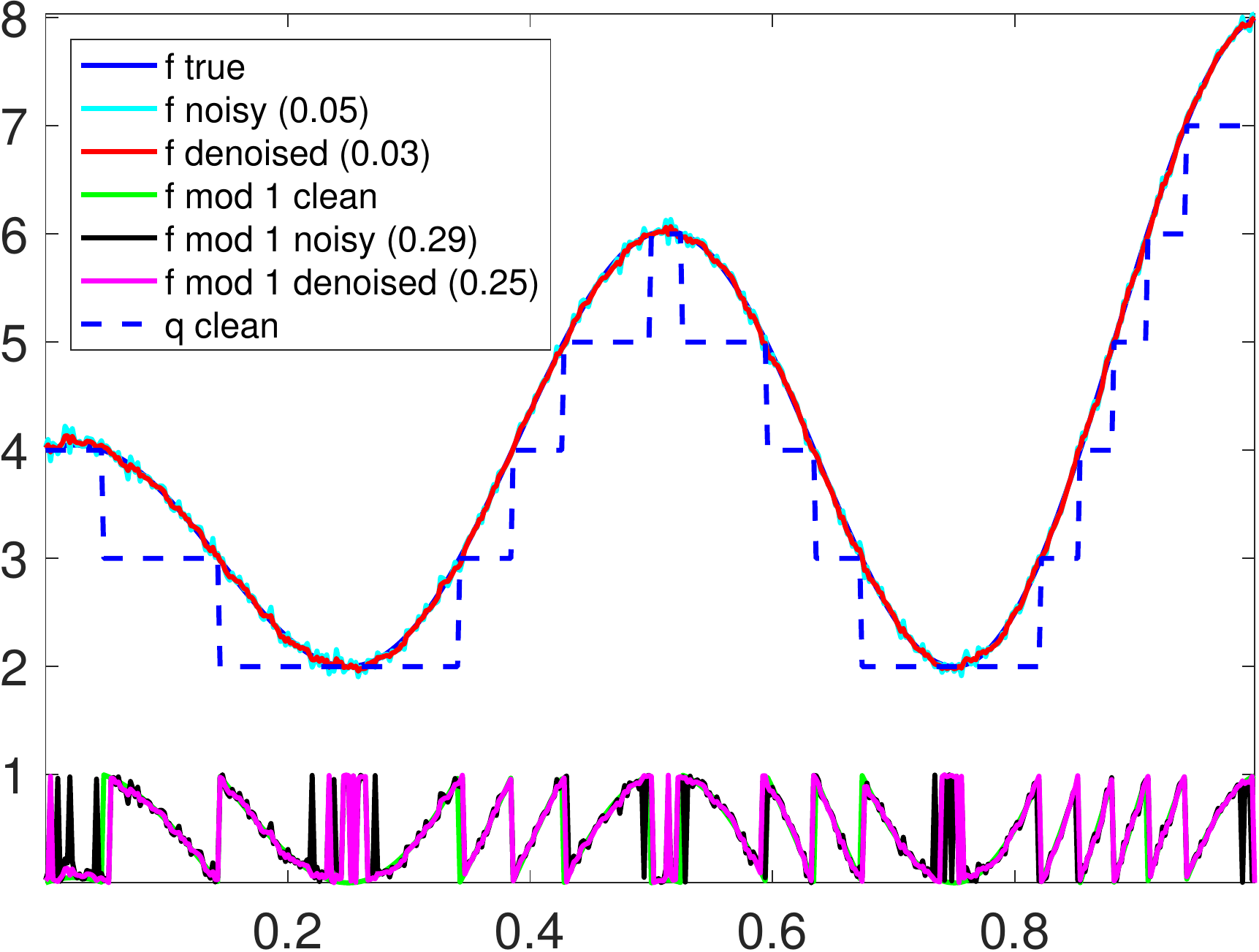} }
%
%
%
\subcaptionbox[]{  $\sigma=0.06$, \textbf{BKR}
}[ 0.24\textwidth ]
{\includegraphics[width=0.24\textwidth] {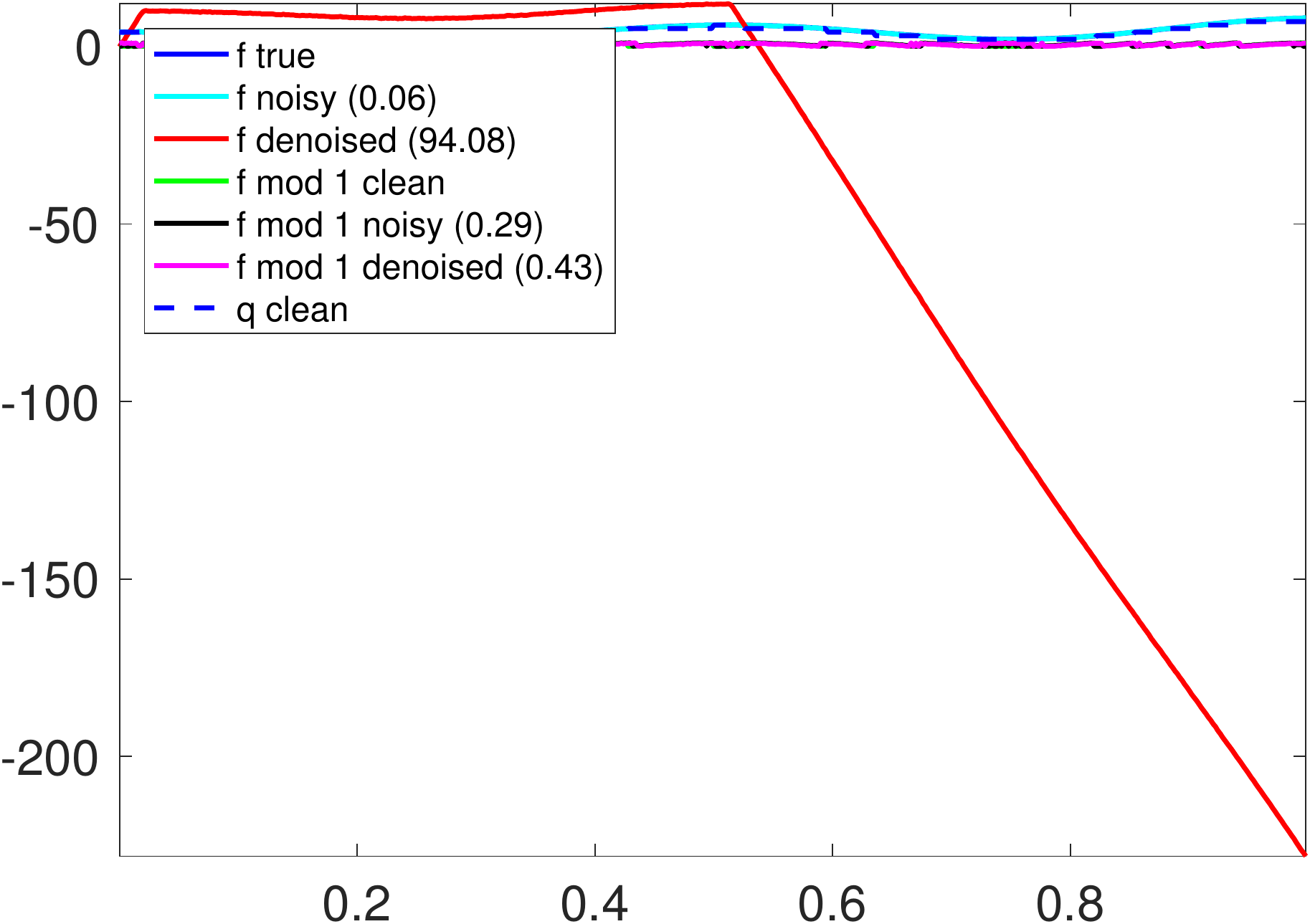} }
%
\subcaptionbox[]{  $\sigma=0.06$, \textbf{OLS}
}[ 0.24\textwidth ]
{\includegraphics[width=0.24\textwidth] {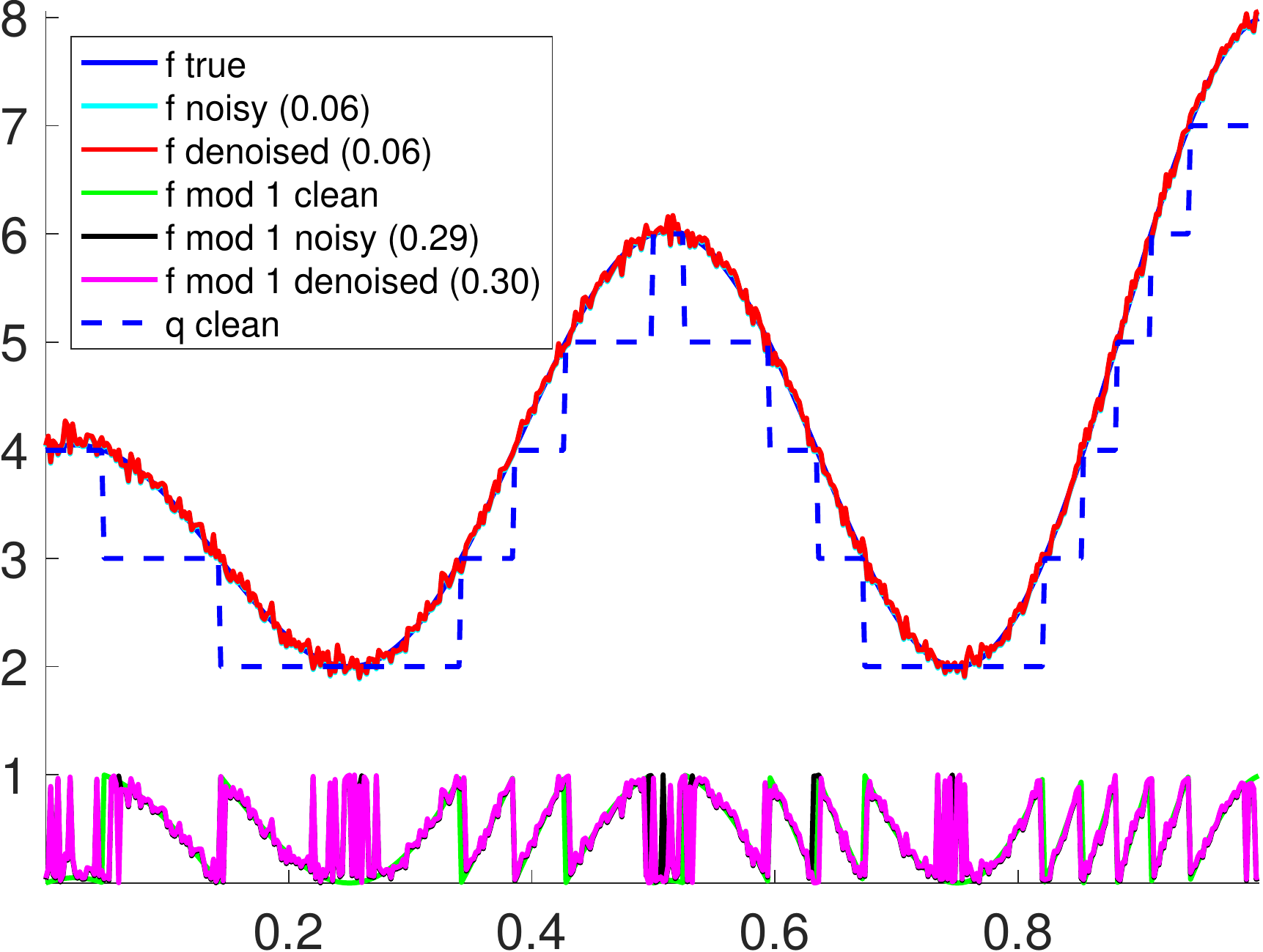} }
%
\subcaptionbox[]{  $\sigma=0.06$, \textbf{QCQP}
}[ 0.24\textwidth ]
{\includegraphics[width=0.24\textwidth] {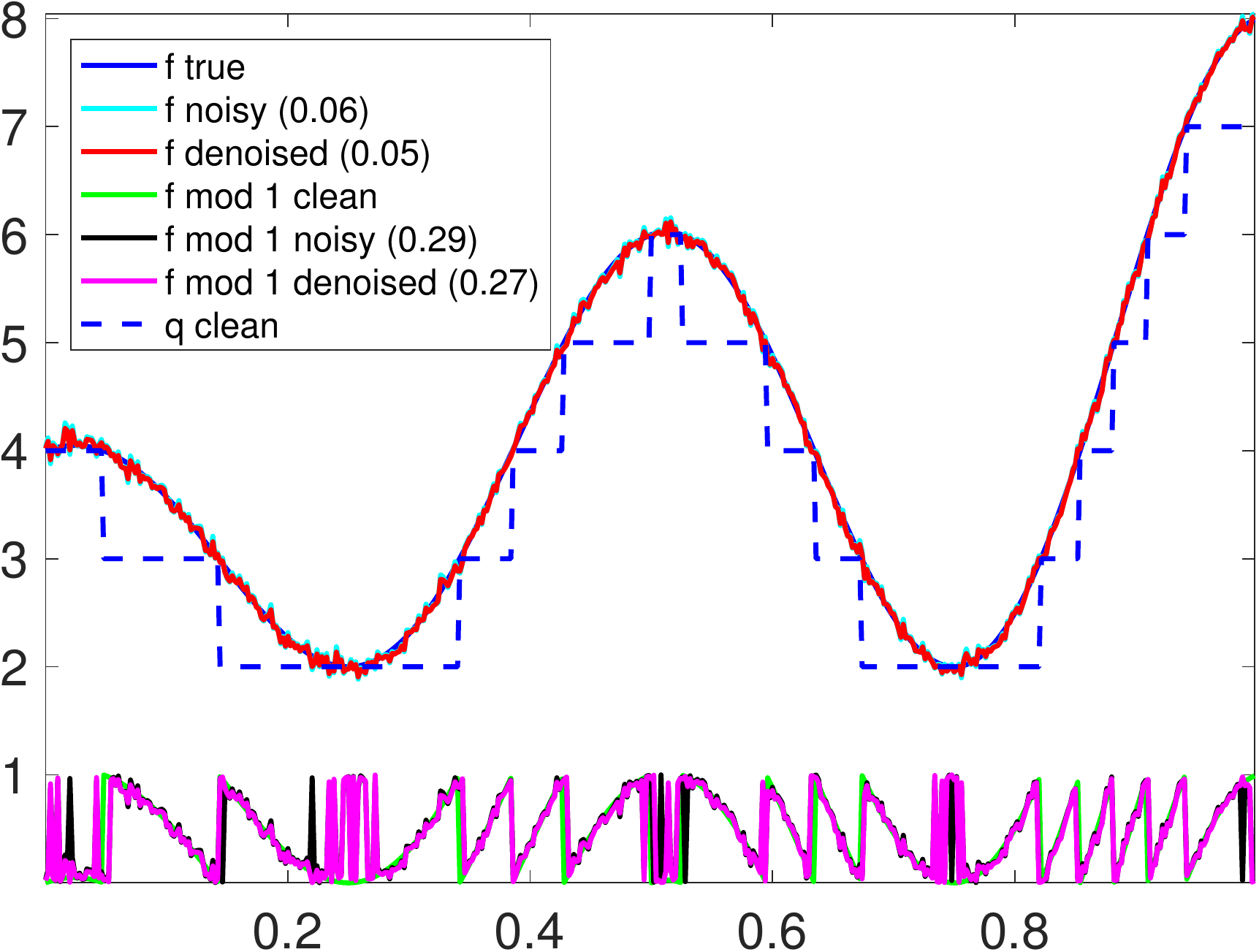} }
%
\subcaptionbox[]{  $\sigma=0.06$, \textbf{iQCQP}
}[ 0.24\textwidth ]
{\includegraphics[width=0.24\textwidth] {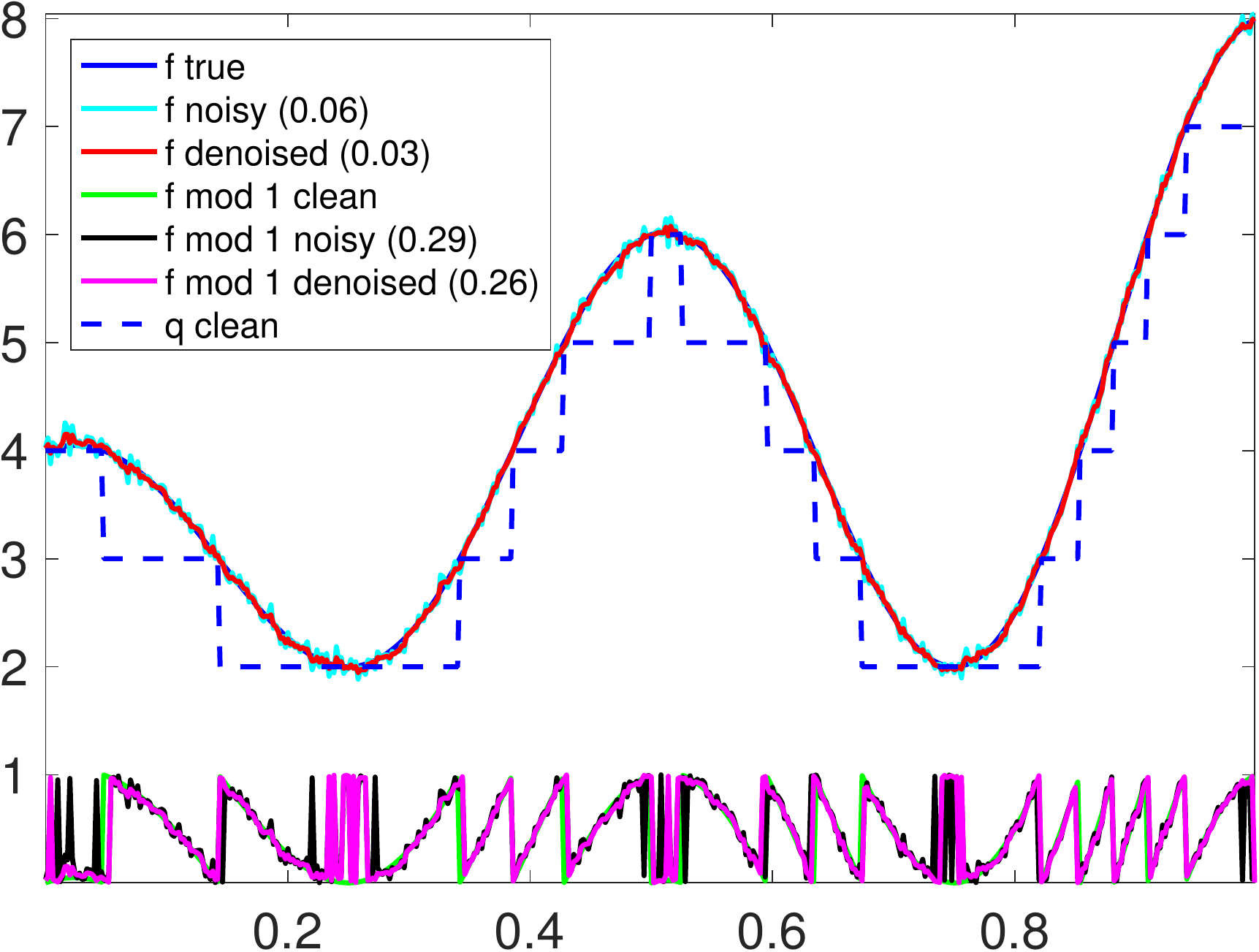} }
%
%
%
%
%
%
\subcaptionbox[]{  $\sigma=0.13$, \textbf{BKR}
}[ 0.24\textwidth ]
{\includegraphics[width=0.24\textwidth] {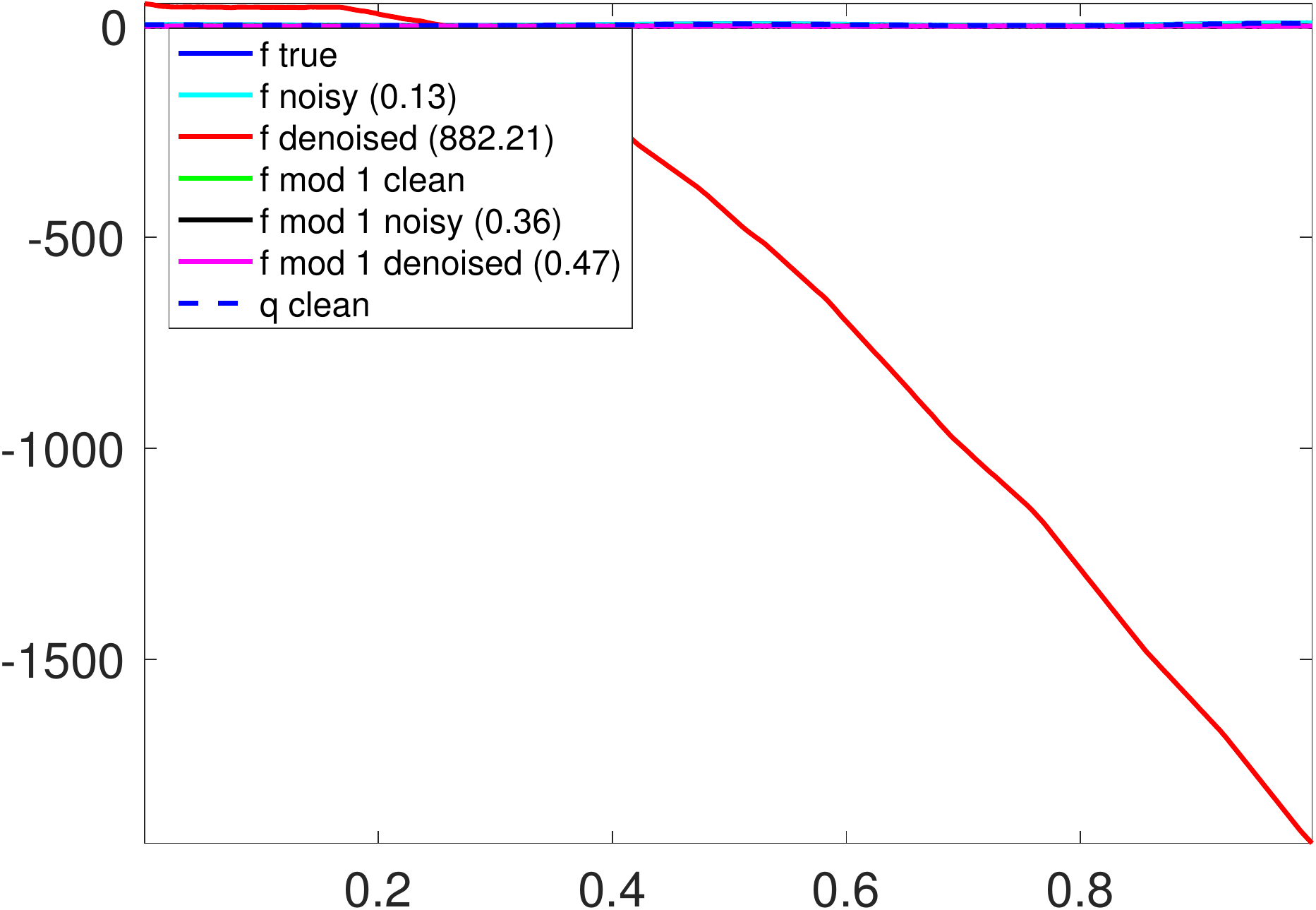} }
%
\subcaptionbox[]{  $\sigma=0.13$, \textbf{OLS}
}[ 0.24\textwidth ]
{\includegraphics[width=0.24\textwidth] {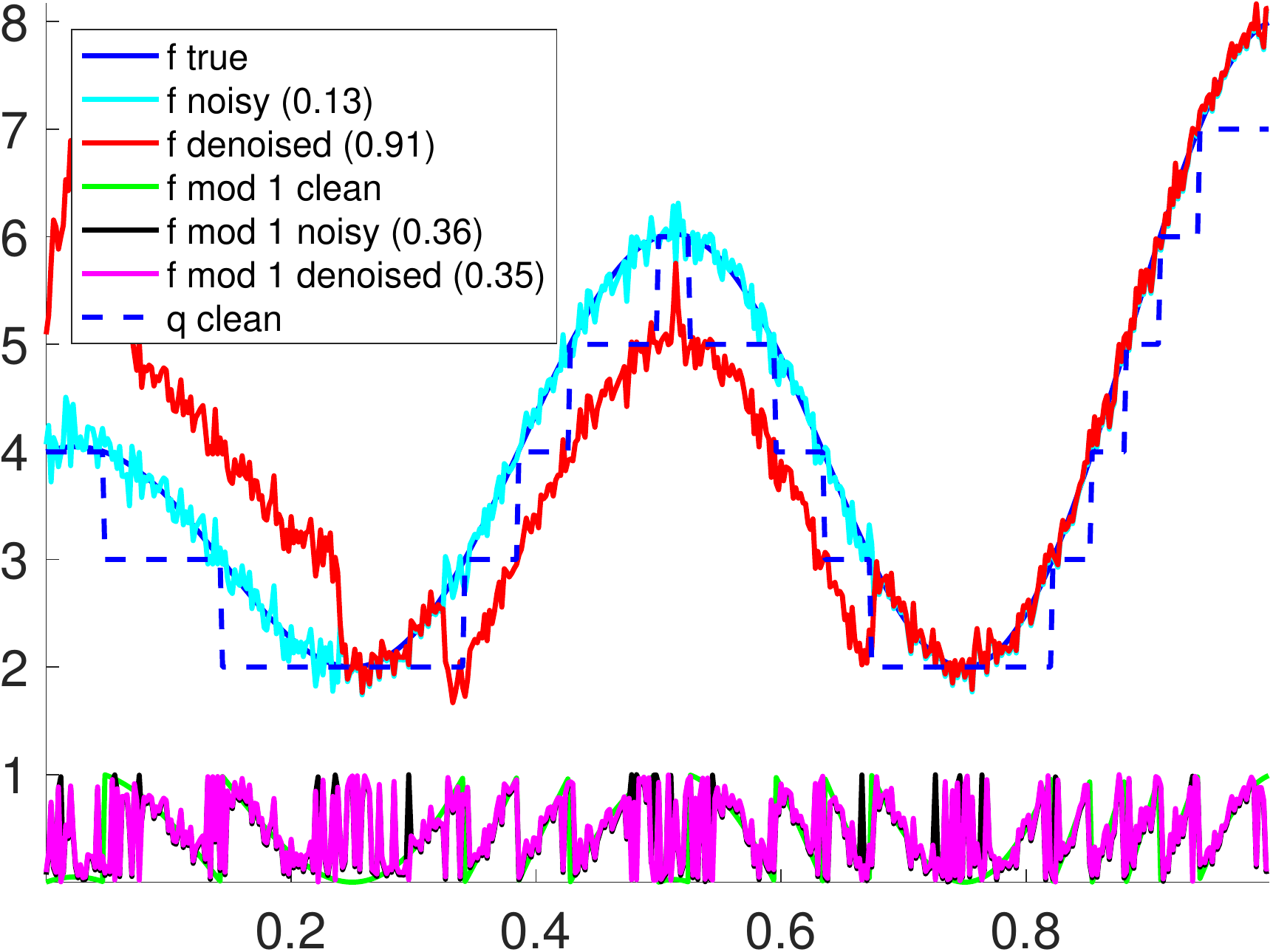} }
%
\subcaptionbox[]{  $\sigma=0.13$, \textbf{QCQP}
}[ 0.24\textwidth ]
{\includegraphics[width=0.24\textwidth] {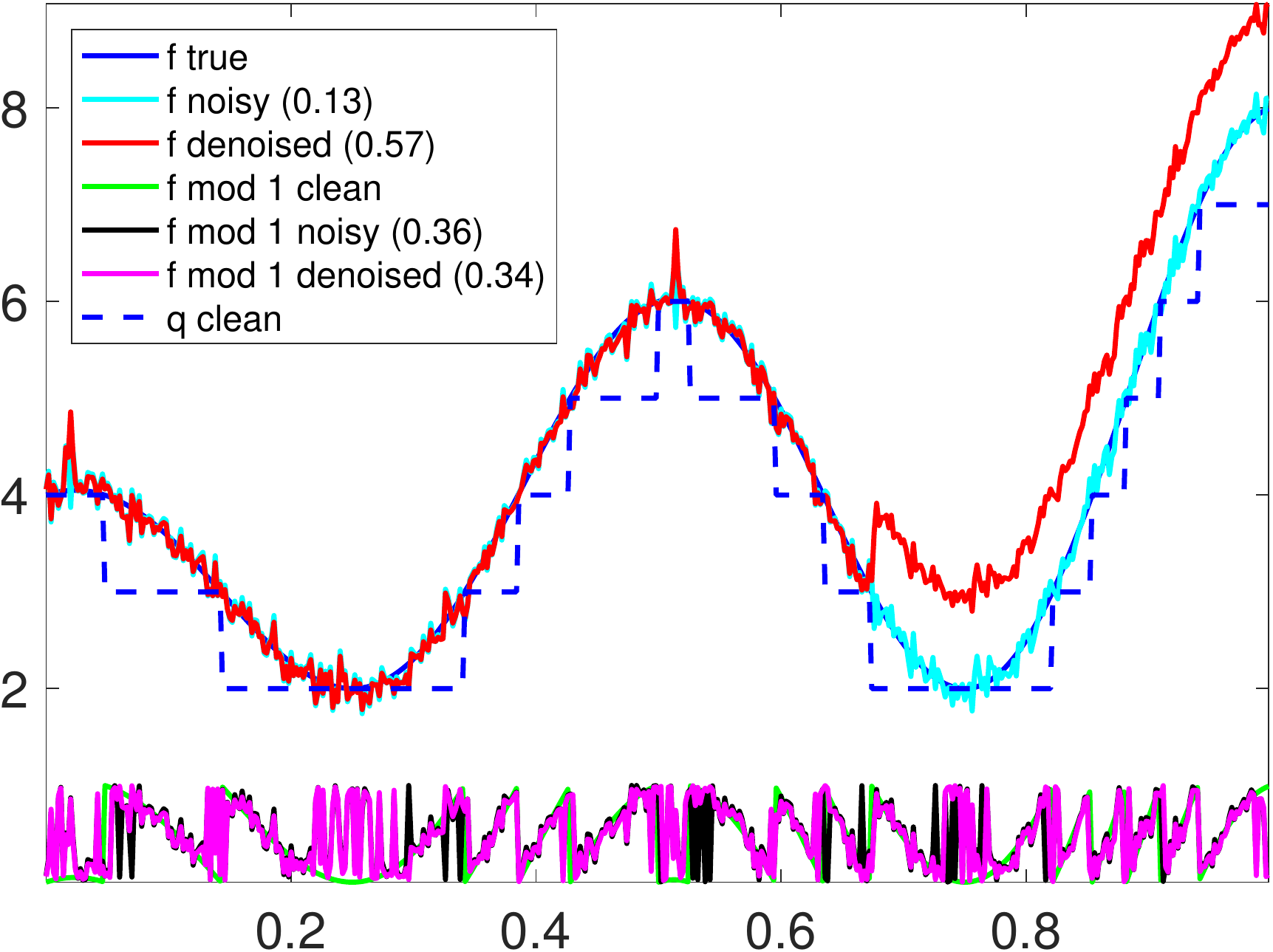} }
%
\subcaptionbox[]{  $\sigma=0.13$, \textbf{iQCQP}
}[ 0.24\textwidth ]
{\includegraphics[width=0.24\textwidth] {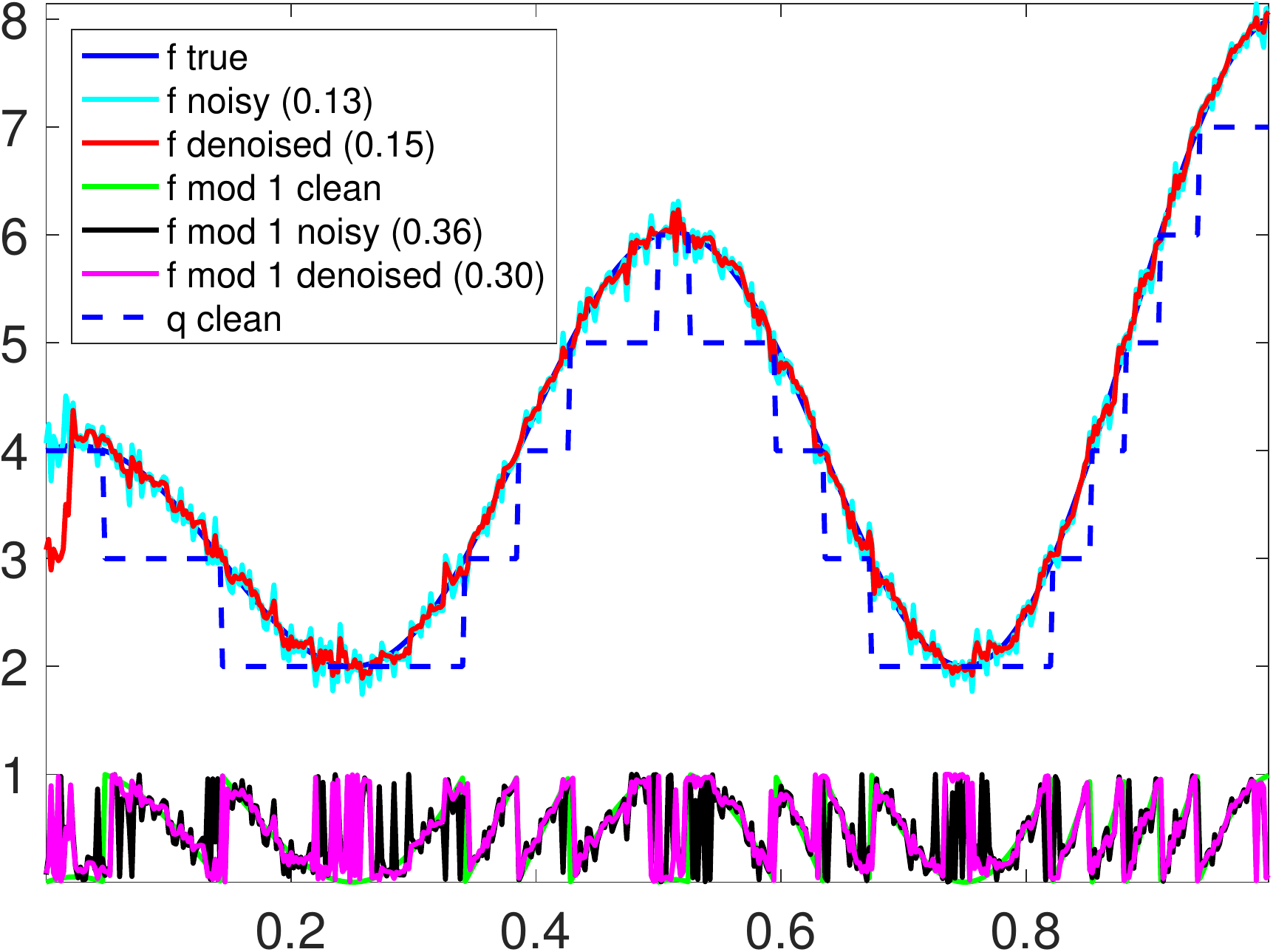} }
%
%
\vspace{-2mm}
\captionsetup{width=0.98\linewidth}
\caption[Short Caption]{Denoised instances for both the $f$ mod 1 and $f$ values, under the Gaussian noise model for function \eqref{def:f1}, for  \textbf{BKR},  \textbf{OLS}, \textbf{QCQP} and \textbf{iQCQP},  as we increase the noise level $\sigma$.  \textbf{QCQP} denotes Algorithm \ref{algo:two_stage_denoise}, for which  the unwrapping stage is  performed via \textbf{OLS} \eqref{eq:ols_unwrap_lin_system}.  
  We keep fixed the parameters $n=500$, $k=2$, $\lambda= 0.1$. The numerical values in the legend denote the RMSE.
}
\label{fig:instances_f1_Gaussian_Sampta}
\end{figure}

Figure  \ref{fig:instances_f1_Gaussian_Sampta}   pertains to function \eqref{def:f1} but for the Gaussian noise model. At a lower level of noise $\sigma=5\%$, all methods are able to recover the function fairly well, with   \textbf{iQCQP}  and \textbf{QCQP} yielding the most accurate reconstructions in terms of the RMSE of both the denoised $f$ mod 1 samples, and the final estimates for $f$. At higher levels of noise $\sigma \in \{ 6\%, 13\%\} $,  \textbf{BKR} breaks down, while the other methods (in order of accuracy) \textbf{iQCQP},  \textbf{QCQP}, \textbf{OLS}  return  meaningful  estimates.

Figures \ref{fig:instances_f_BL_Bounded}, respectively \ref{fig:instances_f_BL_Gaussian}, compare the results of the four methods on the same bandlimited function considered by the authors of \cite{bhandari17}, under the Bounded, respectively Gaussian, noise models with $n=494$ sample points.  
The bandlimited function is generated by multiplying the Fourier spectrum of the sinc-function with weights drawn from the standard distribution, rescaled such that the largest magnitude entry is equal to 1. To make the problem more challenging,  we also scale the function by a factor of $3$, such that the function wraps itself more often when taking modulo 1 values. Finally, to make the resulting figures more visually appealing, we shift the function values by +3 in order to minimize the overlap with the modulo 1 values (clean, noisy and denoised samples). For the Bounded noise model, Figure \ref{fig:instances_f_BL_Bounded} illustrates the fact that at a lower noise level $\gamma = 15 \%$ all methods recover  a good approximation of the function, the winners being \textbf{iQCQP} and  \textbf{QCQP} both in terms of denoising the $f$ mod 1 samples and estimating $f$. 
At  $\gamma = 17\%$ \textbf{BKR} breaks down, while the remaining methods return good approximations. Finally, at $\gamma = 20\% $ the latter three methods have difficulties in the unwrapping stage.
Figure  \ref{fig:instances_f_BL_Gaussian} shows analogous results and conclusions under the more challenging  Gaussian noise model, with $ \sigma \in \{ 5\%,  8\%, 12\%\}$.

Finally,  Figure  \ref{fig:instances_f_BL_200_Bounded}, shows analogous results as in Figures \ref{fig:instances_f_BL_Bounded} under the Bounded noise model, but with a sparser sampling pattern $n=194$, similar to the setup used by the authors of  \cite{bhandari17}.  At a lower noise level,  $\gamma = 10 \%$,  all methods recover  a good approximation of the function, the performance ranking being    
\textbf{iQCQP} (RMSE=0.23), 
\textbf{QCQP} (RMSE=0.23), 
\textbf{OLS} (RMSE=0.31),  
\textbf{BKR} (RMSE=0.31) for the denoised $f$ mod 1 samples,  and  
\textbf{iQCQP} (RMSE=0.06), 
\textbf{QCQP} (RMSE=0.11), 
\textbf{OLS} (RMSE=0.11),  
\textbf{BKR} (RMSE=0.11) for the final $f$ estimates.
At higher levels of noise,  \textbf{iQCQP},  \textbf{QCQP}, \textbf{OLS}  return meaningful  estimates, though all methods experience difficulties in the unwrapping stage, in the interval where the function $f$ has the highest slope.

\begin{figure}[!ht]
\centering
\subcaptionbox[]{  $\gamma=0.15$, \textbf{BKR}
}[ 0.24\textwidth ]
{\includegraphics[width=0.24\textwidth] {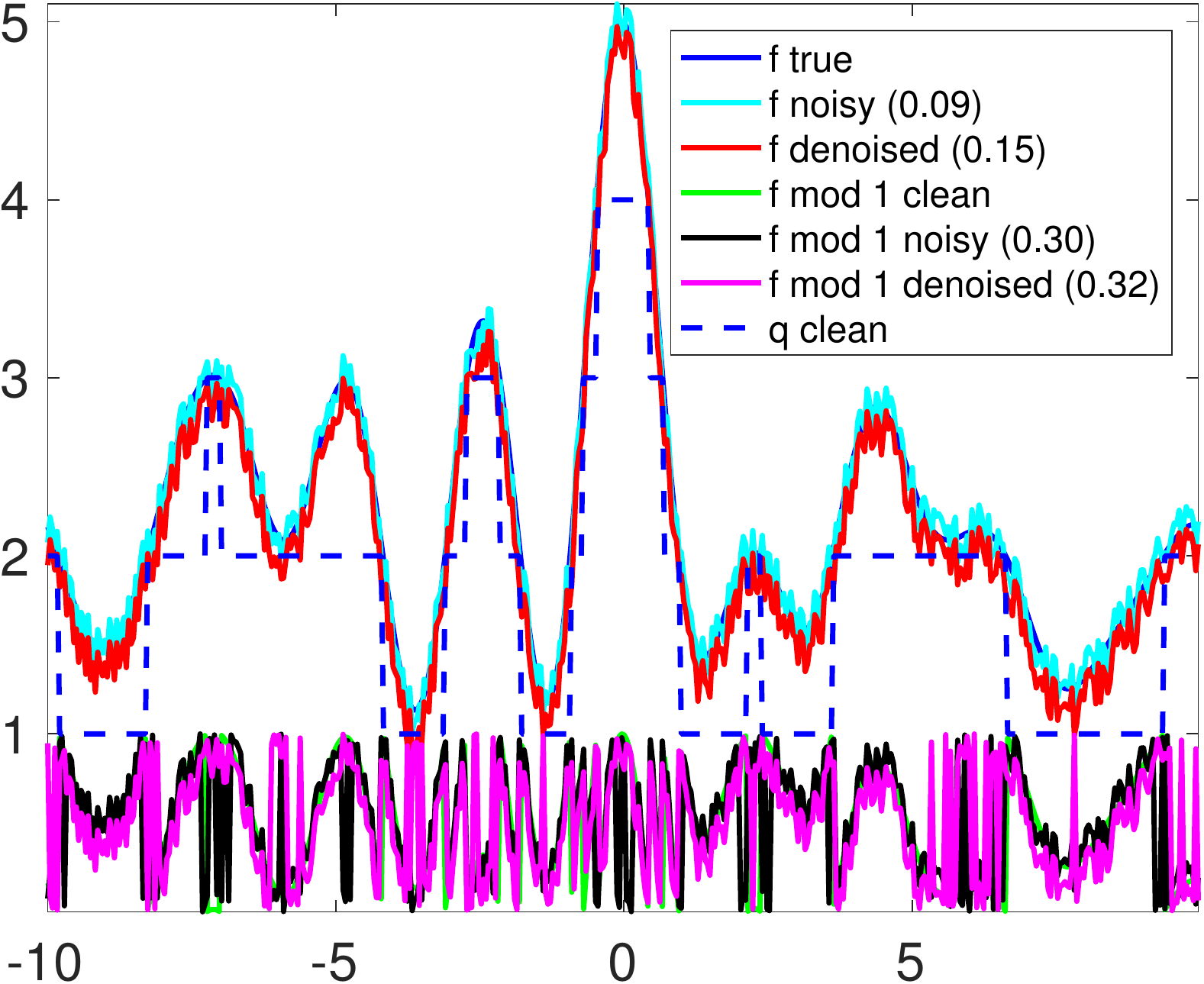} }
%
\subcaptionbox[]{  $\gamma=0.15$, \textbf{OLS}
}[ 0.24\textwidth ]
{\includegraphics[width=0.24\textwidth] {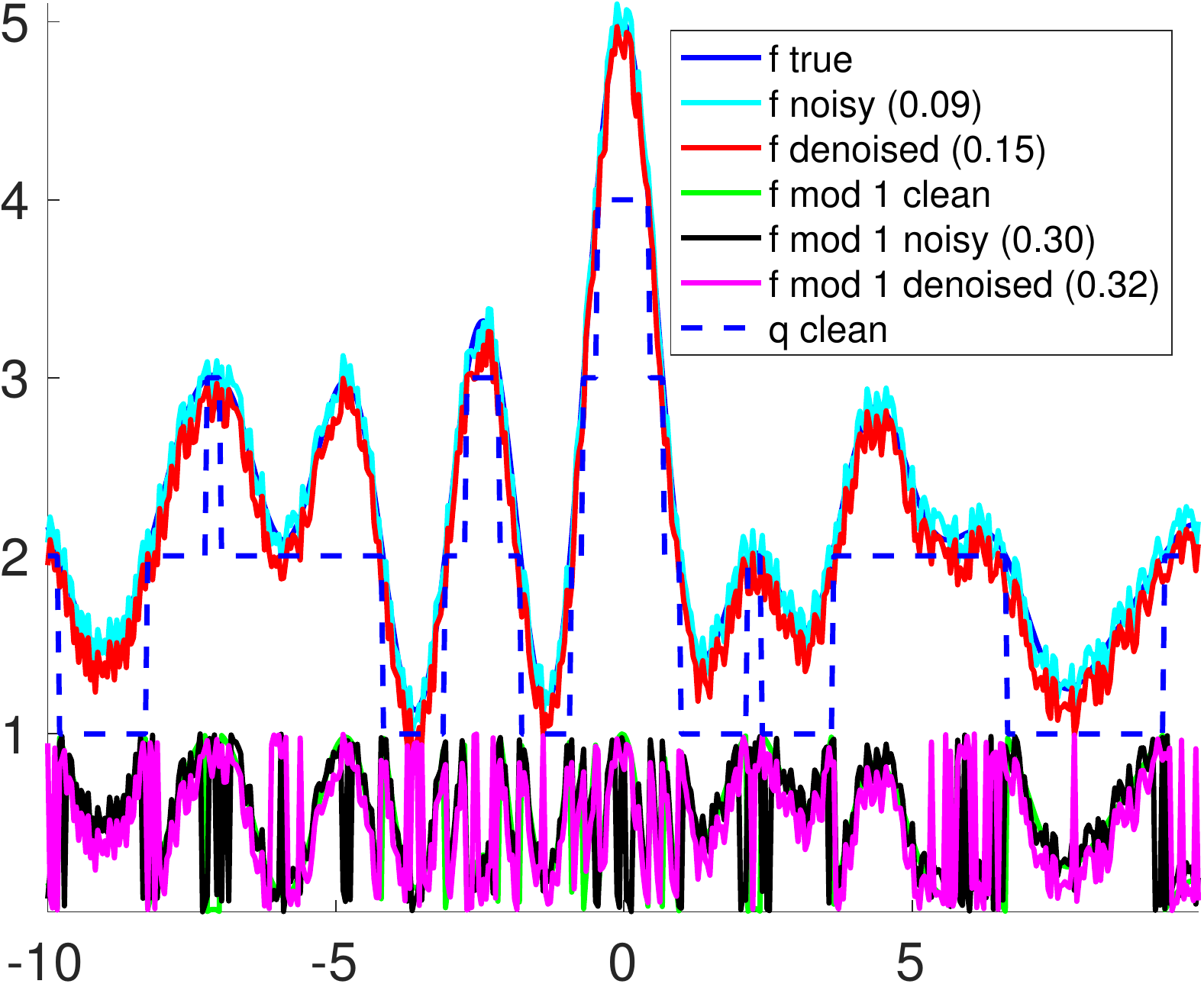} }
%
\subcaptionbox[]{  $\gamma=0.15$, \textbf{QCQP}
}[ 0.24\textwidth ]
{\includegraphics[width=0.24\textwidth] {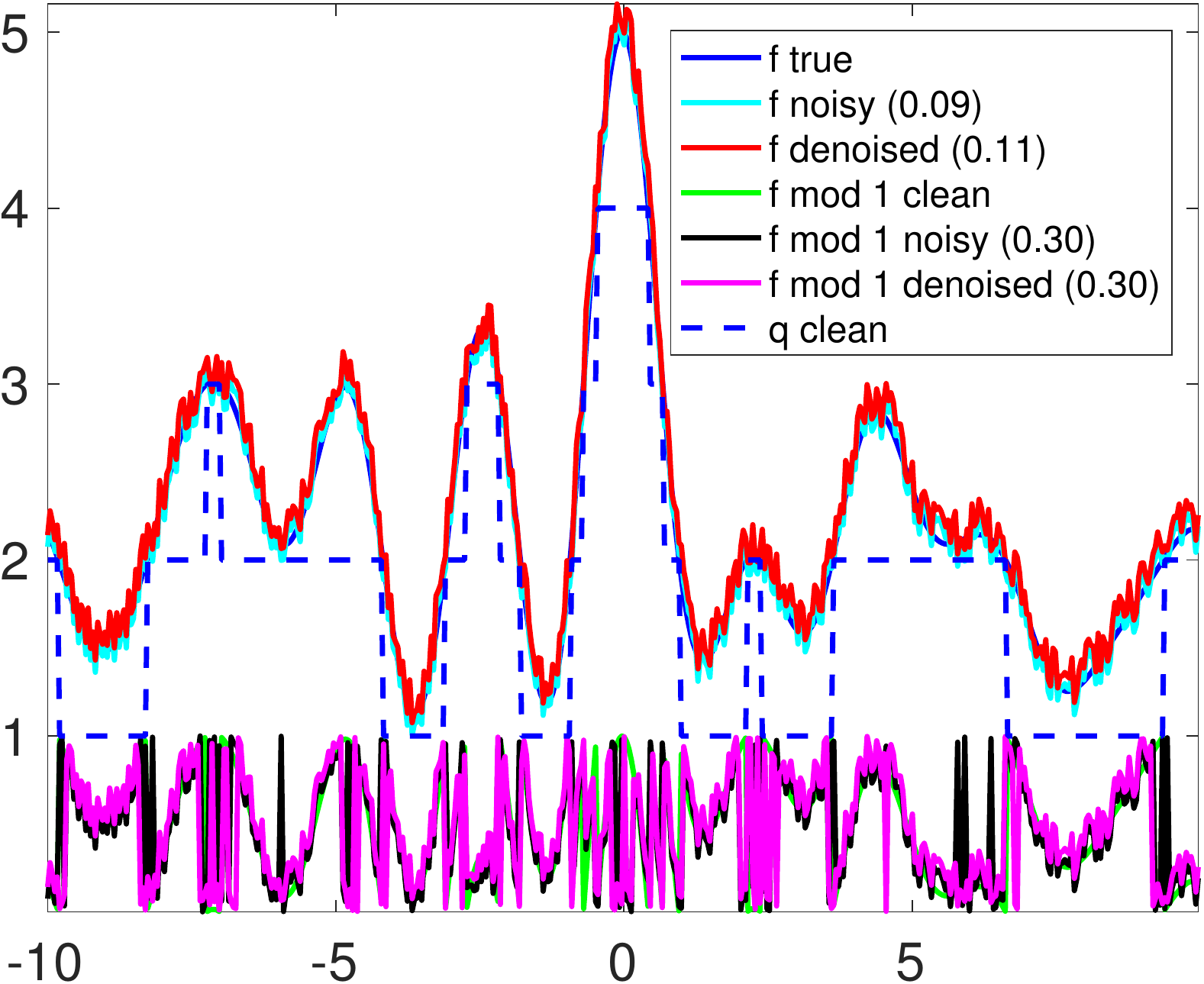} }
%
\subcaptionbox[]{  $\gamma=0.15$, \textbf{iQCQP}
}[ 0.24\textwidth ]
{\includegraphics[width=0.24\textwidth] {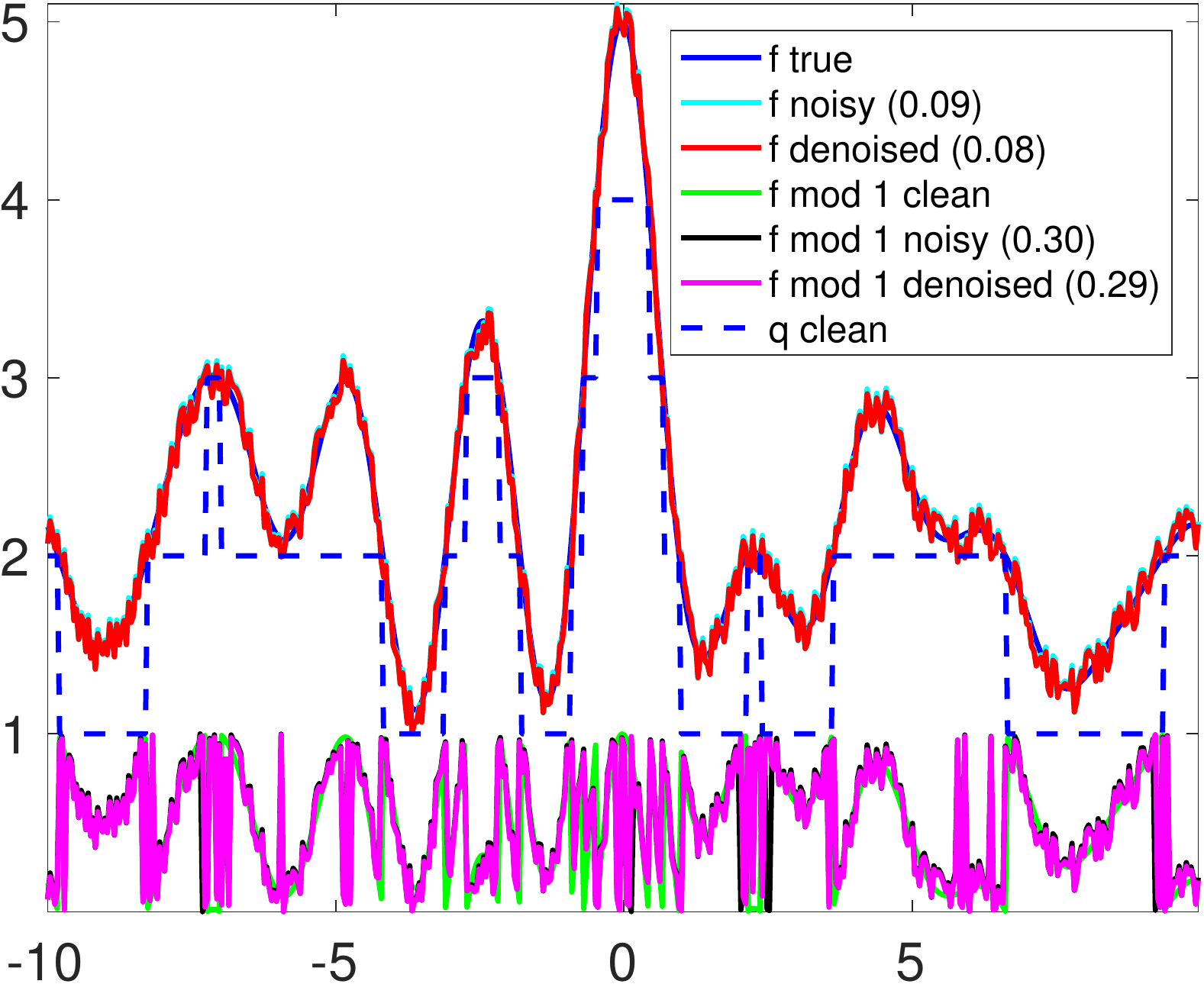} }
%
%
%
\subcaptionbox[]{  $\gamma=0.17$, \textbf{BKR}
}[ 0.24\textwidth ]
{\includegraphics[width=0.24\textwidth] {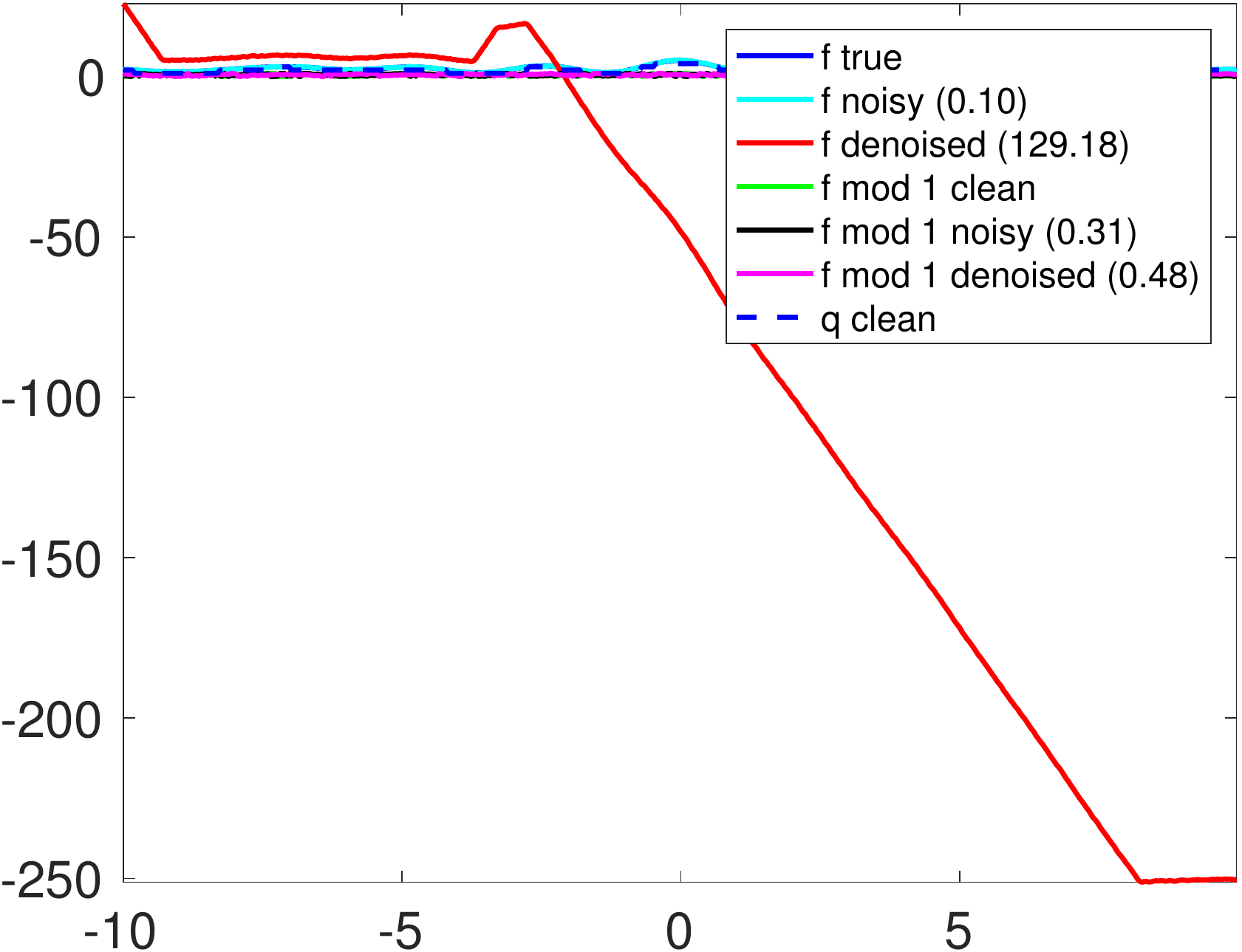} }
%
\subcaptionbox[]{  $\gamma=0.17$, \textbf{OLS}
}[ 0.24\textwidth ]
{\includegraphics[width=0.24\textwidth] {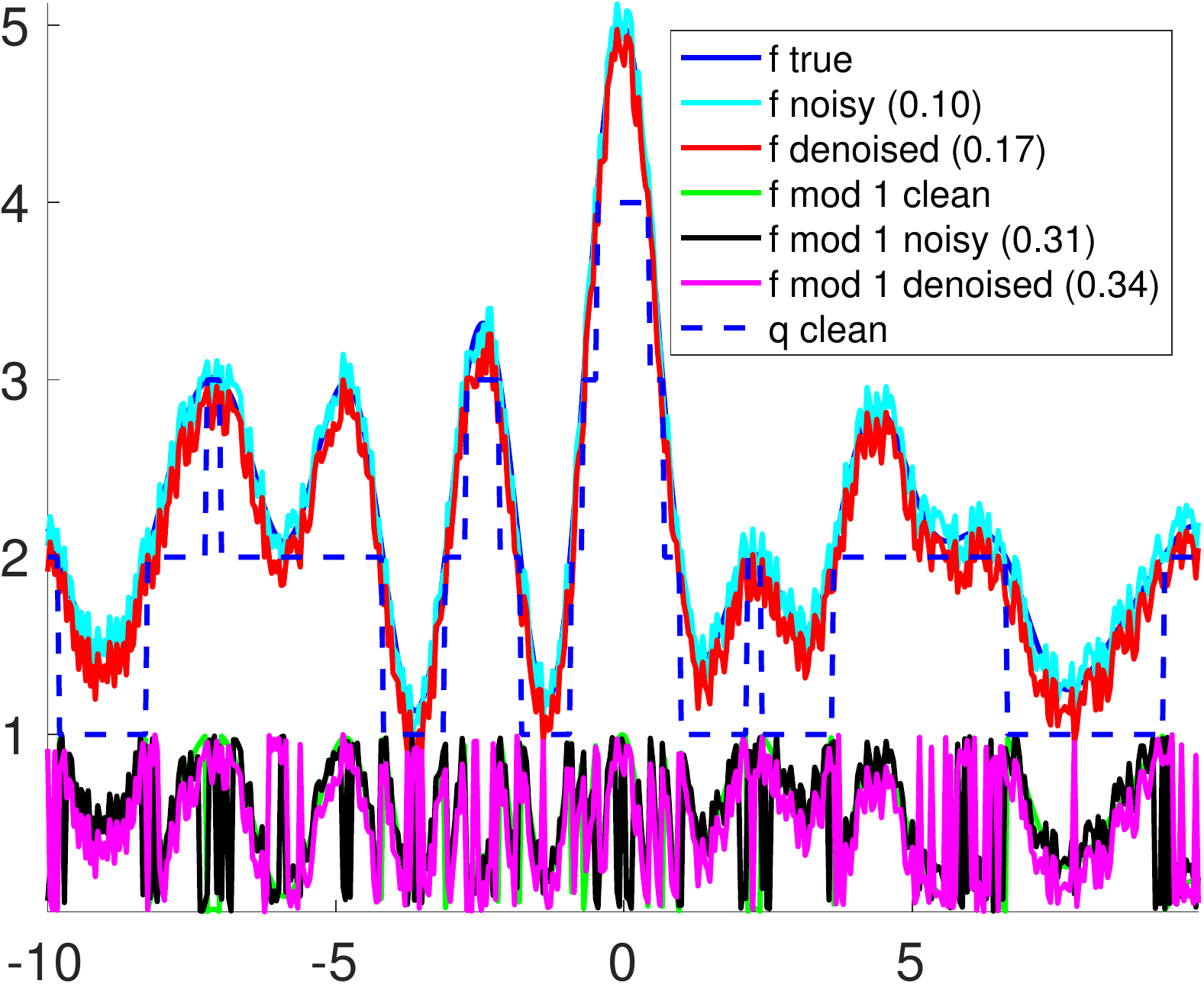} }
%
\subcaptionbox[]{  $\gamma=0.17$, \textbf{QCQP}
}[ 0.24\textwidth ]
{\includegraphics[width=0.24\textwidth] {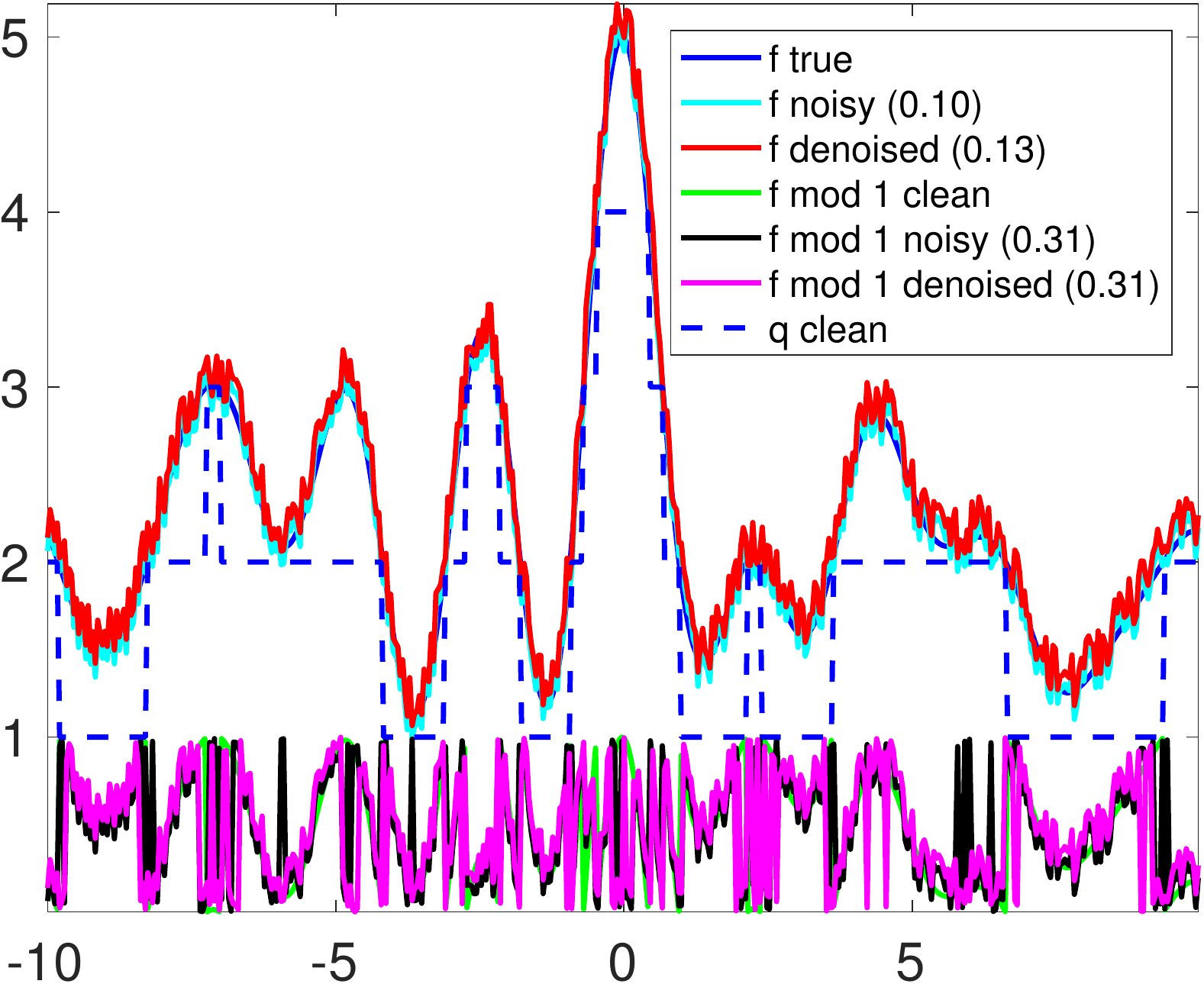} }
%
\subcaptionbox[]{  $\gamma=0.17$, \textbf{iQCQP}
}[ 0.24\textwidth ]
{\includegraphics[width=0.24\textwidth] {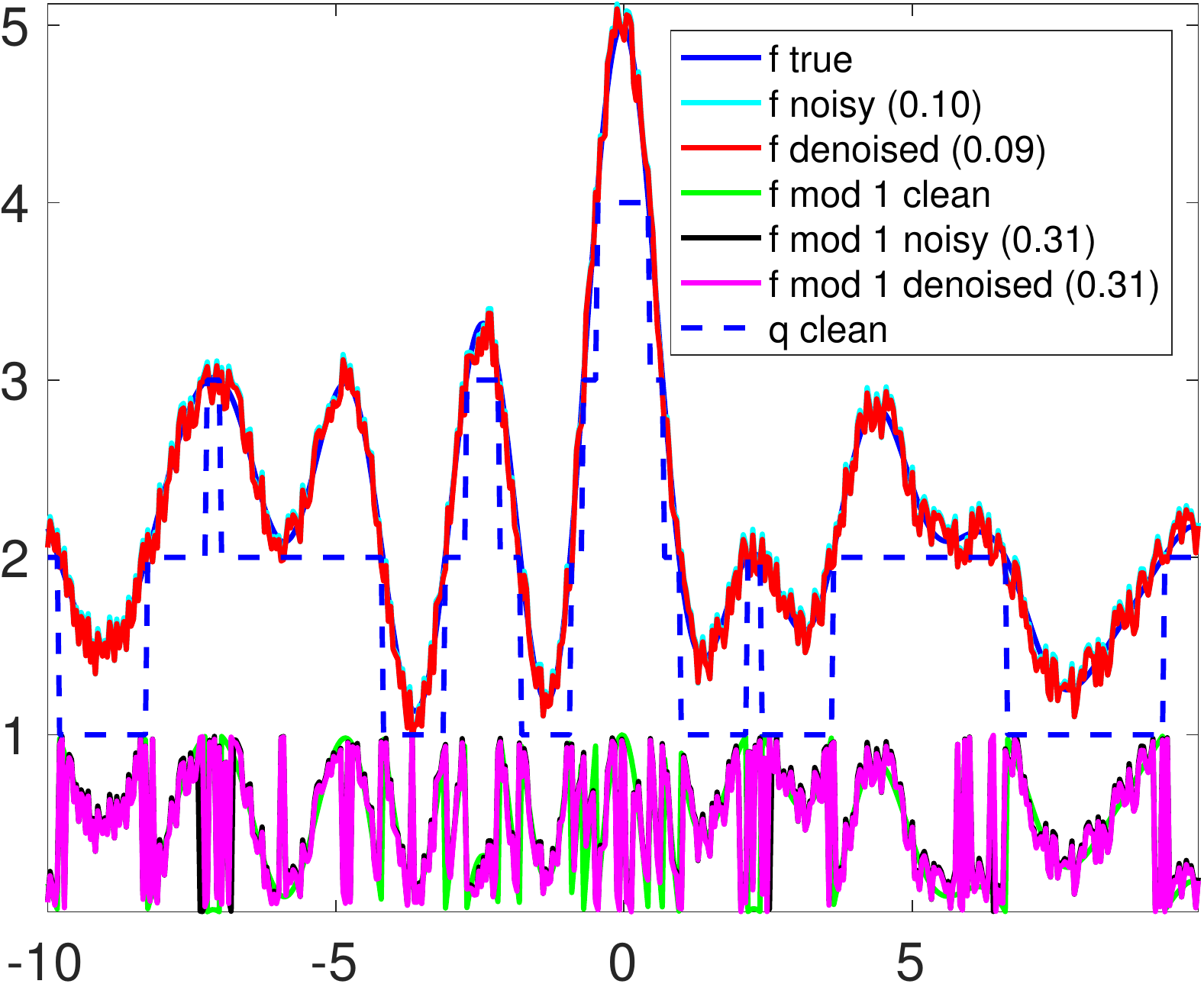} }
%
%
%
%
%
%
\subcaptionbox[]{  $\gamma=0.20$, \textbf{BKR}
}[ 0.24\textwidth ]
{\includegraphics[width=0.24\textwidth] {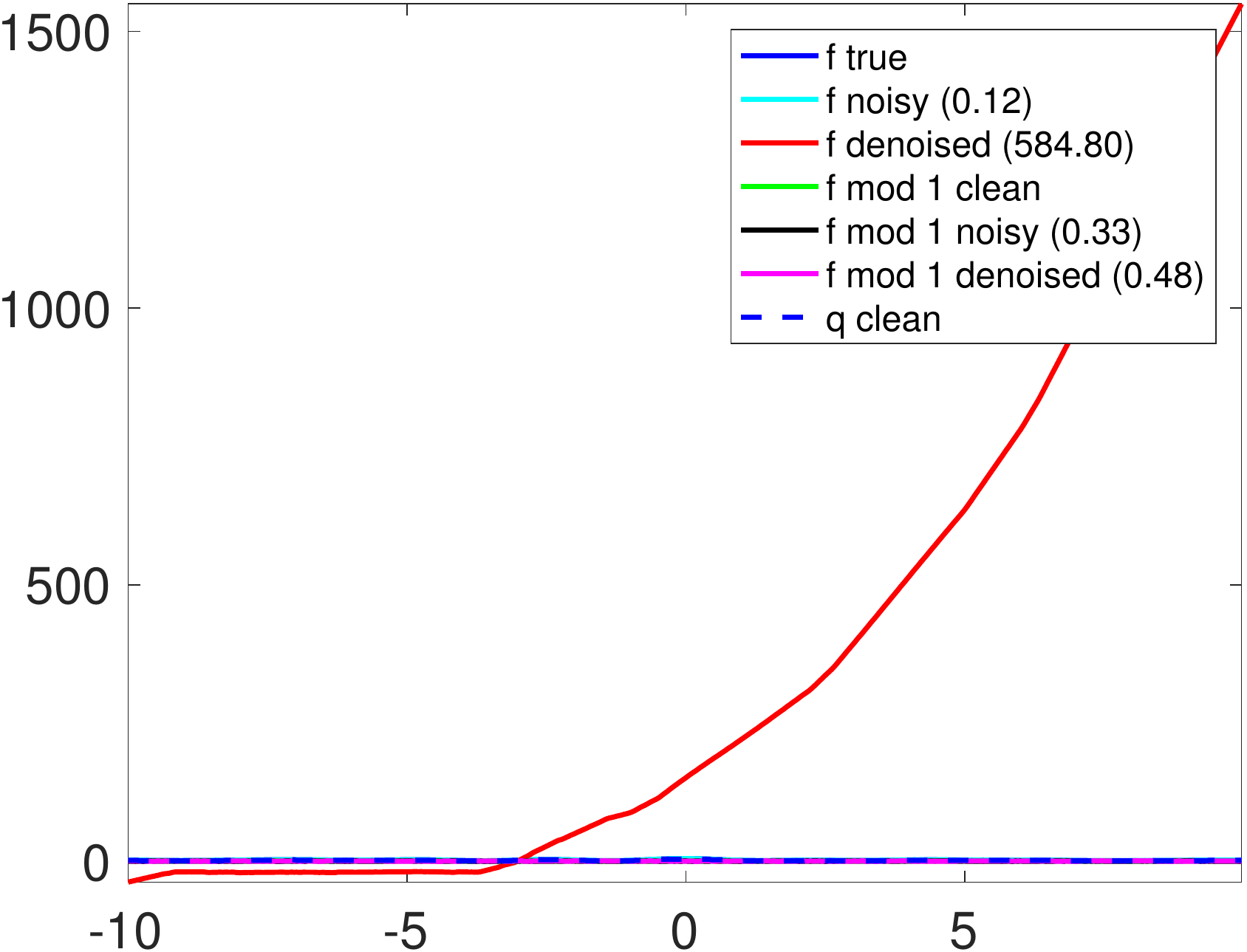} }
%
\subcaptionbox[]{  $\gamma=0.20$, \textbf{OLS}
}[ 0.24\textwidth ]
{\includegraphics[width=0.24\textwidth] {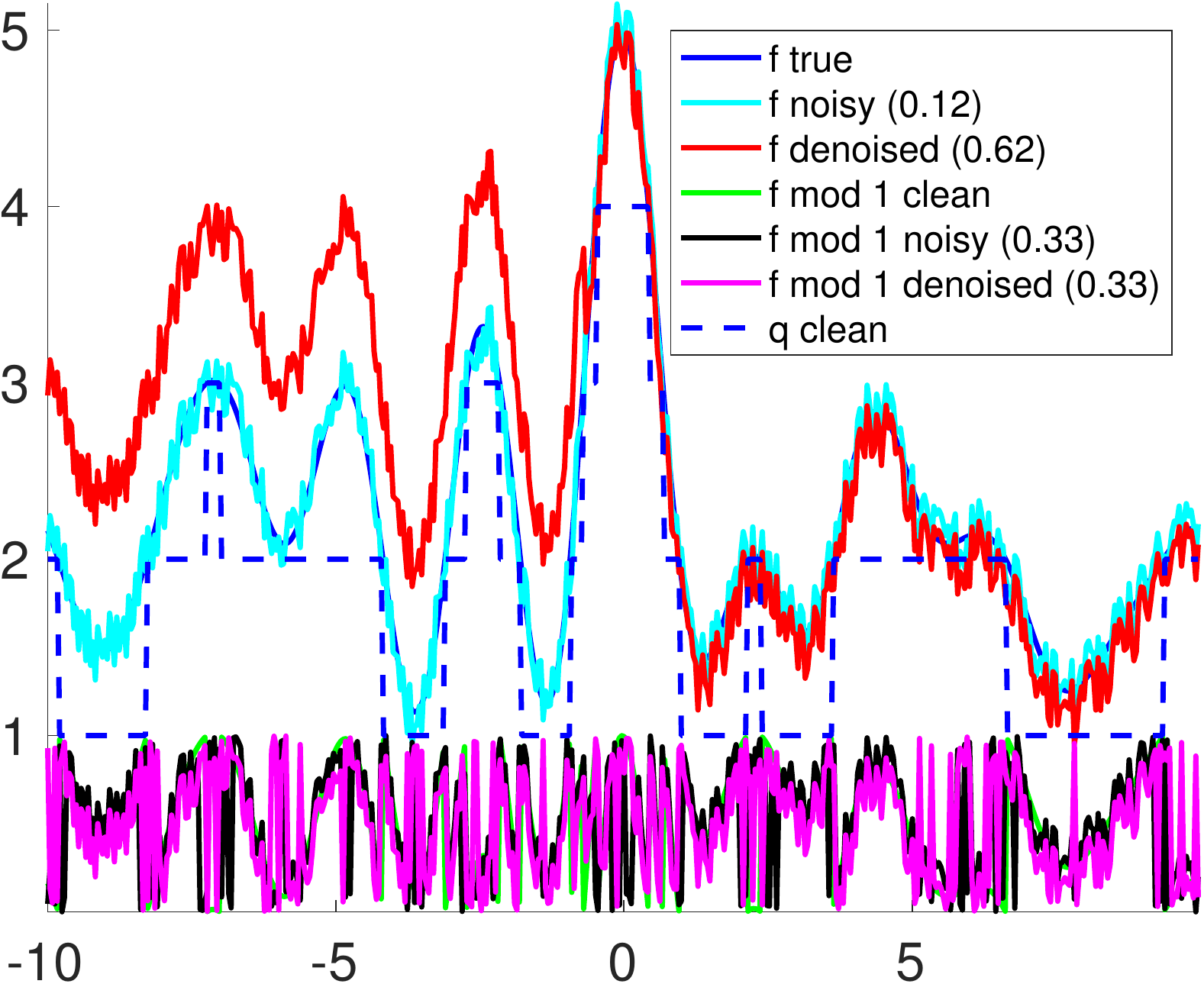} }
%
\subcaptionbox[]{  $\gamma=0.20$, \textbf{QCQP}
}[ 0.24\textwidth ]
{\includegraphics[width=0.24\textwidth] {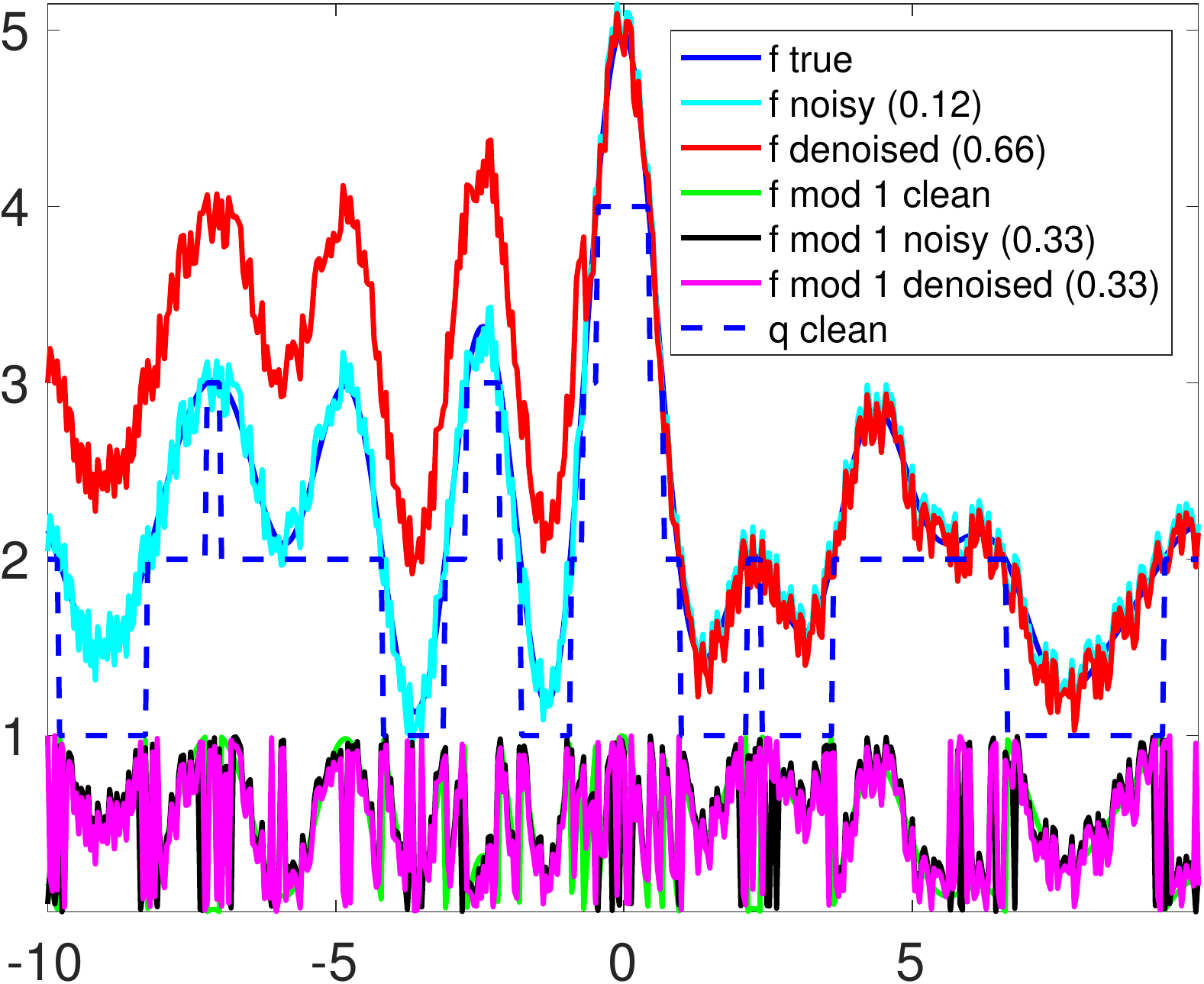} }
%
\subcaptionbox[]{  $\gamma=0.20$, \textbf{iQCQP}
}[ 0.24\textwidth ]
{\includegraphics[width=0.24\textwidth] {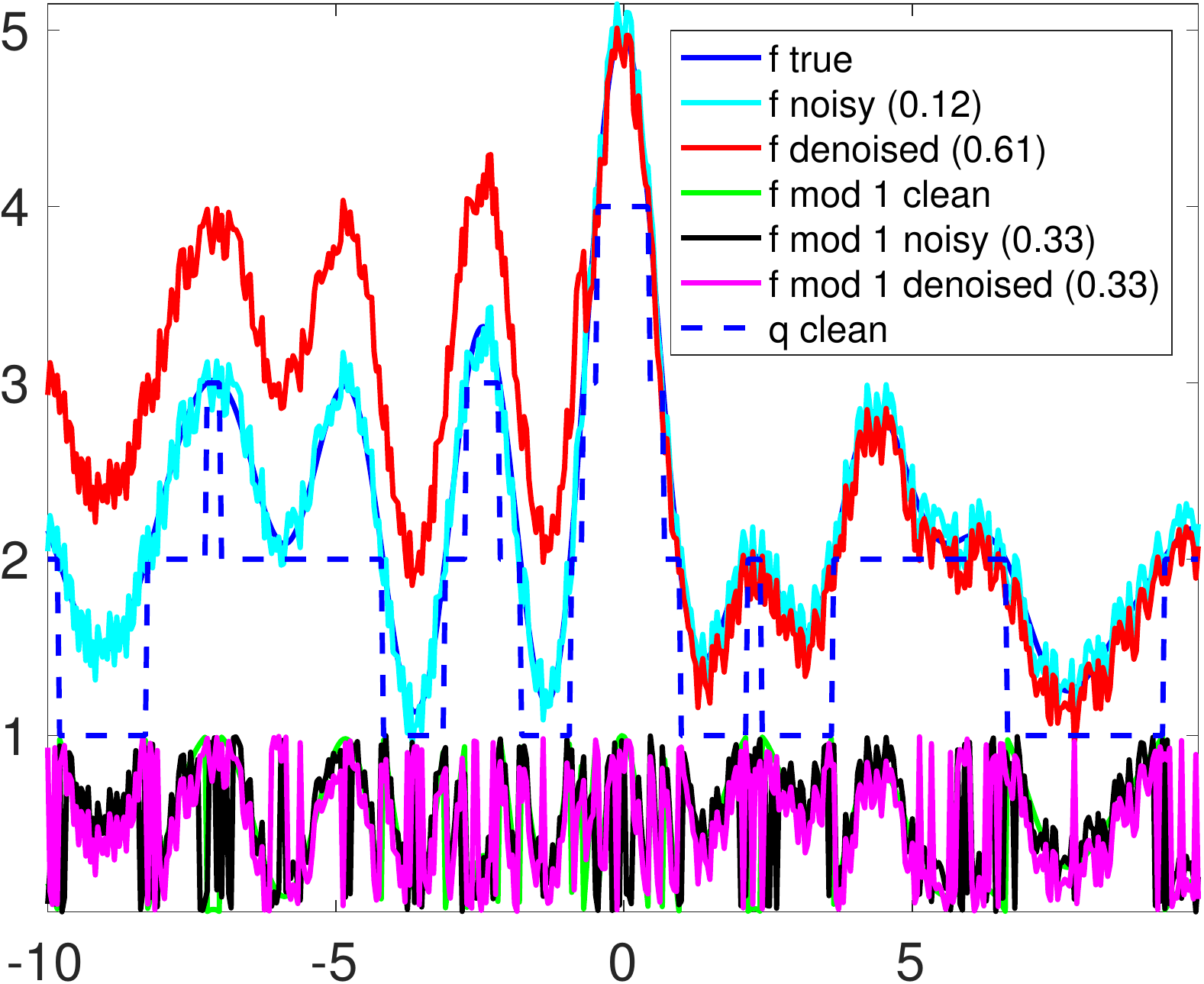} }
%
%
\vspace{-2mm}
\captionsetup{width=0.98\linewidth}
\caption[Short Caption]{Denoised instances for the bandlimited function considered  in  \cite{bhandari17}, under the Bounded-Uniform noise model, for  \textbf{BKR},  \textbf{OLS}, \textbf{QCQP} and \textbf{iQCQP},  as we increase the noise level $\gamma$. We keep fixed the parameters $n=484$, $k=2$, $\lambda= 0.01$. The numerical values in the legend denote the RMSE.  
\textbf{QCQP} denotes Algorithm \ref{algo:two_stage_denoise}, for which  the unwrapping stage is  performed via \textbf{OLS} \eqref{eq:ols_unwrap_lin_system}. 
}
\label{fig:instances_f_BL_Bounded}
\end{figure}

\begin{figure}[!ht]
\centering
\subcaptionbox[]{  $\sigma=0.05$, \textbf{BKR}
}[ 0.24\textwidth ]
{\includegraphics[width=0.24\textwidth] {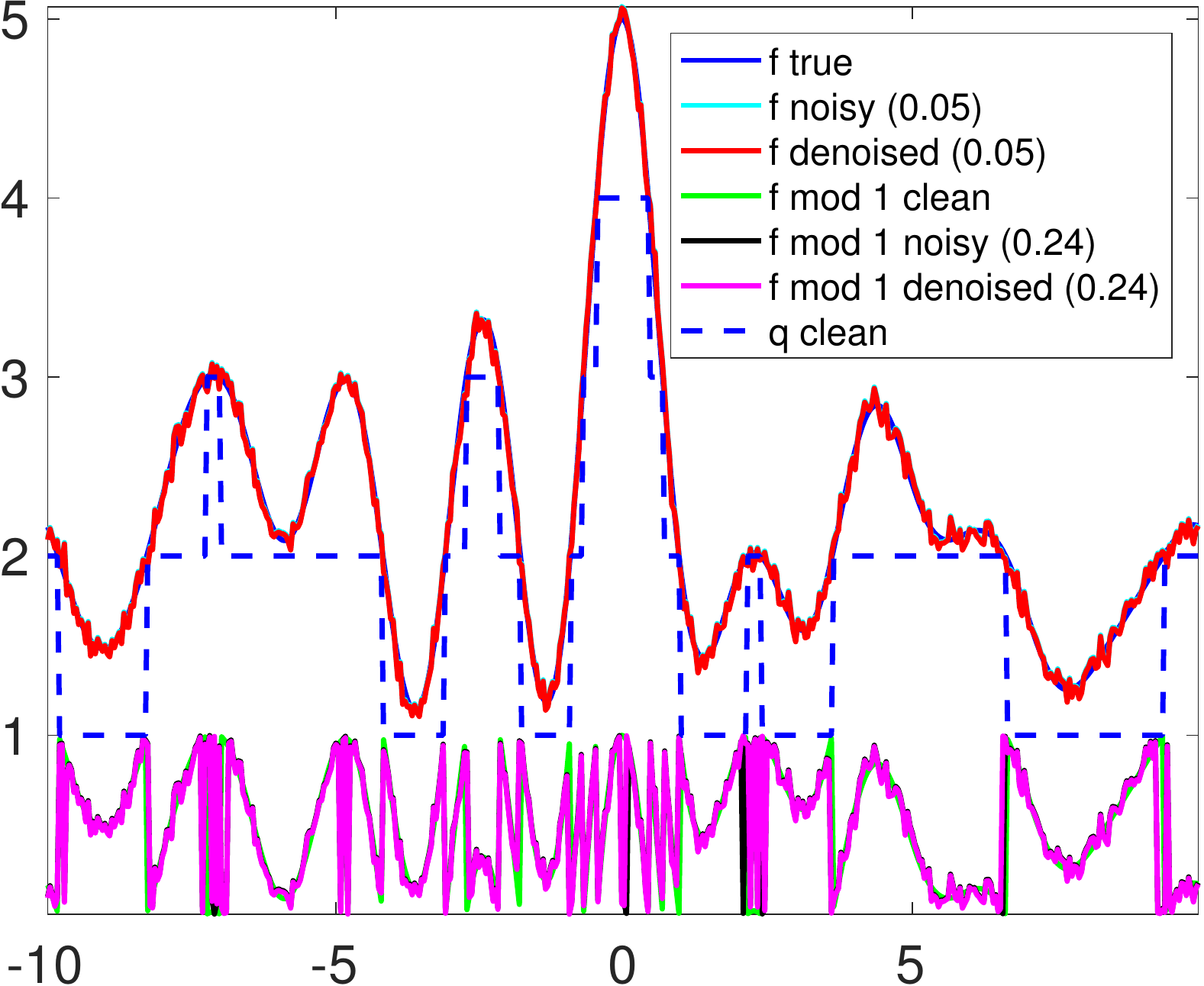} }
%
\subcaptionbox[]{  $\sigma=0.05$, \textbf{OLS}
}[ 0.24\textwidth ]
{\includegraphics[width=0.24\textwidth] {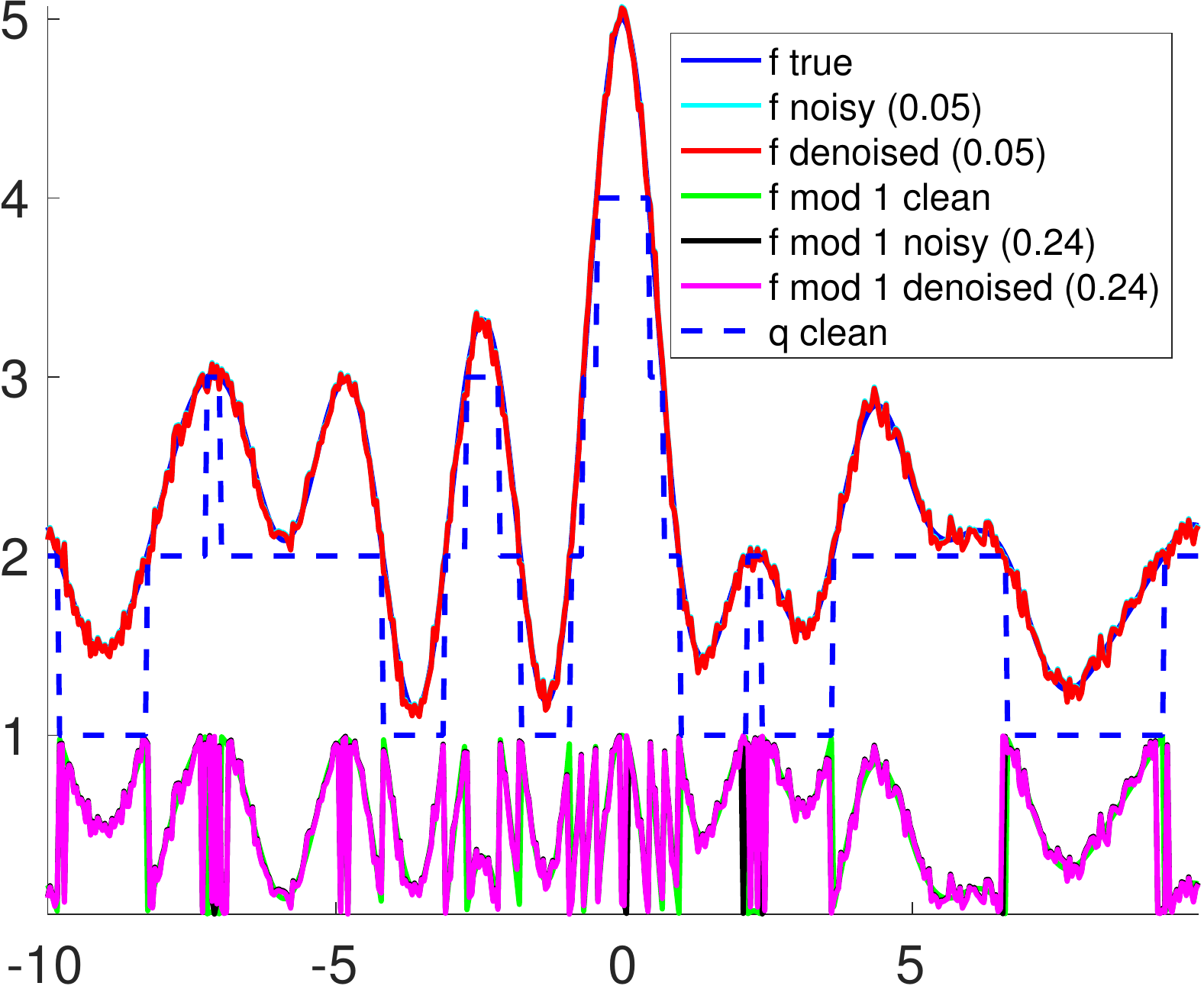} }
%
\subcaptionbox[]{  $\sigma=0.05$, \textbf{QCQP}
}[ 0.24\textwidth ]
{\includegraphics[width=0.24\textwidth] {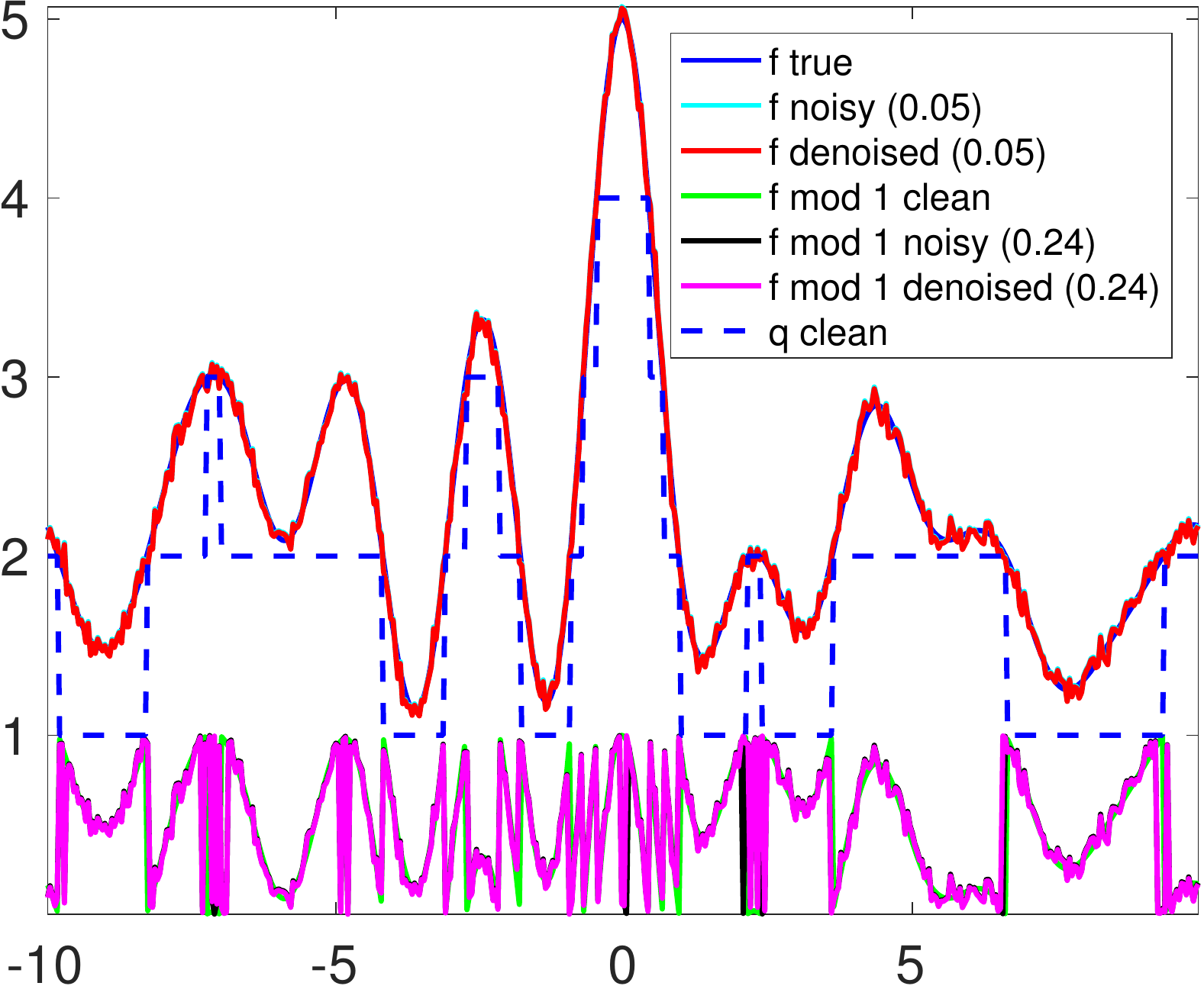} }
%
\subcaptionbox[]{  $\sigma=0.05$, \textbf{iQCQP}
}[ 0.24\textwidth ]
{\includegraphics[width=0.24\textwidth] {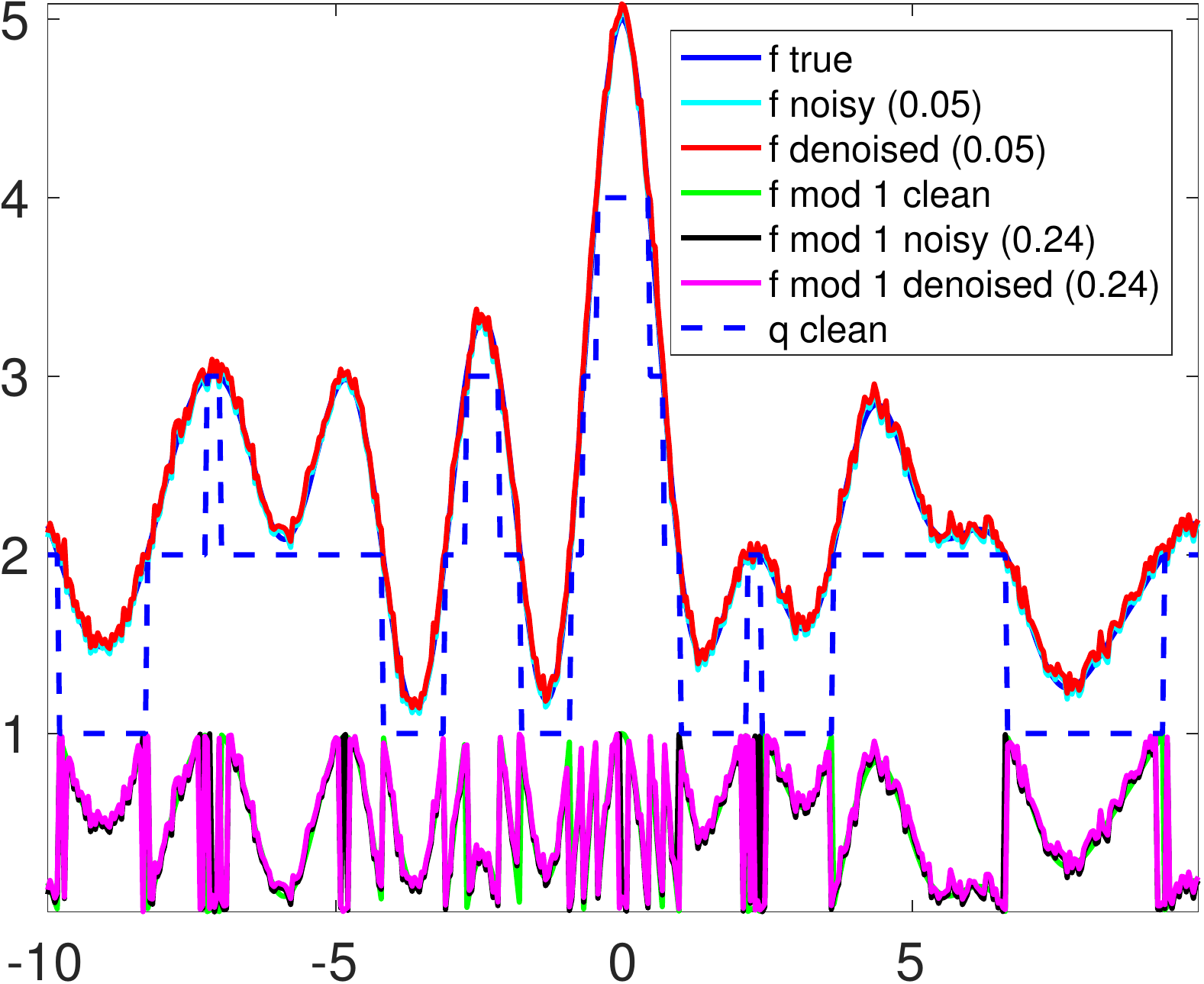} }
%
%
%
\subcaptionbox[]{  $\sigma=0.08$, \textbf{BKR}
}[ 0.24\textwidth ]
{\includegraphics[width=0.24\textwidth] {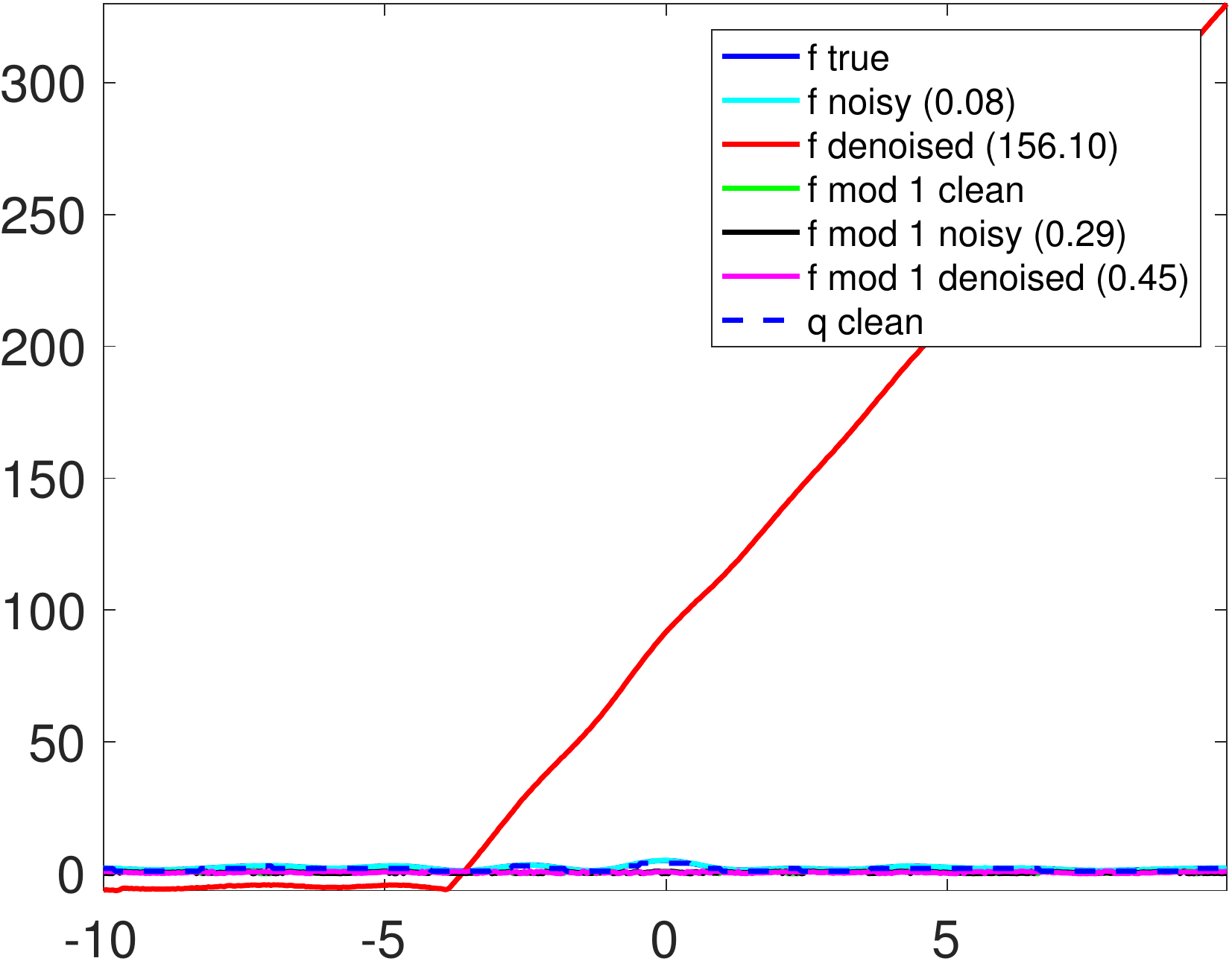} }
%
\subcaptionbox[]{  $\sigma=0.08$, \textbf{OLS}
}[ 0.24\textwidth ]
{\includegraphics[width=0.24\textwidth] {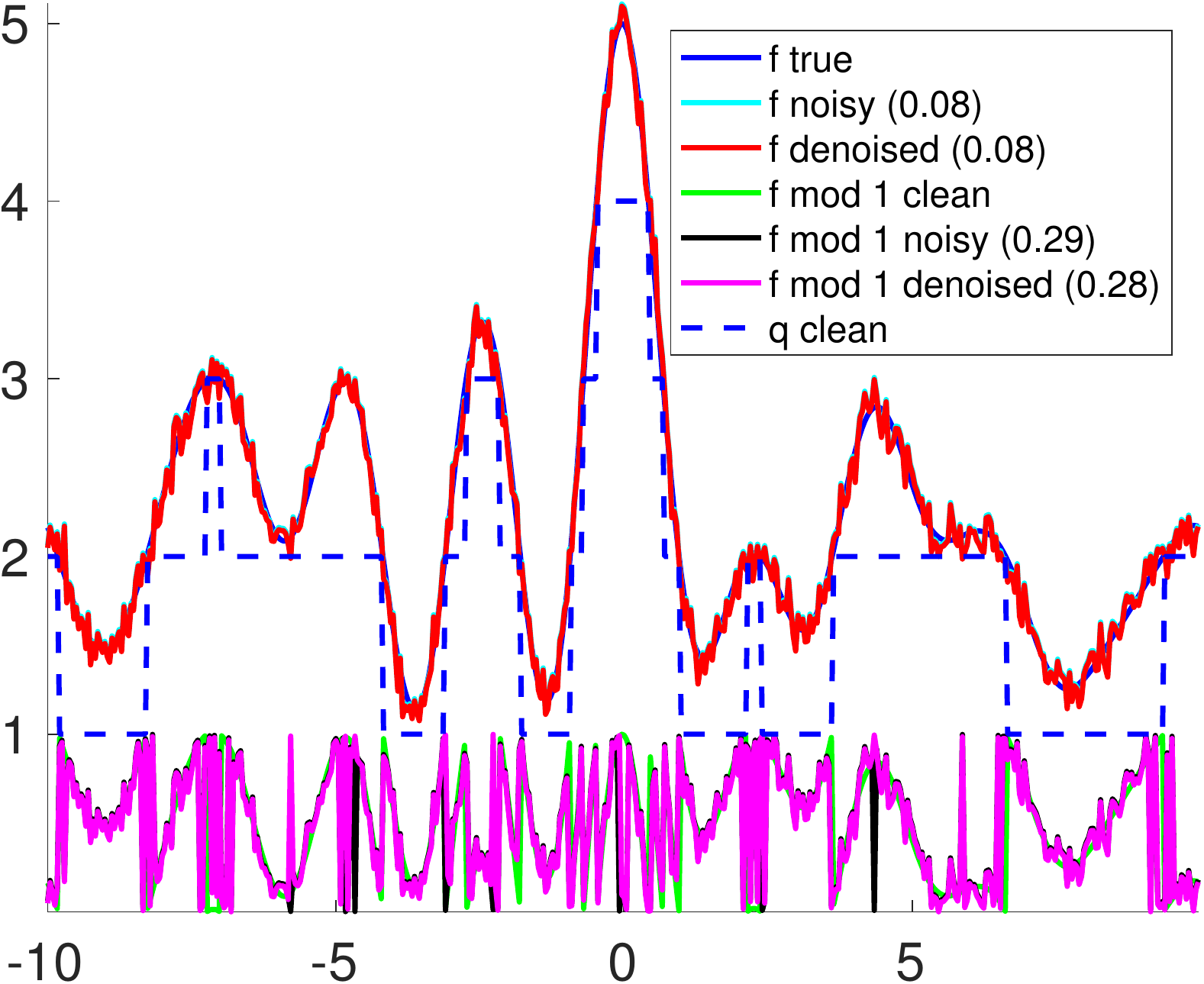} }
%
\subcaptionbox[]{  $\sigma=0.08$, \textbf{QCQP}
}[ 0.24\textwidth ]
{\includegraphics[width=0.24\textwidth] {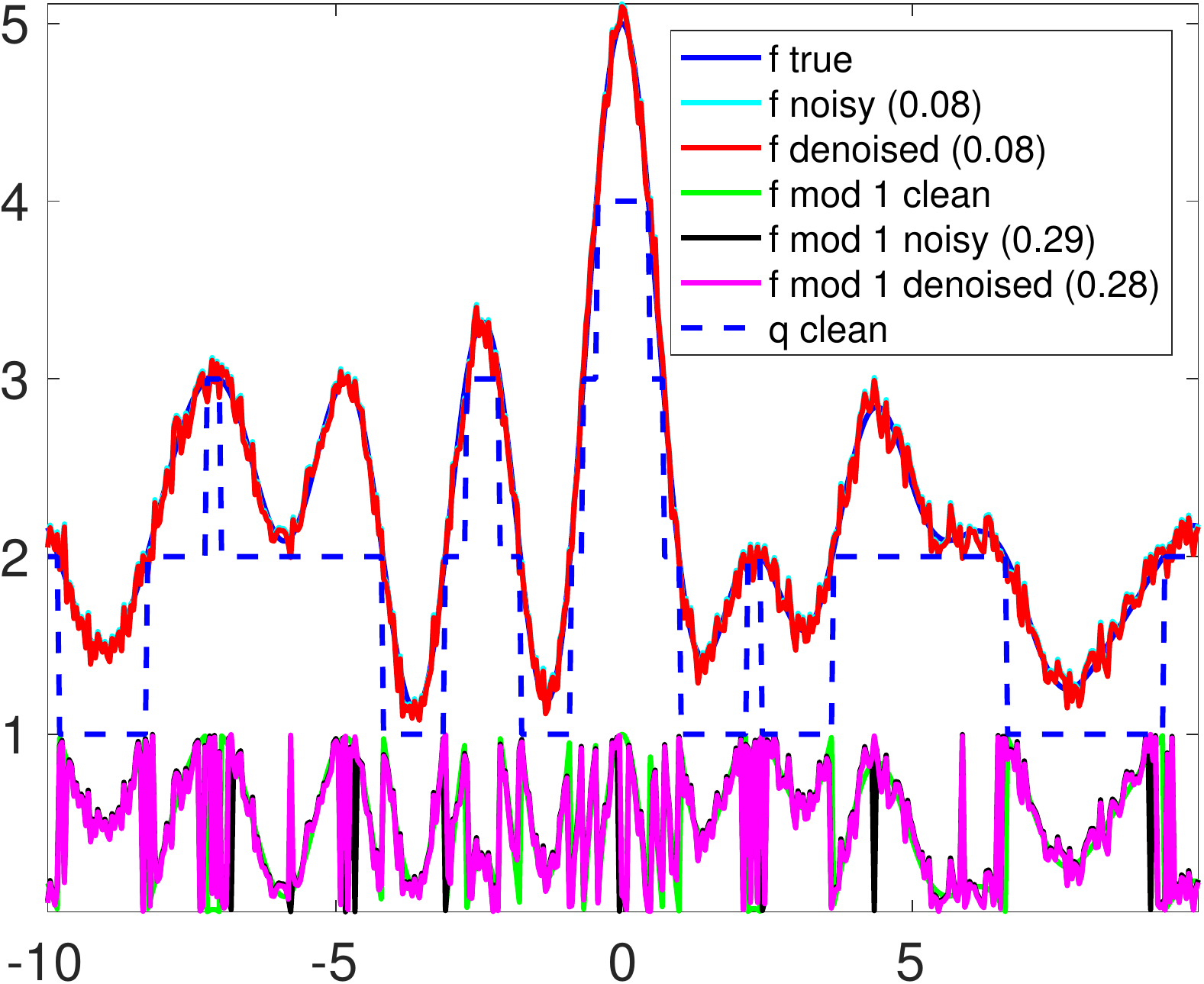} }
%
\subcaptionbox[]{  $\sigma=0.08$, \textbf{iQCQP}
}[ 0.24\textwidth ]
{\includegraphics[width=0.24\textwidth] {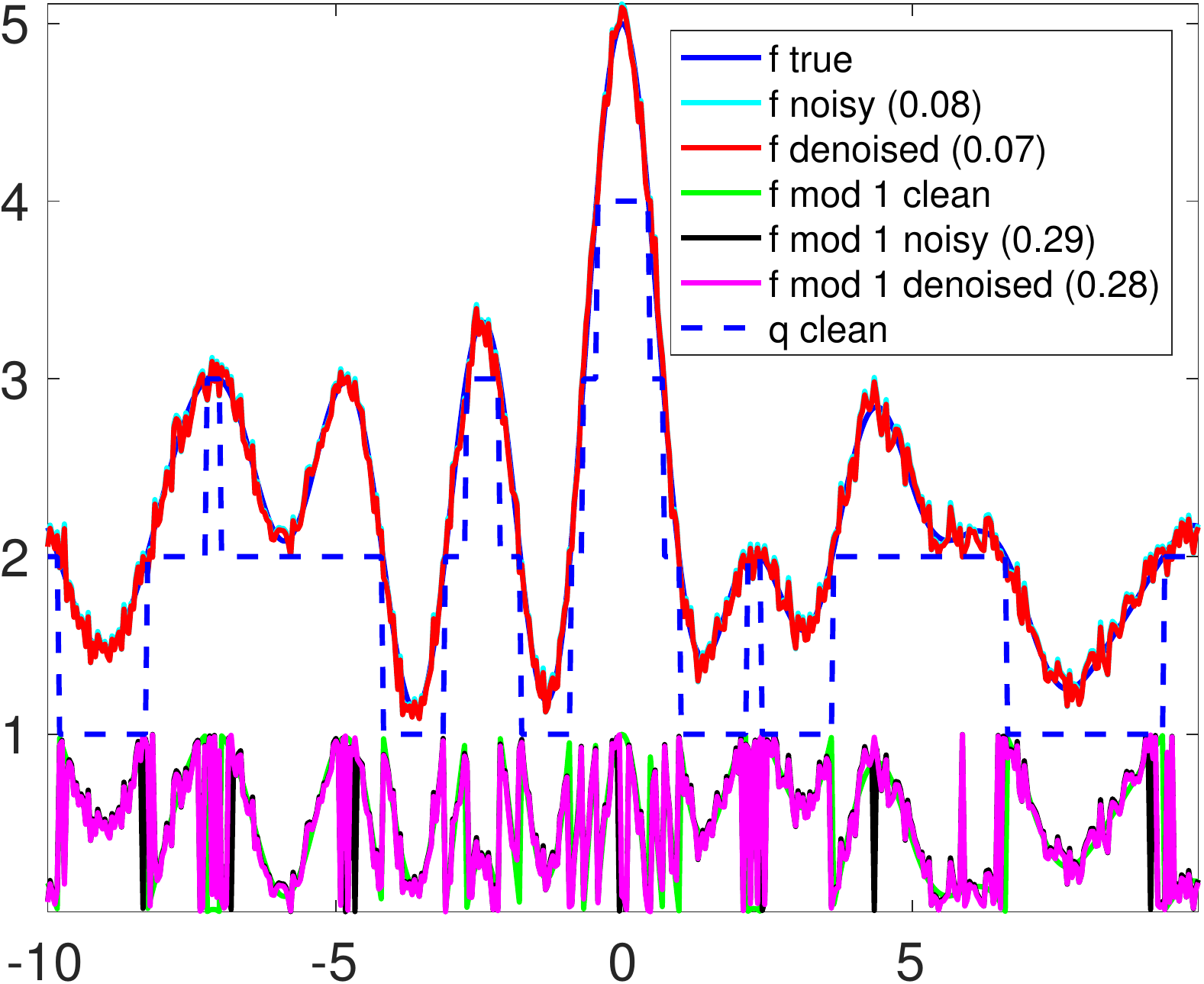} }
%
%
%
%
%
%
\subcaptionbox[]{  $\sigma=0.12$, \textbf{BKR}
}[ 0.24\textwidth ]
{\includegraphics[width=0.24\textwidth] {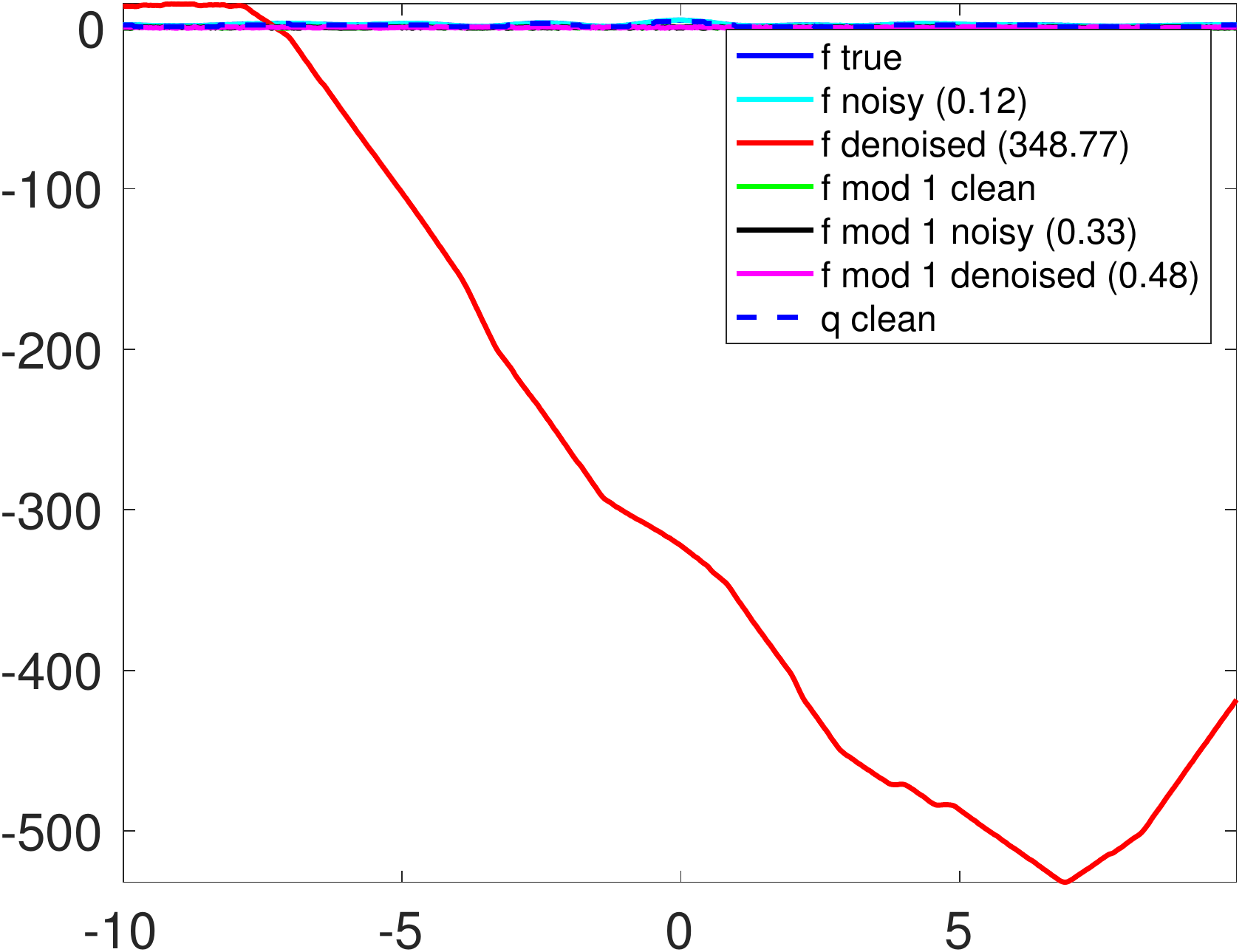} }
%
\subcaptionbox[]{  $\sigma=0.12$, \textbf{OLS}
}[ 0.24\textwidth ]
{\includegraphics[width=0.24\textwidth] {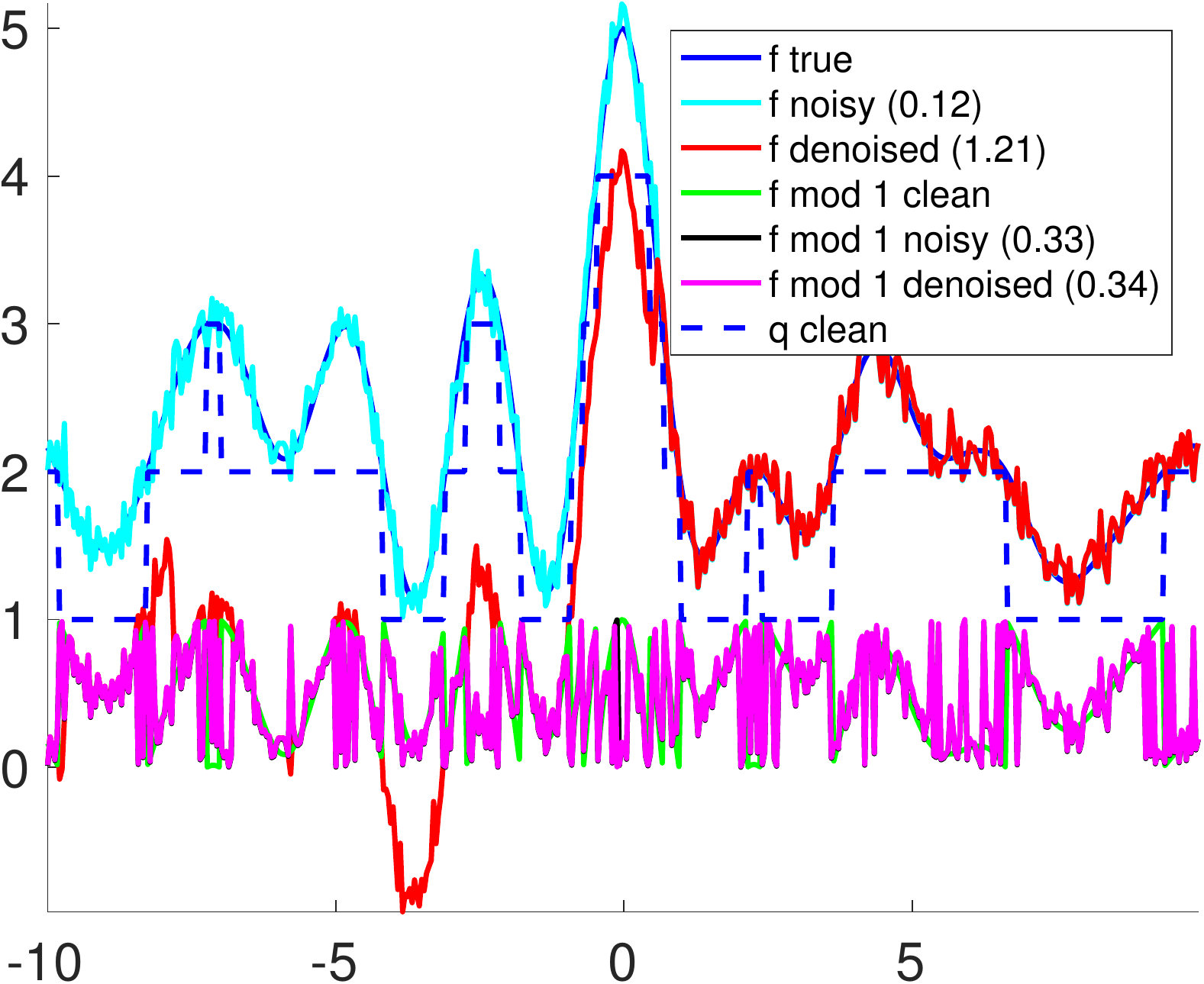} }
%
\subcaptionbox[]{  $\sigma=0.12$, \textbf{QCQP}
}[ 0.24\textwidth ]
{\includegraphics[width=0.24\textwidth] {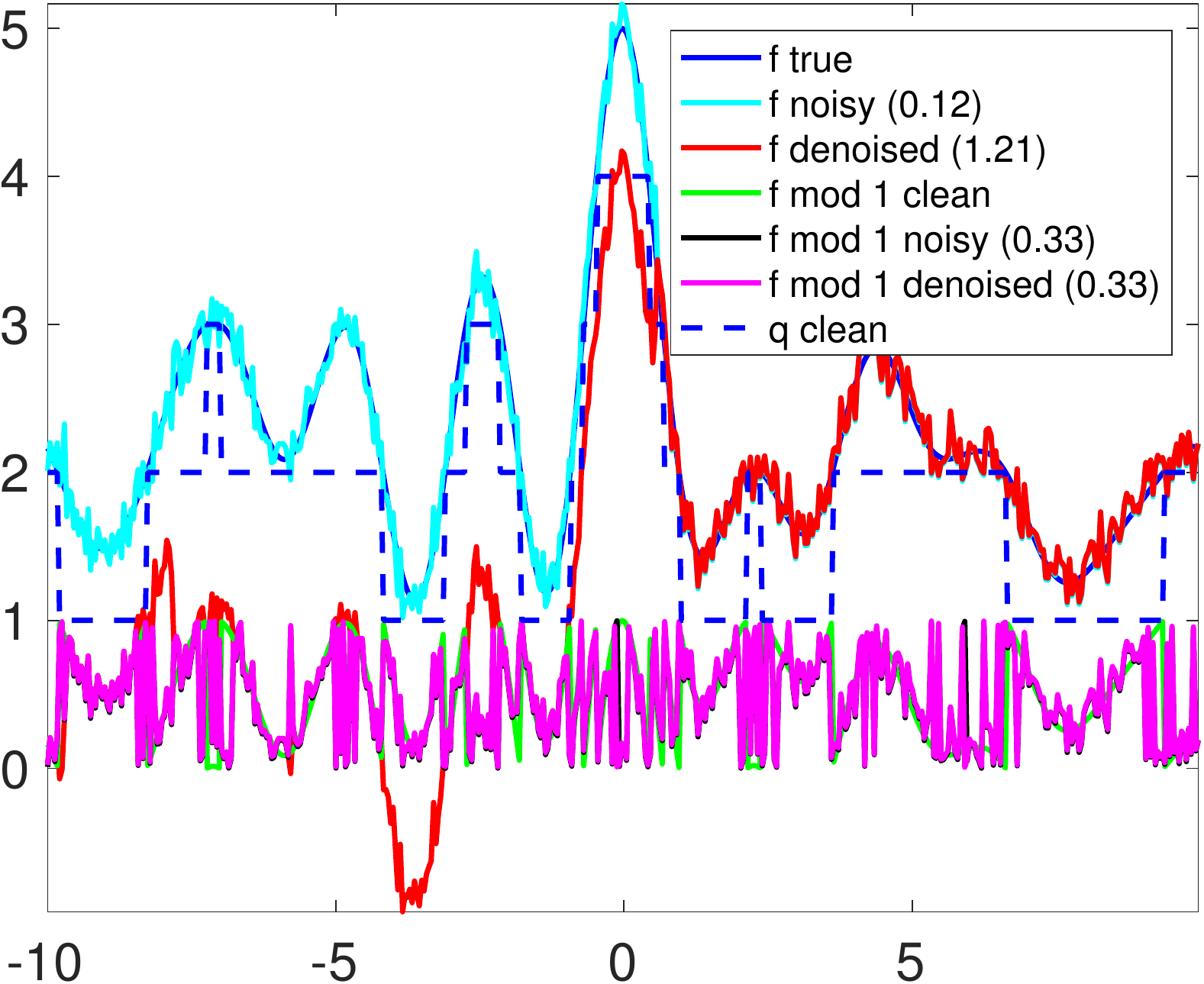} }
%
\subcaptionbox[]{  $\sigma=0.12$, \textbf{iQCQP}
}[ 0.24\textwidth ]
{\includegraphics[width=0.24\textwidth] {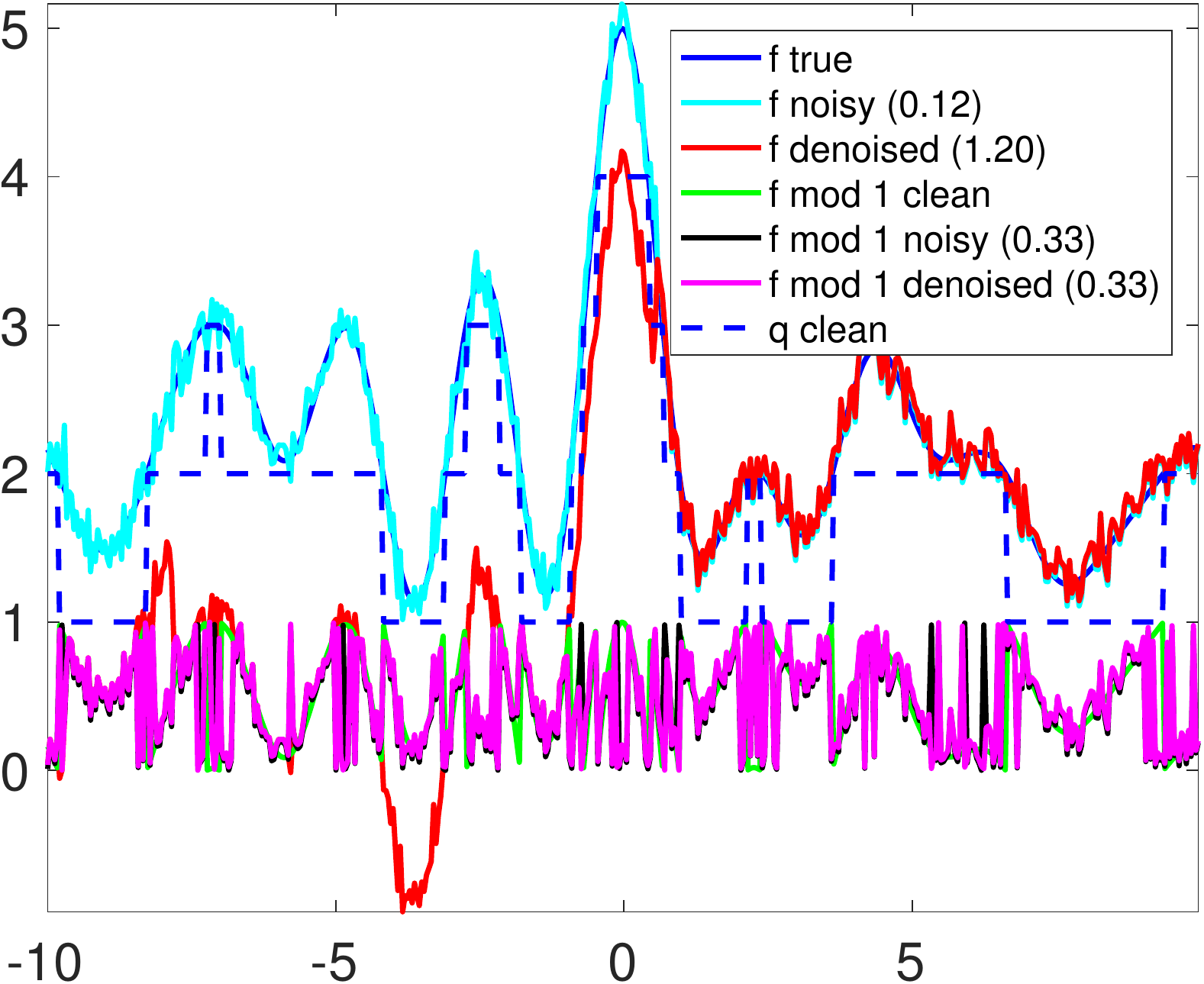} }
%
%
\vspace{-2mm}
\captionsetup{width=0.98\linewidth}
\caption[Short Caption]{Denoised instances for the bandlimited function considered  in  \cite{bhandari17}, under the Gaussian noise model, for  \textbf{BKR},  \textbf{OLS}, \textbf{QCQP} and \textbf{iQCQP},  as we increase the noise level $\sigma$.  \textbf{QCQP} denotes Algorithm \ref{algo:two_stage_denoise}, for which  the unwrapping stage is  performed via \textbf{OLS} \eqref{eq:ols_unwrap_lin_system}.   We keep fixed the parameters $n=484$, $k=2$, $\lambda= 0.01$. The numerical values in the legend denote the RMSE.
}
\label{fig:instances_f_BL_Gaussian}
\end{figure}

\begin{figure}[!ht]
\centering
\subcaptionbox[]{  $\gamma=0.10$, \textbf{BKR}
}[ 0.24\textwidth ]
{\includegraphics[width=0.24\textwidth] {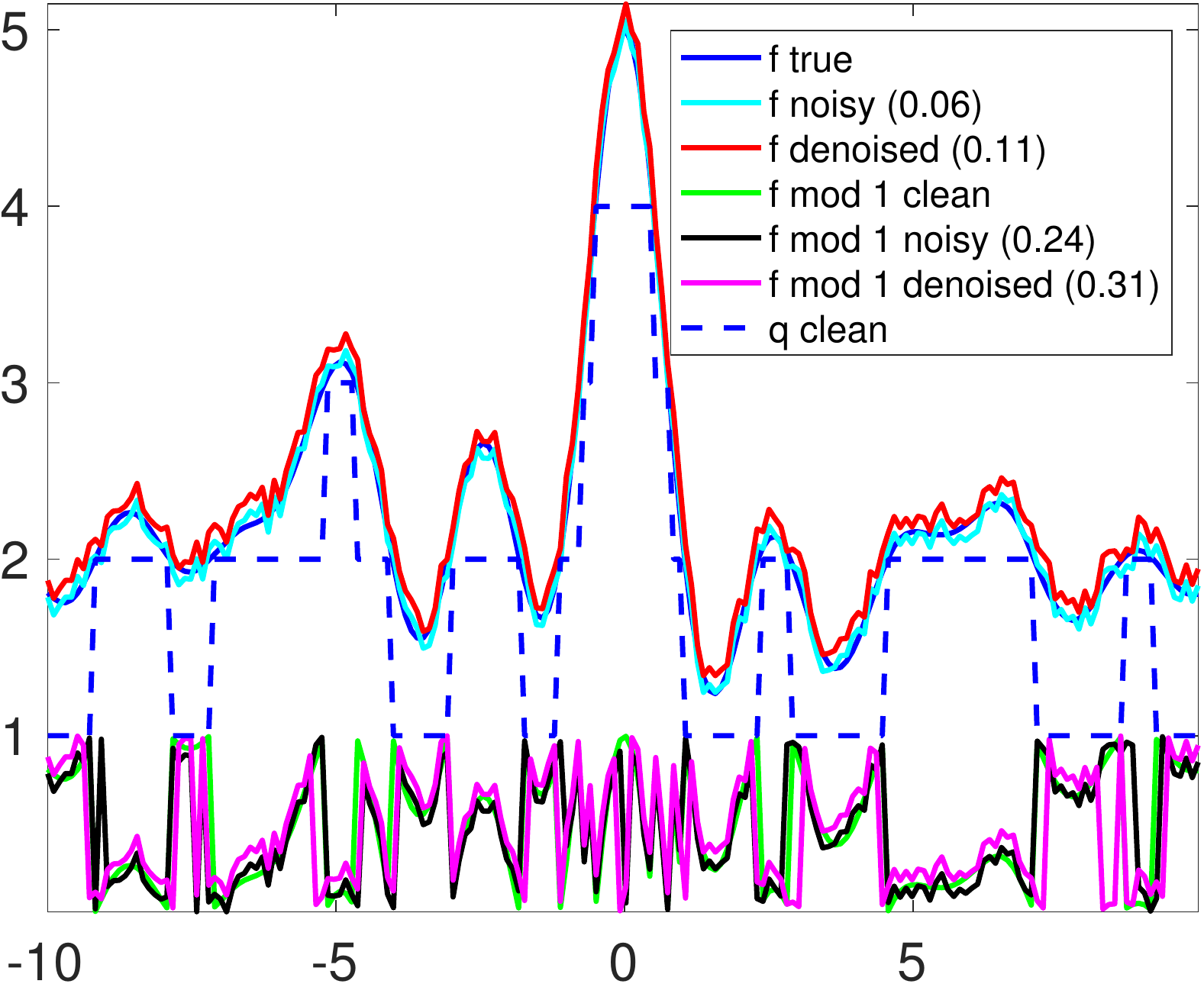} }
%
\subcaptionbox[]{  $\gamma=0.10$, \textbf{OLS}
}[ 0.24\textwidth ]
{\includegraphics[width=0.24\textwidth] {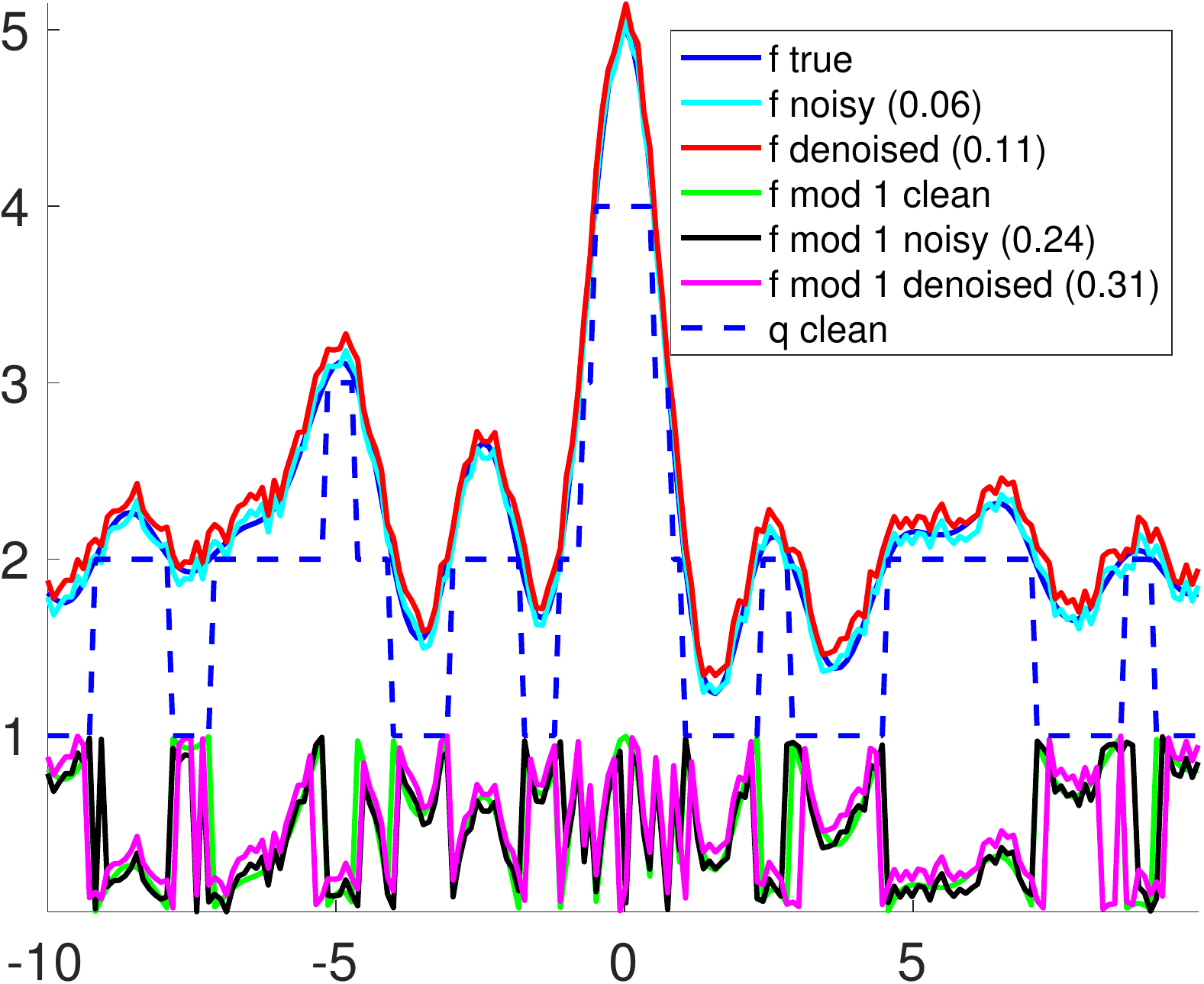} }
%
\subcaptionbox[]{  $\gamma=0.10$, \textbf{QCQP}
}[ 0.24\textwidth ]
{\includegraphics[width=0.24\textwidth] {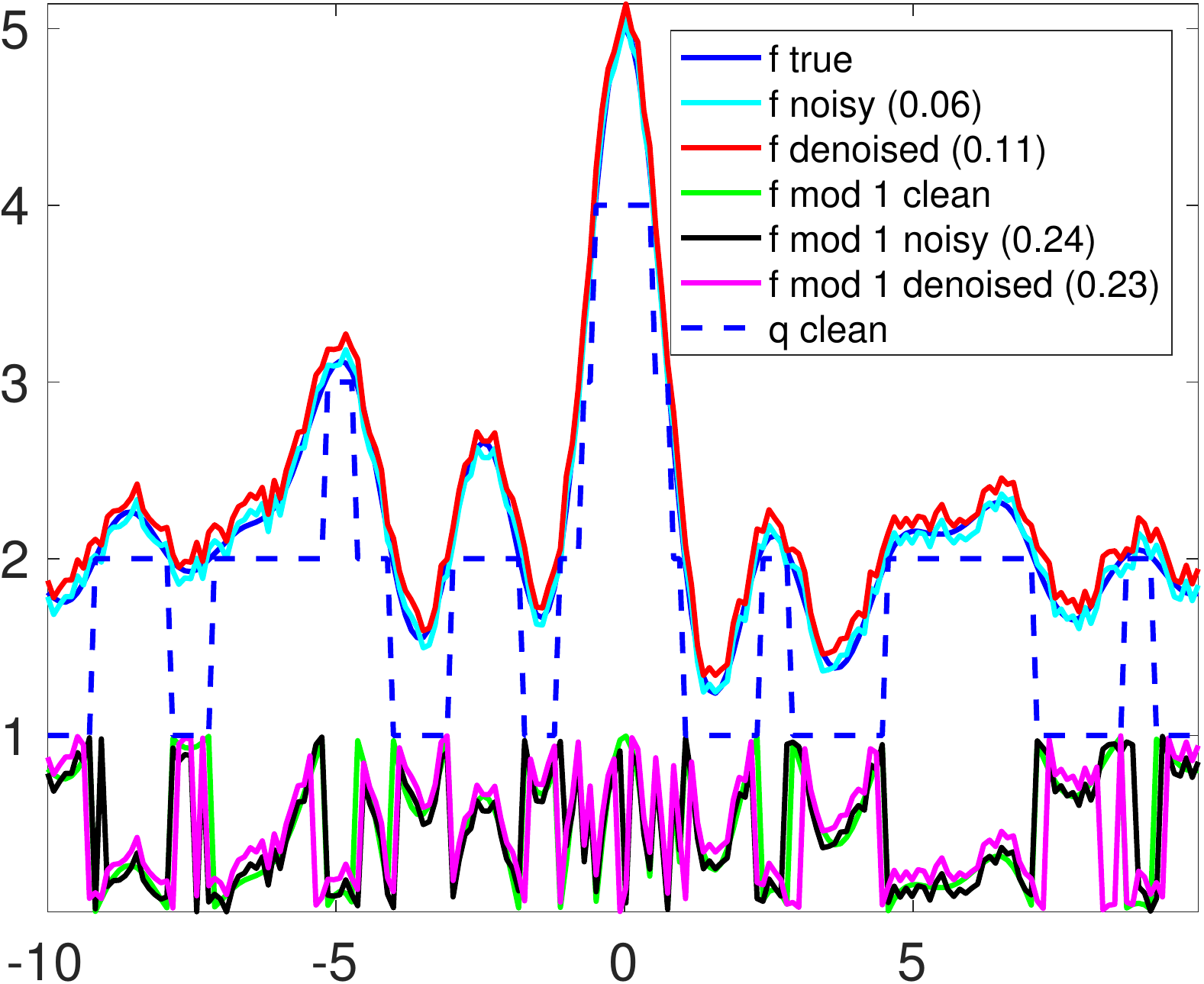} }
%
\subcaptionbox[]{  $\gamma=0.10$, \textbf{iQCQP}
}[ 0.24\textwidth ]
{\includegraphics[width=0.24\textwidth] {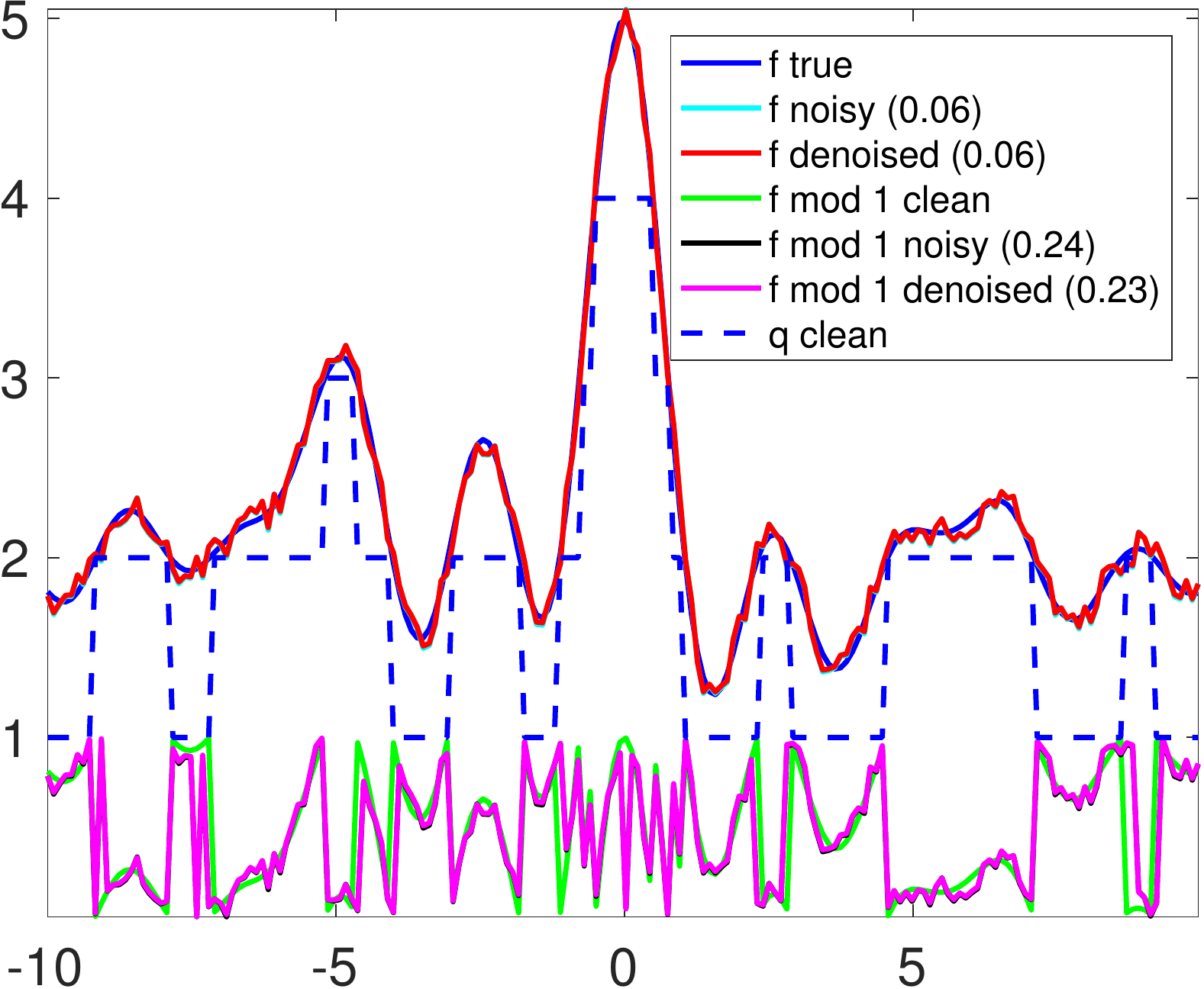} }
%
%
%
\subcaptionbox[]{  $\gamma=0.15$, \textbf{BKR}
}[ 0.24\textwidth ]
{\includegraphics[width=0.24\textwidth] {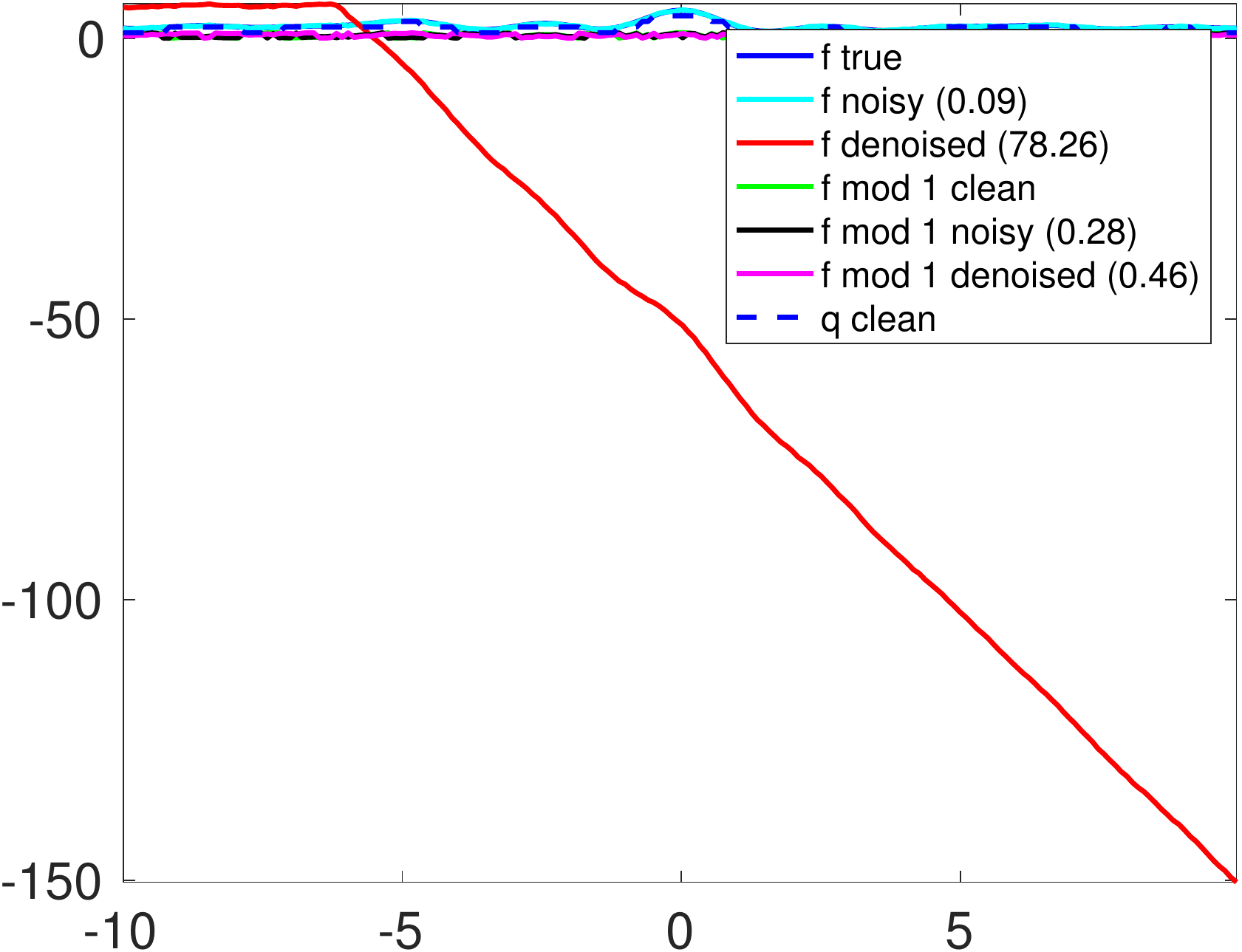} }
%
\subcaptionbox[]{  $\gamma=0.15$, \textbf{OLS}
}[ 0.24\textwidth ]
{\includegraphics[width=0.24\textwidth] {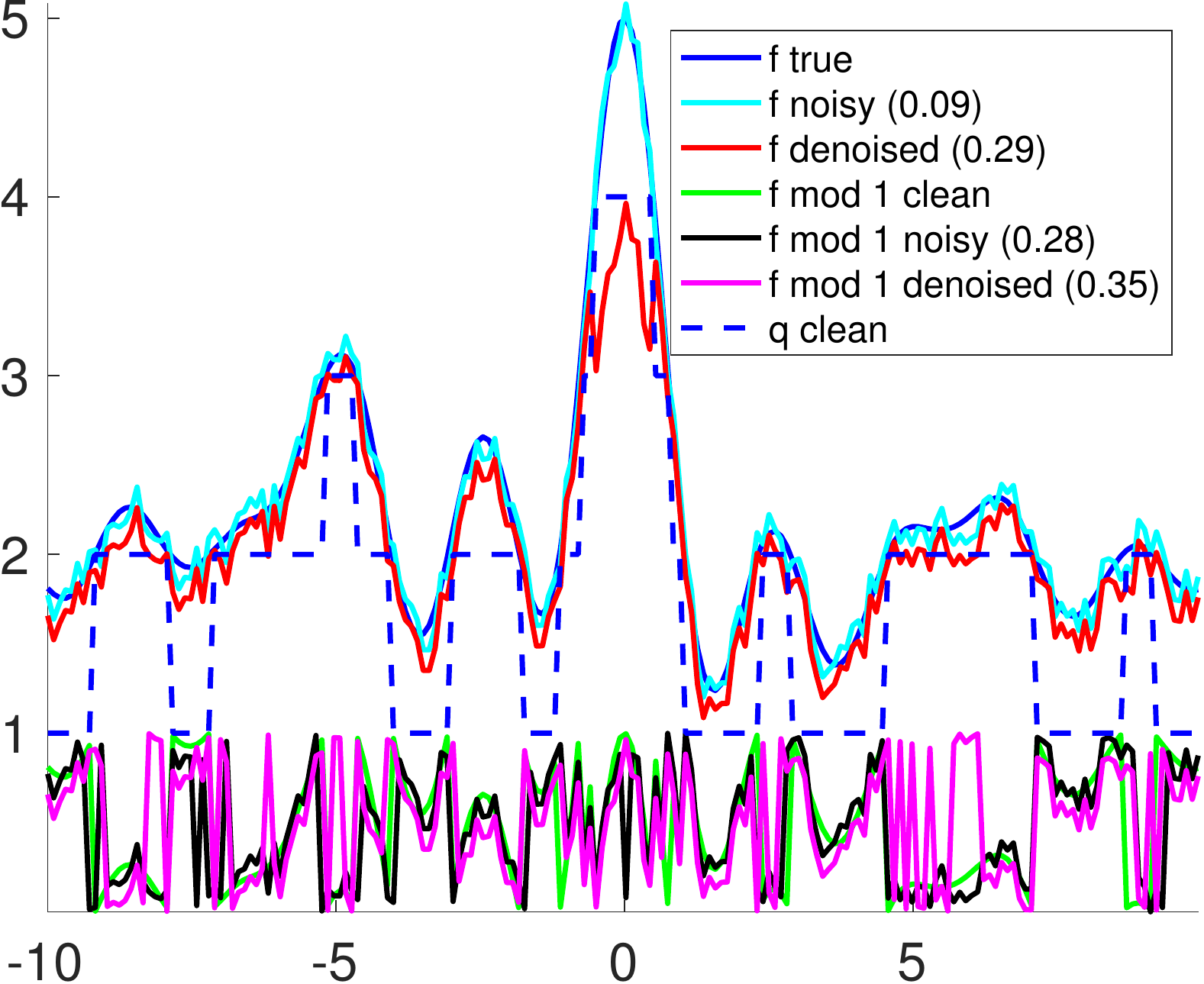} }
%
\subcaptionbox[]{  $\gamma=0.15$, \textbf{QCQP}
}[ 0.24\textwidth ]
{\includegraphics[width=0.24\textwidth] {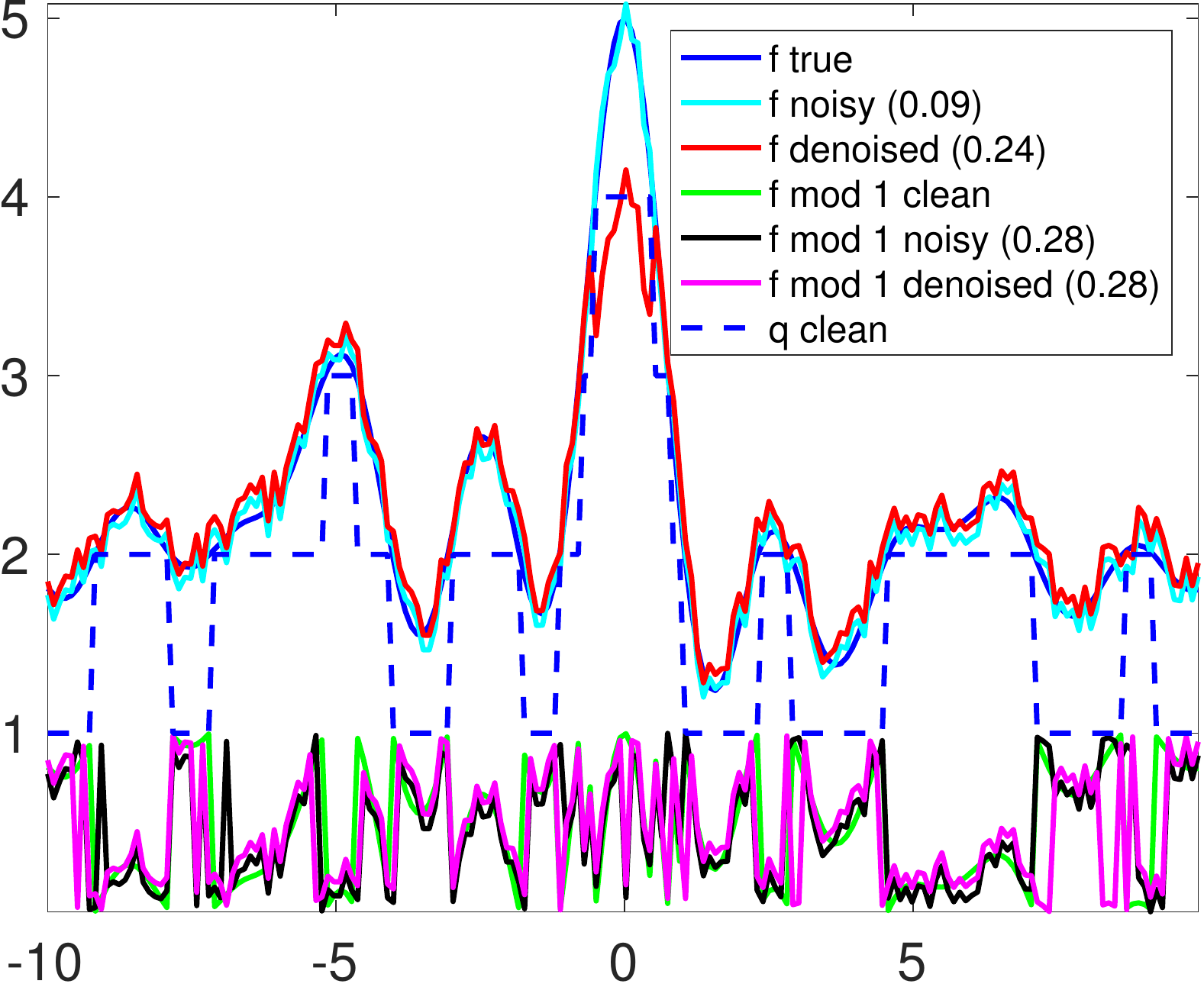} }
%
\subcaptionbox[]{  $\gamma=0.15$, \textbf{iQCQP}
}[ 0.24\textwidth ]
{\includegraphics[width=0.24\textwidth] {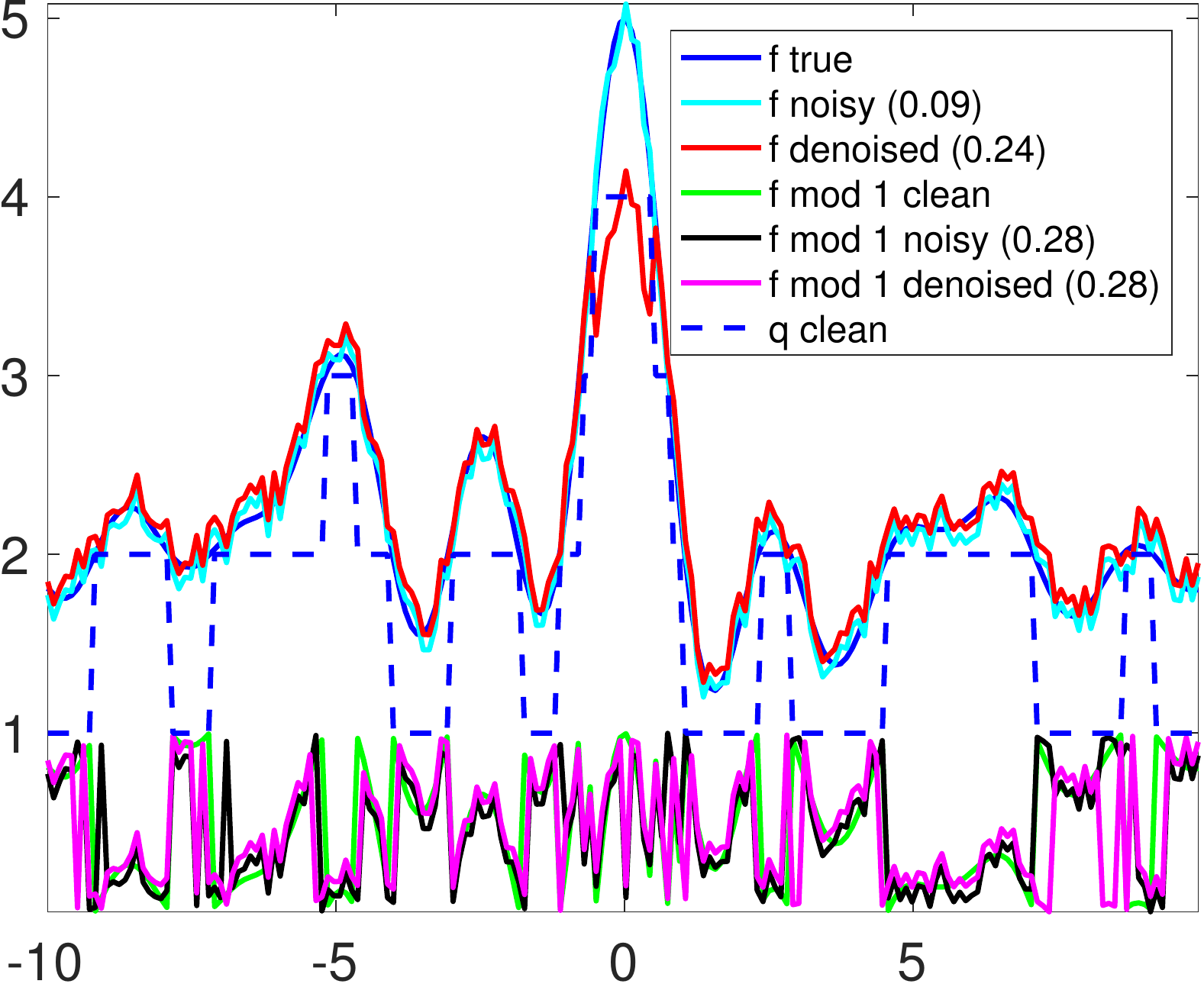} }
%
%
%
%
%
%
\subcaptionbox[]{  $\gamma=0.20$, \textbf{BKR}
}[ 0.24\textwidth ]
{\includegraphics[width=0.24\textwidth] {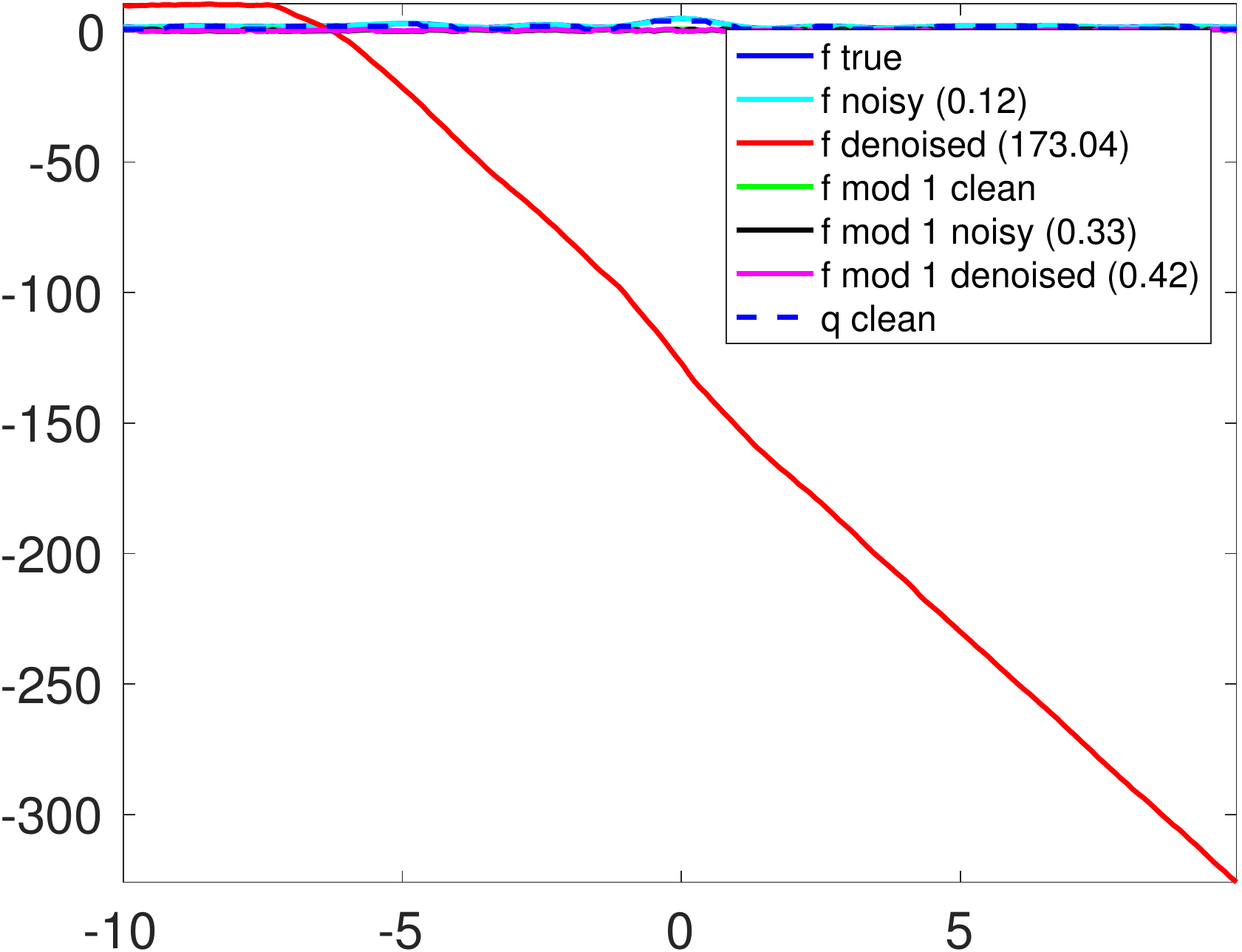} }
%
\subcaptionbox[]{  $\gamma=0.20$, \textbf{OLS}
}[ 0.24\textwidth ]
{\includegraphics[width=0.24\textwidth] {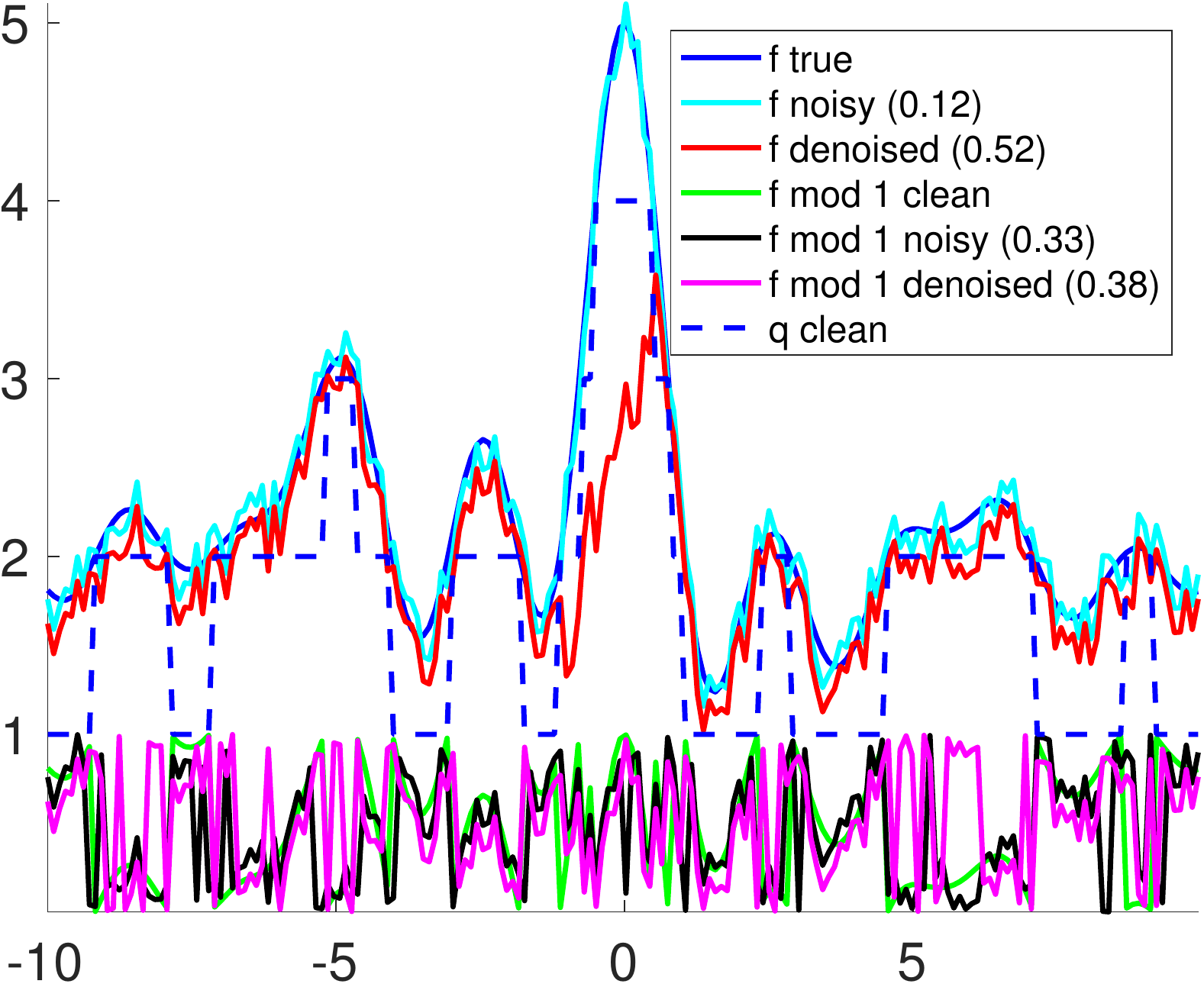} }
%
\subcaptionbox[]{  $\gamma=0.20$, \textbf{QCQP}
}[ 0.24\textwidth ]
{\includegraphics[width=0.24\textwidth] {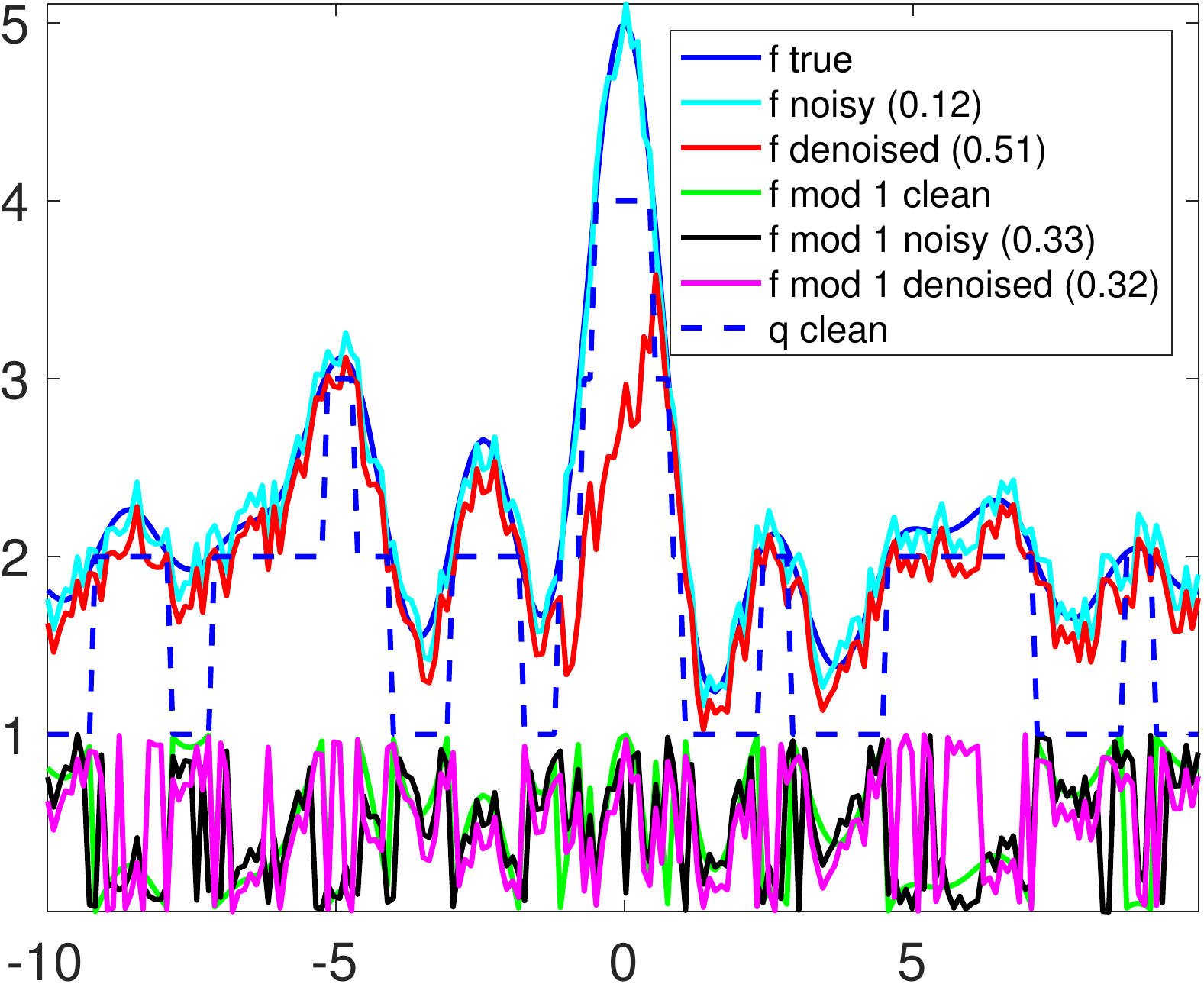} }
%
\subcaptionbox[]{  $\gamma=0.20$, \textbf{iQCQP}
}[ 0.24\textwidth ]
{\includegraphics[width=0.24\textwidth] {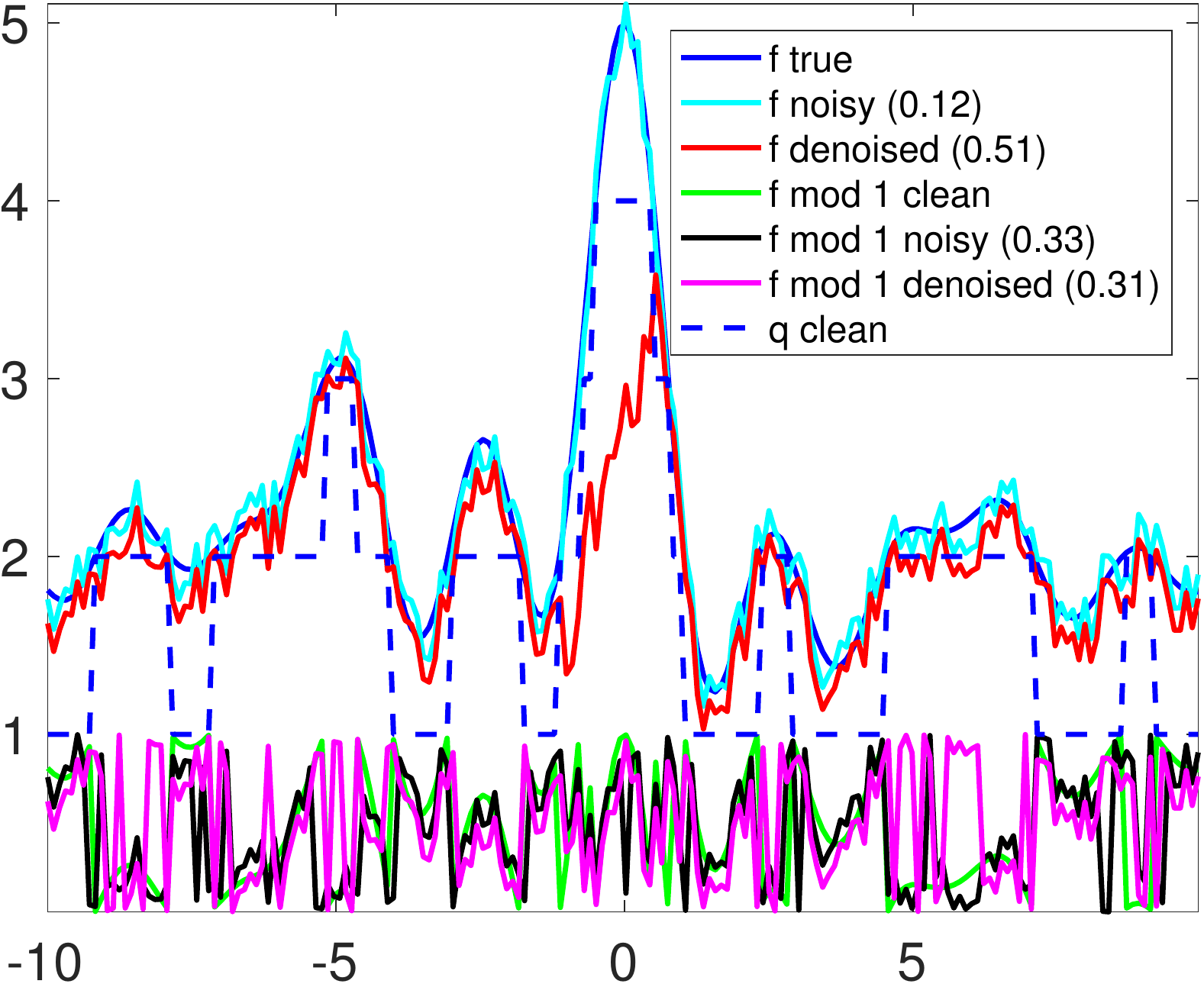} }
%
%
\vspace{-2mm}
\captionsetup{width=0.98\linewidth}
\caption[Short Caption]{Denoised instances for the bandlimited function considered  in  \cite{bhandari17}, under the Bounded-Uniform noise model, for  \textbf{BKR},  \textbf{OLS}, \textbf{QCQP} and \textbf{iQCQP},  as we increase the noise level $\gamma$.   \textbf{QCQP} denotes Algorithm \ref{algo:two_stage_denoise}, for which  the unwrapping stage is  performed via \textbf{OLS} \eqref{eq:ols_unwrap_lin_system}.     The function is sampled less frequently compared to the earlier Figure  \ref{fig:instances_f_BL_Bounded}, which in this experiment leads to  only $n=194$ sample points. We keep fixed  $k=2$ and $\lambda= 0.01$. The numerical values in the legend denote the RMSE. 
}
\label{fig:instances_f_BL_200_Bounded}
\end{figure}

%

\subsection{Additional elevation maps experiments for two-dimensional phase unwrapping} \label{sec:apx_MountainExtras} 


%
Figure \ref{fig:Vesuvius_LOW_Noisy} pertains to the elevation map of  Vesuvius, using  $n = 3600$ samples, under the Gaussian noise model ($\sigma=0.05$). 
Figure  \ref{fig:ETNA_HIGH_Noisy}, pertains to noisy measurements of Mount Etna under the Gaussian noise model ($\sigma=0.10$), 
using $n = 19000$ samples. 

\ifthenelse{\boolean{ShowNoiselessMultivar}}{  
\begin{figure*}
\centering
  
\subcaptionbox[]{  Noiseless function
}[ 0.32\textwidth ]
{\includegraphics[width=0.32\textwidth] {figs/PLOTS_MULTIVARIATE/Vesuvius_LOW/Gaussian_sigma_0_n_3599_kDist_1_lambda_0p03_scale_0_fClean_3d.png} }
%
\subcaptionbox[]{  Noiseless function (depth)
}[ 0.32\textwidth ]
{\includegraphics[width=0.32\textwidth] {figs/PLOTS_MULTIVARIATE/Vesuvius_LOW/Gaussian_sigma_0_n_3599_kDist_1_lambda_0p03_scale_0_fClean_TOP.png} }
%
\subcaptionbox[]{   Clean $f$ mod 1
}[ 0.32\textwidth ]
{\includegraphics[width=0.32\textwidth] {figs/PLOTS_MULTIVARIATE/Vesuvius_LOW/Gaussian_sigma_0_n_3599_kDist_1_lambda_0p03_scale_0_f_mod1_clean_TOP.png} }

\subcaptionbox[]{    Noisy $f$ mod 1
}[ 0.32\textwidth ]
{\includegraphics[width=0.32\textwidth] {figs/PLOTS_MULTIVARIATE/Vesuvius_LOW/Gaussian_sigma_0_n_3599_kDist_1_lambda_0p03_scale_0_f_mod1_noise_TOP.png} }
%
\subcaptionbox[]{      Denoised $f$ mod 1
}[ 0.32\textwidth ]
{\includegraphics[width=0.32\textwidth] {figs/PLOTS_MULTIVARIATE/Vesuvius_LOW/Gaussian_sigma_0_n_3599_kDist_1_lambda_0p03_scale_0_f_mod1_denoised_TOP.png} }
%
\subcaptionbox[]{   Denoised $f$
}[ 0.32\textwidth ]
{\includegraphics[width=0.32\textwidth] {figs/PLOTS_MULTIVARIATE/Vesuvius_LOW/Gaussian_sigma_0_n_3599_kDist_1_lambda_0p03_scale_0_f_denoised_3d.png} }
%
%
\subcaptionbox[]{   Denoised $f$ (depth)
}[ 0.32\textwidth ]
{\includegraphics[width=0.32\textwidth] {figs/PLOTS_MULTIVARIATE/Vesuvius_LOW/Gaussian_sigma_0_n_3599_kDist_1_lambda_0p03_scale_0_f_denoised_TOP.png} }
%
%
\subcaptionbox[]{  Error  $|f - \hat{f}|$ 
}[ 0.32\textwidth ]
{\includegraphics[width=0.32\textwidth] {figs/PLOTS_MULTIVARIATE/Vesuvius_LOW/Gaussian_sigma_0_n_3599_kDist_1_lambda_0p03_scale_0_f_delta_TOP.png} }
\captionsetup{width=0.95\linewidth}
\caption[Short Caption]{Vesuvius $n = 3600$, $k=1$ (Chebychev distance), $\lambda = 0.03$,  and noise level $\sigma=0$.}
\label{fig:Vesuvius_LOW_Clean}
\end{figure*}
  }{}

\begin{figure*}
\centering
   
\subcaptionbox[]{  Noiseless function
}[ 0.24\textwidth ]
{\includegraphics[width=0.24\textwidth] {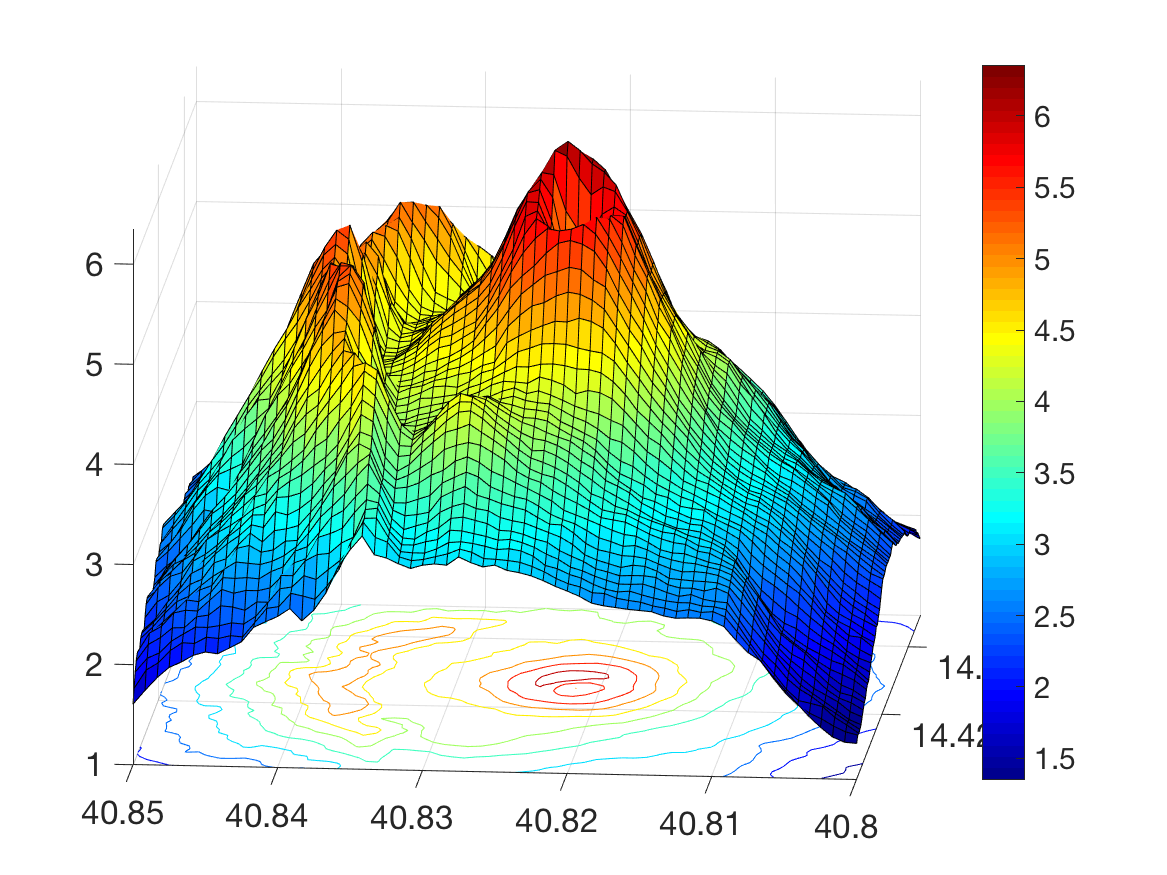} }
%
\subcaptionbox[]{  Noiseless function (depth)
}[ 0.24\textwidth ]
{\includegraphics[width=0.24\textwidth] {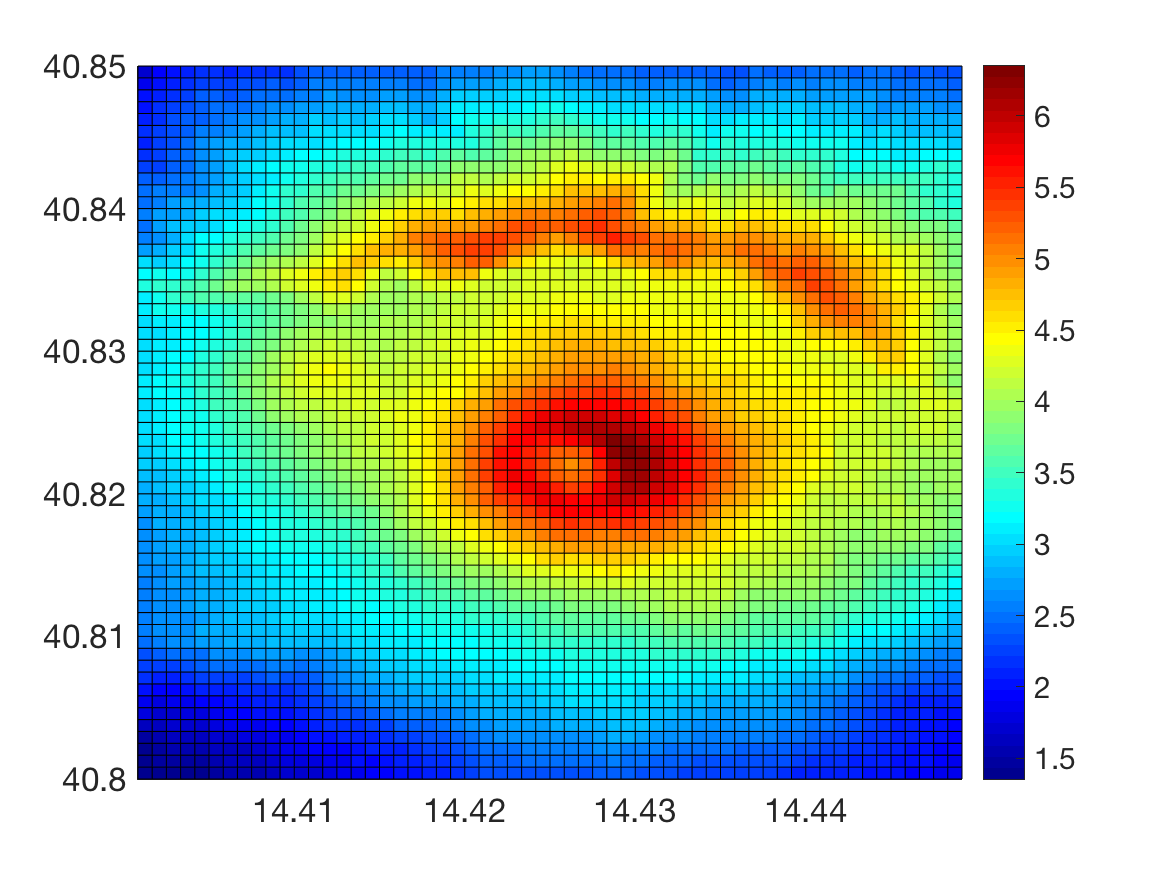} }
%
\subcaptionbox[]{   Clean $f$ mod 1
}[ 0.24\textwidth ]
{\includegraphics[width=0.24\textwidth] {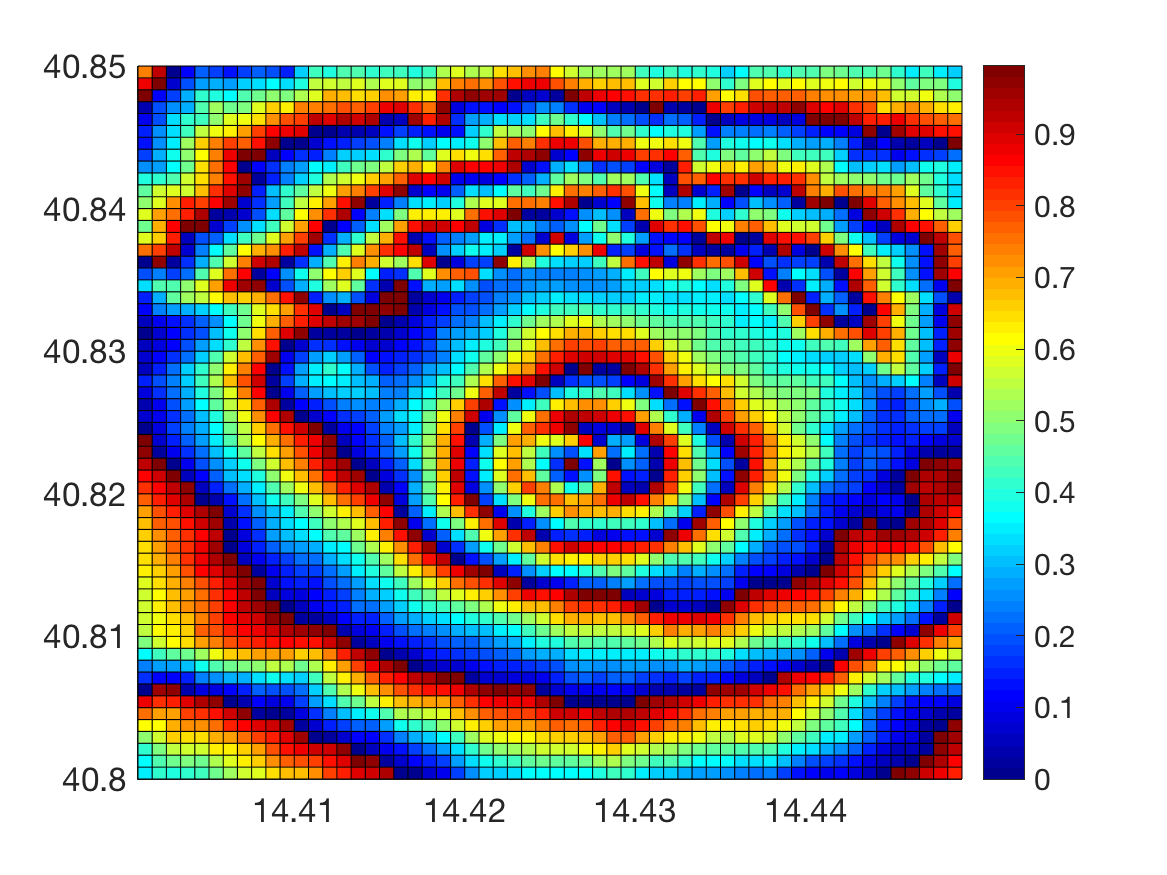} }
\subcaptionbox[]{    Noisy $f$ mod 1
}[ 0.24\textwidth ]
{\includegraphics[width=0.24\textwidth] {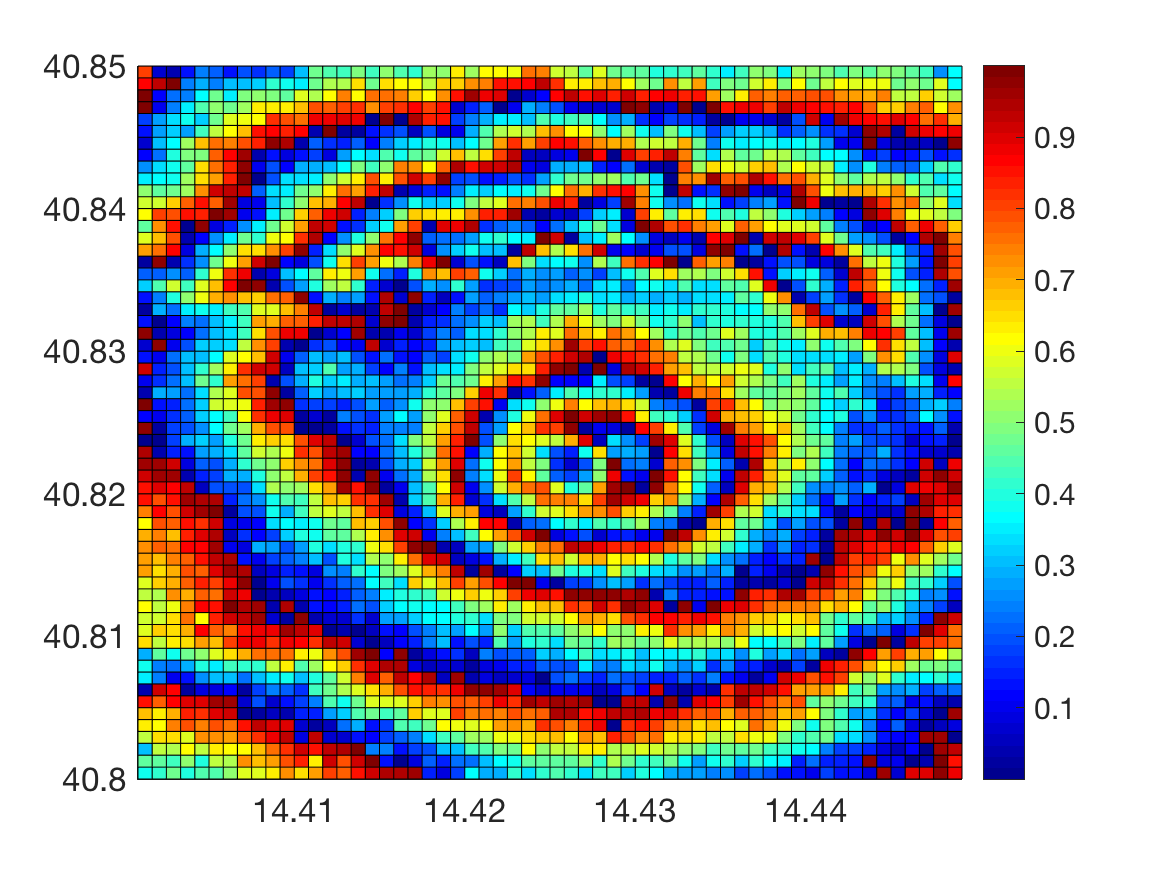} }

\subcaptionbox[]{      Denoised $f$ mod 1 (RMSE = $ 0.173 $) 
}[ 0.24\textwidth ]
{\includegraphics[width=0.24\textwidth] {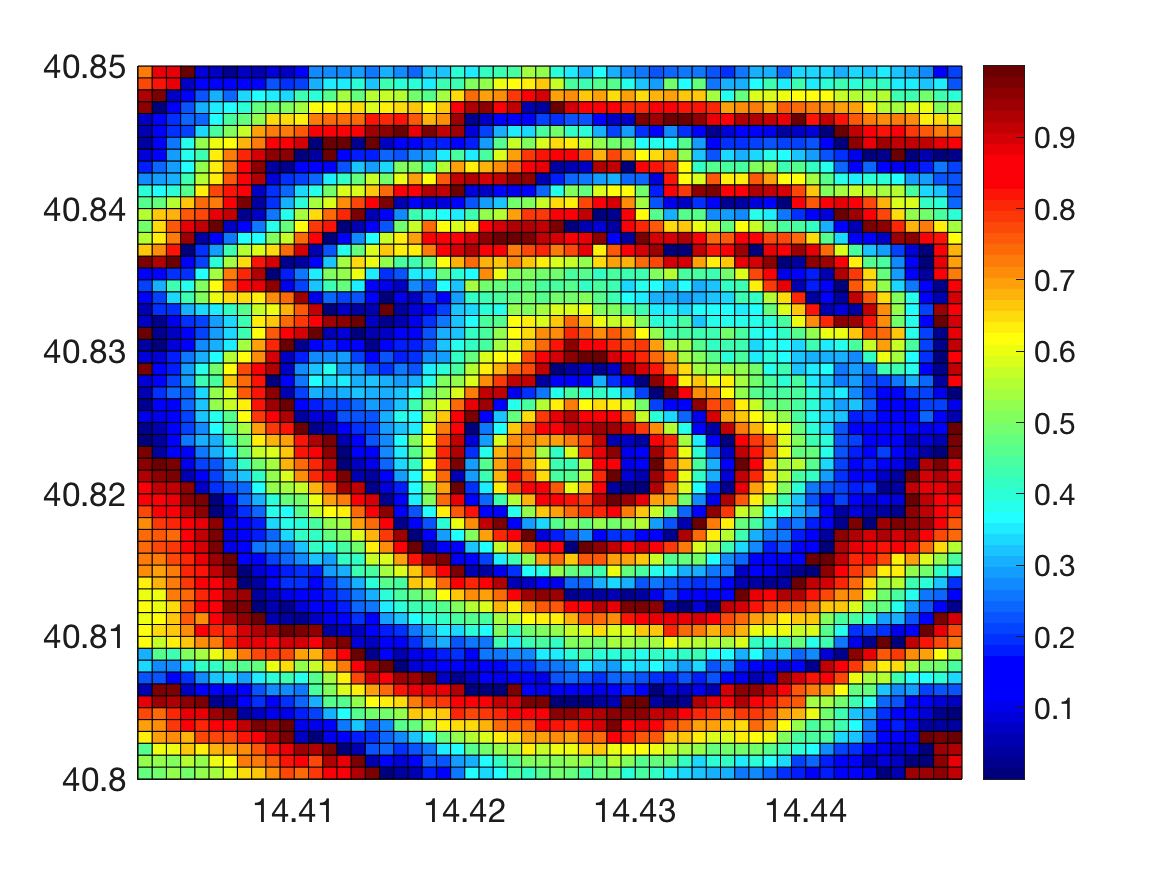} }
%
\subcaptionbox[]{   Denoised $f$   (RMSE = $ 0.114 $)
}[ 0.24\textwidth ]
{\includegraphics[width=0.24\textwidth] {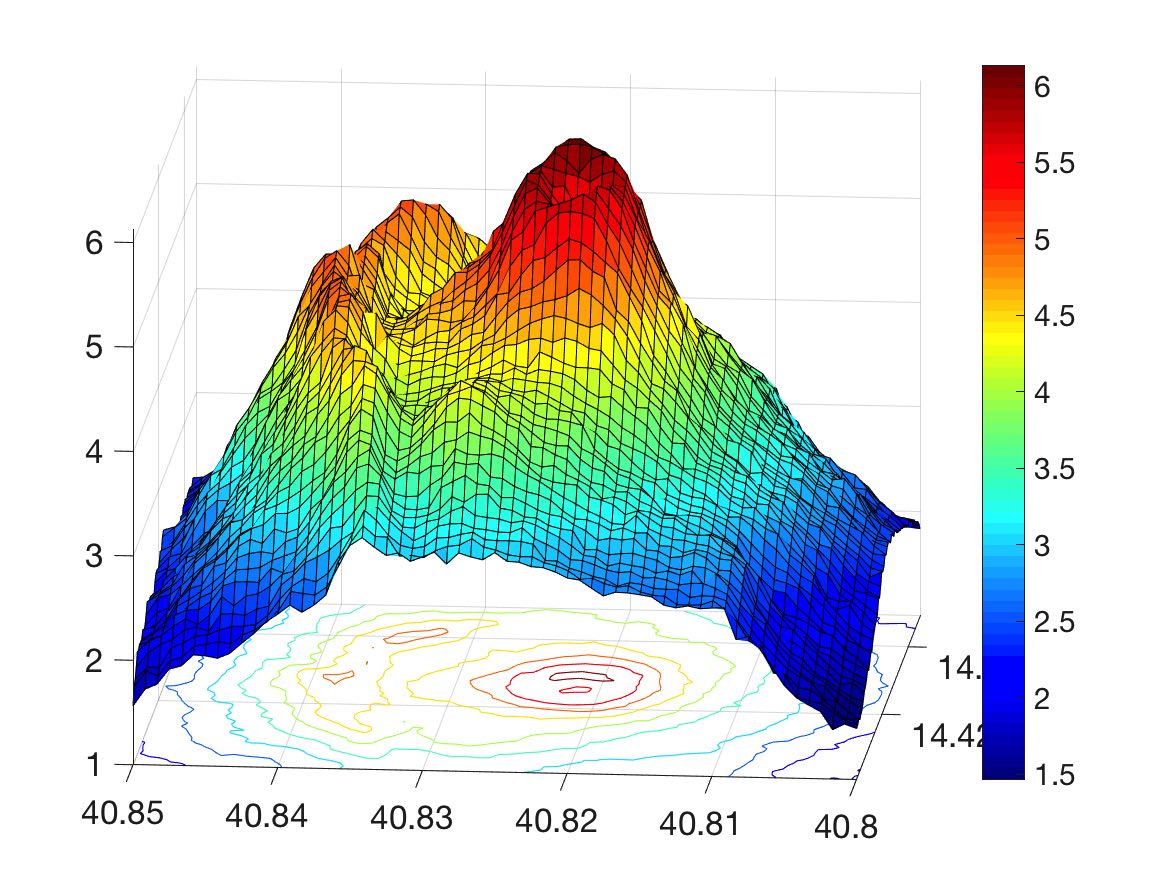} }
%
%
\subcaptionbox[]{   Denoised $f$ (depth)
}[ 0.24\textwidth ]
{\includegraphics[width=0.24\textwidth] {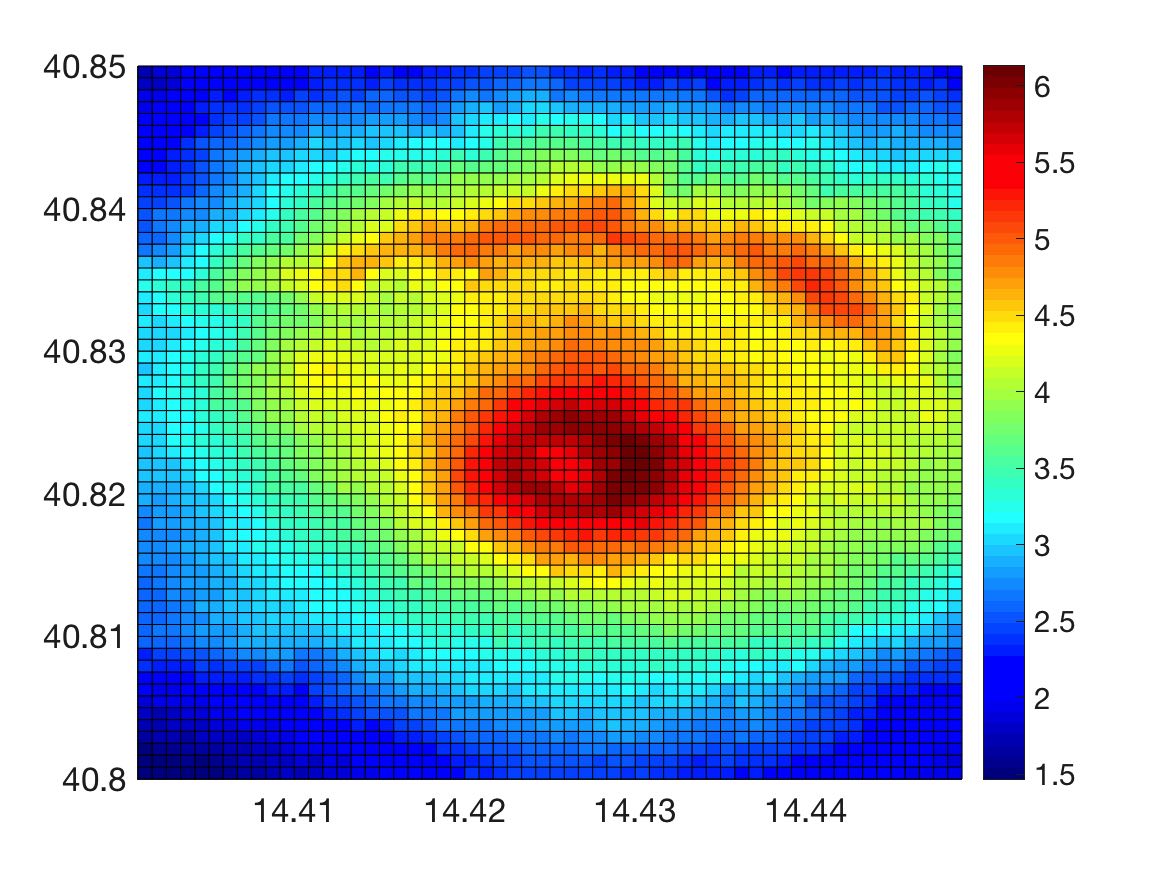} }
%
%
\subcaptionbox[]{  Error  $|f - \hat{f}|$ 
}[ 0.24\textwidth ]
{\includegraphics[width=0.24\textwidth] {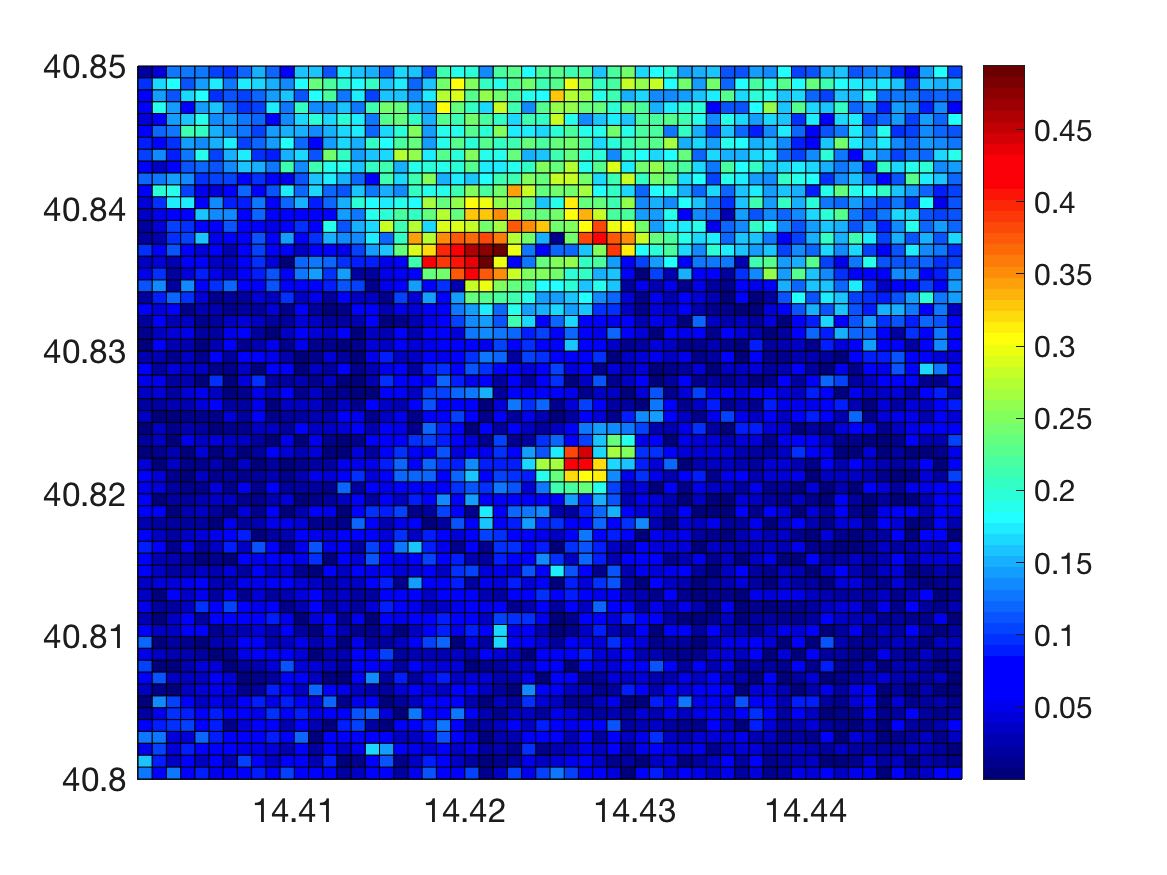} }
\captionsetup{width=0.95\linewidth}
\caption[Short Caption]{Elevation map of Mount Vesuvius with $n = 3600$, $k=1$ (Chebychev distance), $\lambda = 0.1$,  and noise level $\sigma=0.05$ under the Gaussian model, as recovered by Manopt-Phases.} 
\label{fig:Vesuvius_LOW_Noisy}
\end{figure*}

\ifthenelse{\boolean{ShowNoiselessMultivar}}{  
  
 \begin{figure*}
\centering
  \subcaptionbox[]{  Noiseless function
}[ 0.24\textwidth ]
{\includegraphics[width=0.24\textwidth] {figs/PLOTS_MULTIVARIATE/Vesuvius_HIGH/Gaussian_sigma_0_n_32399_kDist_1_lambda_0p03_scale_0_fClean_3d.png} }
%
\subcaptionbox[]{  Noiseless function (depth)
}[ 0.24\textwidth ]
{\includegraphics[width=0.24\textwidth] {figs/PLOTS_MULTIVARIATE/Vesuvius_HIGH/Gaussian_sigma_0_n_32399_kDist_1_lambda_0p03_scale_0_fClean_TOP.png} }
%
\subcaptionbox[]{   Clean $f$ mod 1
}[ 0.24\textwidth ]
{\includegraphics[width=0.24\textwidth] {figs/PLOTS_MULTIVARIATE/Vesuvius_HIGH/Gaussian_sigma_0_n_32399_kDist_1_lambda_0p03_scale_0_f_mod1_clean_TOP.png} }

\subcaptionbox[]{    Noisy $f$ mod 1
}[ 0.24\textwidth ]
{\includegraphics[width=0.24\textwidth] {figs/PLOTS_MULTIVARIATE/Vesuvius_HIGH/Gaussian_sigma_0_n_32399_kDist_1_lambda_0p03_scale_0_f_mod1_noise_TOP.png} }
%
\subcaptionbox[]{      Denoised $f$ mod 1 
}[ 0.24\textwidth ]
{\includegraphics[width=0.24\textwidth] {figs/PLOTS_MULTIVARIATE/Vesuvius_HIGH/Gaussian_sigma_0_n_32399_kDist_1_lambda_0p03_scale_0_f_mod1_denoised_TOP.png} }
%
\subcaptionbox[]{   Denoised $f$ 
}[ 0.24\textwidth ]
{\includegraphics[width=0.24\textwidth] {figs/PLOTS_MULTIVARIATE/Vesuvius_HIGH/Gaussian_sigma_0_n_32399_kDist_1_lambda_0p03_scale_0_f_denoised_3d.png} }
%
%
\subcaptionbox[]{   Denoised $f$ (depth)
}[ 0.24\textwidth ]
{\includegraphics[width=0.24\textwidth] {figs/PLOTS_MULTIVARIATE/Vesuvius_HIGH/Gaussian_sigma_0_n_32399_kDist_1_lambda_0p03_scale_0_f_denoised_TOP.png} }
%
%
\subcaptionbox[]{  Error  $|f - \hat{f}|$ 
}[ 0.24\textwidth ]
{\includegraphics[width=0.24\textwidth] {figs/PLOTS_MULTIVARIATE/Vesuvius_HIGH/Gaussian_sigma_0_n_32399_kDist_1_lambda_0p03_scale_0_f_delta_TOP.png} }
\captionsetup{width=0.95\linewidth}
\caption[Short Caption]{Vesuvius $n = 32400$, $k=1$ (Chebychev distance), $\lambda = 0.03$,  and noise level $\sigma=0$.}
\label{fig:Vesuvius_HIGH_Clean}
\end{figure*}
}{}

\ifthenelse{\boolean{ShowNoiselessMultivar}}{  

\begin{figure*}
\centering
 
\subcaptionbox[]{  Noiseless function
}[ 0.24\textwidth ]
{\includegraphics[width=0.24\textwidth] {figs/PLOTS_MULTIVARIATE/ETNA_LOW/Gaussian_sigma_0_n_4116_kDist_1_lambda_0p03_scale_0_fClean_3d.png} }
%
\subcaptionbox[]{  Noiseless function (depth)
}[ 0.24\textwidth ]
{\includegraphics[width=0.24\textwidth] {figs/PLOTS_MULTIVARIATE/ETNA_LOW/Gaussian_sigma_0_n_4116_kDist_1_lambda_0p03_scale_0_fClean_TOP.png} }
%
\subcaptionbox[]{   Clean $f$ mod 1
}[ 0.24\textwidth ]
{\includegraphics[width=0.24\textwidth] {figs/PLOTS_MULTIVARIATE/ETNA_LOW/Gaussian_sigma_0_n_4116_kDist_1_lambda_0p03_scale_0_f_mod1_clean_TOP.png} }

\subcaptionbox[]{    Noisy $f$ mod 1
}[ 0.24\textwidth ]
{\includegraphics[width=0.24\textwidth] {figs/PLOTS_MULTIVARIATE/ETNA_LOW/Gaussian_sigma_0_n_4116_kDist_1_lambda_0p03_scale_0_f_mod1_noise_TOP.png} }
%
\subcaptionbox[]{      Denoised $f$ mod 1
}[ 0.32\textwidth ]
{\includegraphics[width=0.24\textwidth] {figs/PLOTS_MULTIVARIATE/ETNA_LOW/Gaussian_sigma_0_n_4116_kDist_1_lambda_0p03_scale_0_f_mod1_denoised_TOP.png} }
%
\subcaptionbox[]{   Denoised $f$
}[ 0.24\textwidth ]
{\includegraphics[width=0.24\textwidth] {figs/PLOTS_MULTIVARIATE/ETNA_LOW/Gaussian_sigma_0_n_4116_kDist_1_lambda_0p03_scale_0_f_denoised_3d.png} }
%
%
\subcaptionbox[]{   Denoised $f$ (depth)
}[ 0.24\textwidth ]
{\includegraphics[width=0.24\textwidth] {figs/PLOTS_MULTIVARIATE/ETNA_LOW/Gaussian_sigma_0_n_4116_kDist_1_lambda_0p03_scale_0_f_denoised_TOP.png} }
%
%
\subcaptionbox[]{  Error  $|f - \hat{f}|$ 
}[ 0.24\textwidth ]
{\includegraphics[width=0.24\textwidth] {figs/PLOTS_MULTIVARIATE/ETNA_LOW/Gaussian_sigma_0_n_4116_kDist_1_lambda_0p03_scale_0_f_delta_TOP.png} }
\captionsetup{width=0.95\linewidth}
\caption[Short Caption]{Etna $n = 4100$, $k=1$ (Chebychev distance), $\lambda = 0.03$,  and noise level $\sigma=0$.}
\label{fig:Etna_Low_Clean}
\end{figure*}
}{}

\ifthenelse{\boolean{ShowNoiselessMultivar}}{  

\begin{figure*}
\centering
 
\subcaptionbox[]{  Noiseless function
}[ 0.24\textwidth ]
{\includegraphics[width=0.24\textwidth] {figs/PLOTS_MULTIVARIATE/ETNA_HIGH/Gaussian_sigma_0_n_18876_kDist_1_lambda_0p001_scale_0_fClean_3d.png} }
%
\subcaptionbox[]{  Noiseless function (depth)
}[ 0.24\textwidth ]
{\includegraphics[width=0.24\textwidth] {figs/PLOTS_MULTIVARIATE/ETNA_HIGH/Gaussian_sigma_0_n_18876_kDist_1_lambda_0p001_scale_0_fClean_TOP.png} }
%
\subcaptionbox[]{   Clean $f$ mod 1
}[ 0.24\textwidth ]
{\includegraphics[width=0.24\textwidth] {figs/PLOTS_MULTIVARIATE/ETNA_HIGH/Gaussian_sigma_0_n_18876_kDist_1_lambda_0p001_scale_0_f_mod1_clean_TOP.png} }

\subcaptionbox[]{    Noisy $f$ mod 1
}[ 0.24\textwidth ]
{\includegraphics[width=0.24\textwidth] {figs/PLOTS_MULTIVARIATE/ETNA_HIGH/Gaussian_sigma_0_n_18876_kDist_1_lambda_0p001_scale_0_f_mod1_noise_TOP.png} }
%
\subcaptionbox[]{      Denoised $f$ mod 1
}[ 0.24\textwidth ]
{\includegraphics[width=0.24\textwidth] {figs/PLOTS_MULTIVARIATE/ETNA_HIGH/Gaussian_sigma_0_n_18876_kDist_1_lambda_0p001_scale_0_f_mod1_denoised_TOP.png} }
%
\subcaptionbox[]{   Denoised $f$
}[ 0.24\textwidth ]
{\includegraphics[width=0.24\textwidth] {figs/PLOTS_MULTIVARIATE/ETNA_HIGH/Gaussian_sigma_0_n_18876_kDist_1_lambda_0p001_scale_0_f_denoised_3d.png} }
%
%
\subcaptionbox[]{   Denoised $f$ (depth)
}[ 0.24\textwidth ]
{\includegraphics[width=0.24\textwidth] {figs/PLOTS_MULTIVARIATE/ETNA_HIGH/Gaussian_sigma_0_n_18876_kDist_1_lambda_0p001_scale_0_f_denoised_TOP.png} }
%
%
\subcaptionbox[]{  Error  $|f - \hat{f}|$ 
}[ 0.24\textwidth ]
{\includegraphics[width=0.24\textwidth] {figs/PLOTS_MULTIVARIATE/ETNA_HIGH/Gaussian_sigma_0_n_18876_kDist_1_lambda_0p001_scale_0_f_delta_TOP.png} }
\captionsetup{width=0.95\linewidth}
\caption[Short Caption]{Etna $n = 19000$, $k=1$ (Chebychev distance), $\lambda = 0.001$,  and noise level $\sigma=0$, as recovered by Manopt-Phases.}
\label{fig:ETNA_HIGH_Clean}
\end{figure*}
}{}

\begin{figure*}
\centering
 
\subcaptionbox[]{  Noiseless function
}[ 0.24\textwidth ]
{\includegraphics[width=0.24\textwidth] {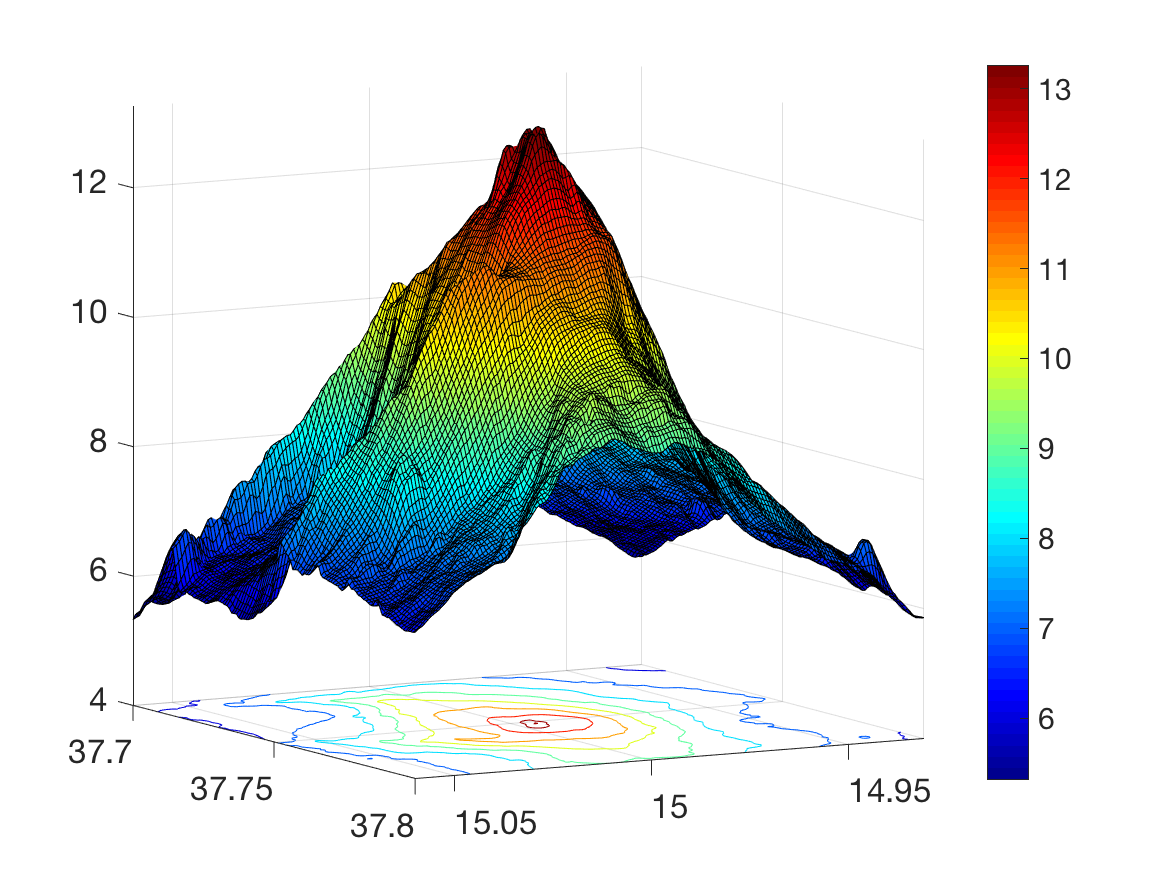} }
%
\subcaptionbox[]{  Noiseless function (depth)
}[ 0.24\textwidth ]
{\includegraphics[width=0.24\textwidth] {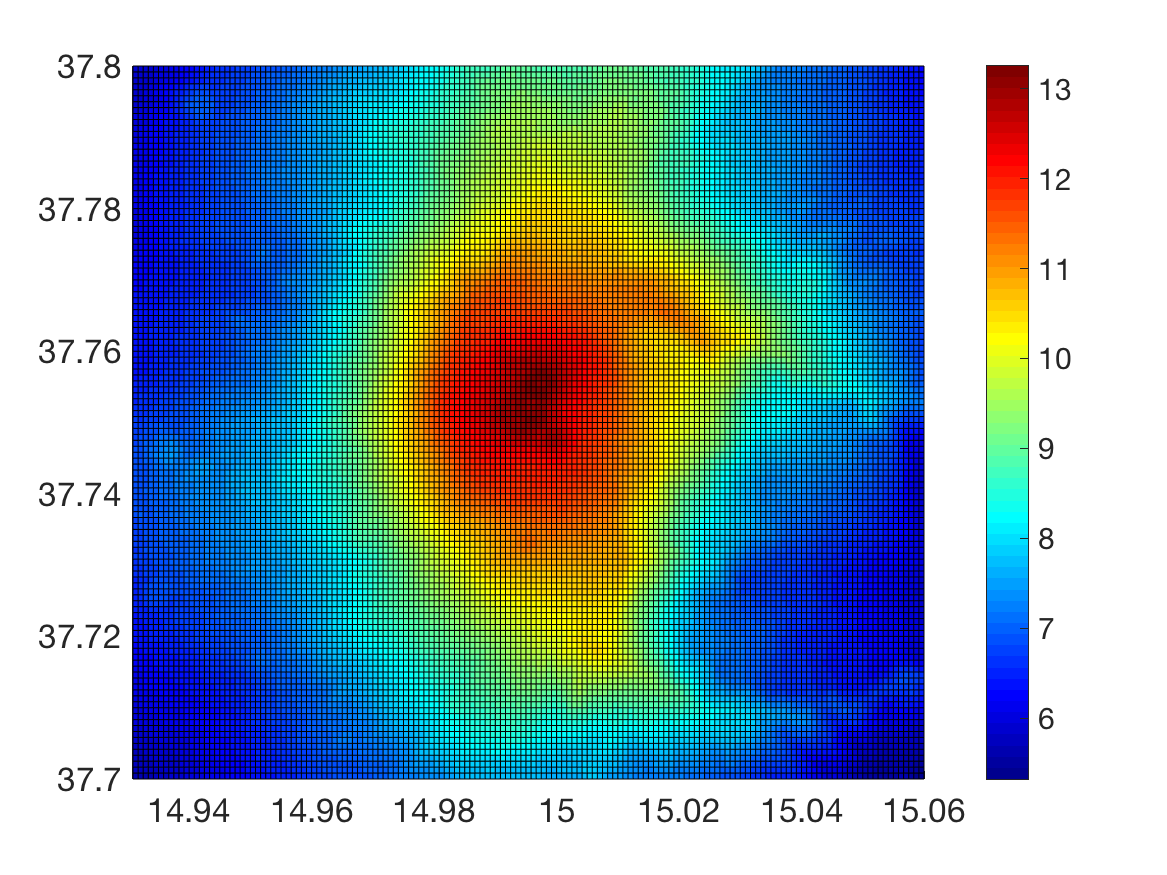} }
%
\subcaptionbox[]{   Clean $f$ mod 1
}[ 0.24\textwidth ]
{\includegraphics[width=0.24\textwidth] {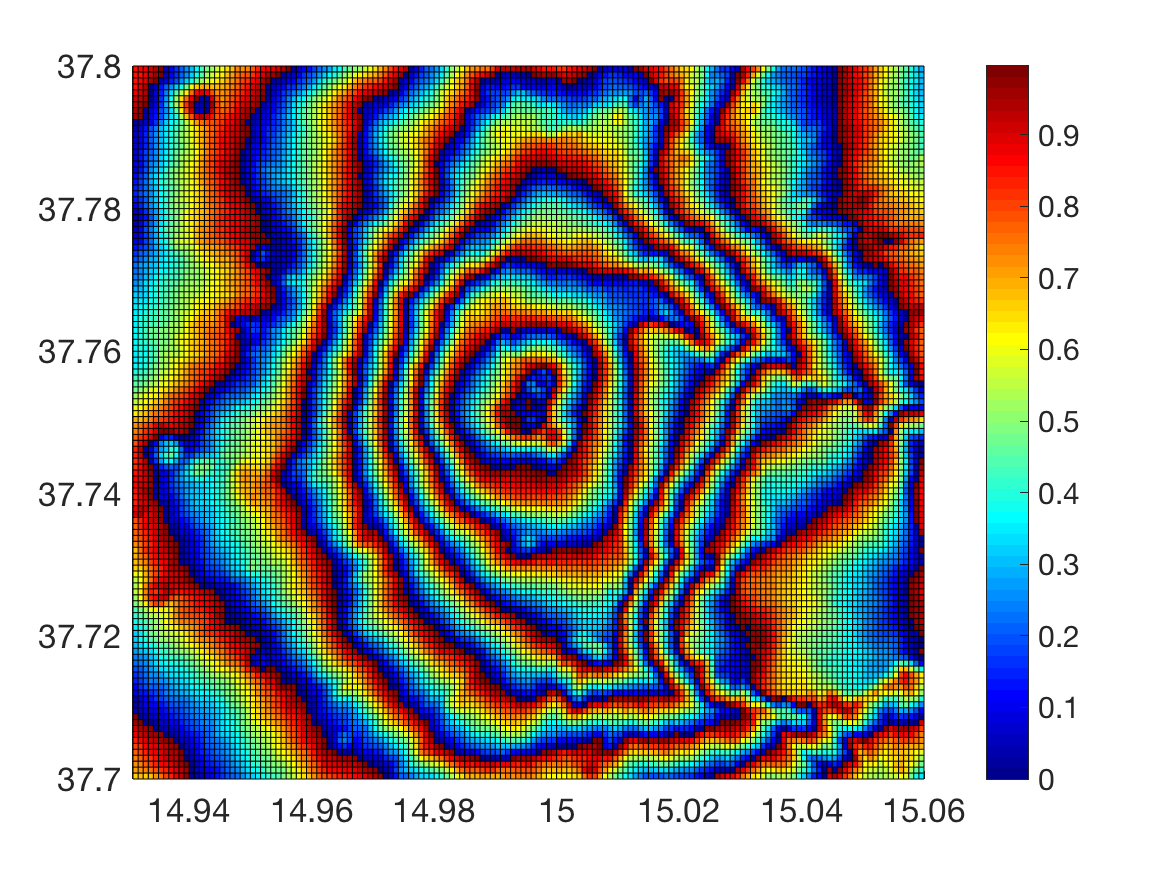} }
\subcaptionbox[]{    Noisy $f$ mod 1
}[ 0.24\textwidth ]
{\includegraphics[width=0.24\textwidth] {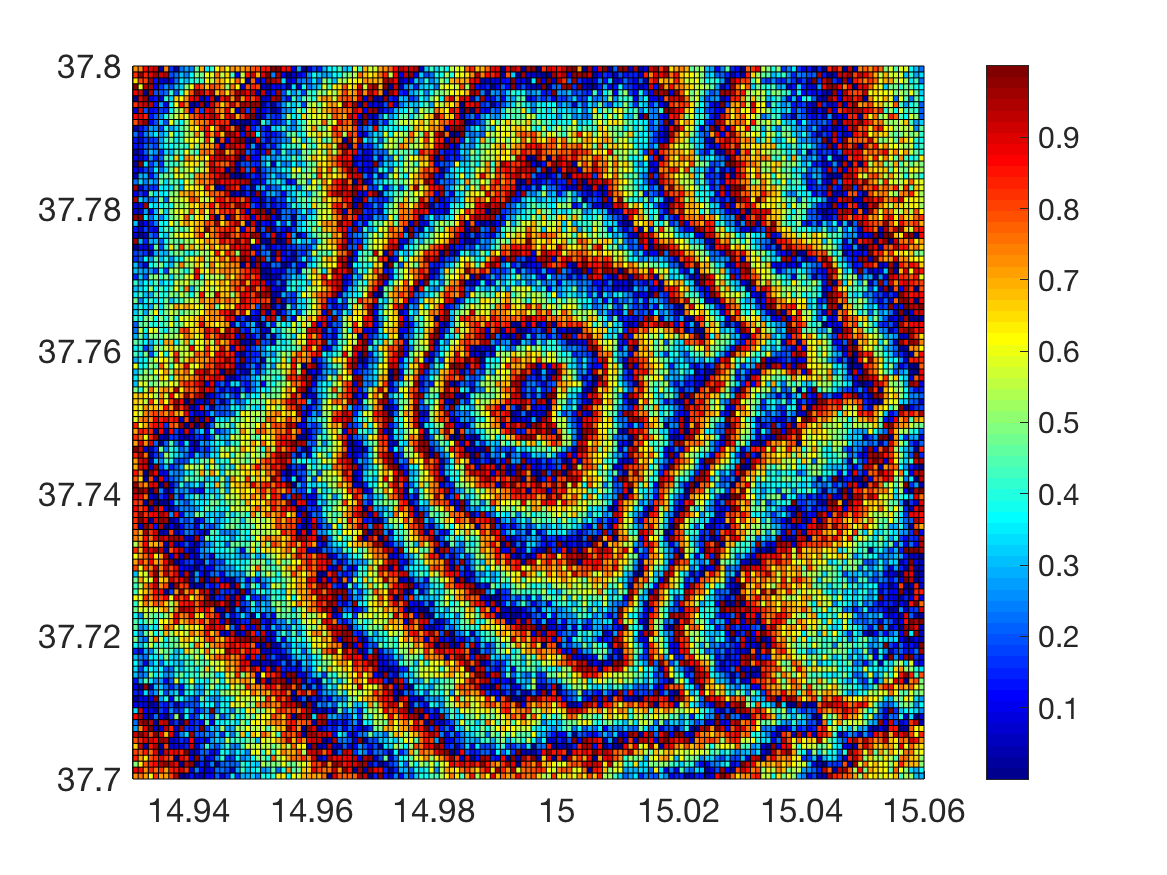} }

\subcaptionbox[]{      Denoised $f$ mod 1 (RMSE = $0.183$)
}[ 0.24\textwidth ]
{\includegraphics[width=0.24\textwidth] {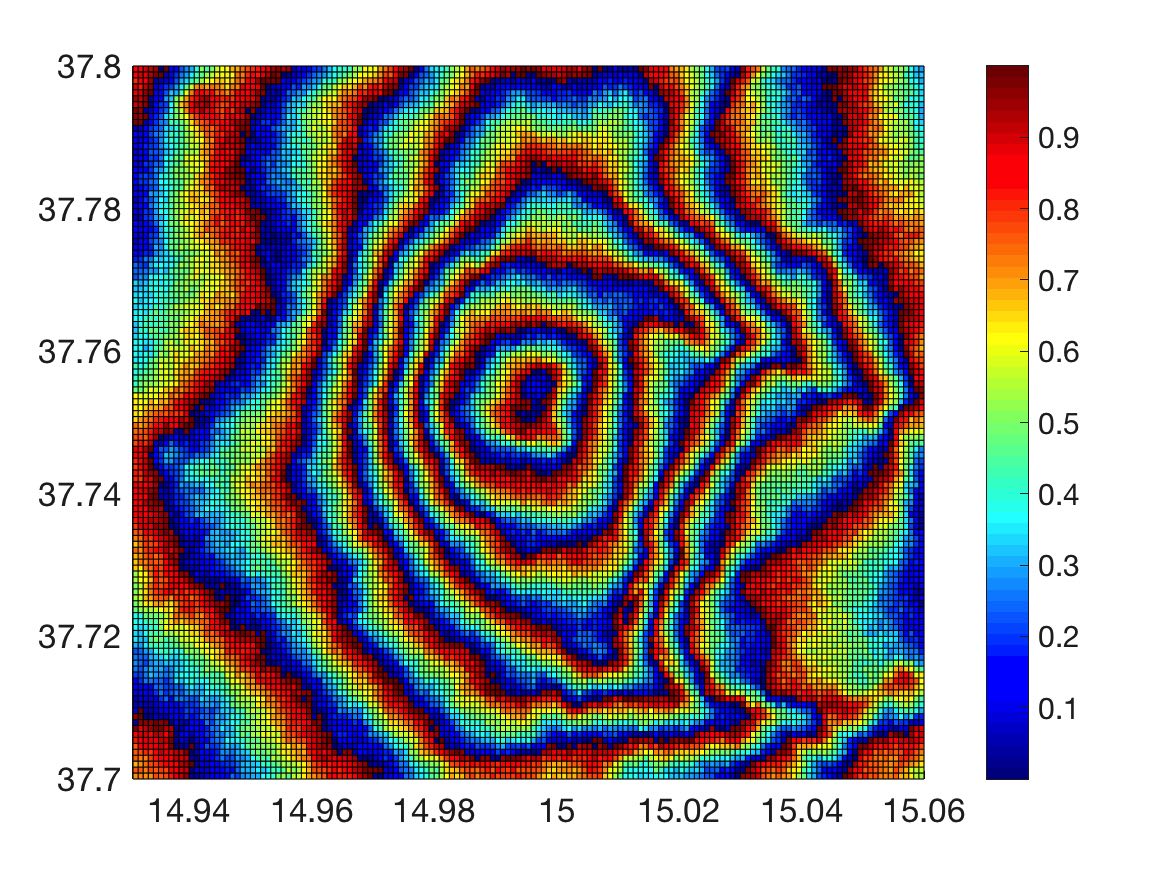} }
%
\subcaptionbox[]{   Denoised $f$ (RMSE=$ 0.073 $)
}[ 0.24\textwidth ]
{\includegraphics[width=0.24\textwidth] {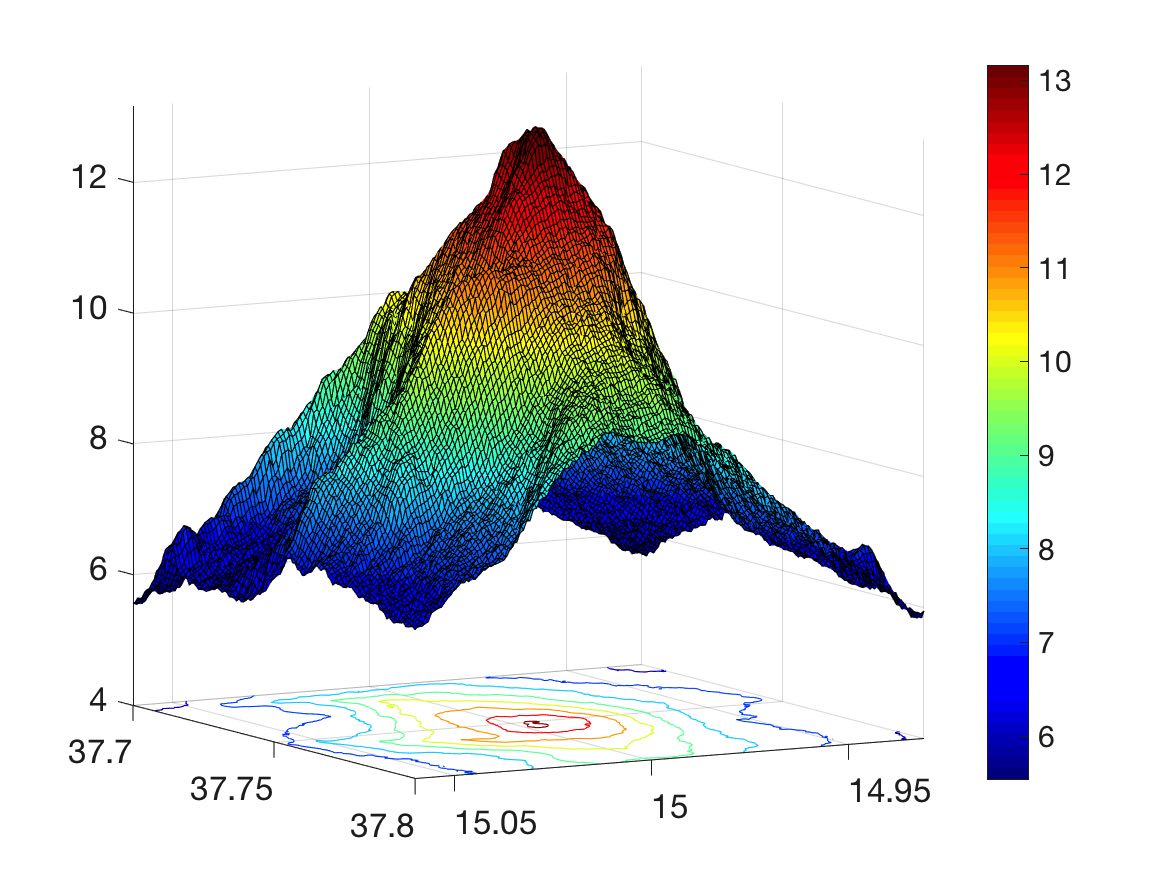} }
%
%
\subcaptionbox[]{   Denoised $f$ (depth)
}[ 0.24\textwidth ]
{\includegraphics[width=0.24\textwidth] {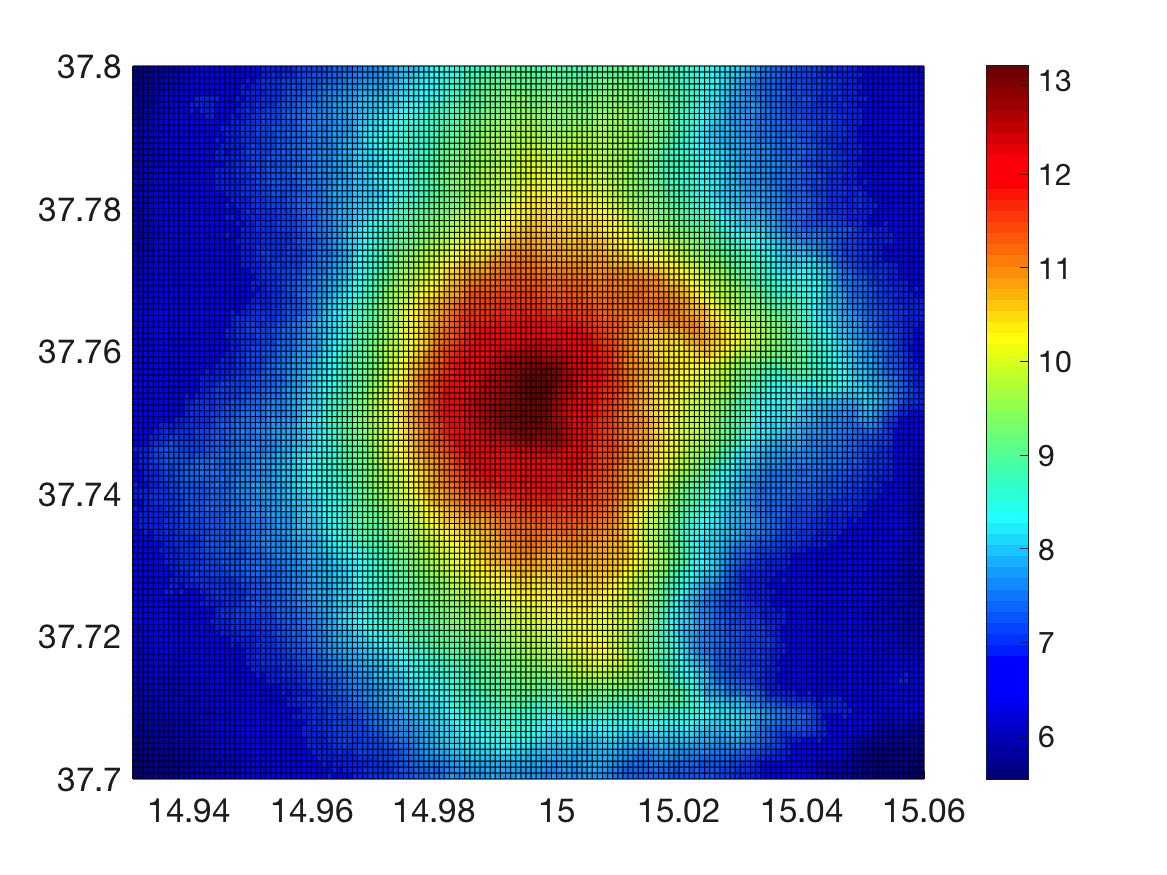} }
%
%
\subcaptionbox[]{  Error  $|f - \hat{f}|$ 
}[ 0.24\textwidth ]
{\includegraphics[width=0.24\textwidth] {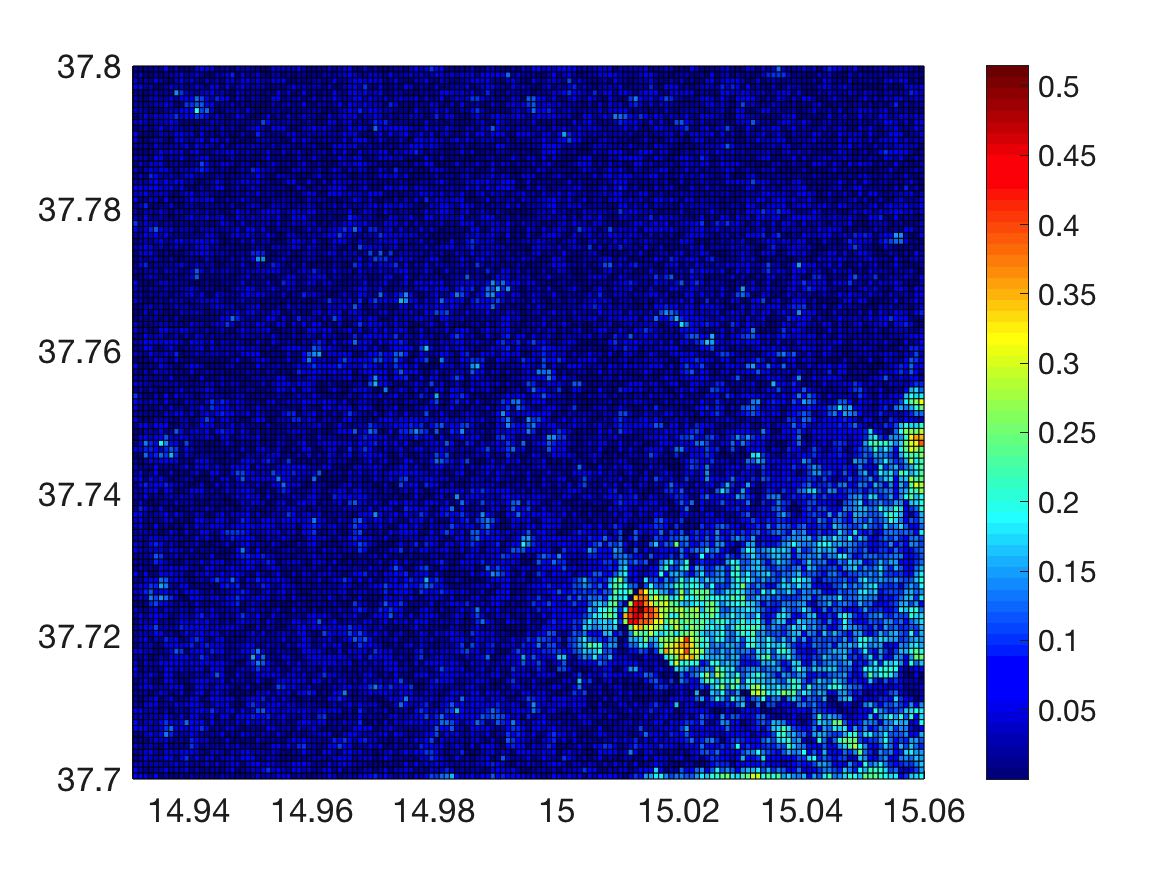} }
\captionsetup{width=0.95\linewidth}
\caption[Short Caption]{  Elevation map of Mount Etna  with  $n = 19000$, $k=1$ (Chebychev distance), $\lambda = 0.3$,  and noise level $\sigma=0.10$ under the Gaussian model, as recovered by Manopt-Phases.}
\label{fig:ETNA_HIGH_Noisy}
\end{figure*}
 


\end{document}